\documentclass[10pt,twoside,openright]{report}

\usepackage[phd]{edmaths}

\usepackage{amsmath,amsfonts,amsthm,amssymb}
\usepackage{graphicx,color}
\usepackage{algpseudocode,algorithm}
\usepackage{mathtools}
\usepackage{multirow}
\usepackage{array}
\usepackage{url,hyperref}
\usepackage{epsfig}
\usepackage{graphicx}
\usepackage{epstopdf}
\usepackage{subcaption}
\usepackage{xspace}
\usepackage{colortbl}
\usepackage{booktabs}
\usepackage[shortlabels]{enumitem}

\DeclareMathOperator*{\argmin}{arg\,min}

\DeclareMathOperator{\prox}{prox}         
\DeclareMathOperator{\dom}{dom}
\DeclareMathOperator{\support}{support}

\newcolumntype{C}[1]{>{\centering\let\newline\\\arraybackslash\hspace{0pt}}m{#1}}

\newcommand{\eqdef}{\overset{\text{def}}{=}}
\newcommand\numberthis{\addtocounter{equation}{1}\tag{\theequation}}

\newcommand{\R}{\mathbb{R}}
\newcommand{\E}[1]{\mathbb{E}\left[ #1 \right]}
\newcommand{\Ei}[1]{\mathbb{E}_i\left[ #1 \right]}
\newcommand{\Prob}[1]{\mathbb{P}\left[ #1 \right]}
\newcommand{\ip}[2]{\left\langle #1, #2 \right\rangle}

\def\<#1,#2>{\left\langle #1, #2 \right\rangle}

\newcommand{\grad}{ g} 
\newcommand{\ivar}{ y} 
\newcommand{\ovar}{ w} 
\newcommand{\iidx}{ t} 
\newcommand{\oidx}{ k} 
\newcommand{\rsample}{ i} 
\newcommand{\csample}{ j} 
\newcommand{\weightinG}{q} 
\newcommand{\weightofnorm}{v}
\newcommand{\hL}{ \hat{L}}

\newcommand{\cocoa}{\textsc{CoCoA}\xspace} 
\newcommand{\cocoap}{\textsc{CoCoA$\!^{\bf \textbf{\footnotesize+}}$}\xspace}
\newcommand{\vsubset}[2]{#1_{[#2]}}
\newcommand{\vc}[2]{#1^{#2}}

\newcommand{\vsub}[2]{#1_{[#2]}}
\newcommand{\hv}{{h}}
\newcommand{\sv}{{s}}
\newcommand{\uv}{{u}}
\newcommand{\vv}{{v}}
\newcommand{\wv}{{w}}
\newcommand{\xv}{{x}}
\newcommand{\alphav}{{\alpha}}
\newcommand{\bD}{D}

\newcommand{\Ggk}{\mathcal{G}^{\sigma'}_k\hspace{-0.08em}}
\newcommand{\gap}{{Gap}}
\newcommand{\aggpar}{\nu}
\newcommand{\hk}{\vsub{\hv}{k}}
\newcommand{\Xk}{\vsub{X}{k}}
\newcommand{\alphak}{\vsub{\alphav}{k}}

\newcommand{\Gk}{\mathcal{G}_k}
\newcommand{\Gks}{\mathcal{G}^{\sigma'}_k\hspace{-0.08em}}

\newcommand{\nBG}{G} 

\newcommand{\FedOpt}{Federated Optimization\xspace}
\newcommand{\fedopt}{federated optimization\xspace}
\newcommand{\SVRG}{SVRG\xspace}
\newcommand{\algname}{FSVRG\xspace}  
\newcommand{\node}{node\xspace} 
\newcommand{\nodes}{nodes\xspace}
\newcommand{\iid}{IID\xspace}
\newcommand{\A}{\mathcal{A}}
\newcommand{\BO}{\mathcal{O}}
\newcommand{\pp}{\mathcal{P}}

\newcommand{\StructLowRank}{\texttt{StructLowRank}\xspace}
\newcommand{\StructMask}{\texttt{StructMask}\xspace}
\newcommand{\SketchMask}{\texttt{SketchMask}\xspace}
\newcommand{\SketchRotMask}{\texttt{SketchRotMask}\xspace}
\newcommand{\SketchMaskBin}{\texttt{SketchMaskBin}\xspace}
\newcommand{\SketchRotMaskBin}{\texttt{SketchRotMaskBin}\xspace}
\newcommand{\bW}{{W}}
\newcommand{\bH}{{H}}
\newcommand{\bA}{{A}}
\newcommand{\bB}{{B}}

\newcommand{\EE}[2]{{\mathbb E}_{#1}\left[#2\right] } 
\newcommand{\VV}[2]{{\bf Var}_{#1}\left[#2\right] }

\newtheorem{theorem}{Theorem}
\newtheorem{lemma}[theorem]{Lemma}
\newtheorem{proposition}[theorem]{Proposition}
\newtheorem{corollary}[theorem]{Corollary}
\newtheorem{definition}[theorem]{Definition}
\newtheorem{assumption}[theorem]{Assumption}
\newtheorem{example}[theorem]{Example}
\newtheorem{remark}[theorem]{Remark}

\title{Stochastic, Distributed and Federated Optimization for Machine Learning}
\author{Jakub Kone\v{c}n\'{y}}
\date{2017}

\begin{document}

\maketitle

\declaration

\begin{abstract}
We study optimization algorithms for the finite sum problems frequently arising in machine learning applications. First, we propose novel variants of stochastic gradient descent with a variance reduction property that enables linear convergence for strongly convex objectives. Second, we study distributed setting, in which the data describing the optimization problem does not fit into a single computing node. In this case, traditional methods are inefficient, as the communication costs inherent in distributed optimization become the bottleneck. We propose a communication-efficient framework which iteratively forms local subproblems that can be solved with arbitrary local optimization algorithms. Finally, we introduce the concept of Federated Optimization/Learning, where we try to solve the machine learning problems without having data stored in any centralized manner. The main motivation comes from industry when handling user-generated data. The current prevalent practice is that companies collect vast amounts of user data and store them in datacenters. An alternative we propose is not to collect the data in first place, and instead occasionally use the computational power of users' devices to solve the very same optimization problems, while alleviating privacy concerns at the same time. In such setting, minimization of communication rounds is the primary goal, and we demonstrate that solving the optimization problems in such circumstances is conceptually tractable.
\end{abstract}

\chapter*{Acknowledgements}

I would like to express my sincere gratitude to my supervisor, Peter Richt\'{a}rik, for his guidance through every facet of the research world. I have not only learned how to develop novel research ideas, but also how to clearly write and communicate those ideas, prepare technical presentations and engage audience, interact with other academics and build research collaborations. All of this was crucial for all I have accomplished and lays a solid foundation for my next steps.

I would also like to thank my other supervisors and mentors Jacek Gondzio and Chris Williams for various discussions about research and differences between related fields. I am also grateful to my examination committee, M. Pawan Kumar, Kostas Zygalakis and Andreas Grothey. For stimulating discussions and fun we had together during last four years, I thank other past and current members of our research group: Dominik Csiba, Olivier Fercoq, Robert Gower, Filip Hanzely, Nicolas Loizou, Zheng Qu, Ademir Ribeiro and Rachael Tappenden.

I am indebted to School of Mathematics of the University of Edinburgh for the wonderful, inclusive and productive work environment, to Principal's Career Development Scholarship for initially funding my PhD and to The Centre for Numerical Analysis and Intelligent Software for funding my visit at the Simons Institute for the Theory of Computing. I am extremely grateful to Gill Law, who was always very helpful in overcoming every formal or administrative problems I encountered.

During my study I had the opportunity to work with many brilliant researchers:
\begin{itemize}[noitemsep]
\item I would like to thank Martin Jaggi and Thomas Hofmann for hosting my visit at ETH Zurich and encouragement in my starting career. 
\item I would like to thank Martin Tak\'{a}\v{c} for the visit at Lehigh University, and his students Chenxin Ma and Jie Liu. I would also like to thank Katya Scheinberg for her insightful advice regarding my PhD study.
\item I would also like to thank Ngai-Man Cheung, Selin Damla Ahipasaoglu and Yiren Zhou who taught me a lot during my visit at Singapore University of Technology and Design.
\item I would like to thank Owain Evans and others at Future of Humanity Institute at Oxford University for hosting my extremely inspiring visit which helped me to understand the implications of our work in a much broader context.
\item I would like to thank the Simons Institute for the Theory of Computing at University of California in Berkeley, for the opportunity I had at the start of my PhD study.
\item Finally, I would like to thank my other collaborators for their ideas, help, jokes and support: Mohamed Osama Ahmed, Dave Bacon, Dmitry Grishchenko, Michal Hagara, Filip Hanzely, Reza Harikandeh, Michael I. Jordan, Nicolas Loizou, H. Brendan McMahan, Barnab\'{a}s P\'{o}czos, Zheng Qu, Daniel Ramage, Sashank J. Reddi, Scott Sallinen, Mark Schmidt, Virginia Smith, Alex Smola, Ananda Theertha Suresh, Alim Virani and Felix X. Yu.
\end{itemize}

For generous support through Google Doctoral Fellowship I am very thankful to Google, which enabled me to pursue my goals without distractions. I appreciate the advice, patience, discussions and fun I had with many amazing people during my summer internships at Google, including Galen Andrew, Dave Bacon, Keith Bonawitz, Hubert Eichner, Jeffrey Falgout, Emily Fortuna, Gaurav Gite, Seth Hampson, Jeremy Kahn, Peter Kairouz, Eider Moore, Martin Pelikan, John Platt, Daniel Ramage, Marco Tulio Ribeiro, Negar Rostamzadeh, Subarna Tripathi, Felix X. Yu and many others. In particular, I would like to thank for the immerse trust and support I received from Brendan McMahan and Blaise Ag\"{u}era y Arcas.

For willingness to help and provide recommendation, reference, or connection at various stages of my study, I would like to thank Blaise Ag\"{u}era y Arcas, Petros Drineas, Martin Jaggi, Michael Mahoney, Mark Schmidt, Nathan Srebro and Lin Xiao. In addition, I am extremely thankful to Isabelle Guyon. I would perhaps not decide to go for PhD without her encouragement and support after we met during my undergraduate study.

I would like to thank the Slovak educational non-profit organizations which helped to shape who I am today, both on academic and personal level --- Trojsten, Sezam, P-Mat and Nexteria --- and all the people involved in their activities.

Finally, I would like to thank my parents and family, for all the obstacles they quietly removed so I could realize my dreams.

\tableofcontents

\chapter{Introduction}

In this thesis, we focus on minimization of a finite sum of functions, with particular motivation by machine learning applications:
\begin{equation}
\label{eq:intro:problem}
\min_{w \in \R^d} \frac1n \sum_{i=1}^n f_i(w).
\end{equation}
This problem is referred to as Empirical Risk Minimization (ERM) in the machine learning community, and represents optimization problems underpinning a large variety of models --- ranging from simple linear regression to deep learning.

This introductory chapter briefly outlines the theoretical framework that gives rise to the ERM problem in Section~\ref{sec:intro:erm}, and clarifies what part of the general objective in machine learning we address in this thesis. We follow by a summary of the thesis, highlighting the central contributions without going into the details.


\section{Empirical Risk Minimization}
\label{sec:intro:erm}

In a prototypical setting of supervised learning, one can access input-output pairs $x, y \in \mathcal{X} \times \mathcal{Y}$, which follow an \emph{unknown} probability distribution $\mathcal{P}(x, y)$. Typically, inputs and outputs are not known at the same time, and a general goal is to understand the conditional distribution of output given input, $\mathcal{P}(y | x)$. For instance, a bank needs to predict whether a transaction is fraudulent, without knowing the true answer immediately, or a recommendation engine predicts which products are users likely to be interested in next. For a more detailed introduction to the following concept, see for instance \cite{shalev2014understanding}.

The typical learning setting, relies on a definition of a loss function $\ell: \mathcal{Y} \times \mathcal{Y} \to \R$ and a predictor function $\phi: \mathcal{X} \to \mathcal{Y}$. Loss function $\ell(\hat y, y)$ measures the discrepancy between predicted output $\hat y$ and true output $y$, and $\phi(x)$ maps an input to the predicted output $\hat y$. With these tools, we are ready to define the \emph{expected risk} of a predictor function $\phi$ as
$$ \mathbf{E}(\phi) \eqdef \int \ell(\phi(x), y) d\mathcal{P}(x, y). $$
The ideal goal is to find $\phi^*(x)$ that minimizes the expected risk, defined pointwise as 
$$ \phi^*(x) \eqdef \argmin_{\hat y \in \mathcal{Y}} \int \ell(\hat y, y) d\mathcal{P}(y | x). $$

Clearly, since we do not assume to know anything about the source distribution $\mathcal{P}(x, y)$, finding $\phi^*$ is an infeasible objective.

Instead, we have access to samples from the distribution. Given a training dataset $\{ (x_i, y_i) \}_{i=1}^n$, which is assumed to be drawn iid from $\mathcal{P}(x, y)$, we can define the \emph{empirical risk} $\mathbf{E}_n (\phi)$ as a proxy to the expected risk; also known as monte-carlo integration:
\begin{equation}
\label{eq:intro:erm}
\mathbf{E}_n (\phi) \eqdef \frac{1}{n} \sum_{i=1}^n \ell(\phi(x_i), y_i).
\end{equation}
If we further restrict the predictor $\phi$ to belong to a specific class of functions, e.g., linear functions, minimizing the empirical risk becomes a tractable objective, motivating the optimization problem \eqref{eq:intro:problem}.

\subsection{Approximation-Estimation-Optimization tradeoff}
\label{sec:intro:tradeoff}

A first learning principle is to restrict the candidate prediction functions to a specific class $\mathcal{F}$. This, together with the choice of loss function $\ell$ effectively corresponds to the choice of a specific machine learning technique. As an example, these can be linear functions of $x$, parametrized by a vector $w$: $\phi_w(x) = w^T x$. A more complex example is a class implicitly defined by the architecture of a neural network.

In practice, an optimization algorithm is applied to obtain an approximate solution $\hat \phi_n$ to the ERM problem. In order to assess the quality of the predictor $\hat \phi_n$, compared to the ideal but intractable $\phi^*$, the standard in learning theory is to define the empirical risk minimizer as
$$ \phi_n \eqdef \argmin_{\phi \in \mathcal{F}} \mathbf{E}_n(\phi), $$
and the best predictor in terms of expected risk as 
$$ \phi_\mathcal{F}^* \eqdef \argmin_{\phi \in \mathcal{F}} \mathbf{E}(\phi). $$

Taking expectation with respect to generation of the sampled data, and possible randomization in an optimization algorithm for solving \eqref{eq:intro:erm}, goal of a machine learner is to minimize the \emph{excess error} $\E{\mathbf{E}(\hat \phi_n) - \mathbf{E}(\phi^*)}$ by choosing $\ell$, $\mathcal{F}$, and optimization algorithm appropriately, subject to constraints such as computational resources available. The excess error can be decomposed as follows:
\begin{equation}
\label{eq:intro:tradeoff}
\E{\mathbf{E}(\hat \phi_n) - \mathbf{E}(\phi^*)} = \E{\mathbf{E}(\phi_\mathcal{F}^*) - \mathbf{E}(\phi^*)} + \E{\mathbf{E}(\phi_n) - \mathbf{E}(\phi_\mathcal{F}^*)} +\E{\mathbf{E}(\hat \phi_n) - \mathbf{E}(\phi_n)}.
\end{equation}

The three terms above are referred to as \emph{approximation error}, \emph{estimation error} and \emph{optimization error}, respectively \cite{BottouBousquet}. Approximation error captures how much one loses by restricting the class of candidate predictor functions to $\mathcal{F}$. Estimation error captures the loss incurred by minimizing the empirical risk instead of the expected risk we would ideally optimize for. Optimization error is a result of finding an approximate optimum of the empirical risk, using an optimization algorithm.

These terms are subject to various tradeoffs that have been studied for decades. For instance, expanding the functional class $\mathcal{F}$ will naturally decrease the approximation error, but can increase the estimation error due to overfitting the training dataset. Increasing the size of the dataset available (increasing $n$) makes the empirical risk a better approximation of the expected risk, thus decreasing estimation error. However, it will likely make it computationally more expensive to attain the same optimization error.

Detailed overview of the interplay of these terms is beyond the scope of this work. We focus only on the optimization error, and what computational resources are necessary to obtain particular levels of the error. This can mean using optimization algorithms on a single compute node, with or without parallel processing units, or in a distributed environment. We describe this in the rest of this chapter. A comprehensive literature overview is deferred to Section~\ref{sec:relatedWork}, in which we explain why none of the existing methods are suitable for Federated Optimization, a novel conceptual setting for the ERM problem.

\subsection{Notation}

With focus on the optimization objective only, we can reformulate \eqref{eq:intro:erm} into notation used throughout the thesis. We are interested in minimizing a function $P(w)$, which in full generality takes the form
\begin{equation}
\label{eq:intro:primal}
P(w) = \frac{1}{n} \sum_{i=1}^n f_i(w) + R(w).
\end{equation}

The functions $f_i$ are assumed to be convex, and hide the dependence on the training data $x_i, y_i$, which is mostly irrelevant for the subsequent analysis. The (optional) function $R(w)$ is referred to as a regularizer, and is in practice used primarily to prevent overfitting or enforce structural properties of the solution. Most common choice are L2 ($R(w) = \lambda/2 \| w \|_2^2$) or L1 ($R(w) = \lambda \| w \|_1$) regularizers for some choice of $\lambda > 0$.

By $\nabla f_i(w)$ we denote the gradient of $f_i$ at point $w$. We denote $\ip{\cdot}{\cdot}$ the standard Euclidean inner product of two vectors, and unless specified otherwise, $\| \cdot \| = \sqrt{\ip{\cdot}{\cdot}}$ refers to the standard Euclidean norm. We denote the proximal operator of function $\psi : \R^d \to \R$ as $\prox_{\psi}(z) = \argmin_{s \in \R^d} \left\{ \frac 12  \| s - z \|^2 + \psi(s) \right\}.$ The convex (Fenchel) conjugate of a function $\phi : \R^d \rightarrow \R$ is defined as the function $\phi^* : \R^d \rightarrow \R \cup \{\infty\}$, with $\phi^*(u) = \sup_{s \in \R^d} \{ s^T u - \phi(s) \} $.

\section{Baseline Algorithms}

Two of the most basic algorithms that can be used to solve the ERM problem \eqref{eq:intro:primal} are Gradient Descent and Stochastic Gradient Descent, which we introduce now. For simplicity, we assume that $R(w) = 0$ for all $w \in \R^d$.

A trivial benchmark for solving \eqref{eq:intro:primal} is \emph{Gradient Descent} (GD) in the case when functions $f_i$ are smooth (or Subgradient Descent for non-smooth functions) \cite{nesterov2004convex}. The GD algorithm performs the iteration 
$$ w \gets w - \frac{h}{n} \sum_{i=1}^n \nabla f_i(w) = w - h \nabla P(w), $$ 
where $h>0$ is a  stepsize parameter.

Common practice in machine learning is to collect vast amounts of data $\{ x_i, y_i \}_{i=1}^n$, which in the context of our objective translates to very large $n$ --- the number of functions $f_i$. This makes GD impractical, as it needs to process the whole dataset in order to evaluate a single gradient and update the model. This makes GD rather impractical for most state-of-the-art applications. An alternative is to use a randomized algorithm, the computational complexity of which is independent of $n$, in a single iteration.

This basic, albeit in practice extremely popular, alternative to GD is \emph{Stochastic Gradient Descent} (SGD), dating back to the seminal work of Robbins and Monro \cite{RM1951}. In the context of~\eqref{eq:intro:primal}, SGD samples a random function $i \in \{1, 2, \dots, n\}$ in iteration $t$, and performs the update $$ w \gets w - h_t \nabla f_{i}(w), $$ where $h_t>0$ is a stepsize parameter. 

Intuitively speaking, this method works because if $i$ is sampled uniformly at random from indices $1$ to $n$, the update direction is an unbiased estimate of the gradient: \hbox{$\E{ \nabla f_i(w) } = \nabla P(w)$.} However, noise introduced by sampling slows down the convergence, and a diminishing sequence of stepsizes $h_t$ is necessary for the method to converge.

If we consider the case of strongly convex $P$, the core differences between GD and SGD can be summarized as follows. Let $\kappa$ denote the condition number defined as the ratio of smoothness and strong convexity parameters of $P$. GD enjoys fast convergence rate, while SGD converges slowly. That is, in order to obtain $\epsilon$-accuracy, GD needs $\mathcal{O}(\kappa \log(1/\epsilon))$ iterations, while SGD needs in general $\mathcal{O}(\kappa^2 / \epsilon)$ iterations. On the other hand, GD requires computation of $n$ gradients of $f_i$, which can be computationally expensive when data is abundant, while SGD needs to evaluate only a single gradient, and thus does not depend on $n$.

In most practical applications in machine learning, a high accuracy is not necessary, as the ERM problem is only a proxy to the original problem of interest, and error will eventually be dominated by approximation and estimation errors described in Section~\ref{sec:intro:tradeoff}. Indeed, SGD can sometimes yield a decent solution in just a single pass through data --- equivalent to a single GD step.

\section{Part I: Stochastic Methods with Variance Reduction}

In Chapters \ref{ch:s2gd} and \ref{ch:s2cd}, we propose and analyze semi-stochastic methods for minimizing the ERM objective \eqref{eq:intro:primal}. These methods interpolate between the baselines (GD and SGD) in the sense that they enjoy benefits of both methods. In particular, we show that using a trick to reduce the variance of stochastic gradients, we are able to maintain the linear convergence of GD, while using stochastic gradients. In order to do so, we still need to evaluate the full gradient $\nabla P$, but only a few times during the entire runtime of the method.

\subsection{Semi-Stochastic Gradient Descent}

The Semi-Stochastic Gradient Descent (S2GD) method, proposed in Chapter~\ref{ch:s2gd} (see Algorithm~\ref{alg:s2gd:s2gd}), runs in two nested loops. In the outer loop, it only computes and stores the full gradient of the objective, $\nabla P(w^t)$, the expensive operation one tries to avoid in general. In the inner loop, with some choice of stepsize $h$, the update step is iteratively computed as 
\begin{equation}
\label{eq:intro:s2gdupdate}
w \gets w - h [\nabla f_i(w) - \nabla f_i(w^t) + \nabla P(w^t)]
\end{equation}
for a randomly sampled $i \in \{1, \dots, n \}$. The core idea is that the gradients of $\nabla f_i$ are used to estimate the change of the full gradient $\nabla P$ between the points $w^t$ and $w$, as opposed to estimating the full gradient directly. It is easy to verify that if $i$ is sampled uniformly at random, the update direction is an unbiased estimate of the gradient $\nabla P(w)$.

We assume that $P$ is $\mu$-strongly convex, and the functions $f_i$ are $L$-smooth. Let $\kappa = L / \mu$ denote the condition number. In a core result, we are able to show that for the update direction form \eqref{eq:intro:s2gdupdate}, we have that
\begin{equation*}
\E{\| \nabla f_i(w) - \nabla f_i(w^t) + \nabla P(w^t) \|^2} \leq 4L [P(w) - P(w^*)] + 4(L- \mu)[P(w^t) - P(w^*)]
\end{equation*}
This shows that as both $w$ and $w^t$ progress towards the optimum $w^*$, the second moment --- and thus also variance --- of the estimate of the gradient diminishes. Together with unbiasedness, we use this to build a recursion which yields (see Theorem~\ref{thm:s2gd:MAIN}) that for iterates $w^t$ in the outer loop of the S2GD algorithm, we have
\begin{equation*}
\E{P(w^t) - P(w^*)} \leq c^t (P(w^0) - P(w^*)),
\end{equation*}
where $c$ is a convergence factor depending on the algorithm parameters and properties of the optimization problem.

Each iteration of S2GD requires evaluation of $\nabla P$ --- or $n$ stochastic gradients $\nabla f_i$, followed by a random number of stochastic updates. In Theorem~\ref{thm:s2gd:main2} we show that we can obtain an $\epsilon$-approximate solution after evaluating $\mathcal{O}((n + \kappa)\log(1/\epsilon))$ stochastic gradients. This is achieved by running the algorithm for $\log(1/\epsilon)$ iterations of the outer loop, with $\mathcal{O}(\kappa)$ stochastic updates \eqref{eq:intro:s2gdupdate} in the inner loop. Contrast this with the rate of GD, which per iteration requires the evaluation of $n$ stochastic gradients, and thus needs a total of $\mathcal{O}(n \kappa \log(1/\epsilon))$ gradient evaluations to attain the same accuracy. Given that $\kappa$ is commonly of the same order as $n$, which is typically very large. This amounts to an improvement by several orders of magnitude!

\subsection{Semi-Stochastic Coordinate Descent}

In Chapter~\ref{ch:s2cd} we present Semi-Stochastic Coordinate Descent (S2CD) method as Algorithm~\ref{alg:s2cd:S2CD}, which builds upon the S2GD algorithm by accessing oracle that returns partial stochastic derivatives $\nabla_j f_i$. In general, one can think of S2GD and similar stochastic methods as sampling rows of a data matrix. S2CD is sampling both rows and columns of the data matrix, in order to get computationally even cheaper stochastic iterations. Contrasted with S2GD, the outer loop stays the same, but stochastic steps in the inner loop update only a single coordinate of the variable $w$, and the update \eqref{eq:intro:s2gdupdate} changes to
$$ w \gets w - h p_j^{-1} \left( \frac{1}{n \weightinG_{ij}} \left( \nabla_{j} f_{i}(w) - \nabla_{j} f_{i}(w^t) \right) + \nabla_j P(w^t) \right) e_{\csample}, $$
where $p_j, q_{ij}$ are parameters of the algorithm determined by the problem structure, and $e_j$ is the $j^{th}$ unit vector in $\R^d$. As before, the update direction is an unbiased estimate of the gradient $\nabla P(w)$. However, the actual update has only one non-zero element.

We prove that the convergence of S2GD algorithm depends on a different notion of condition number (see Corollary~\ref{cor:s2cd:result}), which is always larger or equal to the one driving convergence of S2GD. However, the advantage is the usage of a weaker oracle, which only accesses partial derivatives. Whether S2CD is practically better than S2GD depends on the structure of a given problem, and whether it is possible to implement the oracle efficiently.

\section{Part II: Parallel and Distributed Methods}

In Part II, we do not focus on serial algorithms, but explore possibilities of using parallel and distributed computing architectures.

By parallel computation we mean utilization of multiple computing \nodes with a shared memory architecture, such as a multi-core processor. The main characteristic is that access to all data is equally fast for every computing \node. When we say we solve the ERM problem \eqref{eq:intro:primal} in a distributed setting, we mean that the amount of data describing the problem is too big to fit into a random access memory (RAM) or cannot even be stored on a single computing node. In both cases, the main difference to traditional, or serial, algorithms is that reading any data from a RAM can be several orders of magnitude faster than it is to send it to another \node in a network. This single fact presents a considerable challenge to iterative optimization algorithms that are inherently sequential, particularly to stochastic methods with fast iterations such as those described in Part I.

\subsection{Mini-batch Semi-Stochastic Gradient Descent in the Proximal Setting}

In Chapter~\ref{ch:ms2gd} we present parallel version of the S2GD algorithm, which we call mS2GD (see Algorithm~\ref{alg:ms2gd:mS2GD}), which improves upon S2GD algorithm in two major aspects. First, we allow and analyze the effect of mini-batching --- sampling multiple $f_i$ at the same time to obtain a more accurate stochastic gradient. This admits simple use of parallel computing architectures, as the computation of multiple stochastic gradients can be trivially parallelized. Second, the mS2GD algorithm is applicable to problem \eqref{eq:intro:primal} with general $R(w)$ that admits an efficient proximal operator. This includes non-smooth regularizers such as $R(w) = \| w \|_1$. We demonstrate the algorithm is useful also in the area of signal processing and imaging.

In Section~\ref{sec:ms2gd:speedup} we show that mini-batching alone can decrease the total amount of work necessary for convergence even if we were only to run it as a serial algorithm. More precisely, we show that up to a certain threshold on the mini-batch size (in typical circumstances about 30), the algorithm enjoys superlinear speedup in terms of the number of stochastic iterations needed. Additionally, in Section~\ref{sec:ms2gd:implementation_sparse}, we discuss an efficient implementation of the algorithm for problems with sparse data, which is significantly different and much more efficient than the intuitive straightforward implementation.

\subsection{Distributed Optimization with Arbitrary Local Solvers}

In the following, we review a paradigm for comparing efficiency of algorithms for distributed optimization, and describe what conceptual problem of these algorithms we address in Chapter~\ref{ch:cocoa}. 

Let us suppose we have many algorithms $\A$ readily available to solve problem~\eqref{eq:intro:primal}. The question is: ``How do we decide which algorithm is the best for our purpose?''

First, consider the basic setting on a single machine. Let us define $\mathcal{I}_\A(\epsilon)$ as the number of iterations algorithm $\A$ needs to converge to some fixed $\epsilon$ accuracy. Let $\mathcal{T}_\A$ be the time needed for a single iteration. Then, in practice, the best algorithm is one that minimizes the following quantity:\footnote{Considering only algorithms that can be run on a given machine.}
\begin{equation}
\label{eq:intro:paradigmBasic}
\text{TIME} = \mathcal{I}_\A(\epsilon) \times \mathcal{T}_\A.
\end{equation}

The number of iterations $\mathcal{I}_\A(\epsilon)$ is usually given by theoretical guarantees or observed from experience. The $\mathcal{T}_\A$ can be empirically observed, or one can have an idea of how the time needed per iteration varies between different algorithms in question. The main point of this simplified setting is to highlight a key issue with extending algorithms to the distributed setting.

The natural extension to distributed setting is the formula~\eqref{eq:intro:paradigm}. Let $c$ be the time needed for communication during a single iteration of the algorithm $\A$. For the sake of clarity, we suppose we consider only algorithms that need to communicate a single vector in $\R^d$ per round of communication. Note that essentially all first-order algorithms fall into this category, so it is not a restrictive assumption. This effectively sets $c$ to be a constant, given any particular distributed architecture one has at disposal.

\begin{equation}
\label{eq:intro:paradigm}
\text{TIME} = \mathcal{I}_\A(\epsilon) \times (c + \mathcal{T}_\A).
\end{equation}

The communication cost $c$ does not only consist of actual exchange of the data, but also several other protocols such as setting up and closing a connection between \nodes. Consequently, even if we need to communicate a very small amount of information, $c$ always remains above a nontrivial threshold.

Most, if not all, of the current state-of-the-art algorithms in setting~\eqref{eq:intro:primal} are stochastic and rely on doing very large number (big $\mathcal{I}_\A(\epsilon)$) of very fast (small $\mathcal{T}_\A$) iterations. Even a relatively small $c$ can cause the practical performance of their naively distributed variants drop down dramatically, because we still have $c \gg \mathcal{T}_\A$.

This has been indeed observed in practice, and motivated development of new methods, designed with this fact in mind from scratch, which we review in detail later in Section~\ref{sec:relatedWork:distributedAlgorithms}. Although this is a good development in academia --- motivation to explore a novel problem, it is not necessarily good news for the industry.

Many companies have spent significant resources to build excellent algorithms to tackle their problems of form~\eqref{eq:intro:primal}, fine tuned to the specific patterns arising in their data and side applications required. When the data companies collect grows too large to be processed on a single machine, it is understandable that they would be reluctant to throw away their fine tuned algorithms and start building new ones from scratch. 

We address this issue in Chapter~\ref{ch:cocoa} and propose the \cocoap framework, which works roughly as follows. The framework formulates a general way to form a specific local subproblem on each \node, based on the data available locally, and a single shared vector that needs to be distributed to all \nodes. Within an iteration of the framework, each \node uses \emph{any} optimization algorithm $\A$, to reach a relative $\Theta$ accuracy on the local subproblem. Updates from all \nodes are then aggregated to form an update to the global model.

The efficiency paradigm changes as follows:

\begin{equation}
\label{eq:intro:paradigmNew}
\text{TIME} = \mathcal{I}(\epsilon, \Theta) \times (c + \mathcal{T}_\A(\Theta)).
\end{equation}

Time per iteration $\mathcal{T}_\A(\Theta)$ denotes the time algorithm $\A$ needs to reach the relative $\Theta$ accuracy on the local subproblem. The number of iterations $\mathcal{I}(\epsilon, \Theta)$ is independent of the choice of the algorithm $\A$ used as a local solver. We provide a theoretical result, which specifies how many iterations of the \cocoap framework are needed to achieve overall $\epsilon$ accuracy, if we solve the local subproblems to relative $\Theta$ accuracy. Here, $\Theta = 0$ would mean we require the local subproblem to be solved to optimality, and $\Theta = 1$ that we do not need any progress whatsoever. The general upper bound on the number of iterations of the \cocoap framework is $\mathcal{I}(\epsilon, \Theta) = \frac{\mathcal{O}(\log(1/\epsilon))}{1 - \Theta}$ for strongly convex objectives (see Theorem~\ref{thm:convergenceSmoothCase}). From the inverse dependence on $1 - \Theta$ we can see that there is a fundamental limit to the number of communication rounds needed. Hence, intuitively speaking, it will probably not be efficient to spend excessive resources to attain very high local accuracy (small $\Theta$).

This efficiency paradigm is more powerful for a number of reasons.
\begin{enumerate}
\item It allows practitioners to continue using their fine-tuned solvers for solving subproblems within the \cocoap framework, that can run only on single machine, instead of having to implement completely new algorithms from scratch.

\item The actual performance in terms of the number of rounds of communication is independent from the choice of the optimization algorithm, making it much easier to optimize the overall performance.

\item Since the constant $c$ is architecture dependent, running optimal algorithm on one network does not have to be optimal on another. In the setting~\eqref{eq:intro:paradigm}, this could mean that when moving from one cluster to another, a completely different algorithm might be necessary for strong performance, which is a major change. In the setting~\eqref{eq:intro:paradigmNew}, this can be improved by simply changing $\Theta$, which will be implicitly determined by the number of iterations algorithm $\A$ runs for.
\end{enumerate}

Extensive experimental evaluation in Section~\ref{sec:cocoa:experiments} demonstrates the versatility of the proposed framework, which has already been implemented and adopted in the popular Apache Spark engine.

\section{Part III: Federated Optimization}

Mobile phones and tablets are now the primary computing devices for many people. In many cases, these devices are rarely separated from their owners \cite{CNNSmartphoneUsage}, and the combination of rich user interactions and powerful sensors means they have access to an unprecedented amount of data, much of it private in nature. Machine learning models learned on such data hold the promise of greatly improving usability by powering more intelligent applications. However, the sensitive nature of the data means there are risks and responsibilities related to storing it in a centralized location.

\subsection{Distributed Machine Learning for On-device Intelligence}

In Chapter~\ref{ch:feopt} we move beyond distributed optimization and advocate an alternative --- {\em federated learning} --- that leaves the training data distributed on the mobile devices, and learns a shared model by aggregating locally computed updates via a central coordinating server. This is a direct application of the principle of focused collection or data minimization proposed by the 2012 White House report on the privacy of consumer data \cite{whitehouse13privacy}. Since these updates are specific to improving the current model, they can be purely ephemeral --- there is no reason to store them on the server once they have been applied. Further, they will never contain more information than the raw training data (by the data processing inequality), and will generally contain much less. A principal advantage of this approach is the decoupling of model training from the need for direct access to the raw training data. Clearly, some trust of the server coordinating the training is still required, and depending on the details of the model and algorithm, the updates may still contain private information. However, for applications where the training objective can be specified on the basis of data available on each client, federated learning can significantly reduce privacy and security risks by limiting the attack surface to only the device, rather than the device and the cloud.

The main purpose of the chapter is to bring to the attention of the machine learning and optimization communities a new and increasingly practically relevant setting for distributed optimization, where none of the typical assumptions are satisfied, and communication efficiency is of utmost importance. In particular, algorithms for \fedopt must handle training data with the following characteristics:
\begin{itemize}
\item \textbf{Massively Distributed}: Data points are stored across a large number of \nodes $K$. In particular, the number of \nodes can be much bigger than the average number of training examples stored on a given \node ($n/K$).
\item \textbf{Non-\iid}: Data on each \node may be drawn from a different distribution; that is, the data points available locally are far from being a representative sample of the overall distribution.
\item \textbf{Unbalanced}: Different \nodes may vary by orders of magnitude in the number of training examples they hold.
\end{itemize}

In the work presented in Chapter~\ref{ch:feopt}, we are particularly concerned with \textbf{sparse} data, where some features occur on  a small subset of nodes or data points only. Although this is not a necessary characteristic of the setting of \fedopt, we will show that the sparsity structure can be used to develop an effective algorithm for \fedopt. Note that data arising in the largest machine learning problems being solved currently --- ad click-through rate predictions --- are extremely sparse.

We are particularly interested in the setting where training data lives on users' mobile devices (phones and tablets), and the data may be privacy sensitive. The data $\{x_i, y_i\}$ is generated through device usage, e.g., via interaction with apps. Examples include predicting the next word a user will type (language modeling for smarter keyboard apps), predicting which photos a user is most likely to share, or predicting which notifications are most important. 

To train such models using traditional distributed algorithms, one would collect the training examples in a centralized location (data center), where it could be shuffled and distributed evenly over proprietary compute nodes. We propose and study an alternative model: the training examples are not sent to a centralized location, potentially saving  significant network bandwidth and providing additional privacy protection. In exchange, users allow some use of their devices' computing power, which shall be used to train the model.

In the communication model we use, in each round we send an update
$\delta \in \R^d$ to a centralized server, where $d$ is the dimension of the model
being computed/improved. The update $\delta$ could be a gradient vector, for
example.  While it is certainly possible that in some applications the
$\delta$ may encode some private information of the user, it is likely
much less sensitive (and orders of magnitude smaller) than the
original data itself. For example, consider the case where the raw
training data is a large collection of video files on a mobile
device.  The size of the update $\delta$ will be \emph{independent} of
the size of this local training data corpus. 
We show that a global model can be trained using a small number of
communication rounds, and so this also reduces the network
bandwidth needed for training by orders of magnitude compared to
copying the data to the datacenter.

Communication constraints arise naturally in the massively distributed setting, as network connectivity may be limited (e.g., we may wish to deffer all communication until the mobile device is charging and connected to a wi-fi network).  Thus, in realistic scenarios we may be limited to only a single round of communication per day. This implies that, within reasonable bounds, we have access to essentially unlimited local computational power. Consequently, the practical objective is solely to minimize the number of  communication rounds.

The main purpose of the work is initiate research into, and design a first practical implementation of \fedopt. Our results suggest that with suitable optimization algorithms, very little is lost by not having an \iid sample of the data available, and that even in the presence of a large number of \nodes, we can still achieve convergence in relatively few rounds of communication. Recently, Google announced that they applied this concept in one of their applications used by over $500$ million users \cite{FederatedBlogpost}.

\subsection{Distributed Mean Estimation with Communication Constraints}

In Chapter~\ref{ch:mean} we theoretically address the problem of computing the average of vectors stored on different computing devices, while placing a constraint on the amount of bits communicated. This problem could become a bottleneck in practical application of federated optimization, when a server aggregates the updates $\delta$ from individual users due to in general asymmetric speed of internet connections \cite{speedtest}, or cryptographic protocols used to protect individual update \cite{bonawitz2016practical} that further increase the size of the data needed to be communicated back to server. 

We decompose the problem into a choice of encoding and communicating protocol, of which we propose several types. In the setting when we are allowed to communicate a single bit per element of vectors to be aggregated, we prove the best known bounds on the mean square error of the resulting average.

We apply some of these ideas in the context of federated optimization in Section~\ref{sec:meanFL}, in which we focus on training deep feed-forward models. We propose two major types of techniques to reduce the size of each update --- structured and sketching updates. With structured updates, we enforce the local update to be optimized for to be of a specific structure, such as low rank or sparse, which lets us succinctly represent the update using fewer parameters. By sketching updates, we mean the reduction of size of the update by sketching techniques, such as subsampling and quantization used jointly with random structured rotations. In the main contribution, we show we are able to train a deep convolutional model for the CIFAR-10 data, while in total communicating less bits than necessary to represent the original size of the data.

\section{Summary}

The content of this thesis is based on the following publications and preprints:

\begin{itemize}
\item Chapter~\ref{ch:s2gd}:
\emph{Jakub Kone\v{c}n\'{y} and Peter Richt\'{a}rik: ``Semi-stochastic gradient descent methods.'' arXiv preprint 1312.1666 (2013).} \cite{S2GD}
\item Chapter~\ref{ch:s2cd}: 
\emph{Jakub Kone\v{c}n\'{y}, Zheng Qu and Peter Richt\'{a}rik: ``Semi-stochastic coordinate descent.'' Optimization Methods and Software, 1--13 (2017).} \cite{S2CD}
\item Chapter~\ref{ch:ms2gd}: 
\emph{Jakub Kone\v{c}n\'{y}, Jie Liu, Peter Richt\'{a}rik and Martin Tak\'{a}\v{c}: ``Mini-batch semi-stochastic gradient descent in the proximal setting.'' IEEE Journal of Selected Topics in Signal Processing 10(2), 242--255 (2016).} \cite{konecny2015mini}
\item Chapter~\ref{ch:cocoa}: 
\emph{Chenxin Ma, Jakub Kone\v{c}n\'{y}, Martin Jaggi, Virginia Smith, Michael I Jordan, Peter Richt{\'a}rik and Martin Tak{\'a}{\v{c}}: ``Distributed optimization with arbitrary local solvers.'' Optimization Methods and Software, 1--36 (2017).} \cite{ma2015distributed}
\item Chapter~\ref{ch:feopt}: 
\emph{Jakub Kone\v{c}n\'{y}, Brendan McMahan, Daniel Ramage and Peter Richt\'{a}rik: ``Federated optimization: distributed machine learning for on-device intelligence.'' arXiv preprint 1610.02527 (2016).} \cite{konecny2016federated}
\cite{Federated_learning2016}
\item Chapter~\ref{ch:mean}: 
\emph{Jakub Kone\v{c}n\'{y} and Peter Richt\'{a}rik: ``Randomized Distributed Mean Estimation: Accuracy vs Communication.'' arXiv preprint 1611.07555 (2016).} \cite{konecny2016randomized}
\item Section~\ref{sec:meanFL}: 
\emph{Jakub Kone\v{c}n\'{y}, Brendan McMahan, Felix Yu, Peter Richt\'{a}rik, Ananda Theertha Suresh and Dave Bacon: ``Federated learning: Strategies for improving communication efficiency.'' arXiv preprint 1610.05492 (2016).} 
\end{itemize}

During the course of my study, I also co-authored the following works which were not used in the formation of this thesis:

\begin{itemize}
\item \emph{Reza Harikandeh, Mohamed Osama Ahmed, Alim Virani, Mark Schmidt, Jakub Kone\v{c}n\'{y} and Scott Sallinen: ``Stop wasting my gradients: Practical SVRG.'' Advances in Neural Information Processing Systems 28, 2251--2259 (2015).} \cite{practicalSVRG}
\item \emph{Sashank J Reddi, Jakub Kone\v{c}n\'{y}, Peter Richt\'{a}rik, Barnab\'{a}s P\'{o}cz\'{o}s and Alex Smola ``AIDE: Fast and communication efficient distributed optimization.'' arXiv preprint \linebreak 1608.06879 (2016).} \cite{reddi2016aide}
\item \emph{Filip Hanzely, Jakub Kone\v{c}n\'{y}, Nicolas Loizou, Peter Richt\'{a}rik, Dmitry Grishchenko:. Privacy Preserving Randomized Gossip Algorithms. arXiv preprint arXiv:1706.07636. (2017)} \cite{hanzely2017privacy}
\item \emph{Jakub Kone\v{c}n\'{y} and Peter Richt\'{a}rik. ``Simple complexity analysis of simplified direct search.'' arXiv preprint 1410.0390 (2014).} \cite{konevcny2014simple}
\end{itemize}

In \cite{practicalSVRG}, we propose several practical improvements to the S2GD algorithm from Chapter~\ref{ch:s2gd}. In particular, we show that it is not necessary to compute a full gradient in the outer loop; instead, an inexact estimate is sufficient for the same convergence. Additionally, we prove that the algorithm is not only a superior \emph{optimization} algorithm, but is also a better \emph{learning} algorithm, in the sense of the approximation-estimation-optimization tradeoff outlined in Section~\ref{sec:intro:tradeoff}.

In \cite{reddi2016aide}, we propose a framework for distributed optimization in a similar spirit to the one presented in Chapter~\ref{ch:cocoa}, but one that works only with the primal problem. Accelerated Inexact DANE is the first distributed method for \eqref{eq:intro:primal} that nearly matches communication complexity lower bounds while being implementable using first-order oracle only. This work also makes a link to a distributed algorithm that we propose but do not analyze as Algorithm~\ref{alg:DS2GDnaive}, and indirectly provides its theoretical convergence guarantee.

In \cite{hanzely2017privacy}, we introduce and analyze techniques for preserving privacy of initial values in randomized algorithms for average consensus problem.

Finally, in \cite{konevcny2014simple} we simplify and unify complexity proof techniques for direct search --- a classical algorithm for derivative-free optimization.

\part{Variance Reduced Stochastic Methods}

\chapter{Semi-Stochastic Gradient Descent}
\label{ch:s2gd}

\section{Introduction}

Many problems in data science (e.g., machine learning, optimization and statistics) can be cast as loss minimization problems of the form
\begin{equation}
\label{eq:s2gd:main}
\min_{w \in \R^d} P(w),
\end{equation}
where
\begin{equation}
\label{eq:s2gd:main2} 
P(w) \eqdef \frac{1}{n} \sum_{i=1}^n f_i(w).
\end{equation}

Here $d$ typically denotes the number of features / coordinates, $n$ the number of data points, and $f_i(w)$ is the loss incurred on data point $i$. That is, we are seeking to find a predictor $w \in \R^d$ minimizing the average loss $P(w)$. In big data applications, $n$ is typically very large; in particular, $n \gg d$.  

Note that this formulation includes more typical formulation of $L2$-regularized objectives --- $P(w) = \frac{1}{n} \sum_{i=1}^n \tilde{f}_i(w) + \frac{\lambda}{2} \| w \|^2.$ We hide the regularizer into the function $f_i(w)$ for the sake of simplicity of resulting analysis.

\subsection{Motivation}

Let us now briefly review two basic approaches to solving problem \eqref{eq:s2gd:main}.
\begin{enumerate}
\item \emph{Gradient Descent.} Given $w^k \in \R^d$, the gradient descent (GD) method sets $$ w^{k+1} = w^k - h \nabla P(w^k), $$ where $h$ is a stepsize parameter and $\nabla P(w^k)$ is the gradient of $P$ at $w^k$. We will refer to $\nabla P(x)$ by the name \emph{full gradient}. In order to compute $\nabla P(w^k)$, we need to compute the gradients of $n$ functions. Since $n$ is big, it is prohibitive to do this at every iteration.

\item \emph{Stochastic Gradient Descent (SGD).} Unlike gradient descent,  stochastic gradient descent \cite{nemirovski2009robust, tongSGD}  instead picks a random $i$ (uniformly) and updates $$ w^{k+1} = w^k - h \nabla f_i(w^k). $$ Note that this strategy drastically reduces the amount of work that needs to be done in each iteration (by the factor of $n$). Since $$ \E{\nabla f_i(w^k)} = \nabla P(w^k), $$ we have an unbiased estimator of the full gradient. Hence, the gradients of the component functions $f_1, \dots, f_n$ will be referred to as \emph{stochastic gradients}.  A practical issue with SGD is that consecutive stochastic gradients may vary a lot or even point in opposite directions. This slows down the performance of SGD. On balance, however, SGD is preferable to GD in applications where low accuracy solutions are sufficient. In such cases usually only a small number of passes through the data (i.e., work equivalent to a small number of full gradient evaluations) are needed to find an acceptable $w$. For this reason, SGD is extremely popular in fields such as machine learning.

\end{enumerate}

In order to improve upon GD, one needs to reduce the cost of computing a gradient. In order to improve upon SGD, one has to reduce the variance of the stochastic gradients. In this chapter we propose and analyze a \emph{Semi-Stochastic Gradient Descent} (S2GD) method. Our  method combines GD and SGD steps and reaps the benefits of both algorithms: it inherits the stability and speed of GD and at the same time retains the work-efficiency of SGD.

\subsection{Brief literature review}

Several recent papers, e.g., \cite{RichtarikTakacIteration}, \cite{SAG,SAGjournal2013}, \cite{SDCA} and \cite{SVRG}  proposed methods which achieve similar variance-reduction effect, directly or indirectly. These methods enjoy linear convergence rates when applied to minimizing  smooth strongly convex loss functions.

The method in \cite{RichtarikTakacIteration} is known as Random Coordinate Descent for Composite functions (RCDC), and can be either applied directly to \eqref{eq:s2gd:main}, or to a dual version of \eqref{eq:s2gd:main}. Unless specific conditions on the problem structure are met, application to the primal directly are is not as computationally efficient as its dual version. Application of a coordinate descent method to the dual formulation of \eqref{eq:s2gd:main} is generally referred to as Stochastic Dual Coordinate Ascent (SDCA) \cite{SDCA-2008}. The algorithm in \cite{SDCA} exhibits this duality, and the method in \cite{takac-minibatch} extends the primal-dual framework to the parallel / mini-batch setting. Parallel and distributed stochastic coordinate descent methods were studied in \cite{RT:PCDM, FR:SPCDM2013, fercoq2014fast}.

Stochastic Average Gradient (SAG) by \cite{SAG}, is one of the first SGD-type methods, other than coordinate descent methods, which were shown to exhibit  linear convergence. The method of \cite{SVRG}, called Stochastic Variance Reduced Gradient (SVRG),  arises as a special case in our setting for a suboptimal choice of a single parameter of our method. The Epoch Mixed Gradient Descent (EMGD) method, \cite{zhanglijun}, is similar in spirit to SVRG, but achieves a quadratic dependence on the condition number instead of a linear dependence, as is the case with SDCA, SAG, SVRG and with our method.

Earlier works of \cite{friedlander2012hybrid}, \cite{deng2009variable} and \cite{bastin2006convergence} attempt to interpolate between GD and SGD and decrease variance by varying the sample size. These methods however do not realize the kind of improvements as the recent methods above. For partially related classical work on semi-stochastic approximation methods we refer\footnote{We thank Zaid Harchaoui who pointed us to these papers a few days before we posted our work to arXiv.} the reader to the papers of \cite{MF79, MF86}, which focus on general stochastic optimization.

\subsection{Outline}

We start in Section~\ref{sec:s2gd:S2GD} by describing two algorithms: S2GD, which we analyze, and S2GD+, which we do not analyze, but which exhibits superior performance in practice. We then move to summarizing some of the main contributions of this chapter in Section~\ref{sec:s2gd:summary}. Section~\ref{sec:s2gd:strong} is devoted to establishing expectation and high probability complexity results for S2GD in the case of a strongly convex loss. The results are generic in that the parameters of the method are set arbitrarily. Hence, in Section~\ref{sec:s2gd:OPT} we study the problem of choosing the parameters optimally, with the goal of minimizing the total workload (\# of processed examples) sufficient to produce a result of specified accuracy.
In Section~\ref{sec:s2gd:convex} we establish high probability complexity bounds for S2GD applied to a non-strongly convex loss function. Discussion of efficient implementation for sparse data is in Section~\ref{sec:s2gd:sparses2gd}. Finally, in Section~\ref{sec:s2gd:NUMERICS} we perform very encouraging numerical experiments on real and artificial problem instances. A brief conclusion can be found in Section~\ref{sec:s2gd:CONCLUDE}.

\section{Semi-Stochastic Gradient Descent} 
\label{sec:s2gd:S2GD}

In this section we describe two novel algorithms: S2GD and S2GD+. We analyze the former only. The latter, however, has superior convergence properties in our experiments. 

We assume throughout the chapter that the functions $f_i$ are convex and $L$-smooth.

\begin{assumption}
\label{ass:s2gd:lip}
The functions $f_1, \dots, f_n$ have Lipschitz continuous gradients with constant $L > 0$ (in other words, they are $L$-smooth). That is, for all $x, z \in \R^d$ and all $i=1,2,\dots,n$,
$$ f_i(z) \leq f_i(x) + \ip{ \nabla f_i(x)}{z - x} + \frac{L}{2} \| z - x \|^2. $$
(This implies that the gradient of $P$ is Lipschitz with constant $L$, and hence $P$ satisfies the same inequality.)
\end{assumption}

In one part of this chapter (Section~\ref{sec:s2gd:strong}) we also make the following additional assumption:

\begin{assumption}
\label{ass:s2gd:strong}
The average loss $P$ is $\mu$-strongly convex, $\mu>0$. That is, for all $x,z \in \R^d$,
\begin{equation}
P(z) \geq P(x) + \ip{\nabla P(x)}{z - x} + \frac{\mu}{2} \| z - x \|^2.
\label{eq:s2gd:SVRGstrcvx}
\end{equation}
(Note that, necessarily, $\mu \leq L $.)
\end{assumption}

\subsection{S2GD}

Algorithm~\ref{alg:s2gd:s2gd} (S2GD) depends on three parameters: stepsize $h$, constant $m$ limiting the number of stochastic gradients computed in a single epoch, and a $\nu \in [0,\mu]$, where $\mu$ is the strong convexity constant of $P$. In practice, $\nu$ would be a known lower bound on $\mu$. Note that the algorithm works also without any knowledge of the strong convexity parameter --- the case of $\nu = 0$.

\begin{algorithm}
\begin{algorithmic}
\State \textbf{parameters:} $m$ = max \# of stochastic steps per epoch, $h$ = stepsize, $\nu$ = lower bound on $\mu$
\For {$k = 0, 1, 2, \dots$}
	\State $g^k \gets \frac{1}{n} \sum_{i=1}^n \nabla f_i(w^{k})$
	\State $y^{k, 0} \gets w^k$
	\State Let $t^k \gets t$ with probability $(1 - \nu h)^{m-t} / \beta $ for $t = 1, 2, \dots, m$
	\For {$t = 0$ to $t^k-1$}
		\State Pick $i \in \{ 1, 2, \dots, n \}$, uniformly at random
		\State $ y^{k, t+1} \gets y^{k,t} - h \left( g^{k} + \nabla f_i(y^{k,t}) - \nabla f_i(w^k) \right) $
	\EndFor
	\State $w^{k+1} \gets y^{k, t^{k}}$
\EndFor
\end{algorithmic}

\caption{Semi-Stochastic Gradient Descent (S2GD)}
\label{alg:s2gd:s2gd}
\end{algorithm}

The method has an outer loop, indexed by epoch counter $k$, and an inner loop, indexed by $t$. In each epoch $k$, the method first computes $g^k$---the \emph{full} gradient of $P$ at $w^k$. Subsequently, the method produces a random number $t^k \in [1,m]$ of steps, following a geometric law, where
\begin{equation}
\label{eq:s2gd:beta}
\beta \eqdef \sum_{t = 1}^m (1 - \nu h)^{m-t},
\end{equation}
with only \emph{two stochastic gradients} computed in each step.\footnote{It is possible to get away with computing only a \emph{single} stochastic gradient per inner iteration, namely $\nabla f_i(y^{k,t})$, at the cost of having to store in memory $\nabla f_i(w^k)$ for $i=1,2,\dots,n$. This, however, can be impractical for big $n$.} For each $t = 0, \dots, t^k-1$, the stochastic gradient $\nabla f_i(w^k)$ is subtracted from $g^k$, and $\nabla f_i(y^{k,t})$ is added to $g^k$, which ensures that, one has 
$$ \E{g^k + \nabla f_i(y^{k,t}) - \nabla f_i(w^k)} = \nabla P(y^{k,t}), $$
where the expectation is with respect to the random variable $i$.

Hence, the algorithm is an instance of stochastic gradient descent -- albeit executed in a nonstandard way (compared to the traditional implementation described in the introduction).

Note that for all $k$, the expected number of iterations of the inner loop, $\E{t^k}$, is equal to
\begin{equation}
\label{eq:s2gd:syhs7s5hs}
\xi = \xi(m,h) \eqdef  \sum_{t=1}^{m} t \frac{(1-\nu h)^{m-t}}{\beta}.
\end{equation}
Also note that $\xi \in [\tfrac{m+1}{2},m)$, with the lower bound attained for $\nu=0$, and the upper bound for $\nu h \to 1$. 

\subsection{S2GD+}

We also implement Algorithm~\ref{alg:s2gd:S2GD+}, which we call S2GD+. In our experiments, the performance of this method is superior to all methods we tested, including S2GD. However, we do not analyze the complexity of this method and leave this as an open problem.

\begin{algorithm}
\begin{algorithmic}
\State \textbf{parameters:} $\alpha \geq 1$ (e.g., $\alpha=1$)
\State 1. Run SGD for a single pass over the data (i.e., $n$ iterations); output $w$
\State 2. Starting from $w_0=w$, run a version of S2GD in which $t^k = \alpha n$ for all $k$
\end{algorithmic}
\caption{S2GD+}
\label{alg:s2gd:S2GD+}
\end{algorithm}

In brief, S2GD+ starts by running SGD for 1 epoch (1 pass over the data) and then switches to a variant of S2GD in which the number of the inner iterations, $t^k$, is not random, but fixed to be $n$ or a  small multiple of $n$. 

The motivation for this method is the following. It is common knowledge that SGD is able to progress much more in one pass over the data than GD (where this would correspond to a single gradient step). However, the very first step of S2GD is the computation of the full gradient of $P$. Hence, by starting with a single pass over data using SGD and \emph{then} switching to S2GD, we obtain a superior method in practice.\footnote{Using a single pass of SGD as an initialization strategy was already considered in \cite{SAG}. However, the authors claim that their implementation of vanilla SAG did not benefit from it. S2GD does benefit from such an initialization due to it starting, in theory, with a (heavy) full gradient computation.}

\section{Summary of Results}
\label{sec:s2gd:summary}

In this section we summarize some of the main results and contributions of this work.

\begin{enumerate}
\item \textbf{Complexity for strongly convex $P$.} If $P$ is strongly convex, S2GD needs \begin{equation}
\label{eq:s2gd:sjss5s4s}
{\mathcal W} = O((n + \kappa)\log(1/\varepsilon))
\end{equation} 
work (measured as the total number of evaluations of the stochastic gradient, accounting for the full gradient evaluations as well) to output an $\varepsilon$-approximate solution (in expectation or in high probability), where $\kappa = L / \mu$ is the condition number. This is achieved by running S2GD with stepsize $h = \Theta(1/L)$, $k = \Theta(\log (1/\varepsilon))$ epochs (this is also equal to the number of full gradient evaluations) and $m = \Theta(\kappa)$ (this is also roughly equal to the number of stochastic gradient evaluations in a single epoch). The complexity results are stated in detail in Sections~\ref{sec:s2gd:strong} and \ref{sec:s2gd:OPT} (see Theorems~\ref{thm:s2gd:expalpha}, \ref{thm:s2gd:hpresult} and \ref{thm:s2gd:main2}; see also \eqref{eq:s2gd:m:nuismu0_2} and \eqref{eq:s2gd:0sjsys8jns}).

\item \textbf{Comparison with existing results.} This complexity result \eqref{eq:s2gd:sjss5s4s} matches the best-known results obtained for strongly convex losses in recent work such as \cite{SAG}, \cite{SVRG} and \cite{zhanglijun}. Our treatment is most closely related to \cite{SVRG}, and contains their method (SVRG) as a special case. However, our complexity results have better constants, which has  a discernable effect in practice. In Table~\ref{tbl:s2gd:comparison} we summarize our results in the strongly convex case with other existing results for different algorithms.

\begin{table}[h!]
\begin{center}
\begin{tabular}{|c|c|}
\hline
Algorithm & Complexity/Work \\
\hline \hline
Nesterov's algorithm & $O\left(\sqrt{\kappa}n\log(1/\varepsilon)\right)$ \\ EMGD & $O\left((n + \kappa^2)\log(1/\varepsilon)\right)$ \\ 
SAG & $O\left(\max\{n, \kappa\}\log(1/\varepsilon)\right)$ \\ 
SDCA & $O\left((n + \kappa)\log(1/\varepsilon)\right)$ \\ 
SVRG & $O\left((n + \kappa)\log(1/\varepsilon)\right)$ \\ 
\hline
\textbf{S2GD} & $O\left((n + \kappa)\log(1/\varepsilon)\right)$ \\ \hline
\end{tabular}
\end{center}
\caption{Comparison of performance of selected methods suitable for solving \eqref{eq:s2gd:main}. The complexity/work is measured in the number of stochastic gradient evaluations needed to find an $\varepsilon$-solution.}
\label{tbl:s2gd:comparison}
\end{table}
We should note that the rate of convergence of Nesterov's algorithm \cite{nesterov2004convex} is a deterministic result. EMGD and S2GD results hold with high probability (see Theorem~\ref{thm:s2gd:hpresult} for precise statement). Complexity results for stochastic coordinate descent methods are also typically analyzed in the high probability regime \cite{RichtarikTakacIteration}. The remaining results hold in expectation. Notion of $\kappa$ is slightly different for SDCA, which requires explicit knowledge of the strong convexity parameter $\mu$ to run the algorithm. In contrast, other methods do not algorithmically depend on this, and thus their convergence rate can adapt to any additional strong convexity locally.

\item \textbf{Complexity for convex $f$.}
If $P$ is \emph{not} strongly convex, then we propose that S2GD be applied to a perturbed version of the problem, with strong convexity  constant $\mu=O(L/\varepsilon)$. An $\varepsilon$-accurate solution of the original problem is recovered with arbitrarily high probability (see Theorem~\ref{thm:s2gd:hpresult2} in Section~\ref{sec:s2gd:convex}). The total work in this case is \[{\mathcal W}=O\left( \left(n+ L/\varepsilon)\right)\log\left(1/\varepsilon\right)\right),\]
that is, $\tilde{O}(1/\epsilon)$, which is better than the standard rate of SGD.

\item \textbf{Optimal parameters.} We derive formulas for optimal parameters of the method which (approximately) minimize the total workload, measured in the number of stochastic gradients computed (counting a single full gradient evaluation as $n$ evaluations of the stochastic gradient). In particular, we show that the method should be run for $O(\log(1/\varepsilon))$ epochs, with stepsize $h=\Theta(1/L)$ and $m=\Theta(\kappa)$. No such results were derived for SVRG in \cite{SVRG}.

\item \textbf{One epoch.} Consider the case when S2GD is run for 1 epoch only, effectively limiting the number of full gradient evaluations to 1, while choosing a target accuracy $\epsilon$. We show that S2GD with $\nu=\mu$ needs \[O(n+(\kappa/\varepsilon)\log(1/\varepsilon))\] work only (see Table~\ref{tbl:s2gd:ssus8778}). This compares favorably with the optimal complexity in the $\nu = 0$ case (which reduces to SVRG), where the work needed is \[O(n+\kappa/\varepsilon^2).\]

For two epochs one could just say that we need $\sqrt{\varepsilon}$ decrease in each epoch, thus having complexity of $O(n+(\kappa/\sqrt{\varepsilon})\log(1/\sqrt{\varepsilon}))$. This is already better than general rate of SGD $(O(1 / \varepsilon)).$

\begin{table}[!h]
\begin{center}
\begin{tabular}{|c|l|c|}
\hline
\textbf{Parameters} & \textbf{Method} & \textbf{Complexity}\\
\hline
\phantom{-} & & \\
\begin{tabular}{c}
$\nu=\mu$, $k=\Theta(\log(\tfrac{1}{\varepsilon}))$\\ \& $m=\Theta(\kappa)$
\end{tabular}
  & Optimal S2GD & $O((n+\kappa)\log(\tfrac{1}{\varepsilon}))$\\ 
\phantom{-} & & \\
\hline
$m=1$ & GD & ---\\
$\nu=0$ & SVRG \cite{SVRG} & $O((n+\kappa)\log(\tfrac{1}{\varepsilon}))$ \\
$\nu=0$, $k=1$, $m=\Theta(\tfrac{\kappa}{\varepsilon^2})$ & Optimal SVRG with 1 epoch & $O(n+\tfrac{\kappa}{\varepsilon^2})$\\
$\nu=\mu$, $k=1$, $m = \Theta(\tfrac{\kappa}{\varepsilon} \log(\tfrac{1}{\varepsilon}))$ & Optimal S2GD with 1 epoch & $O(n+\tfrac{\kappa}{\varepsilon}\log(\tfrac{1}{\varepsilon}))$\\
\hline
\end{tabular}
\end{center}
\caption{Summary of complexity results and special cases. Condition number: $\kappa = L/\mu$ if $f$ is $\mu$-strongly convex and $\kappa=2L/\varepsilon$ if $f$ is \emph{not} strongly convex and $\epsilon \leq L$. }
\label{tbl:s2gd:ssus8778}
\end{table}

\item \textbf{Special cases.} 
GD and SVRG arise as special cases of S2GD, for $m=1$ and $\nu=0$, respectively.\footnote{While S2GD reduces to GD for $m=1$, our \emph{analysis} does not say anything meaningful in the $m=1$ case - it is too coarse to cover this case. This is also the reason behind the empty space in the ``Complexity'' box column for GD in Table~\ref{tbl:s2gd:ssus8778}.} 

\item \textbf{Low memory requirements.} Note that SDCA and SAG, unlike SVRG and S2GD, need to store  all gradients $\nabla f_i$ (or dual variables) throughout the iterative process. While this may not be a problem for a modest sized optimization task, this requirement makes such methods less suitable for problems with very large $n$.

\item \textbf{S2GD+.} We propose a ``boosted'' version of S2GD, called S2GD+, which we do not analyze. In our experiments, however, it performs vastly superior to all other methods we tested, including GD, SGD, SAG and S2GD. S2GD alone is better than both GD and SGD if a highly accurate solution is required. The performance of S2GD and SAG is roughly comparable, even though in our experiments S2GD turned to have an edge.

\end{enumerate}

\section{Complexity Analysis: Strongly Convex Loss}
\label{sec:s2gd:strong}

For the purpose of the analysis, let 
\begin{equation}
\mathcal{F}^{k,t} \eqdef \sigma( w^1, w^2, \dots, w^k; y^{k,1}, y^{k,2}, \dots, y^{k,t} )
\label{eq:s2gd:sigmaalgebra}
\end{equation}
be the $\sigma$-algebra generated by the relevant history of S2GD. We first isolate an auxiliary result.

\begin{lemma}
Consider the S2GD algorithm. For any fixed epoch number $k$, the following identity holds:
\begin{equation}
\E{P(w^{k+1})} = \frac{1}{\beta} \sum_{t=1}^m (1 - \nu h)^{m-t}\E{P(y^{k,t-1})}.
\label{eq:s2gd:uberLemma}
\end{equation}
\end{lemma}

\begin{proof} 
By the tower law of conditional expectations and the definition of $w^{k+1}$ in the algorithm, we obtain
\begin{eqnarray*}
\E{P(w^{k+1})} \;\;=\;\; \E{ \E{ P(w^{k+1})\; |\; \mathcal{F}^{k, m}} } &=& \E{ \sum_{t=1}^m \frac{(1 - \nu h)^{m-t}}{\beta} P(y^{k,t-1}) }\\
&=&\frac{1}{\beta}\sum_{t=1}^m (1 - \nu h)^{m-t} \E{P(y^{k,t-1})}.
\end{eqnarray*}
\end{proof}

We now state and prove the main result of this section. 

\begin{theorem} 
\label{thm:s2gd:MAIN} 
Let Assumptions \ref{ass:s2gd:lip} and \ref{ass:s2gd:strong} be satisfied. Consider the S2GD algorithm applied to solving problem \eqref{eq:s2gd:main}. Choose $0\leq \nu \leq \mu$, $0< h  < \frac{1}{2L}$, and let $m$ be sufficiently large so that 
\begin{equation}
\label{eq:s2gd:hshshs7}
c \eqdef \frac{(1 -\nu h)^m}{\beta \mu h (1 - 2Lh)} + \frac{2(L - \mu)h}{1 - 2Lh} < 1.
\end{equation}
Then we have the following convergence in expectation:
\begin{equation}
\label{eq:s2gd:s8shs7}
\E{P(w^k) - P(w^*)} \leq c^{k} (P(w^0) - P(w^*)). 
\end{equation}
\label{thm:s2gd:expalpha}
\end{theorem}

Before we proceed to proving the theorem, note that in the special case with $\nu = 0$, we recover the result of \cite{SVRG} (with a minor improvement in the second term of $c$ where $L$ is replaced by $L-\mu$), namely 
\begin{equation}
c = \frac{1}{\mu h (1 - 2Lh) m} + \frac{2(L - \mu)h}{1 - 2Lh}. 
\label{eq:s2gd:nuiszero}
\end{equation} If we set $\nu = \mu$, then $c$ can be written in the form (see \eqref{eq:s2gd:beta})
\begin{equation}
c = \frac{(1 - \mu h)^m}{(1 - (1 - \mu h)^m) (1 - 2Lh)} + \frac{2(L - \mu)h}{1 - 2Lh}.
\label{eq:s2gd:nuismu}
\end{equation}

Clearly, the latter $c$ is a major improvement on the former one. We shall elaborate on this further later.

\begin{proof}
It is well-known \cite[Theorem 2.1.5]{nesterov2004convex} that since the functions $f_i$ are $L$-smooth, they necessarily satisfy the following inequality:
$$ \| \nabla f_i(w) - \nabla f_i(w^*) \|^2 \leq 2L \left[ f_i(w) - f_i(w^*) - \ip{\nabla f_i (w^*)}{w - w^*} \right]. $$
By summing these inequalities for  $i = 1, \dots, n$, and using $\nabla P(w^*) = 0,$ we get
\begin{equation}
\frac{1}{n} \sum_{i=1}^n \| \nabla f_i(w) - \nabla f_i(w^*) \|^2 \leq 2L \left[ P(x) - P(w^*) - \ip{\nabla P(w^*)}{w - w^*} \right] = 2L (P(w) - P(w^*)).
\label{eq:s2gd:SVRGbound}
\end{equation}

Let $G^{k, t} \eqdef g^{k} + \nabla f_i(y^{k,t-1}) - \nabla f_i(w^{k}) $ be the direction of update at ${k}^{th}$ iteration in the outer loop and $t^{th}$ iteration in the inner loop. Taking expectation with respect to $i$, conditioned on the $\sigma$-algebra $\mathcal{F}^{k, t-1}$ \eqref{eq:s2gd:sigmaalgebra}, we obtain\footnote{For simplicity, we suppress the $\E{\cdot \;|\; \mathcal{F}^{k,t-1}}$ notation here.}
\begin{eqnarray}
\notag
\E{ \|G^{k,t} \|^2 } &=& \E{ \| \nabla f_i(y^{k,t-1}) - \nabla f_i(w^*) - \nabla f_i(w^{k}) + \nabla f_i(w^*) + g^{k} \|^2 } \\ \notag
&\leq& 2\E{ \| \nabla f_i(y^{k,t-1}) - \nabla f_i(w^*) \|^2 } + 2\E{ \| \left[ \nabla f_i(w^{k}) - \nabla f_i(w^*) \right] - \nabla P(w^{k}) \|^2 } \\ \notag
&=& 2\E{ \| \nabla f_i(y^{k,t-1}) - \nabla f_i(w^*) \|^2 } + 2 \E{ \| \nabla f_i(w^{k}) - \nabla f_i(w^*) \|^2 } \\ \notag
&& \qquad - 4\E{ \ip{\nabla P(w^{k})}{\nabla f_i(w^{k}) - \nabla f_i(w^*)}} + 2\| \nabla P(w^{k}) \|^2 \\ \notag
&\overset{\eqref{eq:s2gd:SVRGbound}} {\leq}& 4L \left[ P(y^{k,t-1}) - P(w^*) + P(w^{k}) - P(w^*) \right] \\ \notag
&& \qquad - 2 \| \nabla P(w^{k}) \|^2 - 4 \ip{\nabla P(w^{k})}{\nabla P(w^*)} \\
& \overset{\eqref{eq:s2gd:SVRGstrcvx}} {\leq} & 4L \left[ P(y^{k,t-1}) - P(w^*) \right] + 4(L-\mu) \left[ P(w^{k}) - P(w^*) \right].
\label{eq:s2gd:expvstvariance}
\end{eqnarray}
Above we have used the bound $\| x'+x'' \|^2 \leq 2\|x'\|^2 + 2\|x''\|^2$ and the fact that
\begin{equation}
\E{ G^{k, t} \;|\; \mathcal{F}^{k,t-1}} = \nabla P(y^{k, t-1}).
\label{eq:s2gd:expvst}
\end{equation}

We now study the expected distance to the optimal solution (a standard approach in the analysis of gradient methods):
\begin{eqnarray}
\E{ \| y^{k,t} - w^* \|^2 \;|\; \mathcal{F}^{k, t-1}} &=& \| y^{k,t-1} - w^* \|^2 - 2h \ip{ \E{ G^{k,t} \;|\; \mathcal{F}^{k, t-1}}}{ y^{k, t-1} - w^* } \notag \\
&& \qquad + h^2 \E{ \|G^{k,t}\|^2 \;|\; \mathcal{F}^{k, t-1} } \notag \\
& \overset{\eqref{eq:s2gd:expvstvariance}+ \eqref{eq:s2gd:expvst}} {\leq} & \|y^{k,t-1} - w^* \|^2 - 2h \ip{ \nabla P(y^{k,t-1})}{y^{k, t-1} - w^*} \notag \\
&& \qquad + 4Lh^2 \left[ P(y^{k,t-1}) - P(w^*) \right] \notag \\
&& \qquad + 4(L - \mu)h^2\left[ P(w^{k}) - P(w^*) \right] \notag\\
& \overset{\eqref{eq:s2gd:SVRGstrcvx}}{\leq} & \| y^{k,t-1} - w^* \|^2 - 2h \left[ P(y^{k,t-1}) - P(w^*) \right] \notag \\
&& \qquad - \nu h \| y^{k,t-1} - w^* \|^2 + 4Lh^2 \left[ P(y^{k,t-1}) - P(w^*) \right] \notag\\
&& \qquad + 4(L - \mu)h^2\left[ P(w^{k}) - P(w^*) \right] \notag\\
&=& (1 - \nu h) \| y^{k,t-1} - w^* \|^2 \notag \\ 
&& \qquad - 2h(1 - 2Lh)[P(y^{k,t-1}) - P(w^*)] \notag\\ 
&& \qquad + 4(L - \mu)h^2[P(w^{k}) - P(w^*)].
\label{eq:s2gd:distbound}
\end{eqnarray}

By rearranging the terms in \eqref{eq:s2gd:distbound} and taking expectation over the $\sigma$-algebra $\mathcal{F}^{k,t-1}$, we get the following inequality:

\begin{align}
\E{ \| y^{k,t} - w^* \|^2} &+ 2h(1 - 2Lh) \E{ P(y^{k,t-1}) - P(w^*) } \notag\\ 
&\leq (1 - \nu h) \E{ \| y^{k,t-1} - w^* \|^2 } + 4(L - \mu)h^2\E{ P(w^{k}) - P(w^*)}.
\label{eq:s2gd:shs7shs}
\end{align}

Finally, we can analyze what happens after one iteration of the outer loop of S2GD, i.e., between two computations of the full gradient. By summing up inequalities \eqref{eq:s2gd:shs7shs} for $t = 1, \dots, m$, with inequality $t$ multiplied by $ (1 - \nu h)^{m-t}$, we get the left-hand side
\begin{eqnarray}
LHS &=& \E{ \| y^{k,m} - w^* \|^2 } + 2h(1 - 2Lh) \sum_{t = 1}^m (1 - \nu h)^{m-t} \E{ P(y^{k,t-1}) - P(w^*)} \notag \\
&\overset{\eqref{eq:s2gd:uberLemma}}{=} & \E{ \| y^{k,m} - w^*\|^2 } + 2\beta h(1 - 2Lh) \E{ P(w^{k+1}) - P(w^*)}, \notag
\end{eqnarray}
and the right-hand side
\begin{eqnarray}
RHS &=& (1 - \nu h)^m \E{ \| w^{k} - w^* \|^2 } + 4\beta(L - \mu)h^2 \E{ P(w^{k}) - P(w^*)} \notag\\
&\overset{\eqref{eq:s2gd:SVRGstrcvx}}{\leq} &\frac{2(1 - \nu h)^m}{\mu} \E{ P(w^{k}) - P(w^*) } + 4\beta(L - \mu)h^2 \E{ P(w^{k}) - P(w^*) } \notag \\
&= & 2\left(\frac{(1 - \nu h)^m }{\mu} + 2\beta(L - \mu)h^2\right) \E{ P(w^{k}) - P(w^*)}.\notag
\end{eqnarray}
Since $LHS \leq RHS$, we finally conclude with 
\begin{eqnarray*}
\E{ P(w^{k+1}) - P(w^*) } &\leq & c \E{ P(w^{k}) - P(w^*)} - \frac{\E{ \| y^{k,m} - w^* \|^2}}{2\beta h (1 - 2Lh)} \;\;\leq \;\; c \E{ P(w^{k}) - P(w^*)}.
\end{eqnarray*}
\end{proof}

Since we have established linear convergence of expected values, a high probability result can be obtained in a straightforward way using Markov inequality. 

\begin{theorem}
Consider the setting of Theorem~\ref{thm:s2gd:MAIN}. Then, for any $ 0 < \rho < 1 $, $ 0 < \varepsilon < 1 $ and 
\begin{equation}
k \geq \frac{\log \left( \frac{1}{\varepsilon \rho}\right)}{\log \left(\frac{1}{c}\right)},
\label{eq:s2gd:hprobs}
\end{equation} 
we have 
\begin{equation}
\label{eq:s2gd:sjnd8djd}
\Prob{ \frac{P(w^{k}) - P(w^*)}{P(w^0)-P(w^*)} \leq \varepsilon } \geq 1 - \rho. 
\end{equation}
\label{thm:s2gd:hpresult}
\end{theorem}

\begin{proof}
This follows directly from  Markov inequality and Theorem~\ref{thm:s2gd:expalpha}:
$$ \Prob{ P(w^{k}) - P(w^*) > \varepsilon \left( P(w^0)-P(w^*) \right) } \overset{\eqref{eq:s2gd:s8shs7}}{\leq} \frac{\E{ P(w^{k}) - P(w^*)}}{\varepsilon(P(w^0)-P(w^*))} \leq \frac{c^k}{\varepsilon} \overset{\eqref{eq:s2gd:hprobs}}{\leq} \rho $$
\end{proof}

This result will be also useful when treating the non-strongly convex case.

\section{Optimal Choice of Parameters}
\label{sec:s2gd:OPT}

The goal of this section is to provide insight into the choice of parameters of  S2GD; that is, the number of epochs (equivalently, full gradient evaluations) $k$, the maximal number of steps in each epoch $m$, and the stepsize $h$. The remaining parameters ($L, \mu, n$) are inherent in the problem and we will hence treat them in this section as given.

In particular, ideally we wish to find parameters $k$, $m$ and $h$  solving the following optimization problem:
\begin{equation}
\label{eq:s2gd:pracReq}
\min_{k,m,h}  \quad \tilde{{\mathcal W}}(k,m,h) \eqdef k(n+2\xi(m,h)),
\end{equation}
subject to 
\begin{equation}
\label{eq:s2gd:sjs8s}
\E{ P(w^{k})-P(w^*) } \leq \varepsilon(P(w^0)-P(w^*)).
\end{equation}
Note that $\tilde{{\mathcal W}}(k,m,h)$ is the \emph{expected work}, measured by the number  number of stochastic gradient evaluations, performed by  S2GD when running for $k$ epochs. Indeed, the evaluation of $g^{k}$  is equivalent to $n$ stochastic gradient evaluations, and each epoch further computes on average $2\xi(m,h)$ stochastic gradients (see \eqref{eq:s2gd:syhs7s5hs}). Since $\tfrac{m+1}{2}\leq \xi(m,h) < m$, we can simplify and solve the problem with $\xi$ set to the conservative upper estimate $\xi=m$.

In view of \eqref{eq:s2gd:s8shs7}, accuracy constraint \eqref{eq:s2gd:sjs8s} is satisfied if $c$ (which depends on $h$ and $m$) and $k$ satisfy
\begin{equation}
c^{k} \leq \varepsilon.
\label{eq:s2gd:pracReq2}
\end{equation} 
We therefore instead consider the parameter fine-tuning problem
\begin{equation}
\label{eq:s2gd:shd6dhd7}
\min_{k,m,h} {\mathcal W}(k,m,h) \eqdef k(n+2m) \qquad \text{subject to} \qquad c \leq \varepsilon^{1/{k}}.
\end{equation}

In the following we (approximately) solve this problem in two steps. First, we fix $k$ and find (nearly) optimal $h=h(k)$ and $m=m(k)$. The problem  reduces to minimizing $m$ subject to $c \leq \varepsilon^{1/{k}}$ by fine-tuning $h$. While in the $\nu=0$ case it is possible to obtain closed form solution, this is not possible for $\nu>0$. 

However, it is still possible to obtain a good formula for $h(k)$ leading to expression for good $m(k)$ which depends on $\varepsilon$ in the correct way. We then plug the formula for $m(k)$ obtained this way  back into \eqref{eq:s2gd:shd6dhd7}, and study the quantity ${\mathcal W}(k,m(k),h(k)) = k(n+2m(k))$ as a function of $k$, over which we optimize optimize at the end.

\begin{theorem}[Choice of parameters]
\label{thm:s2gd:main2}
Fix the number of epochs $k \geq 1$, error tolerance $0 < \varepsilon < 1$, and  let $\Delta = \varepsilon^{1/k}$. If we run S2GD  with  the stepsize
\begin{equation}
\label{eq:s2gd:shhsdd998}  
h = h(k) \eqdef \frac{1}{\frac{4}{\Delta}(L-\mu) + 2L}
\end{equation}
and 
\begin{equation}
m \geq m(k) \eqdef
\begin{cases}
\left(\frac{4(\kappa-1)}{\Delta}+2 \kappa\right) \log\left(\frac{2}{\Delta} + \frac{2\kappa - 1}{\kappa-1}\right), & \quad \text{if} \quad \nu=\mu,\\
\frac{8(\kappa-1)}{\Delta^2} + \frac{8\kappa}{\Delta} + \frac{2\kappa^2}{\kappa-1},& \quad \text{if} \quad \nu=0,
\end{cases}
\label{eq:s2gd:m:nuismu0}
\end{equation}
then $\E{ P(w^{k}) - P(w^*) } \leq \varepsilon (P(w^0)-P(w^*)).$

In particular, if we choose $k^* = \lceil \log (1/\varepsilon)\rceil$, then 
$\frac{1}{\Delta} \leq \exp(1)$, and
hence $m(k^*) = O(\kappa)$, leading to the workload 
\begin{equation}
{\mathcal W}({k}^*,m({k}^*),h({k}^*)) = \left\lceil \log\left(\frac{1}{\varepsilon}\right)\right\rceil (n+ O(\kappa)) = O\left((n+\kappa) \log\left(\frac{1}{\varepsilon}\right)\right).
\label{eq:s2gd:0sjsys8jns}
\end{equation}

\end{theorem}

\begin{proof}
We only need to show that $c \leq \Delta$, where
$c$ is given by \eqref{eq:s2gd:nuismu} for $\nu=\mu$ and by  \eqref{eq:s2gd:nuiszero} for $\nu=0$. We denote the two summands in expressions for $c$ as $c_1$ and $c_2$. We choose the $h$ and $m$ so that both $c_1$ and $c_2$ are smaller than $\Delta / 2$, resulting in $c_1 + c_2 = c \leq \Delta$. 

The stepsize $h$ is chosen so that
$$ c_2 \eqdef  \frac{2(L - \mu)h}{1 - 2Lh} = \frac{\Delta}{2},$$
and hence it only remains to verify that $c_1 = c-c_2 \leq \frac{\Delta}{2}$. In the $\nu=0$ case, $m(k)$ is chosen so that $c-c_2=\frac{\Delta}{2}$. In the $\nu=\mu$ case, $c-c_2=\frac{\Delta}{2}$  holds for
$m = \log\left(\frac{2}{\Delta}+ \frac{2\kappa-1}{\kappa-1}\right)/\log\left(\frac{1}{1-H}\right)$, where $H = \left(\frac{4(\kappa-1)}{\Delta} + 2\kappa \right)^{-1}$. We only need to observe that $c$ decreases as $m$ increases, and apply the inequality $\log\left(\frac{1}{1-H}\right) \geq H$.

\end{proof}

We now comment on the above result:

\begin{enumerate}

\item \textbf{Workload.} Notice that for the choice of parameters $k^*$, $h=h(k^*)$, $m=m(k^*)$ and any $\nu \in [0,\mu]$, the method needs $\log(1/\varepsilon)$ computations of the full gradient (note this is independent of $\kappa$), and $O(\kappa \log(1/\varepsilon))$ computations of the stochastic gradient. This result, and special cases thereof, are summarized in Table~\ref{tbl:s2gd:ssus8778}.

\item \textbf{Simpler formulas for $m$.} If $\kappa\geq 2$, we can instead of \eqref{eq:s2gd:m:nuismu0} use the following (slightly worse but) simpler expressions for $m(k)$, obtained from \eqref{eq:s2gd:m:nuismu0} by using the bounds $1\leq\kappa-1$, $\kappa-1\leq \kappa$ and $\Delta<1$ in appropriate places (e.g., $\tfrac{8\kappa}{\Delta} <\tfrac{8\kappa}{\Delta^2}$, $\tfrac{\kappa}{\kappa-1}\leq 2 < \tfrac{2}{\Delta^2}$):
\begin{equation}
m \geq \tilde{m}(k) \eqdef
\begin{cases}
\frac{6\kappa}{\Delta} \log\left(\frac{5}{\Delta}\right), & \quad \text{if} \quad \nu=\mu,\\
\frac{20\kappa}{\Delta^2},& \quad \text{if} \quad \nu=0.
\end{cases}
\label{eq:s2gd:m:nuismu0_2}
\end{equation}

\item \textbf{Optimal stepsize in the $\nu=0$ case.} 
Theorem~\ref{thm:s2gd:main2} does not claim to have solved problem \eqref{eq:s2gd:shd6dhd7}; the problem in general does not have a closed form solution. However, in the $\nu=0$ case a closed-form formula can easily be obtained:
\begin{equation}
\label{eq:s2gd:sjs8djd}
h(k) =  \frac{1}{\frac{4}{\Delta}(L-\mu)+ 4L}, \qquad \qquad m \geq m(k) \eqdef  \frac{8(\kappa-1)}{\Delta^2} + \frac{8\kappa}{\Delta}.
\end{equation}
Indeed, for fixed $k$,  \eqref{eq:s2gd:shd6dhd7} is equivalent to finding $h$ that minimizes $m$ subject to the constraint $c\leq \Delta$. In view of \eqref{eq:s2gd:nuiszero}, this is equivalent to searching for $h>0$ maximizing the quadratic $h \to h(\Delta-2(\Delta L + L -\mu)h)$, which leads to \eqref{eq:s2gd:sjs8djd}.

Note that both the stepsize $h(k)$ and the resulting $m(k)$ are slightly larger in Theorem~\ref{thm:s2gd:main2} than in \eqref{eq:s2gd:sjs8djd}. This is because in the theorem the stepsize was for simplicity chosen to satisfy $c_2=\frac{\Delta}{2}$, and hence is (slightly) suboptimal. Nevertheless, the dependence of $m(k)$ on $\Delta$ is of the correct (optimal) order  in both cases. That is, $m(k) = O\left(\tfrac{\kappa}{\Delta}\log(\tfrac{1}{\Delta})\right)$ for $\nu=\mu$ and $m(k)=O\left(\tfrac{\kappa}{\Delta^2}\right)$ for $\nu=0$.

\item \textbf{Stepsize choice.} In cases when one does not have a good estimate of the strong convexity constant $\mu$ to determine the stepsize via \eqref{eq:s2gd:shhsdd998}, one may choose suboptimal stepsize that does not depend on $\mu$ and derive similar results to those above. For instance, one may choose $h=\frac{\Delta}{6L}$.

\end{enumerate}

In Table~\ref{tbl:s2gd:suboptbounds} we provide comparison of work needed for small values of $k$, and different values of $\kappa$ and $\varepsilon.$ Note, for instance, that for any problem with $n=10^9$ and $\kappa=10^3$, S2GD outputs a highly accurate solution ($\varepsilon=10^{-6}$) in the amount of work equivalent to $2.12$ evaluations of the full gradient of $f$! 

\begin{table}[!h]
\begin{center}
\begin{tabular}{C{15pt}|c|c|}
\cline{2-3}
 & \multicolumn{2}{c|}{$\varepsilon = 10^{-3}, \kappa = 10^3$} \\ \hline
\multicolumn{1}{|c|}{$k$} & $\mathcal{W}_\mu(k)$ & $\mathcal{W}_0(k)$ \\ \hline \hline
\multicolumn{1}{|c|}{$1$} & $\textbf{1.06n}$ & $17.0n$ \\ 
\multicolumn{1}{|c|}{$2$} & $2.00n$ & $\textbf{2.03n}$ \\ 
\multicolumn{1}{|c|}{$3$} & $3.00n$ & $3.00n$ \\ 
\multicolumn{1}{|c|}{$4$} & $4.00n$ & $4.00n$ \\ 
\multicolumn{1}{|c|}{$5$} & $5.00n$ & $5.00n$ \\ 
\hline \end{tabular}
\quad
\begin{tabular}{C{15pt}|c|c|}
\cline{2-3}
 & \multicolumn{2}{c|}{$\varepsilon = 10^{-6}, \kappa = 10^3$} \\ \hline
\multicolumn{1}{|c|}{$k$} & $\mathcal{W}_\mu(k)$ & $\mathcal{W}_0(k)$ \\ \hline \hline
\multicolumn{1}{|c|}{$1$} & $116n$ & $10^7n$ \\ 
\multicolumn{1}{|c|}{$2$} & $\textbf{2.12n}$ & $34.0n$ \\ 
\multicolumn{1}{|c|}{$3$} & $3.01n$ & $\textbf{3.48n}$ \\ 
\multicolumn{1}{|c|}{$4$} & $4.00n$ & $4.06n$ \\ 
\multicolumn{1}{|c|}{$5$} & $5.00n$ & $5.02n$ \\ 
\hline \end{tabular}
\quad
\begin{tabular}{C{15pt}|c|c|}
\cline{2-3}
 & \multicolumn{2}{c|}{$\varepsilon = 10^{-9}, \kappa = 10^3$} \\ \hline
\multicolumn{1}{|c|}{$k$} & $\mathcal{W}_\mu(k)$ & $\mathcal{W}_0(k)$ \\ \hline \hline
\multicolumn{1}{|c|}{$2$} & $7.58n$ & $10^4n$ \\ 
\multicolumn{1}{|c|}{$3$} & $\textbf{3.18n}$ & $51.0n$ \\ 
\multicolumn{1}{|c|}{$4$} & $4.03n$ & $6.03n$ \\ 
\multicolumn{1}{|c|}{$5$} & $5.01n$ & $\textbf{5.32n}$ \\ 
\multicolumn{1}{|c|}{$6$} & $6.00n$ & $6.09n$ \\ 
\hline \end{tabular}
\quad
\newline
\newline

\begin{tabular}{C{15pt}|c|c|}
\cline{2-3}
 & \multicolumn{2}{c|}{$\varepsilon = 10^{-3}, \kappa = 10^6$} \\ \hline
\multicolumn{1}{|c|}{$k$} & $\mathcal{W}_\mu(k)$ & $\mathcal{W}_0(k)$ \\ \hline \hline
\multicolumn{1}{|c|}{$2$} & $4.14n$ & $35.0n$ \\ 
\multicolumn{1}{|c|}{$3$} & $\textbf{3.77n}$ & $8.29n$ \\ 
\multicolumn{1}{|c|}{$4$} & $4.50n$ & $\textbf{6.39n}$ \\ 
\multicolumn{1}{|c|}{$5$} & $5.41n$ & $6.60n$ \\ 
\multicolumn{1}{|c|}{$6$} & $6.37n$ & $7.28n$ \\ 
\hline \end{tabular}
\quad
\begin{tabular}{C{15pt}|c|c|}
\cline{2-3}
 & \multicolumn{2}{c|}{$\varepsilon = 10^{-6}, \kappa = 10^6$} \\ \hline
\multicolumn{1}{|c|}{$k$} & $\mathcal{W}_\mu(k)$ & $\mathcal{W}_0(k)$ \\ \hline \hline
\multicolumn{1}{|c|}{$4$} & $8.29n$ & $70.0n$ \\ 
\multicolumn{1}{|c|}{$5$} & $\textbf{7.30n}$ & $26.3n$ \\ 
\multicolumn{1}{|c|}{$6$} & $7.55n$ & $16.5n$ \\ 
\multicolumn{1}{|c|}{$8$} & $9.01n$ & $\textbf{12.7n}$ \\ 
\multicolumn{1}{|c|}{$10$} & $10.8n$ & $13.2n$ \\ 
\hline \end{tabular}
\quad
\begin{tabular}{C{15pt}|c|c|}
\cline{2-3}
 & \multicolumn{2}{c|}{$\varepsilon = 10^{-9}, \kappa = 10^6$} \\ \hline
\multicolumn{1}{|c|}{$k$} & $\mathcal{W}_\mu(k)$ & $\mathcal{W}_0(k)$ \\ \hline \hline
\multicolumn{1}{|c|}{$5$} & $17.3n$ & $328n$ \\ 
\multicolumn{1}{|c|}{$8$} & $\textbf{10.9n}$ & $32.5n$ \\ 
\multicolumn{1}{|c|}{$10$} & $11.9n$ & $21.4n$ \\ 
\multicolumn{1}{|c|}{$13$} & $14.3n$ & $\textbf{19.1n}$ \\ 
\multicolumn{1}{|c|}{$20$} & $21.0n$ & $23.5n$ \\ 
\hline \end{tabular}
\quad
\newline
\newline

\begin{tabular}{C{15pt}|c|c|}
\cline{2-3}
 & \multicolumn{2}{c|}{$\varepsilon = 10^{-3}, \kappa = 10^9$} \\ \hline
\multicolumn{1}{|c|}{$k$} & $\mathcal{W}_\mu(k)$ & $\mathcal{W}_0(k)$ \\ \hline \hline
\multicolumn{1}{|c|}{$6$} & $378n$ & $1293n$ \\ 
\multicolumn{1}{|c|}{$8$} & $\textbf{358n}$ & $1063n$ \\ 
\multicolumn{1}{|c|}{$11$} & $376n$ & $\textbf{1002n}$ \\ 
\multicolumn{1}{|c|}{$15$} & $426n$ & $1058n$ \\ 
\multicolumn{1}{|c|}{$20$} & $501n$ & $1190n$ \\ 
\hline \end{tabular}
\quad
\begin{tabular}{C{15pt}|c|c|}
\cline{2-3}
 & \multicolumn{2}{c|}{$\varepsilon = 10^{-6}, \kappa = 10^9$} \\ \hline
\multicolumn{1}{|c|}{$k$} & $\mathcal{W}_\mu(k)$ & $\mathcal{W}_0(k)$ \\ \hline \hline
\multicolumn{1}{|c|}{$13$} & $737n$ & $2409n$ \\ 
\multicolumn{1}{|c|}{$16$} & $\textbf{717n}$ & $2126n$ \\ 
\multicolumn{1}{|c|}{$19$} & $727n$ & $2025n$ \\ 
\multicolumn{1}{|c|}{$22$} & $752n$ & $\textbf{2005n}$ \\ 
\multicolumn{1}{|c|}{$30$} & $852n$ & $2116n$ \\ 
\hline \end{tabular}
\quad
\begin{tabular}{C{15pt}|c|c|}
\cline{2-3}
 & \multicolumn{2}{c|}{$\varepsilon = 10^{-9}, \kappa = 10^9$} \\ \hline
\multicolumn{1}{|c|}{$k$} & $\mathcal{W}_\mu(k)$ & $\mathcal{W}_0(k)$ \\ \hline \hline
\multicolumn{1}{|c|}{$15$} & $1251n$ & $4834n$ \\ 
\multicolumn{1}{|c|}{$24$} & $\textbf{1076n}$ & $3189n$ \\ 
\multicolumn{1}{|c|}{$30$} & $1102n$ & $3018n$ \\ 
\multicolumn{1}{|c|}{$32$} & $1119n$ & $\textbf{3008n}$ \\ 
\multicolumn{1}{|c|}{$40$} & $1210n$ & $3078n$ \\ 
\hline \end{tabular}
\quad
\newline
\end{center}
\caption{Comparison of work sufficient to solve a problem with $n = 10^9$, and various values of $\kappa$ and $\varepsilon$. The work was computed using formula \eqref{eq:s2gd:shd6dhd7}, with $m(k)$ as in \eqref{eq:s2gd:m:nuismu0_2}. The notation ${\mathcal W}_\nu(k)$ indicates what parameter $\nu$ was used.}
\label{tbl:s2gd:suboptbounds}
\end{table}

\section{Complexity Analysis: Convex Loss} 
\label{sec:s2gd:convex}

If $P$ is convex but not strongly convex, we define $\hat{f}_i(w)\eqdef f_i(w) + \tfrac{\mu}{2} \| w - w^0 \|^2$, for small enough $\mu>0$ (we shall see below how the choice of $\mu$ affects the results), and consider the perturbed problem
\begin{equation}
\label{eq:s2gd:shs7hss}
\min_{w \in \R^d} \hat{P}(w),
\end{equation} 
where 
\begin{equation}
\label{eq:s2gd:barf}
\hat{P}(w) \eqdef \frac{1}{n} \sum_{i=1}^n\hat{f}_i(w) = P(w) + \frac{\mu}{2} \| w - w^0 \|^2.
\end{equation}
Note that $\hat{P}$ is $\mu$-strongly convex and $(L+\mu)$-smooth. In particular, the theory developed in the previous section applies. We propose that S2GD be instead applied to the perturbed problem, and show that an approximate solution of \eqref{eq:s2gd:shs7hss} is also an approximate solution of \eqref{eq:s2gd:main} (we will assume that this problem has a minimizer).

Let $\hat{w}^*$ be the (necessarily unique) solution of the perturbed problem \eqref{eq:s2gd:shs7hss}. The following result describes an important connection between the original problem and the perturbed problem.
 
\begin{lemma}
\label{eq:s2gd:lemma87878}
If $\hat{w} \in \R^d$ satisfies $\hat{P}(\hat{w})\leq \hat{P}(\hat{w}^*) + \delta$, where $\delta>0$, then $$ P(\hat{w}) \leq P(w^*) + \frac{\mu}{2} \| w^0 - w^* \|^2 + \delta. $$
\end{lemma}
\begin{proof}
The statement is almost identical to Lemma~9 in \cite{RichtarikTakacIteration}; its proof follows the same steps with only minor adjustments.
\end{proof}

We are now ready to establish a complexity result for non-strongly convex losses.

\begin{theorem} 
Let Assumption~\ref{ass:s2gd:lip} be satisfied. Choose $\mu>0$, $0\leq \nu \leq \mu$, stepsize $0< h  < \tfrac{1}{2(L+\mu)}$, and let $m$ be sufficiently large so that 
\begin{equation}
\hat{c} \eqdef \frac{(1 -\nu h)^m}{\beta \mu h (1 - 2(L+\mu)h)} + \frac{2 Lh}{1 - 2(L+\mu)h} < 1.
\end{equation}
Pick $w^0 \in \R^d$ and let $\hat{w}^0 = w^0, \hat{w}^1,\dots,\hat{w}^{k}$ be the sequence of iterates produced by S2GD as applied to problem \eqref{eq:s2gd:shs7hss}.
Then, for any $ 0 < \rho < 1 $, $ 0 < \varepsilon < 1 $ and  
\begin{equation}
\label{eq:s2gd:hprobs2}
k \geq \frac{\log \left(1/(\varepsilon \rho)\right)}{\log (1/\hat{c})},
\end{equation} 
we have 
\begin{equation}
\label{eq:s2gd:prob2}
\Prob{ P(\hat{w}^{k}) - P(w^*) \leq \varepsilon( P(w^0)-P(w^*) ) + \frac{\mu}{2} \| w^0 - w^* \|^2 } \geq 1 - \rho. 
\end{equation}
In particular, if we choose $\mu=\epsilon< L$ and parameters $k^*$, $h(k^*)$, $m(k^*)$ as in Theorem~\ref{thm:s2gd:main2}, the amount of work  performed by S2GD to guarantee \eqref{eq:s2gd:prob2} is
\[{\mathcal W}(k^*,h(k^*),m(k^*)) = O\left( \left( n+\frac{L}{\varepsilon} \right)\log \left( \frac{1}{\varepsilon} \right) \right),\]
which consists of $O(\tfrac{1}{\varepsilon})$ full gradient evaluations and $O(\tfrac{L}{\epsilon}\log(\tfrac{1}{\varepsilon}))$ stochastic gradient evaluations.
\label{thm:s2gd:hpresult2}
\end{theorem}

\begin{proof}
We first note that
\begin{equation}
\label{eq:s2gd:uddjd8}
\hat{P}(\hat{w}^0)-\hat{P}(\hat{w}^*) \overset{\eqref{eq:s2gd:barf}}{=} P(\hat{w}^0)-\hat{P}(\hat{w}^*) \leq P(\hat{w}^0) - P(\hat{w}^*) \leq P(w^0)-P(w^*),
\end{equation}
where the first inequality follows from $f\leq \hat{f}$, and the second one from optimality of $x_*$. Hence, by first applying
Lemma~\ref{eq:s2gd:lemma87878} with $\hat{w}=\hat{w}^{k}$ and $\delta= \varepsilon(P(w^0)-P(w^*))$, and then Theorem~\ref{thm:s2gd:hpresult}, with $c \leftarrow \hat{c}$, $P\leftarrow \hat{P}$, $w^0\leftarrow \hat{w}^0$, $w^* \leftarrow \hat{w}^* $, we obtain 
\begin{eqnarray*}
\Prob{ P(\hat{w}^{k}) - P(w^*) 
\leq \delta + \frac{\mu}{2}\| w^0 - w^* \|^2 } &\overset{(\text{Lemma~}\ref{eq:s2gd:lemma87878})}{\geq}&
\Prob{ \hat{P}(\hat{w}^{k}) - \hat{P}(\hat{w}^*) \leq \delta } \\
& \overset{\eqref{eq:s2gd:uddjd8}}{\geq} & \Prob { \frac{\hat{P}(\hat{w}^{k})-\hat{P}(\hat{w}^*)}{\hat{P}(\hat{w}^0)-\hat{P}(\hat{w}^*)} \leq \varepsilon } \;\; \overset{\eqref{eq:s2gd:sjnd8djd}}{\geq} \;\; 1-\rho.
\end{eqnarray*}
The second statement follows directly from the second part of Theorem~\ref{thm:s2gd:main2} and the fact that the condition number of the perturbed problem is $\kappa = \tfrac{L+\epsilon}{\epsilon} \leq \tfrac{2L}{\epsilon}$. 
\end{proof}

\section{Implementation for sparse data}
\label{sec:s2gd:sparses2gd}

In our sparse implementation of Algorithm~\ref{alg:s2gd:s2gd}, described in this section and formally stated as Algorithm~\ref{alg:s2gd:S2GDsparse}, we make the following structural assumption:

\begin{assumption}
\label{ass:s2gd:sparse}
The loss functions arise as the composition of a univariate smooth loss function $\ell_i$, and an inner product with a data point/example $a_i\in \R^d$: \[f_i(w) = \ell_i(a_i^T w), \qquad i=1,2,\dots,n.\]  In this case, $ \nabla f_i(w) = \nabla \ell_i(a_i^T w) a_i$.
\end{assumption}

This is the structure in many cases of interest, including linear or logistic regression.  

A natural question one might want to ask is whether S2GD can be implemented efficiently for sparse data. 

Let us first take a brief detour and look at SGD, which performs iterations of the type:
\begin{equation}
\label{eq:s2gd:sgdxx}
w^{k+1} \leftarrow w^k - h \nabla \ell_i(a_i^T w) a_i.
\end{equation}
Let $\omega_i$ be the number of nonzero features in example $a_i$, i.e., $\omega_i \eqdef \|a_i\|_0 \leq d$. Assuming that the computation of the derivative of the univariate function $\ell_i$ takes $O(1)$ amount of work, 
the computation of $\nabla f_i(w)$ will take $O(\omega_i)$ work. Hence, the update step \eqref{eq:s2gd:sgdxx} will cost $O(\omega_i)$, too, which means the method can naturally speed up its iterations on sparse data.

The situation is not as simple with S2GD, which for loss functions of the type described in Assumption~\ref{ass:s2gd:sparse} performs inner iterations as follows:
\begin{equation}
\label{eq:s2gd:s2gd-update}
y^{k,t+1} \gets y^{k,t} - h \left( g^{k} + \nabla \ell_i(a_i^T y^{k,t} )a_i - \nabla \ell_i( a_i^T w^{k}) a_i \right).
\end{equation}
Indeed, note that $g^k = \nabla P(w^k)$ is in general be fully dense even for sparse data $\{a_i\}$. As a consequence, the update in \eqref{eq:s2gd:s2gd-update} might be as costly as $d$ operations, irrespective of the sparsity level $\omega_i$ of the active  example $a_i$. However, we can use the following ``lazy/delayed'' update trick. We split the update to the $y$ vector into two parts: immediate, and delayed. Assume index $i=i_t$ was chosen at  inner iteration $t$. We immediately perform the update 
$$ \tilde{y}^{k,t+1} \leftarrow y^{k,t} - h \left( \nabla \ell_{i_t}(a_{i_t}^T y^{k,t}) - \nabla \ell_{i_t}(a_{i_t}^T w^{k}) \right)a_{i_t}, $$ which costs $O(a_{i_t})$. Note that we have not computed the $y^{k,t+1}$. However, we ``know'' that $$ y^{k,t+1} = \tilde{y}^{k,t+1} - h g^k, $$ without having to actually compute the difference. At the next iteration, we are supposed to perform update \eqref{eq:s2gd:s2gd-update} for $i=i_{t+1}$:
\begin{equation}
\label{eq:s2gd:s2gd-updateXX}
y^{k,t+2} \gets y^{k,t+1} - h g^{k} - h \left(\nabla \ell_{i_{t+1}}(a_{i_{t+1}}^T y^{k,t+1})- \nabla \ell_{i_{t+1}} (a_{i_{t+1}}^T w^{k}) \right)a_{i_{t+1}}.
\end{equation}

\begin{algorithm}[!h]
\begin{algorithmic}
\State \textbf{parameters:} $m$ = max \# of stochastic steps per epoch, $h$ = stepsize, $\nu$ = lower bound on $\mu$
\For {$k = 0, 1, 2, \dots$}
	\State $g^{k} \gets \frac{1}{n} \sum_{i=1}^n \nabla f_i(w^{k})$
	\State $y^{k,0} \gets w^{k}$
	\State $\chi_{(s)} \gets 0$ for $s = 1, 2, \dots, d$	
	\Comment Store when a coordinate was updated last time
	\State Let $t^{k} \gets t$ with probability $(1 - \nu h)^{m-t} / \beta $ for $t = 1, 2, \dots, m$
	\For {$t = 0$ to $t^{k}-1$}
		\State Pick $i \in \{ 1, 2, \dots, n \}$, uniformly at random
		\For {$s \in \text{nnz}(a_i)$}
			\State $y^{k,t}_{(s)} \gets y^{k,t}_{(s)} - (t - \chi_{(s)}) h g^k_{(s)} $
			\Comment Update what will be needed
			\State $\chi_{(s)} = t$ 
		\EndFor
		\State $ y^{k,t+1} \gets y^{k,t} - h \left( \nabla \ell_i(a_i^Ty^{k,t}) - \nabla \ell_i(a_i^Tw^{k}) \right)a_i $
		\Comment A sparse update
	\EndFor
	\For {$ s = 1$ to $d$} \Comment Finish all the ``lazy'' updates 
		\State $y^{k, t^{k}}_{(s)} \gets y^{k, t^{k}}_{(s)} - (t^k - \chi_{(s)}) h   g^k_{(s)} $
	\EndFor
	\State $w^{k+1} \gets y^{k, t^{k}}$
\EndFor
\end{algorithmic}
\caption{Semi-Stochastic Gradient Descent (S2GD) for sparse data; ``lazy'' updates}
\label{alg:s2gd:S2GDsparse}
\end{algorithm}

However, notice that we can't compute 
\begin{equation}
\label{eq:s2gd:iusi89s}
\nabla \ell_{i_{t+1}}(a_{i_{t+1} }^T y^{k,t+1})
\end{equation}
as we never computed $y^{k,t+1}$. However, here lies the trick: as $a_{i_{t+1}}$ is sparse, we only need to know those coordinates $s$ of $y^{k,t+1}$ for which $(a_{i_{t+1}})_{(s)}$ is nonzero. So, just before we compute the (sparse part of) of the update \eqref{eq:s2gd:s2gd-updateXX}, we perform the update 
$$ y^{k,t+1}_{(s)} \leftarrow \tilde{y}^{k,t+1}_{(s)} - h g^k_{(s)} $$
for coordinates $s$ for which $(a_{i_{t+1}})_{(s)}$ is nonzero. This way we know that the inner product appearing in \eqref{eq:s2gd:iusi89s} is computed correctly (despite the fact that $y^{k,t+1}$ potentially is not!). In turn, this means that we can compute the sparse part of the update in \eqref{eq:s2gd:s2gd-updateXX}. 

We now continue as before, again only computing $\tilde{y}^{k,t+3}$. However, this time we have to be more careful as it is no longer true that
$$ y^{k,t+2} = \tilde{y}^{k,t+2} - h g^k. $$
We need to remember, for each coordinate $s$, the last iteration counter $t$ for which $(a_{i_t})_{(s)} \neq 0$. This way we will know how many times did we ``forget'' to apply the dense update $-h g^k_{(s)}$. We do it in a just-in-time fashion, just before it is needed.

Algorithm~\ref{alg:s2gd:S2GDsparse} (sparse S2GD) performs these lazy updates as described above. It produces exactly the same result as Algorithm~\ref{alg:s2gd:s2gd} (S2GD), but is much more efficient for sparse data as  iteration picking example $i$ only costs $O(\omega_i)$. This is done with a memory overhead of only $O(d)$ (as represented by vector $\chi \in \R^d$).

\section{Numerical Experiments}
\label{sec:s2gd:NUMERICS}

In this section we conduct computational experiments to illustrate some aspects of the performance of our algorithm. In Section~\ref{sec:s2gd:theoryvspractice} we consider the least squares problem with synthetic data to compare the practical performance and the theoretical bound on convergence in expectations. We demonstrate that for both SVRG and S2GD, the practical rate is substantially better than the theoretical one. In Section~\ref{sec:s2gd:othermethods} we compare the S2GD algorithm on several real datasets with other algorithms suitable for this task. We also provide efficient implementation of the algorithm, as described in Section~\ref{sec:s2gd:sparses2gd}, for the case of logistic regression in the MLOSS repository\footnote{\url{http://mloss.org/software/view/556/}}.

\subsection{Comparison with theory} 
\label{sec:s2gd:theoryvspractice}

Figure~\ref{fig:s2gd:ThVsPracLin} presents a comparison of the theoretical rate and practical performance on a larger problem with artificial data, with a condition number we can control (and choose it to be poor). In particular, we consider the L2-regularized least squares  with $$ f_i(w) = \frac{1}{2}(a_i^T w - b_i)^2 + \frac{\lambda}{2} \| w \|^2, $$ for some $a_i \in \R^d$, $b_i \in \R$ and $\lambda>0$ is the regularization parameter.

\begin{figure}[!h]
\begin{center}
\includegraphics[width = 5.3in]{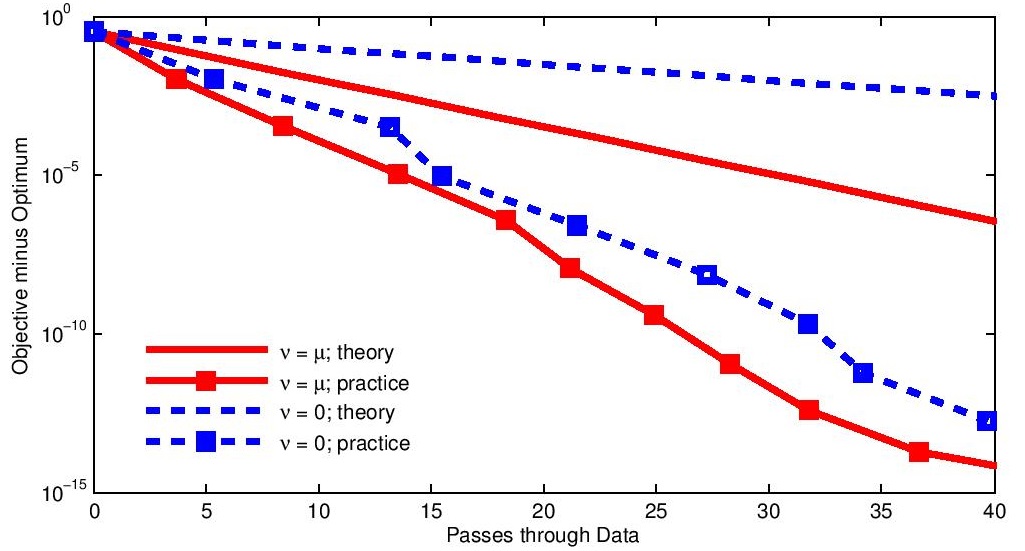}
\end{center}
\caption{Least squares with $n = 10^5$, $\kappa = 10^4$. Comparison of theoretical result and practical performance for cases $\nu = \mu$ (full red line) and $\nu = 0$ (dashed blue line).}
\label{fig:s2gd:ThVsPracLin}
\end{figure}

We consider an instance with $n = 100,000$, $d = 1,000$ and $\kappa = 10,000.$ We run the algorithm with both parameters $\nu = \lambda$ (our best estimate of $\mu$) and $\nu = 0$. Recall that the latter choice leads to the SVRG method of \cite{SVRG}. We chose parameters $m$ and $h$ as a (numerical) solution of the work-minimization problem \eqref{eq:s2gd:pracReq}, obtaining $m = 261,063$ and $h = 1/11.4L$ for $\nu = \lambda$ and $m = 426,660$ and $h = 1/12.7L$ for $\nu = 0$. The practical performance is obtained after a single run of the S2GD algorithm. 

The figure demonstrates linear convergence of S2GD in practice, with the convergence rate being significantly better than the already strong theoretical result. Recall that the bound is on the expected function values. We can observe a rather strong convergence to machine precision in work equivalent to evaluating the full gradient only $40$ times. Needless to say, neither SGD nor GD have such speed. Our method is also an improvement over \cite{SVRG}, both in theory and practice.

\subsection{Comparison with other methods}
\label{sec:s2gd:othermethods}

The S2GD algorithm can be applied to several classes of problems. We perform experiments on an important and in many applications used L2-regularized logistic regression for binary classification on several datasets. The functions $f_i$ in this case are: $$ f_i(w) = \log\left(1 + \exp\left(b_i a_i^T w \right) \right) + \frac{\lambda}{2} \| w \|^2, $$ where $b_i \in \{-1, +1\}$ is the label of $i^{th}$ training example $a_i$. In our experiments we set the regularization parameter $\lambda = \Theta(1/n)$ so that the condition number $\kappa = \Theta(n)$, which is about the most ill-conditioned problem used in practice. We added a (regularized) bias term to all datasets.

All the datasets we used, listed in Table~\ref{tbl:s2gd:datasets}, are freely available\footnote{Available at \href{http://www.csie.ntu.edu.tw/~cjlin/libsvmtools/datasets/}{http://www.csie.ntu.edu.tw/$\sim$cjlin/libsvmtools/datasets/}.} benchmark binary classification datasets.

\begin{table}[!h]
\centering
\begin{tabular}{| r | r | r | r | r | r |} 
\hline
Dataset & Training examples ($n$) & Variables ($d$) & $L$ & $\lambda$ & $\kappa$ \\
\hline
\textit{ijcnn} & 49 990 & 23 & 1.23 & 1/$n$ & 61 696 \\
\textit{rcv1} & 20 242 & 47 237 & 0.50 & 1/$n$ & 10 122 \\
\textit{real-sim} & 72 309 & 20 959 & 0.50 & 1/$n$ & 36 155 \\
\textit{url} & 2 396 130 & 3 231 962 & 128.70 & 100/$n$ & 3 084 052 \\
\hline 
\end{tabular}
\caption{Datasets used in the experiments.}
\label{tbl:s2gd:datasets}
\end{table}

In the experiment, we compared the following algorithms:
\begin{itemize}
\item \textbf{SGD:} Stochastic Gradient Descent. After various experiments, we decided to use a variant with constant step-size that gave the best practical performance in hindsight.
\item \textbf{L-BFGS:} A publicly-available limited-memory quasi-Newton method that is suitable for broader classes of problems. We used a popular implementation by Mark Schmidt.\footnote{\href{http://www.di.ens.fr/~mschmidt/Software/minFunc.html}{http://www.di.ens.fr/$\sim$mschmidt/Software/minFunc.html}}
\item \textbf{SAG:} Stochastic Average Gradient, \cite{SAGjournal2013}. This is the most important method to compare to, as it also achieves linear convergence using only stochastic gradient evaluations. Although the methods has been analyzed for stepsize $h = 1/16L$, we experimented with various stepsizes and chose the one that gave the best performance for each problem individually.
\item \textbf{SDCA:} Stochastic Dual Coordinate Ascent, where we used approximate solution to the one-dimensional dual step, as in Section 6.2 of \cite{SDCA}.
\item \textbf{S2GDcon:} The S2GD algorithm with conservative stepsize choice, i.e., following the theory. We set $m = \Theta(\kappa)$ and $h = 1/10L$, which is approximately the value you would get from Equation~\eqref{eq:s2gd:shhsdd998}.
\item \textbf{S2GD:} The S2GD algorithm, with stepsize that gave the best performance in hindsight. The best value of $m$ was between $n$ and $2n$ in all cases, but optimal $h$ varied from $1/2L$ to $1/10L$.
\end{itemize}

Note that SAG needs to store $n$ gradients in memory in order to run. In case of relatively simple functions, one can store only $n$ scalars, as the gradient of $f_i$ is always a multiple of $a_i$. If we are comparing with SAG, we are implicitly assuming that our memory limitations allow us to do so. Although not included in Algorithm~\ref{alg:s2gd:s2gd}, we could also store these gradients we used to compute the full gradient, which would mean we would only have  to compute a single stochastic  gradient per inner iteration (instead of two).

We plot the results of these methods, as applied to various different, in the Figure~\ref{fig:s2gd:plotTogether} for first  15-30 passes through the data (i.e., amount of work work equivalent to 15-30 full gradient evaluations).

\begin{figure}[!h]
\begin{center}
\includegraphics[width = \linewidth]{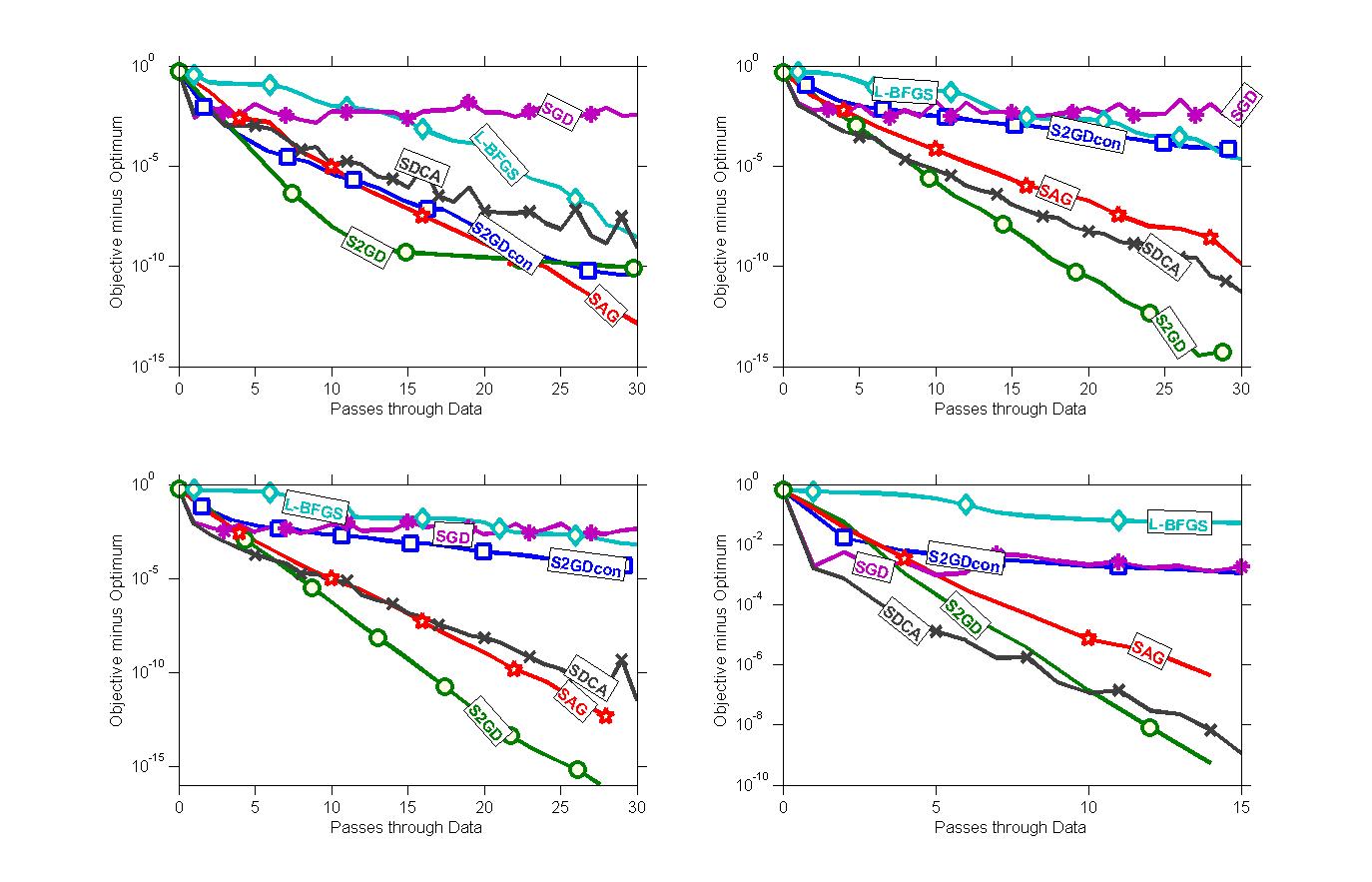}
\end{center}
\caption{Practical performance for logistic regression and  the following datasets: \textit{ijcnn, rcv} (first row), \textit{realsim, url} (second row)}
\label{fig:s2gd:plotTogether}
\end{figure}

There are several remarks  we would like to make. First, our experiments confirm the insight from \cite{SAGjournal2013} that for this types of problems, reduced-variance methods consistently exhibit substantially better performance than the popular L-BFGS algorithm.

The performance gap between S2GDcon and S2GD differs from dataset to dataset. A possible explanation for this can be found in an extension of SVRG to proximal setting \cite{proxSVRG}, released after the first version of this work was put onto arXiv (i.e., after December 2013) . Instead Assumption~\ref{ass:s2gd:lip}, where all loss functions are assumed to be associated with  the same constant $L$, the authors of \cite{proxSVRG} instead assume that each loss function $f_i$ has its own constant  $L_i$. Subsequently, they sample proportionally to these quantities as opposed to the  uniform sampling. In our case, $L = \max_i L_i$. This weighted sampling has an impact on the convergence: one gets dependence on the average of the quantities $L_i$ and not in their maximum.

The number of passes through data seems a reasonable way to compare performance, but some algorithms could need more time to do the same amount of passes through data than others. In this sense, S2GD should be in fact faster than SAG due to the following property. While SAG updates the test point after each evaluation of a stochastic gradient, S2GD does not always make the update --- during the evaluation of the full gradient. This claim is supported by computational evidence: SAG needed about $20-40\%$ more time than S2GD to do the same amount of passes through data.

Finally, in Table~\ref{tbl:s2gd:experimentsTime} we provide the time it took the algorithm  to produce these plots on a desktop computer with Intel Core i7 3610QM processor, with 2 $\times$ 4GB DDR3 1600 MHz memory. The numbers for the \textit{url} dataset is are not representative, as the algorithm needed extra memory, which slightly exceeded the memory limit of our computer.

\begin{table}
\centering
\begin{tabular}{c|r|r|r|r|}
\cline{2-5}
 & \multicolumn{4}{c|}{Time in seconds} \\
\hline
\multicolumn{1}{|c|}{Algorithm} & \textit{ijcnn} & \textit{rcv1} & \textit{real-sim} & \textit{url} \\
\hline
\multicolumn{1}{|c|}{S2GDcon} & 0.25 & 0.43 & 1.01 & 125.53 \\
\multicolumn{1}{|c|}{S2GD}    & 0.29 & 0.49 & 1.02 & 54.04 \\
\multicolumn{1}{|c|}{SAG}     & 0.41 & 0.73 & 1.87 & 71.74 \\
\multicolumn{1}{|c|}{L-BFGS}  & 0.15 & 0.67 & 0.76 & 309.14 \\
\multicolumn{1}{|c|}{SGD}     & 0.39 & 0.57 & 1.54 & 62.73 \\
\multicolumn{1}{|c|}{SDCA}    & 0.33 & 0.38 & 1.10 & 126.32 \\
\hline
\end{tabular}
\caption{Time required to produce plots in Figure~\ref{fig:s2gd:plotTogether}.}
\label{tbl:s2gd:experimentsTime}
\end{table}

\subsection{Boosted variants of S2GD and SAG}
\label{sec:plusVariants}

In this section we study the practical performance of boosted methods, namely S2GD+ (Algorithm~\ref{alg:s2gd:S2GD+}) and variant of SAG suggested by its authors \cite[Section 4.2]{SAGjournal2013}.

\begin{figure}[!h]
\begin{center}
\includegraphics[angle = 270, width = \linewidth]{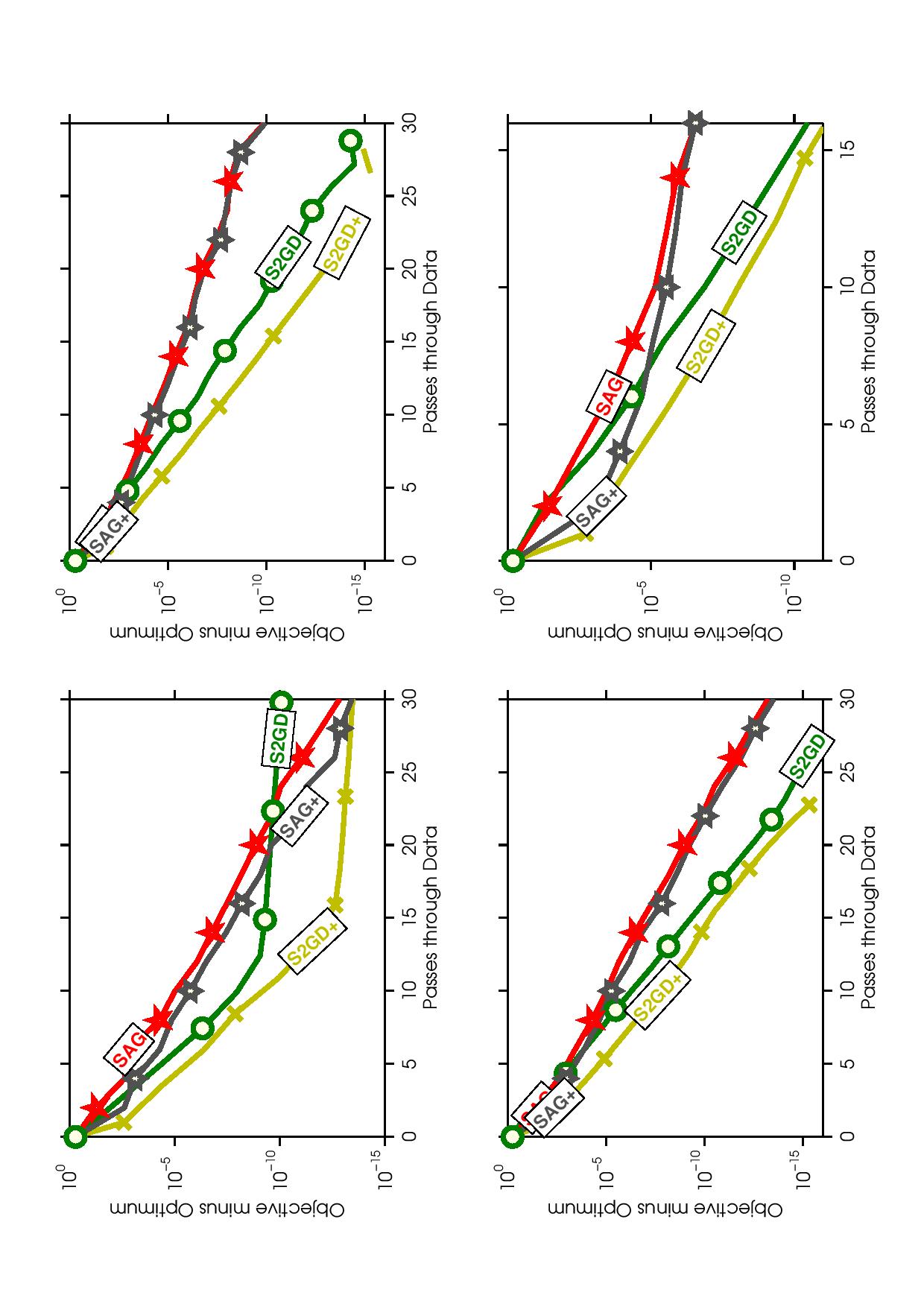}
\end{center}
\caption{Practical performance of boosted methods on datasets \textit{ijcnn, rcv} (first row), \textit{realsim, url} (second row)}
\label{fig:s2gd:plotTogetherPlus}
\end{figure}

SAG+ is a simple modification of SAG, where one does not divide the sum of the stochastic gradients by $n$, but by the number of training examples seen during the run of the algorithm, which has the effect of producing larger steps at the beginning. The authors claim that this method performed better in practice than a hybrid SG/SAG algorithm. 

We have observed that, in practice, starting SAG from a point close to the optimum, leads to an initial  ``away jump``. Eventually, the method exhibits   linear convergence. In contrast, S2GD converges linearly from the start, regardless of the starting position.

Figure~\ref{fig:s2gd:plotTogetherPlus} shows that S2GD+ consistently improves over S2GD, while SAG+ does not improve always: sometimes it performs essentially the same as SAG. Although S2GD+ is overall a superior algorithm, one should note that this comes at the cost of having to choose stepsize parameter for SGD initialization. If one chooses these parameters poorly, then S2GD+ could perform worse than S2GD. The other three algorithms can work well without any parameter tuning.

\section{Conclusion} 
\label{sec:s2gd:CONCLUDE}
We have developed a new semi-stochastic gradient descent method (S2GD) and analyzed its complexity for smooth convex and strongly convex loss functions. Our methods  need $O((\kappa/n)\log(1/\varepsilon))$ work only, measured in  units equivalent to the evaluation of the full gradient of the loss function, where $\kappa=L/\mu$  if the loss is $L$-smooth and $\mu$-strongly convex, and $\kappa\leq 2L/\varepsilon$ if the loss is merely $L$-smooth.

Our results in the strongly convex case match or improve on a few very recent results, while at the same time generalizing and simplifying the analysis. Additionally, we proposed  S2GD+ ---a method which equips S2GD with an SGD pre-processing step---which in our experiments exhibits superior performance to all methods we tested. We left the analysis of this method as an open problem.

\chapter{Semi-Stochastic Coordinate Descent}
\label{ch:s2cd}

\section{Introduction}

In this chapter we study the problem of unconstrained minimization of a strongly convex function represented as the average of a large number of smooth convex functions:
\begin{equation} 
\label{eq:s2cd:maineq} 
\min_{\ovar \in \R^d} P(\ovar) \equiv \frac{1}{n} \sum_{i=1}^n f_i(\ovar).
\end{equation}

Many computational problems in various disciplines are of this form. In machine learning,  $f_i(\ovar)$ represents the loss/risk of classifier $\ovar \in \R^d$ on data sample $i$, $P$ represents the empirical risk (=average loss), and the goal is to find a predictor minimizing $P$. An L2-regularizer of the form $\tfrac{\mu}{2} \|\ovar\|^2$, for $\mu>0$, could be added to the loss, making it strongly convex and hence easier to minimize.

\paragraph*{Assumptions.}  We assume that the functions  $f_i : \R^d \rightarrow \R$ are  differentiable and convex function, with Lipschitz continuous partial derivatives. Formally, we assume that for each  $i \in [n]\eqdef \{1,2,\dots,n\}$ and  $j \in [d]\eqdef \{1,2,\dots,d\}$ there exists $L_{ij}\geq 0$ such that for all $x\in \R^d$ and $h\in \R$, 
\begin{equation}
\label{eq:s2cd:sjs7shd}
f_\rsample(x + he_\csample) \leq f_\rsample(x) + \left\langle \nabla f_\rsample(x), h e_\csample \right\rangle + \frac{L_{\rsample \csample}}{2} h ^2,
\end{equation}
where $e_\csample$ is the $j^{th}$ standard basis vector in $\R^d$, $\nabla f_\rsample(x)\in\R^d$ is the gradient of $f_\rsample$ at point $x$ and $\ip{\cdot}{\cdot}$ is the standard inner product.  This assumption was recently used in the analysis of the   accelerated coordinate descent method APPROX \cite{APPROX}. We further assume that $P$ is $\mu$-strongly convex. That is, we assume that there exists $\mu>0$ such that for all $x,y\in \R^d$,
\begin{equation}
\label{s2cd:SVRGstrcvx}
P(y) \geq P(x) + \langle \nabla P(x), y - x \rangle + \frac{\mu}{2} \| y - x \|^2.
\end{equation}

\paragraph*{Context.}  Batch methods such as gradient descent (GD) enjoy a fast  (linear) convergence rate: to achieve $\epsilon$-accuracy, GD  needs $\mathcal{O}(\kappa\log(1 / \epsilon))$ iterations, where $\kappa$ is a condition number. The drawback of GD is  that in each iteration one needs to compute the gradient of $P$, which requires a pass through the entire dataset. This is prohibitive to do many times if $n$ is very large.

Stochastic gradient descent (SGD)  in each iteration computes the gradient of  a single randomly chosen function $f_i$ only---this constitutes  an unbiased (but noisy) estimate of the gradient of $P$---and makes a step in that direction \cite{RM1951,nemirovski2009robust, tongSGD,pegasos}. The rate of convergence of SGD is slower, $\mathcal{O}(1 / \epsilon)$,  but the cost of each iteration is independent of $n$. Variants with nonuniform selection probabilities were considered in \cite{IProx-SDCA}, a mini-batch variant (for SVMs with hinge loss) was analyzed  in \cite{takac-minibatch}.

Recently, there has been progress in designing algorithms that achieve the fast $O(\log(1/\epsilon))$ rate without the need to scan the entire dataset in each iteration. The first class of methods to have achieved this are stochastic/randomized coordinate descent methods.

When applied to \eqref{eq:s2cd:maineq}, coordinate descent methods (CD) \cite{nesterovRCDM, RichtarikTakacIteration}  can, like SGD,  be seen as an attempt to keep the benefits of GD (fast linear convergence) while reducing the complexity of each iteration. A CD method only computes a single partial derivative  $\nabla_j P(\ovar)$ at each iteration  and updates a single coordinate of vector $w$ only. When  chosen uniformly at random, partial derivative is also an unbiased estimate of the gradient. However, unlike the SGD estimate, its variance decrease to zero as one approaches the optimum.  While CD methods are able to obtain linear convergence, they typically need $O((d/\mu)\log(1/\epsilon))$ iterations when applied to \eqref{eq:s2cd:maineq} directly\footnote{The complexity can be improved to $O(\tfrac{d\alpha}{\tau \mu}\log(1/\epsilon))$ in the case when $\tau$ coordinates are updated in each iteration, where $\alpha \in [1,\tau]$ is a problem-dependent constant \cite{RT:PCDM}. This has been further studied  for nonsmooth problems via smoothing \cite{FR:SPCDM2013}, for arbitrary nonuniform distributions governing the selection of coordinates \cite{nsync,QUARTZ} and in the distributed setting \cite{richtarik2013distributed, fercoq2014fast,QUARTZ}. Also, efficient accelerated variants with $O(1/\sqrt{\epsilon})$ rate were developed \cite{APPROX, fercoq2014fast}, capable of solving problems with 50 billion variables.}. CD method typically significantly outperform GD, especially  on sparse problems with a very large number of variables/coordinates \cite{nesterovRCDM, RichtarikTakacIteration}. 

An alternative to applying CD to \eqref{eq:s2cd:maineq} is to apply it to the dual problem. This is possible under certain additional structural assumptions on the functions $f_i$. This is the strategy employed by stochastic dual coordinate ascent (SDCA) \cite{SDCA,QUARTZ}, whose rate is \[O((n+\kappa)\log(1/\epsilon)).\] The condition number $\kappa$ here is the same as the condition number appearing in the rate of GD.  Despite this, this is a vast improvement on the computational complexity
achieved by GD which has an iteration cost $n$ times larger than SDCA. Also, the linear convergence rate  is superior to the sublinear rate $O(1/\epsilon)$ achieved by SGD, and the method indeed  typically performs much better in practice. Accelerated \cite{ShalevShwartz:2014dy} and mini-batch \cite{takac-minibatch} variants of SDCA have also been  proposed. We refer the reader to QUARTZ \cite{QUARTZ} for a general analysis involving the update of a random subset of dual coordinates, following an arbitrary distribution.

Recently, there has been progress in designing primal methods which match the fast rate of SDCA.  Stochastic average gradient (SAG)  \cite{SAGjournal2013}, and more recently SAGA \cite{saga}, move in a direction composed of old stochastic gradients. The semi-stochastic gradient descent (S2GD) \cite{S2GD, konecny2014mS2GD} and stochastic variance reduced gradient (SVRG) \cite{SVRG, proxSVRG} methods employ a different strategy: one first computes the gradient of $P$, followed by $O(\kappa)$ steps where only stochastic gradients are computed. These are used to estimate the change of the gradient, and it is this direction which combines the old gradient and the new stochastic gradient information which is used in the update.

\paragraph*{Main result.} In this work we develop a new  method---semi-stochastic coordinate descent (S2CD)---for solving \eqref{eq:s2cd:maineq}, enjoying a fast rate similar to methods such as SDCA, SAG, S2GD, SVRG, SAGA, mS2GD and QUARTZ. S2CD can be seen as a hybrid between S2GD and CD. In particular, the complexity of our method  is the sum of two terms: \[O(n\log (1/\epsilon))\] evaluations $\nabla f_i$ (that is, $\log(1/\epsilon)$ evaluations of the gradient of $P$) and \[O(\hat{\kappa} \log(1/\epsilon))\] evaluations of $\langle e_j, \nabla f_i\rangle$ for randomly chosen functions $f_i$ and randomly chosen coordinates $j$, where $\hat{\kappa}$ is a new condition number which is defined in ~\eqref{eq:s2cd:defhatkappa} and  larger than $\kappa$.   We summarize in Table~\ref{s2cd:tab1} the runtime complexity of the various algorithms. Note that $\hat{\kappa}$ enters the complexity only in the term involving the evaluation cost of a partial derivative $\nabla_j f_i$, which can be substantially smaller than the evaluation cost of  $\nabla f_i$. Hence, our complexity result can be both  better or worse than previous results, depending on whether the increase of the condition number can or can not be compensated by the lower cost of the stochastic steps based on the evaluation of partial derivatives.

\begin{table}[!t]
\centering
\renewcommand{\arraystretch}{1.5}
\begin{tabular}{|c|c|c|}
\hline
Method & Runtime & paper \\
\hline 
\begin{tabular}{c}CD 
\end{tabular} & $\displaystyle 
 \mathcal{O}(n \kappa \mathcal{C}_{grad}\log(1/\epsilon))$ &  e.g. \cite{nesterov2004convex}\\ \hline 
SGD &   $\displaystyle 
 \mathcal{O}( \mathcal{C}_{grad}/\epsilon)$ &  \cite{tongSGD,pegasos} \\ \hline 
CD & $\mathcal{O}(n\kappa\mathcal{C}_{pd}  \log(1/\epsilon)) $  &  \cite{nesterovRCDM, RichtarikTakacIteration} \\   \hline 
 SDCA & $\displaystyle 
 \mathcal{O}((n +\kappa) \mathcal{C}_{grad}\log(1/\epsilon))$&  \cite{SDCA,IProx-SDCA,QUARTZ}\\ \hline 
  SVRG/S2GD &  $\mathcal{O}\left( (n  \mathcal{C}_{grad} +  \kappa \mathcal{C}_{grad}) \log \left( {1}/{\epsilon}\right)\right)$&   \cite{SVRG, proxSVRG,S2GD}\\ \hline 
S2CD & $ \mathcal{O} \left( (n  \mathcal{C}_{grad} +  \hat{\kappa} \mathcal{C}_{pd}) \log \left({1}/{\epsilon}\right) \right) $& this work \cite{S2CD} \\
\hline
\end{tabular}
\caption{Runtime complexity of various algorithms. We use $\mathcal{C}_{grad}$ to denote the
the evaluation cost of the gradient of one single function  $\nabla f_i$  and use $\mathcal{C}_{pd}$ to denote the evaluation cost of a partial derivative $\nabla_j f_i$.}
\label{s2cd:tab1}
\end{table}

\paragraph*{Outline.} This chapter is organized as follows. In Section~\ref{sec:s2cd:algo} we describe the S2CD algorithm and in Section~\ref{sec:s2cd:complexity} we state a key lemma and our main complexity result. The proof of the lemma is provided in Section~\ref{subsec:s2cd:key_lemma} and the proof of the main result in Section~\ref{sec:s2cd:proofs}.

\section{S2CD Algorithm}
\label{sec:s2cd:algo}

In this section we describe the Semi-Stochastic Coordinate Descent method (Algorithm~\ref{alg:s2cd:S2CD}). 

\begin{algorithm}[h]
\begin{algorithmic}
\State \textbf{parameters:}  $m$ (max \# of stochastic steps per epoch);  $h>0$ (stepsize parameter);  $\ovar^0\in \R^d$
\For {$\oidx = 0, 1, 2, \dots$}
	\State  Compute and store $ \nabla P(\ovar^{\oidx}) = \tfrac{1}{n}\sum_i \nabla f_i(\ovar^{\oidx})$
	\State Initialize the inner loop: $\ivar^{\oidx, 0} \gets \ovar^{\oidx}$
	\State Let $\iidx^{\oidx} = T \in \{1,2,\dots,m\}$ with probability $\left(1 - \mu h\right)^{m-T} / \beta$ 
	\For {$\iidx = 0$ to $\iidx^{\oidx}-1$}
                \State Pick coordinate $\csample \in \{ 1, 2, \dots, d \}$ with probability $p_\csample$
		\State Pick function index $\rsample$ from the set $\{i\;:\;L_{ij}>0\}$ with probability  $q_{\rsample\csample}$
		\State $ \ivar^{\oidx, \iidx+1} \gets \ivar^{\oidx,\iidx} - h p_{\csample}^ {-1}  \big( \nabla_{\csample} P(\ovar^{\oidx}) + \frac{1}{n \weightinG_{\rsample \csample}} \left( \nabla_{\csample} f_{\rsample}(\ivar^{\oidx, \iidx}) - \nabla_{\csample} f_{\rsample}(\ovar^{\oidx}) \right) \big) e_{\csample}$
	\EndFor
	\State Reset the starting point: $\ovar^{\oidx+1} \gets \ivar^{\oidx, \iidx^{\oidx}}$
\EndFor
\end{algorithmic}

\caption{Semi-Stochastic Coordinate Descent (S2CD)}
\label{alg:s2cd:S2CD}
\end{algorithm}

The discussion on the choice of $m$ and $h$ in Algorithm~\ref{alg:s2cd:S2CD} is deferred to Section~\ref{sec:s2cd:complexity}. As we will see, the parameters $m$ and $h$ depends on the target accuracy and the number of iterations. We next provide a more detailed description of the algorithm.

The method has an outer loop (an ``epoch''), indexed by  counter $\oidx$, and an inner loop, indexed by $\iidx$. At the beginning of epoch $k$, we compute and store the gradient of  $f$ at $\ovar^\oidx$.  Subsequently, S2CD enters the inner loop in which 
a sequence of vectors $\ivar^{\oidx,\iidx}$ for $t=0,1\dots,t^k$ 
is computed in a stochastic way, starting from $\ivar^{\oidx,0}=\ovar^{\oidx}$. The number $t^{\oidx}$ of stochastic steps in the inner loop is  random, following a geometric law: 
\[\Prob{(t^k=T)}= \frac{(1-
\mu h )^{m-T}}{\beta}, 
\qquad T\in\{1,\dots,m\},\]
where
\begin{align}
\label{def:s2cd:beta}
\beta \eqdef \sum_{t = 1}^m (1 - \mu h )^{m-t}.\end{align}
In each step of the inner loop, we seek to compute $\ivar^{\oidx,\iidx+1}$, given $\ivar^{\oidx,\iidx}$. In order to do so, we sample coordinate $j$ with probability $p_j$ and subsequently\footnote{In S2CD, as presented, coordinates $\csample$ is selected first, and then function $\rsample$ is selected, according to a distribution conditioned on the choice of $\csample$. However,  one could equivalently sample $(\rsample, \csample)$ with joint probability $p_{\rsample \csample}$. We opted for the sequential sampling for clarity of presentation purposes.} sample $i$ with probability $q_{ij}$, where the probabilities are given by
\begin{equation}
\label{eq:s2cd:sjs7tbjd}
\omega_\rsample \eqdef | \{ j : L_{\rsample \csample} \neq 0 \} |,\quad \weightofnorm_\csample \eqdef \sum_{\rsample = 1}^ n \omega_\rsample L_{\rsample \csample}, 
\quad
p_{\csample} \eqdef \weightofnorm_\csample / \sum_{\csample = 1}^d \weightofnorm_\csample, \quad  q_{\rsample \csample} \eqdef \frac{\omega_\rsample L_{\rsample \csample}}{\weightofnorm_j}, \quad  p_{\rsample \csample} \eqdef p_\csample q_{\rsample \csample}.
\end{equation}
Note that $L_{\rsample \csample}=0$ means that function $f_\rsample$ does not depend on the $\csample^{th}$ coordinate of $x$. Hence, $\omega_\rsample$ is the number of coordinates  function $f_\rsample$ depends on -- a measure of sparsity of the data\footnote{The quantity $\omega\eqdef \max_i \omega_i$ (degree of partial separability of $P$) was used in the analysis of a large class of randomized parallel  coordinate descent methods in \cite{RT:PCDM}. The more informative quantities $\{\omega_i\}$ appear in the analysis of parallel/distributed/mini-batch coordinate descent methods \cite{richtarik2013distributed, APPROX, fercoq2014fast}.}. It can be shown that $P$ has a $1$-Lipschitz gradient with respect to the weighted Euclidean norm with weights $\{v_j\}$ (\cite[Theorem 1]{APPROX}). Hence, we sample coordinate $j$ proportionally to this weight $v_j$. Note that $p_{\rsample \csample}$ is the joint probability of choosing the pair $(\rsample, \csample)$.

Having sampled coordinate $j$ and function index $i$, we compute two partial derivatives: $\nabla_\csample f_\rsample(\ovar^\oidx)$ and $\nabla_\csample f_\rsample(\ivar^{\oidx, \iidx})$ (we compressed the notation here by writing $\nabla_j f_i(\ovar)$ instead of $\ip{ \nabla f_i(\ovar)}{e_j}$), and combine these with the pre-computed value $\nabla_j P(\ovar^{\oidx})$ to form an update of the form 
\begin{equation}
\label{eq:s2cd:87gsb8s9} 
\ivar^{\oidx,\iidx+1} \leftarrow \ivar^{\oidx,\iidx} - h p_j^{-1} G_{\rsample\csample}^{\oidx\iidx} e_j = \ivar^{\oidx,\iidx} - h \grad_{\rsample\csample}^{\oidx\iidx},
\end{equation}
where 
\begin{equation}
\label{eq:s2cd:g_kl}
\grad_{\rsample\csample}^{\oidx\iidx} \eqdef p_j^{-1} G_{\rsample\csample}^{\oidx\iidx} e_j
\end{equation} 
and
\begin{equation}
\label{eq:s2cd:0j9j0s9s}
G_{\rsample\csample}^{\oidx\iidx} \eqdef \nabla_{\csample} P(\ovar^{\oidx}) + \frac{1}{n \weightinG_{\rsample \csample}} \left( \nabla_{\csample} f_{\rsample}(\ivar^{\oidx, \iidx}) - \nabla_{\csample} f_{\rsample}(\ovar^{\oidx}) \right).
\end{equation}  
Note that only a single coordinate of $\ivar^{\oidx,\iidx}$ is updated at each iteration. 

In the entire text (with the exception of the statement of Theorem~\ref{thm:s2cd:S2CD} and a part of Section~\ref{eq:s2cd:s98h9s8h}, where $\E{\cdot}$ denotes the total expectation) we will assume that all expectations are conditional on the entire history of the random variables generated up to the point when $y^{k,t}$ was computed. With this convention, it is possible to think that there are only two random variables: $j$ and $i$.  By $\E{\cdot}$ we then mean the expectation with respect to both of these random variables, and by $\Ei{\cdot}$ we mean expectation with respect to $i$ (that is, conditional on $j$). With this convention,  we can write  
\begin{eqnarray}
\label{s2cd:EGij}
\Ei{G_{\rsample\csample}^{\oidx\iidx}}
& = & \sum_{i=1}^n q_{\rsample\csample} G_{\rsample\csample}^{\oidx\iidx} \notag\\
& \overset{\eqref{eq:s2cd:0j9j0s9s}}{=} & \nabla_{\csample} P(\ovar^{\oidx}) + \frac{1}{n} \sum_{i=1}^n \left( \nabla_{\csample} f_{\rsample}(\ivar^{\oidx, \iidx}) - \nabla_{\csample} f_{\rsample}(\ovar^{\oidx}) \right) \;\;\overset{\eqref{eq:s2cd:maineq} }{=}\;\; \nabla_j P(\ivar^{\oidx,\iidx}),
\label{eq:s2cd:s98j6dd}
\end{eqnarray}
which means that conditioned on $j$, $G_{\rsample\csample}^{\oidx\iidx}$ is an unbiased estimate of the $\csample^{th}$ partial derivative of $P$ at $\ivar^{\oidx,\iidx}$. An equally easy calculation reveals that  the random vector
$g_{\rsample\csample}^{\oidx\iidx}$ is an unbiased estimate of the  gradient of $P$ at $\ivar^{\oidx,\iidx}$:
\begin{eqnarray*}
\E{ \grad_{\rsample\csample}^{\oidx\iidx} } &\overset{\eqref{eq:s2cd:g_kl}}{=}& \E{ p_\csample^{-1} G_{\rsample\csample}^{\oidx\iidx} e_\csample } = \E{ \Ei{ p_\csample^{-1} G_{\rsample\csample}^{\oidx\iidx} e_\csample } } \\
&=& \E{ p_{\csample}^{-1} e_\csample \Ei{ G_{\rsample\csample}^{\oidx\iidx} } }
\;\; \overset{\eqref{eq:s2cd:s98j6dd}}{=} \;\;   \E{ p_{\csample}^{-1} e_\csample \nabla_\csample P(\ivar^{\oidx,\iidx}) } \;\;  = \;\; \nabla P(\ivar^{\oidx,\iidx}).
\end{eqnarray*}

Hence, the update step performed by S2CD is a stochastic gradient step of fixed stepsize $h$.  

Before we describe our main complexity result in the next section, let us briefly comment on a few special cases of S2CD:

\begin{itemize}
\item If  $n = 1$ (this can be always achieved simply by grouping all functions in the average into a single function), S2CD  reduces to a stochastic CD algorithm with importance sampling\footnote{A   parallel CD method  in which every subset of coordinates can be assigned a different probability of being chosen/updated was analyzed in \cite{nsync}.}, as studied in \cite{nesterovRCDM,RichtarikTakacIteration, QUARTZ}, but written with many redundant computations. Indeed, the method in this case  does not require the $\ovar^\oidx$ iterates, nor does it need to compute the gradient of $P$, and instead takes on the form: $$ \ivar^{0,\iidx+1} \gets \ivar^{0,\iidx} - h p_\csample^{-1} \nabla_\csample P(\ivar^{0,\iidx})e_\csample, $$ where $p_\csample=L_{1\csample}/\sum_s {L_{1s}}$.  

\item It is possible to extend the S2CD algorithm and results to the case when coordinates are replaced by (nonoverlapping) blocks of coordinates, as in \cite{RichtarikTakacIteration} --- we did not do it here for the sake of keeping the notation simple. In such a setting, we would obtain semi-stochastic {\em block} coordinate descent. In the special case with {\em all variables forming a single block}, the algorithm reduces to the S2GD method described in \cite{S2GD}, but with nonuniform probabilities for the choice of $i$ --- proportional to the Lipschitz constants of the gradient of the functions $f_i$ (this is also studied in \cite{proxSVRG}). As in \cite{proxSVRG}, the complexity result then depends on the average of the Lipschitz constants.
\end{itemize}

Note that the algorithm, as presented, assumes knowledge of the strong convexity parameter $\mu$. We have done this for simplicity of exposition: the method works also if $\mu$ is not explicitly known --- in that case, we can simply replace $\mu$ by $0$ and the method will still depend on the true strong convexity parameter. The change to the complexity results will be only minor in constants and all our conclusions hold. Likewise, it is possible to give an $O(1/\epsilon)$ complexity result in the non-strongly convex case of $P$, using standard regularization arguments (e.g., see \cite{S2GD}).

\section{Complexity Result}
\label{sec:s2cd:complexity}

In this section, we state and describe our complexity result; the proof is provided in Section~\ref{sec:s2cd:proofs}.

An important step in our analysis is  proving a  good upper bound on the variance of the (unbiased) estimator $\grad_{\rsample\csample}^{\oidx\iidx} = p_\csample^{-1}G_{\rsample\csample}^{\oidx\iidx}e_\csample$ of $\nabla P(\ivar^{\oidx,\iidx})$, one that we can ``believe'' would  diminish to zero as the algorithm progresses.  This is important for several reasons. First, as the method approaches the optimum, we wish $\grad_{\rsample\csample}^{\oidx\iidx}$ to be progressively closer to the true gradient, which in turn will be close to zero. Indeed, if this was the case, then   S2CD behaves like gradient descent with fixed stepsize $h$ close to optimum. In particular, this would indicate that using fixed stepsizes makes sense. 

In light of the above discussion, the following lemma plays a key role in our analysis:

\begin{lemma}
\label{lem:s2cd:main} 
The iterates of the S2CD algorithm satisfy 
\begin{equation}
\label{eq:s2cd:iuhs98s} 
\E{ \left\| \grad_{\rsample\csample}^{\oidx\iidx} \right\|^2 } \leq 4 \hL \left( P(\ivar^{\oidx, \iidx}) - P(\ovar^*)\right) + 4 \hL \left( P(\ovar^{\oidx}) - P(\ovar^*) \right),
\end{equation}
where
\begin{equation}
\label{eq:s2cd:us886vs5} 
\hL := \frac{1}{n}\sum_{\csample = 1}^d \weightofnorm_\csample \overset{\eqref{eq:s2cd:sjs7tbjd}}{=} \frac{1}{n} \sum_{\csample = 1}^d \sum_{\rsample = 1}^n \omega_\rsample L_{\rsample\csample}.
\end{equation} 
\end{lemma}

The proof of this lemma can be found in Section~\ref{subsec:s2cd:key_lemma}.

Note that as $\ivar^{\oidx,\iidx} \to \ovar^*$ and $\ovar^{\oidx} \to \ovar^*$, the bound \eqref{eq:s2cd:iuhs98s}  decreases to zero. This  is the main feature of modern fast stochastic gradient methods: the squared norm of the stochastic gradient estimate progressively diminishes to zero, as the method progresses, in expectation. Therefore it is possible to use constant step-size in this type of algorithms.
Note that the standard SGD method does not have this property: indeed, there is no reason for $\Ei{ \|\nabla f_i(\ovar) \|^2 }$ to be small even if $\ovar=\ovar^*$.
 
We are now ready to state the main result of this chapter.

\begin{theorem}[Complexity of S2CD]
\label{thm:s2cd:S2CD}
If $0< h  < 1/ (2\hL)$, then for all $\oidx \geq 0$ we have:
\footnote{It is possible to modify the argument slightly and replace the term $\hat{L}$ appearing in the {\em numerator} by $\hat{L}-\frac{\mu}{\max_{s} p_s}$. However, as this does not bring any significant improvements, we decided to present the result in this simplified form.} 
\begin{equation} 
\label{eq:s2cd:main} 
\E{ P(\ovar^{\oidx+1}) - P(\ovar^*)} \leq \left(\frac{(1 - \mu h )^m}{(1 - (1 - \mu h )^m) (1 - 2\hL h ) } + \frac{2\hL h }{1 - 2\hL h }\right) \E{ P(\ovar^{\oidx}) - P(\ovar^*) }.
\end{equation}
\end{theorem}

By analyzing the above result (one can follow the steps in \cite[Theorem 6]{S2GD}), we get the following useful corollary:
\begin{corollary}
\label{cor:s2cd:result}
Fix the number of epochs $\oidx \geq 1$, error tolerance $\epsilon \in(0,1)$ and let $\Delta \eqdef \epsilon^{1/\oidx}$ and 
\begin{align}
\label{eq:s2cd:defhatkappa}
\hat{\kappa}\eqdef \hL/\mu \overset{\eqref{eq:s2cd:us886vs5}}{=} \frac{1}{\mu n}  \sum_{\csample = 1}^d \sum_{\rsample = 1}^n \omega_\rsample L_{\rsample\csample}.
\end{align}
If we run Algorithm~\ref{alg:s2cd:S2CD} with stepsize $h$ and $m$ set as
\begin{equation}
\label{eq:s2cd:sjs8s} 
h = \frac{\Delta}{(4   + 2\Delta)\hL}, \qquad m \geq \left( \frac{4}{\Delta}  + 2\right) \log \left( \frac{2}{\Delta} + 2 \right) \hat{\kappa},
\end{equation}
then $ \E{ P(\ovar^{\oidx}) - P(\ovar^{*}) } \leq \epsilon( P(\ovar^{0}) - P(\ovar^{*}))$. In particular, for $\oidx = \lceil \log(1/\epsilon)\rceil$ we have $\tfrac{1}{\Delta} \leq \exp(1)$, and we can pick
\begin{equation}
\label{eq:s2cd:jd98ydO}
k = \lceil \log(1/\epsilon)\rceil, \qquad h = \frac{\Delta}{(4+2\Delta) \hL} \approx\frac{1}{(4 \exp(1) +2) \hL} \approx \frac{1}{12.87 \hL},\qquad   m\geq 26\hat{\kappa}.
\end{equation}
\end{corollary}

\begin{remark}
Note that in order to define $h$ and $m$ as in~\eqref{eq:s2cd:sjs8s}, we need to fix the target accuracy $\epsilon$ and the number of iterations $k$ beforehand.
\end{remark}

If we run S2CD with the parameters set as in \eqref{eq:s2cd:jd98ydO}, then in each epoch  the gradient of $f$ is evaluated once (this is equivalent to $n$ evaluations of $\nabla f_i$), and the partial derivative of some function $f_i$ is evaluated  $2m\approx 52\hat{\kappa}=O(\hat{\kappa})$ times. If we let $C_{grad}$ be the average cost of evaluating the gradient $\nabla f_i$ and $C_{pd}$ be the average cost of evaluating the partial derivative $\nabla_j f_i$, then the total work of S2CD can be written as
\begin{equation}
\label{eq:s2cd:complexityc}
 (n {\cal C}_{grad} + m {\cal C}_{pd}) k  \overset{\eqref{eq:s2cd:jd98ydO}}{=}  \mathcal{O} \left( (n  \mathcal{C}_{grad} +  \hat{\kappa} \mathcal{C}_{pd}) \log \left(\frac{1}{\epsilon}\right) \right),
\end{equation}

The complexity results of methods such as S2GD/SVRG \cite{S2GD, SVRG, proxSVRG} and SAG/SAGA \cite{SAGjournal2013, saga}---in a similar but not identical setup to ours (these papers assume $f_i$ to be $L_i$-smooth)---can be written in a similar form: 
\begin{equation}
\label{eq:s2cd:complexity} 
\mathcal{O}\left( (n  \mathcal{C}_{grad} +  \kappa \mathcal{C}_{grad}) \log \left( \frac{1}{\epsilon}\right)\right),
\end{equation}
where $\kappa = L_{avg}/\mu$ with $L_{avg}\eqdef \tfrac{1}{n}\sum_i L_i$ \cite{proxSVRG} (or slightly weaker where $\kappa=L_{max}/\mu$ with $L_{max}\eqdef \max_i L_i$ \cite{SAGjournal2013, SVRG, S2GD, saga}). The difference between our result \eqref{eq:s2cd:complexityc} and existing results \eqref{eq:s2cd:complexity} is in the term  $\hat{\kappa}{\cal C}_{pd}$ -- previous results have $\kappa {\cal C}_{grad}$ in that place. This difference constitutes a trade-off:  while $\hat{\kappa}\geq \kappa$ (we comment on this below), we clearly have ${\cal C}_{pd}\leq {\cal C}_{grad}$. The comparison of the quantities $\kappa {\cal C}_{grad}$ and $\hat{\kappa}{\cal C}_{pd}$ is in general not straightforward and problem dependent.

Let us now compare the condition numbers $\hat{\kappa}$ and $\kappa=L_{avg}/\mu$. It can be shown that (see \cite{RichtarikTakacIteration}) \[L_i \leq \sum_{j=1}^d L_{ij}\] and, moreover, this inequality can be tight. Since $\omega_i\geq 1$ for all $i$, we have
\[\hat{\kappa} = \frac{\hL}{\mu} \overset{\eqref{eq:s2cd:us886vs5}}{=} \frac{1}{\mu n}  \sum_{j = 1}^d \sum_{i = 1}^n \omega_i L_{ij} 
\geq \frac{1}{\mu n} \sum_{i=1}^n   \sum_{j=1}^d L_{ij}  \geq \frac{1}{\mu n} \sum_{i=1}^n L_i  = \frac{L_{avg}}{\mu} = \kappa.\]

Let us denote
$$
\underline \omega=\min_i \omega_i,\enspace \bar \omega=\max_i \omega_i.
$$
That is, $\underline \omega$ and $\bar \omega$ are respectively the smallest and largest number of coordinates that a subfunction depends on. In the case when
\begin{align}
\label{a:s2cd:Lisum}
L_i = \sum_{j=1}^d L_{ij},
\end{align}
it is easy to see that
$$ \underline \omega \kappa \leq \hat{\kappa} \leq \bar \omega \kappa. $$
In addition, when~\eqref{a:s2cd:Lisum} holds,  $\hat{\kappa}$   is smaller than $\kappa_{max}\eqdef L_{max}/\mu$ if 
$$ \bar \omega \sum_{i=1}^n L_i \leq {n}{} \max_i L_i. $$

\section{Proof of Lemma~\ref{lem:s2cd:main} } 
\label{subsec:s2cd:key_lemma}

We will prove the following stronger inequality:
\begin{align}
\label{s2cd:a-EGij2bounds2gd}
\E{ \left\| \grad_{\rsample\csample}^{\oidx\iidx} \right\|^2 } 
 \leq 4 \hL \left( P(\ivar^{\oidx, \iidx}) - P(\ovar^*)\right) + 4 \left( \hL - \frac{\mu}{\max_s p_s} \right) \left( P(\ovar^{\oidx}) - P(\ovar^*) \right).
\end{align}
Lemma~\ref{lem:s2cd:main} follows by dropping the negative term.

\paragraph*{STEP 1.} We first break down the left hand side of \eqref{s2cd:a-EGij2bounds2gd} into $d$ terms each of which we will bound separately. By first taking expectation conditioned on $\csample$ and then taking the full expectation, we can write:
\begin{eqnarray}
\E{ \left\| \grad_{\rsample\csample}^{\oidx\iidx} \right\|^2 } &\overset{\eqref{eq:s2cd:g_kl}}{=} & \E{ \Ei{ \|p_\csample^{-1} G_{\rsample\csample}^{\oidx\iidx} e_\csample \|^2 } } \notag \\
&  = & \E{ p_\csample^{-2}  \Ei{ \left(G_{\rsample\csample}^{\oidx\iidx} \right)^2 } } \;\; =\;\; \sum_{s=1}^d p_s^{-1} \Ei{ \left( G_{\rsample s}^{\oidx \iidx} \right)^2 }.
\label{eq:s2cd:98gsjs9t87}
\end{eqnarray}

\paragraph*{STEP 2.} We now further break each of these $d$ terms into three pieces. That is, for each $j=1,\dots,d$ we have:

\begin{align*}
\Ei{ \left( G_{\rsample\csample}^{\oidx\iidx} \right)^2 } & \overset{\eqref{eq:s2cd:0j9j0s9s}}{=} \Ei{ \left( \nabla_{\csample} P(\ovar^{\oidx}) + \frac{ \nabla_{\csample} f_{\rsample}(\ivar^{\oidx, \iidx}) - \nabla_{\csample} f_{\rsample}(\ovar^{\oidx})}{n \weightinG_{\rsample\csample}}  + \frac{ \nabla_{\csample} f_{\rsample}(\ovar^*) - \nabla_{\csample} f_{\rsample}(\ovar^*) }{n \weightinG_{\rsample\csample}} \right)^2 } \notag \\
&\hspace{-20pt}= \Ei{ \left( \frac{\nabla_{\csample} f_{\rsample}(\ivar^{\oidx, \iidx}) - \nabla_{\csample} f_{\rsample}(\ovar^*)}{n \weightinG_{\rsample\csample}}  + \nabla_{\csample} P(\ovar^{\oidx}) - \frac{\nabla_{\csample} f_{\rsample}(\ovar^{\oidx}) - \nabla_{\csample} f_{\rsample}(\ovar^*)}{n \weightinG_{\rsample\csample}} \right)^2 } \notag \\
&\hspace{-20pt}\leq 2 \Ei{ \left( \frac{ \nabla_{\csample} f_{\rsample}(\ivar^{\oidx, \iidx}) - \nabla_{\csample} f_{\rsample}(\ovar^*)}{n \weightinG_{\rsample\csample}} \right)^2 }  + 2 \Ei{ \left( \nabla_{\csample} P(\ovar^{\oidx}) - \frac{\nabla_{\csample} f_{\rsample}(\ovar^{\oidx}) - \nabla_{\csample} f_{\rsample}(\ovar^*) }{n \weightinG_{\rsample\csample}} \right)^2 } \notag \\
&\hspace{-20pt}= 2 \Ei{ \left( \frac{\nabla_{\csample} f_{\rsample}(\ivar^{\oidx, \iidx}) - \nabla_{\csample} f_{\rsample}(\ovar^*)}{n \weightinG_{\rsample\csample}} \right)^2 } \\
&\hspace{-20pt}\qquad + 2\Ei{ \left( \frac{\nabla_{\csample} f_{\rsample}(\ovar^{\oidx}) - \nabla_{\csample} f_{\rsample}(w^*)}{n \weightinG_{\rsample\csample}}  - \left( \nabla_{\csample} P(\ovar^{\oidx}) - \nabla_{\csample} P(\ovar^*) \right) \right)^2 } \notag \\
&\hspace{-20pt}= 2 \Ei{ \left( \frac{\nabla_{\csample} f_{\rsample}(\ivar^{\oidx, \iidx}) - \nabla_{\csample} f_{\rsample}(\ovar^*) }{n \weightinG_{\rsample\csample}} \right)^2 } + 2  \Ei{ \left(\frac{\nabla_{\csample} f_{\rsample}(\ovar^{\oidx}) - \nabla_{\csample} f_{\rsample}(\ovar^*) }{n \weightinG_{\rsample\csample}} \right)^2 } \\
&\hspace{-20pt}\qquad - 2  (\nabla_j P(\ovar^{\oidx}) - \nabla_{j} P(\ovar^*))^2, \numberthis \label{s2cd:eq-dzfdfd}
\end{align*}
where the last equality follows from the fact that 
$$ \Ei{ \frac{\nabla_{\csample} f_{\rsample}(\ovar^{\oidx}) - \nabla_{\csample} f_{\rsample}(\ovar^*) }{n \weightinG_{\rsample\csample}} } = \sum_{i=1}^n q_{\rsample\csample}\frac{\nabla_{\csample} f_{\rsample}(\ovar^{\oidx}) - \nabla_{\csample} f_{\rsample}(\ovar^*) }{n \weightinG_{\rsample\csample}} = \nabla_j P(\ovar^{\oidx}) - \nabla_{j} P(\ovar^*). $$

\paragraph*{STEP 3.} In this step we bound the first two terms in the right hand side of inequality \eqref{s2cd:eq-dzfdfd}. It will now be useful to introduce the following notation:
\begin{equation}
\label{eq:s2cd:Qj}
Q_\csample \eqdef \{ i : L_{\rsample \csample} \neq 0 \}, \qquad j=1,\dots,d,
\end{equation}
and 
\[
1_{\rsample \csample}  \eqdef \begin{cases}  1 & \mathrm{~if~} L_{ij}\neq 0 \\
 0 & \mathrm{~otherwise}
 \end{cases} , \qquad \rsample =1,\dots, n, \quad j=1,\dots,d.
\]

Let us fist examine the first term in the right-hand side of~\eqref{s2cd:eq-dzfdfd}. Using the coordinate co-coercivity lemma (Lemma~\ref{lemma:s2cd:coerc}) with $y=\ovar^*$, we obtain the inequality
\begin{align}
\label{s2cd:a-dfzeff}
\left( \nabla_{\csample} f_{\rsample}(\ovar) - \nabla_{\csample} f_{\rsample}(\ovar^*) \right)^2 \leq 2L_{ij} \left( f_{\rsample}(\ovar) 
- f_{\rsample}(\ovar^*) - \ip{ \nabla f_\rsample(\ovar_*)}{\ovar - \ovar^*} \right),
\end{align}
using which we get the bound:
\begin{eqnarray}
& & 2 \sum_{s = 1}^d p_s^{-1} \Ei{ \left( \frac{1}{n \weightinG_{\rsample, s}} \left( \nabla_{s} f_{\rsample}(\ivar^{\oidx, \iidx}) - \nabla_{s} f_{\rsample}(\ovar^*) \right)\right)^2 } \notag \\
&= & 2 \sum_{s = 1}^d p_s^{-1} \sum_{\rsample \in Q_s} \frac{1}{n^2 \weightinG_{\rsample, s}} ( \nabla_{s} f_{\rsample}(\ivar^{\oidx, \iidx}) - \nabla_{s} f_{\rsample}(\ovar^*) )^2 \notag \\
& \overset{\eqref{s2cd:a-dfzeff}}{\leq} & 4 \sum_{s = 1}^d p_s^{-1}\sum_{\rsample \in Q_s} \frac{L_{\rsample s}}{n^2 \weightinG_{\rsample, s}} \left( f_{\rsample}(\ivar^{\oidx, \iidx}) 
- f_{\rsample}(\ovar^*) - \ip{ \nabla f_i(\ovar^*)}{\ivar^{\oidx, \iidx} - \ovar^*} \right) \notag \\
& \overset{\eqref{eq:s2cd:Qj}}{=} & 4 \sum_{\rsample = 1}^n \sum_{s = 1}^ d p_s^{-1} 1_{\rsample s} \frac{ \weightofnorm_s}{n^2 \omega_\rsample}\left( f_{\rsample}(\ivar^{\oidx, \iidx}) - f_{\rsample}(\ovar^*) - \ip{ \nabla f_i(\ovar^*)}{\ivar^{\oidx, \iidx} - \ovar^*} \right).
\label{eq:s2cd:iud7dke}
\end{eqnarray}

Note that by~\eqref{eq:s2cd:sjs7tbjd} and \eqref{eq:s2cd:us886vs5}, we have that for all $s =1,2,\dots,d$,
$$ p_s^ {-1}\weightofnorm_{s}=n\hat L.$$
Continuing from \eqref{eq:s2cd:iud7dke}, we can therefore further write
\begin{eqnarray}
& & 2 \sum_{s = 1}^d p_s^{-1} \Ei{ \left( \frac{1}{n \weightinG_{\rsample\csample}} \left( \nabla_{\csample} f_{\rsample}(\ivar^{\oidx, \iidx}) - \nabla_{\csample} f_{\rsample}(\ovar^*) \right) \right)^2 } \notag \\
&\leq & 4 \sum_{\rsample = 1}^n \sum_{s = 1}^ d 1_{\rsample s} \frac{\hL}{n \omega_\rsample} \left( f_{\rsample}(\ivar^{\oidx, \iidx}) - f_{\rsample}(\ovar^*) - \ip{ \nabla f_i(\ovar^*)}{\ivar^{\oidx, \iidx} - \ovar^*} \right) \notag \\
&= & \frac{4 \hL}{n}\sum_{i = 1}^n \left( f_{\rsample}(\ivar^{\oidx, \iidx}) - f_{\rsample}(\ovar^*) - \ip{ \nabla f_i(\ovar^*)}{\ivar^{\oidx, \iidx} - \ovar^*} \right) \notag \\
&= & 4 \hL ( P(\ivar^{\oidx, \iidx}) - P(\ovar^*) ).
\label{s2cd:eq-dfdfs1}
\end{eqnarray}

The same reasoning applies to the second term on the right-hand side of the inequality~\eqref{s2cd:eq-dzfdfd} and we have:
\begin{align}
\label{s2cd:a-derq2}
& 2 \sum_{s = 1}^d p_s^{-1} \Ei{ \left( \frac{1}{n \weightinG_{\rsample\csample}} \left( \nabla_{\csample} f_{\rsample}(\ovar^{\oidx}) - \nabla_{\csample} f_{\rsample}(\ovar^*) \right) \right)^2 } \leq 4 \hL (P(\ovar^{\oidx}) - P(\ovar^*)).
\end{align}

\paragraph*{STEP 4.}

Next we bound the third term on the right-hand side of the inequality~\eqref{s2cd:eq-dzfdfd}. First note that since $P$ is $\mu$-strongly convex (see \eqref{s2cd:SVRGstrcvx}), for all $\ovar \in \R^d$ we have:
\begin{align}
\label{s2cd:a-xxsna}
\ip{\nabla P(\ovar)}{\ovar - \ovar^*} \geq  P(\ovar) - P(\ovar^*) + \frac{\mu}{2} \| \ovar - \ovar^* \|^2.
\end{align}
We can now write:
\begin{eqnarray}
2 \sum_{s = 1}^d p_s^{-1} (\nabla_{s} P(\ovar^{\oidx})-\nabla_{s} P(\ovar^*))^2  
&\geq & \frac{2}{\max_s p_s} \sum_{j = 1}^d (\nabla_{\csample} P(\ovar^{\oidx}) - \nabla_{\csample} P(\ovar^*))^2\notag  \\
&\overset{\eqref{s2cd:a-xxsna}}{ \geq} & \frac{4 \mu }{\max_s p_s}(P(\ovar^{\oidx}) - P(\ovar^*)).\label{s2cd:eq-eeff}
\end{eqnarray}

\paragraph*{STEP 5.} We conclude by combining~\eqref{eq:s2cd:98gsjs9t87}, \eqref{s2cd:eq-dzfdfd}, \eqref{s2cd:eq-dfdfs1}, \eqref{s2cd:a-derq2} and~\eqref{s2cd:eq-eeff}.

\section{Proof of the Main Result}
\label{sec:s2cd:proofs}

In this section we provide the proof of our main result.  In order to present the proof in an organize fashion,  we first establish two technical lemmas.

\subsection{Coordinate co-coercivity}

It is a well known and widely used fact (see, e.g. \cite{nesterov2004convex}) that for a continuously differentiable function  $\phi:\R^d\to \R$ and constant $L_\phi>0$, the following two conditions are equivalent:
\[\phi(x) \leq \phi(y) + \ip{ \nabla \phi(y)}{x-y} + \frac{L_\phi}{2}\|x-y\|^2, \quad \forall x,y\in \R^d\]
and
\[\|\nabla \phi(x) - \nabla \phi(y)\|^2 \leq 2L_{\phi} (\phi(x)-\phi(y) - \ip{\nabla \phi(y)}{x-y}), \qquad \forall x,y\in \R^d.\]
The second condition is often referred to by the name co-coercivity.  Note that our assumption \eqref{eq:s2cd:sjs7shd} on $f_i$ is similar to the first inequality. In our first lemma we establish a coordinate-based co-coercivity result which applies to functions $f_i$ satisfying  \eqref{eq:s2cd:sjs7shd}. 

\begin{lemma}[Coordinate co-coercivity]
\label{lemma:s2cd:coerc}
For all $\ovar, \ivar \in \R^d$ and $\rsample=1,\dots, n$, $\csample=1,\dots,d$, we have:
\begin{align}
\label{s2cd:a-dfzeff1}
\left(\nabla_{\csample} f_{\rsample}(\ovar) - \nabla_{\csample} f_{\rsample}(\ivar) \right)^2 \leq 2L_{ij} \left( f_{\rsample}(\ovar) 
- f_{\rsample}(\ivar) - \ip{ \nabla f_i(\ivar)}{\ovar - \ivar } \right).
\end{align}
\end{lemma}
\begin{proof}
Fix any $i,j$ and $\ivar \in \R^d$.
Consider the function $g_i:\R^ d\rightarrow \R$ defined by:
\begin{equation}
\label{eq:s2cd:9sh98hs}
g_\rsample(\ovar) \eqdef f_{\rsample}(\ovar) 
- f_{\rsample}(\ivar) - \ip{ \nabla f_i(\ivar)}{\ovar - \ivar}.
\end{equation}
Then since $f_i$ is convex, we know that $g_i(\ovar)\geq 0$ for all $\ovar$, with $g_\rsample(\ivar)=0$. Hence, $\ivar$ minimizes $g_\rsample$. We also know that for any $\ovar \in \R^d$:
\begin{equation}
\label{eq:s2cd:09s98hs}
\nabla_j g_\rsample(\ovar)=\nabla_j f_\rsample(\ovar)-\nabla_j f_\rsample(\ivar).
\end{equation}
Since $f_i$ satisfies \eqref{eq:s2cd:sjs7shd}, so does $g_i$, and hence for all $\ovar \in \R^d$ and $h \in \R$, we have 
\[g_i(\ovar + he_j) \leq g_i(\ovar) + \ip{\nabla g_i(\ovar)}{he_j} + \frac{L_{ij}}{2}h^2.\]
Minimizing both sides in $h$, we obtain 
$$
g_i(\ivar) \leq \min_{h} g_i(\ovar + he_j) \leq g_i(\ovar)- \frac{1}{2L_{ij}} (\nabla_j g_i(\ovar))^2,
$$ which together with \eqref{eq:s2cd:9sh98hs} yield the result.
\end{proof}

\subsection{Recursion}

We now proceed to the final lemma, establishing a key recursion which ultimately yields the proof of the main theorem, which we present in Section~\ref{eq:s2cd:s98h9s8h}.

\begin{lemma}[Recursion]
\label{s2cd:l-gah0} The iterates of S2CD satisfy the following recursion:
\begin{equation}
\begin{split}
\frac{1}{2} \E{ \| \ivar^{\oidx, \iidx + 1} - \ovar^* \|^2 } &+ h(1 - 2h \hat L) \left( P(\ivar^{\oidx, \iidx}) - P(\ovar^*) \right) \\
&\leq (1-h\mu) \frac{1}{2} \| \ivar^{\oidx, \iidx} - \ovar^* \|^2 + 2h^2 \hat L \left( P(\ovar^{\oidx}) - P(\ovar^*) \right).
\end{split}
\end{equation}

\end{lemma}
\begin{proof}
\begin{align*}
\frac{1}{2} &\E{ \| \ivar^{\oidx, \iidx+1} - \ovar^* \|^2 } \overset{\eqref{eq:s2cd:87gsb8s9}}{=} \frac{1}{2} \E{ \left\| \ivar^{\oidx, \iidx} - h p_{\csample}^{-1} G_{\rsample \csample}^{\oidx \iidx} e_{\csample} - \ovar^* \right\|^2 } \\
&\qquad\qquad= \frac{1}{2} \| \ivar^{\oidx, \iidx} - \ovar^* \|^2 - \E{ \ip{h p_{\csample}^{-1} G_{\rsample \csample}^{\oidx \iidx} e_{\csample}}{\ivar^{\oidx, \iidx} - \ovar^*} } + \frac{1}{2} \E{ \left\| h p_{\csample}^{-1} G_{\rsample \csample}^{\oidx \iidx} e_{\csample} \right\|^2 } \\
&\qquad\qquad\overset{\eqref{s2cd:EGij}}{=} \frac{1}{2} \| \ivar^{\oidx, \iidx} - \ovar^* \|^2 - h \ip{\nabla P(\ivar^{\oidx, \iidx})}{\ivar^{\oidx, \iidx} - \ovar^*} + \frac{h^2}{2} \E{ \left\| \grad_{\rsample \csample}^{\oidx \iidx} \right\|^2 } \\
&\qquad\qquad\overset{\eqref{s2cd:a-xxsna}}{\leq} \frac{1}{2} \|\ivar^{\oidx, \iidx} - \ovar^* \|^2 - h \left( P(\ivar^{\oidx, \iidx}) - P(\ovar^*) + \frac{\mu}{2} \left\| \ivar^{\oidx, \iidx} - \ovar^* \right\|^2 \right) + \frac{h^2}{2} \E{ \left\| g_{\rsample \csample}^{\oidx \iidx} \right\|^2 } \\
&\qquad\qquad\overset{\eqref{eq:s2cd:iuhs98s}}{\leq} \frac{1}{2} \| \ivar^{\oidx, \iidx} - \ovar^*\|^2 - h \left( P(\ivar^{\oidx, \iidx}) - P(\ovar^*) + \frac{\mu}{2} \| \ivar^{\oidx, \iidx} - \ovar^* \|^2 \right) \\
&\qquad\qquad \qquad\qquad + 2 h^2 \hL \left( P(\ivar^{\oidx, \iidx}) - P(\ovar^*) \right) + 2 h^2  \hL \left( P(\ovar^{\oidx}) - P(\ovar^*) \right)
\\
&\qquad\qquad= (1 - \mu h) \frac{1}{2} \| \ivar^{\oidx, \iidx} - \ovar^* \|^2 - h (1 - 2h \hL)\left( P(\ivar^{\oidx, \iidx}) - P(\ovar^*) \right) \\
&\qquad\qquad \qquad\qquad + 2 h^2 \hL \left( P(\ovar^{\oidx}) - P(\ovar^*) \right).
\end{align*}
\end{proof}

\subsection{Proof of Theorem~\ref{thm:s2cd:S2CD}}
\label{eq:s2cd:s98h9s8h}

For simplicity, let us denote:
$$
\eta^{\oidx, \iidx} \eqdef \frac{1}{2} \E{ \| \ivar^{\oidx, \iidx} - \ovar^* \|^2 }, \qquad  \xi^{\oidx, \iidx} \eqdef \E{ P(\ivar^{\oidx, \iidx}) - P(\ovar^*) },
$$
where the expectation now is with respect to the entire history. Notice that
$$
\ivar^{\oidx+1,0} = \ivar^{\oidx,\iidx^\oidx},
$$
where $t^k = T\in\{1,\dots,m\}$ with probability $(1-\mu h)^{m-T}/\beta$ with $\beta$ defined in~\eqref{def:s2cd:beta}. Conditioning on $t^k$ we obtain that
\begin{align}
\label{s2cd:a-derr}
\xi^{\oidx+1,0} = \frac{1}{\beta}\sum_{t = 0}^{m-1}(1 - \mu h)^t \xi^{\oidx, m-1-t}.
\end{align} 
See also~\cite[Lemma 3]{S2GD} for a proof.
By Lemma~\ref{s2cd:l-gah0} we have the following $m$ inequalities:
\begin{align*}
\eta^{\oidx, m} + h(1 - 2 h \hL) \xi^{\oidx, m-1} &\leq (1 - \mu h) \eta^{\oidx, m-1} + 2 h^2 \hL \xi^{\oidx, 0}, \\
(1 - \mu h) \eta^{\oidx, m-1} + h(1 - 2 h \hL)(1 - \mu h) \xi^{\oidx, m-2} &\leq (1 - \mu h)^2 \eta^{\oidx, m-2} + 2 h^2  \hL (1 - \mu h)\xi^{\oidx, 0}, \\
&\,\,\,\vdots \\
(1 - \mu h)^t \eta^{\oidx, m-t} + h(1 - 2 h \hL)(1 - \mu h)^t \xi^{\oidx, m-t-1} &\leq (1 - \mu h)^{t+1} \eta^{\oidx, m-t-1} + 2 h^2 \hL (1 - \mu h)^t \xi^{\oidx, 0}, \\
&\,\,\,\vdots \\
(1 - \mu h)^{m-1} \eta^{\oidx, 1} + \gamma(1 - 2 h \hL)(1 - \mu h)^{m-1} \xi^{\oidx, 0} &\leq (1 - \mu h)^m \eta^{\oidx, 0} + 2 h^2  \hL(1 - \mu h)^{m-1} \xi^{\oidx, 0}.
\end{align*}
By summing up the above $m$ inequalities, we get:
$$
\eta^{\oidx, m} + \gamma(1 - 2 h \hL) \sum_{t = 0}^{m-1}(1 - \mu h)^t \xi^{\oidx, m-1-t} \leq (1 - \mu h)^m \eta^{\oidx, 0} + 2 h^2 \hL \beta \xi^{\oidx, 0},
$$
It follows from the strong convexity assumption~\eqref{s2cd:SVRGstrcvx} that
$P(\ovar^\oidx) - P(\ovar^*) \geq \frac{\mu}{2} \| \ovar^k - \ovar^* \|^2, $
that is,
$ \xi^{\oidx,0} \geq \mu \eta^{\oidx,0}.$
Therefore, together with~\eqref{s2cd:a-derr} we get:
$$ h(1 - 2 h \hL) \xi^{\oidx+1, 0} \leq \left( \frac{(1 - \mu h)^m}{\beta \mu} + 2 h^2 \hL \right) \xi^{\oidx, 0} $$
Hence if $0 < 2 h \hL < 1$, then we obtain:
$$
\xi^{\oidx+1, 0} \leq \left( \frac{(1 - \mu h)^m}{(1 - (1 - \mu h)^m)(1 - 2 h \hL)} + \frac{2 h \hL }{1 - 2 h \hL} \right) \xi^{\oidx, 0},
$$
which finishes the proof.

\part{Parallel and Distributed Methods with Variance Reduction}

\chapter{Mini-batch Semi-Stochastic Gradient Descent in the Proximal Setting}
\label{ch:ms2gd}

\section{Introduction}
\label{sec:ms2gd:Intro}

In this work we are concerned with the problem of minimizing the sum of two convex functions,
\begin{equation}
\label{eq:ms2gd:Px1}
\min_{w \in \R^d} \{P(w) := f(w) + R(w)\},
\end{equation}
where the first component, $f$, is smooth,  and the second component, $R$, is  possibly nonsmooth (and extended real-valued, which allows for the modeling of constraints).

In the last decade, an intensive amount of research was conducted into algorithms for  solving problems of the form  \eqref{eq:ms2gd:Px1}, largely motivated by the realization that the underlying problem has a considerable modeling power. One of the most popular and practical methods  for \eqref{eq:ms2gd:Px1} is the accelerated proximal gradient method of Nesterov \cite{nesterov2007acc}, with its most successful variant being FISTA \cite{fista}.

In many applications in optimization, signal processing and machine learning, $f$ has an additional structure. In particular, it is often the case that $f$ is the {\em average of a number of  convex functions:}
\begin{equation}
\label{eq:ms2gd:Px2}
f(w) = \frac{1}{n} \sum_{i=1}^n f_i(w).
\end{equation}

Indeed, even one of the most basic optimization problems---least squares regression---lends itself to a natural representation of the form \eqref{eq:ms2gd:Px2}. 

\subsection{Stochastic methods.}

For problems of the form \eqref{eq:ms2gd:Px1}+\eqref{eq:ms2gd:Px2}, and especially when $n$ is large and when a solution of low to medium accuracy is sufficient, deterministic methods do not perform as well as classical {\em stochastic} methods. The prototype method in this category is stochastic gradient descent (SGD), dating back to the 1951 seminal work of Robbins and Monro~\cite{RM1951}. SGD  selects an index $i\in \{1,2,\dots,n\}$ uniformly at random, and then updates the variable $w$ using $\nabla f_i(w)$ --- a stochastic estimate of $\nabla f(w)$. Note that the computation of $\nabla f_i$ is $n$ times cheaper than the computation of the full gradient $\nabla f$. For problems where $n$ is very large, the per-iteration savings can be extremely large, spanning several orders of magnitude. 

These savings do not come for free, however (modern methods, such as the one we propose, overcome this -- more on that below). Indeed, the stochastic estimate of the gradient embodied by $\nabla f_i$ has a non-vanishing variance. To see this, notice that even when started from an optimal solution $w^*$, there is no reason for $\nabla f_i(w^*)$ to be zero, which means that SGD drives  away from the optimal point. Traditionally, there have been two ways of dealing with this issue. The first one consists in choosing a decreasing sequence of stepsizes. However, this means that a much larger number of iterations is needed. A second approach is to use  a subset (``minibatch'') of indices $i$, as opposed to a single index, in order to form a better stochastic estimate of the gradient. However, this results in a method which performs more work per iteration. In summary, while traditional approaches manage to decrease the variance in the stochastic estimate, this comes at a cost.

\subsection{Modern stochastic methods}

Very recently, starting with the SAG \cite{SAG}, SDCA \cite{SDCA}, SVRG \cite{SVRG} and S2GD \cite{S2GD}  algorithms from year 2013, it has transpired that neither decreasing stepsizes nor mini-batching are necessary to resolve the non-vanishing variance issue inherent in the vanilla SGD methods. Instead, these modern stochastic\footnote{These methods are randomized algorithms. However, the term ``stochastic''  (somewhat incorrectly) appears in their names for historical reasons, and quite possibly due to their aspiration to improve upon {\em stochastic} gradient descent (SGD).} method are able to dramatically improve upon SGD in various different ways, but without having to resort to the usual variance-reduction techniques (such as decreasing stepsizes or mini-batching) which carry with them considerable costs drastically reducing their power. Instead, these modern methods were able to improve upon SGD  without any unwelcome side effects. This development led to a revolution in the area  of first order methods for solving problem \eqref{eq:ms2gd:Px1}+\eqref{eq:ms2gd:Px2}. Both the theoretical complexity and practical efficiency of these modern methods vastly outperform prior gradient-type methods.

In order to achieve $\epsilon$-accuracy, that is,
\begin{equation}
\label{eq:ms2gd:epsilonaccuracy}
\E{ P(w^k) - P(w^*) } \leq \epsilon [ P(w^0) - P(w^*) ],
\end{equation}
modern stochastic methods such as SAG, SDCA, SVRG and S2GD require only 
\begin{equation}
\label{eq:ms2gd:modern_complexity} 
{\cal O}((n+\kappa)\log(1/\epsilon))
\end{equation}
units of work, where $\kappa$ is a condition number associated with $f$, and one unit of work corresponds to the computation of the gradient of $f_i$ for a random index $i$, followed by a call to a prox-mapping involving $R$.
More specifically, $\kappa=L/\mu$, where $L$ is a uniform bound on the Lipschitz constants of the gradients of functions $f_i$ and $\mu$ is the strong convexity constant of $P$. These quantities will be defined precisely in Section~\ref{sec:ms2gd:analysis}.

The complexity bound \eqref{eq:ms2gd:modern_complexity} should be contrasted with that of proximal gradient descent (e.g., ISTA), which requires $O(n\kappa \log(1/\epsilon))$ units of work, or FISTA, which requires \linebreak$O(n\sqrt{\kappa} \log(1/\epsilon))$ units of work\footnote{However, it should be remarked that the condition number $\kappa$ in these latter methods is slightly different from that appearing in the bound \eqref{eq:ms2gd:modern_complexity}.}.  Note that while all these methods enjoy linear convergence rate, the modern stochastic methods can be many orders of magnitude faster than classical deterministic methods. Indeed, one can have \[n+\kappa \ll n\sqrt{\kappa} \leq n\kappa.\] Based on this, we see that these modern methods always beat (proximal) gradient descent ($n+\kappa \ll n\kappa$), and also  outperform FISTA as long as $\kappa \leq {\cal O}(n^2)$.  In machine learning, for instance, one usually has $\kappa \approx n$, in which case the improvement is by a factor of $\sqrt{n}$ when compared to FISTA,  and by a factor of $n$ over ISTA. For applications where $n$ is massive, these improvements are indeed dramatic.

For more information about modern dual and primal   methods we refer the reader to the literature on  randomized coordinate descent methods \cite{nesterovRCDM, RichtarikTakacIteration, RT:PCDM, necoara2014random, APPROX,SDCA, marecek2014distributed,necoara2013distributed, richtarik2013distributed, fercoq2014fast,ALPHA, combettes2015} and stochastic gradient methods \cite{SAG, tongSGD, Ma:2015ti, Jaggi:cocoa, takac-minibatch, nitanda2014stochastic, ALPHA, rosasco2014}, respectively.

\subsection{Linear systems and sketching.}

In the case when $R\equiv 0$, all stationary points (i.e., points satisfying $\nabla f(w)=0$) are optimal for \eqref{eq:ms2gd:Px1}+\eqref{eq:ms2gd:Px2}. In the special case when the functions $f_i$ are convex quadratics of the form $f_i(w)= \tfrac{1}{2}(a_i^T w - b_i)$, the equation $\nabla f(w)=0$ reduces to the linear system $A^T Aw = A^T b$, where $A=[a_1, \dots, a_n]$. Recently, there has been considerable interest in designing and analyzing  randomized methods for solving linear systems; also known under the name of {\em sketching} methods. Much of this work was done independently from the developments in (non-quadratic) optimization, despite the above connection between optimization and linear systems. A randomized version of the classical Kaczmarz method was studied in a seminal paper by Strohmer and Vershynin \cite{Strohmer2009}. Subsequently, the method was extended and improved upon in several ways \cite{Needell2010, Zouzias2012, 
Ma2015, Oswald2015}. The randomized Kaczmarz method is equivalent to SGD with a specific stepsize choice \cite{NeedellWard2015,GowerRichtarik2015-linear}. The first randomized coordinate descent method, for linear systems, was analyzed  by Lewis and Leventhal~\cite{Leventhal2008}, and subsequently generalized in various ways by numerous authors (we refer the reader to \cite{ALPHA} and the references therein).
Gower and Richt\'{a}rik \cite{GowerRichtarik2015-linear} have recently studied randomized iterative methods for linear systems in a general {\em sketch and project} framework, which in special cases includes randomized Kaczmarz, randomized coordinate descent, Gaussian descent, randomized Newton, their block variants, variants with importance sampling, and also an infinite array  of new specific  methods. For approaches of a combinatorial flavor, specific to diagonally dominant systems, we refer to the influential work of Spielman and Teng  \cite{Spielman2006}.

\section{Contributions}

In this chapter we equip modern stochastic methods---methods which already enjoy the fast rate 
\eqref{eq:ms2gd:modern_complexity}---with the ability to process data in {\em mini-batches}. None of the {\em primal}\footnote{By a primal method we refer to an algorithm which operates directly to solve \eqref{eq:ms2gd:Px1}+\eqref{eq:ms2gd:Px2} without explicitly operating on the dual problem. {\em Dual} methods have very recently   been analyzed in the mini-batch setting. For a review of such methods we refer the reader to  the paper describing the QUARTZ method \cite{QUARTZ} and the references therein.} modern methods have been analyzed in the mini-batch setting. This work fills this gap in the literature. 

While we have argued above that the modern methods, S2GD included, do not have the ``non-vanishing variance'' issue that SGD does, and hence do not {\em need} mini-batching for that purpose,  mini-batching is still useful.  In particular,  we develop and analyze the complexity of mS2GD (Algorithm~\ref{alg:ms2gd:mS2GD}) --- a mini-batch proximal variant of {\em semi-stochastic gradient descent} (S2GD) \cite{S2GD}. While the S2GD method was analyzed in the $R=0$ case only, we develop and analyze our method in the proximal\footnote{Note that the Prox-SVRG method   \cite{proxSVRG} can also handle the composite problem \eqref{eq:ms2gd:Px1}.} setting  \eqref{eq:ms2gd:Px1}. We show that mS2GD enjoys several benefits when compared to previous modern methods. First, it trivially admits a parallel implementation, and hence enjoys a speedup in clocktime in an HPC environment. This is critical for applications with massive datasets and is the main motivation and advantage of our method. Second, our results show that in order to attain a  specified accuracy $\epsilon$, mS2GD can get by with fewer gradient evaluations than S2GD. This is formalized in Theorem~\ref{thm:ms2gd:optimalM}, which predicts more than linear speedup up to a certain threshold mini-batch size after which the complexity deteriorates. Third, compared to \cite{proxSVRG}, our method does not need to average the iterates produced in each inner loop;  we instead simply continue from the last one. This is the approach employed in S2GD \cite{S2GD}.


\section{The Algorithm}
\label{sec:ms2gd:algorithms}

In this section we first briefly motivate the mathematical setup of deterministic and stochastic proximal gradient methods in Section \ref{sec:ms2gd:algorithms:A}, followed by the introduction of semi-stochastic gradient descent in Section~\ref{sec:ms2gd:algorithms:B}. We will the be ready to describe the mS2GD method in Section~\ref{sec:ms2gd:algorithms:C}.

\subsection{Deterministic and stochastic proximal gradient methods}
\label{sec:ms2gd:algorithms:A}

The classical {\em deterministic proximal gradient} approach \cite{fista,combettes2011,parikh2014} to solving \eqref{eq:ms2gd:Px1} is to form a sequence $\{y^{t}\}$ via 
\begin{align*}
y^{t+1} = \arg \min_{w \in \R^d} U^t(w),
\end{align*} 
where $U^t(w) \eqdef f(y^t) + \nabla f (y^{t})^T (w-y^{t}) + \tfrac{1}{2 h} \| w - y^{t} \|^2 + R(w)$.
Note that in view of Assumption~\ref{ass:ms2gd:ass1}, which we shall use in our analysis in Section~\ref{sec:ms2gd:analysis}, $U^{t}$ is an upper bound on $P$  whenever $ h>0$ is a stepsize parameter satisfying $1/ h \geq L$. This procedure can be compactly written using the {\em proximal operator} as follows:
\begin{equation*}
y^{t+1} = \prox_{ h R} (y^{t} -  h \nabla f (y^{t})),
\end{equation*} 
where
\begin{equation*}
\prox_{hR}(z) \eqdef \argmin_{w \in \R^d} \left\{ \frac 12  \| w - z \|^2 + h R(w)\right\}.
\end{equation*}

In a large-scale setting it is more efficient to instead consider  the \emph{stochastic proximal  gradient} approach, in which the proximal operator is applied to a stochastic gradient step:
\begin{equation}
\label{eqn:ms2gd:prox-svrg update}
y^{t+1} = \prox_{ h R}(y^{t} -  h G^{t}),
\end{equation}
where $G^{t}$ is a stochastic estimate of the gradient $\nabla f(y^t)$. 

\subsection{Semi-stochastic methods}
\label{sec:ms2gd:algorithms:B}

Of particular relevance to our work are the SVRG \cite{SVRG}, S2GD \cite{S2GD} and Prox-SVRG  \cite{proxSVRG} methods where the stochastic estimate of $\nabla f(y^t)$ is of the form
\begin{equation}
\label{eq:ms2gd:sjs8js}
G^{t} = \nabla f(w) + \frac{1}{nq_{i_t}}(\nabla f_{i_t}(y^{t}) - \nabla f_{i_t}(w)),
\end{equation}
where $w$ is an ``old'' reference point for which the gradient $\nabla f(w)$ was already computed in the past, and $i_t \in  [n] \eqdef \{1,2,\dots,n\}$ is a random index equal to $i$ with probability $q_i>0$.  Notice that $G^{t}$ is an unbiased estimate of the gradient of $f$ at $y^t$: 
\[\mathbb{E}_{i_t} \left[ G^{t} \right] \overset{\eqref{eq:ms2gd:sjs8js}}{=} \nabla f(w) + \sum_{i=1}^n q_i \frac{1}{n q_i} (\nabla f_i(y^t) - \nabla f_i(w)) \overset{\eqref{eq:ms2gd:Px2}}{=} \nabla f(y^t).\]

Methods such as  S2GD, SVRG, and  Prox-SVRG update the points $y_t$ in an inner loop, and the reference point $w$ in an outer loop (``epoch'') indexed by $k$. With this new outer iteration counter we will have $w^k$ instead of $w$, $y^{k,t}$ instead of $y^t$ and $G^{k,t}$ instead of $G_t$. This is the notation we will use in the description of our algorithm in Section~\ref{sec:ms2gd:algorithms:C}. The outer loop ensures that the squared norm of $G^{k,t}$ approaches zero as $k,t\to \infty$ (it is easy to see that this is equivalent to saying that the stochastic estimate $G^{k,t}$ has a diminishing variance), which ultimately leads to extremely fast convergence.

\subsection{Mini-batch S2GD}
\label{sec:ms2gd:algorithms:C}

We are now ready to describe the mS2GD method\footnote{
A more detailed algorithm and the associated analysis (in which we benefit from the knowledge of lower-bound on the strong convexity parameters of the functions $f$ and $R$) can be found in the arXiv preprint~\cite{konecny2015mini}.
The more general algorithm mainly differs in $t^k$ being chosen according to a geometric  probability law which depends on the estimates of the convexity constants.
} (Algorithm~\ref{alg:ms2gd:mS2GD}).

\begin{algorithm}[H]
\caption{mS2GD}
\label{alg:ms2gd:mS2GD}
\begin{algorithmic}[1]
\State \textbf{Input:} $m$ (max \# of stochastic steps per epoch); $ h>0$ (stepsize); $w^0 \in \R^d$ (starting point);  mini-batch size $b \in  [n]$ 
\For {$k=0,1, 2, \dots$}
    \State Compute and store $g^{k} \leftarrow \nabla f(w^{k}) = \tfrac{1}{n}\sum_{i} \nabla f_i(w^{k})$ 
    \State Initialize the inner loop: $y^{k,0} \gets w^{k}$
    \State Choose $t^k \in \{1,2,\dots,m\}$ uniformly at random
    \For {$t=0$ to $t^k-1$}
	    \State Choose mini-batch $A^{kt} \subseteq  [n]$ of size $b$; uniformly at random 
        \State 
\label{ms2gd:step:8}        
        Compute a stochastic estimate of $\nabla f(y^{k,t})$: \newline
        {\color{white}ij} \quad \quad $G^{k,t} \gets g^k + \frac{1}{b}\sum_{i\in A^{kt}}(\nabla f_{i}(y^{k,t}) - \nabla f_{i}(w^{k})) $
        \State
\label{ms2gd:step:9}        
         $y^{k,t+1} \gets \prox_{ h R}(y^{k,t} -  h G^{k,t})$
     \EndFor
    \State Set $w^{k+1} \leftarrow y^{k,t^k}$
\EndFor
\noindent
\end{algorithmic}
\end{algorithm}

The algorithm includes an outer loop, indexed by epoch counter $k$, and an inner loop, indexed by $t$. Each epoch is started by computing $g^k$, which is the  (full) gradient of $f$ at $w^k$. It then immediately  proceeds to the inner loop. The inner loop is run for $t^k$ iterations, where $t^k$ is chosen uniformly at random from  $ \{1,\dots,m\}$. Subsequently, we run $t^k$ iterations in the inner loop (corresponding to Steps 6--10). Each new iterate is given by the proximal update~\eqref{eqn:ms2gd:prox-svrg update}, however with the stochastic estimate of the gradient $G^{k,t}$ in~\eqref{eq:ms2gd:sjs8js}, which is formed by using a {\em mini-batch} of examples $A^{kt} \subseteq [n]$ of size $|A^{kt}|=b$. Each inner iteration requires $2b$ units of work\footnote{It is possible to finish each iteration with only $b$ evaluations for  component gradients, namely $\{\nabla f_i(y^{k,t})\}_{i\in A^{kt}}$, at the cost of having to store $\{\nabla f_i(w^k)\}_{i\in [n]}$, which is exactly the way that SAG~\cite{SAG} works. This speeds up the algorithm; nevertheless, it is impractical for big $n$.}.

\section{Analysis}
\label{sec:ms2gd:analysis}
 
In this section, we lay down the assumptions, state our main complexity result, and  comment on how to optimally choose the parameters of the method.

\subsection{Assumptions}

Our analysis is performed under the following two assumptions.

\begin{assumption}
\label{ass:ms2gd:ass1}
Function  $R:\R^d\to \R\cup \{+\infty\}$ (regularizer/proximal term) is  convex and closed. The functions  $f_i:\R^d\to \R$  have Lipschitz continuous gradients with constant $L > 0$. That is, 
$
\| \nabla f_i(w)-\nabla f _i(y) \|\leq L \|w-y\|,
$
for all $x,y\in \R^d$, where $\|\cdot\|$ is the $\ell_2$-norm.
\end{assumption}

Hence, the gradient of $f$ is also Lipschitz continuous with the same constant $L$.
\begin{assumption}
\label{ass:ms2gd:ass2}
$P$ is strongly convex with parameter $\mu>0$. That is for all $x,y \in \dom(R)$ and any $ \xi \in \partial{P(x)}$, 
\begin{equation}
\label{eq:ms2gd:strongconv}
P(y) \geq P(x) + \xi^T(y-x) + \frac\mu2 \|y-x\|^2,
\end{equation}
 where $\partial P(x)$ is the subdifferential of $P$ at $x$.
\end{assumption}

Lastly, by $\mu_f \geq 0$ and $\mu_R\geq 0$ we denote the strong convexity constants of $f$ and $R$, respectively. We allow both of these quantities to be equal to $0$, which simply means that the functions are convex (which we already assumed above). Hence, this is not an additional assumption.

\subsection{Main result}

We are now ready to formulate our complexity result.
 
\begin{theorem}
\label{thm:ms2gd:s2convergence}
Let Assumptions~\ref{ass:ms2gd:ass1} and~\ref{ass:ms2gd:ass2} be satisfied, let $w^* \eqdef \argmin_w P(w)$ and choose $b\in \{1,2\dots,n\}$. Assume  that
$0< h \leq 1/L$, $4h L \alpha(b)<1$ and that $m,h$ are further chosen so that
\begin{equation}
\label{eq:ms2gd:s2rho}
c \eqdef 
\frac{    1
  }
  {
  m 
 h \mu
(
1
-
  4 h  L \alpha(b)  
)
  }
+
\frac{      
  4 h  L \alpha(b) 
\left(      
 m
+ 
  1
\right)  
  }
  {
  m 
(
1
-
  4 h   L \alpha(b)  
)
  } < 1,
\end{equation}
where $\alpha(b) \eqdef \frac{n-b}{b (n-1)} $. Then mS2GD has linear convergence in expectation with rate $c$:
\begin{equation*}
\E{ P(w^k) - P(w^*) } \leq c^k [P(w^0)-P(w^*)].
\end{equation*}
\end{theorem}

Notice that for any fixed $b$, by properly adjusting the parameters $h$ and $m$ we can force $c$ to be arbitrarily small. Indeed, the second term can be made arbitrarily small by choosing $h$ small enough. Fixing the resulting $h$, the first term can then be made arbitrarily small by choosing $m$ large enough. This may look surprising, since this means that only a single outer loop ($k=1$) is needed in order  to obtain a solution of any prescribed accuracy. While this is indeed the case, such a choice of the parameters of the method ($m$, $h$, $k$) would not be optimal -- the resulting workload would to be too high as the complexity of the method would  depend sublinearly on $\epsilon$. In order to obtain a logarithmic dependence on $1/\epsilon$, i.e., in order to obtain linear convergence, one needs to perform $k=O(\log(1/\epsilon))$ outer loops, and set the parameters $h$ and $m$ to appropriate values (generally, $h=\Theta(1/L)$ and $m=\Theta(\kappa)$).

\subsection{Special cases: $b=1$ and $b=n$}

In the special case with  $b=1$ (no mini-batching), we get $\alpha(b)=1$, and the rate given by \eqref{eq:ms2gd:s2rho} exactly recovers the rate achieved by Prox-SVRG \cite{proxSVRG} (in the case when the Lipschitz constants of $\nabla f_i$ are all equal). The rate is also identical to the rate of S2GD \cite{S2GD} (in the case of $R=0$, since S2GD was only analyzed in that case). If we set the number of outer iterations to $k =\lceil \log(1/\epsilon)\rceil$, choose the stepsize as $h =  \tfrac{1}{(2+4e)L}$, where $e=\exp(1)$, and choose $m=43 \kappa$, then the total workload of mS2GD  for achieving \eqref{eq:ms2gd:epsilonaccuracy} is $(n+43 \kappa)\log(1/\epsilon)$ units of work. Note that this recovers the fast rate \eqref{eq:ms2gd:modern_complexity}.  

In the batch setting, that is when $b=n$, we have $\alpha(b)=0$ and hence $c = 1/(mh \mu)$. By choosing $k=\lceil \log (1/\epsilon)\rceil$, $h=1/L$, and $m=2\kappa$, we obtain the rate ${\cal O}\left(n\kappa \log(1/\epsilon)\right)$. This is the standard rate of (proximal) gradient descent. 

Hence, by modifying the mini-batch size $b$ in mS2GD, we interpolate between the fast rate  of S2GD and the slow rate of GD.

\subsection{Mini-batch speedup} 
\label{sec:ms2gd:speedup}

In this section we will derive formulas for good choices  of the parameter $m,h$ and $k$ of our method as a function of $b$. Hence, throughout this section we shall consider $b$ fixed. 

Fixing $0<c<1$, it is easy to see that in order for $w^k$ to be an $\epsilon$-accurate solution (i.e., in order for \eqref{eq:ms2gd:epsilonaccuracy} to hold), it suffices to choose $k\geq (1-c)^{-1}\log(\epsilon^{-1})$. Notice that the total workload  mS2GD will do in order to arrive at $w^k$ is \[ k(n+2m) \approx (1-c)^{-1}\log(\epsilon^{-1}) (n+2m)\] units of work. If we now consider $c$ fixed (we may try to optimize for it later), then clearly the total workload is proportional to $m$. The free parameters of the method are the stepsize $h$ and the inner loop size $m$. Hence, in order to set the parameters so as to minimize the workload (i.e., optimize the complexity bound), we would like to (approximately) solve the optimization problem
\[\min m \quad \text{subject to} \quad 0< h \leq \frac{1}{L}, \; h < \frac{1}{4L\alpha(b)}, \; c \; \text{is fixed}.\]

Let $( h^*_b, m^*_b)$ denote the optimal pair (we highlight the dependence on $b$ as it will be useful). Note that if $m^*_b \leq m^*_1 / b$ for some $b>1$, then mini-batching can help us reach the $\epsilon$-solution with smaller overall workload. The following theorem presents the formulas for $ h^*_b$ and $m^*_b$.

\begin{theorem}
\label{thm:ms2gd:optimalM}
Fix $b$ and $0<c<1$ and let
\begin{equation*}
\tilde  h_b\ \eqdef\ \sqrt{ \left( \frac{1+c}{c\mu} \right)^2 + \frac1{4\mu\alpha(b)L}} - \frac{1+c}{c\mu}.
\end{equation*}
If $\tilde  h_b \leq \frac1L$, then $ h^*_b = \tilde  h_b$ and
\begin{align*}
m^*_b 
=
\frac{2\kappa}{c} \left\{ \left( 1+\frac1c \right)4\alpha(b) + \sqrt{
  \frac{4\alpha(b)}{ \kappa} +   \left( 1+\frac1c \right)^2[4\alpha(b)]^2}
 \right\},
\numberthis
\label{eq:ms2gd:vfewdfwaefvawvgfeefewafa}
\end{align*}
where $\kappa \eqdef \frac{L}{\mu}$ is the condition number. If $\tilde{h}_b > \tfrac{1}{L}$, then $ h^*_b = \frac1L$ and
\begin{equation}
\label{eq:ms2gd:fasfawefwafewa}
m^*_b
  = \frac{\kappa + 4 \alpha(b) }
  { c -4 \alpha(b)   (1+c)}.
\end{equation}
\end{theorem}

Note that if $b=1$, we recover the optimal choice of parameters without mini-batching. Equation~\eqref{eq:ms2gd:vfewdfwaefvawvgfeefewafa} suggests that as long as the condition $\tilde  h_b \leq \frac1L$ holds, $m^*_b$ is decreasing at a rate  faster than $1/b$. Hence, we can  find the solution with less overall work when using a minibatch of size $b$ than when using a minibatch of size $1$.

\subsection{Convergence rate}

In this section we study the total workload of mS2GD in the regime of small mini-batch sizes.

\begin{corollary}
\label{thm:ms2gd:minibatch}
Fix $\epsilon\in (0,1)$, choose the number of outer iterations equal to \[ k = \left\lceil \log(1 / \epsilon) \right\rceil,\]  and fix the target decrease in Theorem~\ref{thm:ms2gd:optimalM} to satisfy $c = \epsilon^{1 / k}$. Further, pick a mini-batch size satisfying $1 \leq b \leq 29$,  let the stepsize $h$ be as  in~\eqref{eqn:ms2gd:stepsizehb} and let $m$ be as in \eqref{eqn:ms2gd:maxiter}. Then in order for mS2GD to find $w^k$ satisfying \eqref{eq:ms2gd:epsilonaccuracy}, mS2GD needs at most
\begin{equation}
\label{eqn:ms2gd:complexity}
(n + 2b m_b) \lceil \log(1 / \epsilon) \rceil
\end{equation}
units of work, where $bm_b = {\cal O}(\kappa)$, which leads to the overall complexity of
$$ \mathcal{O} \left( (n + \kappa) \log(1 / \epsilon) \right)$$
units of work.
\end{corollary}
\begin{proof}
Available in Section \ref{thm:ms2gd:proof3}.
\end{proof}

This result shows that as long as the mini-batch size is small enough, the total work performed by mS2GD is the same as in the  $b=1$ case. If the $b$ updates can be performed in parallel, then this leads to linear speedup.

\subsection{Comparison with Acc-Prox-SVRG}

The Acc-Prox-SVRG~\cite{nitanda2014stochastic} method of Nitanda, which  was not available online before the first version of this work appeared on arXiv, incorporates both a mini-batch scheme and Nesterov's acceleration~\cite{nesterov2004convex, nesterov2007acc}. 
The author claims that when $b< \lceil b_0 \rceil$, with the threshold $b_0$ defined as $\frac{8\sqrt{\kappa}n}{\sqrt{2} p(n-1) +8\sqrt{\kappa}}$, the overall complexity of the method is 
\[\mathcal{O} \left( 
\left( 
n + \frac{n-b}{n-1}\kappa 
\right) 
\log(1/\epsilon) 
\right);\]
and otherwise it is
\[\mathcal{O} \left( 
\left( 
n + b\sqrt{\kappa} 
\right) 
\log(1/\epsilon)  
\right).\]
 This suggests that acceleration will only be realized when the mini-batch size is large, while for small $b$, Acc-Prox-SVRG achieves the same overall complexity, $\mathcal{O} \left((n+\kappa)\log(1/\epsilon) \right)$, as mS2GD.
 
We will now take a closer look at  the theoretical results given by Acc-Prox-SVRG and mS2GD, for each $\epsilon\in(0,1)$. In particular, we shall numerically minimize the total work of mS2GD, i.e.,
\[\left(
 n+2b\lceil m_b\rceil
 \right)
 \left\lceil
\log(1/\epsilon) / \log(1/c)
  \right\rceil,\]
  over $c\in(0,1)$ and $h$ (compare this with \eqref{eqn:ms2gd:complexity}); and compare these results with similar fine-tuned quantities for Acc-Prox-SVRG.\footnote{$m_b$ is the best choice of $m$ for Acc-Prox-SVRG and mS2GD, respectively. Meanwhile, $h$ is within the safe upper bounds for both methods.}

Fig.~\ref{figure:ms2gd:acc-prox-svrg} illustrates these theoretical complexity bounds  for both ill-conditioned and well-conditioned data. With small-enough mini-batch size $b$, mS2GD is better than Acc-Prox-SVRG. However, for a large mini-batch size $b$, the situation reverses because of the acceleration inherent in Acc-Prox-SVRG.\footnote{We have experimented with different values for $n, b$ and $\kappa$, and this result always holds.} Plots with $b=64$ illustrate the cases where we cannot observe any differences between the methods.

\begin{figure*}[htbp]
\centering
  \epsfig{file=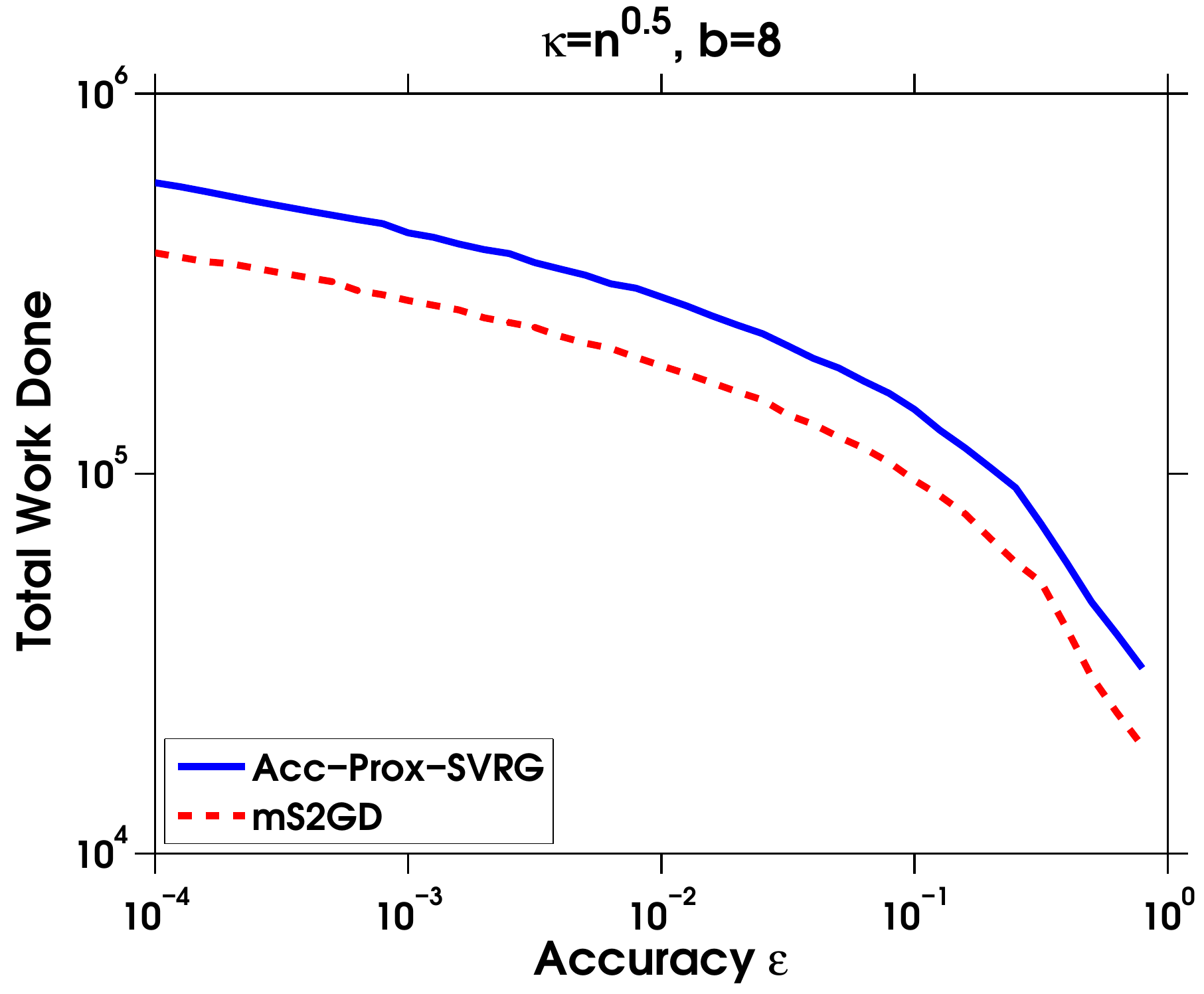,width=4.5cm,height=3.8cm}
  \epsfig{file=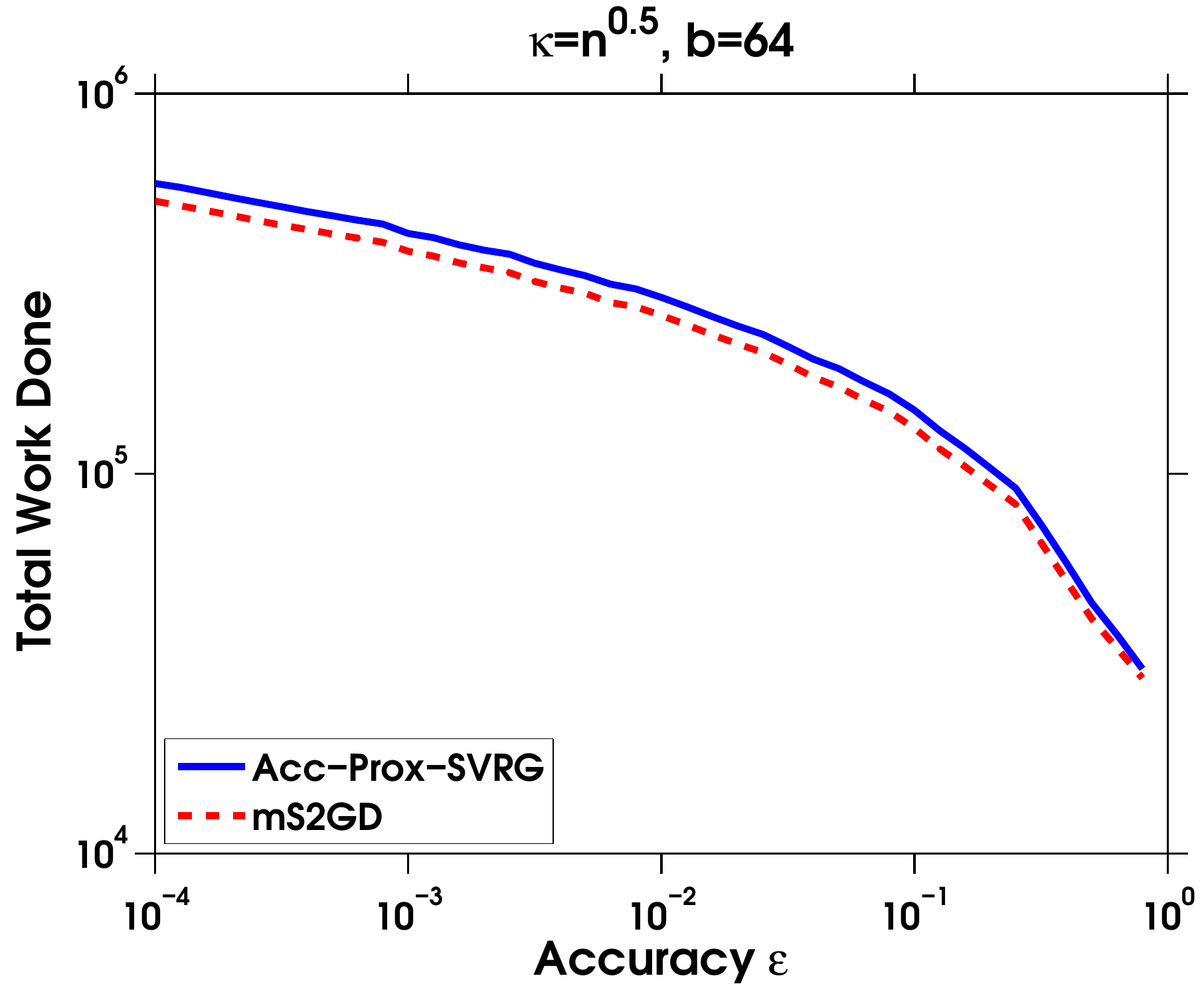,width=4.5cm,height=3.8cm}
  \epsfig{file=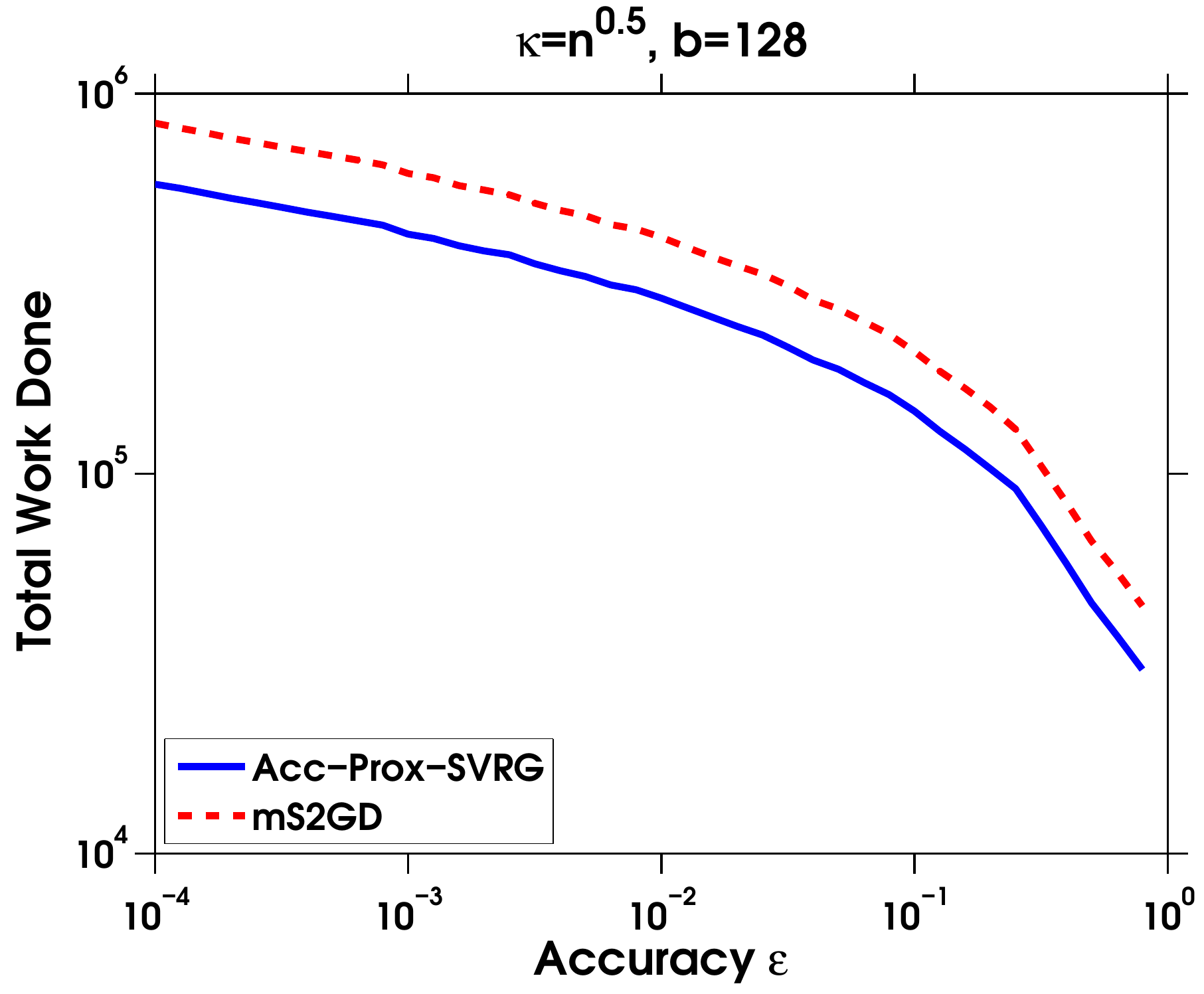,width=4.5cm,height=3.8cm}
  \epsfig{file=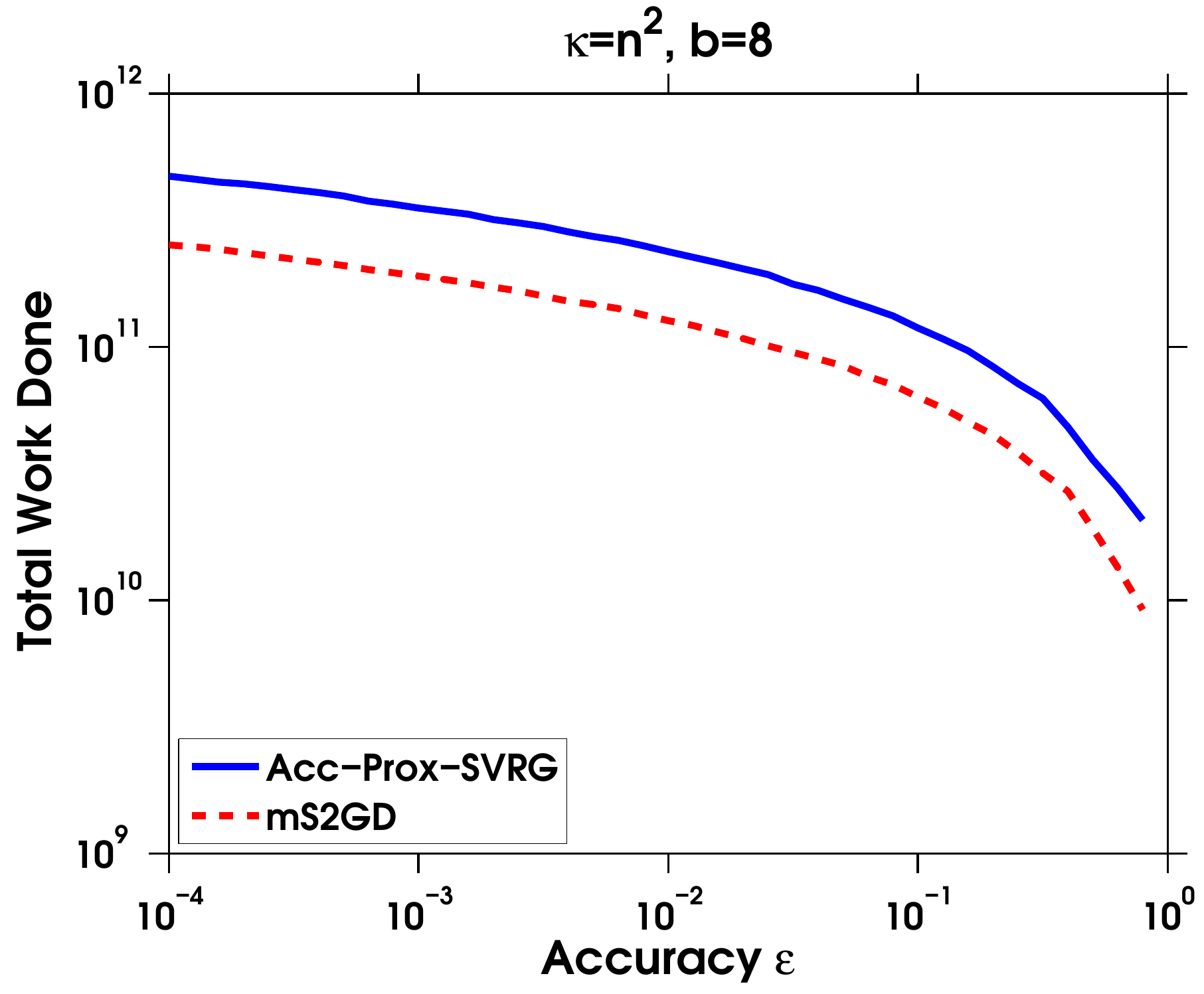,width=4.5cm,height=3.8cm}
  \epsfig{file=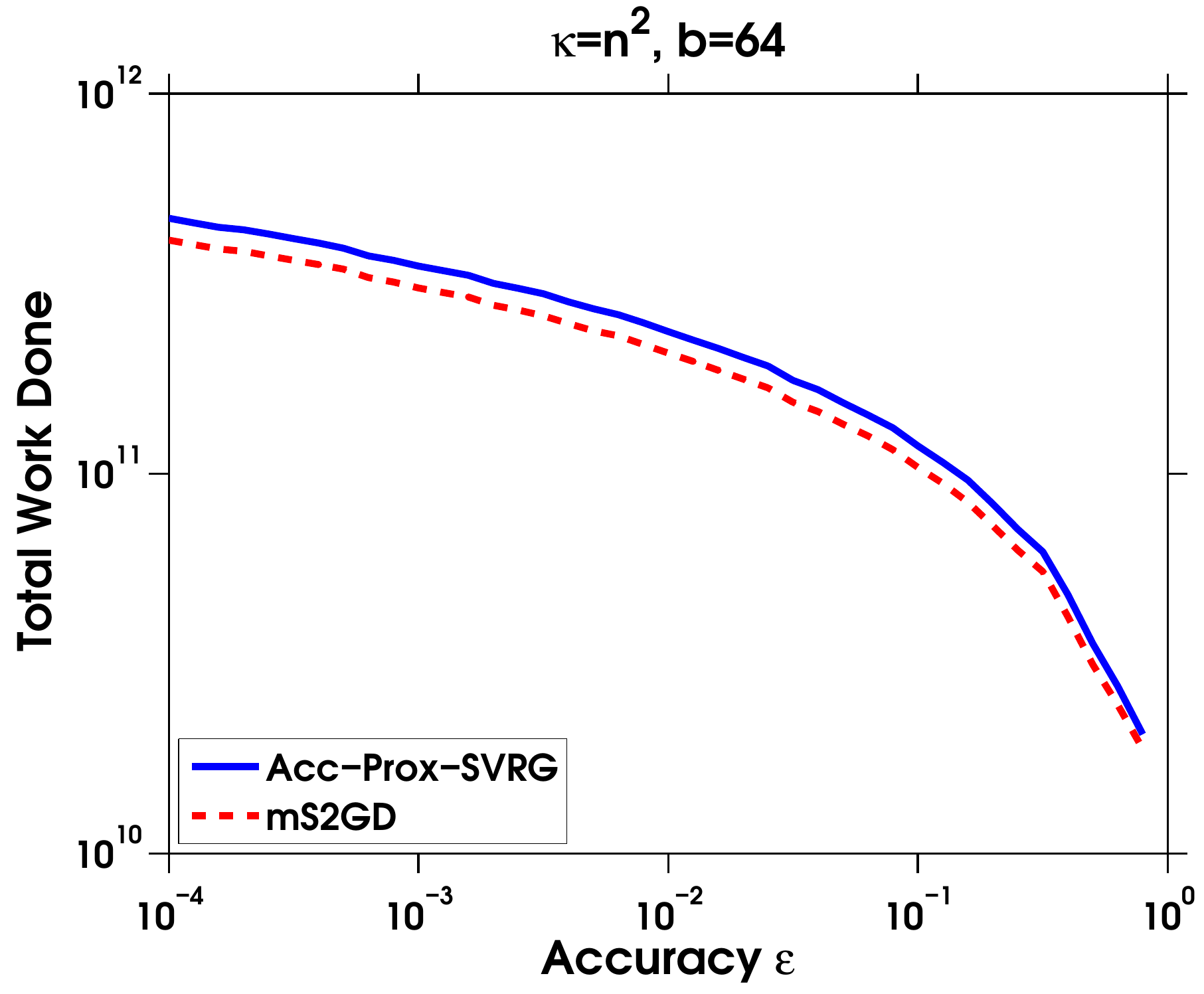,width=4.5cm,height=3.8cm}
  \epsfig{file=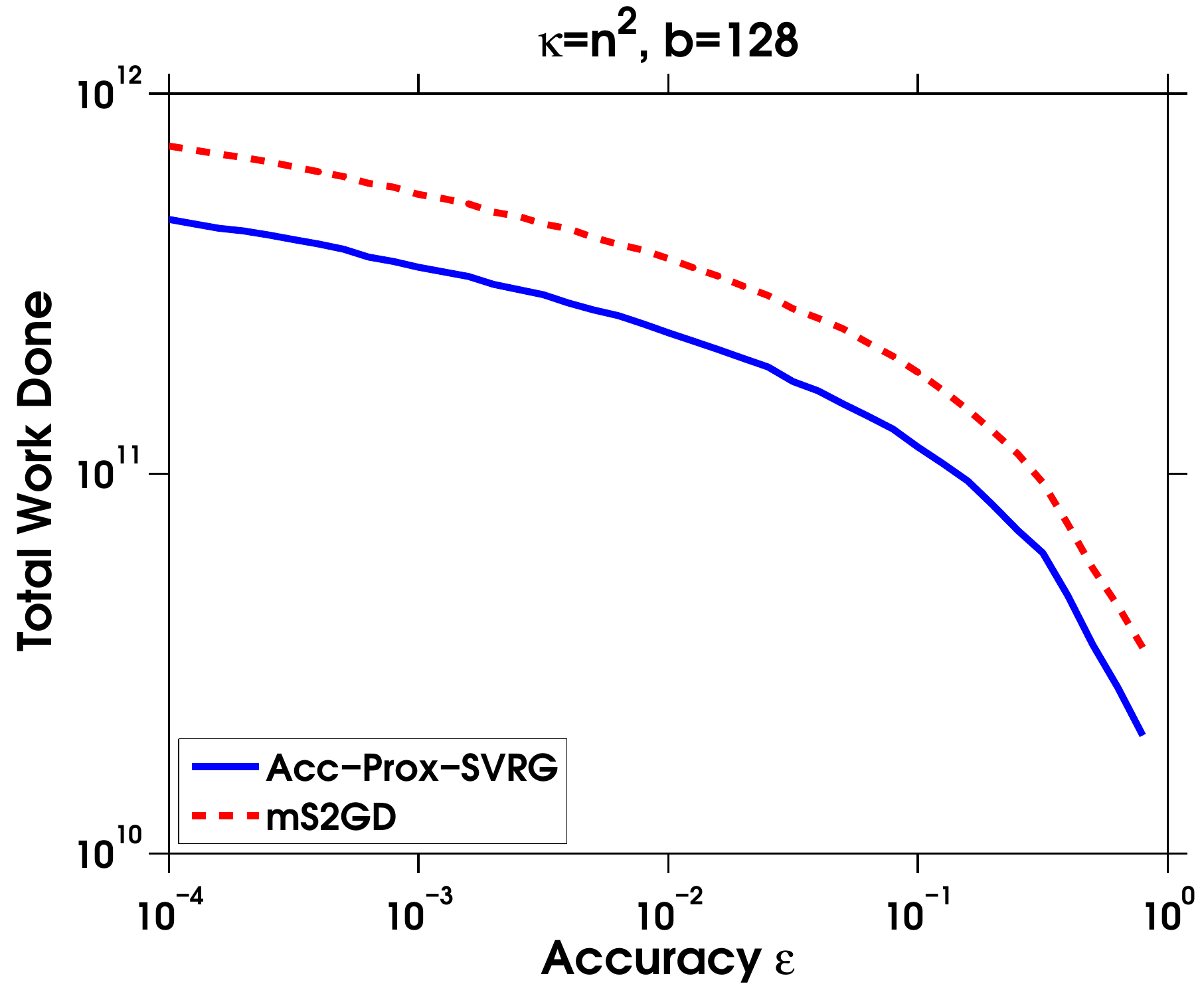,width=4.5cm,height=3.8cm}
\caption{\footnotesize Complexity of Acc-Prox-SVRG and mS2GD in terms of total work done for  $n=10,000$, and small ($\kappa=\sqrt{n}$; top row) and large ($\kappa = n^2$; bottom row) condition number.}
 \label{figure:ms2gd:acc-prox-svrg}
 \end{figure*}
 
Note however that accelerated methods are very prone to error accumulation. Moreover, it is not clear that an efficient implementation of Acc-Prox-SVRG is possible for sparse data.  As shall show  in the next section, mS2GD allows for such an implementation.


\section{Efficient implementation for sparse data}
\label{sec:ms2gd:implementation_sparse}

Let us make the following assumption about the structure of functions $f_i$ in \eqref{eq:ms2gd:Px2}.
\begin{assumption}
The functions $f_i$ arise as the composition of a univariate smooth function $\ell_i$ and an inner product with a datapoint/example $a_i \in \R^d$:
$f_i(w) = \ell(a_i^Tw)$  for $i=1,\dots,n$.
\label{asm:ms2gd:structure}
\end{assumption}
Many functions of common practical interest satisfy this assumption including linear and logistic regression. 
Very often, especially for large scale datasets, the data are extremely sparse, i.e.
the vectors $\{a_i\}$ contains many zeros.
Let us denote the number of non-zero coordinates of $a_i$ by $\omega_i = \|a_i\|_0 \leq d $
and the set of indexes corresponding to  non-zero coordinates by
$\support(a_i)=\{j : (a_i)_{j} \neq 0\} $,
where $(a_i)_{j}$ denotes the $j^{th}$ coordinate of vector $a_i$.
\begin{assumption}
\label{asm:ms2gd:structure2}
The regularization function $R(w)$ is separable in coordinates of $w$.
\end{assumption}
This includes the most commonly used  regularization functions as 
$\frac{\lambda}{2} \|w\|^2$
or $\lambda \|w\|_1$.

Let us take a brief detour and look at the classical SGD algorithm
with $R(w) = 0$. The update would be of the form
\begin{equation}
w^{k+1} \gets w^{k} - h \nabla \ell_i(a_i^T w^k) a_i = w^{k} -  h \nabla f_i(w^k). 
\label{eq:ms2gd:adsfawlfcawefcava}
\end{equation}
If evaluation of the univariate function $\nabla \ell_i$ takes $O(1)$ amount of work, the computation of $\nabla f_i$ will account for $O(\omega_i)$ work.
Then the update 
\eqref{eq:ms2gd:adsfawlfcawefcava}
would cost $O(\omega_i)$ too, which implies that the classical SGD method can naturally benefit from sparsity of data.

Now, let us get back to the Algorithm~\ref{alg:ms2gd:mS2GD}.
Even under the sparsity assumption and structural Assumption \ref{asm:ms2gd:structure}
the Algorithm~\ref{alg:ms2gd:mS2GD} suggests that each inner iteration will cost $O(\omega + d ) \sim O(d)$
because $g^k$ is in general fully dense and hence in 
Step \ref{ms2gd:step:9} of Algorithm~\ref{alg:ms2gd:mS2GD}
we have to update all $d$ coordinates.

However, in this Section, we will 
introduce and describe the implementation trick 
which is based on ``lazy/delayed'' updates.
The main idea of this trick is not to perform Step  \ref{ms2gd:step:9}
of Algorithm~\ref{alg:ms2gd:mS2GD}
for all coordinates, but only for 
coordinates $j \in \cup_{i \in A^{kt}} \support(a_i)$.
The efficient algorithm is described in Algorithm~\ref{alg:ms2gd:mS2GDsparse}. 
\begin{algorithm}[H]
\caption{"Lazy" updates for mS2GD (these replace steps 6--10 in Algorithm~\ref{alg:ms2gd:mS2GD})}
\label{alg:ms2gd:mS2GDsparse}
\begin{algorithmic}[1]
    \State $\chi_j \gets 0$ for $j=1,2,\dots,d$
    \For {$t=0$ to $t^k-1$}
	\State Choose mini-batch $A^{kt}\subseteq  [n]$ of size $b$; uniformly at random
        \For {$i\in A^{kt}$}
            \For {$j\in \support(a_i)$}
            \label{ms2gd:asdfasfsafdsafa}
                \State $y^{k,t}_{j} \gets\prox^{t-\chi_j} [ y^{k,\chi_{j}}_{j}, g^k_j, R, h ]$ 
                \State $\chi_j \gets t$
            \EndFor
            \label{ms2gd:asdfasfsafdsafa2}
        \EndFor
        \State $y^{k,t+1} \gets y^{k,t} -\frac{ h}{b}\sum_{i \in A^{kt}} \left( \nabla \ell_{i}(a_i^T y^{k,t}) - \nabla \ell_{i}(a_i^Tw^{k}) \right) a_i$ 
     \EndFor
        \For {$j = 1$ to $d$}
        \label{ms2gd:asdfsafdasfa}
            \State $y^{k,t^k}_j \gets \prox^{t^k-\chi_j}
            [y^{k,\chi_j}_j, g^k_j, R, h]$
        \EndFor 
\label{ms2gd:asdfsafdasfa2}        
\end{algorithmic}
\end{algorithm}

To explain the main idea behind the lazy/delayed updates,
consider that it happened that
during the fist $\tau$ iterations of the inner loop, 
the value of the fist coordinate in all datapoints which we have used was 0.
Then given the values of $y^{k,0}_1$
and $g^k_1$ we can compute the  true value
of  $y^{k,t}_1$ easily.
We just need to apply the $\prox$ operator $\tau$ times, i.e. 
$
y^{k,\tau}_1
 = \prox^\tau_1 [ y^{k,0} , g^k ,R,h],
$
where 
the function 
$\prox^\tau_1$ is described in Algorithm \ref{ms2gd:asdfsafdsaf}.
\begin{algorithm}
\caption{$\prox^{\tau}_j[y,g,R,h]$}
\label{ms2gd:asdfsafdsaf}
\begin{algorithmic}
\State $\tilde y^0 = y$
\For{
    $s=1,2,\dots, \tau$}
\State $\tilde y^s
  \gets 
  \prox_{hR}( \tilde y^{s-1}  
   - h 
   g  )$
  \EndFor
\State {\bf return} $\tilde y^\tau_j$
\end{algorithmic}
\end{algorithm}

The vector $\chi$ in Algorithm \ref{alg:ms2gd:mS2GDsparse}
is enabling us to keep track of the iteration when corresponding coordinate of $y$ was updated for the last time.
E.g. if in iteration $t$ we will be updating the $1^{st}$ coordinate for the first time, $\chi_1=0$
and after we compute and update the true value of $y_1$, its value will be set to 
$\chi_1=t$.
Lines \ref{ms2gd:asdfasfsafdsafa}-\ref{ms2gd:asdfasfsafdsafa2}
in 
 Algorithm \ref{alg:ms2gd:mS2GDsparse}
 make sure 
 that the coordinates of $y^{k,t}$ which will be read and used afterwards
 are up-to-date. 
At the end of the inner loop, we
 will updates all coordinates of $y$ to the most recent value (lines \ref{ms2gd:asdfsafdasfa}-\ref{ms2gd:asdfsafdasfa2}). Therefore, those lines make sure that the $y^{k,t^k}$ of Algorithms \ref{alg:ms2gd:mS2GD} and \ref{alg:ms2gd:mS2GDsparse} will be the same.

However, one could claim that we are not saving any work, as when needed, we still have to compute the proximal operator many times. Although this can be true for a general function $R$, for particular cases
$R(w) = \frac{\lambda}{2} \|w\|^2$
and $R(w) = \lambda \|w\|_1^2$, we provide following Lemmas which give a closed form expressions for the $\prox^\tau_j$ operator.

\begin{lemma}[Proximal Lazy Updates with $\ell_2$-Regularizer]
\label{lemma:ms2gd:L2regularizer}
If $R(w) = \frac{\lambda}{2} \| w \|^2$
with $\lambda > 0$ then
\begin{align*}
\prox_j^\tau[ y, g, R, h]
= \beta^{\tau} y_j - \frac{ h\beta}{1-\beta}
                \left(1 - \beta^{\tau} \right)  g_j ,
\end{align*}
where $\beta \eqdef 1/(1+\lambda  h)$.
\end{lemma}

\begin{lemma}[Proximal Lazy Updates with $\ell_1$-Regularizer]
\label{lemma:ms2gd:L1regularizer}
Assume that $R(w) = \lambda \|w\|_1$
with $\lambda > 0$.
Let us define 
$M$ and $m$ as follows,
\begin{equation*}
M=[\lambda+ g_j]h,  \qquad m = -[\lambda - g_j]h,
\end{equation*}
and let $[\cdot]_+ \eqdef \max\{\cdot, 0\}$.
Then the value of
$\prox_j^\tau[ y, g, R, h]$
can be expressed based on one of the 3 situations described below:
\begin{enumerate}
\item If $g_j \geq \lambda$, then by letting $p \eqdef \left\lfloor \frac{y_j}{M}\right\rfloor$, the operator can be defined as
\begin{align*}
\prox_j^\tau[ y, g, R, h] =
\begin{cases}
y_j - \tau M, &\text{if }p\geq\tau,\\
\min\{y_j - [p]_+ M, m\} -(\tau - [p]_+)m, &\text{if }p<\tau.
\end{cases}
\end{align*}

\item If $-\lambda < g_j < \lambda$, then the operator can be defined as
\begin{align*}
\prox_j^\tau[ y, g, R, h] &
=\begin{cases}
\max\{ y_j - \tau M, 0\}, & \text{if } y_j \geq 0,\\
\min\{ y_j - \tau m, 0\}, & \text{if } y_j < 0.
\end{cases}
\end{align*}

\item If $g_j \leq -\lambda$, then by letting $q \eqdef \left\lfloor\frac{y_j}{m}\right\rfloor$, the operator can be defined as
\begin{align*}
&\prox_j^\tau[ y, g, R, h]  
\\&=\begin{cases}
y_j - \tau m, &\text{if }q\geq\tau,\\
\max\{y_j - [q]_+ m, M\} -(\tau-[q]_+)M, &\text{if }q<\tau.
\end{cases}
\end{align*}
\end{enumerate}

\end{lemma}

The proofs of Lemmas~\ref{lemma:ms2gd:L2regularizer} and~\ref{lemma:ms2gd:L1regularizer} are available in Section~\ref{section:ms2gd:proxl}.

{\em Remark:} Upon completion of this work, we learned that similar ideas of lazy updates were proposed in \cite{langford2009} and \cite{carpenter2008} for online learning and multinomial logistic regression, respectively. However, our method can be seen as a more general result applied to a stochastic gradient method and its variants under Assumptions~\ref{asm:ms2gd:structure} and \ref{asm:ms2gd:structure2}.


\section{Experiments}

In this section we perform numerical experiments to illustrate the properties and performance of our algorithm. In Section~\ref{sec:ms2gd:mS2GD_properties} we study the total workload and parallelization speedup of mS2GD as a function of the mini-batch size $b$. In  Section~\ref{sec:ms2gd:mS2GD_algorithms} we compare  mS2GD with several other algorithms.
Finally, in Section~\ref{sec:ms2gd:deblur}  we briefly illustrate that our method can be efficiently applied to a deblurring problem.

In Sections~\ref{sec:ms2gd:mS2GD_properties} and \ref{sec:ms2gd:mS2GD_algorithms} we conduct experiments with $R(w)=\tfrac{\lambda}{2}\|w\|^2$ and $f$ of the form \eqref{eq:ms2gd:Px2}, where $f_i$ is the logistic loss function:
\begin{equation}
\label{eq:ms2gd:logisticL2}
f_i(w) = \log [1+\exp(-b_i a_i^Tw)].
\end{equation}
These functions are often used in machine learning, with  $(a_i, b_i) \in \R^d\times \{+1,-1\}$, $i=1,\dots,n$, being a training dataset of example-label pairs. The resulting optimization problem  \eqref{eq:ms2gd:Px1}+\eqref{eq:ms2gd:Px2} takes the form
\begin{equation}
\label{eq:ms2gd:L2problem}
P(w) = \frac1n\sum_{i=1}^n f_i(w) + \frac{\lambda}{2}\| w \|^2,
\end{equation}
and is used in machine learning for binary classification. In these sections we have performed experiments on four publicly available binary classification datasets, namely \emph{rcv1,  news20, covtype}~\footnote{\emph{rcv1, covtype} and \emph{news20} are available at \url{http://www.csie.ntu.edu.tw/~cjlin/libsvmtools/datasets/}.} and \emph{astro-ph}~\footnote{Available at \url{http://users.cecs.anu.edu.au/~xzhang/data/}.}.

In the logistic regression problem, the Lipschitz constant of function $\nabla f_i$ is equal to $L_i = \|a_i\|^2/4$. Our analysis assumes (Assumption~\ref{ass:ms2gd:ass1}) the same constant $L$ for all functions. Hence, we have $L = \max_{i\in [n]} L_i$. We set the regularization parameter $\lambda = \frac1n$ in our experiments, resulting in the problem having the condition number $\kappa = \frac{L}{\mu} = \mathcal{O}(n)$. In Table~\ref{table:ms2gd:datasets} we  summarize  the four datasets, including the sizes $n$, dimensions $d$, their sparsity levels as a proportion of nonzero elements, and the Lipschitz constants $L$.

\begin{table}[H]
\small
\centering
\begin{tabular}{|c|c|c|c|c|c|}
\hline
Dataset  & $n$ & $d$ & Sparsity & $L$  \\
\hline \hline 
\emph{rcv1} & 20,242 & 47,236 & 0.1568\% & 0.2500\\
\hline 
\emph{news20} & 19,996  & 1,355,191  & 0.0336\% & 0.2500\\
\hline 
\emph{covtype} & 581,012 &  54  & 22.1212\% & 1.9040\\
\hline 
\emph{astro-ph} & 62,369  & 99,757  & 0.0767\% & 0.2500\\
\hline
\end{tabular}
\captionsetup{justification=centering,margin=0.5cm}
\caption{Summary of datasets used for experiments.}
\label{table:ms2gd:datasets}
\end{table}

\subsection{Speedup of mS2GD}
\label{sec:ms2gd:mS2GD_properties}

Mini-batches allow mS2GD to be accelerated on a  computer with a parallel processor. In Section~\ref{sec:ms2gd:speedup}, we have shown in  that up to some threshold mini-batch size, the total workload of mS2GD remains unchanged. Figure~\ref{fig:ms2gd:MBSpeedup} compares the best performance of mS2GD used with various  mini-batch sizes on datasets \emph{rcv1} and \emph{astro-ph}. An effective  pass (through the data) corresponds to $n$ units of work. Hence, the evaluation of a gradient of $f$ counts as one effective pass. In both cases, by increasing the mini-batch size to $b=2, 4, 8$, the performance of mS2GD is the same or better than that of S2GD ($b=1$) without any parallelism. 

\begin{figure*}[!htbp]
   \centering
    \epsfig{file=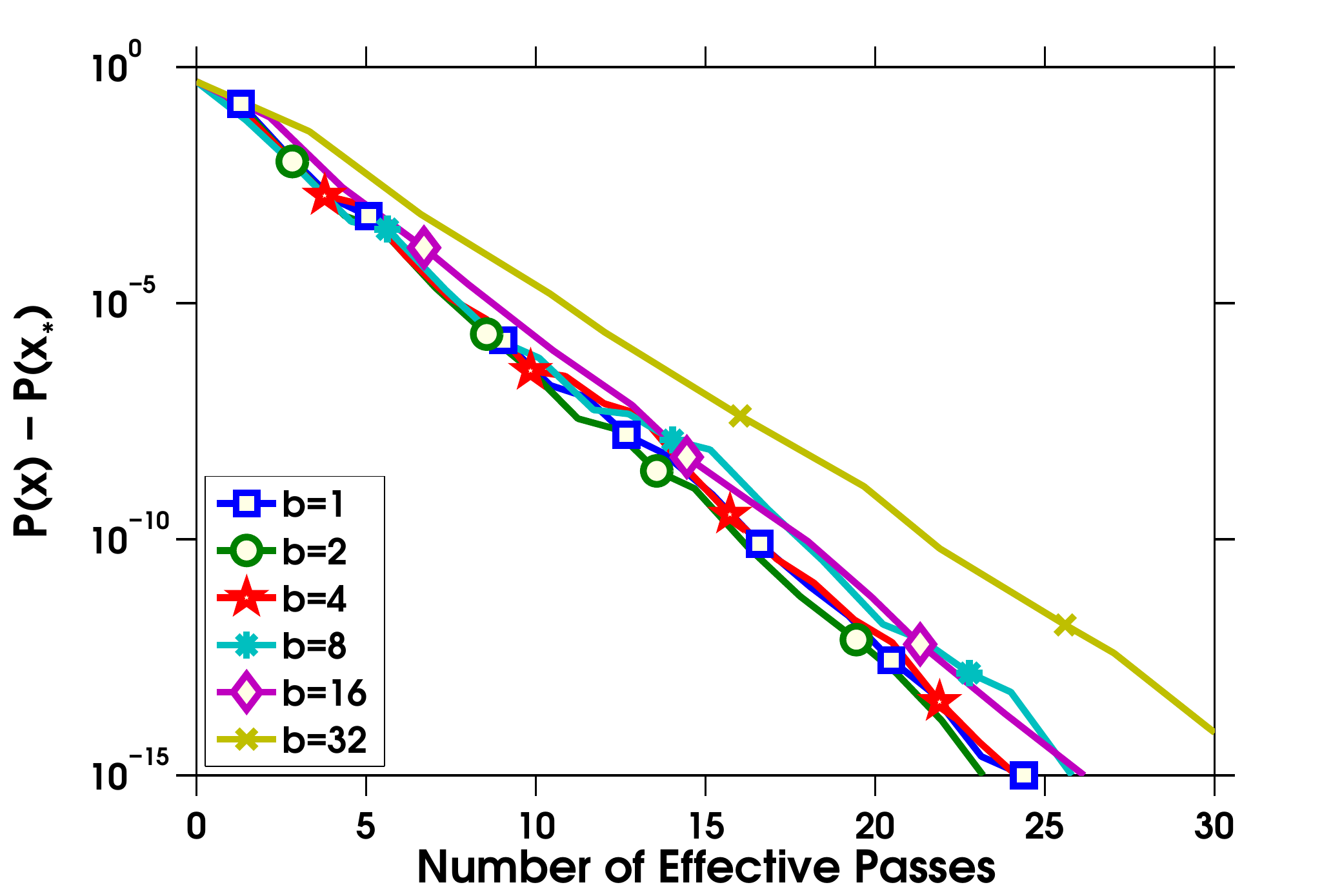,width=0.40\textwidth}
    \epsfig{file=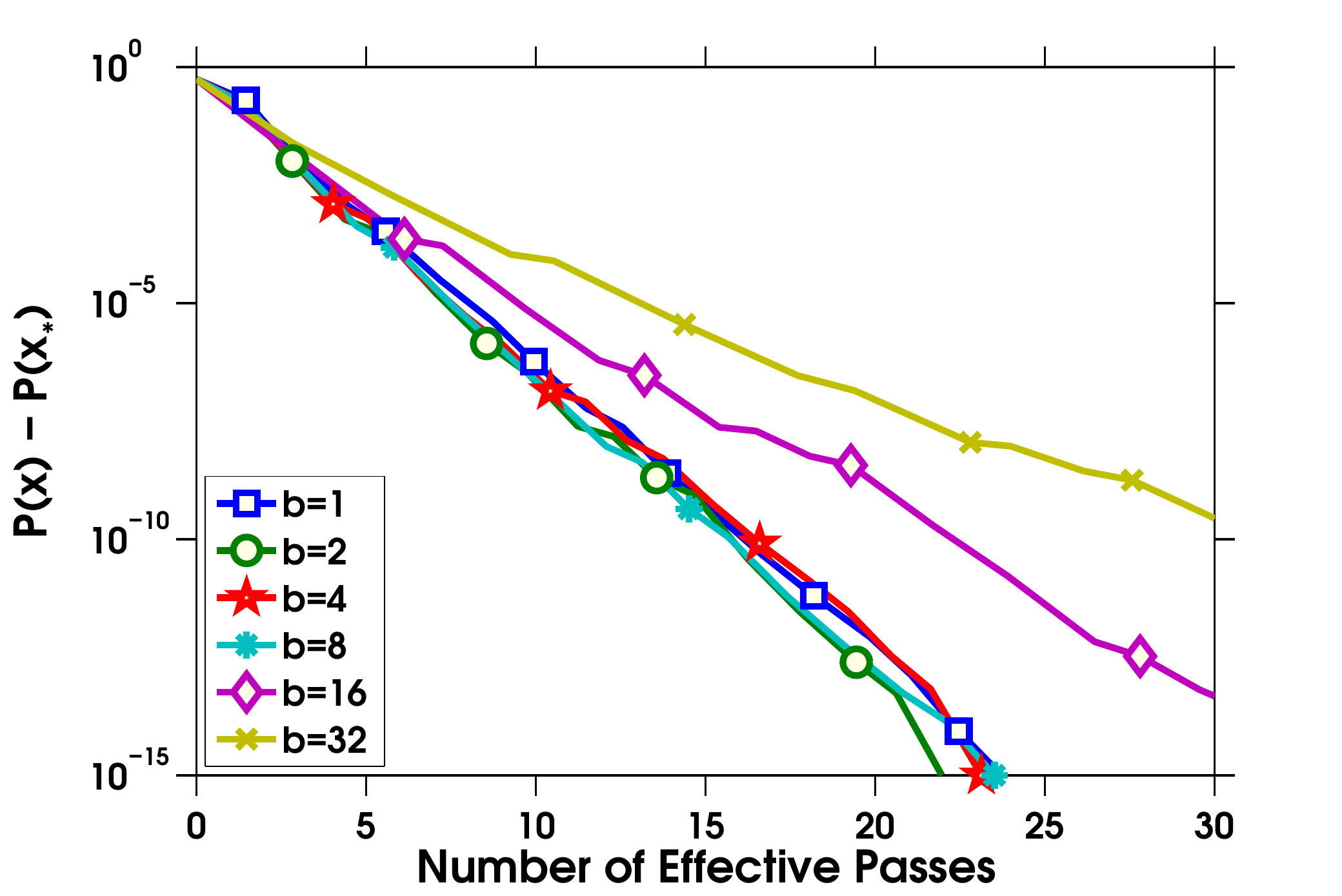,width=0.40\textwidth}
    \caption{\footnotesize Comparison of mS2GD with different mini-batch sizes on \emph{rcv1} (left) and \emph{astro-ph} (right).}
  \label{fig:ms2gd:MBSpeedup} 
\end{figure*}

\begin{figure*}[!htbp]
    \centering
    \epsfig{file=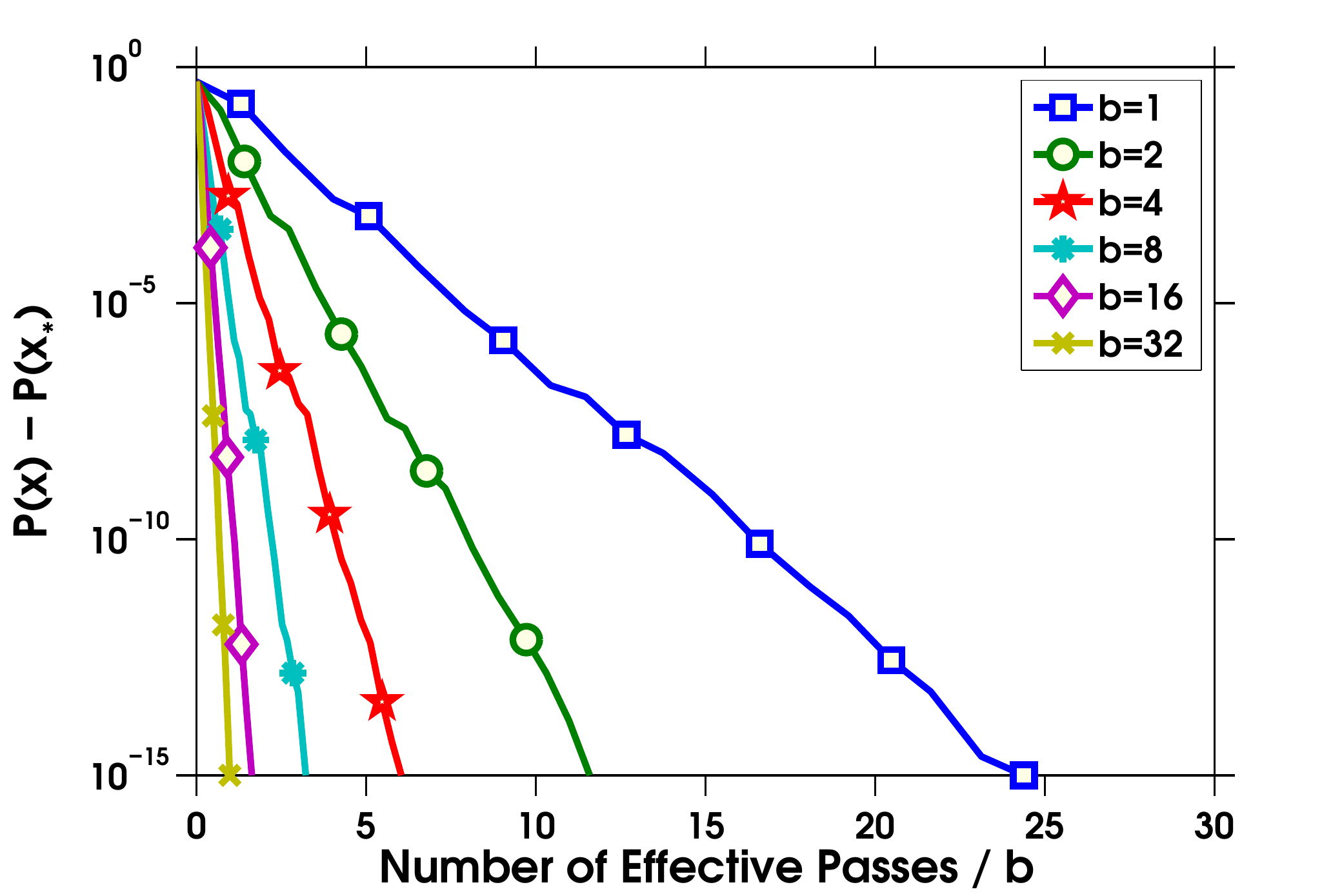,width=0.40\textwidth}
    \epsfig{file=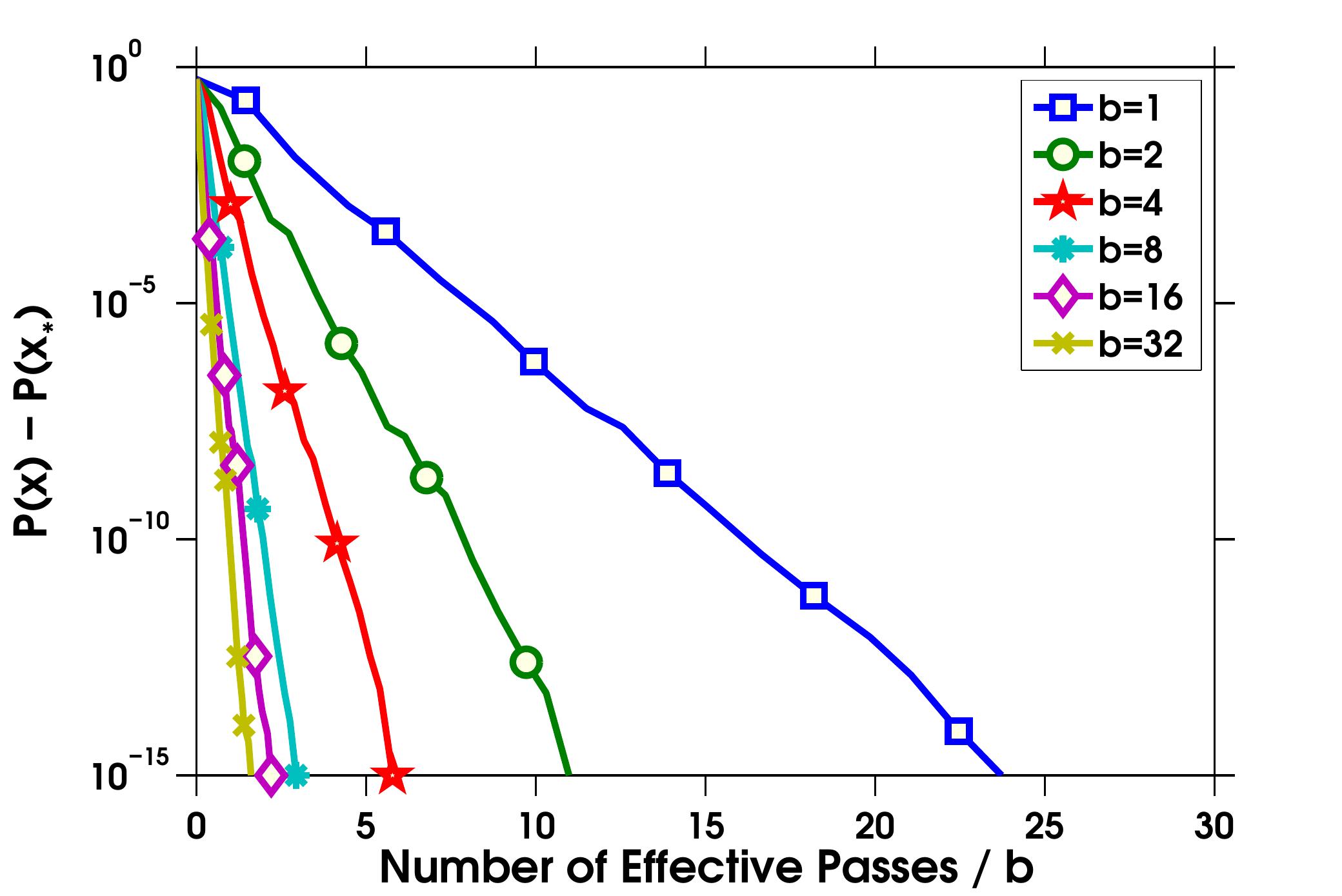,width=0.40\textwidth}
  \caption{\footnotesize Parallelism speedup   for \emph{rcv1} (left) and \emph{astro-ph} (right) in theory (unachievable in practice).}
  \label{fig:ms2gd:MBSpeedup with parallelism} 
\end{figure*}

Although for larger mini-batch sizes mS2GD would be obviously worse, the results are still promising with parallelism. In Figure~\ref{fig:ms2gd:MBSpeedup with parallelism},we show the ideal speedup---one that  would be achievable if we could always evaluate the $b$ gradients  in parallel in exactly the same amount of time as it would take to evaluate a single gradient.\footnote{In practice, it is impossible to ensure that the times of evaluating different component gradients are the same.}.

\subsection{mS2GD vs other algorithms}
\label{sec:ms2gd:mS2GD_algorithms}

In this part, we implemented the following algorithms to conduct a numerical comparison:\\
1) \textbf{SGDcon}: Proximal stochastic gradient descent method with a constant step-size which gave the best performance in hindsight.\\
2) \textbf{SGD+}: Proximal stochastic gradient descent with variable step-size $h=h_0/(k+1)$, where $k$ is the number of effective passes, and $h_0$ is some initial constant step-size. \\
3) \textbf{FISTA}: Fast iterative shrinkage-thresholding algorithm proposed in~\cite{fista}. \\
4) \textbf{SAG}: Proximal version of the stochastic average gradient algorithm~\cite{SAG}. Instead of using $h=1/16L$, which is analyzed in the reference, we used a constant step size.\\
5) \textbf{S2GD}: Semi-stochastic gradient descent method proposed in~\cite{S2GD}. We applied proximal setting to the algorithm and used a constant stepsize.\\
6) \textbf{mS2GD}:  mS2GD with mini-batch size $b=8$. Although a safe step-size is given in our theoretical analyses in Theorem~\ref{thm:ms2gd:s2convergence}, we ignored the bound, and used a constant step size.

In all cases, unless otherwise stated, we have used the best constant stepsizes in hindsight.

\begin{figure*}[!htbp]
\centering
 \epsfig{file=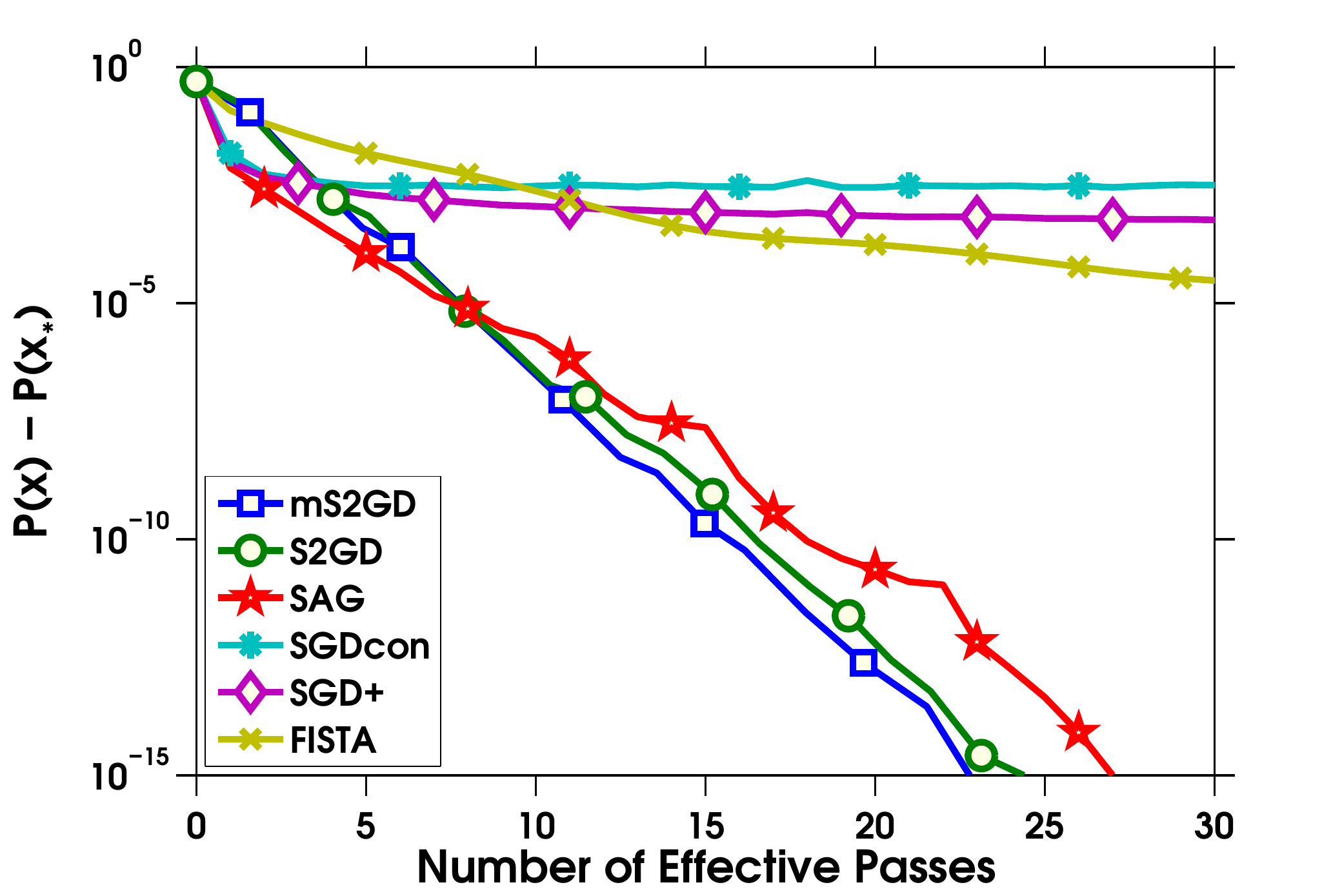,width=0.40\textwidth}
 \epsfig{file=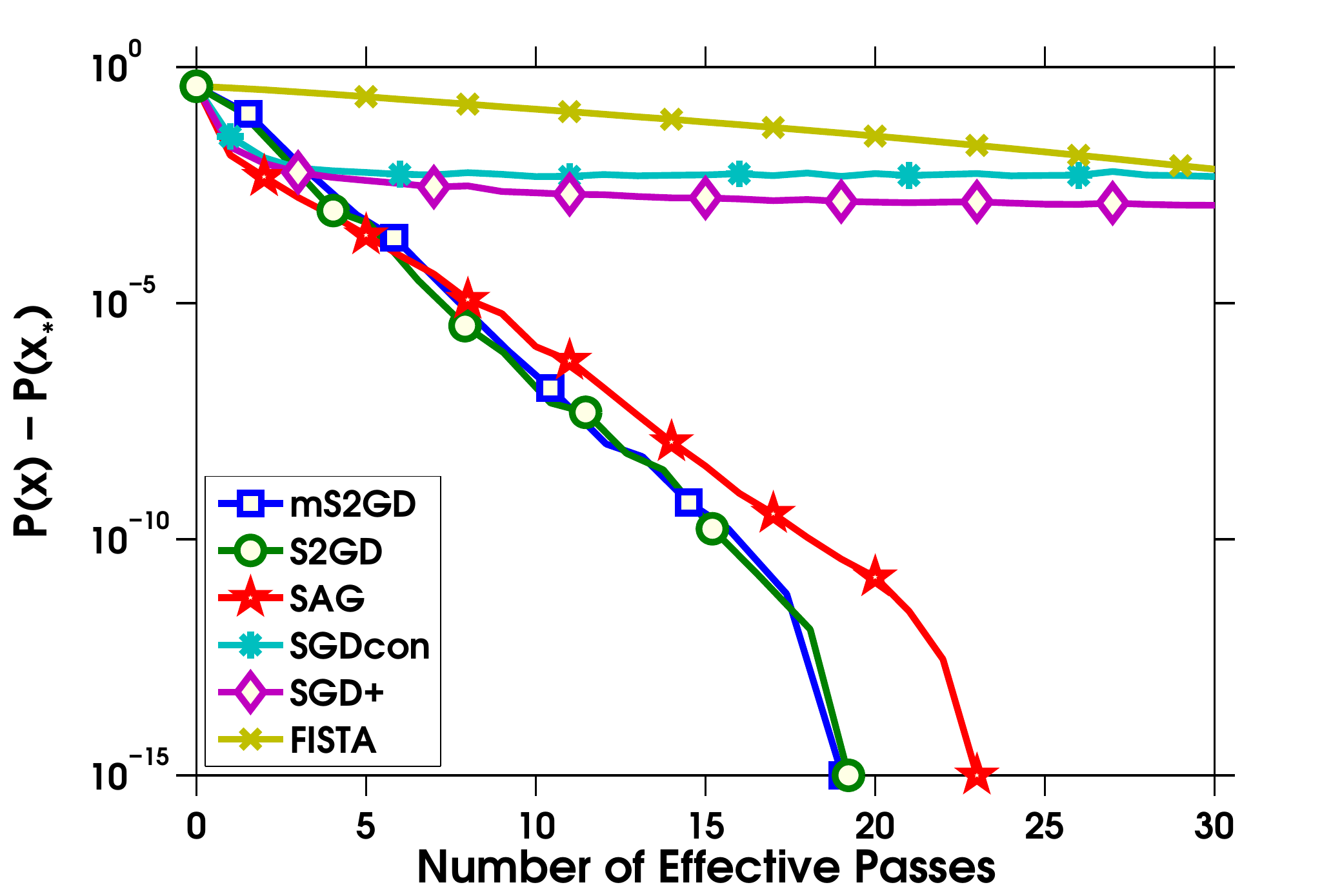,width=0.40\textwidth}
 \\
 \epsfig{file=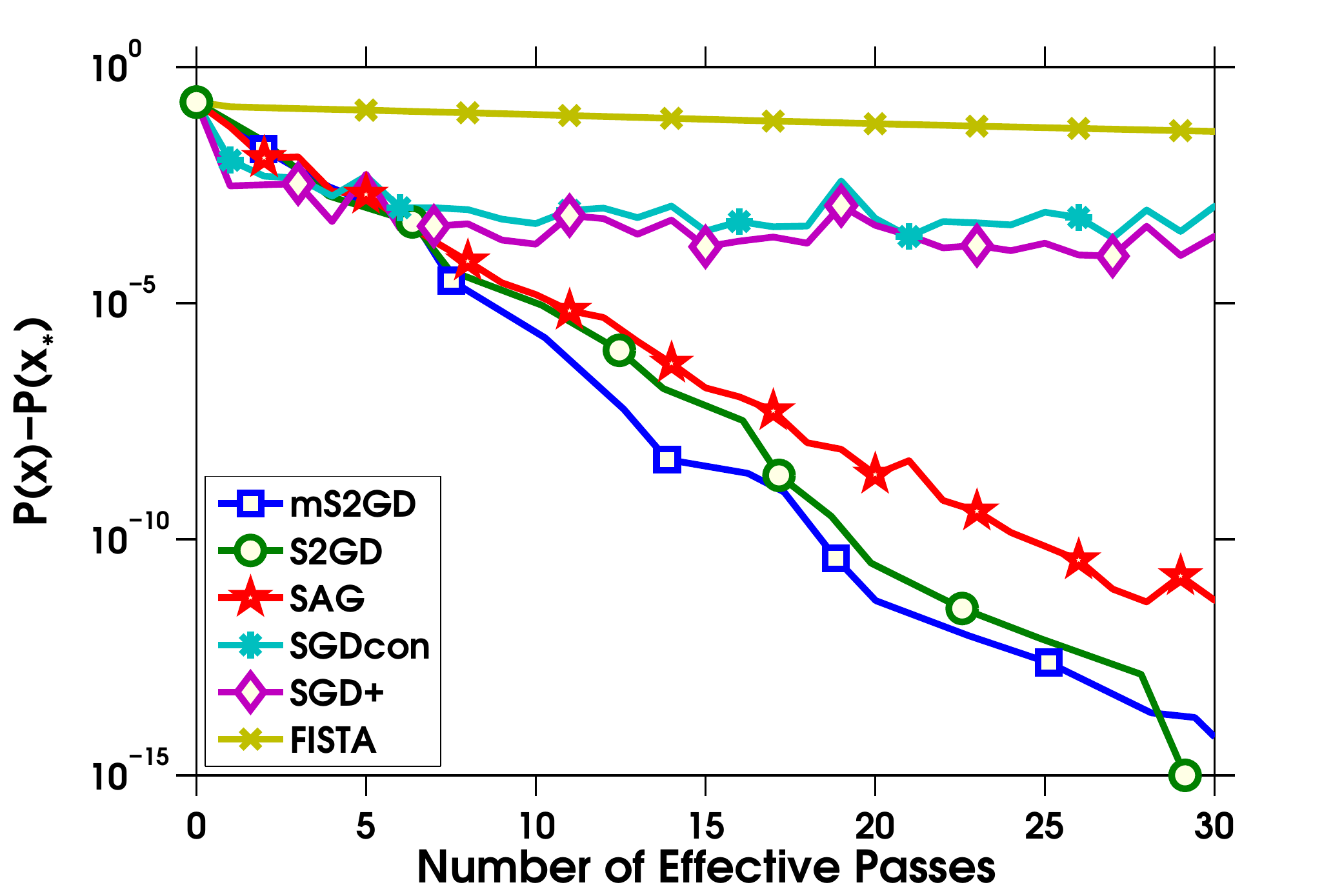,width=0.40\textwidth}
  \epsfig{file=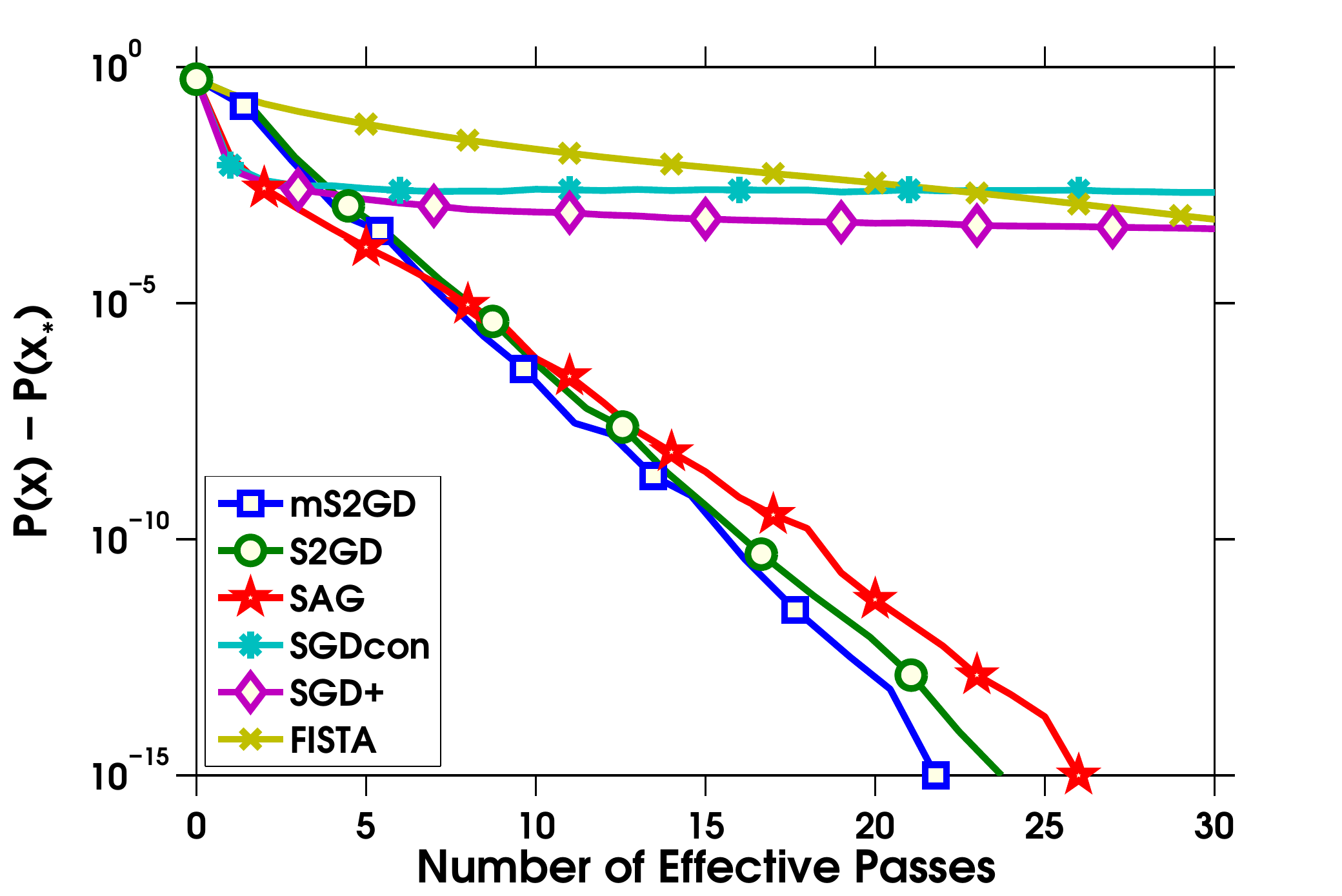,width=0.40\textwidth}
 \caption{\footnotesize Comparison of several algorithms on four datasets: \emph{rcv1} (top left), \emph{news20} (top right), \emph{covtype} (bottom left) and \emph{astro-ph} (bottom right). We have used  mS2GD with $b=8$.}
\label{fig:ms2gd:comparison_algorithms}
\end{figure*}

Figure~\ref{fig:ms2gd:comparison_algorithms} demonstrates the superiority of mS2GD over other  algorithms in the test pool on the four datasets described above. For mS2GD, the best choices of parameters with $b=8$ are given in Table~\ref{table:ms2gd:parameters}.
\begin{table}[H]
\small
\centering
\begin{tabular}{|c|c|c|c|c|c|}
\hline
Parameter & \emph{rcv1} & \emph{news20} & \emph{covtype} & \emph{astro-ph} \\
\hline \hline 
$ m$ & 0.11$n$ & 0.10$n$ & 0.26$n$ & 0.08$n$\\
\hline
  $h$ & $5.5/L$  & $6/L$  & $4.5/L$ & $6.5/L$\\
\hline
\end{tabular}
\captionsetup{justification=centering,margin=0.5cm}
\caption{Best choices of parameters in mS2GD.}
\label{table:ms2gd:parameters}
\end{table}

\subsection{Image deblurring}
\label{sec:ms2gd:deblur}

In this section we utilize the Regularization Toolbox 
\cite{hansen2007regularization}.\footnote{Regularization Toolbox 
available for Matlab can be obtained from \url{http://www.imm.dtu.dk/~pcha/Regutools/} .}  We use the \emph{blur} function 
available therein to obtain the original image and generate a blurred image (we choose following values of parameters for blur function: $N=256$, band=9, sigma=10). The purpose of the blur function is to generate a test problem with an atmospheric turbulence blur. In addition, an additive Gaussian white noise with stand deviation of $10^{-3}$ is added to the blurred image. This forms our testing image as a vector $b$. The image dimension of the test image is $256\times 256$, which means that $n=d=65,536$.  We would not expect our method to work particularly well on this problem since mS2GD works best when $d\ll n$. However, as we shall see, the method's performance is on a par with the performance of the best methods in our test pool.

Our  goal is to reconstruct (deblur) the original image $x$ by solving a LASSO problem: $\min_w \| Aw - b\|_2^2+\lambda \| w \|_1$. We have chosen  $\lambda =10^{-4}$. In our implementation, we normalized the objective function by $n$, and hence our objective value  being optimized is in fact $\min_w \frac1n\| Aw - b\|_2^2+\lambda \|w\|_1$, where $\lambda = \frac{10^{-4}}{n}$,
similarly as was done in \cite{fista}. 

\begin{figure} [!htbp]
\centering
\epsfig{file=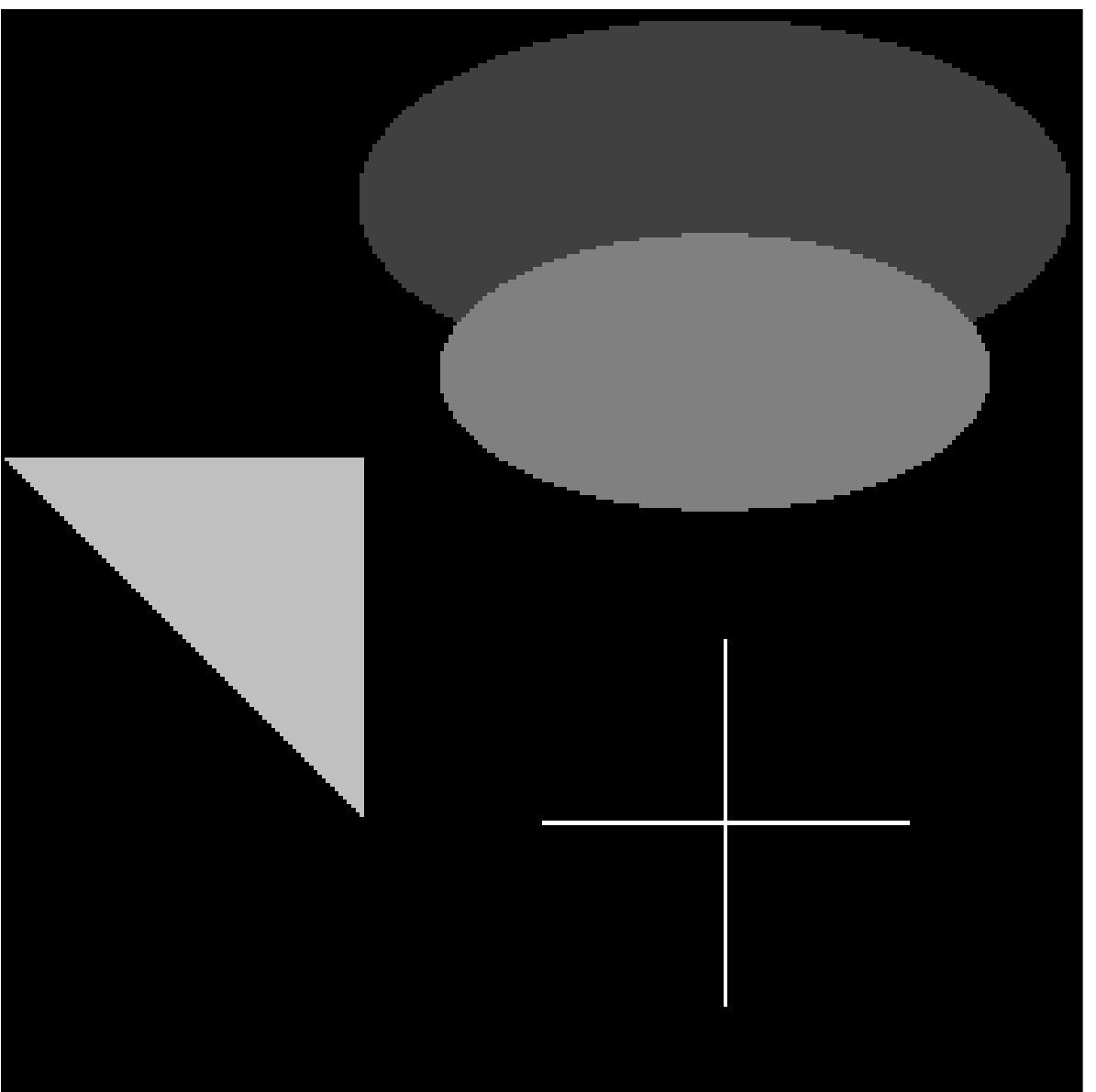,width=0.16\textwidth}
\epsfig{file=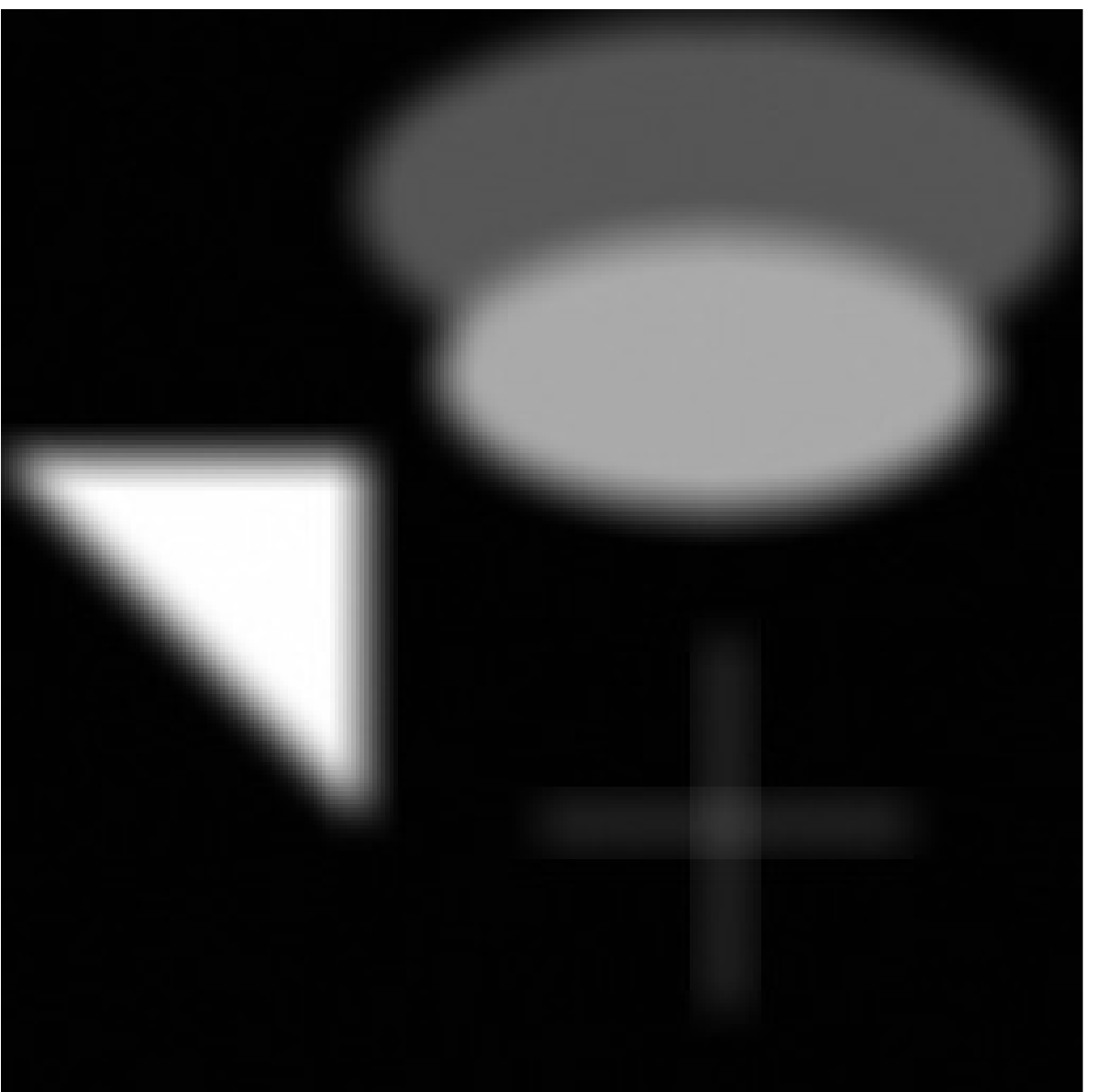,width=0.16\textwidth}
\caption{\footnotesize Original (left) and blurred \& noisy (right) test image.}
\label{fig:ms2gd:afcewvcawfecvwa}
\end{figure}

\begin{figure} [!htbp]
\centering
\epsfig{file=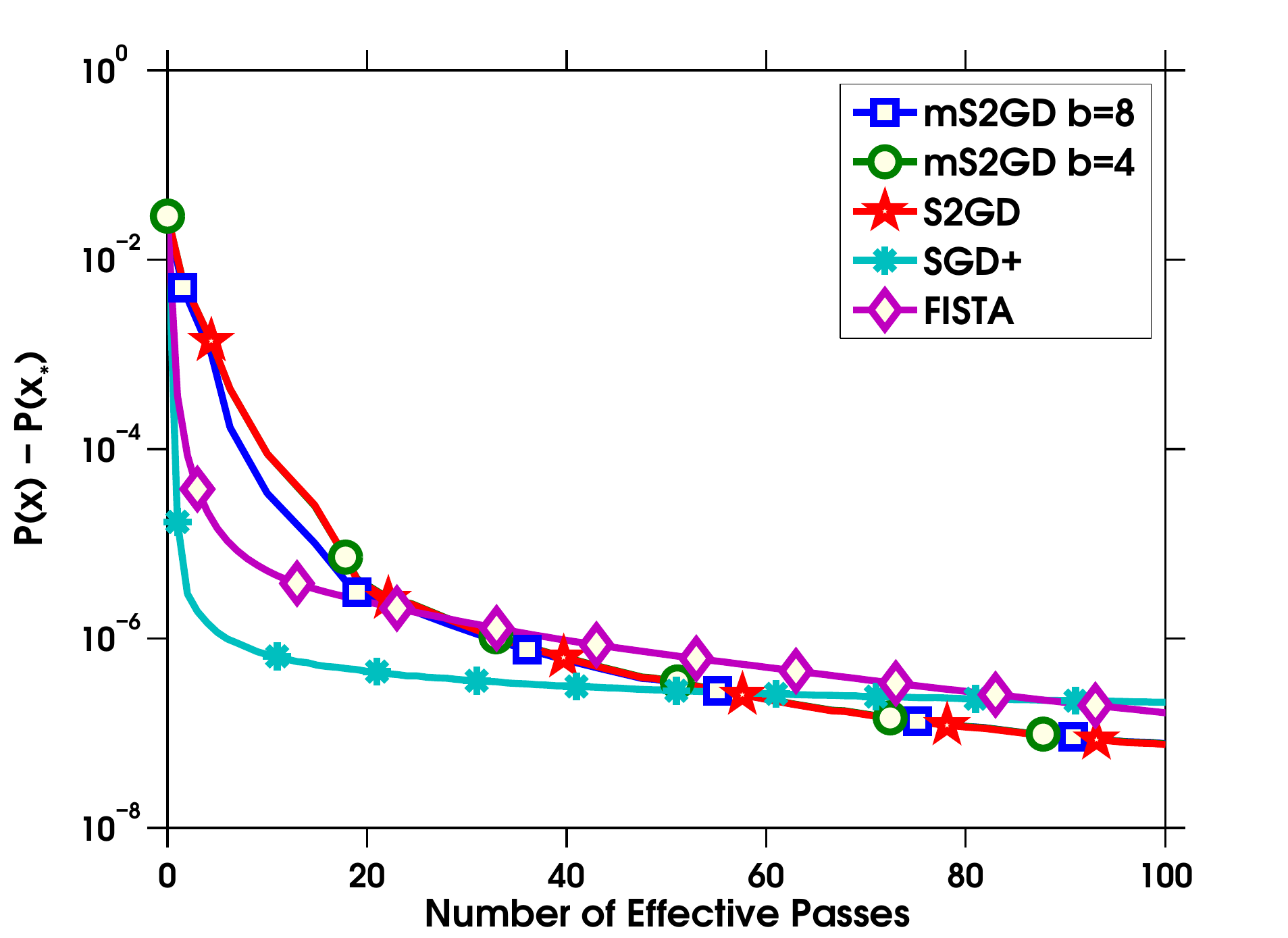,width=0.40\textwidth}
\caption{\footnotesize Comparison of several algorithms for the deblurring problem.}
\label{fig:ms2gd:asfvfrawfrvwa}
\end{figure}

Figure~\eqref{fig:ms2gd:afcewvcawfecvwa} shows the original test image (left) and a blurred image with added Gaussian noise (right).
Figure~\ref{fig:ms2gd:asfvfrawfrvwa} compares the mS2GD algorithm with SGD+, S2GD and FISTA. We run all algorithms for 100 epochs and plot the error. The plot suggests that SGD+ decreases the objective function very rapidly at beginning, but slows down after 10-20 epochs. 
 
Finally, Fig.~\ref{fig:ms2gd:adsfaewfawefa}
shows the reconstructed image after
$T=20, 60, 100$ epochs.

\begin{figure*}[!htbp]
\centering
    \begin{tabular}{r l}
    $T = 20$ \ &
        \begin{subfigure}{0.16\textwidth}
        \centering
        \caption*{FISTA}
        \epsfig{file=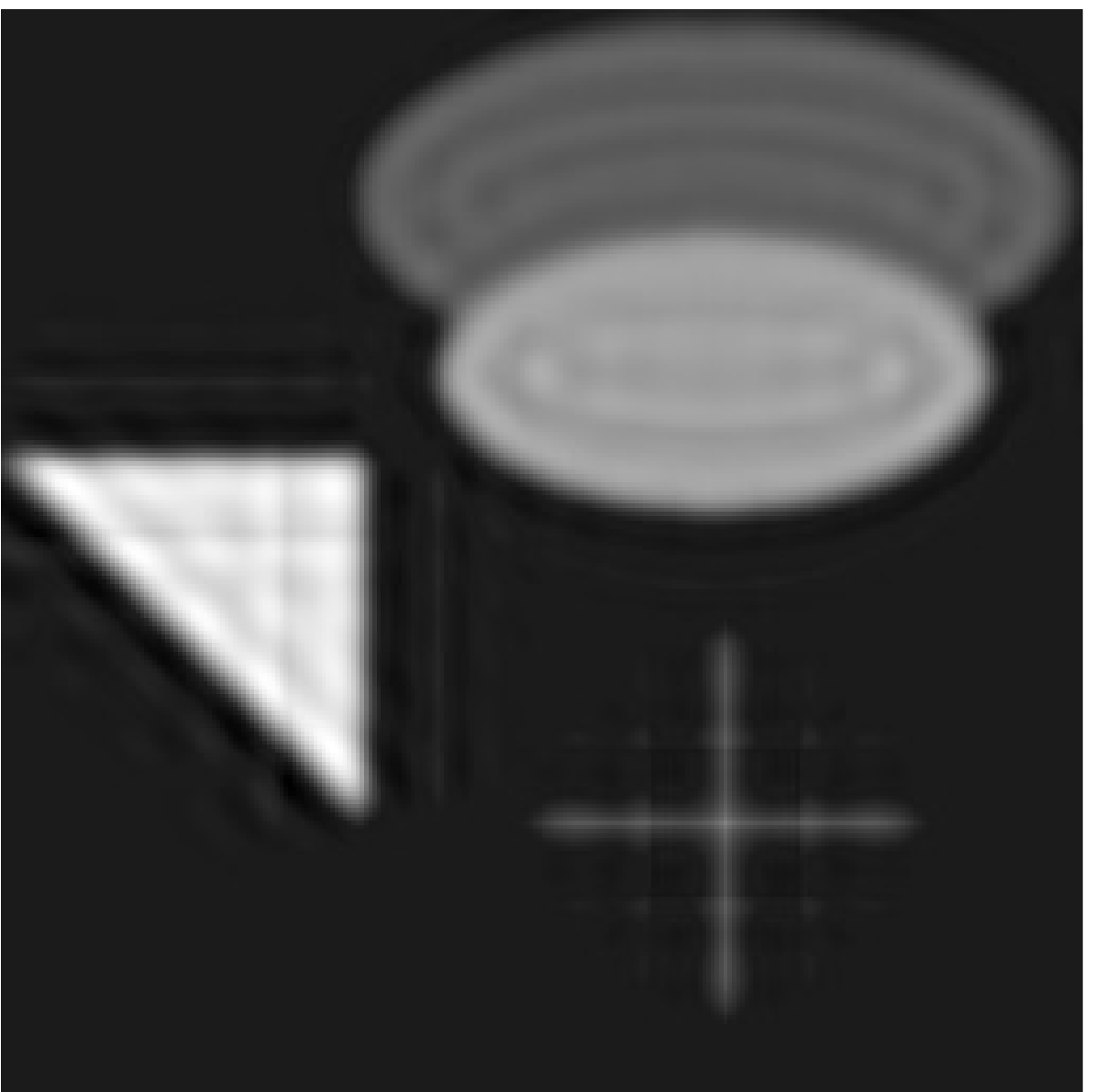,width=1\textwidth}
        \end{subfigure}
        \begin{subfigure}{0.16\textwidth}
        \centering
        \caption*{SGD+}
        \epsfig{file=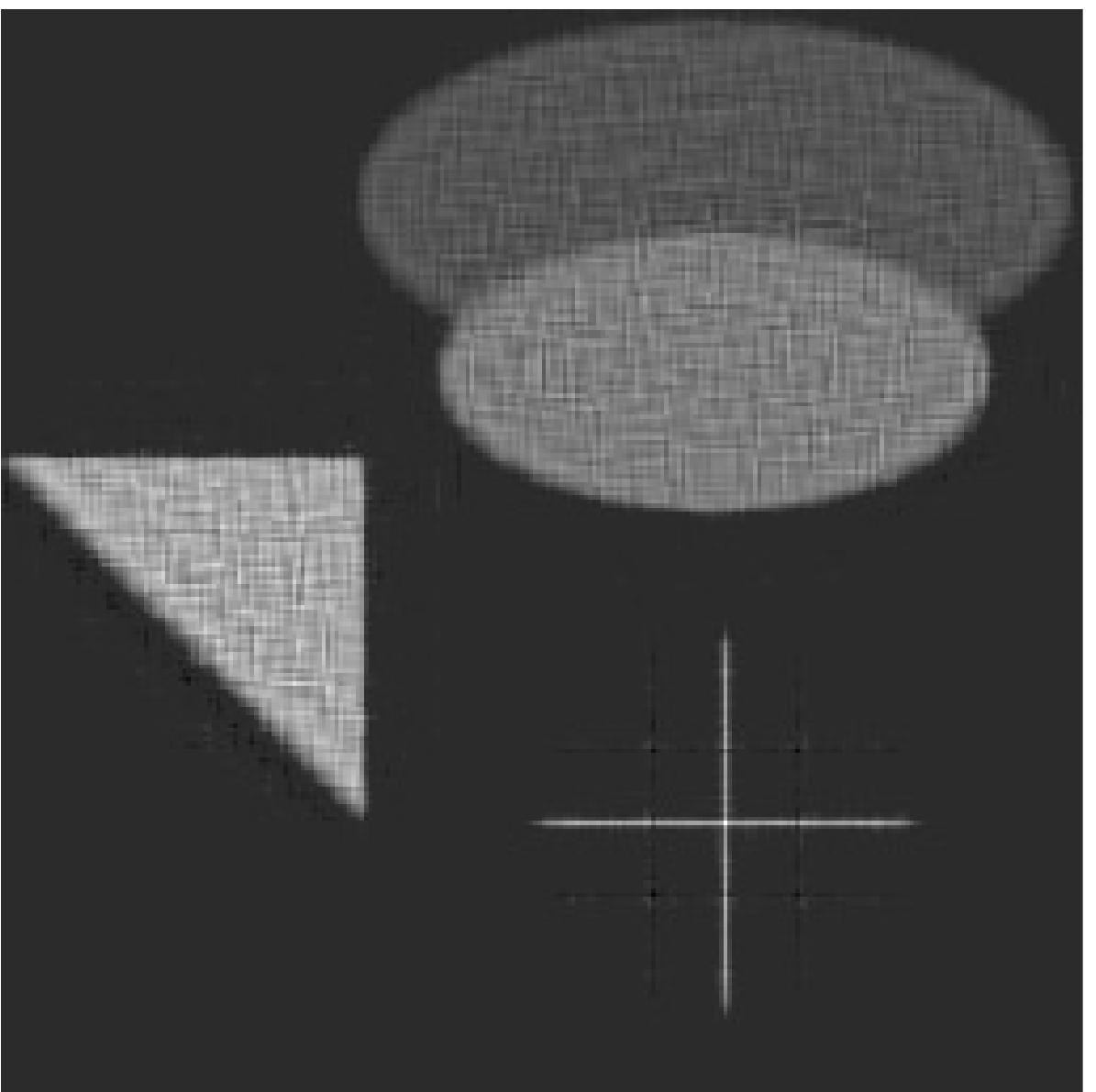,width=1\textwidth}
        \end{subfigure}
        \begin{subfigure}{0.16\textwidth}
        \centering
        \caption*{S2GD}
        \epsfig{file=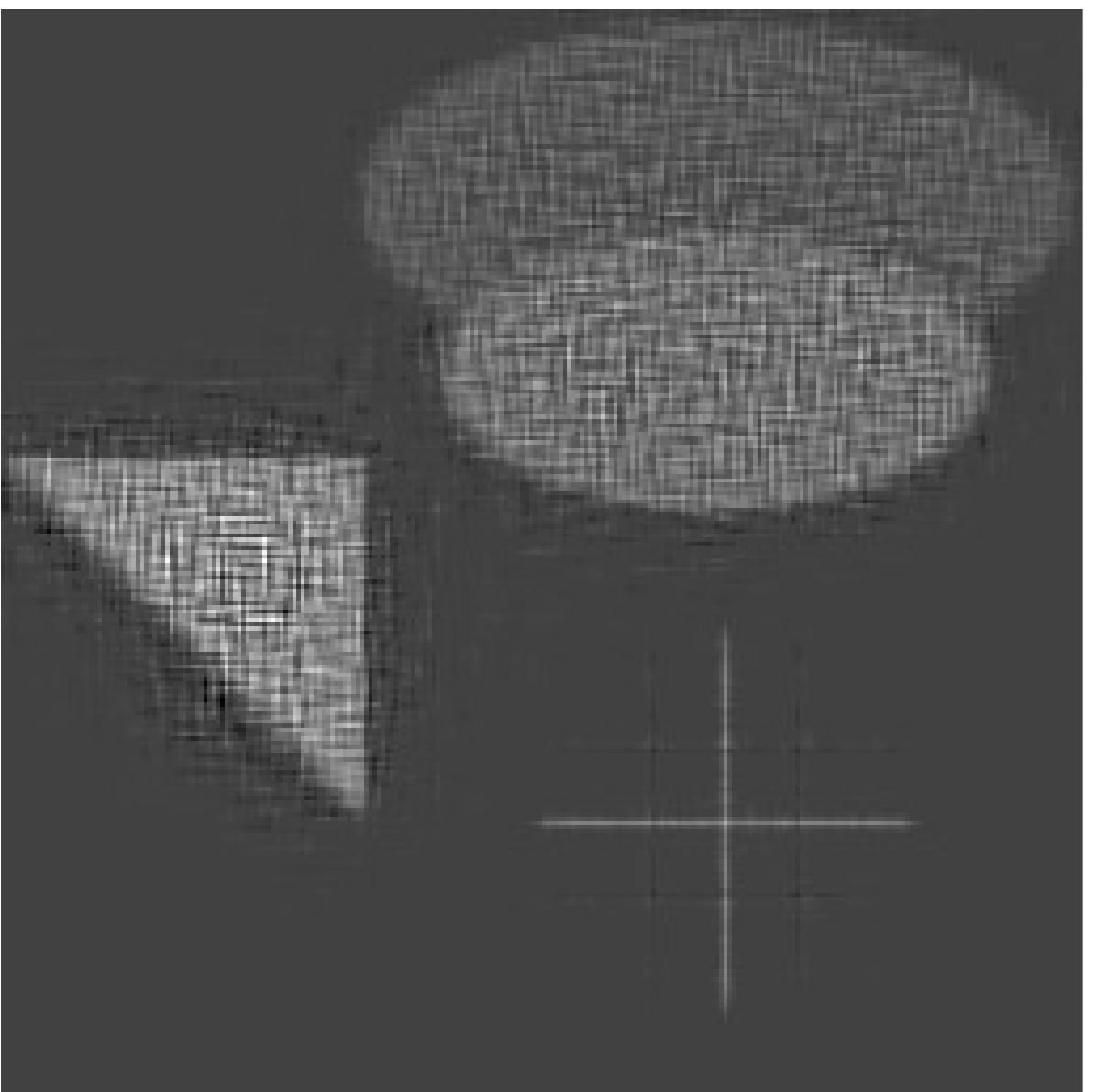,width=1\textwidth}
        \end{subfigure}
        \begin{subfigure}{0.16\textwidth}
        \centering
        \caption*{mS2GD $(b=4)$}
        \epsfig{file=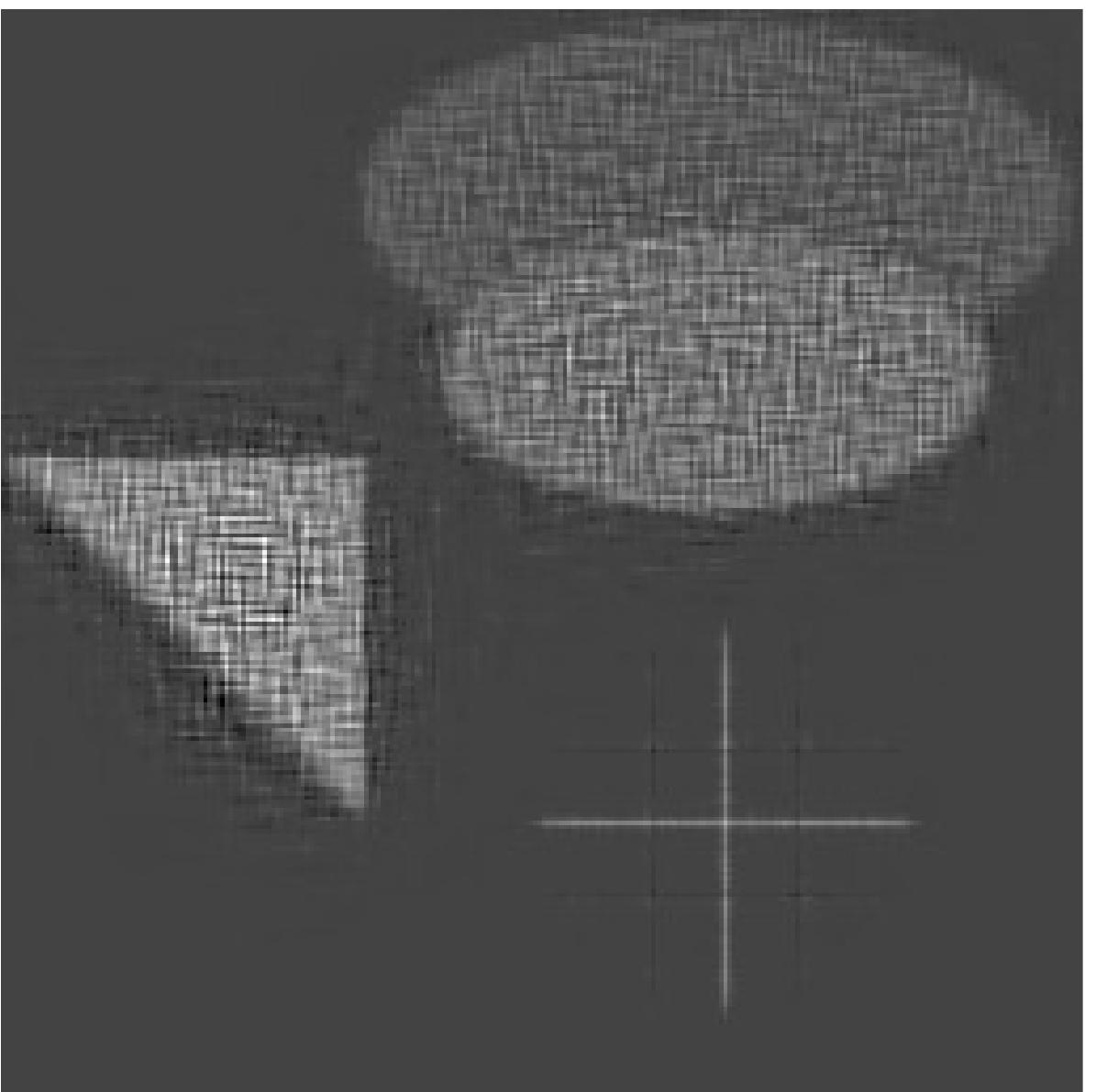,width=1\textwidth}
        \end{subfigure}
        \begin{subfigure}{0.16\textwidth}
        \centering
        \caption*{mS2GD $(b=8)$}
        \epsfig{file=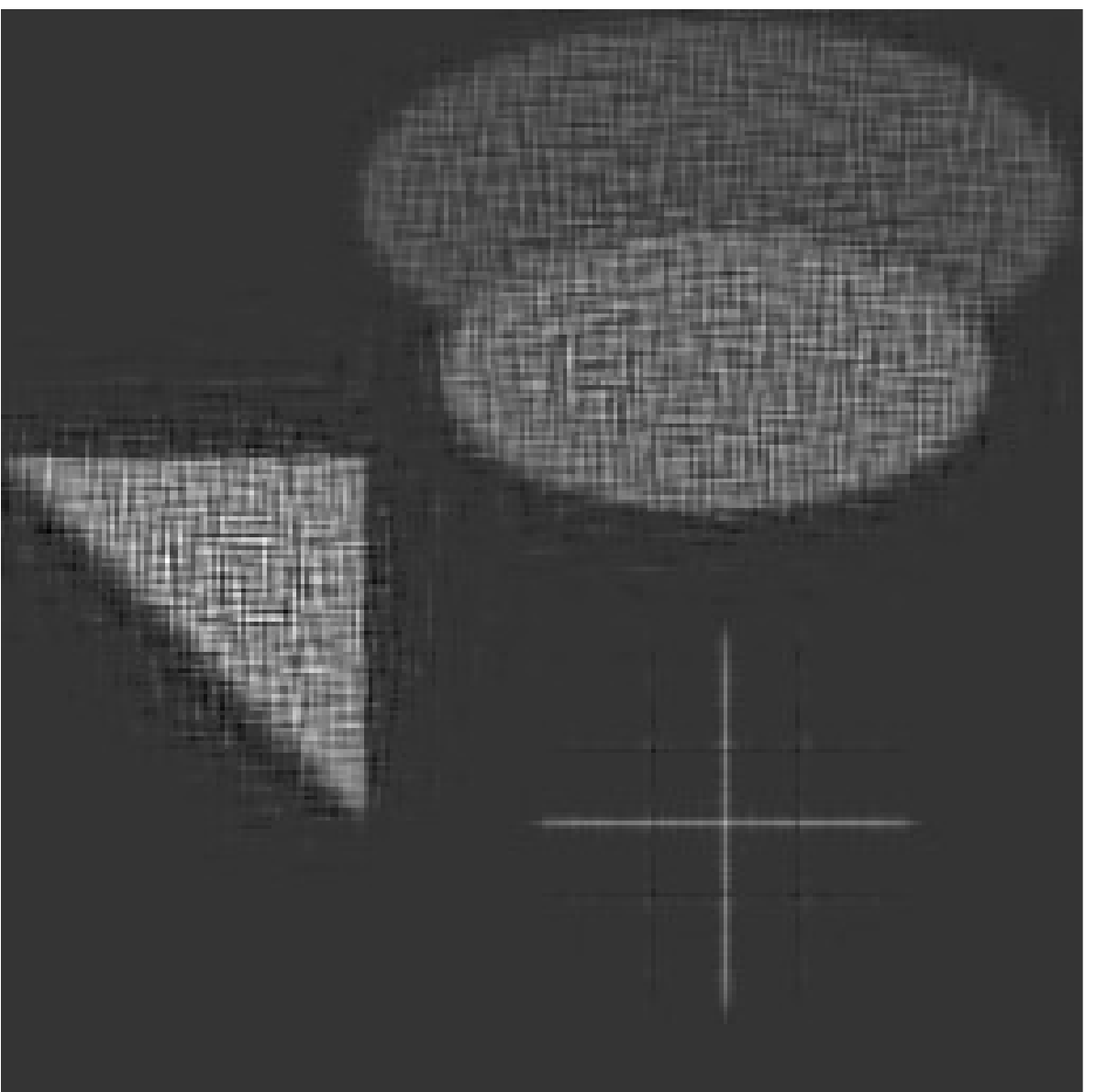,width=1\textwidth}
        \end{subfigure}
  \\$T = 60$ \ &
        \begin{subfigure}{0.16\textwidth}
        \centering
        \epsfig{file=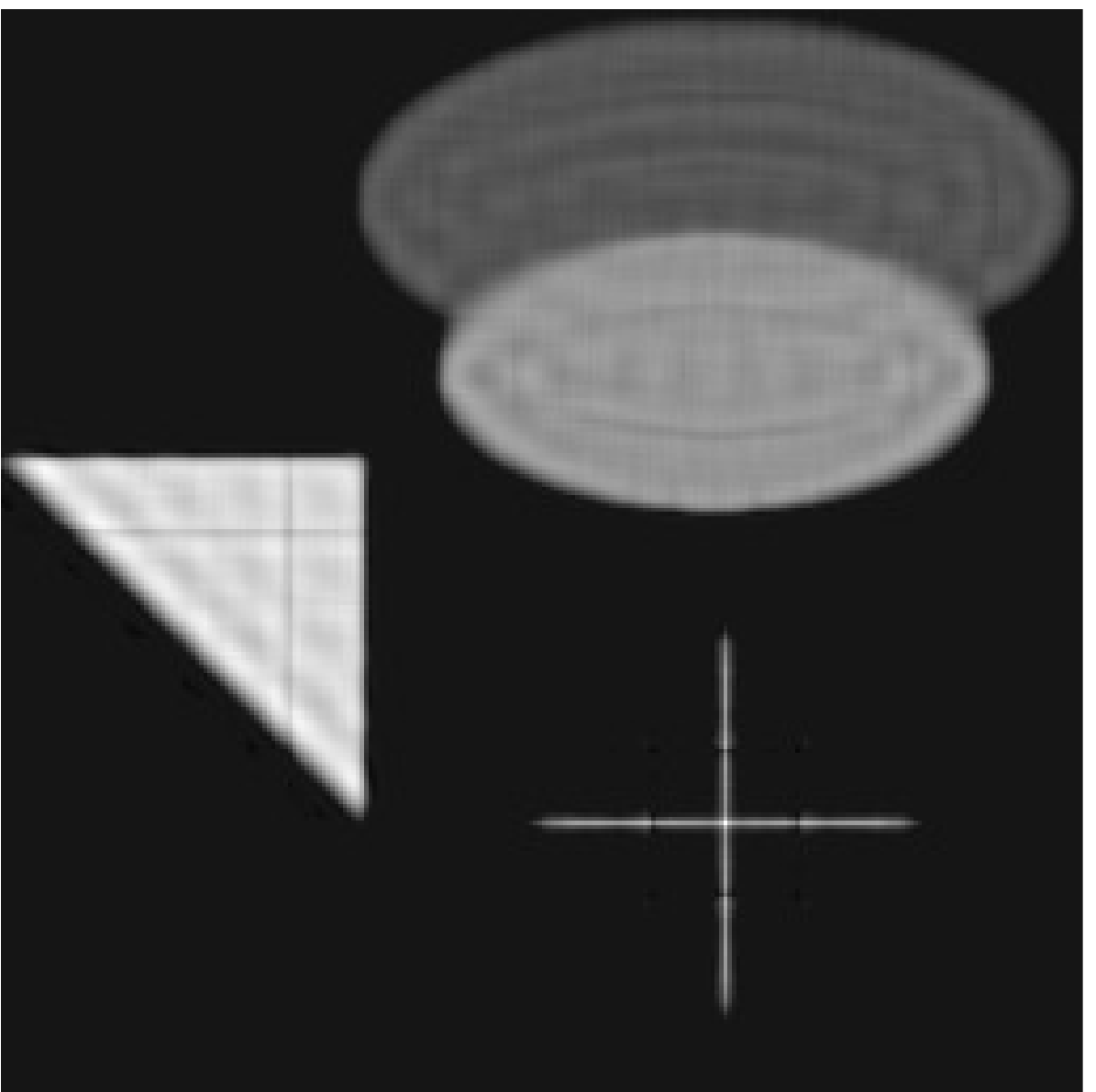,width=1\textwidth}
        \end{subfigure}
        \begin{subfigure}{0.16\textwidth}
        \centering
        \epsfig{file=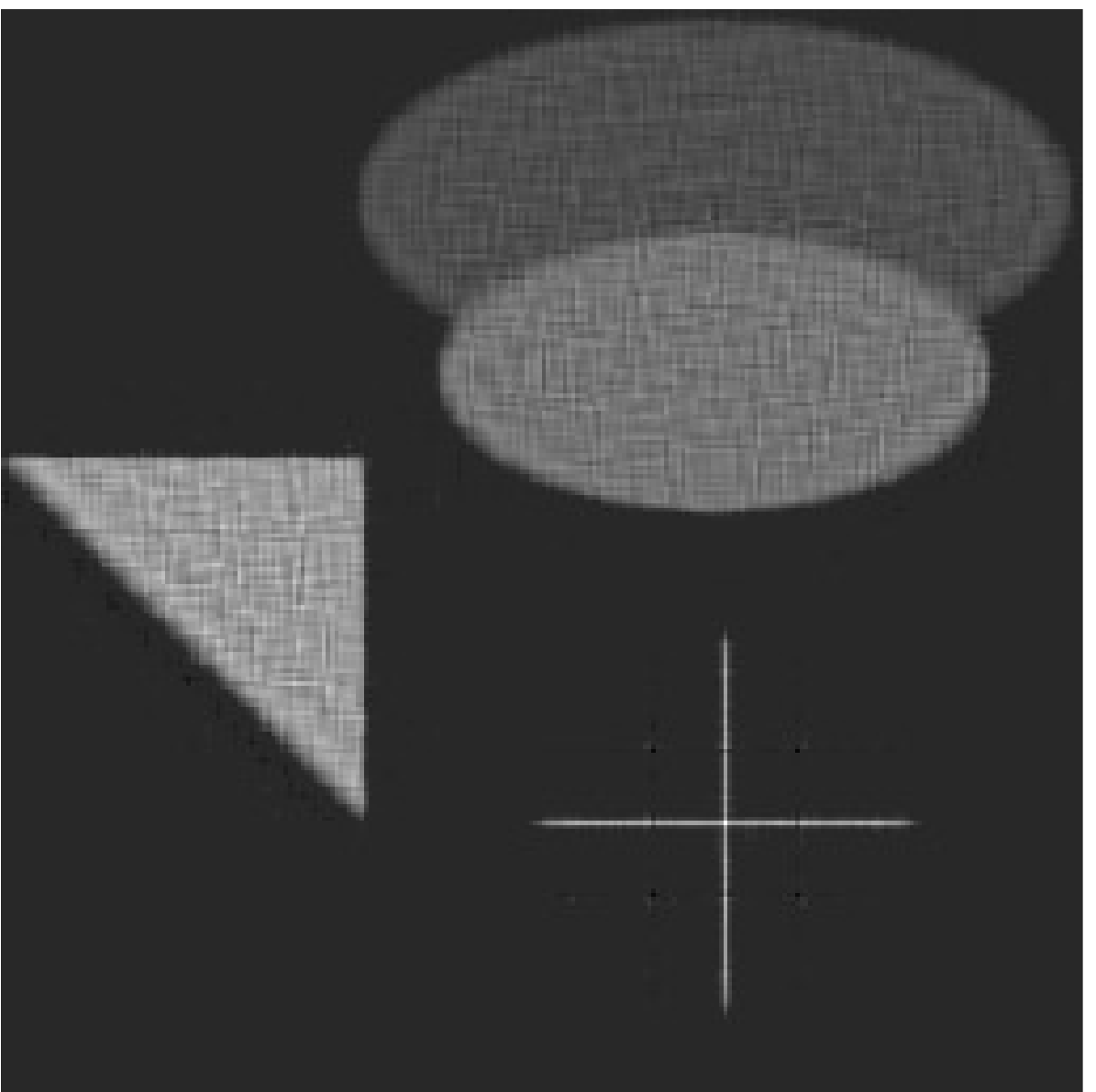,width=1\textwidth}
        \end{subfigure}
        \begin{subfigure}{0.16\textwidth}
        \centering
        \epsfig{file=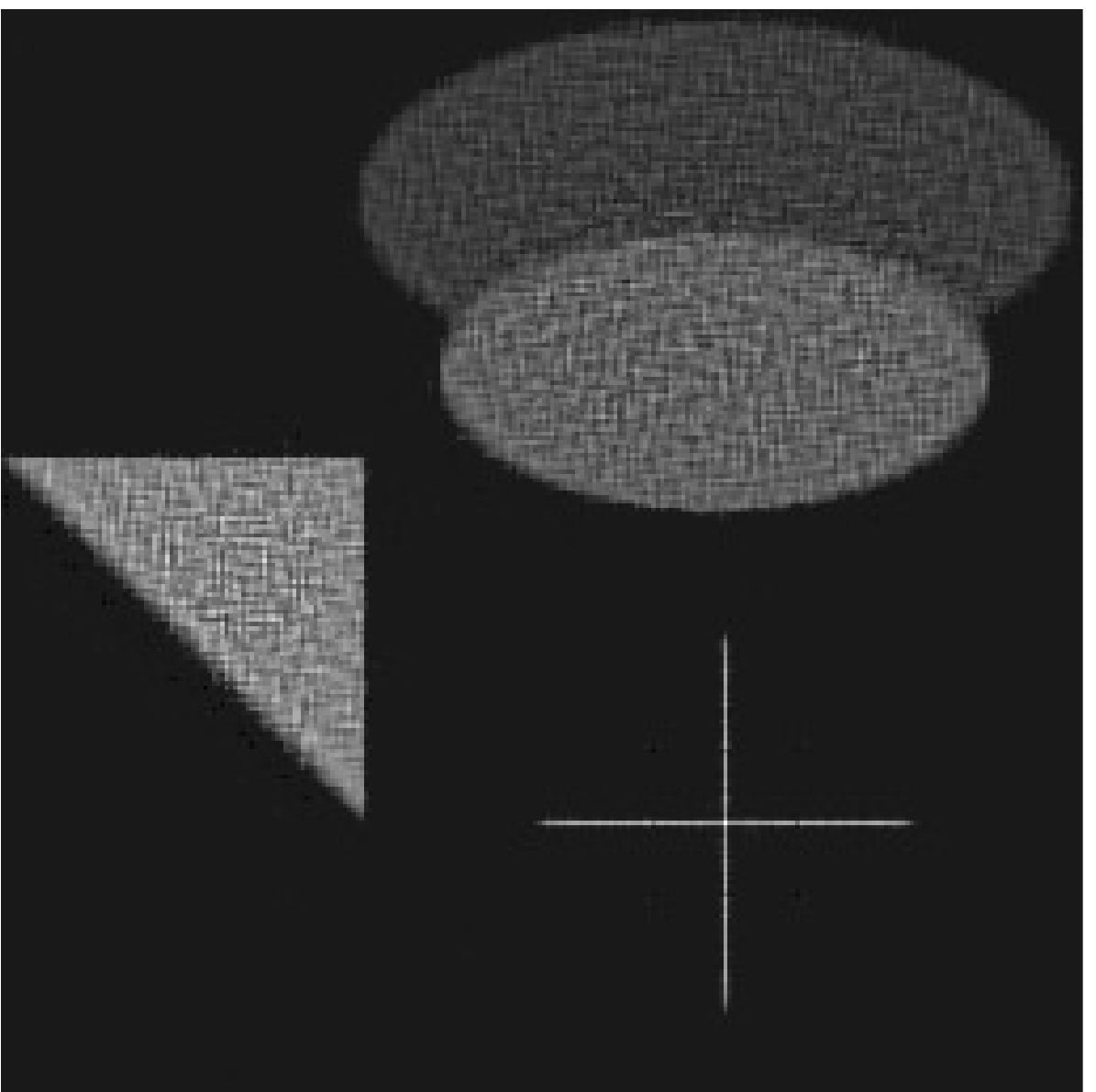,width=1\textwidth}
        \end{subfigure}
        \begin{subfigure}{0.16\textwidth}
        \centering
        \epsfig{file=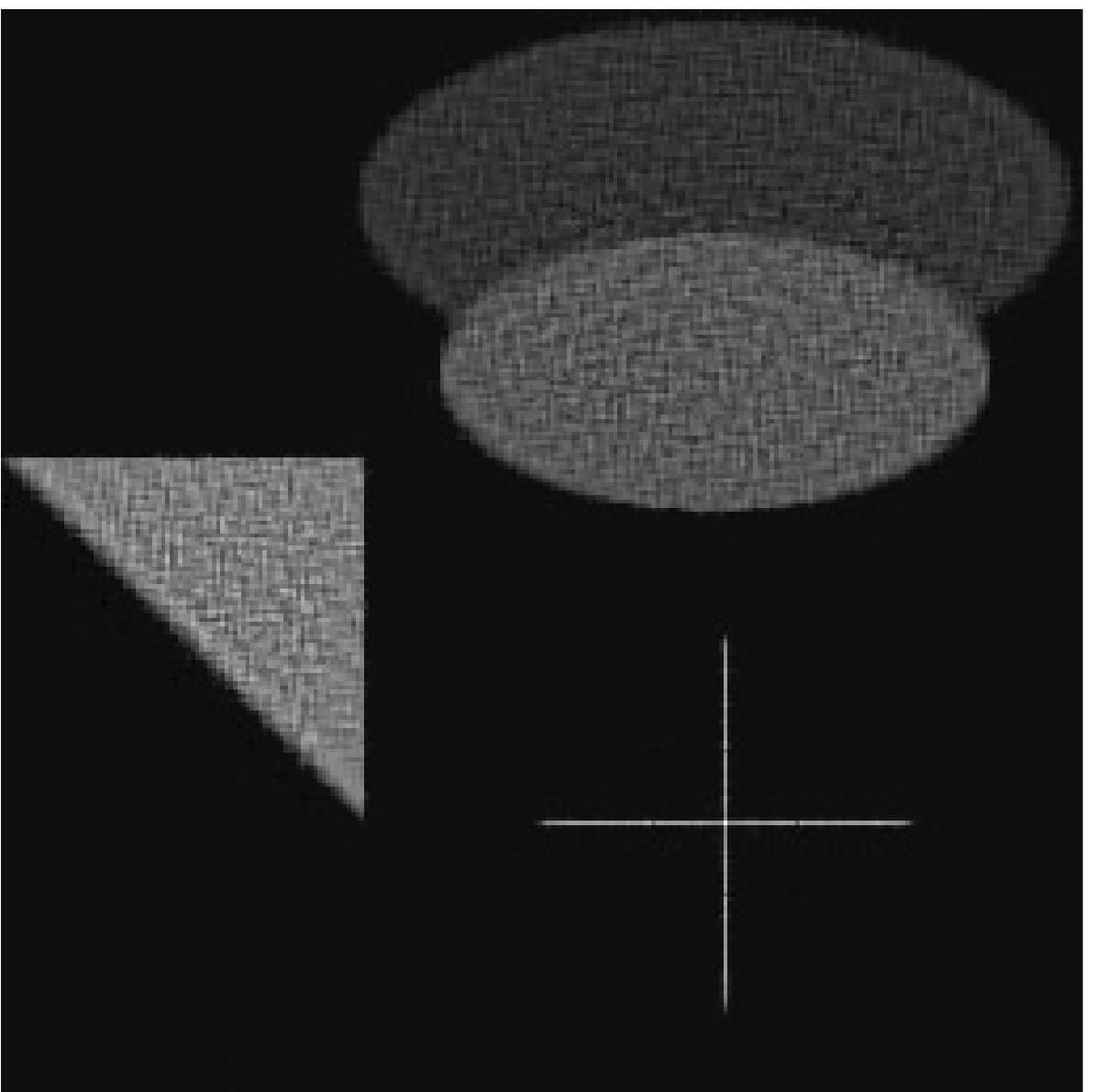,width=1\textwidth}
        \end{subfigure}
        \begin{subfigure}{0.16\textwidth}
        \centering
        \epsfig{file=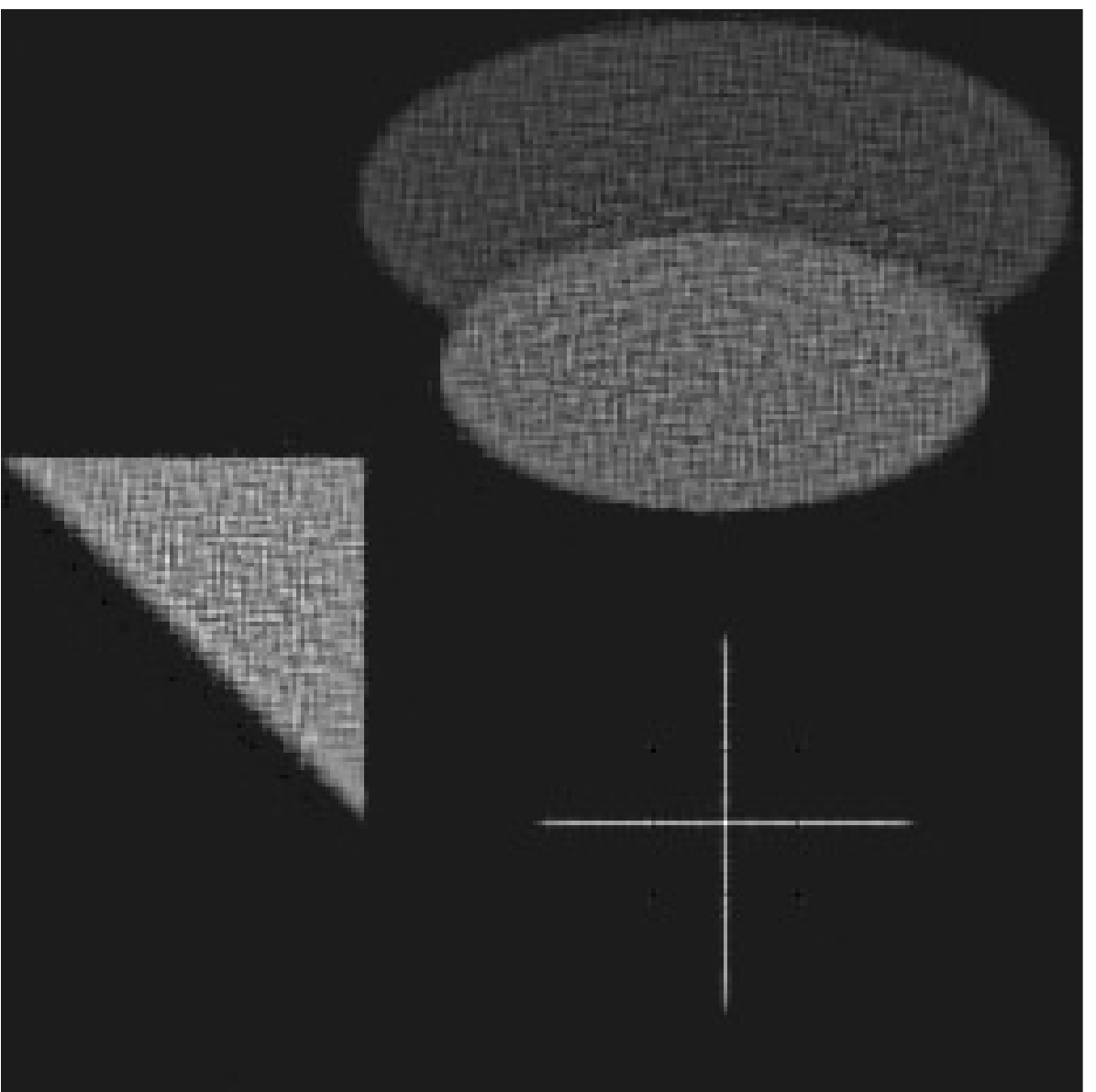,width=1\textwidth}
        \end{subfigure}
   \\$T = 100$ &
        \begin{subfigure}{0.16\textwidth}
        \centering
        \epsfig{file=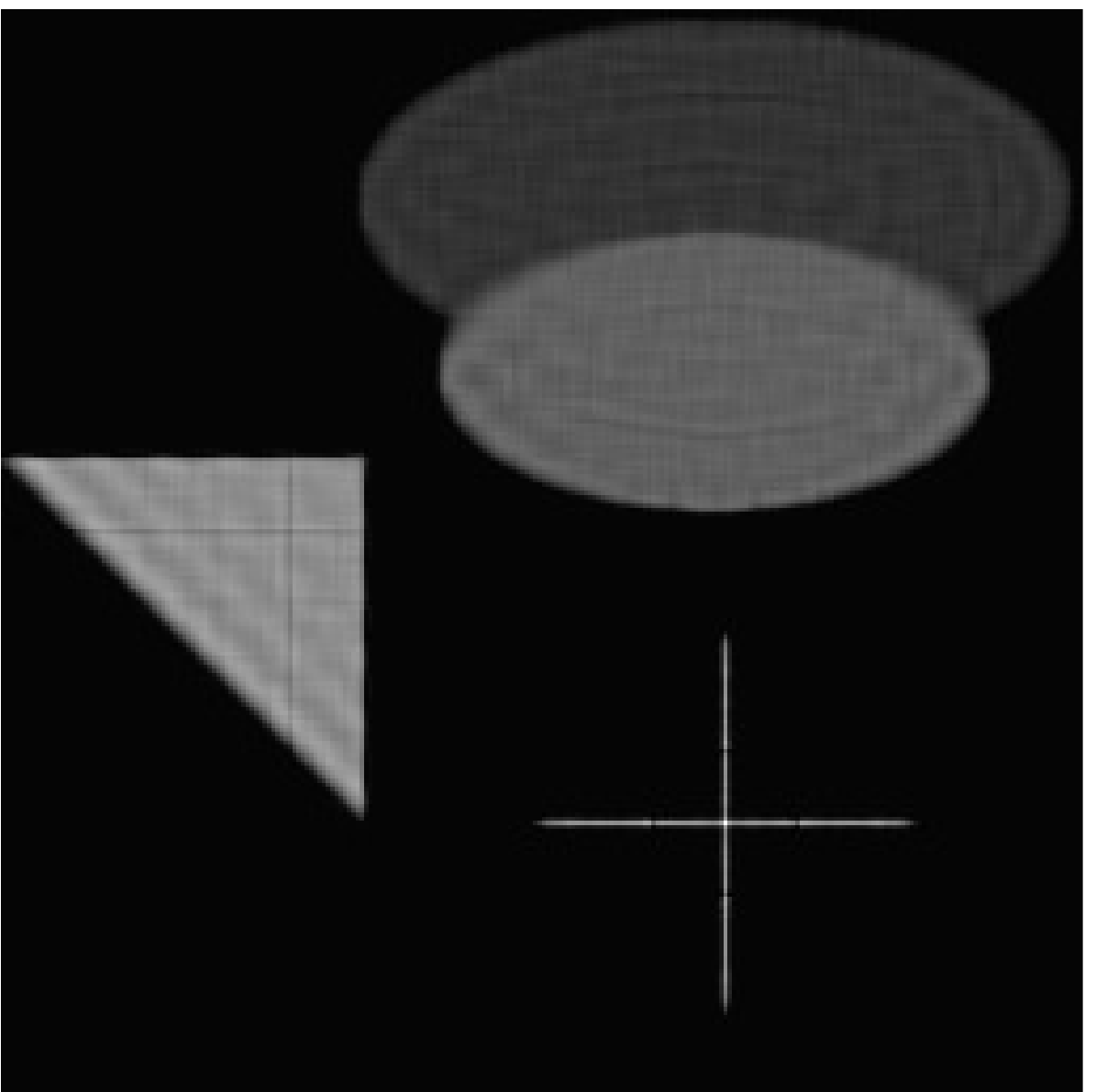,width=1\textwidth}
        \end{subfigure}
        \begin{subfigure}{0.16\textwidth}
        \centering
        \epsfig{file=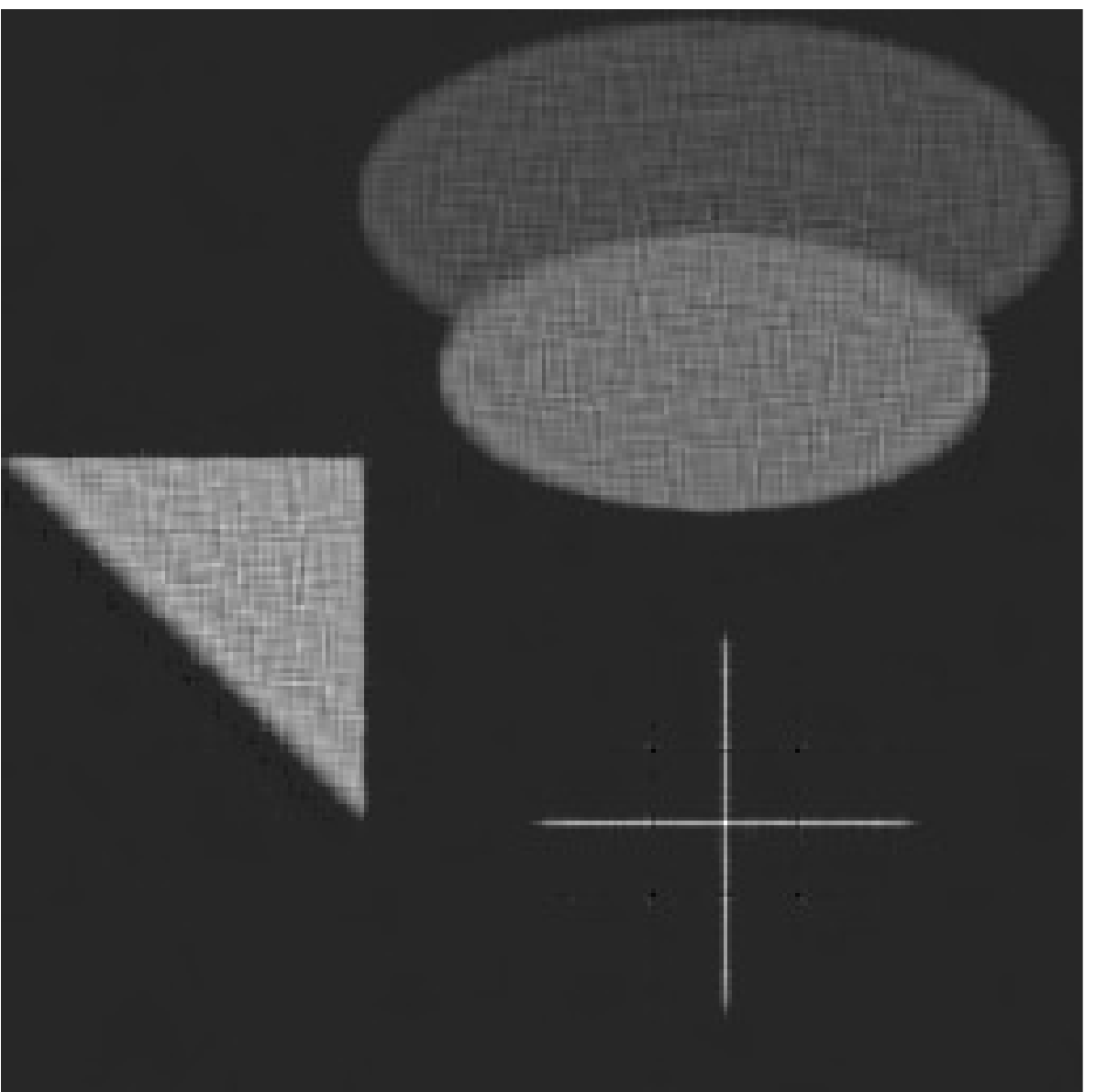,width=1\textwidth}
        \end{subfigure}
        \begin{subfigure}{0.16\textwidth}
        \centering
        \epsfig{file=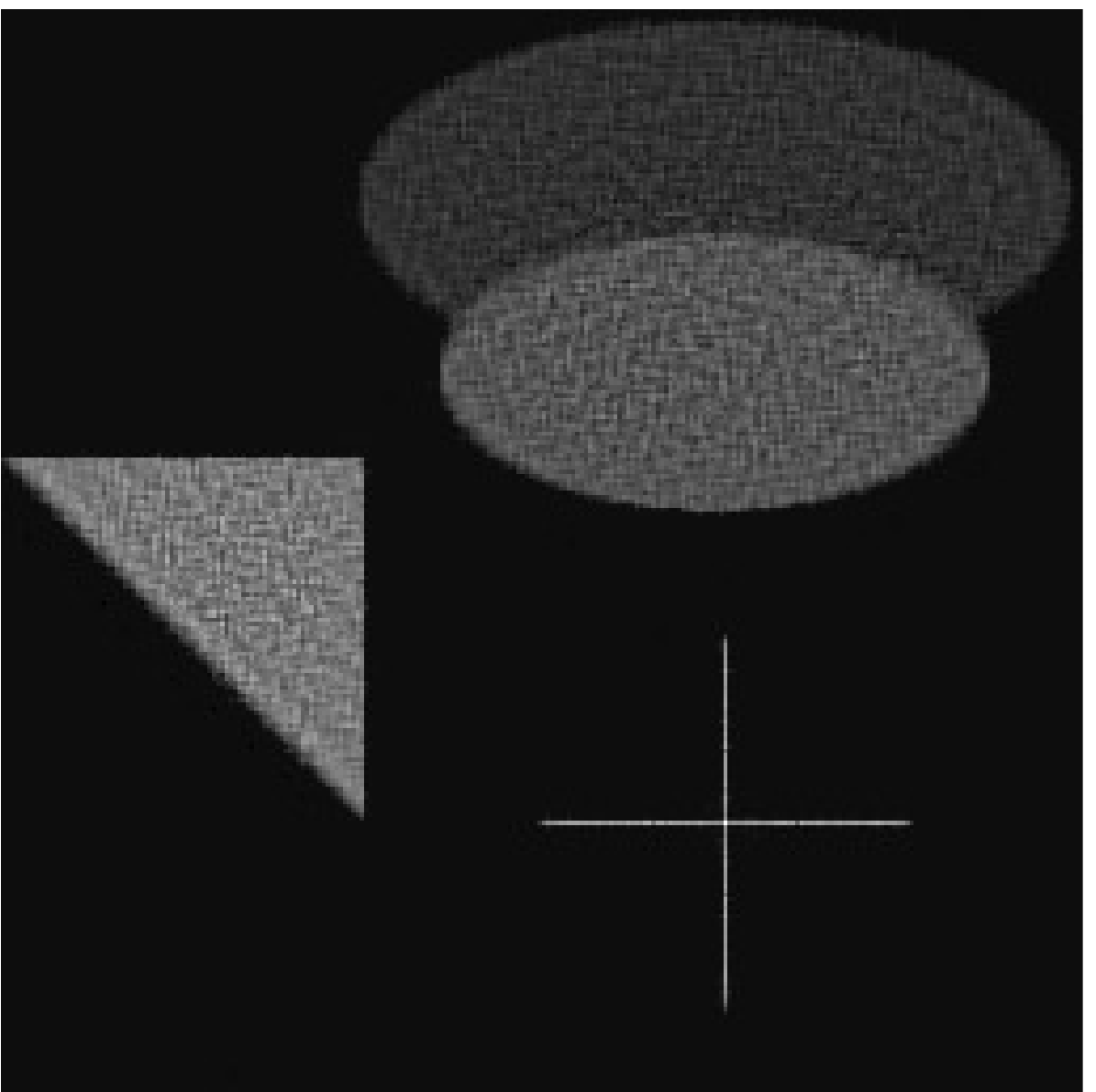,width=1\textwidth}
        \end{subfigure}
        \begin{subfigure}{0.16\textwidth}
        \centering
        \epsfig{file=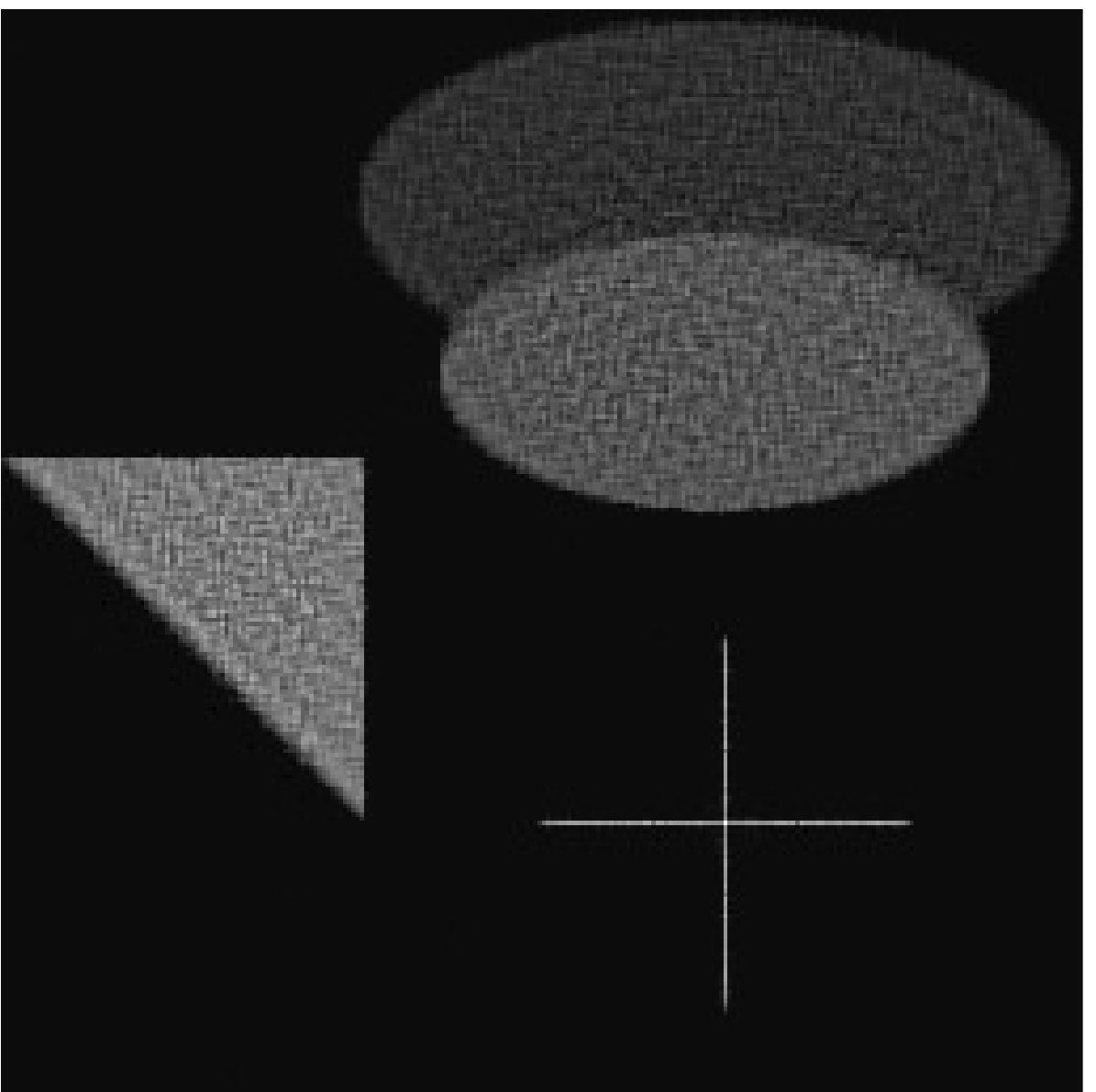,width=1\textwidth}
        \end{subfigure}
        \begin{subfigure}{0.16\textwidth}
        \centering
        \epsfig{file=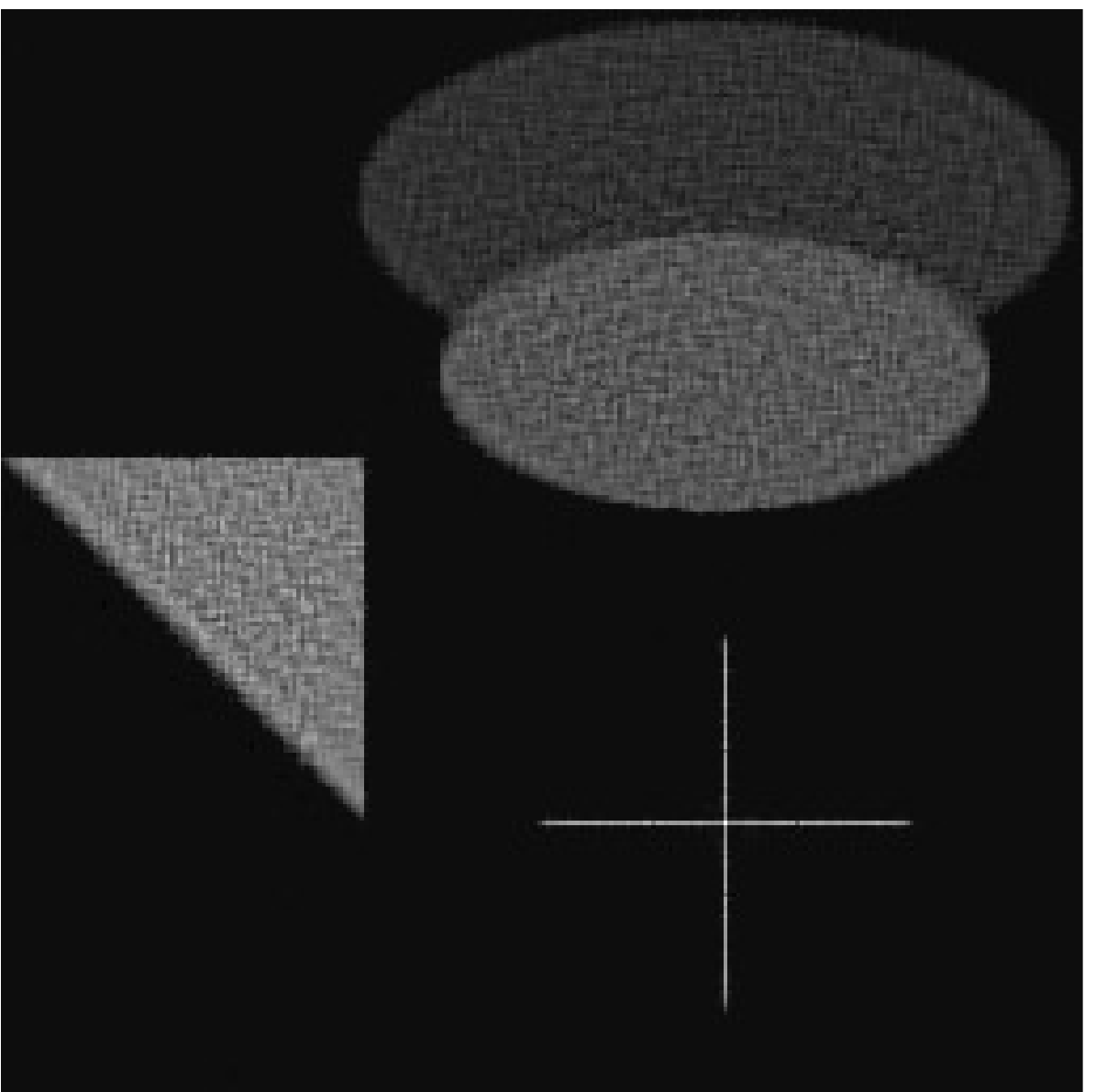,width=1\textwidth}
        \end{subfigure}
  \end{tabular}
  \caption{\footnotesize Reconstruction of the  test image from Figure~\ref{fig:ms2gd:afcewvcawfecvwa} via FISTA, SGD+, S2GD and mS2GD after $T=20, 60, 100$ epochs (one epoch corresponds to work equivalent to the computation of one gradient.)}
\label{fig:ms2gd:adsfaewfawefa}
\end{figure*}

\section{Technical Results}

We first state technical results used in this chapter, followed by proofs deferred from the main body.

\begin{lemma}[Lemma 3.6 in \cite{proxSVRG}]
\label{lem:ms2gd:nonexpansiveness}
Let $R$ be a closed convex function on $\R^d$ and $x,y \in \dom(R)$, then
$
\| \prox_R(x) - \prox_R(y)\| \leq \|x-y\|.
$
\end{lemma}
Note that non-expansiveness of the proximal operator is a standard result in optimization literature \cite{moreau1962,rockafellar1970}.

\begin{lemma}
\label{lem:ms2gd:randvar}
Let $\{\xi_i\}_{i=1}^n$ be vectors in $\R^d$
and $\bar{\xi} \eqdef \frac1n \sum_{i=1}^n \xi_i \in \R^d$.
Let $\hat S$ be a random subset of $ [n]$ of size $\tau$, chosen uniformly at random from all subsets of this cardinality. Taking expectation with respect to $\hat{S}$, we have
\begin{equation}
\label{eq:ms2gd:varianceBound}
\E{ \left\|\frac1\tau \sum_{i\in \hat S} \xi_i - \bar{\xi}  \right\|^2 }
\leq
\frac1{n\tau}
\frac{  n-\tau}{ (n-1)}
\sum_{i=1}^n \left\| \xi_i\right\|^2.
\end{equation}
\end{lemma}

Following from the proof of Corollary 3.5  in \cite{proxSVRG}, by applying Lemma \ref{lem:ms2gd:randvar} with $\xi_i := \nabla f _{i}(y^{k, t}) - \nabla f _{i}(w^k)$, we have the bound for variance as follows.

\begin{theorem}[Bounding Variance]
\label{thm:ms2gd:boundvariance}
Let $\alpha(b)\eqdef \tfrac{n-b}{b (n-1)}$. Considering the definition of $G^{k, t}$ in Algorithm~\ref{alg:ms2gd:mS2GD}, conditioned on $y^{k,t}$, we have $\E{G^{k, t}}=\nabla f (y^{k, t})$ and the variance satisfies,
\begin{align}
&\E{ \|G^{k, t} - \nabla f(y^{k, t})\|^2 }
\nonumber
\\ &\leq   
4L \alpha(b)    [P(y^{k, t})-P(w^*) + P(w^k)-P(w^*)].
\label{eq:ms2gd:fa<wvgfawvgfwagva}
\end{align}
\end{theorem}

\subsection{Proofs} 

\subsubsection{Proof of Lemma \ref{lem:ms2gd:randvar}}
As in the statement of the lemma, by $\E{\cdot}$ we denote expectation with respect to the random set $\hat S$. First, note that
\begin{align*}
 \eta \eqdef & \E{ \left\|\frac1\tau \sum_{i\in \hat S} \xi_i - \bar{\xi} \right\|^2 }
= \E{
  \frac1{\tau^2} \left\|\sum_{i\in \hat S} \xi_i \right\|^2 }
  - \|\bar{\xi} \|^2 \\
&= \frac1{\tau^2}
\E{ \sum_{i\in \hat S}\sum_{j\in \hat S} \xi_i ^T \xi_j }
  - \|\bar{\xi} \|^2.
\end{align*}
If we let $C\eqdef \left\| \bar{\xi} \right\|^2=  \frac1{n^2} \left( \sum_{i,j} \xi_i^T\xi_j \right)$, we can thus write
\begin{align*}
\eta &=
\tfrac1{\tau^2}
\left( 
  \frac{\tau (\tau-1)}{n(n-1)}  \sum_{i \neq j} \xi_i ^T \xi_j 
  + \frac{\tau}{n} \sum_{i=1}^n \xi_i^T\xi_i
  \right)
  - C
\\
=& 
\tfrac1{\tau^2}
\left( 
  \frac{\tau (\tau-1)}{n(n-1)}  \sum_{i, j} \xi_i ^T \xi_j 
  + \left(\frac{\tau}{n}-\frac{\tau (\tau-1)}{n(n-1)}\right) \sum_{i=1}^n \xi_i^T\xi_i
  \right)  - C
\\
=&
\frac1{n\tau}
\left[
-
\left(
-  \frac{(\tau-1)}{(n-1)} 
+\frac{\tau}{n}
\right)
\sum_{i, j} \xi_i ^T \xi_j 
  +  
  \frac{  n-\tau }{n-1}    \sum_{i=1}^n \xi_i^T\xi_i
\right]
\\
=&
\frac1{n\tau}
\frac{  n-\tau}{ (n-1)}
\left[
      \sum_{i=1}^n \xi_i^T\xi_i
      -
\frac1n
\sum_{i, j} \xi_i ^T \xi_j 
\right]
\leq
\frac1{n\tau}
\frac{  n-\tau}{ (n-1)}
\sum_{i=1}^n \left\| \xi_i\right\|^2,
\end{align*}
where in the last step we have used the bound
$
\frac1n
\sum_{i, j} \xi_i ^T \xi_j
= n
\left\|
\sum_{i=1}^n \frac1n\xi_i
\right\|^2
\geq 0.$

\subsubsection*{Proof of Theorem \ref{thm:ms2gd:s2convergence}}

The proof is following the core steps in \cite{proxSVRG}.
For convenience, let us define the stochastic gradient mapping
\begin{equation*}
d^{k, t} = \frac1 h(y^{k, t}- y^{k, t+1}) = \frac1 h(y^{k, t} - \prox_{ h R}(y^{k, t}- h G^{k, t})),
\end{equation*}
then the iterate update can be written as $y^{k, t+1} = y^{k, t} -  h d^{k, t}.$ Let us estimate the change of $\|y^{k, t+1} - w^* \|$. It holds that
\begin{align*}
\| y^{k, t+1} - w^* \|^2 &= \| y^{k, t} -  h d^{k, t} - w^* \|^2
\\
&= \| y^{k, t} - w^* \|^2 - 2 h \ip{d^{k, t}}{y^{k, t-1} - w^*} +  h^2 \| d^{k,t} \|^2.
\numberthis
\label{eq:ms2gd:safwavgfwafa}
\end{align*}

Applying Lemma 3.7 in \cite{proxSVRG} (this is why we need to assume that $h\leq 1/L$) with $x = y^{k, t}$, $v=G^{k, t}$, $x^+ = y^{k, t+1}$, $g=d^{k, t}$, $y=x^*$ and $\Delta = \Delta^{k, t} = G^{k, t} - \nabla f (y^{k, t})$, we get
\begin{align*}
- \ip{d^{k, t}}{y^{k, t}-w^*} + \frac h2\| d^{k, t} \|^2 
&\leq  P(w^*) - P(y^{k, t+1}) - \ip{\Delta^{k, t}}{y^{k, t+1} - w^*}
\\
&\qquad-\frac{\mu_F}2 \| y^{k, t} - w^* \|^2 - \frac{\mu_R}2 \| y^{k, t+1} - w^* \|^2,
\numberthis
\label{eq:ms2gd:vaawfwavfwafdewafca}
\end{align*}
and therefore,
\begin{align*}
\| y^{k, t+1} - w^* \|^2 &\overset{\eqref{eq:ms2gd:safwavgfwafa},\eqref{eq:ms2gd:vaawfwavfwafdewafca}}{\leq}  
2 h
\left( 
P(w^*) - P(y^{k, t+1}) \right.
\left. - \ip{\Delta^{k, t}}{y^{k, t+1} - w^*}
\right) + \| y^{k, t} - w^* \|^2
\\
&\hspace{18pt}= \| y^{k, t} - w^* \|^2 
- 2 h \ip{\Delta^{k, t}}{y^{k, t+1}-w^*} 
- 2 h[P(y^{k, t+1}) - P(w^*)].
\numberthis
\label{eq:ms2gd:safawvfdwavfda}
\end{align*}

In order to bound $-\ip{\Delta^{k, t}}{y^{k, t+1} - w^*}$, let us define the proximal full gradient update as\footnote{Note that this quantity is never computed during the algorithm. We can use it in the analysis nevertheless.}
$
\bar{y}^{k, t+1} = \prox_{ h R}(y^{k, t} -  h\nabla f (y^{k, t})).
$
We get
\begin{align*}
- \ip{\Delta^{k, t}}{y^{k, t+1} - w^*} 
&=
- \ip{\Delta^{k, t}}{y^{k, t+1} - \bar{y}^{k, t+1}}
- \ip{\Delta^{k, t}}{\bar{y}^{k, t+1} - w^*}
\nonumber
\\
&= 
- \ip{\Delta^{k, t}}{\bar{y}^{k, t+1}-w^*}
\\
&\quad- \ip{ \Delta^{k, t}}
{\prox_{ h R}(y^{k, t} -  h G^{k, t})
-\prox_{ h R}(y^{k, t-1} -  h\nabla f (y^{k, t-1}))}
\nonumber
\end{align*}
Using Cauchy-Schwarz and Lemma~\ref{lem:ms2gd:nonexpansiveness}, we  conclude that
\begin{align*}
-\ip{\Delta^{k, t}}{y^{k, t+1} - w^*}
&\leq 
 \|\Delta^{k, t}\| \| (y^{k, t} -  h G^{k, t}) - (y^{k, t}- h\nabla f (y^{k, t}))\|
\\
&\qquad - \ip{\Delta^{k, t}}{\bar{y}^{k, t+1} - w^*}
\nonumber
\\
&=  h \| \Delta^{k, t}\|^2 - \ip{\Delta^{k, t}}{\bar{y}^{k, t+1}-w^*}.
\numberthis
\label{eq:ms2gd:asfaevfwavfa}
\end{align*}

Further, we obtain 
\begin{align*}
\| y^{k, t+1} - w^* \|^2 
&\overset{\eqref{eq:ms2gd:asfaevfwavfa},\eqref{eq:ms2gd:safawvfdwavfda}}{\leq}
  \left\|
y^{k, t} - w^*
\right\|^2 
\\
&\qquad\qquad + 2 h 
\left(    h \| \Delta^{k, t}\|^2 - \ip{\Delta^{k, t}}{\bar{y}^{k, t+1} - w^*} 
- [P(y^{k, t+1}) - P(w^*)]\right). 
\end{align*}

By taking expectation, conditioned on $y^{k, t}$\footnote{For simplicity, we omit the $\E{ \cdot \,|\, y^{k, t}}$ notation in further analysis} we obtain

\begin{align*}
\E{ \|y^{k, t+1} - w^* \|^2 }
  &\leq
  \left\|
y^{k, t} - w^*
\right\|^2 
+ 2  h 
\left(  h \E{ \|\Delta^{k, t}\|^2 } - \E{ P(y^{k, t+1}) - P(w^*) } \right),
\numberthis
\label{eq:ms2gd:asfavgfwavgfwavgfa}
\end{align*}
where we have used that
$\E{ \Delta^{k, t}} = \E{ G^{k, t} } - \nabla f (y^{k, t})=0$,
thus $\E{ -\ip{\Delta^{k, t}}{\bar{y}^{k, t+1} - w^*} } = 0$\footnote{$\bar{y}^{k, t+1}$ is constant, conditioned on $y^{k, t}$}.
Now, if we substitute \eqref{eq:ms2gd:fa<wvgfawvgfwagva}
into \eqref{eq:ms2gd:asfavgfwavgfwavgfa} and decrease index $t$ by 1, we obtain
\begin{align*}
\E{ \|y^{k, t} - w^*\|^2 }
&\leq
 \left\|
y^{k, t-1} - w^* \right\|^2 
- 2h\E{ P(y^{k, t}) - P(w^*) }
\\
&\qquad+ \theta [P(y^{k, t-1}) - P(w^*)
+ P(w^k) - P(w^*)], 
\numberthis
\label{eq:ms2gd:vgafrvgwaevgfa}
\end{align*}
where 
$\theta\eqdef 8 L h^2 \alpha(b)$
and $\alpha(b)=\frac{n-b}{b (n-1)}$.
Note that 
\eqref{eq:ms2gd:vgafrvgwaevgfa}
is equivalent to
\begin{align*}
\E{ \|y^{k, t} - w^* \|^2 }
+  2 h 
(  \E{ P(y^{k, t}) - P(w^*) })
&\leq
  \left\|
y^{k, t-1} - w^*
\right\|^2 
\\
&\qquad+ \theta 
     \left(P(y^{k, t-1}) - P(w^*) 
     + P(w^k) - P(w^*)\right). 
\numberthis
\label{eq:ms2gd:aswdevfwavfaawca}
\end{align*}

Now, by the definition of 
$w^k$ in Algorithm~\ref{alg:ms2gd:mS2GD} we have that
\begin{align*}
\E{ P(w^{k+1}) }
&=
\frac1{m}
\sum_{t=1}^m 
   \E{ P(y^{k, t}) }.
\numberthis
\label{eqn:ms2gd:s2exp}
\end{align*}
By summing \eqref{eq:ms2gd:aswdevfwavfaawca}
for $1\leq t \leq m$, we  get on the left hand side
\begin{align*}
LHS =& \sum_{t=1}^m 
   \E{ \|y^{k, t} - w^*\|^2 }
+      2 h 
 \E{ P(y^{k, t}) - P(w^*) }
\numberthis
\label{eqn:ms2gd:s2LHS}
\end{align*}
and for the right hand side we have:
\begin{align*}
RHS &= \sum_{t=1}^m 
\left\{  \E{ \| y^{k, t-1} - w^*\|^2 } + 
\theta \E{ P(y^{k, t-1}) - P(w^*) + P(w^k) - P(w^*) } \right\}
\\
&\leq \sum_{t=0}^{m-1} 
 \E{ \|y^{k, t}-w^*\|^2 }
+ \theta
   \sum_{t=0}^{m} \E{ P(y^{k, t}) - P(w^*) }
+ \theta
\E{ P(w^k) - P(w^*) }m.
\numberthis
\label{eqn:ms2gd:s2RHS}
\end{align*}
Combining \eqref{eqn:ms2gd:s2LHS} and \eqref{eqn:ms2gd:s2RHS}
and using the fact that $LHS\leq RHS$, we have
\begin{align*}
\E{ \| y^{k, m} - w^*\|^2 } + 
2 h
\sum_{t=1}^m  \E{ P(y^{k, t}) - P(w^*) }
&\leq
 \E{ \| y^{k, 0} - w^*\|^2 } +
  \theta
 \E{ P(w^k) - P(w^*) }m
\\ 
&\qquad+  \theta
 \sum_{t=0}^{m} 
   \E{ P(y^{k, t}) - P(w^*) }.
\end{align*}  
Now, using \eqref{eqn:ms2gd:s2exp}, we obtain
\begin{align*}
\E{ \|y^{k, m} - w^* \|^2 }
+ 
2 h m
  \E{ P(w^{k+1}) - P(w^*) }
&\leq
 \E{ \| y^{k, 0} - w^*\|^2 }
 +
  \theta m \E{ P(w^k) - P(w^*) }
\\
&\qquad+ 
\theta m \E{ P(w^{k+1}) - P(w^*) }
\\
&\qquad+
 \theta
   \E{ P(y^{k, 0}) - P(w^*) }.
\numberthis
\label{eq:ms2gd:vfrwavfwafaewfcwa}
\end{align*}  
  
Strong convexity~\eqref{eq:ms2gd:strongconv} and optimality of $w^*$ imply that $0\in \partial P(w^*)$, and hence for all $w \in \R^d$ we have
\begin{equation*}
\| w - w^* \|^2\leq\frac2\mu [P(w) - P(w^*)]
\numberthis
\label{eqn:ms2gd:strongconvex}.
\end{equation*}
Since $\E{ \| y^{k, m} - w^*\|^2 } \geq 0$ and $y^{k, 0} = w^k$, by combining \eqref{eqn:ms2gd:strongconvex}
and \eqref{eq:ms2gd:vfrwavfwafaewfcwa} we get
\begin{align*}
m(2h -\theta)
 \E{ P(w^{k+1}) - P(w^*) }
\leq
( P(w^k)-P(w^*))
\left( \frac2{\mu} 
  +  \theta
 \left(m+1 \right) \right).
\end{align*} 
Notice that in view of our assumption on $h$ and definition of $\theta$, we have $
2h>\theta$, and hence
\begin{equation*}
\E{ P(w^{k+1})-P(w^*) } \leq c[P(w^k) - P(w^*)],
\end{equation*}
where 
$
c = 
\frac{    2
  }
  {
 m \mu (2h-\theta)
  }
+
\frac{  \theta(m+1)   } { m(2h-\theta)}.
$
Applying the above linear convergence relation recursively with chained expectations, we finally obtain
$
\E{ P(w^k) - P(w^*) } \leq c^k[P(w^0) - P(w^*)].
$

\subsubsection*{Proof of Theorem \ref{thm:ms2gd:optimalM}}

Clearly, if we choose some value of $ h$ then the value of $m$ will be determined from \eqref{eq:ms2gd:s2rho} (i.e. we need to choose $m$ such that we will get desired rate).
Therefore, $m$ as a function of $ h$ obtained from \eqref{eq:ms2gd:s2rho} is
\begin{equation}
\label{eq:ms2gd:fsarewavgfwavgfw}
m( h) = \frac{1 + 4 \alpha(b)  h^2 L \mu}
            { h \mu (c - 4 \alpha(b)  h L
   ( c + 1 )  )}.
\end{equation}
Now, we can observe that the nominator is always positive
and the denominator is positive only if
$
c   > 4 \alpha(b)  h L
   ( c + 1)
$, which implies
   $\tfrac{1}{4 \alpha(b) L
   } 
   \cdot  \tfrac{c}{ c + 1} >   h 
$ (note that $\tfrac{c}{ c + 1}\in [0,\frac12]$).
Observe that this condition is stronger than the one in the assumption of Theorem \ref{thm:ms2gd:s2convergence}.
It is easy to verify that
$$\lim_{ h \searrow 0} m( h) = +\infty,
\qquad
\lim_{ h \nearrow \frac{1}{4 \alpha(b) L
   } 
   \cdot   \frac{c}{ c + 1}   } m( h) = +\infty.
$$
Also note that $m( h)$ is differentiable (and continuous) at any 
$ h \in \left( 0,\frac{1}{4 \alpha(b) L
   } 
   \cdot   \frac{c}{ c + 1} \right) =: I_ h$.
The derivative of $m$ is given by
$
m'( h) = 
\tfrac{-c + 4 \alpha(b)  h L (2 + (2 +  h \mu) c)}
     { h^2 \mu (c - 4 \alpha(b)  h L (1 + c))^2}.
$ Observe that $m'( h)$ is defined and continuous for any $ h \in I_ h$.
Therefore there have to be some stationary points (and in case that there is just on $I_ h$) it will be the global minimum on $I_ h$.
The first order condition gives
\begin{align*}
\tilde  h_b
&=
 \frac{-2 \alpha(b) L (1 + c) + \sqrt{
 \alpha(b) L (\mu c^2 + 4 \alpha(b) L (1 + c)^2)}}
 {2 \alpha(b) L \mu c}
 \\
&=
 \sqrt{ 
 \frac{ 
   1  }
 {4 \alpha(b) L  \mu   } 
+
 \frac{ 
    (1 + c)^2 }
 {  \mu^2 c^2} 
 }
 -
  \frac{    1 + c   }
 {  \mu c}.
\numberthis
\label{eq:ms2gd:avfawefwafewafwafe} 
\end{align*}
If this $\tilde h_b \in I_ h$
and also $\tilde h_b \leq \frac1{L}$
then this is the optimal choice and plugging 
\eqref{eq:ms2gd:avfawefwafewafwafe}
into \eqref{eq:ms2gd:fsarewavgfwavgfw}
gives us \eqref{eq:ms2gd:vfewdfwaefvawvgfeefewafa}.

\paragraph{Claim \#1} It always holds that  $\tilde h_b \in I_ h$.
We just need to verify that
$$
 \sqrt{ 
 \frac{ 
   1  }
 {4 \alpha(b) L  \mu   } 
+
 \frac{ 
    (1 + c)^2 }
 {  \mu^2 c^2} 
 }
 -
  \frac{    1 + c   }
 {  \mu c}
 < \frac{1}{4 \alpha(b) L
   } 
   \cdot   \frac{c}{ c + 1}, 
$$
which is equivalent to 
$
\mu c^2 + 4 \alpha(b) L (1 + c)^2
 > 
 2 (1 + c) \sqrt{\alpha(b) L (\mu c^2 + 4 \alpha(b) L (1 + c)^2)}. $
Because both sides are positive, we can square them to obtain the equivalent condition
$$
\mu c^2 (\mu c^2 + 4 \alpha(b) L (1 + c)^2) > 0.
$$
\paragraph{Claim \#2} If $\tilde  h_b > \frac1{L}$ then 
$ h_b^* = \frac1L$.
The only detail which needs to be verified is that the denominator
of \eqref{eq:ms2gd:fasfawefwafewa} is positive
(or equivalently we want to show that
$ c > 4 \alpha(b)   (1+c)$.
To see that, we need to realize that in that case
we have $\tfrac1L \leq \tilde  h_b \leq \frac{1}{4 \alpha(b) L
   } 
   \cdot   \frac{c}{ c + 1}$,
which implies that
$4 \alpha(b)   (1+c) < c.
$

\subsubsection*{Proof of Corollary \ref{thm:ms2gd:minibatch}}
\label{thm:ms2gd:proof3}

By substituting definition of $\tilde  h_b$ in Theorem~\ref{thm:ms2gd:optimalM}, we get
\begin{equation}
\label{eqn:ms2gd:b0}
\tilde  h_b < \frac1L
\quad \Longleftrightarrow \quad 
b < b_0 \eqdef \frac{8c n\kappa + 8n\kappa + 4c n}
{c n\kappa + (7c+8)\kappa + 4c},
\end{equation}
where $\kappa = L/\mu$. Hence, it follows that if $b < \lceil b_0 \rceil$, then $h_b = \tilde  h_b$ and $m_b$ is defined in~\eqref{eq:ms2gd:vfewdfwaefvawvgfeefewafa}; otherwise, $h_b=\frac1L$ and $m_b$ is defined in~\eqref{eq:ms2gd:fasfawefwafewa}. Let $e$ be the base of the natural logarithm. By selecting $b_0 =\frac{8 n\kappa + 8 e n\kappa + 4n}
{n\kappa + (7+8e)\kappa + 4}$, choosing mini-batch size $b<\lceil b_0 \rceil$, and running the inner loop of mS2GD for
\begin{equation}
\label{eqn:ms2gd:maxiter}
m_b = \left\lceil 8 e\alpha(b)\kappa \left(e + 1 + \sqrt{
 \frac1{4 \alpha(b) \kappa} +   (1 + e)^2}\right) \right\rceil
 \end{equation}
 iterations with constant stepsize 
\begin{equation}
\label{eqn:ms2gd:stepsizehb}
h_b = \sqrt{ \left( \frac{1+e}{\mu} \right)^2 + \frac1{4\mu\alpha(b)L}} - \frac{1+e}{\mu},
\end{equation}
we can achieve a convergence rate 
\begin{align*}
c &\overset{\eqref{eq:ms2gd:s2rho}}{=}
\frac{    1
  }
  {
  m_b 
 h_b \mu
(
1
-
  4 h_b  L \alpha(b)  
)
  }
+
\frac{      
  4 h_b  L \alpha(b) 
\left(      
 m_b
+ 
  1
\right)  
  }
  {
  m_b 
(
1
-
  4 h_b   L \alpha(b)  
)
  } \overset{\eqref{eqn:ms2gd:maxiter}, \eqref{eqn:ms2gd:stepsizehb}}{=}
\frac{1}{e}. 
\numberthis
\label{eqn:ms2gd:convergencerate}
\end{align*}
Since
$k = \left\lceil \log(1 / \epsilon) \right\rceil \geq \log(1 / \epsilon)$ if and only if $ 
e^k  \geq 1 / \epsilon$ if and only if
$ 
e^{-k}  \leq  \epsilon,
$
 we can conclude that $c^k \overset{\eqref{eqn:ms2gd:convergencerate}}{=} (e^{-1})^k = e^{-k} \leq \epsilon$. Therefore, running mS2GD for $k$ outer iterations achieves $\epsilon$-accuracy solution defined in \eqref{eq:ms2gd:epsilonaccuracy}. Moreover, since in general $\kappa \gg e, n\gg e$, it can be concluded that
$$ b_0 \overset{\eqref{eqn:ms2gd:b0}}{=}\frac{8(1+e) n\kappa + 4n}
{n\kappa + (7+8e)\kappa + 4} \approx 8 \left( e+1 \right) \approx 29.75,$$
then with the definition $\alpha(b) = \frac{(n-b)}{b(n-1)}$, we derive
 \begin{align*}
 & b m_b 
 \overset{\eqref{eqn:ms2gd:maxiter}}{=}
 \left\lceil 8 e\kappa\frac{(n-b)}{(n-1)} \left(e + 1 + \sqrt{
 \frac1{4 \alpha(b) \kappa} +   (1 + e)^2}\right) \right\rceil
\\&
\quad\overset{1\leq b < 30}{\leq} 
 \left\lceil 8 e\kappa \left((e + 1) + \sqrt{
 \frac{b}{4 \kappa} +   (1 + e)^2}\right) \right\rceil = \mathcal{O}(\kappa),
 \end{align*}
so from \eqref{eqn:ms2gd:complexity}, the total complexity can be translated to
$ \mathcal{O} \left( (n + \kappa) \log(1 / \epsilon) \right).$
This result shows that we can reach efficient speedup by mini-batching as long as the mini-batch size is smaller than some threshold $b_0 \approx 29.75$, which finishes the proof for Corollary \ref{thm:ms2gd:minibatch}.

\subsection{Proximal lazy updates for $\ell_1$ and $\ell_2$-regularizers}\label{section:ms2gd:proxl}

\subsubsection*{Proof of Lemma~\ref{lemma:ms2gd:L2regularizer}}

For any $s \in \{1,2,\dots,\tau\}$
we have
$
\tilde y^s
=\prox_{hR}(\tilde y^{s-1} - hg)
= \beta 
  (\tilde y^{s-1} - hg),
$
where  $\beta \eqdef 1/(1+\lambda  h)$.
Therefore,
\[
 \tilde y^{\tau} 
= \beta^\tau \tilde y^{0} 
- h\left(\sum_{j=1}^\tau \beta^j\right) g
=
\beta^{\tau} y  - \frac{ h\beta}{1-\beta}
                \left[1 - \beta^{\tau} \right] g.
\]

\subsubsection*{Proof of Lemma~\ref{lemma:ms2gd:L1regularizer}}

\begin{proof}
For any $s \in \{1,2,\dots,\tau\}$ and $j\in \{1,2,\dots,d\}$, 
\begin{align*}
\tilde y^{s}_j &= \arg\min_{x\in\R} \frac12(x - \tilde y^{s-1}_j + hg_j)^2+ \lambda h |x| 
\\
&= 
\begin{cases}
\tilde y^{s-1}_j - ( \lambda + g_j) h, &\text{ if } \tilde y^{s-1}_j  > (\lambda+g_j) h,\\
\tilde y^{s-1}_j + ( \lambda - g_j) h, &\text{ if } \tilde y^{s-1}_j  <-(\lambda-g_j) h,\\
0, &\text{ otherwise, }
\end{cases}
\\
&= 
\begin{cases}
\tilde y^{s-1}_j - M, &\text{ if } \tilde y^{s-1}_j  > M, \\
\tilde y^{s-1}_j - m, &\text{ if } \tilde y^{s-1}_j  < m ,\\
0, &\text{ otherwise.}
\end{cases}
\end{align*}
where $M \eqdef (\lambda + g_j) h$, $m \eqdef -(\lambda - g_j) h$ and $M-m = 2\lambda h > 0$.  Now, we will distinguish several cases based on $g_j$:

\begin{enumerate}
\item[(1)] When $g_j \geq \lambda$, then $M>m = - (\lambda - g_j)h \geq 0$, thus by letting $p = \left\lfloor \frac{y_j}{M}\right\rfloor $, we have that: if $y_j < m$, then
$\tilde y_j^\tau = y_j - \tau m;$ if $m  \leq y_j < M$, then
$ \tilde y^{\tau}_j =  -(\tau-1)m;$ and if $y_j \geq M$, then
\begin{align*}
\tilde y^{\tau}_j &= 
\begin{cases}
y_j - \tau M, &\text{ if }\tau \leq p,\\
y_j - pM -(\tau - p)m, &\text{ if }\tau > p \text{ \& }y_j - pM < m,\\
-(\tau - p-1)m, &\text{ if }\tau > p \text{ \& } y_j - pM \geq m,
\end{cases}
\\&= 
\begin{cases}
y_j - \tau M, &\text{ if }\tau \leq p,\\
\min\{y_j - pM, m\} -(\tau - p)m, &\text{ if }\tau > p.
\end{cases}
\end{align*}
\item[(2)] When $-\lambda<g_j<\lambda$, then $M = (\lambda + g_j)h> 0,   m = - (\lambda - g_j)h  < 0$, thus we have that 
\begin{align*}
&\tilde y^{\tau}_j =
\begin{cases}
\max\{  y_j - \tau M, 0\},&\text{ if } y_j\geq 0,\\
\min\{  y_j - \tau m, 0\},&\text{ if } y_j< 0.
\end{cases}
\end{align*}
\item[(3)] When $g_j \leq - \lambda$, then $m < M = (\lambda + g_j)h \leq 0$, thus by letting $q = \left\lfloor\frac{y_j}{m}\right\rfloor$, we have that: if $y_j \leq m$, then
 \begin{align*}
 \tilde y^\tau_j &= 
 \begin{cases}
 y_j - \tau m, &\text{ if }\tau \leq q,\\
 y_j - p m -(\tau-q)M, &\text{ if }\tau > q \text{ \& } y_j - qm > M,\\
 -(\tau-q - 1)M, &\text{ if }\tau > q \text{ \& }y_j - qm \leq M,
 \end{cases}
 \\&= 
  \begin{cases}
 y_j - \tau m, &\text{ if }\tau \leq q,\\
 \max\{y_j - q m, M\} -(\tau-q)M, &\text{ if }\tau > q;
 \end{cases}
 \end{align*}
 if $m < y_j \leq M$, then
$\tilde y^\tau_j = -(\tau-1)M;$ if $y_j >  M$, then
$\tilde y^\tau_j = y_j - \tau M.$
\end{enumerate}

Now, we will perform a few simplifications:
Case (1). When $y_j < M$, we can conclude that
$
\tilde y^\tau_j= \min\{y_j, m\} - \tau m.
$
Moreover, since the following equivalences hold if $g_j \geq \lambda$: $y_j \geq M \ \Leftrightarrow \ \tfrac{y_j}{M} \geq 1
\ \Leftrightarrow \ p\geq 1$, and
$y_j < M\ \Leftrightarrow \ \tfrac{y_j}{M} < 1
\ \Leftrightarrow \ p\leq 0$,
the situation  simplifies to
\begin{align*}
\tilde y^\tau_j&\quad =
\begin{cases}
y_j - \tau M, &\text{if }p\geq\tau,\\
\min\{y_j - pM, m\} -(\tau - p)m, &\text{if }1 \leq p<\tau,\\
\min\{y_j, m\} - \tau m, &\text{if }p\leq 0,
\end{cases}
\\&\quad =
\begin{cases}
y_j - \tau M, &\text{if }p\geq\tau,\\
\min\{y_j - [p]_+ M, m\} -(\tau - [p]_+)m, &\text{if }p<\tau,
\end{cases}
\end{align*}
where $[\cdot]_+ \eqdef \max\{\cdot, 0\}$. For Case (3), when $y_j > m$, we can conclude that
$
\tilde y^\tau_j = \max\{y_j, M\} - \tau M,
$
and in addition, the following equivalences hold when $g_j \leq -\lambda$:
\begin{align*}
y_j \leq m \ \Leftrightarrow \  \frac{y_j}{m} \geq 1
\  \Leftrightarrow \ q\geq 1,
\\
y_j > m\  \Leftrightarrow \  \frac{y_j}{m} < 1
\ \Leftrightarrow \ q\leq 0,
\end{align*}
which  summarizes the situation as follows:
\begin{align*}
\tilde y^\tau_j &\quad=
\begin{cases}
y_j - \tau m, &\text{if }q\geq\tau,\\
 \max\{y_j - q m, M\} -(\tau-q)M, &\text{if }1 \leq q<\tau,\\
 \max\{y_j, M\} - \tau M, &\text{if }q\leq 0,
\end{cases}
\\&\quad=
\begin{cases}
y_j - \tau m, &\text{if }q\geq\tau,\\
 \max\{y_j - [q]_+ m, M\} -(\tau-[q]_+)M, &\text{if }q<\tau.
\end{cases}\qedhere
\end{align*}

\end{proof}

\section{Conclusion}

We have proposed mS2GD---a mini-batch semi-stochastic gradient method---for minimizing a strongly convex composite function. Such optimization problems arise frequently in inverse problems in signal processing and statistics. Similarly to SAG, SVRG, SDCA and S2GD, our algorithm also outperforms existing deterministic method such as ISTA and FISTA. Moreover, we have shown that the method is by design amenable to a simple parallel implementation. Comparisons to state-of-the-art algorithms suggest that mS2GD, with a small-enough mini-batch size, is competitive in theory and faster in practice than other competing methods even without parallelism. The method can be efficiently implemented for sparse data sets.

\chapter{Distributed Optimization with Arbitrary Local Solvers}
\label{ch:cocoa}

\section{Motivation}
\label{sec:cocoa:motivation}
Regression and classification techniques, represented in the general class of regularized loss minimization problems \cite{vapnik1998statistical}, are among the most central tools in modern big data analysis, machine learning, and signal processing.
For these tasks, much effort from both industry and academia has gone into the development of highly tuned and customized solvers.
However, with the massive growth of available datasets, major roadblocks still persist in the distributed setting, where data no longer fits in the memory of a single computer, and computation must be split across multiple machines in a network \cite{bottou2010large,marecek2014distributed,Ma:2015tt,
Balcan:2012tc,richtarik2013distributed,
Duchi:2013te,takac2015distributed,
fercoq2014fast,chen2014large,ma2015linear,
shamir2014distributed,DANE,ALPHA, 
zhang2015communication, konecny2015federated}.

On typical real-world systems, communicating data between machines is several orders of magnitude slower than reading data from main memory, e.g., when leveraging commodity hardware.
Therefore when trying to translate existing highly tuned single machine solvers to the distributed setting, great care must be taken to avoid this significant communication bottleneck \cite{Jaggi:cocoa,Yang:2013vl}.

While several distributed solvers for the problems of interest have been recently developed, they are often unable to fully leverage the competitive performance of their tuned and customized single machine counterparts, which have already received much more research attention. 
More importantly, it is unfortunate that distributed solvers cannot automatically benefit from improvements made to the single machine solvers, and therefore are forced to lag behind the most recent developments.

In this chapter we make a step towards resolving these issues by proposing a general communication-efficient distributed framework that can employ arbitrary single machine local solvers and thus directly leverage their benefits and problem-specific improvements.
Our framework works in rounds, where in each round the local solvers on each machine find a (possibly weak) solution to a specified subproblem of the same structure as the original master problem. On completion of each round, the partial updates between the machines are efficiently combined by leveraging the primal-dual structure of the global problem \cite{Yang:2013vl,Jaggi:cocoa,Ma:2015ti}. The framework therefore completely decouples the local solvers from the distributed communication. Through this decoupling, it is possible to balance communication and computation in the distributed setting, by controlling the desired accuracy and thus computational effort spent to determine {the solution to each local subproblem}. Our framework holds with this abstraction even if the user wishes to use a different local solver on each machine.

\subsection{Contributions}

\paragraph*{\it Reusability of Existing Local Solvers.}
The proposed framework allows for distributed optimization with the use of arbitrary local solvers on each machine. This abstraction makes the resulting framework highly flexible, and means that it can easily leverage the benefits of well-studied, problem-specific single machine solvers. In addition to increased flexibility and ease-of-use, this can result in large performance gains, as single machine solvers for the problems of interest have typically been thoroughly tuned for optimal performance. Moreover, any performance improvements that are made to these local solvers can be automatically translated by the framework into the distributed setting.

\paragraph*{\it Adaptivity to Communication Cost.}
On real-world compute systems, the cost of communication versus computation typical varies by many orders of magnitude, from high-performance computing environments to very slow disk-based distributed workflow systems such as MapReduce/Hadoop. 
For optimization algorithms, it is thus essential to accommodate varying amounts of work performed locally per round, while still providing convergence guarantees.  Our framework provides exactly such control. 

\paragraph*{\it Strong Theoretical Guarantees.}
In this chapter we extend and improve upon the \cocoa~\cite{Jaggi:cocoa} method. Our theoretical convergence rates apply to both smooth and non-smooth losses, and for both \cocoa as well as \cocoap, the more general framework presented here.
Our new rates exhibit favorable \emph{strong scaling} properties for the class of problems considered, as the number of machines~$K$ increases and the data size is kept fixed. More precisely, while the convergence rate of \cocoa degrades as $K$ is increased, the stronger theoretical convergence rate here is---in the worst case {complexity}---\emph{independent} of $K$. 
As only one vector is communicated per round and worker, 
this favorable scaling might be surprising. Indeed, for existing methods, splitting data among more machines generally
increases communication requirements \cite{shamir2014distributed,Arjevani:2015vka}, which
can severely affect overall runtime.

\paragraph*{\it Primal-Dual Convergence.} We additionally strengthen the rates by showing stronger 
primal-dual convergence for both algorithmic frameworks, which are almost 
tight to their dual-only (or primal-only) counterparts.
Primal-dual rates for \cocoa had not previously been analyzed in the general convex case.
Our primal-dual rates allow efficient and practical certificates for the 
optimization quality, e.g., for stopping criteria. 

\paragraph*{\it Experimental Results.}
Finally, we provide an extensive experimental comparison that highlights the impact of using various arbitrary solvers locally on each machine, with experiments on several real-world, distributed datasets. We compare the performance of \cocoa and \cocoap across these datasets and choices of solvers, in particular illustrating the performance on a {}280 GB dataset. Our code is available in an open source C\texttt{++} library, at: \url{https://github.com/optml/CoCoA}.

\subsection{Outline} The rest of the chapter is organized as follows. Section~\ref{sec:cocoa:problem} provides context and states the problem of interest, including necessary assumptions and their consequences. In Section~\ref{sec:cocoa:algorithm} we formulate the algorithm in detail and explain how to implement it efficiently in practice. The main theoretical results are presented in Section~\ref{sec:cocoa:result},
followed by a discussion of relevant related work in Section~\ref{sec:cocoa:relatedWork}.
 Practical experiments demonstrating the strength of the proposed framework are given in Section~\ref{sec:cocoa:experiments}.
Finally, we prove the main results in Section~\ref{sec:cocoa:analysis}.

\section{Background and Problem Formulation}
\label{sec:cocoa:problem}

To provide context for our framework, we first state traditional complexity measures and convergence rates for single machine algorithms,  and then demonstrate that these must be adapted to more accurately represent the performance of an algorithm in the distributed setting.

When running an iterative optimization algorithm $\mathcal{A}$ on a single machine, its performance is typically measured by the total runtime:
\begin{equation}
\label{eq:runtime:default}
\tag{T-A}
\text{TIME}(\mathcal{A}) = \mathcal{I}_\mathcal{A} (\epsilon) \times \mathcal{T}_\mathcal{A} \, .
\end{equation}
Here, $\mathcal{T}_\mathcal{A}$ stands for the time it takes to perform a single iteration of algorithm $\mathcal{A}$, and $\mathcal{I}_\mathcal{A} (\epsilon)$ is the number of iterations  $\mathcal{A}$ needs to attain an $\epsilon$-accurate objective.\footnote{While for many algorithms the cost of a single iteration will vary throughout the iterative process, this simple model will suffice for our purpose to highlight the key issues associated with extending algorithms to a distributed framework.} 
 
{On a single machine, most of the state-of-the-art first-order optimization methods can achieve quick convergence in practice in terms of \eqref{eq:runtime:default} by performing a large amount of relatively fast iterations.} In the distributed setting, however, time to communicate between two machines can be several orders of magnitude slower than even a single iteration of such an algorithm. As a result, the overall time needed to perform {this} single iteration can increase significantly. 

Distributed timing can therefore be more accurately illustrated using the following practical distributed efficiency model (see also \cite{marecek2014distributed}), where
\begin{equation}
\label{eq:runtime:distributed}
\tag{T-B}
\text{TIME}(\mathcal{A}) = \mathcal{I}_\mathcal{A} (\epsilon) \times \left( c + \mathcal{T}_\mathcal{A} \right).
\end{equation}
The extra term $c$ is the time required to perform one round of communication.\footnote{%
For simplicity, we assume here a fixed network architecture, and compare only to classes of algorithms that communicate a single vector in each iteration, rendering $c$ to be a constant, assuming we have a fixed number of machines. Most first-order algorithms would fall into this class.} As a result, an algorithm that performs well in the setting of \eqref{eq:runtime:default} does not necessarily perform well in the distributed setting \eqref{eq:runtime:distributed}, especially when implemented in a straightforward or na\"{i}ve way. 
In particular, if $c \gg \mathcal{T}_\mathcal{A}$, we could intuitively expect less potential for improvement {from fast computation}, as most of the time in the method will be spent on communication, not on actual computational effort to solve the problem. In this setting, novel optimization procedures are needed that carefully consider the amount of communication and the distribution of data across multiple machines.

One approach to this challenge is to design novel optimization algorithms from scratch, designed to be efficient in the distributed setting. This approach has one obvious practical drawback: 
There have been numerous highly efficient solvers developed and fine-tuned to  particular problems of interest, as long as the problem fits onto a single machine. These solvers are ideal if run on a single machine, but with the growth of data and necessity of data distribution, they must be re-designed to work in modern data regimes.

Recent work \cite{Jaggi:cocoa,Ma:2015ti,Yang:2013vl,Yang:2013ui,Yu:2012fp,Smith:2015ua} has attempted to address this issue by designing algorithms that reduce the communication bottleneck by allowing infrequent communication, while utilizing already existing algorithms as local sub-procedures. 
{The presented work here builds on the promising approach of \cite{Yang:2013vl,Jaggi:cocoa} in this direction.
See Section \ref{sec:cocoa:relatedWork} for a detailed discussion of the related literature.}

The core idea in this line of work is that one can formulate a local subproblem for each individual machine, and run an arbitrary local solver dependent only on local data for a number of iterations---obtaining a partial local update. After each worker returns its partial update, a global update is formed by their aggregation.  

The big advantage of this is that companies and practitioners do not have to implement new algorithms that would be suitable for the distributed setting. We provide a way for them to utilize their existing algorithms that work on a single machine, and provide a novel communication protocol on top of this. 

In the original work on \cocoa \cite{Jaggi:cocoa}, authors provide convergence analysis only for the case when {the overall update is formed as an average of the partial updates}, and note that in practice it is possible to improve performance  {by making a longer step in the same direction}. The main contribution of this work is a more general convergence analysis of various settings, which enables us to do better than averaging. In one case, we can even sum the partial updates  {to obtain the overall update}, which yields the best result, both in theory and practice. We will see that this can result in significant performance gains, see also \cite{Ma:2015ti,Smith:2015ua}.

In the analysis, we will allow local solvers of arbitrarily weak accuracy, each working on its own subproblem which is defined in a completely data-local way for each machine. The relative accuracy obtained by each local solver will be denoted by $\Theta\in[0,1]$, where $\Theta=0$ describes an exact solution of the subproblem, and~$\Theta=1$ means that the local subproblem objective has not improved at all, for this run of the local solver. 
This paradigm results in a substantial change in how we analyze efficiency in the distributed setting. The formula practitioners are interested in minimizing thus changes to become:
\begin{equation}
\label{eq:runtime:cocoa}
\tag{T-C}
\text{TIME}(\mathcal{A}, \Theta) = \mathcal{I} (\epsilon, \Theta) \times \left( c + \mathcal{T}_\mathcal{A}(\Theta) \right).
\end{equation}
Here, the function $\mathcal{T}_\mathcal{A}(\Theta)$ represents the time the local algorithm $\mathcal{A}$ needs to obtain an accuracy of~$\Theta$ on the local subproblem.  
Note that the number of outer iterations~$\mathcal{I} (\epsilon, \Theta)$ is independent of choice of the inner algorithm $\mathcal{A}$, which will also be reflected by our convergence analysis presented in Section~\ref{sec:cocoa:result}. Our convergence rates will hold for any local solver $\mathcal{A}$ achieving local {accuracy of} $\Theta$.
For strongly convex problems, the general form will be 
$ \mathcal{I} (\epsilon, \Theta) = \frac{\mathcal{O}(\log(1/\epsilon))}{1 - \Theta} \, . $
The inverse dependence on $1 - \Theta$ suggests that there is a limit to how much we can gain by solving local subproblems to high accuracy, i.e., for $\Theta$ close to~$0$. There will always be on the order of $\log(1 / \epsilon)$ outer iterations needed. Hence, excessive local accuracy should not be necessary. On the other hand, if $\Theta \rightarrow 1$, meaning that the cost and quality of the local solver diminishes, then the number of outer iterations $\mathcal{I} (\epsilon, \Theta)$ will increase dramatically, which is to be expected.

To illustrate the strength of the paradigm \eqref{eq:runtime:cocoa} compared to \eqref{eq:runtime:distributed}, suppose that we run just a single iteration of gradient descent as the local solver $\mathcal{A}$. Within our framework, choosing this local solver would lead to a method which is equivalent to naively distributed gradient descent\footnote{Note that this is not obvious at this point. They are identical, subject to choice of local subproblems as specified in Section~\ref{sec:cocoa:subproblem}.}. Indeed, running gradient descent for a single iteration would {attain} a particular value of $\Theta$. Note that we typically do not set {this explicitly: $\Theta$} is implicitly chosen by the number of iterations or stopping criterion specified by the user for the local solver. There is no reason to think that the value attained by single iteration of gradient descent would be optimal. For instance, it may be the case that running gradient descent for, say,  $200$ iterations, instead of 
just one, would give substantially better {results} in practice, due to better communication efficiency. {Considerations of this form} are discussed {in detail} in Section~\ref{sec:cocoa:experiments}.

In general, one would intuitively expect that the optimal choice would be to have~$\Theta$ such that $\mathcal{T}_\mathcal{A}(\Theta) = \mathcal{O}(1) \times c$. In practice, however, the best strategy for any given local solver {is to estimate the optimal choice by trying several values for the number of local iterations}. {We discuss the importance of $\Theta$}, both theoretically and empirically, in Sections~\ref{sec:cocoa:result} and \ref{sec:cocoa:experiments}.
\vspace{-1mm}

\subsection{Problem Formulation}
\label{sec:cocoa:problemformulation}
Let the training data $\{\xv_i \in \R^d, y_i \in \R \}_{i=1}^n$ be the set of input-output pairs, where $y_i$ can be real valued {or} from a discrete set in the case of classification problems.
We will assume without loss of generality  that $\forall i: \|\xv_i\|\leq 1$. Many common tasks in machine learning and signal processing can be cast as the following optimization problem:
\begin{equation}
\label{eq:primal}
\min_{\wv\in \R^d} \left\{ P(\wv) \eqdef \frac1n \sum_{i=1}^n \ell_i( \xv_i^T \wv) + \lambda g(\wv) \right\},
\end{equation}
where $\ell_i$ is some convex loss function and $g$ is a regularizer. Note that $y_i$ is typically hidden in the formulation of functions $\ell_i$. Table \ref{tbl:differentLossFunctions} lists several common loss functions together with their convex conjugates $\ell^*_i$ \cite{ShalevShwartz:2014dy}.

\begin{table} 
{\begin{tabular}{llll}
\toprule
\multicolumn{1}{c}{Loss function} 
  & \multicolumn{1}{c}{ $\ell_i(a)$}  & \multicolumn{1}{c}{$\ell^*_i(b)$}
  & \multicolumn{1}{c}{Property  of $\ell$}
   \\ \midrule
Quadratic & $\frac12 (a-y_i)^2$ &  $\frac12 b^2 + y_ib $
& Smooth 
 \\
Hinge  & $\max\{0, y_i-a\}$ &  $y_i b, \quad b\in[-1,0] $ 
 & Continuous \\
Squared hinge & $(\max\{0, y_i-a\})^2$ & $\frac{b^2}{4}, \quad b\in[-\infty,0] $ 
 & Smooth \\
Logistic & $\log(1+\exp{(-y_ia)})$ & $-\frac{b}{y_i}\log \left(-\frac{b}{y_i}\right) + \left( 1 +\frac{b}{y_i} \right) \log \left( 1 + \frac{b}{y_i} \right) $ 
& Smooth \\
\bottomrule
\end{tabular}
}
\caption{Examples of commonly used loss functions.}
\label{tbl:differentLossFunctions}
\end{table}

The dual optimization problem for formulation \eqref{eq:primal}{---}as a special case of Fenchel duality{---}can be written as follows \cite{yuan2012recent,ShalevShwartz:2014dy}:
\begin{equation}
\label{eq:dual}
\max_{\alphav \in \R^n}
 \left\{
 D(\alphav )\eqdef  
 \frac1n \left(\sum_{i=1}^n -\ell_i^*(- \alpha_i)\right)
 - \lambda g^* \left( \frac1{\lambda n} X \alphav \right) \right\},
\end{equation}
where $X = [\xv_1, \xv_2, \dots, \xv_n] \in \R^{d \times n}$, and $\ell_i^*$ and $g^*$ are the convex conjugate functions of $\ell_i$ and $g$, respectively. The convex (Fenchel) conjugate of a function $\phi : \R^k \rightarrow \R$ is defined as the function $\phi^* : \R^k \rightarrow \R$, with $\phi^*(\uv) := \sup_{\sv\in\R^k} \{ \sv^T \uv - \phi(\sv) \} $.

For simplicity throughout the chapter, let us denote 
\begin{equation}
\label{eq:fRdefinition}
f(\alphav) \eqdef  \lambda g^*\left( \frac{1}{\lambda n} X \alphav \right) 
\qquad \text{and} \qquad R(\alphav) \eqdef \frac1n \sum_{i=1}^n \ell_i^*(- \alpha_i)\, ,
\end{equation}
such that
$
D(\alphav) \overset{\eqref{eq:dual}+\eqref{eq:fRdefinition}}{=} - f(\alphav) - R(\alphav).
$

It is well known \cite{QUARTZ,takac-minibatch,ShalevShwartz:2014dy,Dunner:2016vga} that the first-order optimality conditions give rise to {a natural mapping that relates pairs of primal and dual variables. {This} mapping employs the linear map given by the data $X$, and maps any dual variable $\alphav \in\R^n$ to a primal candidate vector $\wv \in\R^d$ as follows:}
$$ \wv(\alphav) :=  \nabla g^*( \vv(\alphav) ) 
= \nabla g^*\left( \frac1{\lambda n} X \alphav \right),$$
where we denote
$
\vv(\alphav) := \tfrac1{\lambda n} X \alphav \, .
$

For this mapping, under the assumptions that we make in Section~\ref{sec:cocoa:techassump} below, it holds that if $\alphav^\star$ is an optimal solution of \eqref{eq:dual}, then $\wv(\alphav^\star)$ is an optimal solution of~\eqref{eq:primal}. In particular, {\em strong duality} holds between the primal and dual problems. If we define the duality gap function as
\begin{align}
\label{eq:gap}
\gap(\alphav)
 := P(\wv(\alphav))-\bD(\alphav),
\end{align}
then $\gap(\alphav^\star)=0$, 
which ensures that by solving the dual problem \eqref{eq:dual} we also solve the original primal problem of interest~\eqref{eq:primal}. As we will later see, there are many benefits to leveraging this primal-dual relationship, including the ability to use the duality gap as a certificate of solution quality, and, in the distributed setting, {the ability to effectively distribute computation.}

\vspace{-1em}
\paragraph*{\it Notation.}

We assume that to solve problem \eqref{eq:dual}{,} we have a network of $K$ machines at our disposal.
The data $\{\xv_i,y_i\}_{i=1}^n$ is residing on the $K$ machines in a distributed {fashion}, with every machine {holding a subset of the} whole dataset.
{We distribute the dual variables in the same manner}, with each {dual variable $\alpha_i$} corresponding to an individual data point~$\xv_i$. The given data distribution is described using a partition $\mathcal{P}_1, \dots, \mathcal{P}_K$ that corresponds to the indices of {the} data and dual variables residing on machine~$k$. Formally, $\mathcal{P}_k \subseteq \{1, 2, \dots, n\}$ for each $k${;}  $\mathcal{P}_k \cap \mathcal{P}_l = \emptyset$ whenever $k \neq l${;} and $\bigcup_{k = 1}^K \mathcal{P}_k = \{1, 2, \dots, n\}$.

{Finally, we introduce the following notation dependent on this partitioning.} For any $\hv \in \R^n$, let $\hk$ be the vector in $\R^n$ defined {such} that
$ (\hk)_i = h_i$ if $i \in \mathcal{P}_k$ and $0$ otherwise.
Note that, in particular, 
$
\hv = \sum_{k = 1}^ K   \hk.
$
Analogously, we write~$\Xk$ for the matrix consisting only of the columns $i \in \mathcal{P}_k$, padded with zeros in all other columns.

\subsection{Technical Assumptions}
\label{sec:cocoa:techassump}
Here we first state the properties and assumptions used throughout the chapter.
We assume that for all $i\in \{1,\dots,n\}$, the function 
$\ell_i$ {in~\eqref{eq:primal}} is convex, i.e.,
$\forall \lambda \in [0, 1]$ and $\forall x,y \in \R$ we have 
$\ell_i(\lambda x + (1 - \lambda) y) \leq \lambda \ell_i(x) + (1 - \lambda) \ell_i(y) \, .$
 
We also assume {that the} function 
$g$ {is} $1$-strongly convex, i.e., for all $\wv, \uv \in \R^d$ it holds that
$
g(\wv + \uv) \geq g(\wv) + \< \nabla g(\wv), \uv > + \tfrac{1}{2} \| \uv \|^2,
$
where $\nabla g(\wv)$ is any subgradient\footnote{A subgradient of a convex function $\phi$ in a point $\xv' \in \R^d$ is defined as any $\xi \in \R^d$ satisfying for all $\xv \in \R^d$, $\phi(\xv) \geq \phi(\xv') + \< \xi, \xv - \xv' >$.} of the function $g$.
Here, $\| \cdot \|$ denotes the standard Euclidean norm.

Note that we use subgradients in the definition of strong convexity. This is due to the fact that while we will need the function $g$ to be strongly convex in our analysis, we do not require smoothness. An example used in practice is $g(\wv) = \| \wv \|^2 + \lambda' \| \wv \|_1$ for some $\lambda' \in \R$. Also note that in the problem formulation \eqref{eq:primal} we have a regularization parameter~$\lambda$, which controls the strong convexity parameter of the entire second term. Hence, fixing the strong convexity parameter of $g$ to $1$ is not restrictive in this regard. For instance, this setting has been used previously in~\cite{ShalevShwartz:2014dy,QUARTZ,csiba2015stochastic}.

The following assumptions state properties of the functions $\ell_i$, which we use only in certain results in the chapter. We always explicitly state when we require each assumption.

\begin{assumption}[$(1/\gamma)$-Smoothness]
\label{ass:1gammasmooth}
Functions $\ell_i: \R \rightarrow \R$ are $1/\gamma$-smooth, if $\forall i \in \{ 1, \dots, n \}$ and $\forall x, h \in \R$ it holds that
\begin{equation}
\label{def:Lsmoothness}
\ell_i(x + h) \leq \ell_i(x) + h \nabla \ell_i(x) + \frac{1}{2 \gamma} h^2,
\end{equation}
where $\nabla \ell_i(x)$ denotes the gradient of the function $\ell_i$.
\end{assumption}
\begin{assumption}[$L$-Lipschitz Continuity]
\label{ass:LLip}
Functions $\ell_i: \R \rightarrow \R$ are $L$-Lipschitz continuous, if $\forall i \in \{ 1, \dots, n \}$ and $\forall x, h \in \R$ it holds that
\begin{equation}
\label{def:LLip}
|\ell_i(x + h) - \ell_i(x)| \leq L |h|.
\end{equation}
\end{assumption}

\begin{remark}
As a consequence of having $1/\gamma$-smoothness of $\ell_i$ and $1$-strong convexity of $g$, we have that the functions $\ell_i^*(\cdot)$ are $\gamma$-strongly convex and $g^*(\cdot)$ is $1$-smooth~\cite{rockafellar1970}. These are the properties we will ultimately use as we will be solving the dual problem~\eqref{eq:dual}. Note that $1$-smoothness of $g^* : \R^d \rightarrow \R$ means that for all~$\xv, \hv \in \R^d$,
\begin{equation}
\label{def:Lsmoothness:gstar}
g^*(\xv + \hv) \leq g^*(\xv) + \< \nabla g^*(\xv), \hv > + \frac{1}{2} \| \hv \|^2.
\end{equation}
\end{remark}
{The following} lemma, {which is} a consequence of 1-smoothness of $g^*$ and the definition of $f$,  will be crucial in deriving a meaningful local subproblem for the proposed distributed framework.
\begin{lemma}
\label{lem:quartz}
Let $f$ be defined in \eqref{eq:fRdefinition}. Then for all $\alphav, \hv \in \R^n$ we have
\begin{equation}
\label{eq:quartz}
f(\alphav + \hv) \leq f(\alphav) + \< \nabla f(\alphav), \hv> + \frac{1}{2 \lambda n^2} \hv^T X^T X \hv.
\end{equation}
\end{lemma}

\begin{remark} Note that although the above inequality appears as a consequence of the problem structure~\eqref{eq:dual} and of {the} strong convexity of $g$, there are other ways to satisfy it. Hence, our dual analysis holds for all optimization problems of the form $\max_{\alphav} D(\alphav)$, where $D(\alphav) = -f(\alphav) - R(\alphav)$, and where $f$ satisfies inequality \eqref{eq:quartz}. However, for the  duality gap  analysis we naturally {do require 
that} the dual problem arises from the primal problem{,} with $g$ being strongly convex.
\end{remark}

\section{The Framework}
\label{sec:cocoa:algorithm}

In this section we start by giving a general view of the proposed framework, explaining the most important concepts needed to make the framework efficient. In Section~\ref{sec:cocoa:subproblem} we discuss the formulation of the local subproblems, and in Section~\ref{sec:cocoa:implementationnotes} {we provide} specific details and best practices for implementation.

The data distribution plays a crucial role in Algorithm~\ref{alg:cocoa}, where in each outer iteration indexed by $t$, machine $k$ runs an arbitrary local solver on a problem described only by the data that particular machine owns and other fixed constants or linear functions. 

The crucial property is that the optimization algorithm on machine $k$ changes only coordinates of the dual optimization variable $\alphav^t$ corresponding to the partition~$\mathcal{P}_k$ to obtain an approximate solution to the local subproblem. We will formally specify this in Assumption~\ref{ass:localImprovement}. After each such step, updates from all machines are aggregated to form a new iterate $\alphav^{t+1}$. The aggregation parameter $\aggpar$ will typically be between $\aggpar = 1/K$, corresponding to averaging, and $\aggpar = 1$, {adding}. 

\begin{algorithm} 
\caption{Improved \textsc{CoCoA}\texttt{+} Framework} 
\label{alg:cocoa}
\begin{algorithmic}[1]
\State {\bf Input:} starting point $\alphav^0 \in \R^{n}$, aggregation parameter $\aggpar \in (0,1]$, data partition $\{\mathcal{P}_k\}_{k=1}^K$
\For {$t = 0, 1, 2, \dots $}
  \For {$k \in \{1,2,\dots,K\}$ {\bf in parallel over machines}}
     \State Let $\hk^t$ be an approximate solution of the local problem \eqref{eq:subproblem:sigma1}, i.e.\vspace{-2mm}
     \[
     \max_{ \hk \in \R^n } \Gk(\hk; \alphav^t)
     \vspace{-2mm}
     \]
  \EndFor
  \State Set $\alphav^{t+1} := \alphav^t + \aggpar \sum_{k=1}^K \hk^t$
\EndFor  
\end{algorithmic}
\end{algorithm}

Here we list the core conceptual properties of Algorithm~\ref{alg:cocoa}, which are important qualities that allow it to run efficiently.  
\begin{description}

\item [{Locality.}] The local subproblem $\Gk$ \eqref{eq:subproblem:sigma1} is defined purely based on the data points residing on machine $k$, as well as a single shared vector in $\R^d$ (representing the state of the $\alphav^t$ variables of the other machines). Each local solver can then run independently and in parallel, i.e., there is no need for communication while solving the local subproblems.
\item [{Local changes.}] The optimization algorithm used to solve the local subproblem $\Gk$ outputs a vector $\hk^t$ with nonzero elements only in coordinates corresponding to variables $\alphak$ stored locally (i.e., $i\in\mathcal{P}_k$). 
\item [{Efficient maintenance.}] Given the description of the local problem $\Gk(\,\cdot\,; \alphav^t)$ at time $t$, the new local problem $\Gk(\,\cdot\,; \alphav^{t+1})$ at time $t+1$ can be formed on each machine, requiring only communication of a single vector in $\R^d$ from each machine $k$ to the master node, and vice versa, back to each machine $k$.
 
\end{description}

Let us now comment on these properties in more detail. Locality is important for making the method versatile, and is the way we escape the restricted setting described by \eqref{eq:runtime:distributed} that allows us much greater flexibility in designing the overall optimization scheme. 
Local changes result from the fact that {we distribute coordinates of the dual variables $\alphav$ in the same manner as the data}, and thus only make updates to the coordinates stored locally.
As we will see, efficient maintenance of the subproblems can be obtained. For this, a  communication-efficient encoding of the current shared state $\alphav$ is necessary. To this goal, we will in Section~\ref{sec:cocoa:implementationnotes} show that communication of a single $d$-dimensional vector is enough to formulate the subproblems \eqref{eq:subproblem:sigma1} in each round, by carefully exploiting  their partly separable structure.

Note that Algorithm~\ref{alg:cocoa} is the ``analysis friendly'' formulation of our algorithm framework, and it is not yet fully illustrative for implementation purposes. In Section~\ref{sec:cocoa:implementationnotes} we will precisely formulate the actual communication scheme, and illustrate how the above properties can be achieved.

Before that, {we formulate} the precise subproblem $\Gk$ in the following section.

\subsection{The Local Subproblems}
\label{sec:cocoa:subproblem}

We can define a data-local subproblem of the original dual optimization problem~\eqref{eq:dual}, which can be solved on machine $k$ and only requires accessing data which is already available locally, i.e., datapoints with $i\in\mathcal{P}_k$. More formally, each machine $k$ is assigned the following local subproblem, depending only on the previous shared primal vector~$\wv\in\R^d$, and the change in the local dual variables~$\alpha_i$ with $i\in\mathcal{P}_k$:
\begin{equation} 
\max_{\hk\in\R^{n}} \Gks(  \hk; \alphav).
\end{equation}
We are now ready to define the local objective $\Gks(\,\cdot\,; \alphav)$ as follows:
\begin{align}
\label{eq:subproblem:sigma1}
\tag{LO}
\Gks(\hk; \alphav) :=  -\frac{1}K f(\alphav) - \< \nabla f(\alphav), \hk > - \frac{\lambda \sigma'}{2}  \left\| \frac{1}{\lambda n}\Xk \hk \right\|^2 
 - R_k\!\left( \alphak + \hk \right),
\end{align}
where $ R_k(\alphak) \eqdef  \frac1n \sum_{i \in \mathcal{P}_k} \ell_i^*(-\alphav_i) $.
The role of the parameter $\sigma'\geq 1$ is to measure the ``difficulty'' of the data partition, in a sense which we will discuss in detail in {Section \ref{sec:cocoa:choiceofsigma}}.

The interpretation of {the subproblems defined above} is that they will form a quadratic approximation of the smooth part of the true objective $D$, which becomes separable over the machines. The approximation keeps the non-smooth $R$ part intact.
The variable $\hk$ expresses the update proposed by machine $k$. In this spirit, note also that the approximation coincides with $D$ at the reference point $\alphav$, i.e. $\sum_{k=1}^K \Gks( {\bf 0}; \alphav) = D(\alphav)$.
We will discuss the interpretation and properties of these subproblems in more detail below in Section \ref{sec:cocoa:choiceofsigma}.

\subsection{Practical Communication-Efficient Implementation}
\label{sec:cocoa:implementationnotes}

We {now} discuss how Algorithm~\ref{alg:cocoa} can efficiently be implemented in a distributed environment. Most importantly, {we} clarify how the ``local'' subproblems {can be} formulated and solved {while} using only local information from the corresponding {machines}, and {we} make precise what information needs to be communicated in each round.

Recall that the local subproblem objective $\Gks(\,\cdot\,; \alphav)$ was defined in \eqref{eq:subproblem:sigma1}.
We will now equivalently rewrite this optimization problem, {illustrating how it can be} expressed {using only} \emph{local} information. To do so, we use our simplifying notation~$\vv = \vv(\alphav) := \tfrac1{\lambda n} X \alphav$ for {a} given~$\alphav$. As we see in the reformulation, it is precisely this vector $\vv\in\R^d$ which contains all the necessary shared information between the machines. Given the vector~$\vv$, the subproblem \eqref{eq:subproblem:sigma1} {can be equivalently written} as
\begin{align}
\label{eq:subproblemPr}
\tag{LO'}
\Gks(\hk; \vv, \alphak) &:=  
 -\frac{\lambda}K g^*(\vv) 
 - \left\langle \frac{1}{n} \Xk^T \nabla g^*(\vv), \hk \right\rangle - \frac{\lambda}{2} \sigma' \left\| \frac{1}{\lambda n}\Xk \hk \right\|^2 
 \\&\qquad  - R_k\!\left( \alphak + \hk \right). \notag
\end{align} 
Here for the reformulation of the gradient term, we have simply used the chain rule on the objective $f$ (recall the definition $f(\alphav) \eqdef  \lambda g^*( \vv )$), giving
$$
\vsub{\nabla f(\alphav)}{k} = \frac{1}{n} \Xk^T \nabla g^* ( \vv ) .
$$

\paragraph*{\it Practical Distributed Framework.}
In summary, we have seen that each machine can formulate the local subproblem given purely local information (the local data $\Xk$ as well as the local dual variables $\alphak$). No information about the {data or variables $\alphav$ stored on the} other machines is necessary.

The only requirement for the method to work is that between the rounds, the changes in the $\alphak$ variables on each machine and the resulting global change in~$\vv$ are kept consistent, in the sense that $\vv^t = \vv(\alphav^t) := \tfrac1{\lambda n} X \alphav^t$ must always hold. Note that for the evaluation of $\nabla g^*(\vv)$, the vector $\vv$ is all that is needed. In practice, $g$ as well as its conjugate $g^*$ are simple {vector-valued} regularization functions, the most prominent example being $g(\vv) = g^*(\vv) = \frac12 \| \vv \|^2$.

In the following more detailed formulation of the \cocoap framework shown in Algorithm~\ref{alg:cocoaPractical} ({an} equivalent reformulation of Algorithm~\ref{alg:cocoa}), the crucial communication pattern of the framework finally becomes more clear: Per round, \emph{only a single vector} (the update on $\vv\in\R^d$) needs to be sent over the communication network. 
The reduce-all operation in line 10 means that each machine sends their vector $\Delta \vv_k^t\in\R^d$ to the network, which performs the addition operation of the $K$ vectors to the old~$\vv^t$. The resulting vector $\vv^{t+1}$ is then communicated back to all machines, so that all  have the same copy of $\vv^{t+1}$ before the beginning of the next round.

The framework as shown below in Algorithm~\ref{alg:cocoaPractical} clearly maintains the consistency of $\alphav^t$ and $\vv^t = \vv^t(\alphav^t)$ after each round, no matter which local solver is used to approximately solve \eqref{eq:subproblemPr}. A diagram illustrating the communication and computation involved in the first two full iterations of Algorithm~\ref{alg:cocoaPractical} is given in Figure~\ref{fig:diagram}.

\begin{algorithm} 
\caption{Improved \textsc{CoCoA}\texttt{+} Framework, Practical Implementation}
\label{alg:cocoaPractical}
\begin{algorithmic}[1]
\State {\bf Input:} starting point $\alphav^0 \in \R^n$, aggregation parameter $\aggpar \in (0,1]$, data partition $\{\mathcal{P}_k\}_{k=1}^K$
\State $\vv^0 := \frac1{\lambda n} X \alphav^0  \in \R^d$
\For {$t = 0, 1, 2, \dots $}
  \For {$k \in \{1,2,\dots,K\}$ {\bf in parallel over machines}}
	\State Precompute $ \Xk^T \nabla g^*(\vv^t)$
     \State Let $\hk^t$ be an approximate solution of the local problem \eqref{eq:subproblemPr}, i.e.\vspace{-2mm}
\[
\max_{ \hk \in \R^n } \Gks(\hk; \vv^t, \alphak^t)  
\vspace{-6mm}
\]  
	\Comment{{\it computation}}
	\State Update local variables $\alphak^{t+1} := \alphak^t + \aggpar \hk^t$
  	
	\State Let $\Delta \vv_k^t := \frac{1}{\lambda n} \Xk \hk^t$
  \EndFor 
  \State \textbf{reduce all} to compute
	$
	\vv^{t+1} := \vv^t + \aggpar \sum_{k=1}^K \Delta \vv_k^t
	$
 \Comment{{\it communication}}
\EndFor  
\end{algorithmic}
\end{algorithm}

\begin{figure}[t]
\centering
\includegraphics[width=\linewidth,trim={10 60 20 100}]{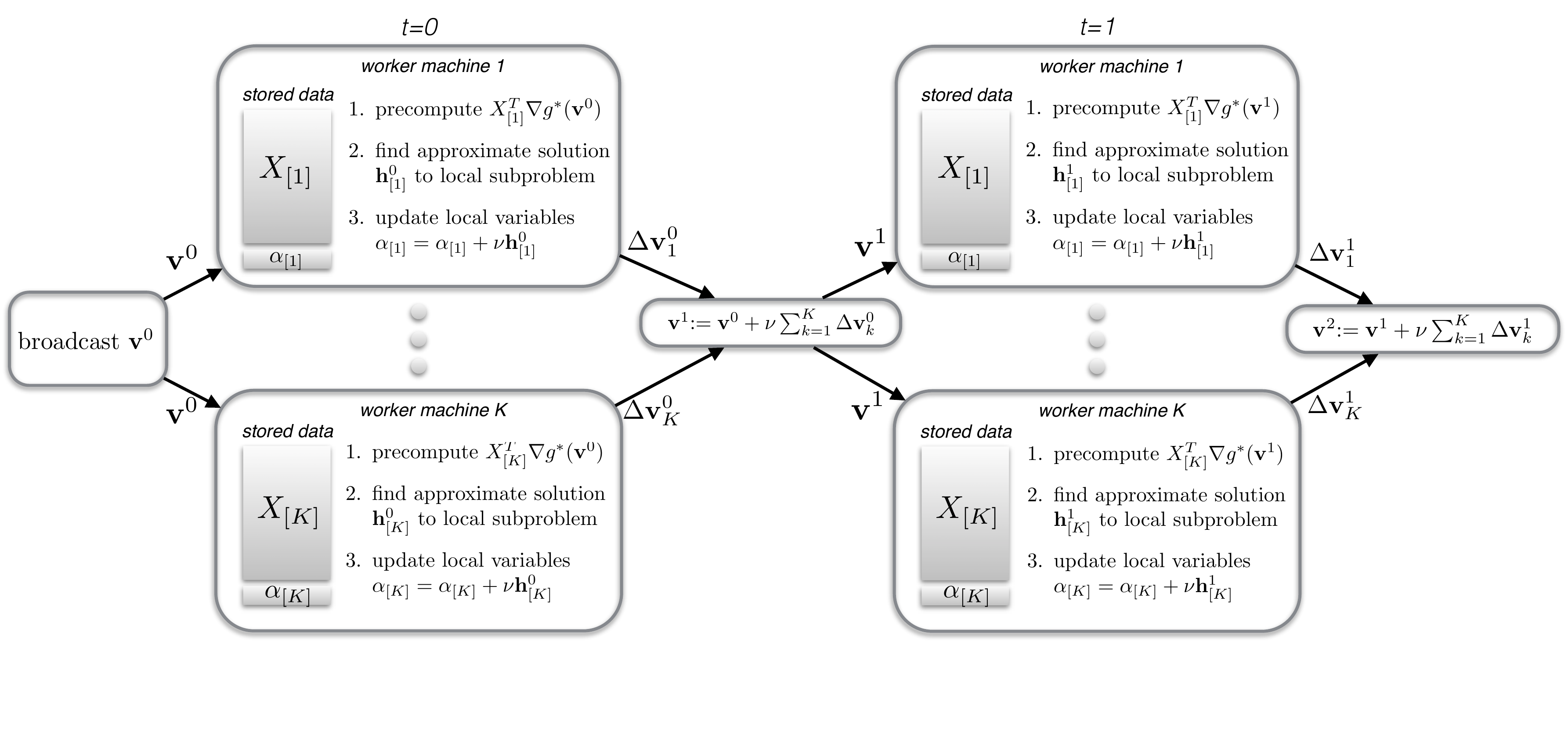}
\caption{The first two iterations of the improved framework (practical implementation).}
\label{fig:diagram}
\end{figure}

\subsection{Compatibility of the Subproblems for Aggregating Updates}
\label{sec:cocoa:choiceofsigma}
In this subsection, we shed more light on the local subproblems on each machine, as defined in \eqref{eq:subproblem:sigma1} above, and their interpretation.
More formally, {we show} how the aggregation {parameter} $\nu$ (controlling the level of adding versus averaging the resulting updates from each machine) and $\sigma'$ (the subproblem parameter) interplay together, {so that in each round they} achieve a valid approximation to the global objective function~$D$.

The role of the subproblem parameter~$\sigma'$ is to measure the difficulty of the given data partition. For the convergence results discussed below to hold, $\sigma'$ must be chosen not smaller than
\begin{equation}
\label{eq:sigmaPrimeSafeDefinition}
\sigma'
\geq
\sigma'_{min}
 \eqdef
 \nu \cdot
 \max_{\hv\in \R^n}
 \big\{
 \hv^T X^T X \hv \ \big|\ \hv^T \nBG \hv \le 1\big\} \, . \vspace{-1mm}
\end{equation}

Here, $\nBG$ is the block diagonal submatrix of the data covariance matrix $X^T X$, corresponding to the partition $\{\mathcal{P}_k\}_{k=1}^K$, i.e.,
\begin{equation}
\label{eq:nBGDefinition}
\nBG_{ij} \eqdef
\begin{cases}
	\xv_i^T \xv_j = (X^TX)_{ij}, & \mbox{if}\ \exists k \ \mbox{such that} \ i,j \in \mathcal{P}_k, \\ 
	0,& \mbox{otherwise}. 
\end{cases}
\end{equation}

In this notation, it is easy to see that the crucial quantity defining $\sigma'_{min}$ above is written as $\hv^T \nBG \hv = \sum_{k=1}^K \|X_{[k]} \hk\|^2$.

The following lemma shows that if the aggregation and subproblem parameters
$\nu$ and~$\sigma'$ satisfy \eqref{eq:sigmaPrimeSafeDefinition}, then the sum of the subproblems $\sum_k \Gks$ will closely approximate the global objective function $D$. More precisely, this sum is a block-separable lower bound on $D$.

\begin{lemma}
\label{lem:step:sigma1}
\label{lowerboundOnD}
\label{lem:RelationOfDTOSubproblems}
Let $\sigma' \geq 1$ and
$\nu \in [0, 1]$ 
satisfy \eqref{eq:sigmaPrimeSafeDefinition} (that is $\sigma' \geq \sigma'_{min}$).
Then
$\forall \alphav, \hv  \in \R^n$, it holds that
\begin{equation}
\label{eq:asfdjalkfjlsaflasdfa}
D\left(\alphav+\nu \sum_{k=1}^K \hk \right)
\geq (1-\nu) D(\alphav) +
\nu \sum_{k=1}^K \Gks(\hk; \alphav),
\end{equation}
\end{lemma}

The following lemma gives a simple choice for the subproblem parameter $\sigma'$, which is trivial to calculate for all values of the aggregation parameter $\aggpar\in\R$, and safe in the sense of the desired condition \eqref{eq:sigmaPrimeSafeDefinition} above. 
Later we will show experimentally (Section \ref{sec:cocoa:experiments}) that the choice of this safe upper bound for $\sigma'$ only has a minimal effect on the overall performance of the algorithm. 

\begin{lemma}\label{lem:sigmaPrimeNotBad}
For any 
aggregation parameter $\aggpar\in[0,1]$, the choice of the subproblem parameter $\sigma' := \aggpar K$ is valid for \eqref{eq:sigmaPrimeSafeDefinition}, i.e.,
$
\aggpar K
\geq
\sigma'_{min}. 
$
\end{lemma}

\section{Main Results}
\label{sec:cocoa:result}

In this section we state the main theoretical results of this chapter. Before doing so, we elaborate on one of the most important aspects of the algorithmic framework: the \emph{quality of approximate local solutions}. 

\subsection{Quality of Local Solutions}

The notion of approximation quality provided by the local solvers is measured according to the following:
\begin{assumption}[Quality of local solution]
\label{ass:localImprovement}
Let $\Theta \in [0, 1)$ and $\alphav \in \R^n$ be fixed, and let $\hk^\star$ be the optimal solution of a local subproblem $\Gk(\,\cdot\,; \alphav)$.
We assume the local optimization procedure run on every node $k \in [K]$ in each iteration $t$ produces a (possibly random) output $\hk$ satisfying
\begin{equation}
\label{eq:localQualityOfImprovement}
\E{ \Gk(\hk^\star; \alphav) - \Gk(\hk; \alphav) } \leq \Theta \left[ \Gk(\hk^\star; \alphav) - \Gk({\bf 0}; \alphav) \right].
\end{equation}
\end{assumption}

The assumption specifies the (relative) accuracy $\Theta$ obtained on solving the local subproblem $\Gk$. Considering the two extreme examples, setting $\Theta = 0$ would require to find the exact maximum, while $\Theta = 1$ states that no improvement was achieved at all by the local solver. Intuitively, we would prefer $\Theta$ to be small, but spending many computational resources to drive $\Theta$ to $0$ can be excessive in practice, since~$\Gk$ is actually not the problem we are interested in solving \eqref{eq:dual}, but is the problem to be solved per communication round.
The best choice in practice will therefore be to choose $\Theta$ such that the local solver runs for a time comparable to the time it takes for a single communication round. This freedom of choice of $\Theta \in [0,1]$ is a crucial property of our proposed framework, allowing it to adapt to the full range of communication speeds on real world systems, ranging from supercomputers on one extreme to very slow communication rounds like MapReduce systems on the other extreme.

In Section \ref{sec:cocoa:experiments} we study {the} impact of different values of this parameter {on} the overall performance {of} solving \eqref{eq:dual}.

\subsection{Complexity Bounds}

Now we are ready to state the main results.
Theorem \ref{thm:mainResult} covers the case
when $\forall i${,} the loss function 
$\ell_i$ is $1/\gamma$ smooth{,}
and Theorem \ref{thm:mainResult:gcc}
covers the case
when 
$\ell_i$ is   $L$-Lipschitz continuous. For simplicity in the rates, we define the following two quantities:
$$
\forall k: \sigma_k \eqdef
 \max_{\vsubset{\alphav}{k} \in \R^n}
 \frac{\|\Xk \vsubset{\alphav}{k}\|^2}{
 \|\vsubset{\alphav}{k}\|^2}
 \qquad\mbox{and}\qquad
 \sigma \eqdef \sum _{k=1}^K 
\sigma_k  |\mathcal{P}_k|.
$$

\begin{theorem}[Smooth loss functions]
\label{thm:convergenceSmoothCase}
\label{thm:mainResult}
Assume the loss functions functions 
$\ell_i$ are $(1/\gamma)$-smooth $\forall i\in[n]$.
We define $\sigma_{\max} = 
\max_{k\in[K]} \sigma_k$. Then after $T$ iterations of Algorithm \ref{alg:cocoaPractical}, with  
$$
 T
    \geq 
\frac{1}
   {\aggpar
(1-\Theta)}
\frac
{\lambda\gamma n+
\sigma_{\max} \sigma'}
{ \lambda\gamma n }
    \log \frac1{\epsilon_\bD} , \vspace{-1mm}
$$
it holds that\vspace{-3mm}
$$\E{\bD(\alphav^\star)
  - \bD(\vc{\alphav}{T})}
   \leq \epsilon_\bD.$$
Furthermore, after $T$ iterations with\vspace{-1mm}
\begin{equation}
\label{afdsafdafsafdsafda}
 T 
    \geq 
\frac{1}
   {\aggpar
(1-\Theta)}
\frac
{\lambda\gamma n+
\sigma_{\max} \sigma'}
{ \lambda\gamma n }
    \log 
\left(
\frac{1}
   {\aggpar
(1-\Theta)}
\frac
{\lambda\gamma n+
\sigma_{\max} \sigma'}
{ \lambda\gamma n }
    \frac1{\epsilon_\gap}
    \right)  ,
\end{equation}
we have the expected duality gap
$$
\E{P( \wv(\vc{\alphav}{T})) - \bD(\vc{\alphav}{T})
} \leq \epsilon_\gap.
$$
\end{theorem}

\begin{theorem}[Lipschitz continuous loss functions]
\label{thm:convergenceNonsmooth}
 \label{thm:mainResult:gcc}
Consider Algorithm \ref{alg:cocoaPractical} with Assumption \ref{ass:localImprovement}. 
Let $\ell_i(\cdot)$ be $L$-Lipschitz continuous,
and $\epsilon_\gap$ $>0$ be the desired duality gap (and hence an upper-bound on primal sub-optimality).
Then after $T$ iterations, where
\begin{align}\label{eq:dualityRequirements}
T
&\geq
T_0 + 
\max\left\{\left\lceil \frac1{\aggpar (1-\Theta)}\right\rceil,\frac
{4L^2  \sigma   \sigma'}
{\lambda n^2 \epsilon_\gap
\aggpar (1-\Theta)}\right\},  
\\
T_0
&\geq t_0+
 \max\left\{0,
\frac{2}{ \aggpar (1-\Theta) }
\left(
\frac
{8L^2  \sigma   \sigma'}
{\lambda n^2 \epsilon_\gap}
-1
\right)
\right\},\notag
\\
t_0 &\geq 
  \max\left\{0, \left\lceil \frac1{\aggpar (1-\Theta)}
\log\left(
\frac{
 2\lambda n^2 (\bD(\alphav^\star )-\bD(\vc{\alphav}{0}))
  }{4 L^2 \sigma \sigma'}
  \right)
 \right\rceil \right\},\notag
\end{align}
we have that the expected duality gap satisfies
\[
\E{ P( \wv(\overline\alphav)) - \bD(\overline \alphav) } \leq \epsilon_\gap,
\]
at the averaged iterate
\begin{equation}\label{eq:averageOfAlphaDefinition}
\overline \alphav: = \frac1{T-T_0}\sum_{t=T_0+1}^{T-1} \vc{\alphav}{t}. 
\end{equation}

\end{theorem}

The most important observation regarding the above result is that we do not impose any assumption on the choice of the local solver, apart from {the} sufficient decrease condition on the local objective in Assumption~\ref{ass:localImprovement}.

Let us now comment on the leading terms of the complexity results.
The inverse dependence on $1 - \Theta$ suggests that it is worth pushing the rate of local accuracy $\Theta$ down to zero. However, when thinking about overall complexity, we have to bear in mind that achieving high accuracy on the local subproblems might be too expensive. The optimal choice would depend on the time we estimate a round of communication would take. In general, if communication is slow, it would be worth spending more time on solving local subproblems, but not so much if communication is relatively fast. We discussed this tradeoff in Section~\ref{sec:cocoa:problem}.

We achieve a significant speedup by replacing the slow averaging aggregation (as in~\cite{Jaggi:cocoa}) by more aggressive adding instead, that is $\aggpar = 1$ instead of  $\aggpar = 1/K$.
Note that the safe subproblem parameter for the averaging case ($\aggpar = 1/K$) is $\sigma' := 1$, while for adding ($\aggpar = 1$) it is given by $\sigma' := K$,
both proven in Lemma~\ref{lem:sigmaPrimeNotBad}.
The {speedup that results from more aggressive adding is reflected in the convergence rate} as shown above, when plugging in the actual parameter values $\aggpar$ and $\sigma'$ for the two cases, as we will illustrate more clearly in the next subsection.

\subsection{Discussion and Interpretations of Convergence Results}

As the above theorems suggest, it is not possible to meaningfully change the aggregation parameter $\aggpar$ in isolation. It comes naturally coupled with a particular subproblem.

In this section, we explain a simple way to be able to {set the aggregation parameter as} $\aggpar = 1$, that is to aggressively add up the updates from each machine. The motivation for this comes from a common practical setting. When solving the SVM dual (Hinge loss: $\ell_i(a) = \max\{0, y_i-a\}$), the optimization problem comes with ``box constraints'', i.e., for all $i \in \{ 1, \dots, n \}$, we have $\alpha_i \in [0, 1]$ (see Table \ref{tbl:differentLossFunctions}). The particular values of $\alpha_i$ being~$0$ or~$1$ have a particular interpretation in the context of original problem~\eqref{eq:primal}. If we used $\aggpar < 1$, we would never be able reach the upper boundary of any variable $\alpha_i$, when starting the algorithm with all-zeros $\alpha$.
This example illustrates some of the downsides of averaging {vs.} adding updates, coming from the fact that the step-size from using averaging (by being $1/K$ times shorter) can result in $1/K$ times slower convergence.

For the case of aggressive adding, the convergence {from Theorem \ref{thm:mainResult} becomes}:
\begin{corollary}[Smooth loss functions - adding]
\label{thm:mainResult:adding}
Let the assumptions of Theorem \ref{thm:mainResult} be satisfied.
If we run Algorithm~\ref{alg:cocoa} with $\aggpar = 1, \sigma'=K$
for  
\begin{equation}
\label{asfdafdafa}
 T 
   \overset{\eqref{afdsafdafsafdsafda}}{=} 
\frac{1}
   {
1-\Theta}
\frac
{\lambda\gamma n+
\sigma_{\max} K}
{ \lambda\gamma n }
    \log 
\left(
\frac{1}
   {1-\Theta}
\frac
{\lambda\gamma n+
\sigma_{\max} K}
{ \lambda\gamma n }
    \frac1{\epsilon_\gap}
    \right)
\end{equation}
iterations, we have 
$\E{
P( \wv(\vc{\alphav}{T})) - \bD(\vc{\alphav}{T})
} \leq \epsilon_\gap.$
\end{corollary}
On the other hand, if we would just average results (as proposed in \cite{Jaggi:cocoa}), we would obtain following corollary:
\begin{corollary}[Smooth loss functions - averaging] 
Let the assumptions of Theorem~\ref{thm:mainResult} be satisfied.
If we run Algorithm~\ref{alg:cocoa} with $\aggpar = 1/K, \sigma'=1$
for  
 \begin{equation}
 \label{asfdafdafa2}
 T 
    \overset{\eqref{afdsafdafsafdsafda}}{\geq }
\frac{1}
   {
1-\Theta}
\frac
{K\lambda\gamma n+
\sigma_{\max} K}
{ \lambda\gamma n }
    \log 
\left(
\frac{1}
   {1-\Theta}
\frac
{K \lambda\gamma n+
\sigma_{\max} K}
{ \lambda\gamma n }
    \frac1{\epsilon_\gap}
    \right)  
\end{equation} 
 iterations, we have 
$\E{
P( \wv(\vc{\alphav}{T})) - \bD(\vc{\alphav}{T})
} \leq \epsilon_\gap.$
\end{corollary}
Comparing the leading terms in 
Equations \eqref{asfdafdafa} and \eqref{asfdafdafa2}{,}
we see that 
the leading term for the $\nu=1$ choice is
$\mathcal{O}(\lambda\gamma n+
\sigma_{\max} K)$,
which is always better than for the $\nu=1/K$ case, when the leading term is 
{$\mathcal{O}(K\lambda\gamma n+
\sigma_{\max} K)$}.
This strongly suggests that adding in Framework \ref{alg:cocoaPractical} is preferable, especially when 
$\lambda \gamma n \gg \sigma_{\max} $.
 
An {analogous improvement (by a factor on the order of $K$)} follows for the case of the sub-linear convergence rate for general Lipschitz loss functions, as shown in Theorem~\ref{thm:mainResult:gcc}.

Note that the differences in the convergence rate are bigger for relatively big values of the regularizer $\lambda$. When the regularizer is $\mathcal{O}(1 / n)$, the difference is negligible. This behavior is also present in practice, as we will {illustrate} in Section~\ref{sec:cocoa:experiments}.

\section{Discussion and Related Work}
\label{sec:cocoa:relatedWork}

In this section, we review a number of methods designed to solve optimization problems of the form of interest here, which are typically referred to as regularized empirical risk minimization (ERM) {problems} in the machine learning literature.
{This problem class~ \eqref{eq:primal}, which is formally described in Section~\ref{sec:cocoa:problemformulation},} underlies many prominent methods {in} supervised machine learning.

\paragraph*{\it Single-Machine Solvers.}
Stochastic Gradient Descent (SGD) is the simplest stochastic method one can use to {solve \eqref{eq:primal}}, and dates back to the work of Robbins and Monro \cite{RM1951}. We refer the reader to \cite{moulines2011non, NeedellWard2015, nemirovski2009robust, bottou2012stochastic} for {a} recent theoretical and practical assessment of SGD. Generally speaking, the method is extremely easy to implement, and converges to modest accuracies very quickly, which is often satisfactory in applications in machine learning. On the other hand, {the method can sometimes be rather cumbersome because {}it can be difficult to tune {its} hyperparameters, and it can be} impractical if higher solution accuracy is needed.

The current state of the art for empirical loss minimization with strongly convex regularizers is randomized coordinate ascent on the dual objective---Stochastic Dual Coordinate Ascent (SDCA) \cite{SDCA}. In contrast to primal SGD methods, the SDCA algorithm family is often preferred as it is free of learning-rate parameters, and has faster (geometric) convergence guarantees.  This algorithm and its variants are increasingly used in practice \cite{Wright:2015bn,ShalevShwartz:2014dy}. On the other hand, primal-only methods apply to a larger problem class, not only of form \eqref{eq:primal} that enables formation of dual problem \eqref{eq:dual} as considered here.

Another class of algorithms gaining attention in recent very few years are `variance reduced' modifications of the original SGD algorithm. 
They are applied directly to the primal problem \eqref{eq:primal}, but unlike SGD, have {the} property that {the} variance of estimates of the gradients tend to zero as they approach {the} optimal solution. 
Algorithms such as SAG \cite{SAGjournal2013}, SAGA \cite{saga} and others \cite{shalev2015sdcaWODual, Defazio:2014vx} come at the cost of extra memory requirements---they have to store a gradient for each training example. 
This can be addressed efficiently in the case of generalized linear models, but prohibits its use in more complicated models such as in deep learning. 
On the other hand, Stochastic Variance Reduced Gradient (SVRG) and its variants \cite{SVRG, S2GD, proxSVRG, konecny2015mini, nitanda2014stochastic} are often interpreted as `memory-free' methods with variance reduction. 
However, these methods need to compute the full gradient occasionally to drive the variance reduction, which requires a full pass through the data and is an operation one generally tries to avoid. 
This and several other practical issues have been recently addressed in~\cite{practicalSVRG}.
Finally, another class of extensions to SGD are stochastic quasi-Newton methods~\cite{bordes2009sgd, byrd2014stochastic}. Despite their clear potential, a lack of theoretical understanding and complicated implementation issues compared to those above may still limit their adoption in the wider community. A stochastic dual Newton ascent (SDNA) method was proposed and analyzed in~\cite{SDNA}. However, the method needs to modified substantially before it can be implemented in a  distributed environment.

\paragraph*{\it SGD-based Algorithms.}
For the empirical loss minimization problems of interest, stochastic subgradient descent (SGD) based methods are well-established.
Several distributed variants of SGD have been proposed, many of which build on the idea of a parameter server \cite{Niu:2011wo, richtarik2013distributed,  Duchi:2013te}.
Despite their simplicity and accessibility in terms of implementation, the downside of this approach is that the amount of required communication is equal to the amount of data read locally, {since one data point is accessed per machine per round} (e.g., mini-batch SGD with a batch size of 1 per worker). These variants are in practice not competitive with the more communication-efficient methods considered in this work, which allow more local updates per communication round.

\paragraph*{\it One-Shot Communication Schemes.}
At the other extreme, there are distributed methods using only a single round of communication, such as \cite{Zhang:2013wq, zinkevich2010parallelized, Mann:2009tr,McWilliams:2014tl,Heinze:2016tu}.
These methods require additional assumptions on the partitioning of the data, which are usually not satisfied in practice if the data are distributed ``as is'', i.e., if we do not have the opportunity to distribute the data in a specific way beforehand. 
Furthermore, some cannot guarantee convergence rates beyond what could be achieved if we ignored data residing on all but a single computer, as shown in \cite{DANE}.
Additional relevant lower bounds on the minimum number of communication rounds necessary for a given approximation quality are presented in \cite{Balcan:2012tc,Arjevani:2015vka}.

\paragraph*{\it Mini-Batch Methods.} Mini-batch methods (which instead of just one data-example use updates from several examples per iteration) are more flexible and lie within these two communication vs. computation extremes. However,
mini-batch versions of both SGD and coordinate descent (CD) \cite{RT:PCDM, richtarik2013distributed,ShalevShwartz:2014dy, marecek2014distributed, Yang:2013vl, Tappenden:2015vh, nsync, ALPHA, ESO, QUARTZ, csiba2015primal, csiba2016importanceminibatch} suffer from their convergence rate degrading towards the rate of batch gradient descent as the size of the mini-batch is increased. 
This follows because mini-batch updates are made based on the outdated previous parameter vector $\wv$, in contrast to methods that allow immediate local updates like \cocoa.

Another disadvantage of mini-batch methods is that the aggregation parameter is harder to tune, as it can lie anywhere in the order of mini-batch size. The optimal choice is often either unknown, or difficult to compute.
In the \cocoa setting, the parameter lies in the typically much smaller range given by $K$. In this work the aggregation parameter is further simplified and can be simply set to $1$, i.e., adding updates, which is achieved by formulating a more conservative local problem as described in Section~\ref{sec:cocoa:subproblem}.

\paragraph*{\it Distributed Batch Solvers.}
With traditional batch gradient solvers not being competitive for the problem class  \eqref{eq:primal}, improved batch methods have also received much research attention recently, in the single machine case as well as in the distributed setting.
In distributed environments, {popular methods include} the alternating direction method of multipliers (ADMM) \cite{Boyd:2010bw}  
as well as quasi-Newton methods such as L-BFGS, which can be attractive because of their relatively low communication requirements. Namely, communication is in the order of a constant number of vectors (the batch gradient information) per full pass through the data.

ADMM also comes with an additional penalty parameter balancing between the equality constraint on the primal variable vector $\wv$ and the original optimization objective~\cite{Boyd:2010bw}, which is typically hard to tune in many applications. Nevertheless, the method has been used for distributed SVM training in, e.g., \cite{Forero:2010vv}. The known convergence rates for ADMM are weaker than the more problem-tailored methods mentioned we study here, and the choice of the penalty parameter is often unclear in practice. 

Standard ADMM and quasi-Newton methods  do not allow a gradual trade-off between communication and computation available here. An exception is the approach of Zhang, Lee and Shin \cite{Zhang:2012usa}, which is similar to our approach in spirit, albeit based on ADMM, in that they allow for the subproblems to be solved inexactly. However, this work focuses on L2-regularized problems and a few selected loss functions, and offers no complexity results.

Interestingly, our proposed \cocoap { framework}---despite {being} aimed at  cheap stochastic local solvers---does have similarities to block-wise variants of batch proximal methods{. In particular, the} purpose of our subproblems as defined in \eqref{eq:subproblem:sigma1} is to form a data-dependent block-separable quadratic approximation to the smooth part of the original (dual) objective~\eqref{eq:dual}, while leaving the non-smooth part $R$ intact (recall that $R(\alphav)$ was defined to collect the~$\ell^*_i$ functions, and is separable over the coordinate blocks).
Now if hypothetically each of our regularized quadratic subproblems \eqref{eq:subproblem:sigma1} were to be minimized exactly, the resulting steps could be interpreted as block-wise proximal Newton-type steps on each coordinate block~$k$ of the dual \eqref{eq:dual}, where the Newton-subproblem is modified to also contain the proximal part $R$. 
This connection only holds for the special case of adding ($\aggpar=1$), and would correspond to a carefully adapted step-size in the block-wise Newton case.

One of the main crucial differences of our proposed \cocoap framework compared to all known batch proximal methods (no matter if block-wise or not) is that the latter do require high accuracy subproblem solutions, and do not allow arbitrary solvers of weak accuracy~$\Theta$ such as we do here, see also the next paragraph.
Distributed Newton methods have been analyzed theoretically only when the subproblems are solved to high precision, see e.g. \cite{DANE}. 
This makes the local solvers very expensive and the convergence rates less general than in our framework (which allows weak local solvers). 
Furthermore, the analysis of \cite{DANE} requires additional strong assumptions on the data partitioning, such that the local Hessian approximations are consistent between the machines.

\paragraph*{\it Distributed Methods Allowing Local Optimization.}
Developing distributed optimization methods that allow for arbitrary weak local optimizers requires carefully devising data-local subproblems to be solved after each communication round.

By making use of the primal-dual structure in the line of work of \cite{Yu:2012fp,Pechyony:2011wi,Yang:2013vl,Yang:2013ui,Lee:2015vra}, the \cocoa and \cocoap frameworks proposed here are the first to allow the use of any local solver---of weak local approximation quality---in each round.
Furthermore, the approach here also allows more control over the aggregation of updates between machines. 
The practical variant of the DisDCA {algorithm} of \cite{Yang:2013vl}, called DisDCA-p, also allows additive updates but is restricted to coordinate decent  {(CD)} being the local solver, and was  {initially} proposed without convergence guarantees. 
 {The work of \cite{Yang:2013ui} has provided the first theoretical convergence analysis for an ideal case, when the distributed data parts are all orthogonal to each other{, which is} an unrealistic setting in practice.}
DisDCA-p can be recovered as a special case of the \cocoap framework when using  {CD} as a local solver, if $|\mathcal{P}_k| = n/K${,} and when using  {the conservative bound} $\sigma':=K${;} see  
also \cite{Lee:2015vra,Ma:2015ti}.
The convergence theory presented here therefore also covers that method, {and extends it to arbitrary local solvers.}

{Since the first version of this work, Accelerated Inexact Dane (AIDE) \cite{reddi2016aide}---a method based on related set of ideas but applied to the primal problem---was developed. Like \cocoap, AIDE promotes an efficient balance between communication and computation costs in the sense of \eqref{eq:runtime:cocoa}.}

\paragraph*{\it Inexact Block Coordinate Descent.} 
Our framework is related, but not identical, to running an {\em inexact} version of  block coordinate ascent, applied to all block{s} in parallel, and to the dual problem. {From this perspective,} the level of inexactness is controlled by the parameter $\Theta$ through the use of a (possibly randomized) iterative ``local'' solver applied to the {local subproblems}. For previous work on {\em randomized} block coordinate descent we refer to {the reader to~\cite{TRG-inexact} and \cite{DQA}}.

\section{Numerical Experiments}
\label{sec:cocoa:experiments}

In this section we explore numerous aspects of our distributed framework and demonstrate  its competitive performance in practice.
Section~\ref{sec:cocoa:LocalSolverExps} first explores the impact of the local solver on overall performance, by comparing examples of various local solvers that can be used in the framework (the improved \cocoap framework as shown in Algorithms \ref{alg:cocoa} and \ref{alg:cocoaPractical}) as well as testing the effect of approximate solution quality. The results indicate that the choice of local solver can have a significant impact on overall performance. 
In Sections~\ref{sec:cocoa:adingVsAveraging} and \ref{sec:cocoa:subproblemParamExps} we further explore framework parameters, looking at the impact of the aggregation parameter $\nu$ and the subproblem parameter $\sigma'$, respectively. Finally, Section~\ref{sec:cocoa:hugeDatasetExp} demonstrates {the} competitive practical performance of the overall framework on a large 280GB distributed dataset.

We conduct experiments on three datasets of moderate and large size, namely {\emph{rcv1\_test}}, \emph{epsilon} and \emph{splice-site.t}\footnote{The datasets are available at \url{http://www.csie.ntu.edu.tw/~cjlin/libsvmtools/datasets/}.}. The details of these datasets are listed in Table~\ref{tab:datasets}.
\begin{table}[h]
\centering
{
      \begin{tabular}{crrr}
      \toprule
    Dataset & \multicolumn{1}{c}{$n$} &
    \multicolumn{1}{c}{$d$} & 
    \multicolumn{1}{c}{size (GB)} \\
    \midrule 
	rcv1\_test & 677,399 &
	  47,236 & 1.2 \\
		epsilon & 400,000 &
	  2,000 & 3.1 \\	
	  splice-site.t & 4,627,840 &
	  11,725,480 & 273.4
	\\  \bottomrule  
\end{tabular}
}    
\caption{Datasets used for numerical experiments.}   
\label{tab:datasets}
\end{table}

For solving subproblems, we compare numerous local solver methods, as listed in Table~\ref{tbl:othersolvers}. {We use the} Euclidean norm as {the} regularizer $g(x) = \|x\|^2$ for all the experiments.  All the algorithms are implemented in C\texttt{++} with MPI, and experiments are run on a cluster of 4 Amazon EC2 m3.xlarge instances. Our {open-source} code is available online at: \url{https://github.com/optml/CoCoA}.
\begin{table}[H]
\centering
{
\begin{tabular}{ll}
\toprule
CD  & Coordinate Descent \cite{RichtarikTakacIteration}\\

APPROX & Accelerated, Parallel and Proximal Coordinate Descent \cite{APPROX} \\

GD & Gradient Descent with Backtracking Line Search \cite{NocedalWrightBook}\\

CG & Conjugate Gradient Method \cite{CG}\\

L-BFGS &Quasi-Newton with Limited-Memory BFGS Updating \cite{byrd1995limited}\\

BB &  Barzilai-Borwein Gradient Method   \cite{barzilai1988two}\\

FISTA & Fast Iterative Shrinkage-Thresholding Algorithm \cite{fista}\\
\bottomrule
\end{tabular}
}
\caption{Local solvers used in numerical experiments.}
\label{tbl:othersolvers}
\end{table}

\subsection{Exploration of Local Solvers within the Framework}
\label{sec:cocoa:LocalSolverExps}
In this section we compare the performance of our framework for various local solvers and various choices of inner iterations performed by a given local solver, resulting in different local accuracy measures $\Theta$. 
For simplicity, we choose the subproblem parameter $\sigma' := \nu K$ (see Lemma \ref{lem:sigmaPrimeNotBad}) as a simple obtainable and theoretically {safe value.}

\subsubsection{Comparison of Different Local Solvers}

Here we compare the performance of the {seven} local solvers listed in Table~\ref{tbl:othersolvers}. {We} show results for {the} quadratic loss function $\ell_i(a) = \frac12 (a-y_i)^2$
with three different values of the regularization parameter, $\lambda$=$10^{-3}$, $10^{-4}$, and $10^{-5}$, {and} $g(\cdot)$ being the default Euclidean squared norm regularizer: {$g(\cdot)=\tfrac{1}{2}\|\cdot\|^2$.} The dataset is {rcv1\_test} and we ran the {\cocoap} framework for a maximum of $T:=100$ communication rounds. We set $\aggpar=1$ (adding) and choose $H$ which gave the best performance in CPU time (see Table \ref{tbl:optH}) for each solver.

\begin{table}[H]
\centering
{
\begin{tabular}{c | c c c c c c c}
\toprule
Local Solver &{CD}  & APPROX & GD& CG& L-BFGS& BB& FISTA\\
$H $ &40,000 & 40,000 & 20&  5 & 10 & 15 &20 \\
\bottomrule
\end{tabular}
}
\caption{Optimal $H$ for different local solvers {and the} rcv1\_test dataset. }
\label{tbl:optH}
\end{table}

From Figure~\ref{fig:dffsolvers}, we find that {if a high-enough accuracy solution is needed,} the  coordinate descent (CD) local solver always outperforms the other solvers. {However, when a low accuracy solution is sufficient, as is often the case in machine learning applications, and if the regularization parameter is not too small, then L-BFGS performs best. The local subproblems  arising with the rcv1\_test dataset are reasonably well conditioned. If more ill-conditioning was present, however, we would expect the APPROX local solver to do better than CD. This is because  this method is an {\em accelerated} variant of CD. In summary, randomized methods, such as CD and APPROX, and quasi-Newton methods (L-BFGS), perform best on this dataset.} 

{Based on the above observations, it seems reasonable to expect that a method combining the power of both of these successful approaches---randomization and second-order information---would perform even better. One might therefore want to look at local solvers based on ideas appearing in  \cite{SDNA} or \cite{SBFGS}.  }

{Note that it is not the goal of this work to decide on what the best local solver is. Our goals are quite the opposite, we provide a framework which allows the incorporation of {\em any} local solver. This choice might depend on which solvers are readily available to the practitioner/company. It will also depend on the conditioning of the local subproblems, their size, and other similar considerations. Future research will undoubtedly lead to the development of new and better local solvers which can be incorporated within \cocoap.}

Finally, note that some of the solvers cannot guarantee strict decrease of the duality gap, and sometimes this fluctuation can be very dramatic.

\begin{figure}[H]
\centering
\includegraphics[scale=.19]{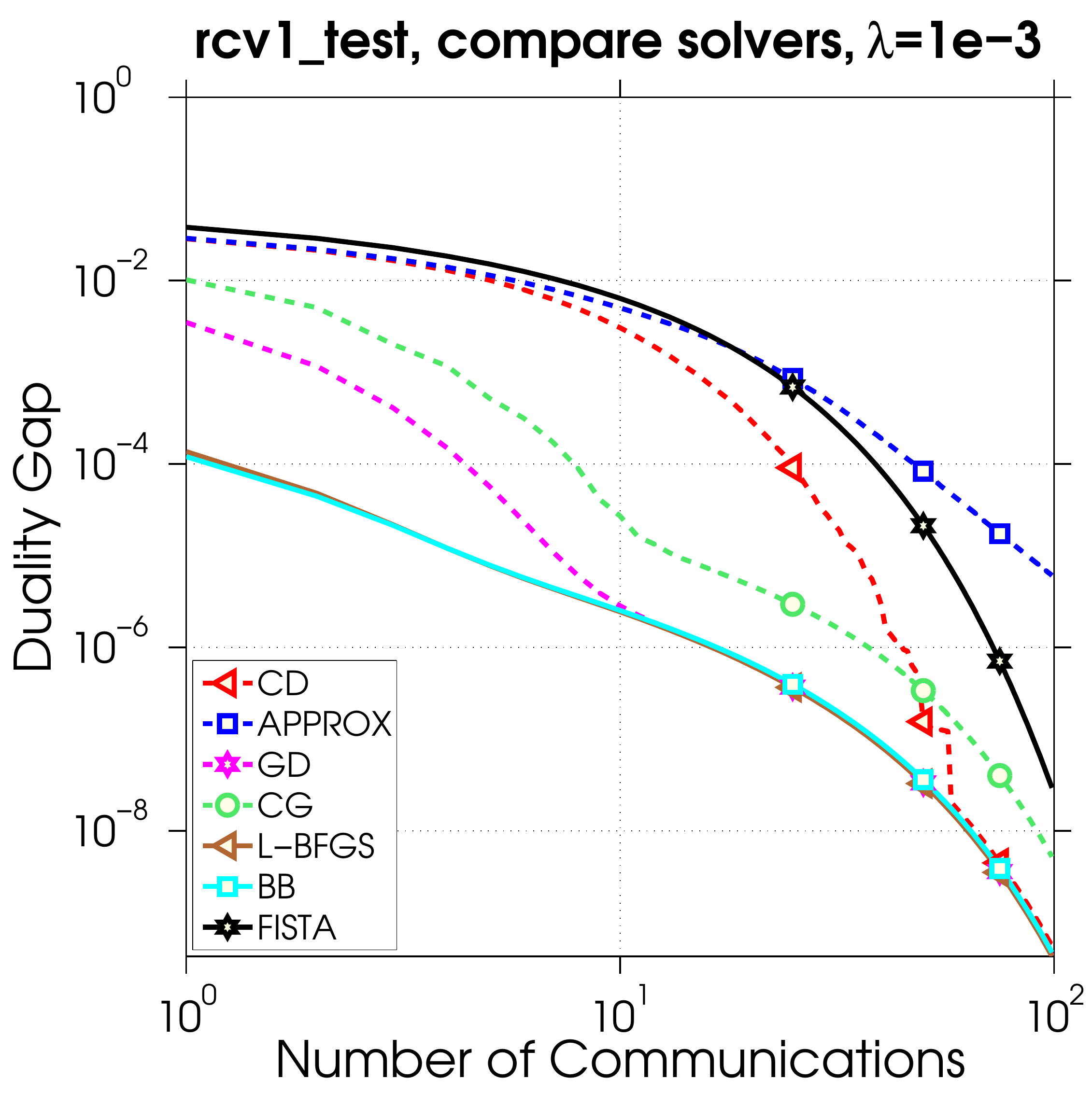}
\includegraphics[scale=.19]{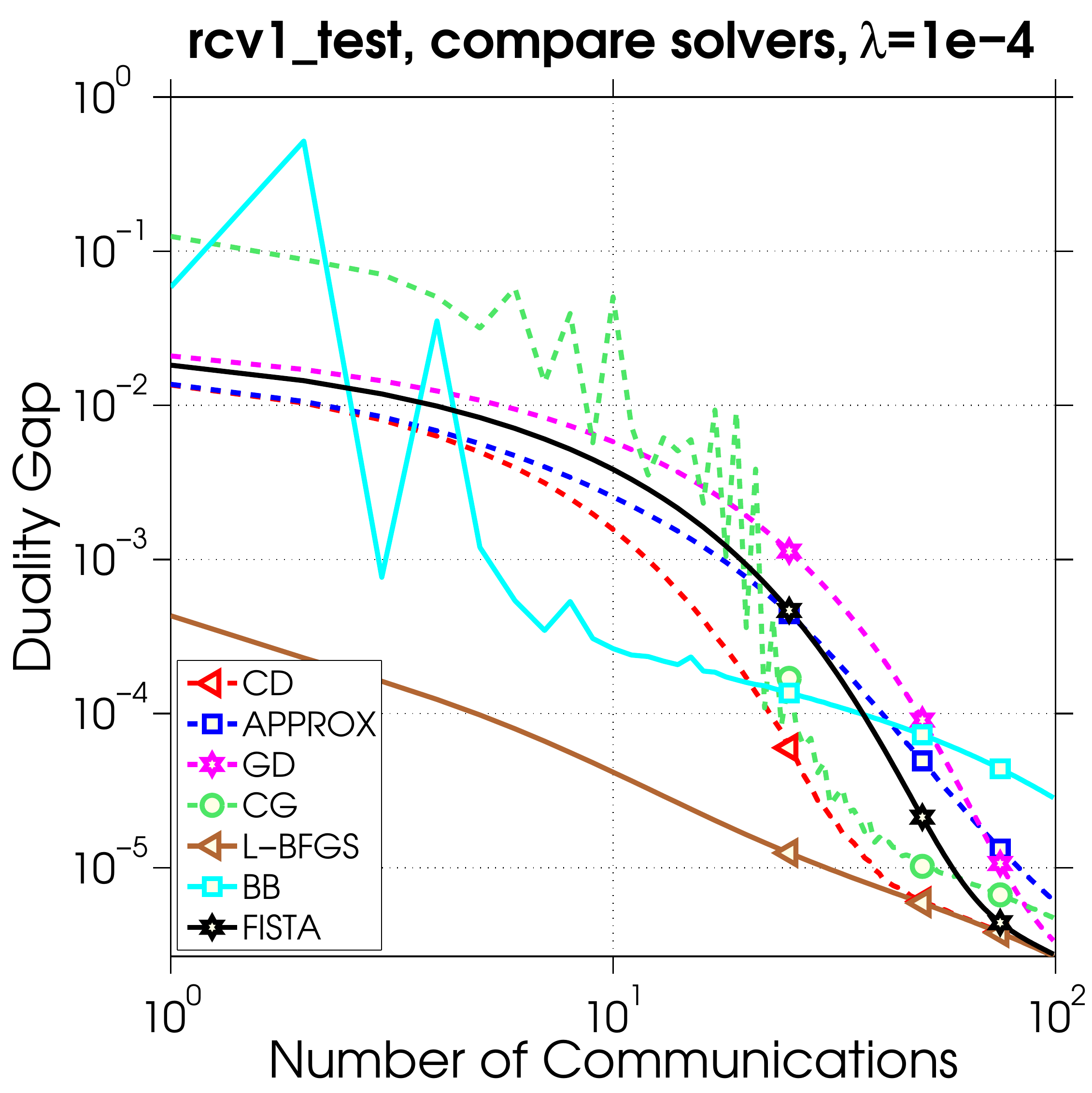}
\includegraphics[scale=.19]{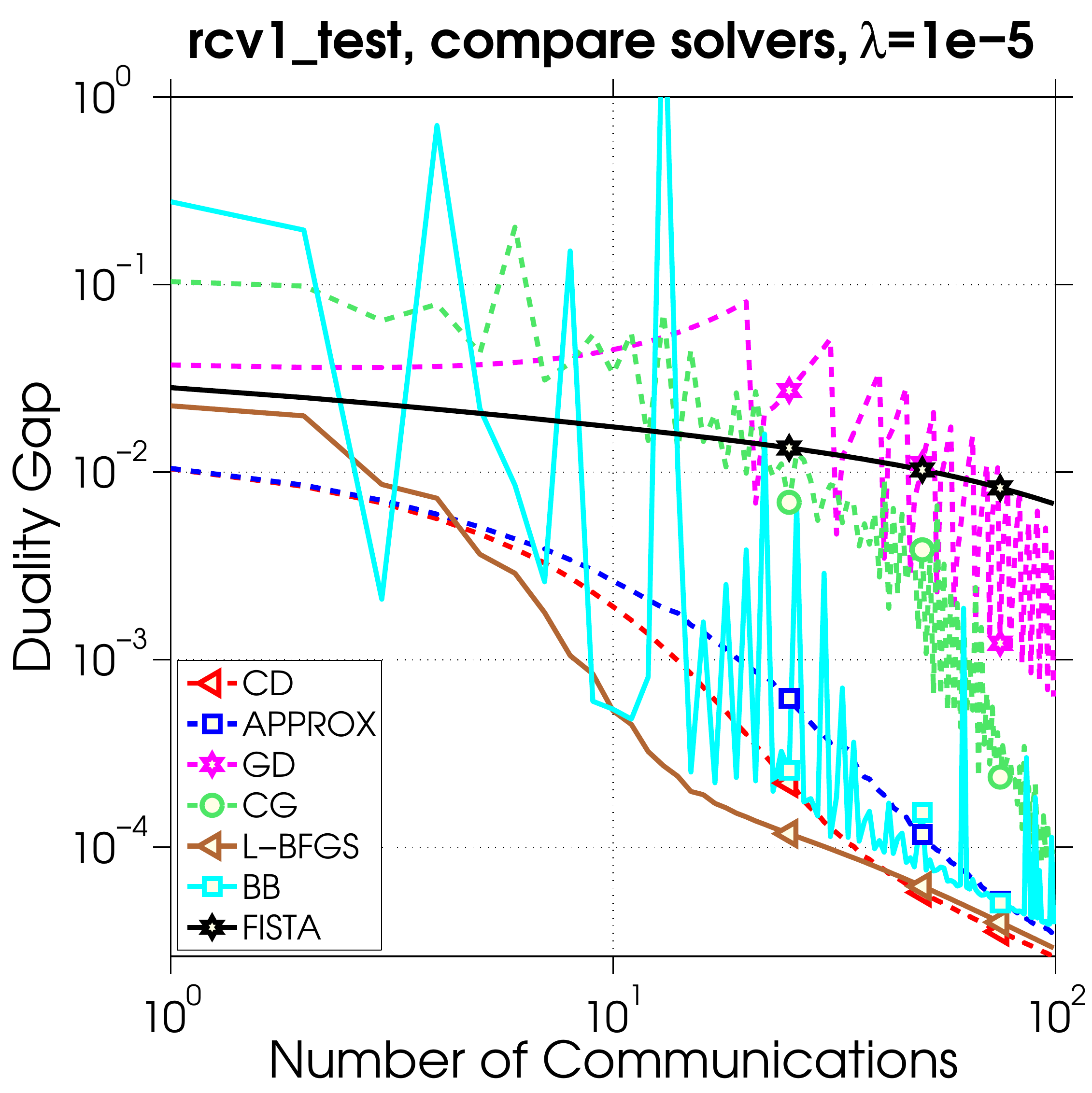}

\includegraphics[scale=.19]{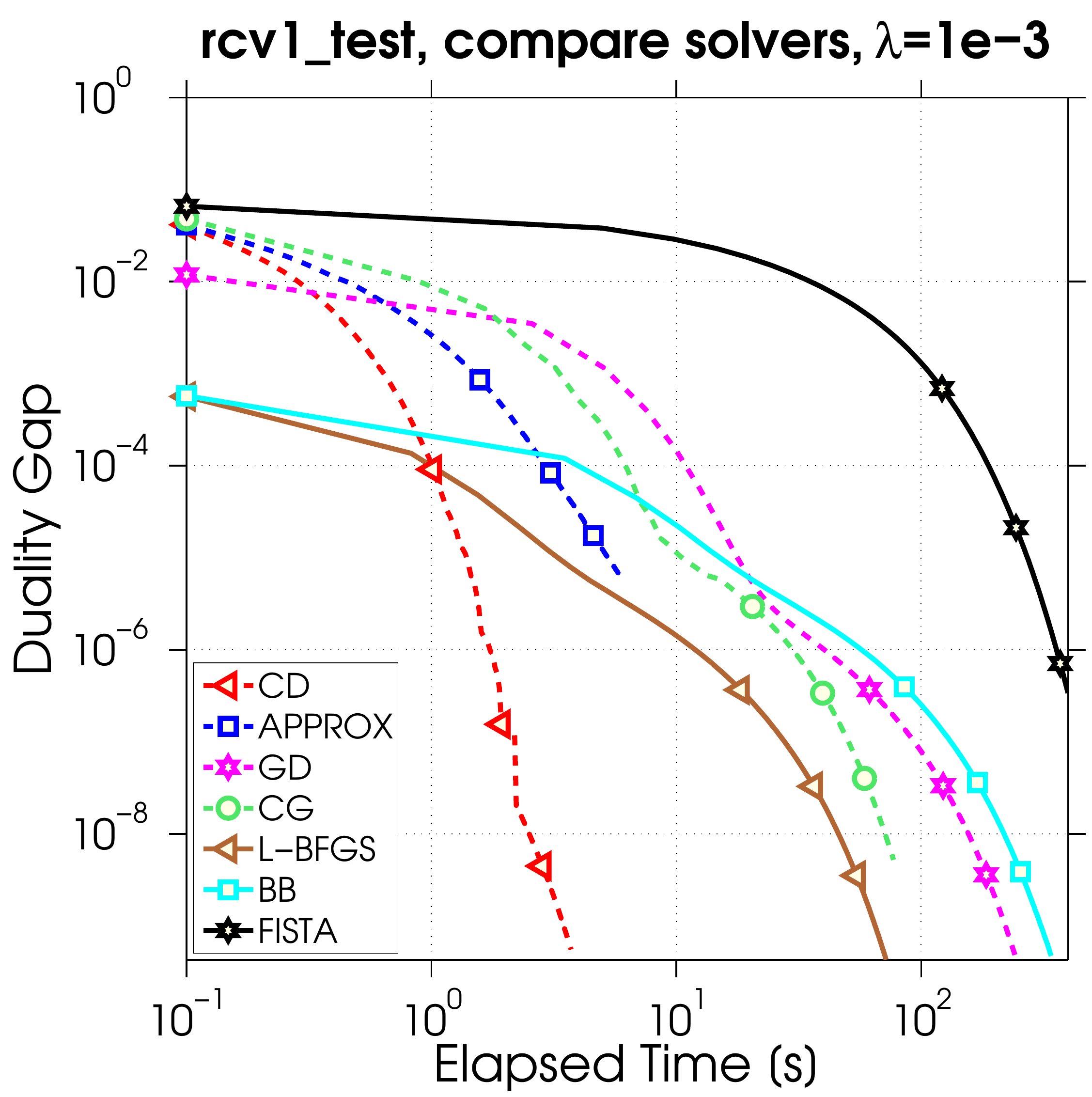}
\includegraphics[scale=.19]{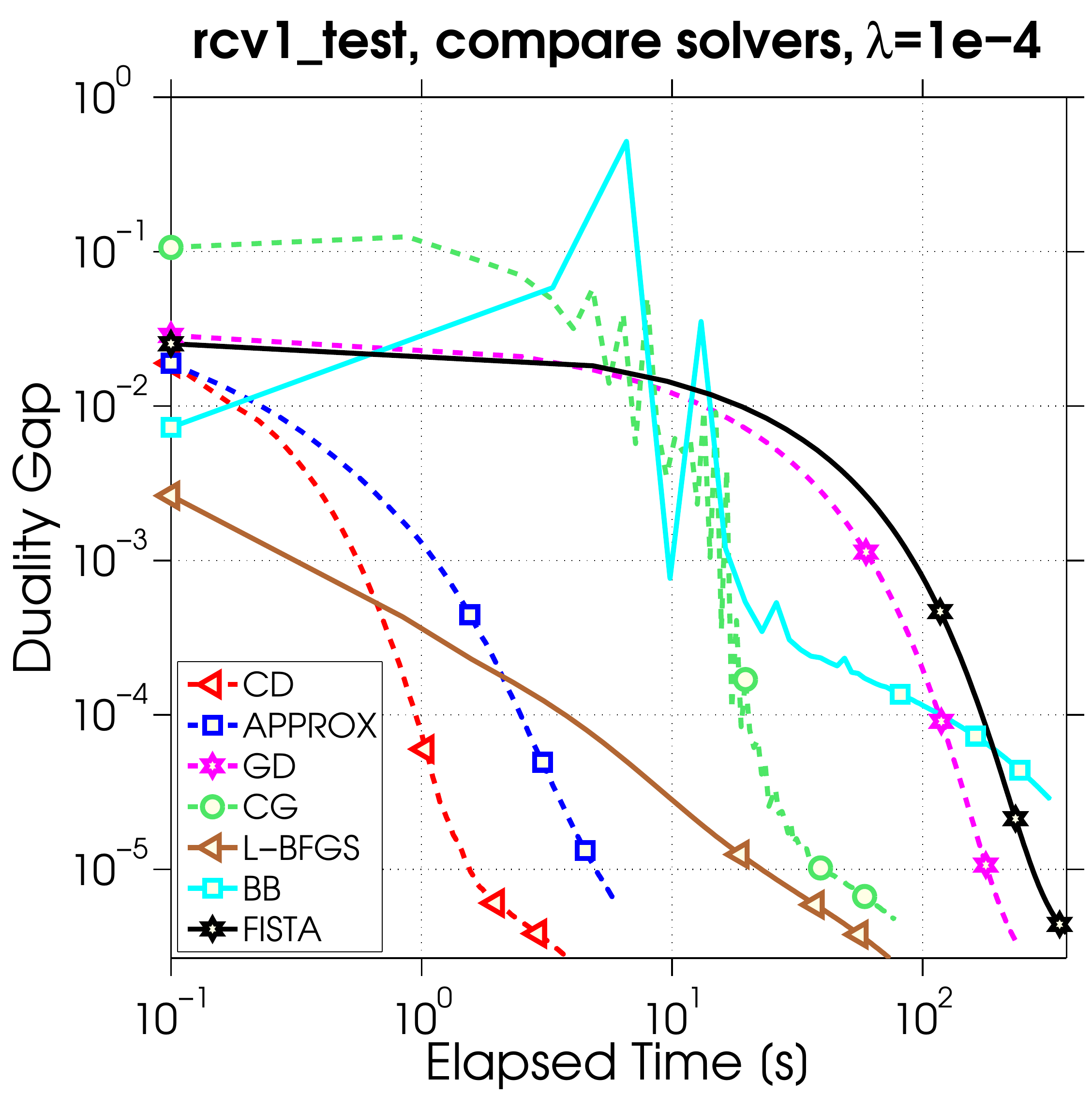}
\includegraphics[scale=.19]{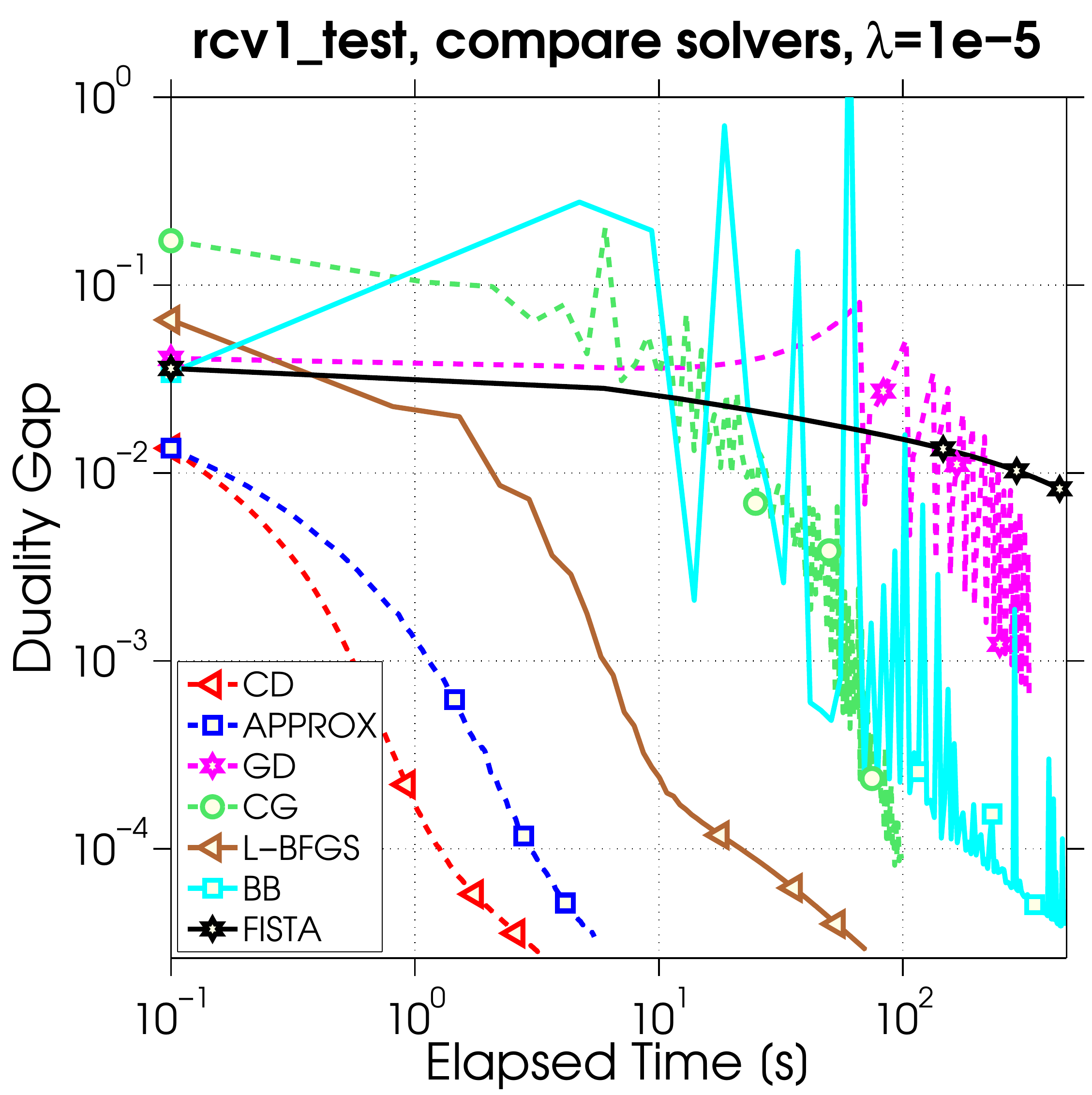}
\caption{Performance of 7 local solvers on rcv1\_test dataset for three values of the regularization parameter.} 
\label{fig:dffsolvers}
\end{figure}
 
\subsubsection{Effect of the Quality of Local Solver Solutions on Overall Performance}

Here we discuss how the quality of subproblem solutions affects the overall performance of Algorithm~\ref{alg:cocoaPractical}. In order to do so, we denote $H$ as the number of iterations the local solver is run for, within each communication round of the framework.
We choose various values for $H$ {for the two local solvers that had the best performance in general},  {CD}~\cite{RichtarikTakacIteration,SDCA} and L-BFGS~\cite{byrd1995limited}. For  {CD}, $H$ represents the number of local iterations performed on the subproblem. For L-BFGS, $H$ not only means the number of iterations, but also stands for the size of past information used to approximate the Hessian (i.e., the size of limited memory).  

Looking at Figures~\ref{fig:dffsollbfgs} and \ref{fig:dffsollbfgsxxx}, we see that for both of these local {solvers} and all values of~$\lambda$, increasing $H$ will lead to less iterations of Algorithm~\ref{alg:cocoaPractical}. Of course, increasing $H$ comes at the cost of the time spent on local solvers increasing. Hence, a larger value of $H$ is not always the optimal choice with respect to total elapsed time. For example, for the {rcv1\_test} dataset, when choosing  {CD} to solve the subproblems, choosing $H$ to be $40,000$ uses less time and provides faster convergence. When using L-BFGS, $H = 10$ seems to be the best choice.

\begin{figure}[H]
\centering
\includegraphics[scale=.19]{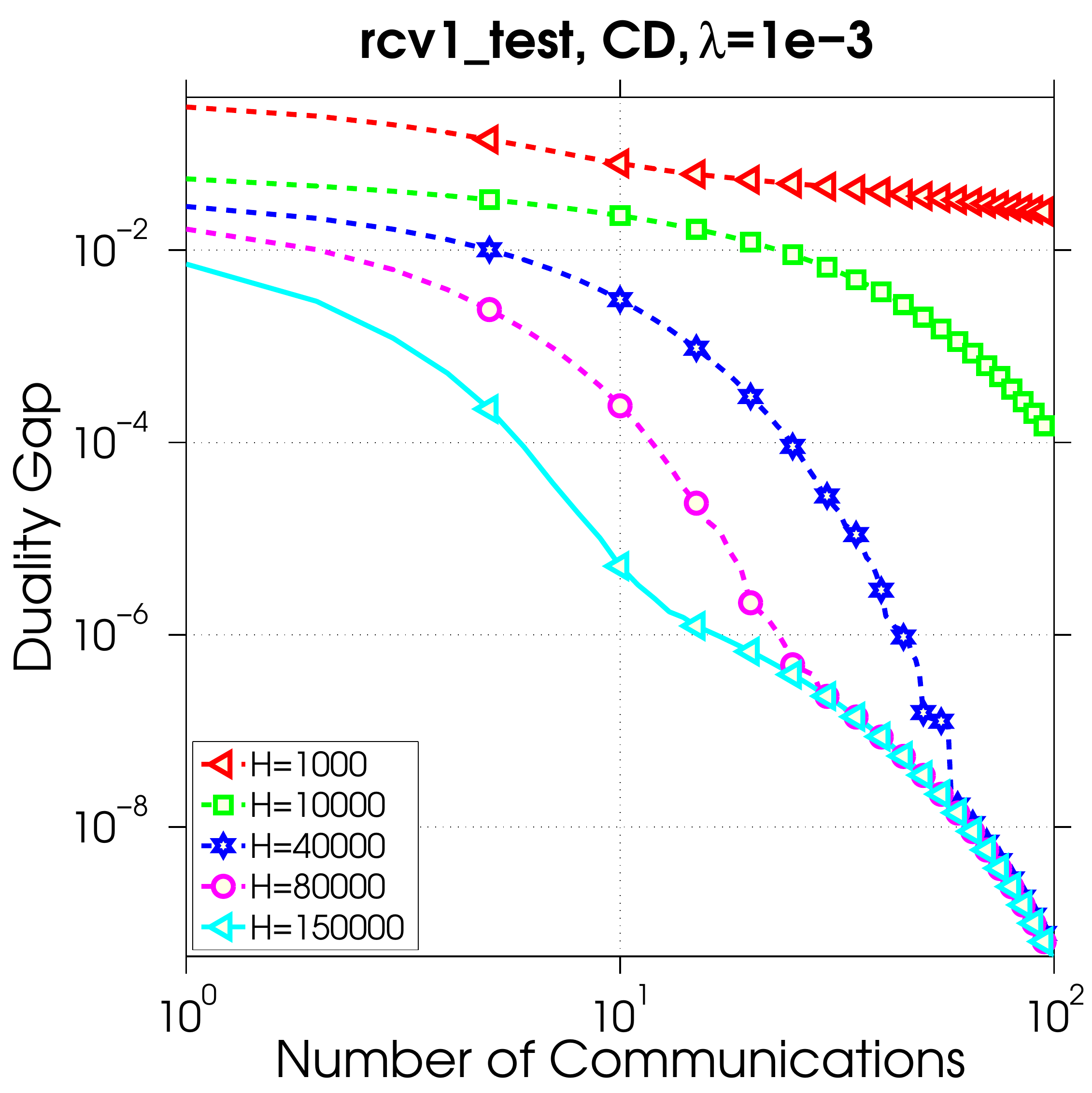}
\includegraphics[scale=.19]{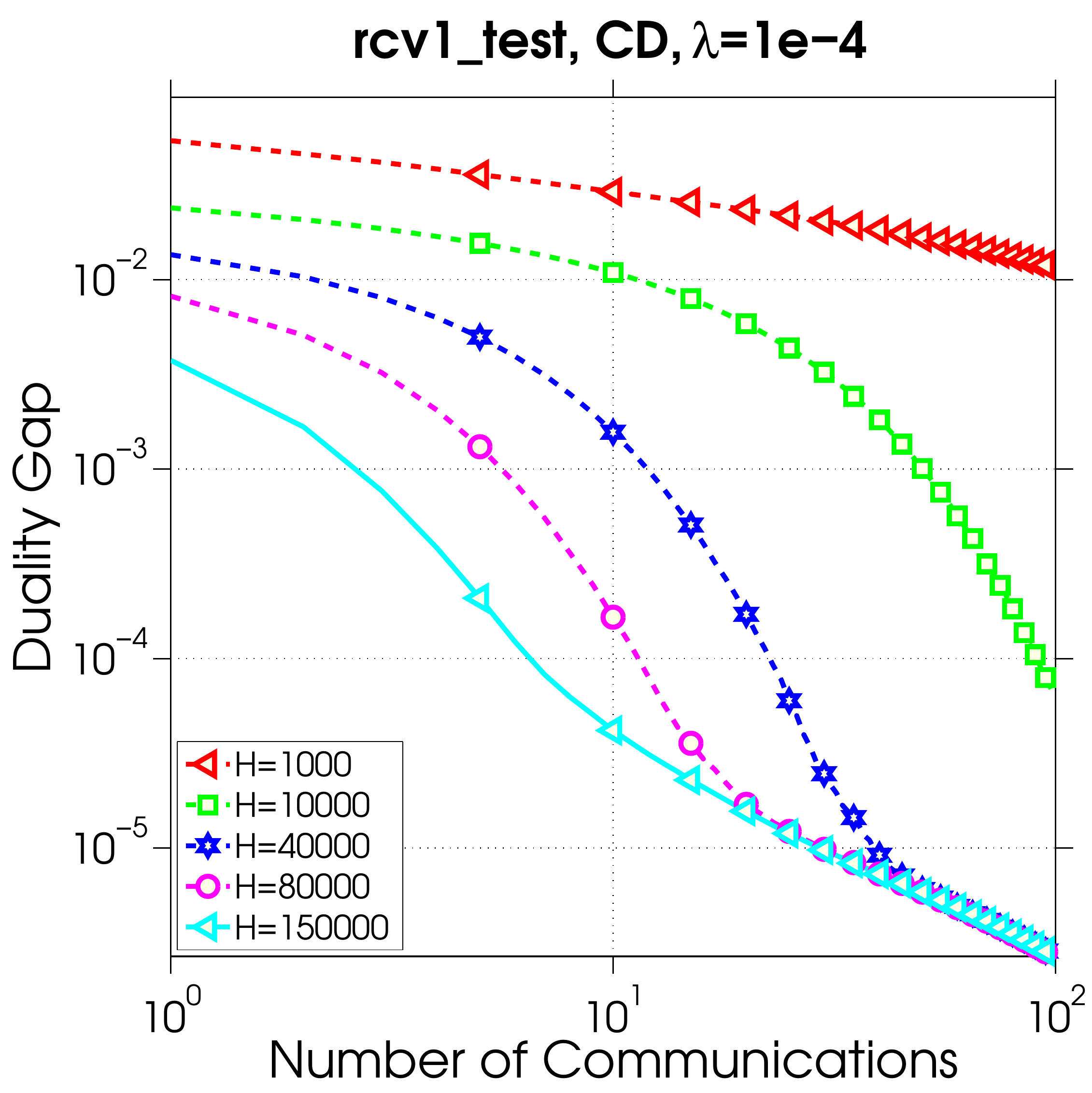}
\includegraphics[scale=.19]{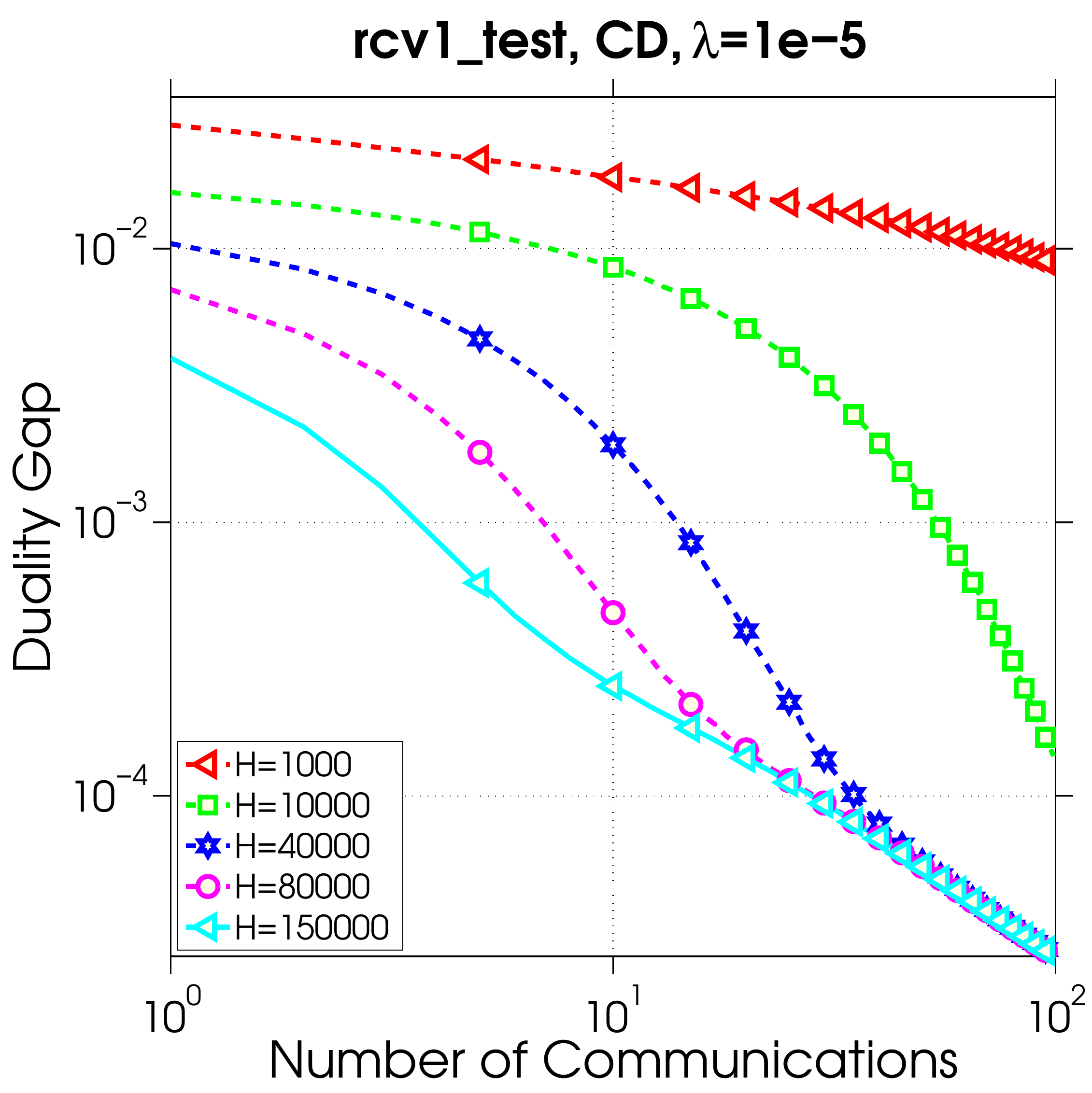}

\includegraphics[scale=.19]{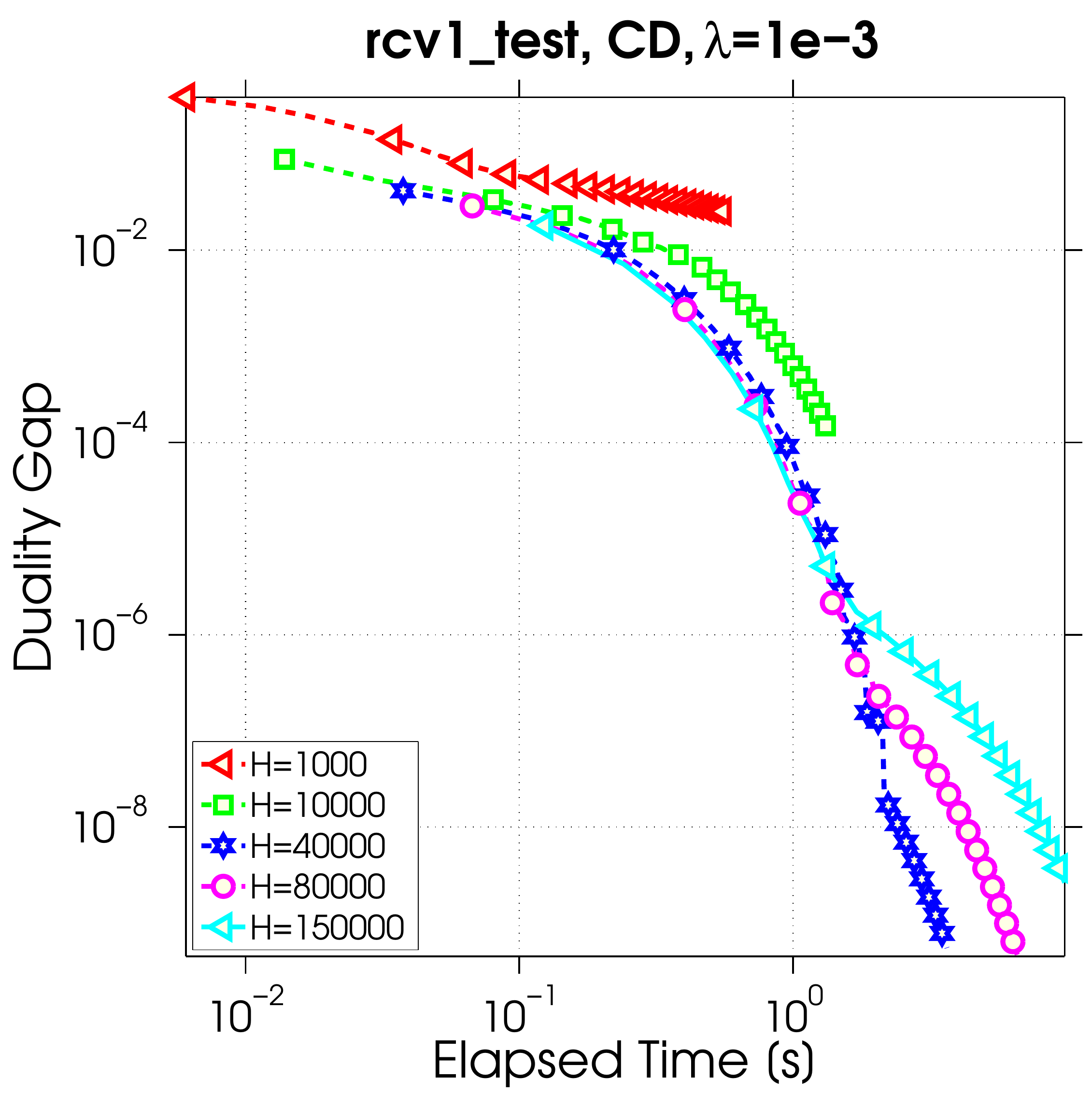}
\includegraphics[scale=.19]{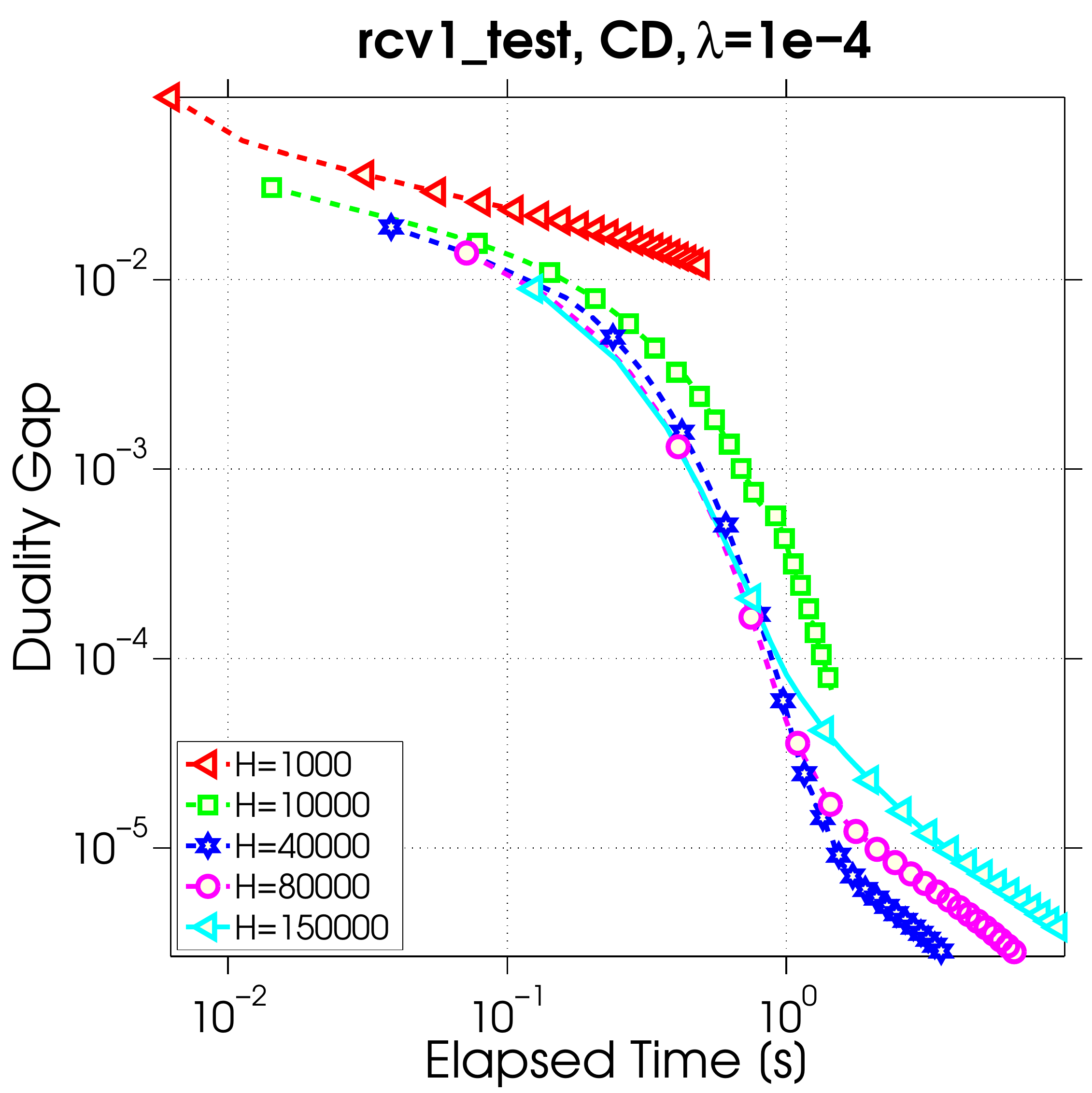}
\includegraphics[scale=.19]{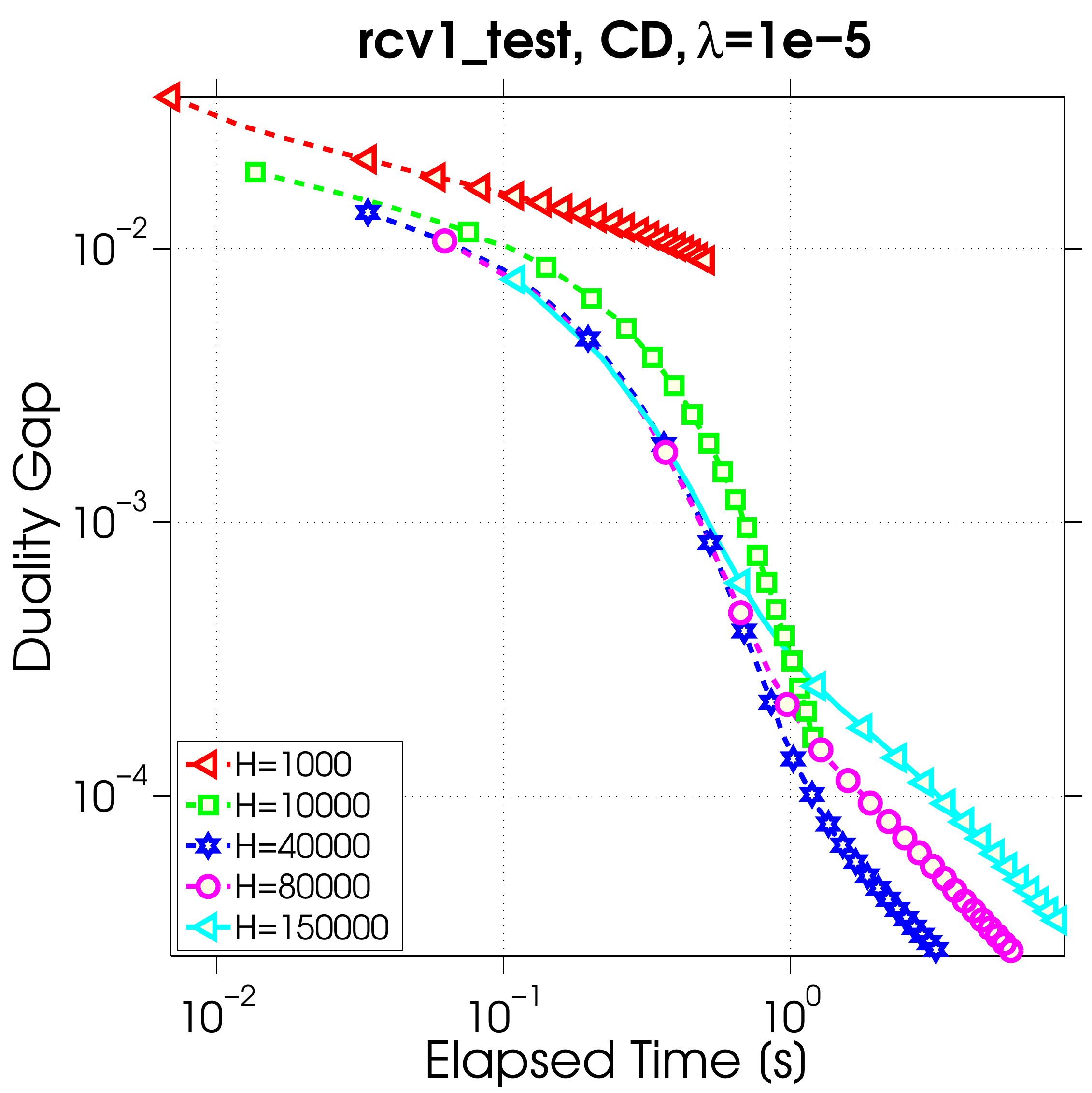}
\caption{Varying the number of iterations of {CD}  as a local solver.} 
\label{fig:dffsollbfgs}
\end{figure}

\begin{figure}[H]
\centering
\includegraphics[scale=.19]{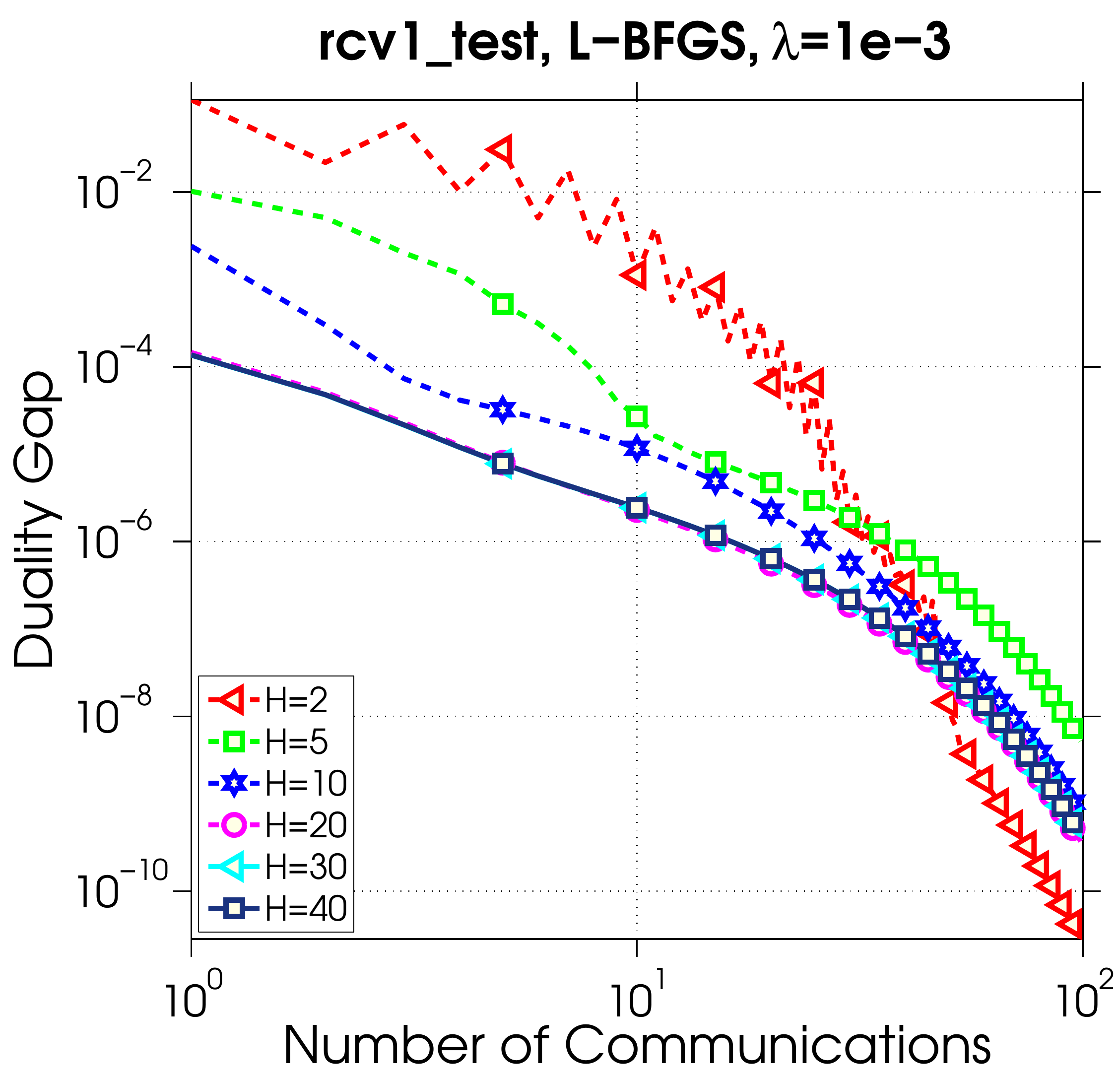}
\includegraphics[scale=.19]{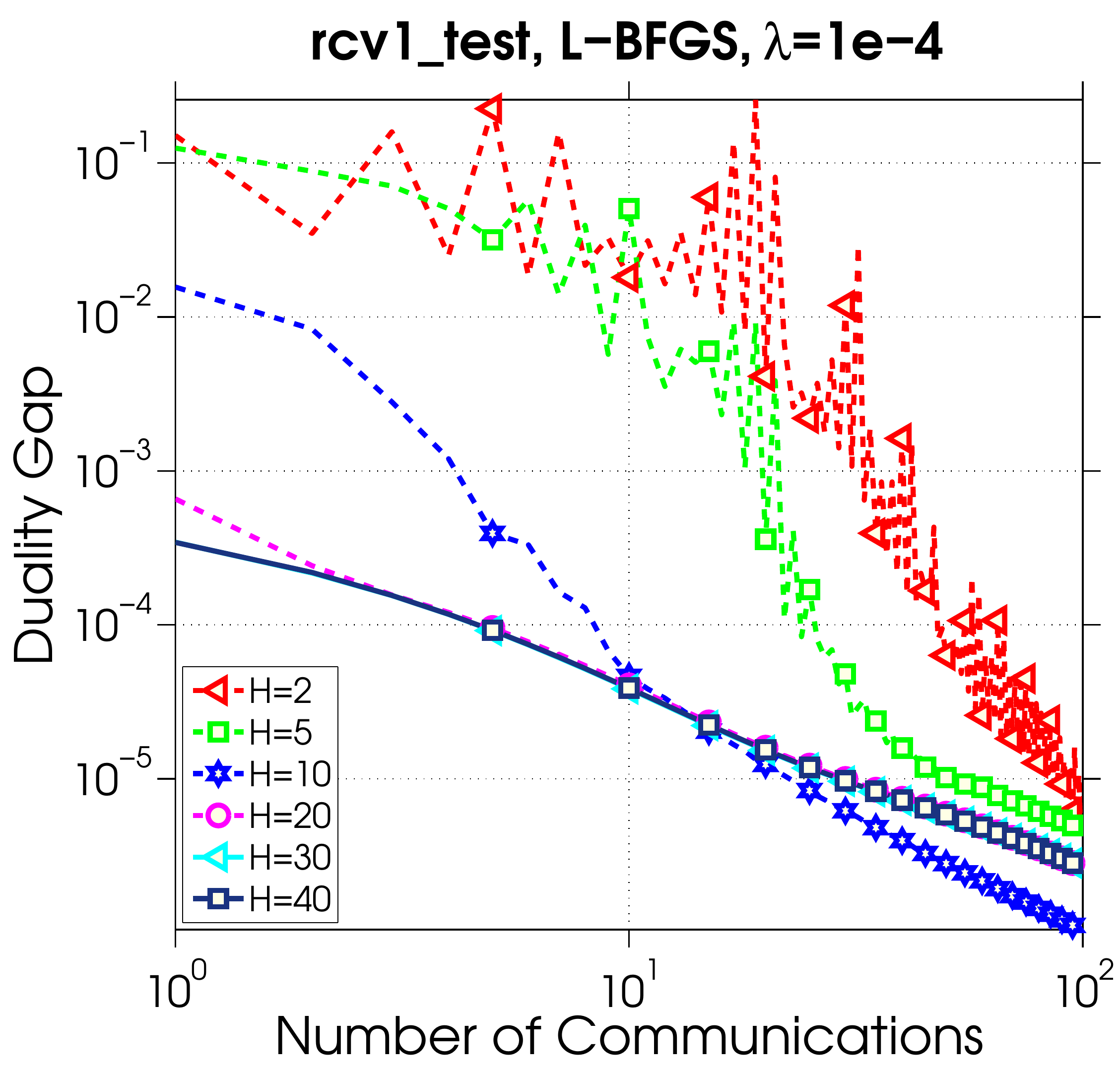}
\includegraphics[scale=.19]{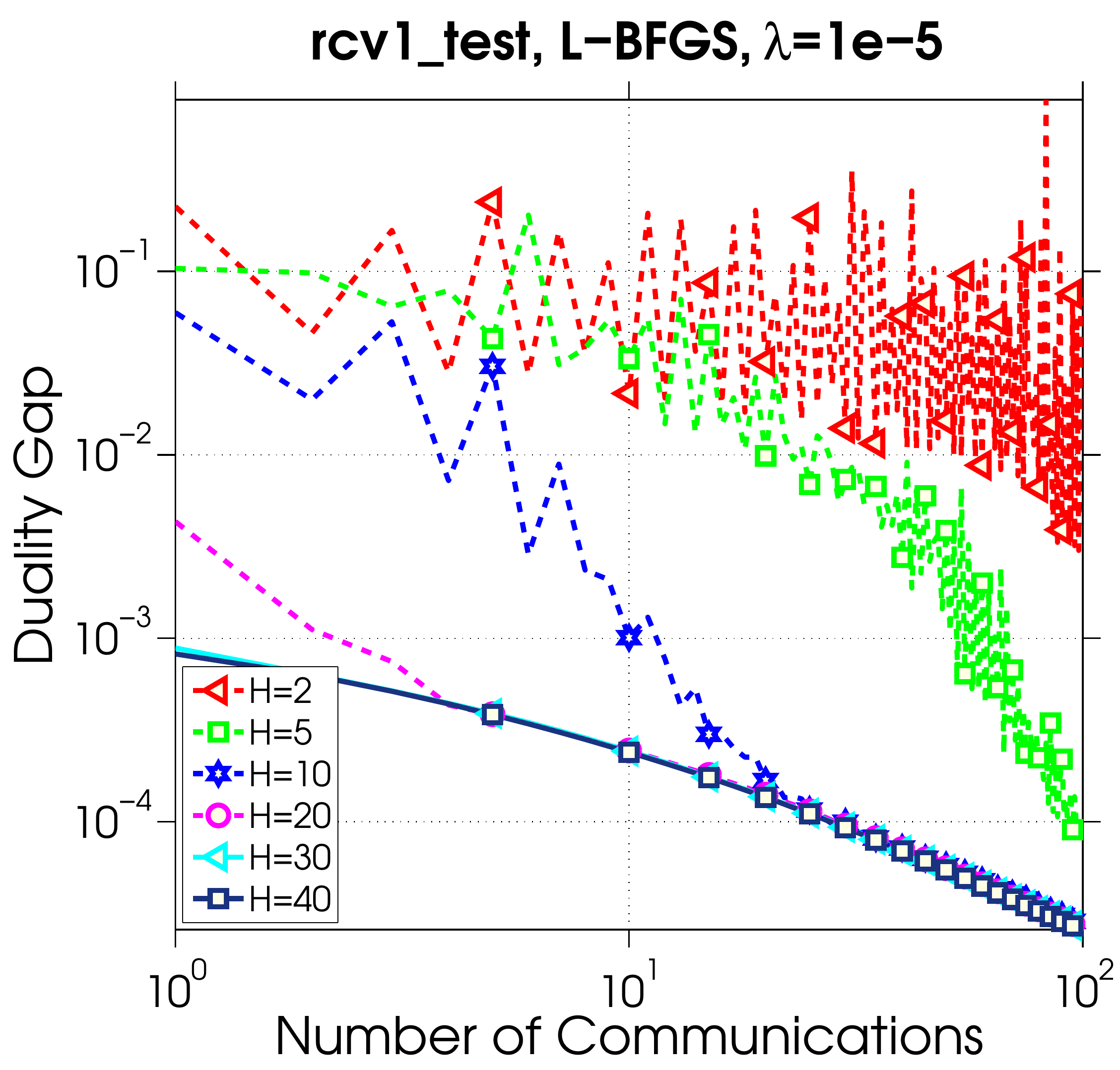}

\includegraphics[scale=.19]{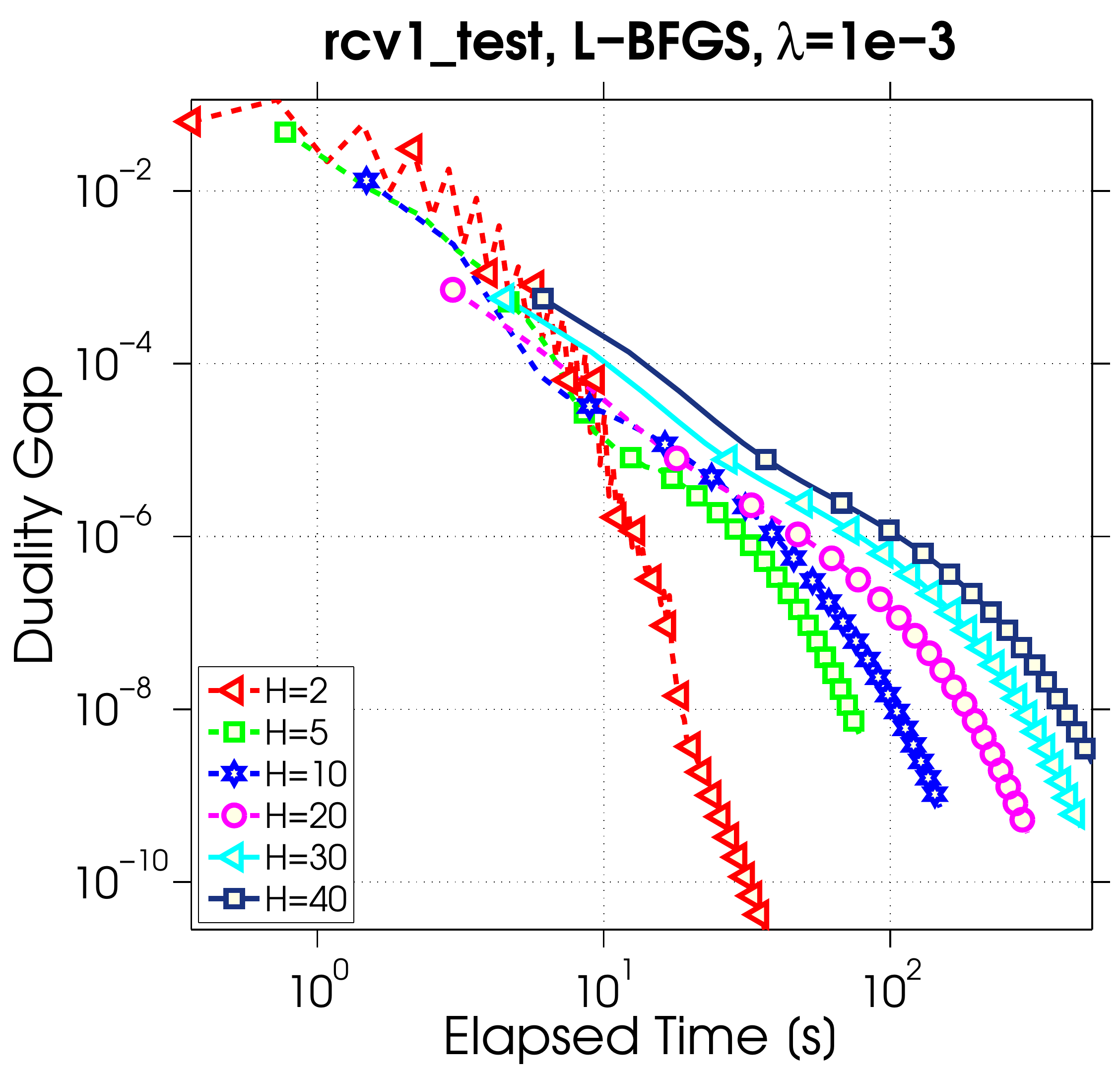}
\includegraphics[scale=.19]{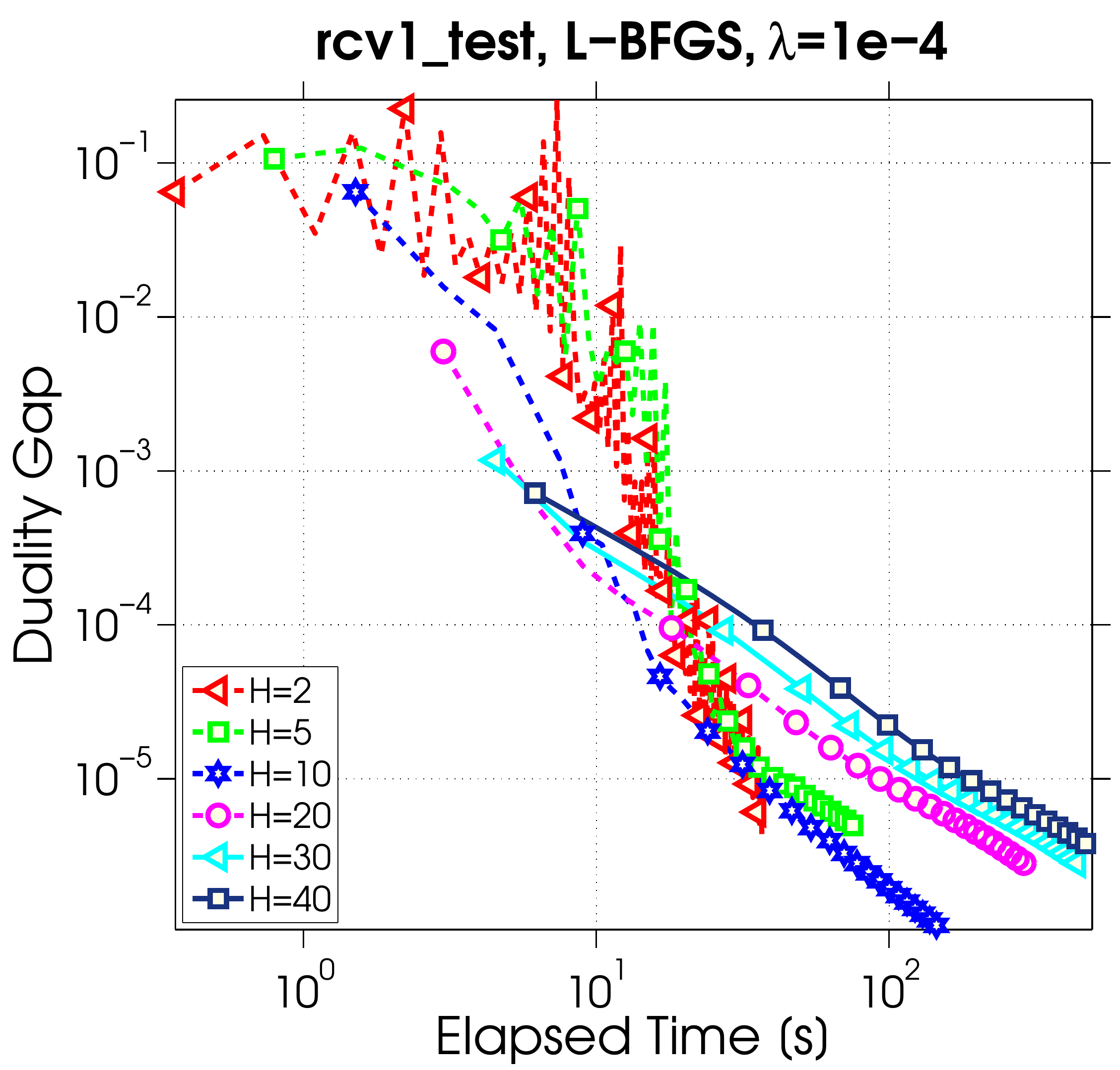}
\includegraphics[scale=.19]{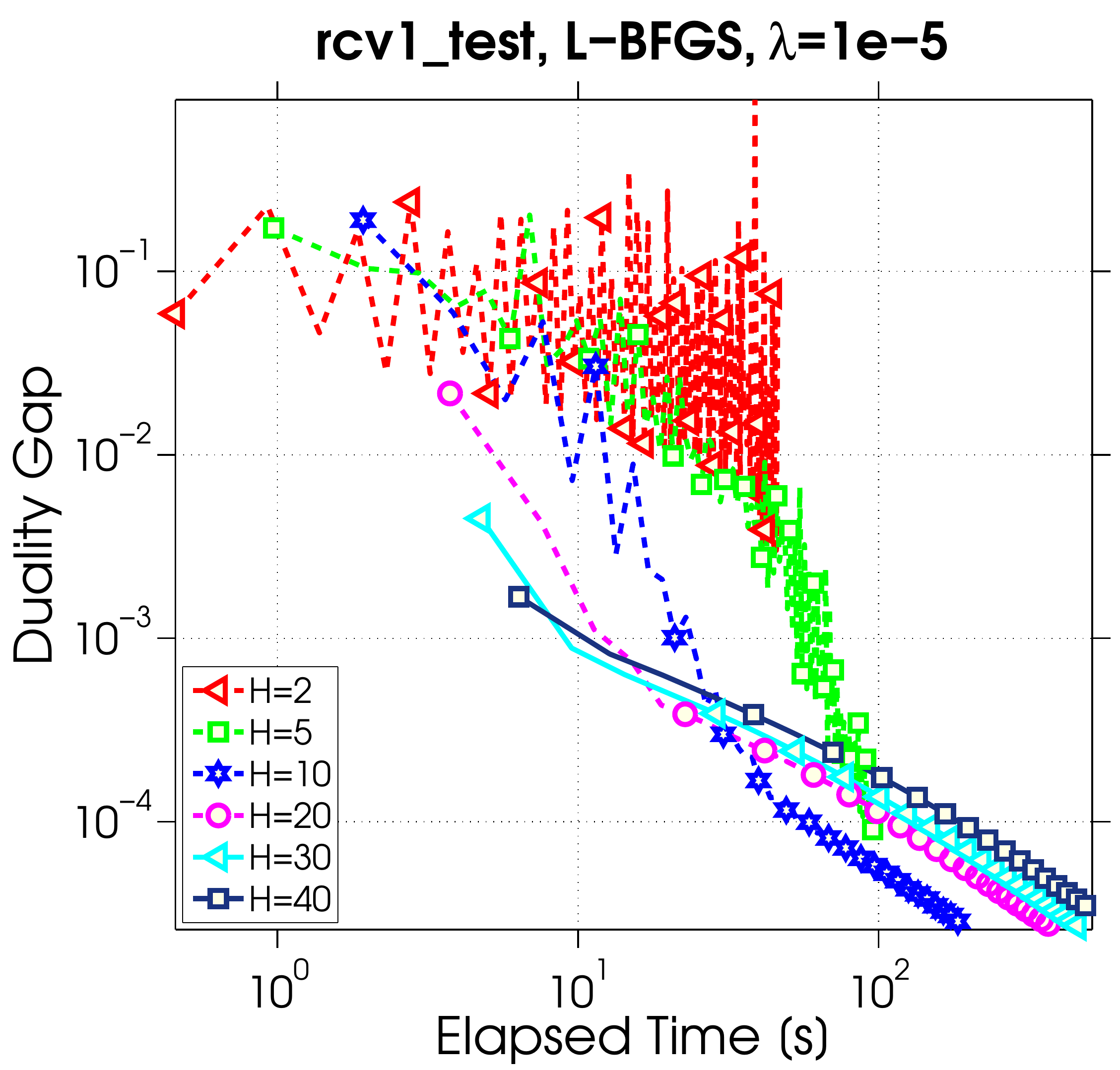}
\caption{Varying the number of iterations of L-BFGS as a local solver.} 
\label{fig:dffsollbfgsxxx}
\end{figure}

\subsection{Averaging vs.\ Adding the Local Updates}  
\label{sec:cocoa:adingVsAveraging}
In this section, we compare the performance of our algorithm using two different schemes for aggregating partial updates: adding vs. averaging. This corresponds to comparing two extremes for the parameter $\nu$, either $\nu:=\frac1K$ (averaging partial solutions) or $\nu:=1$ (adding partial solutions). As discussed in Section~\ref{sec:cocoa:result}, adding the local updates ($\aggpar=1$) will lead to less iterations {than averaging}, due to choosing different $\sigma'$ in the subproblems. We verify this experimentally by considering several of the local solvers listed in Table~ \ref{tbl:othersolvers}.

We show results for {the rcv1\_test} dataset, and we apply {the} quadratic loss function with three different choices for the regularization parameter, $\lambda$=$1e-03$, $1e-04$, and $1e-05$. The experiments in Figures~\ref{fig:soler2}--\ref{fig:soler7} indicate that the {``adding'' strategy} will  always lead to faster convergence than averaging, even though the difference is minimal when we apply a large number of iterations in the local solver. All the blue solid plots (adding) outperform the red dashed plots (averaging), which indicates the advantage of choosing $\aggpar= 1$. Another note here is that for smaller $\lambda$, we will have to spend more iterations to get the same accuracy{, because the original objective function \eqref{eq:primal} is less strongly convex.}

\begin{figure}[H]
\centering
\includegraphics[scale=.19]{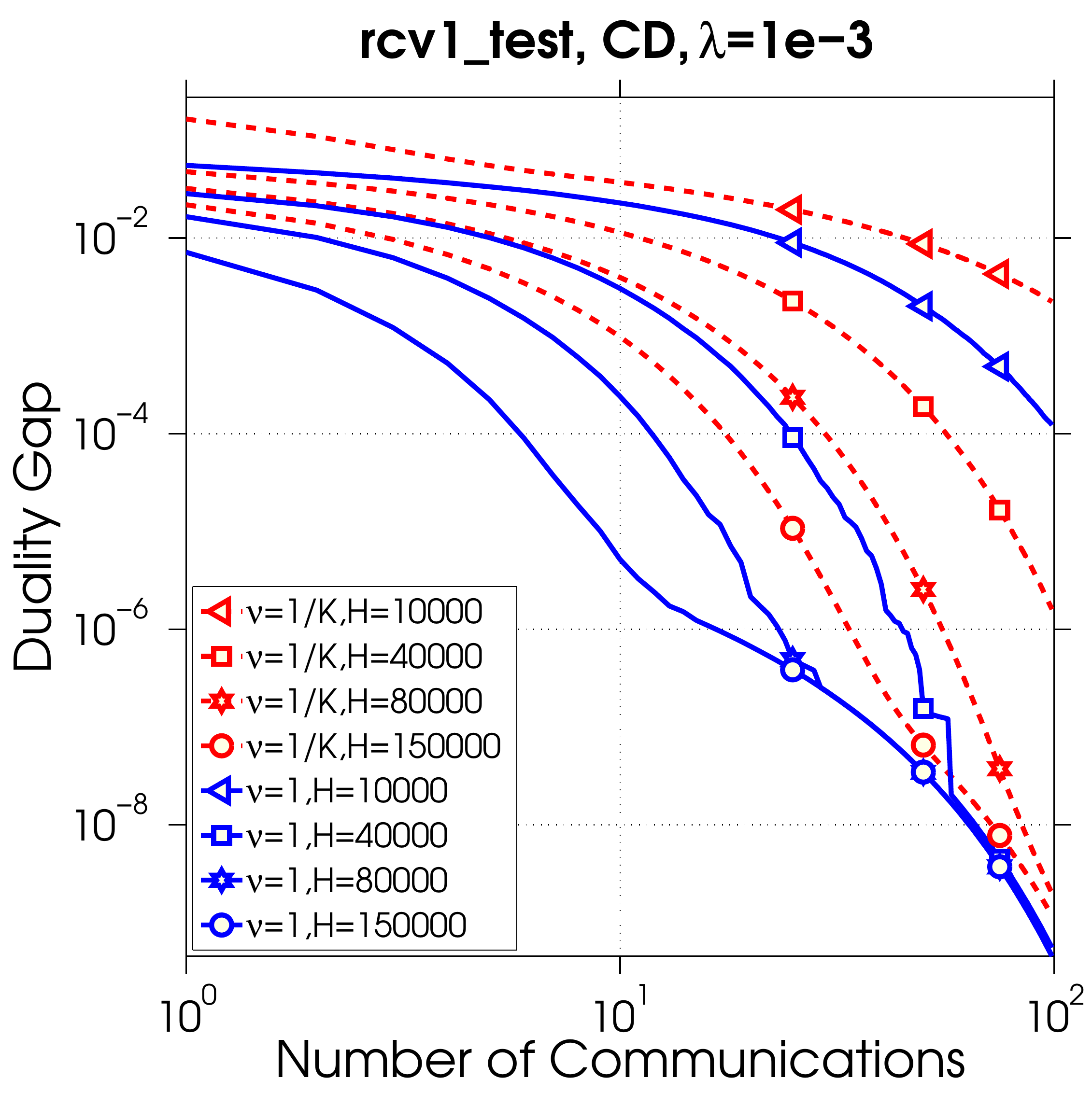}
\includegraphics[scale=.19]{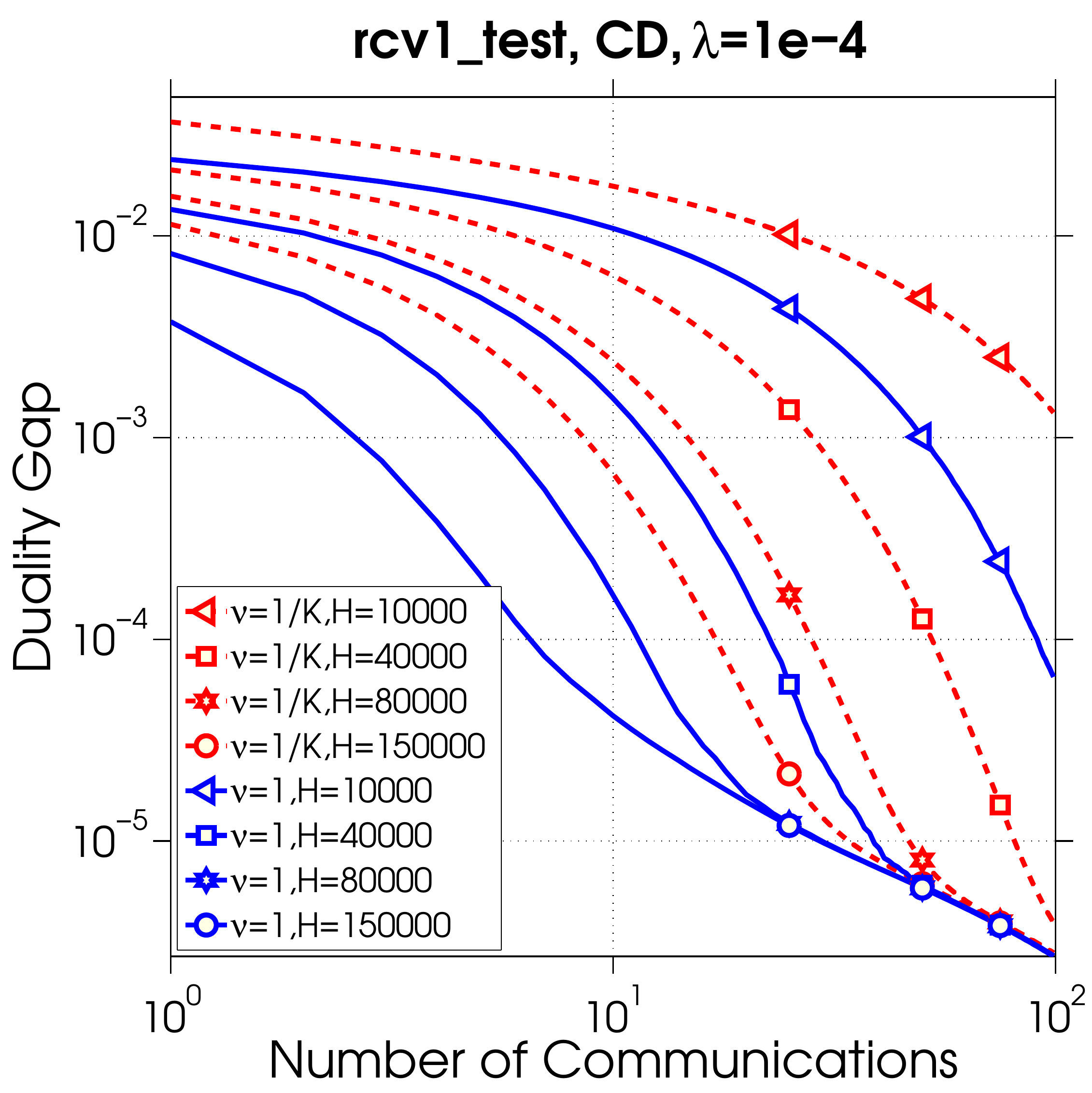}
\includegraphics[scale=.19]{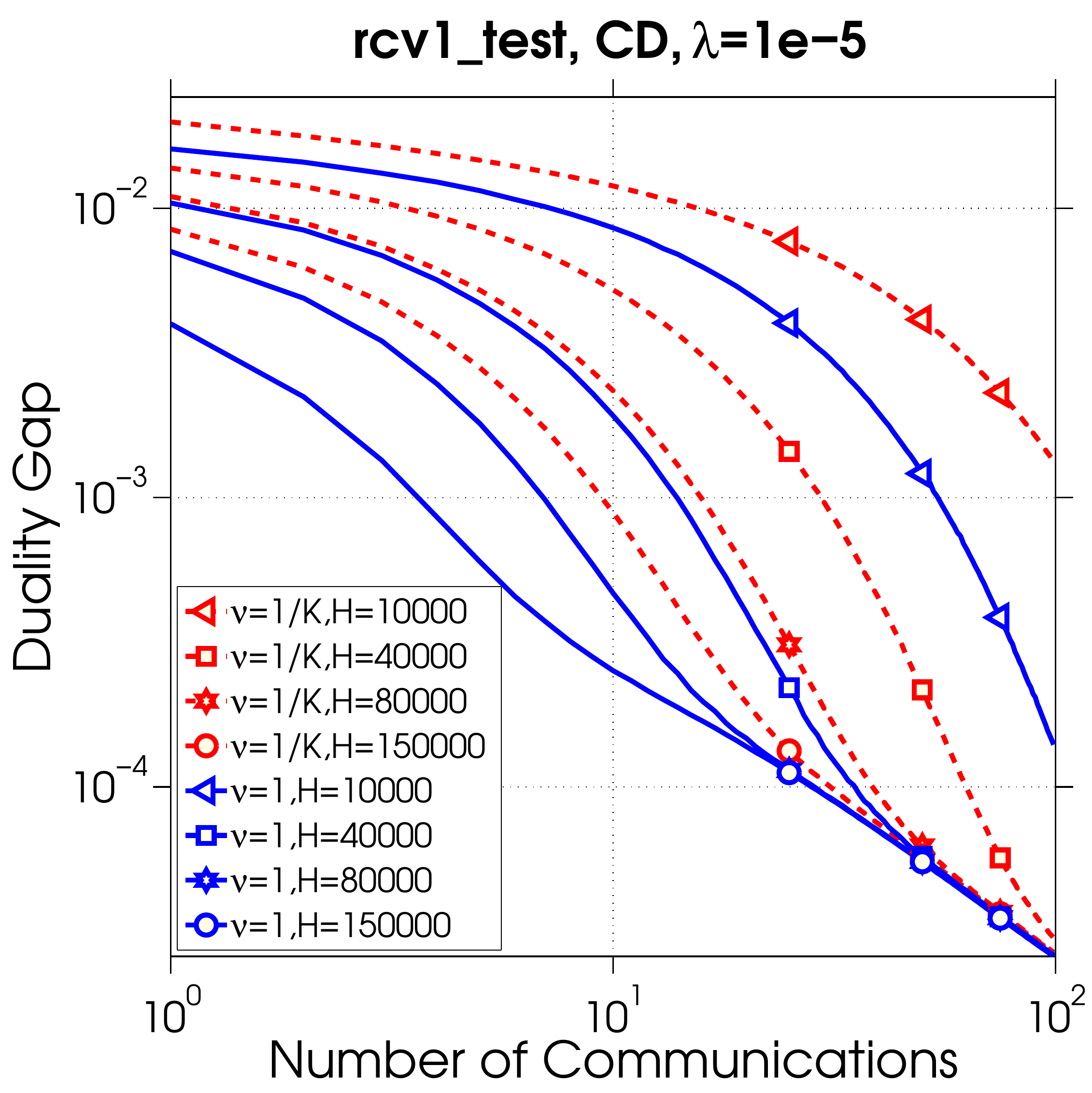}

\includegraphics[scale=.19]{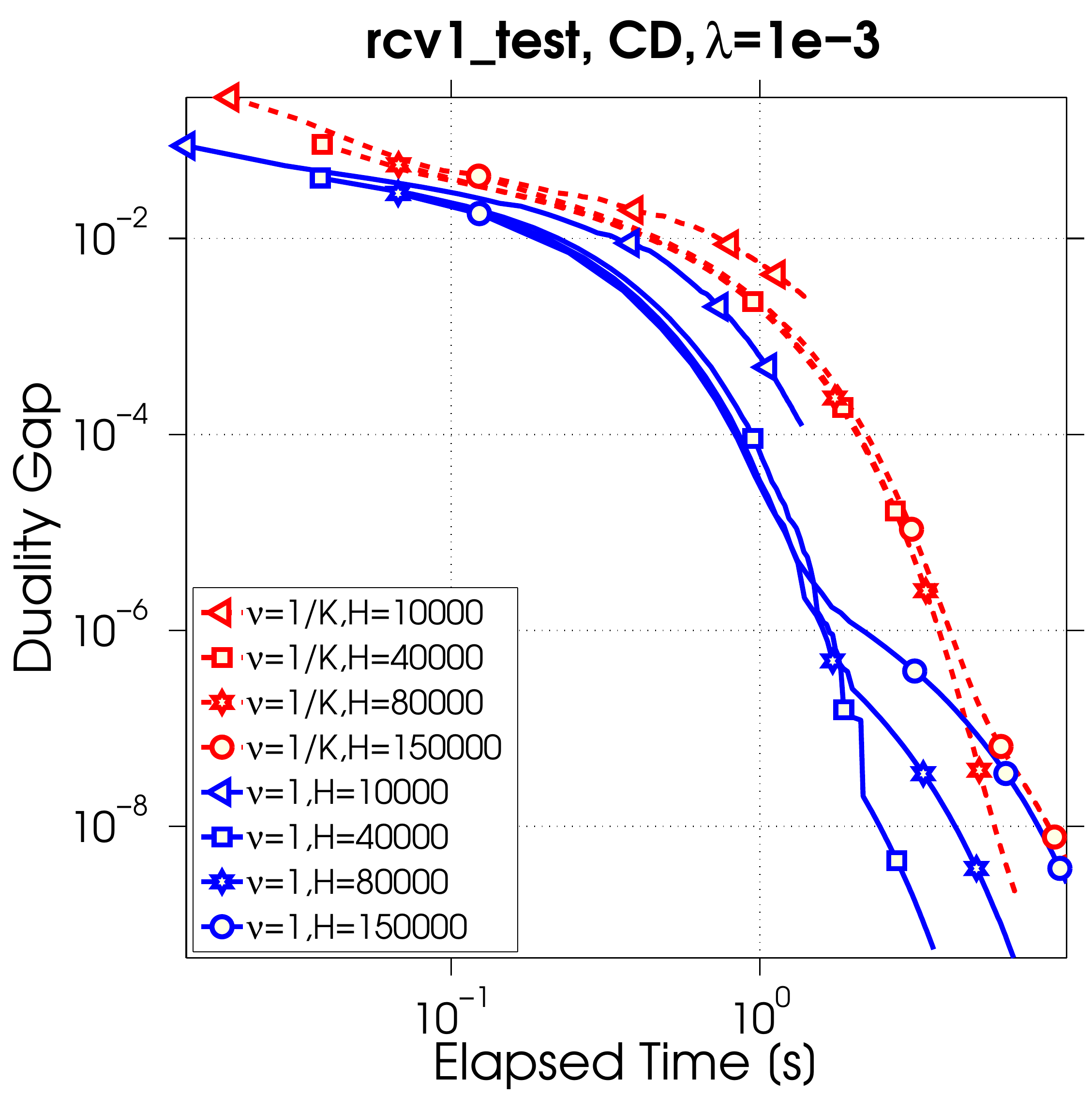}
\includegraphics[scale=.19]{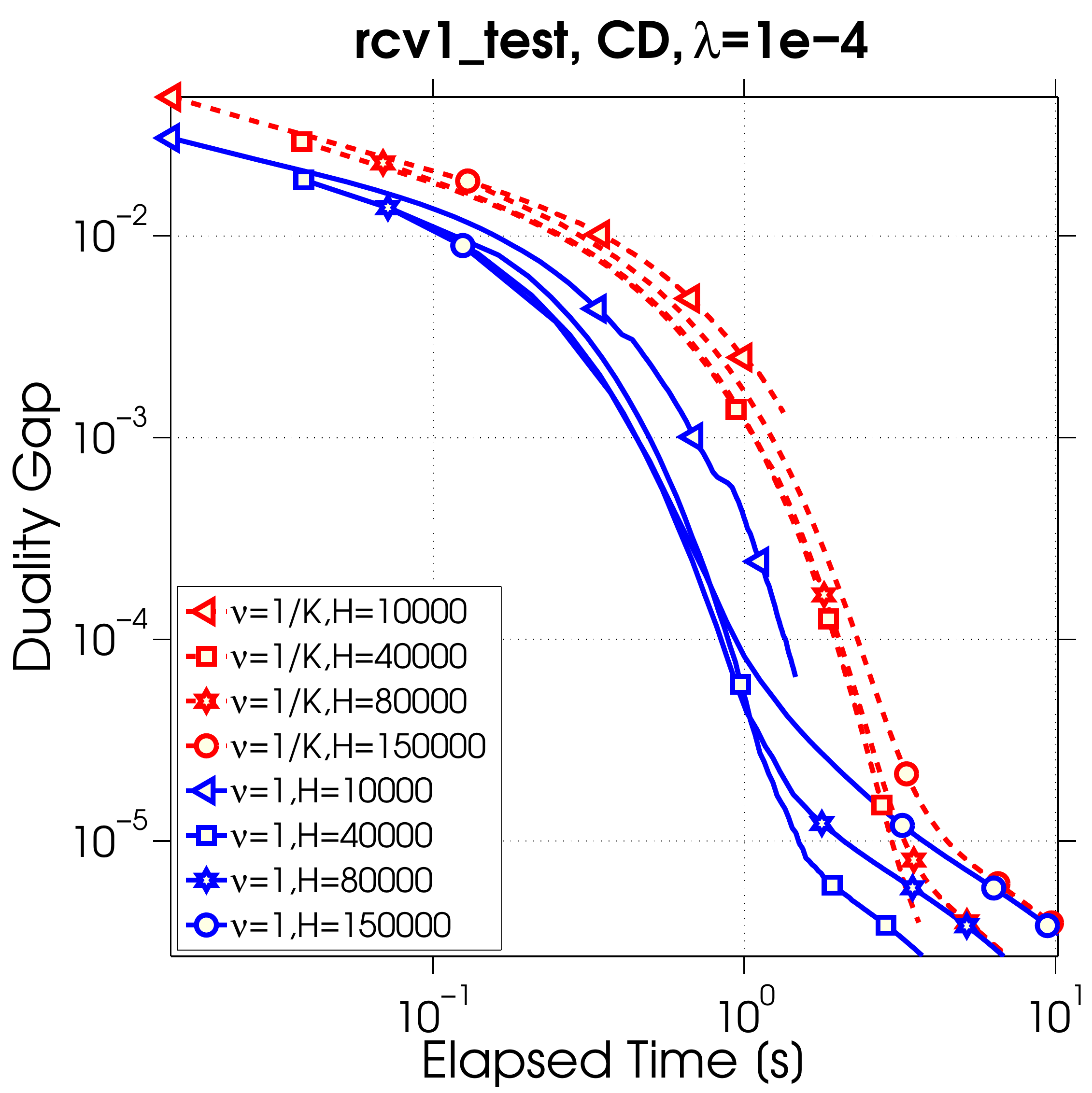}
\includegraphics[scale=.19]{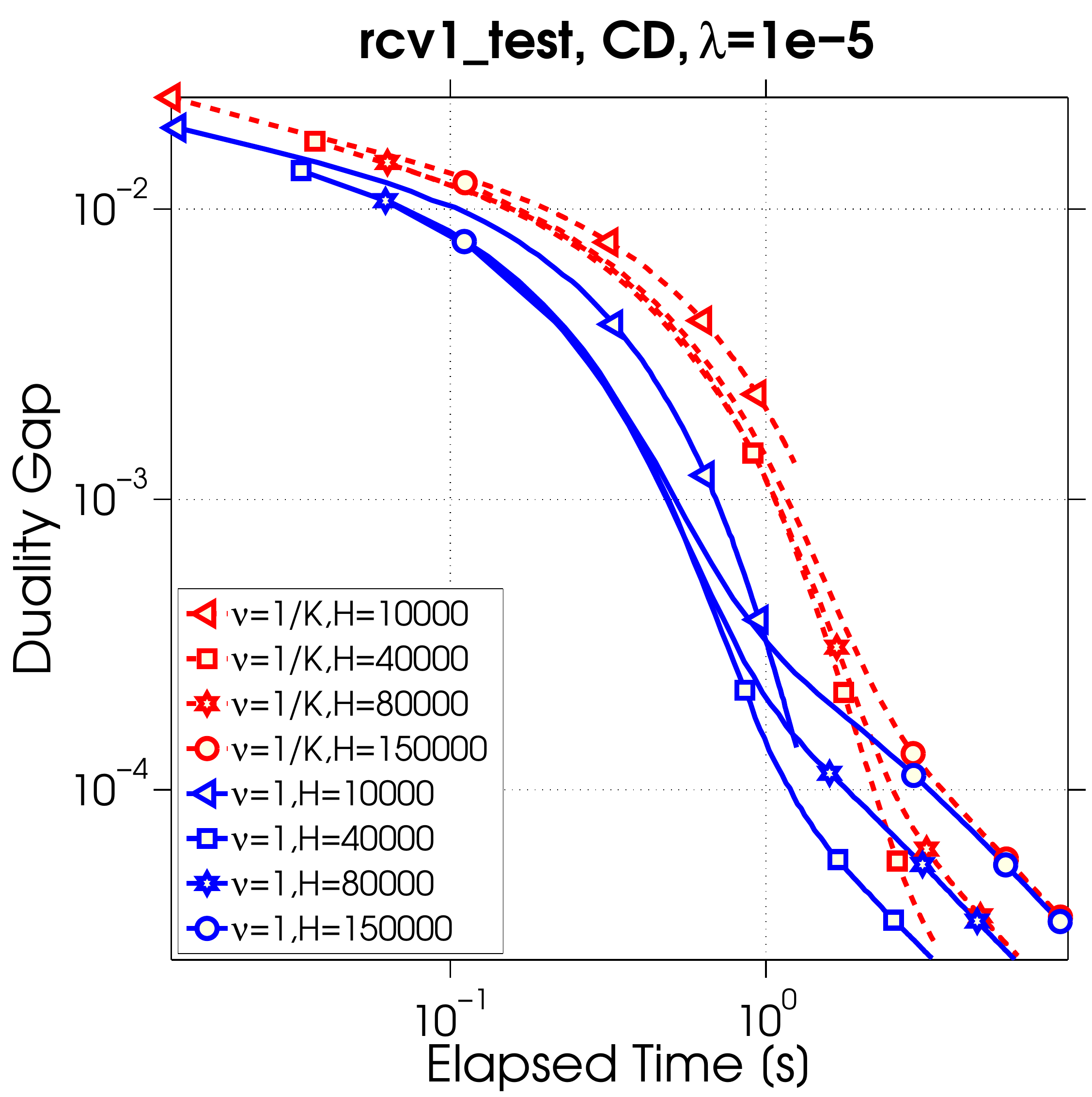}
\caption{Adding (blue solid line) vs Averaging (red dashed line) for {CD} as the local solver. } 
\label{fig:solver1}
\end{figure}

\begin{figure}[H]
\centering
\includegraphics[scale=.19]{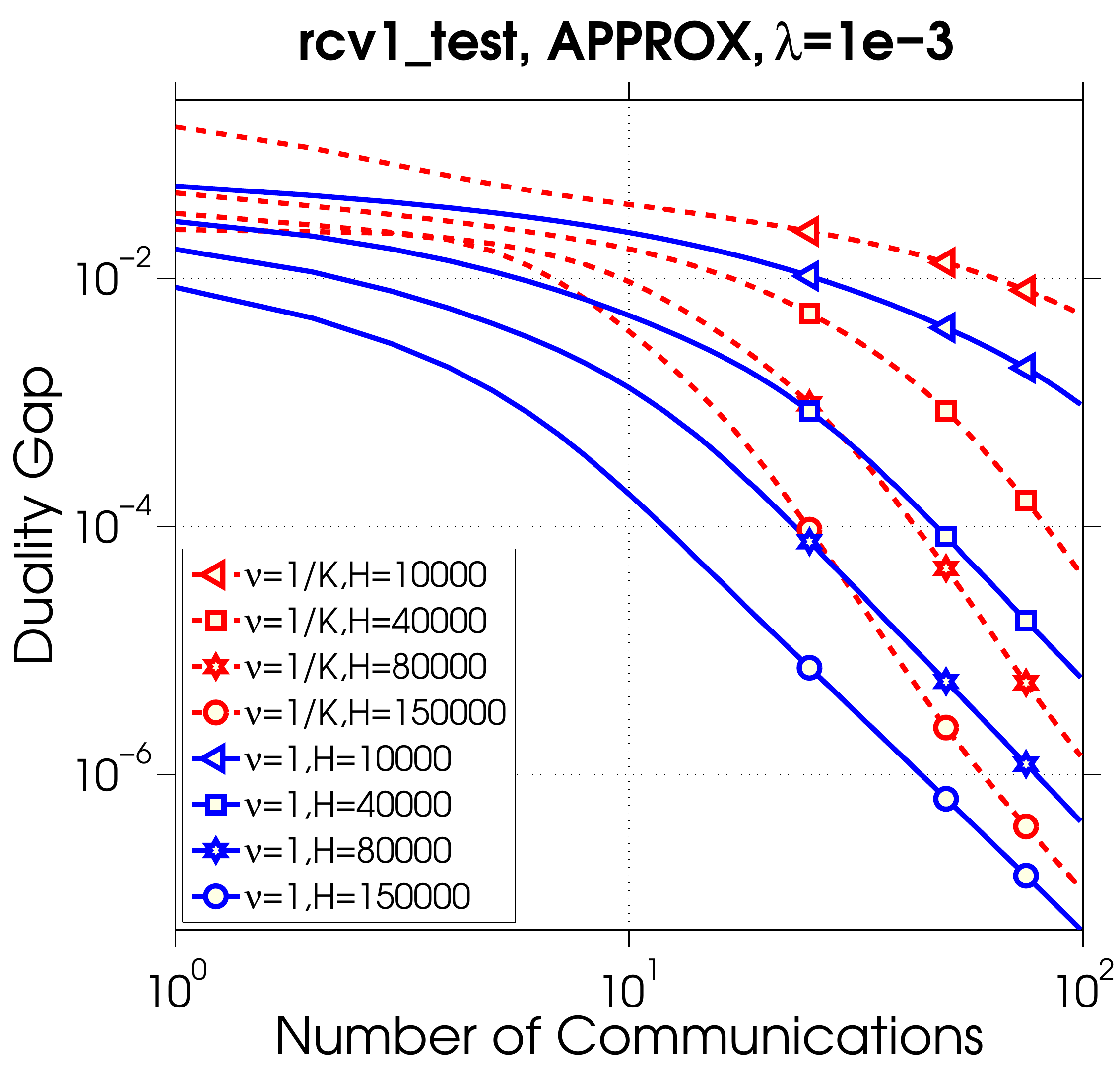}
\includegraphics[scale=.19]{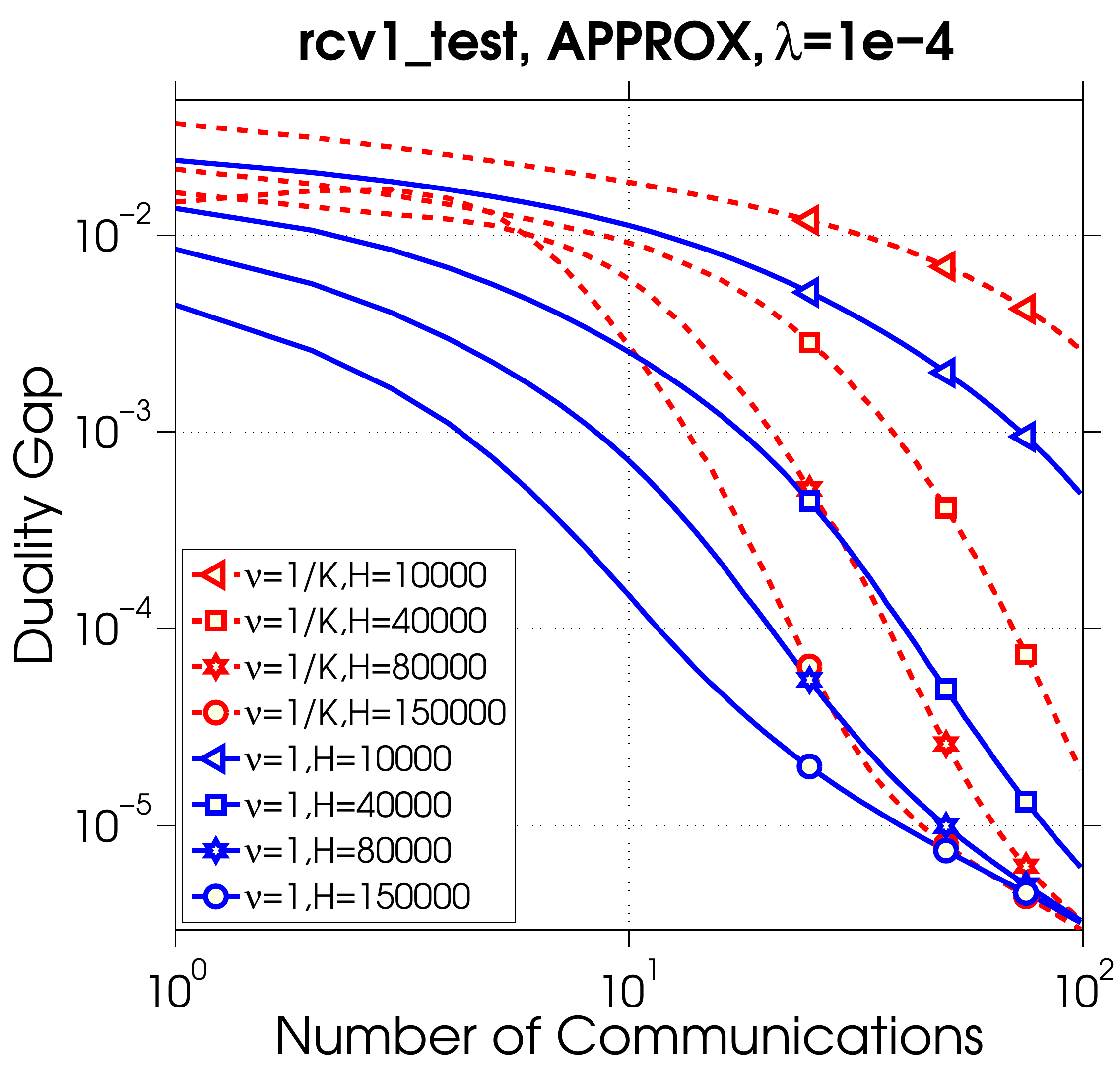}
\includegraphics[scale=.19]{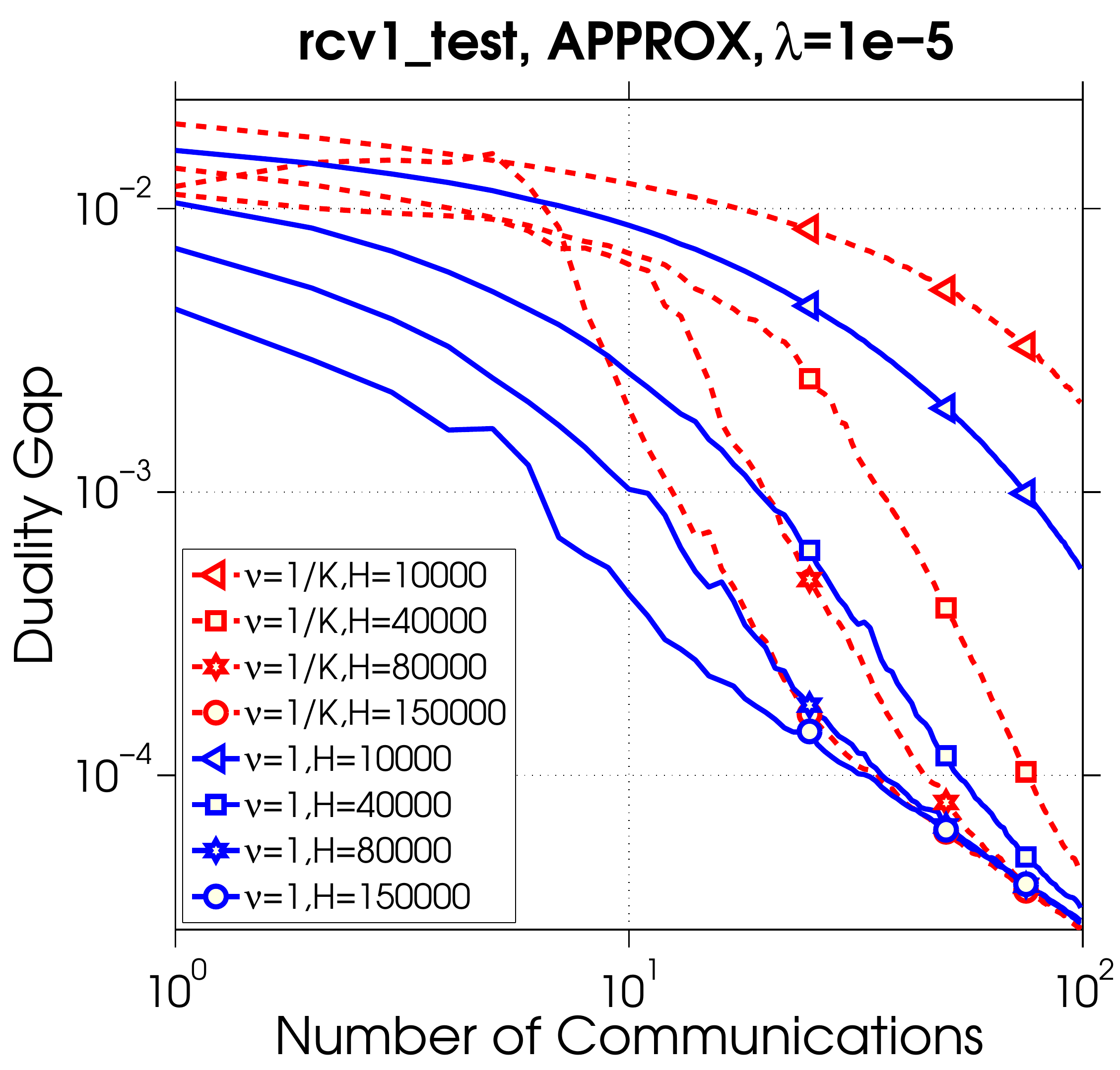}

\includegraphics[scale=.19]{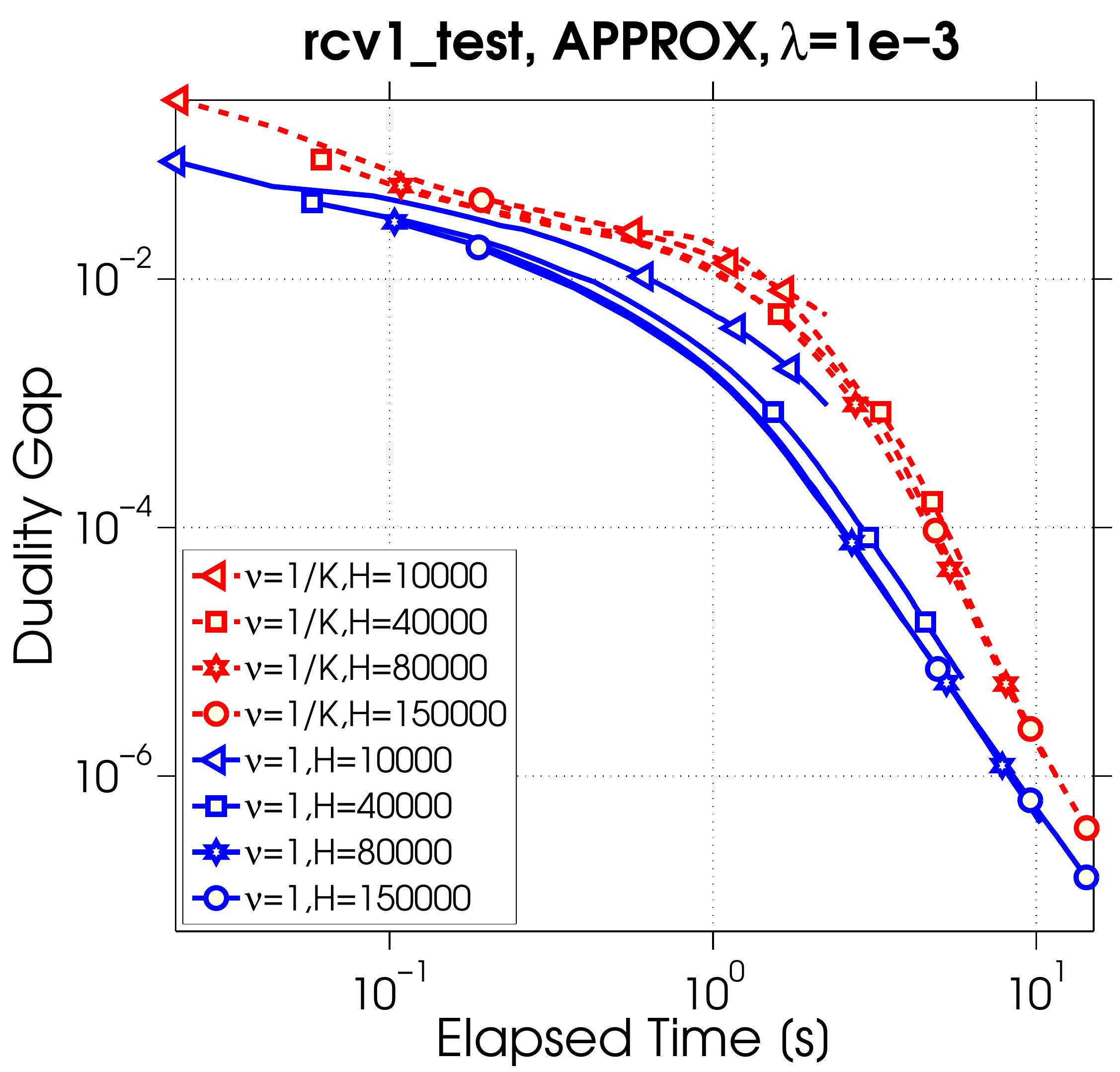}
\includegraphics[scale=.19]{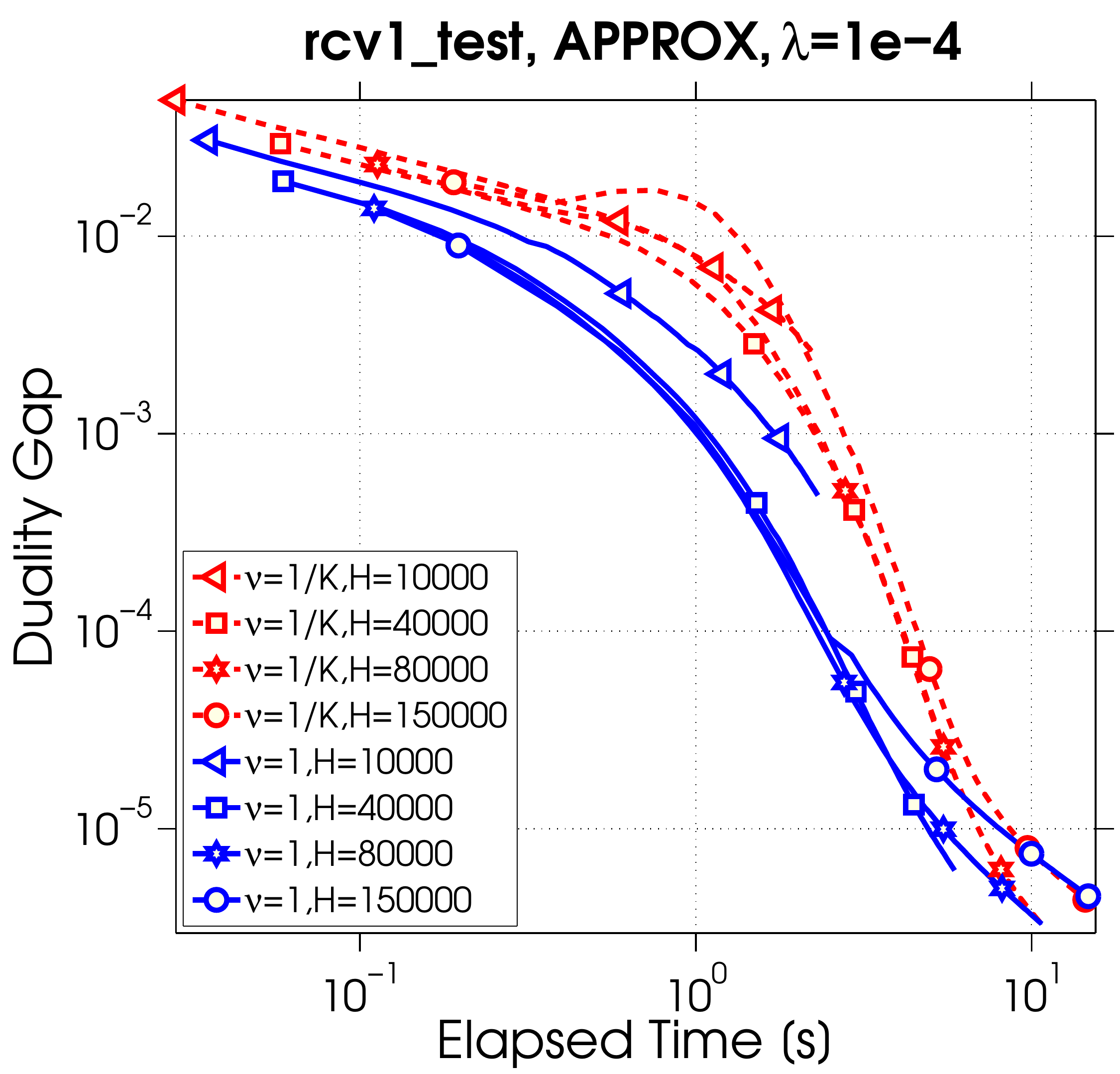}
\includegraphics[scale=.19]{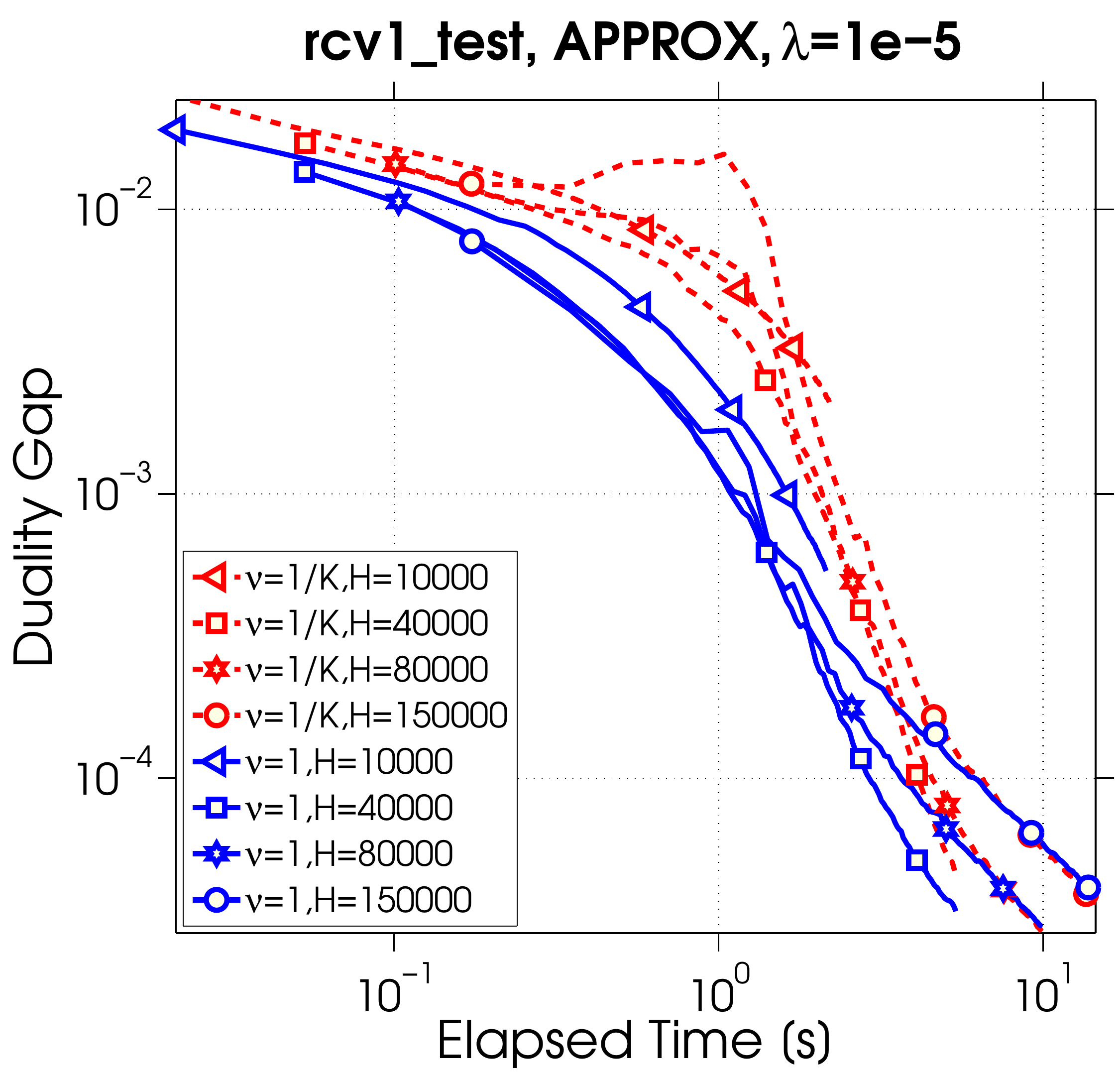}
\caption{Adding (blue solid line) vs Averaging (red dashed line) for APPROX as the local solver.} 
\label{fig:soler2}
\end{figure}

\begin{figure}[H]
\centering
\includegraphics[scale=.19]{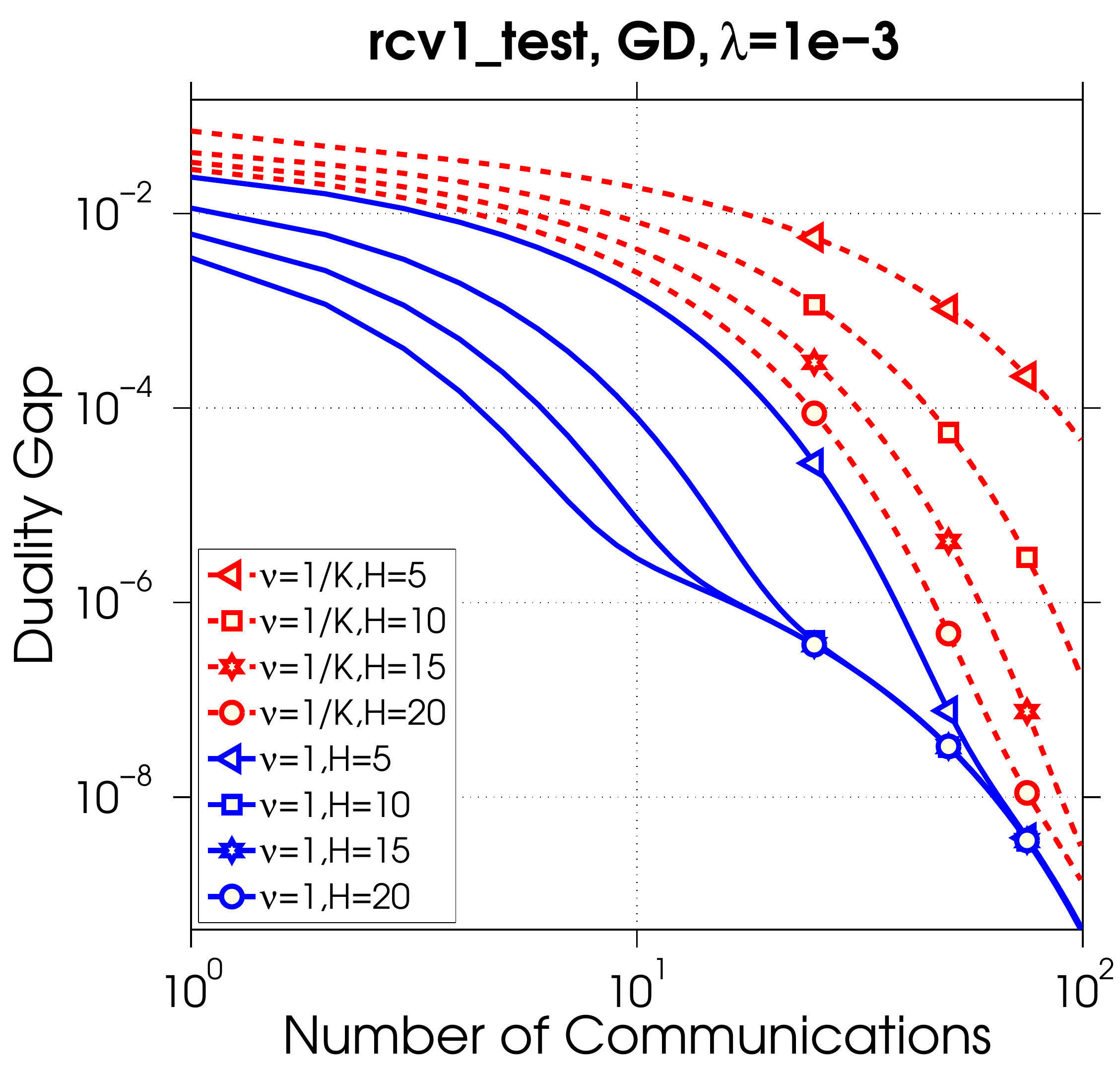}
\includegraphics[scale=.19]{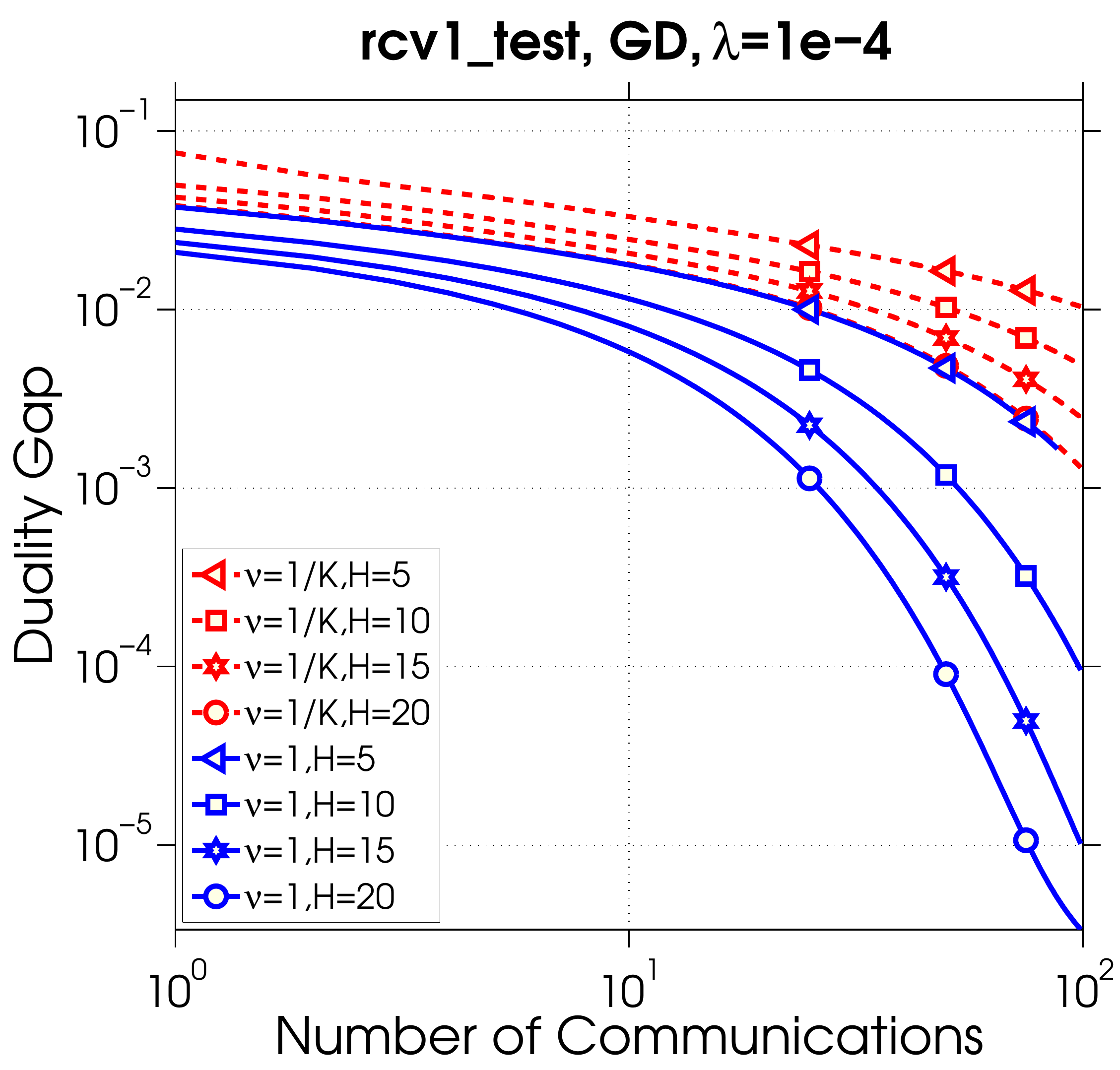}
\includegraphics[scale=.19]{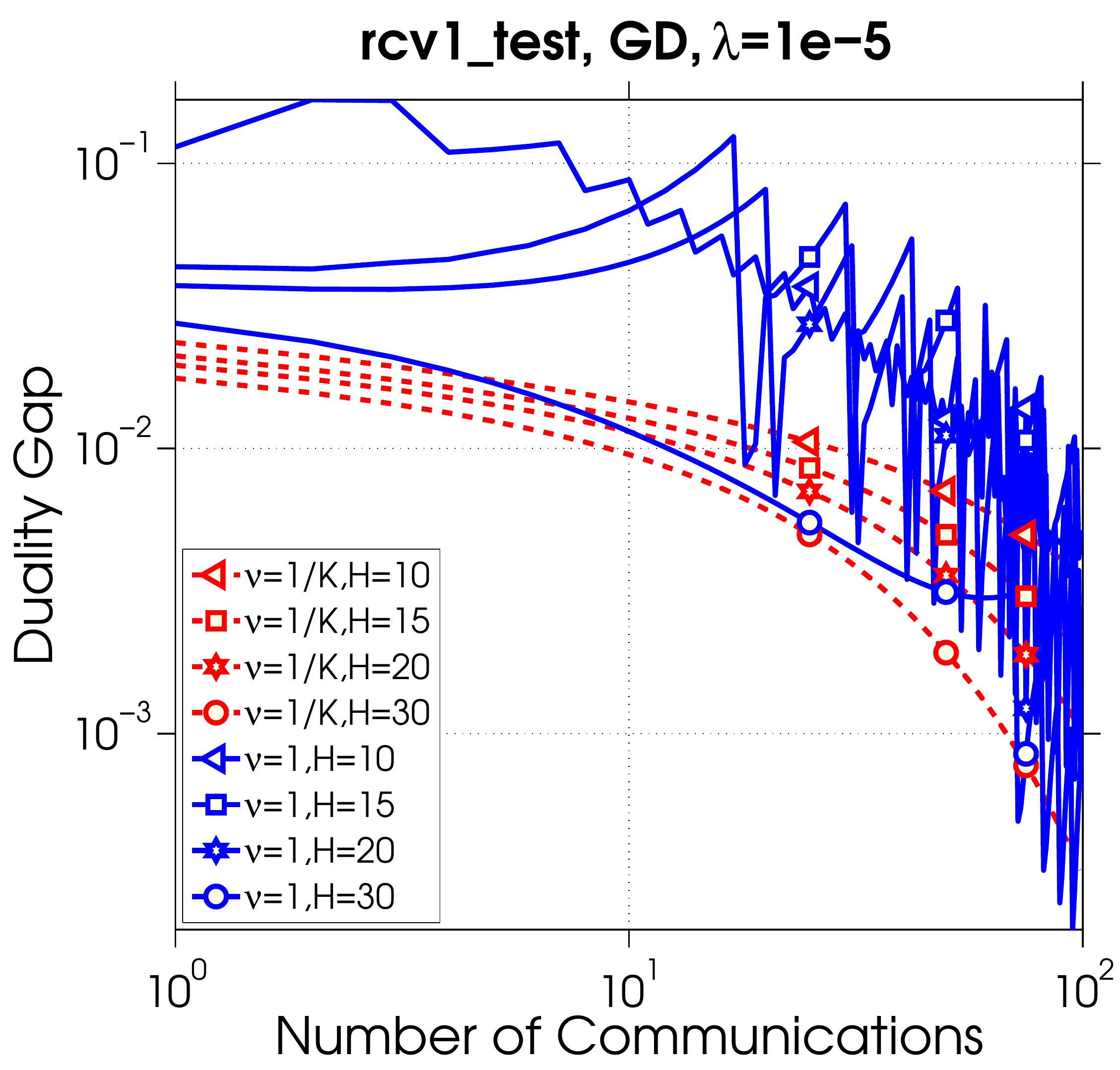}

\includegraphics[scale=.19]{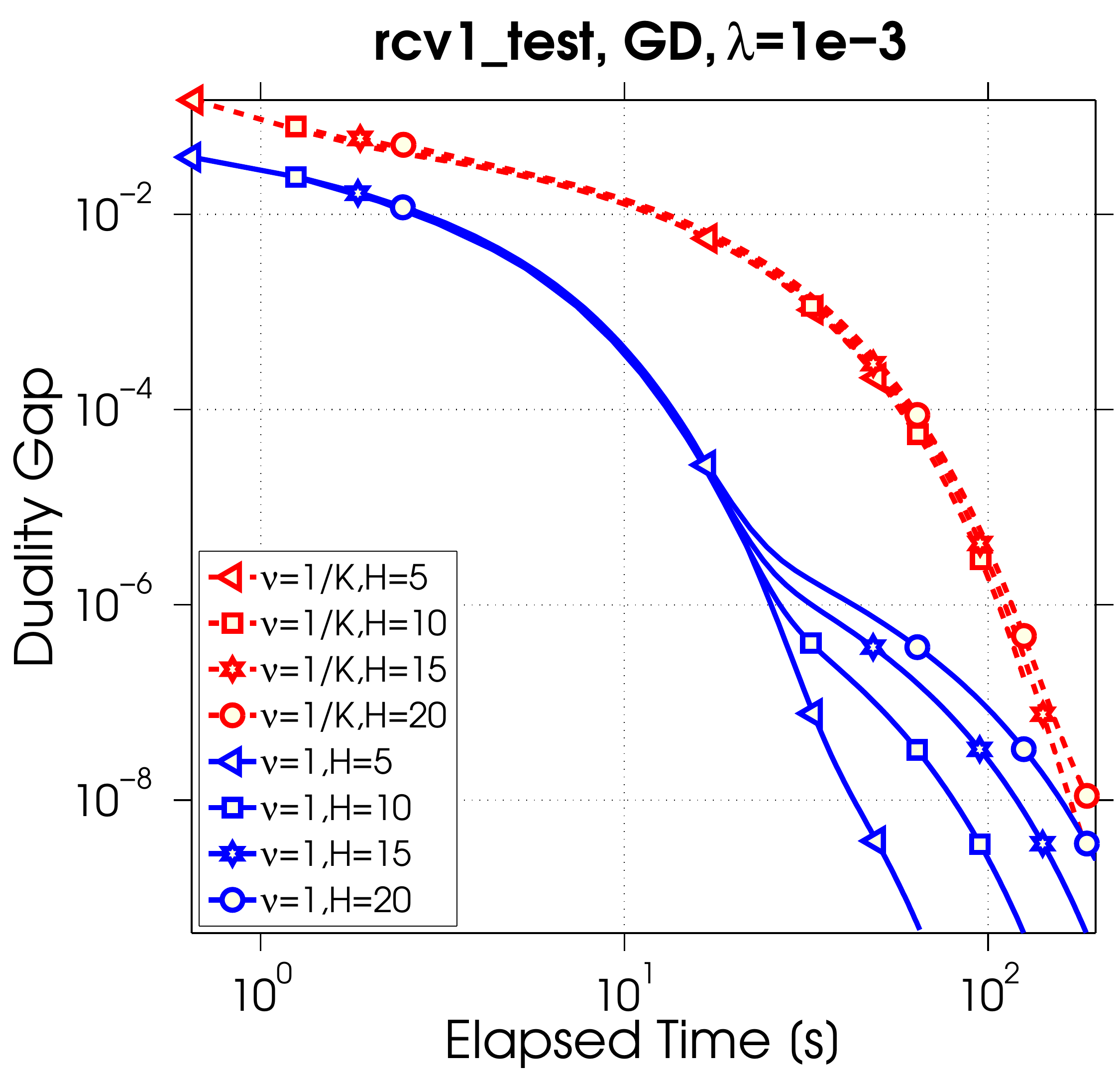}
\includegraphics[scale=.19]{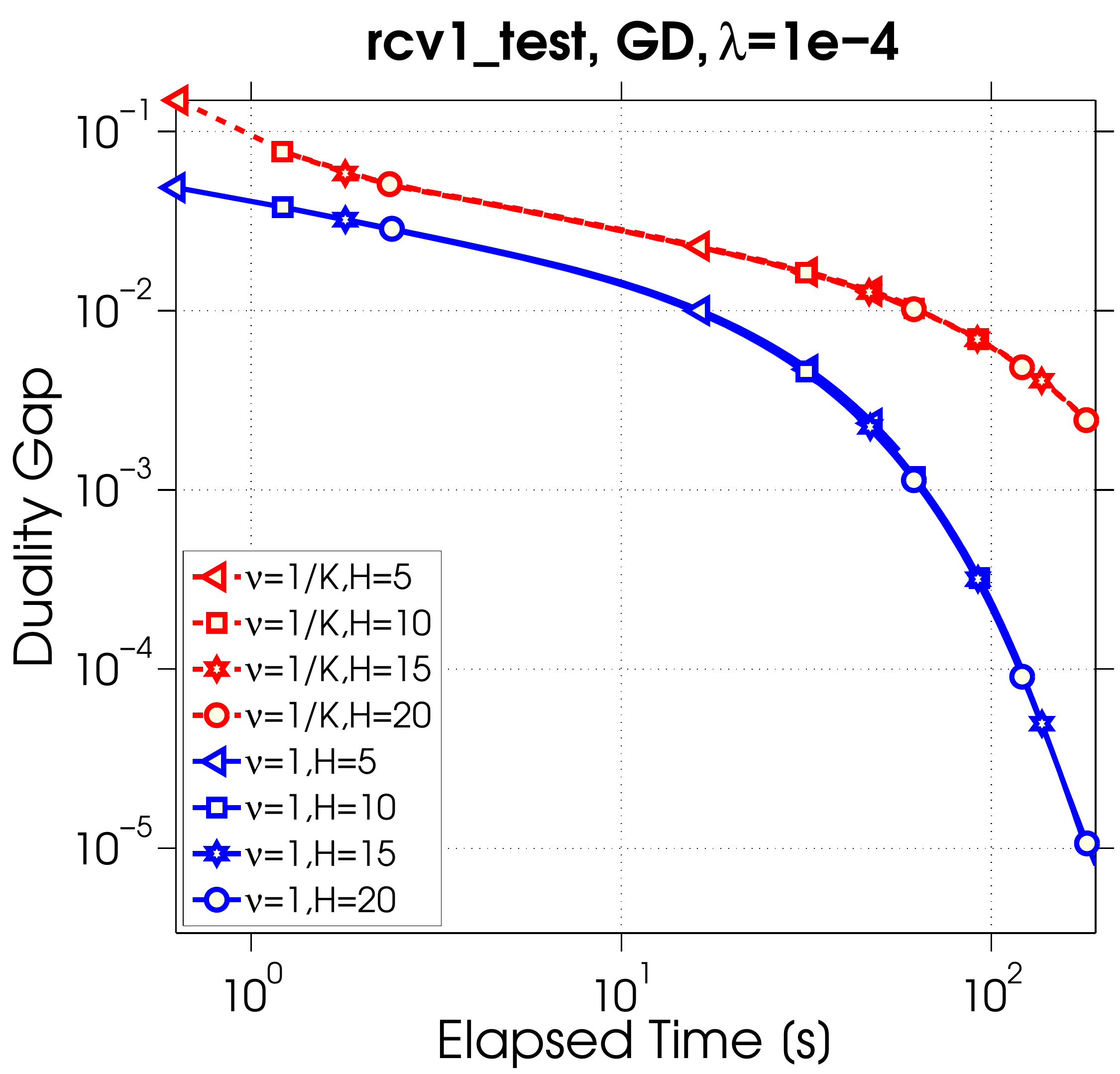}
\includegraphics[scale=.19]{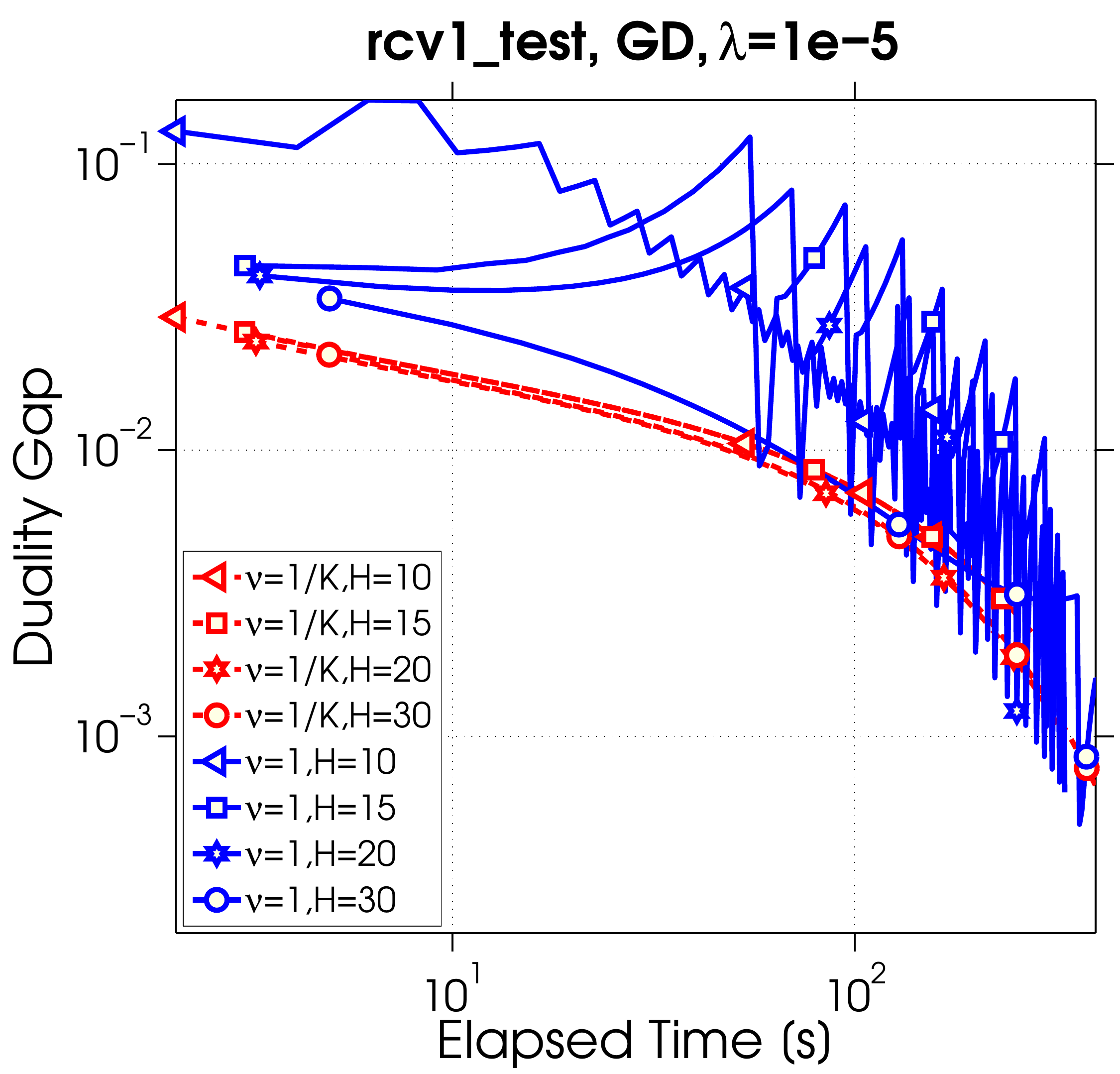}
\caption{Adding (blue solid line) vs Averaging (red dashed line) for Gradient Descent as the local solver.} 
\label{fig:soler3}
\end{figure}

\begin{figure}[H]
\centering
\includegraphics[scale=.19]{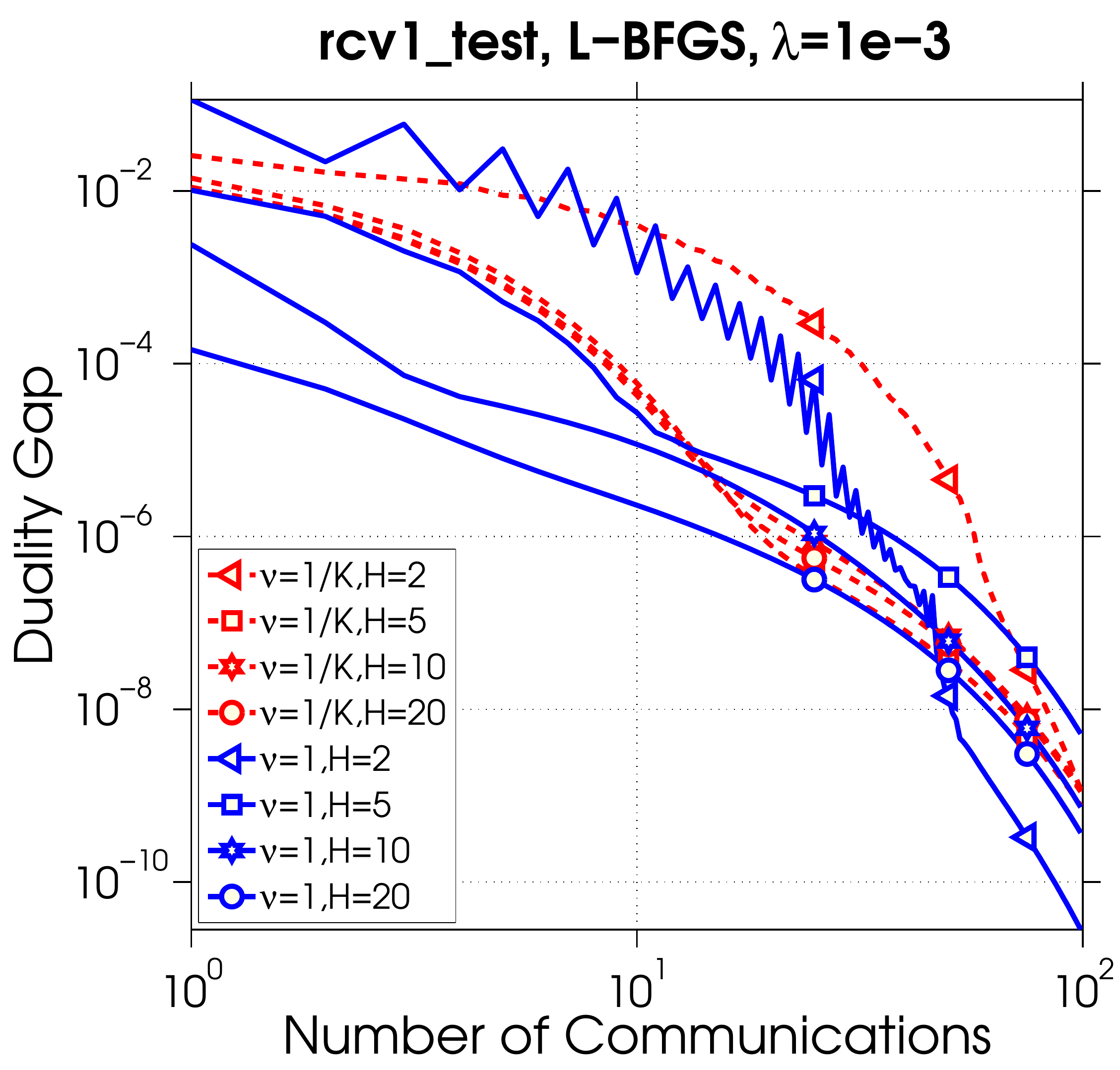}
\includegraphics[scale=.19]{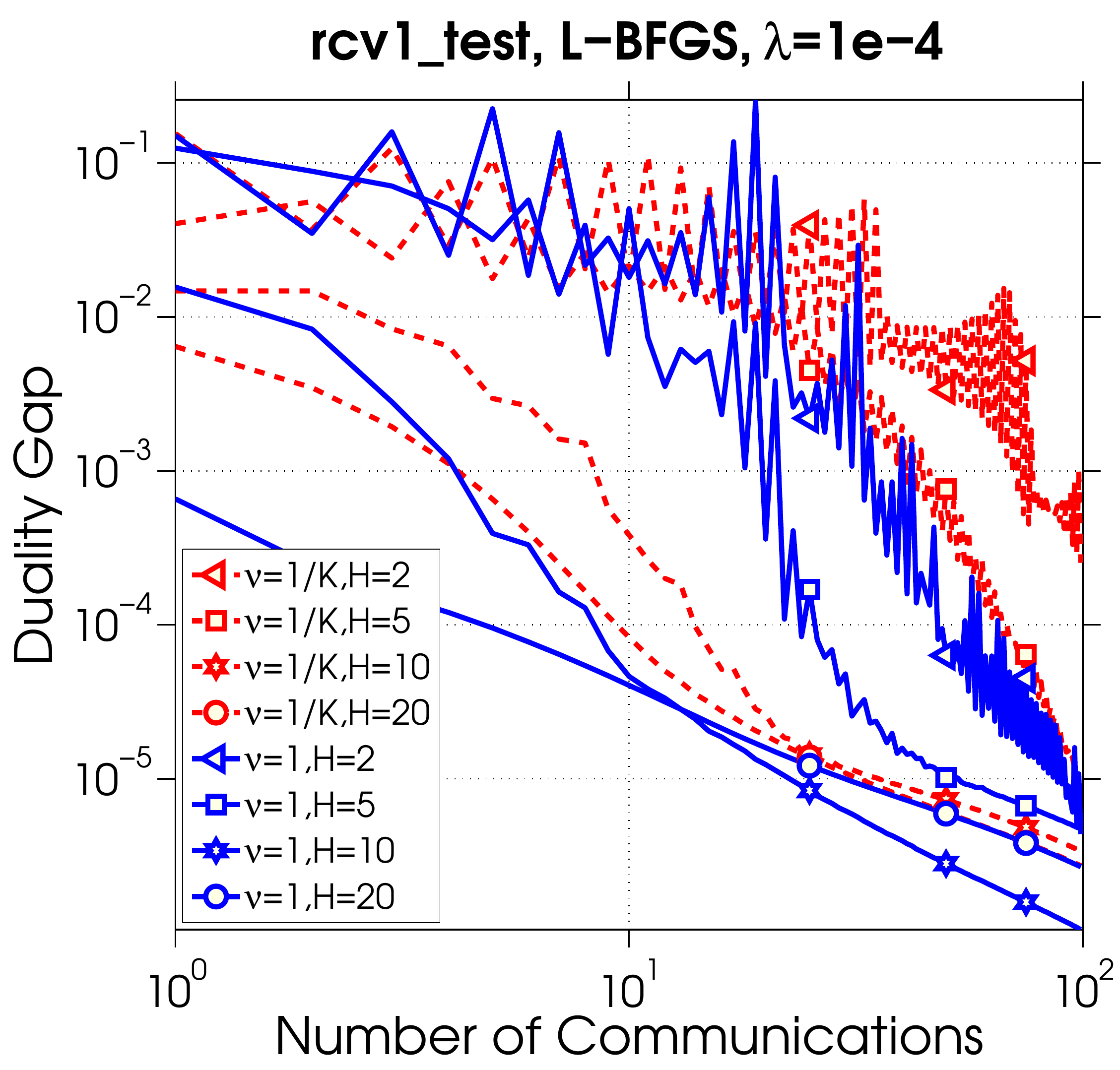}
\includegraphics[scale=.19]{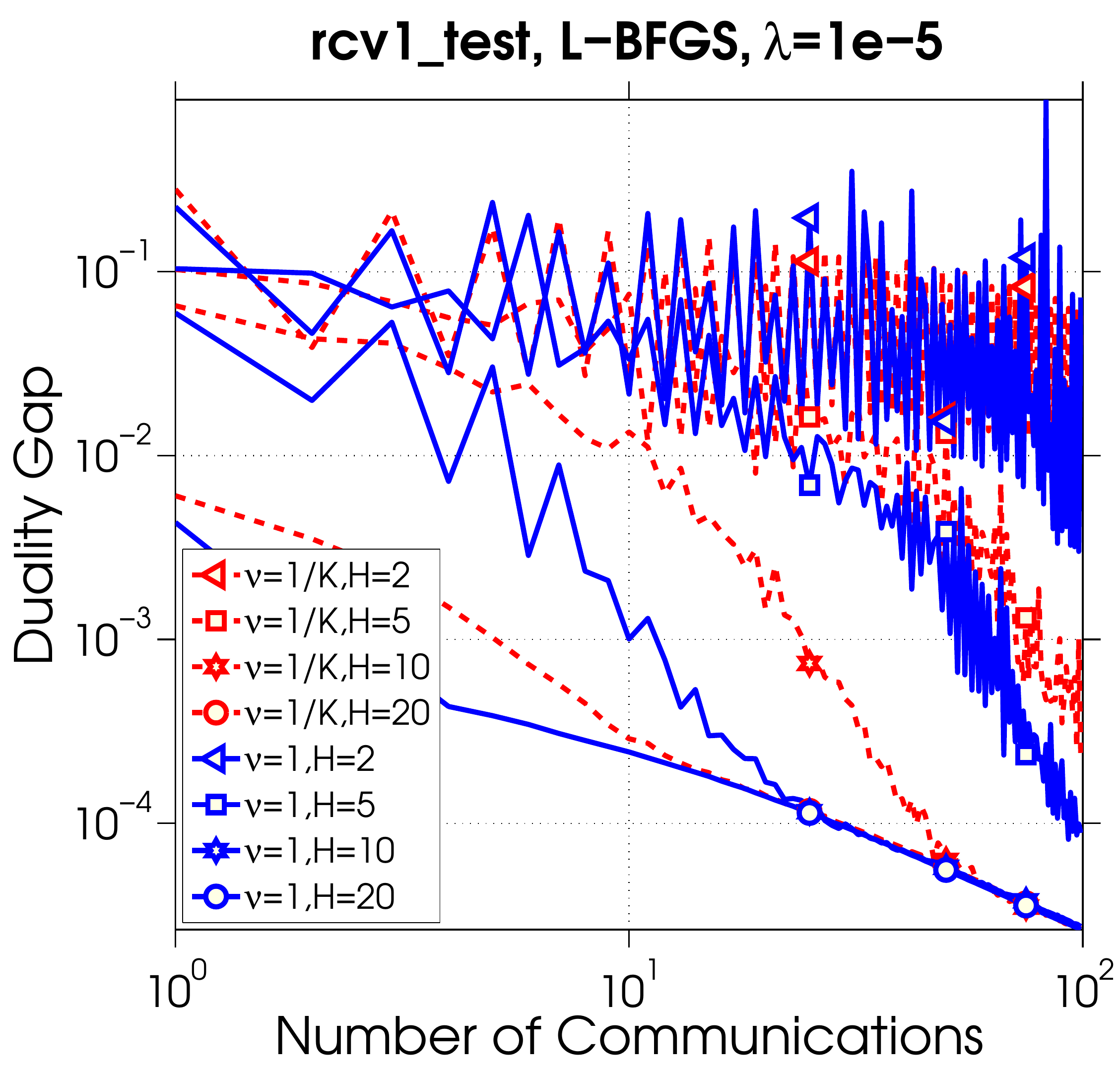}

\includegraphics[scale=.19]{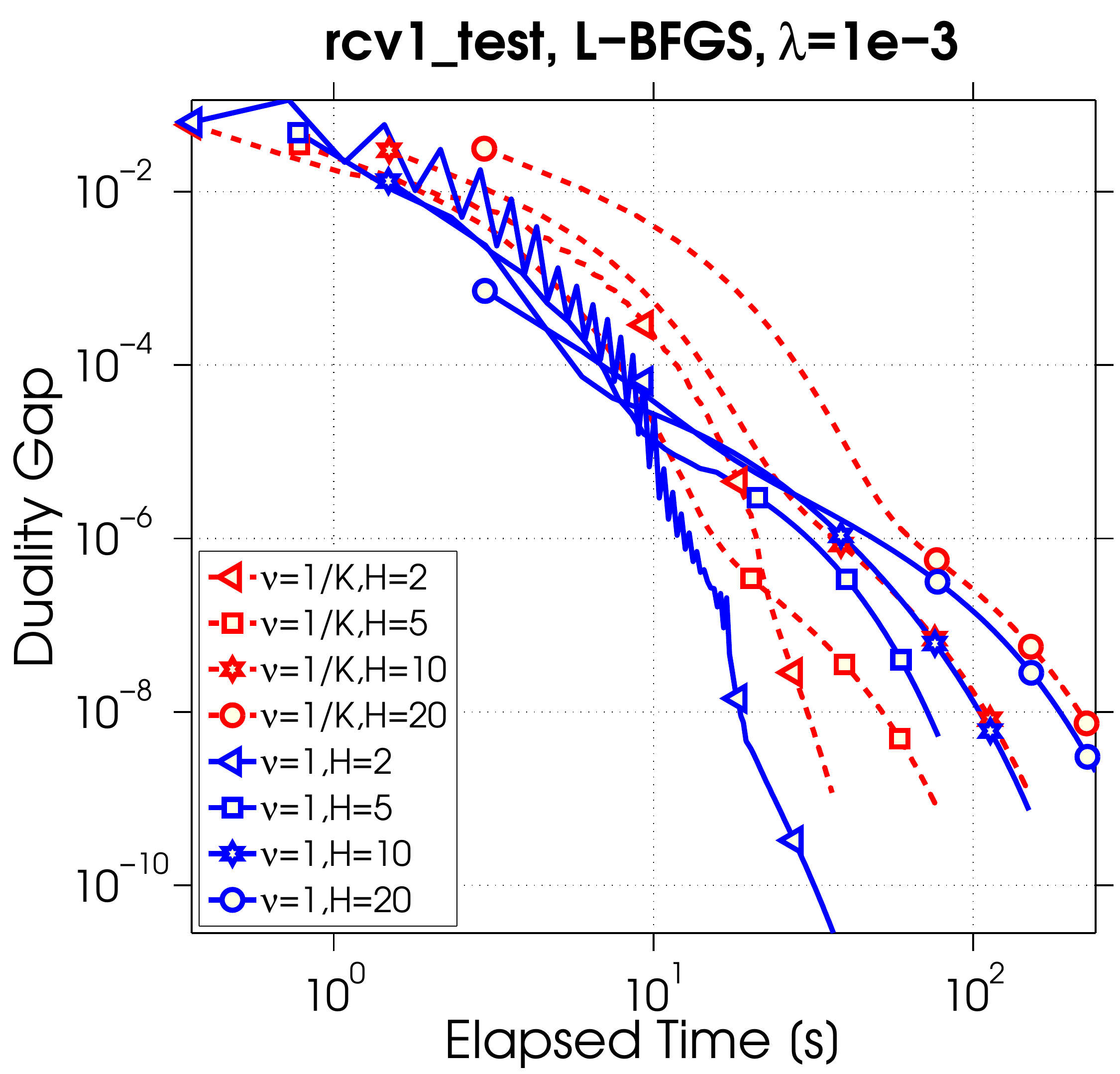}
\includegraphics[scale=.19]{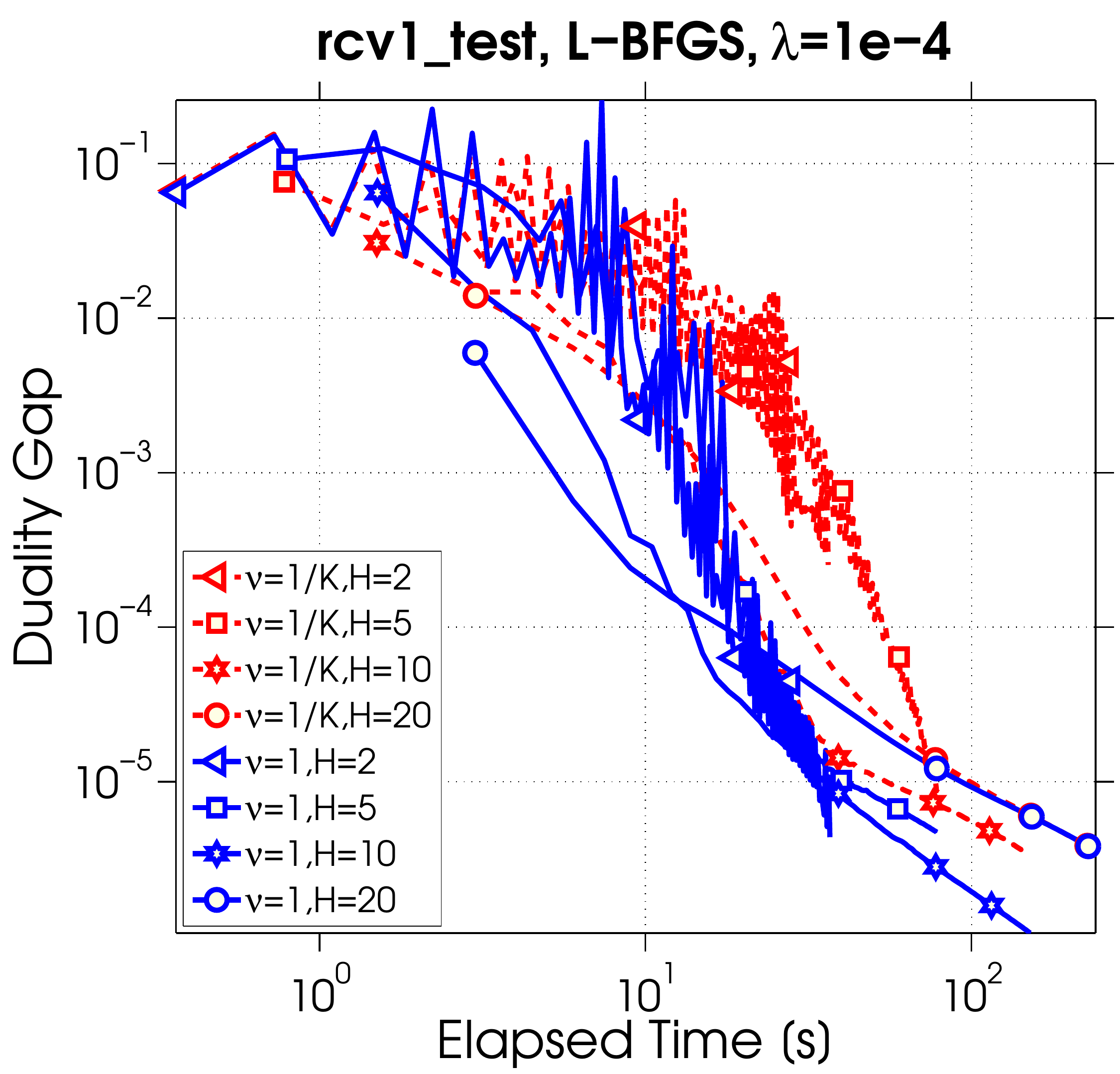}
\includegraphics[scale=.19]{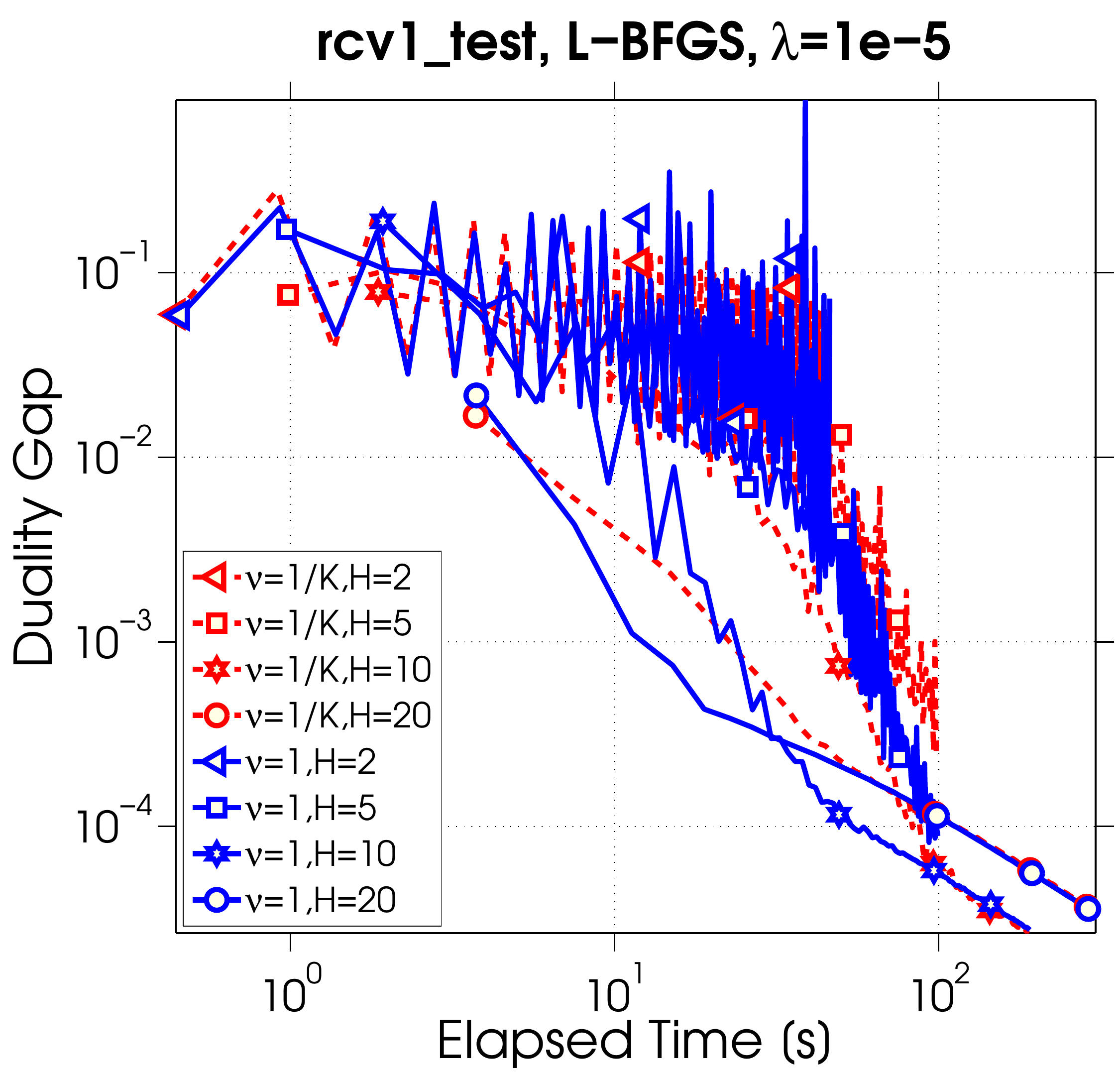}
\caption{Adding (blue solid line) vs Averaging (red dashed line) for L-BFGS as the local solver.} 
\label{fig:soler4}
\end{figure}

\begin{figure}[H]
\centering
\includegraphics[scale=.19]{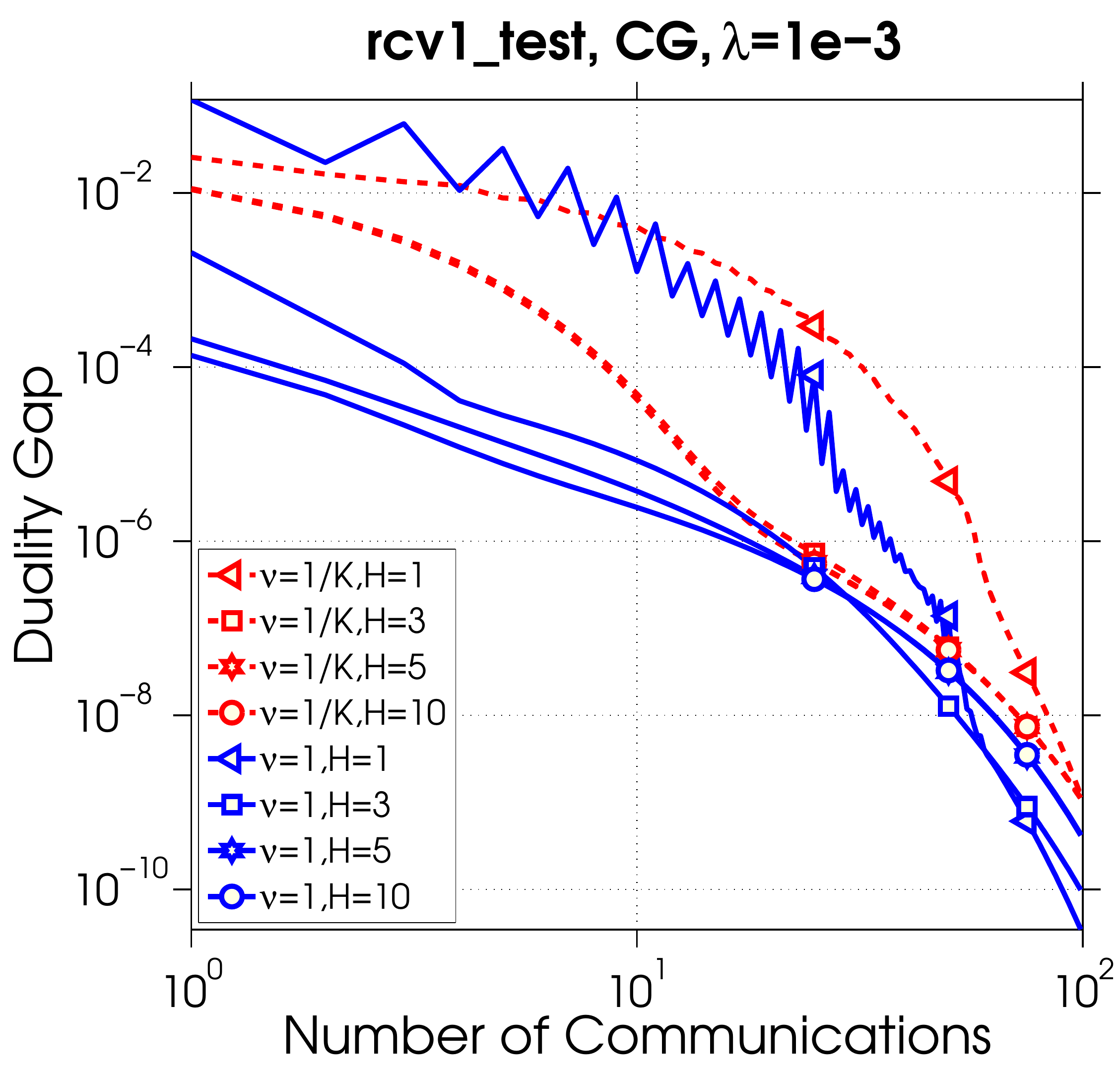}
\includegraphics[scale=.19]{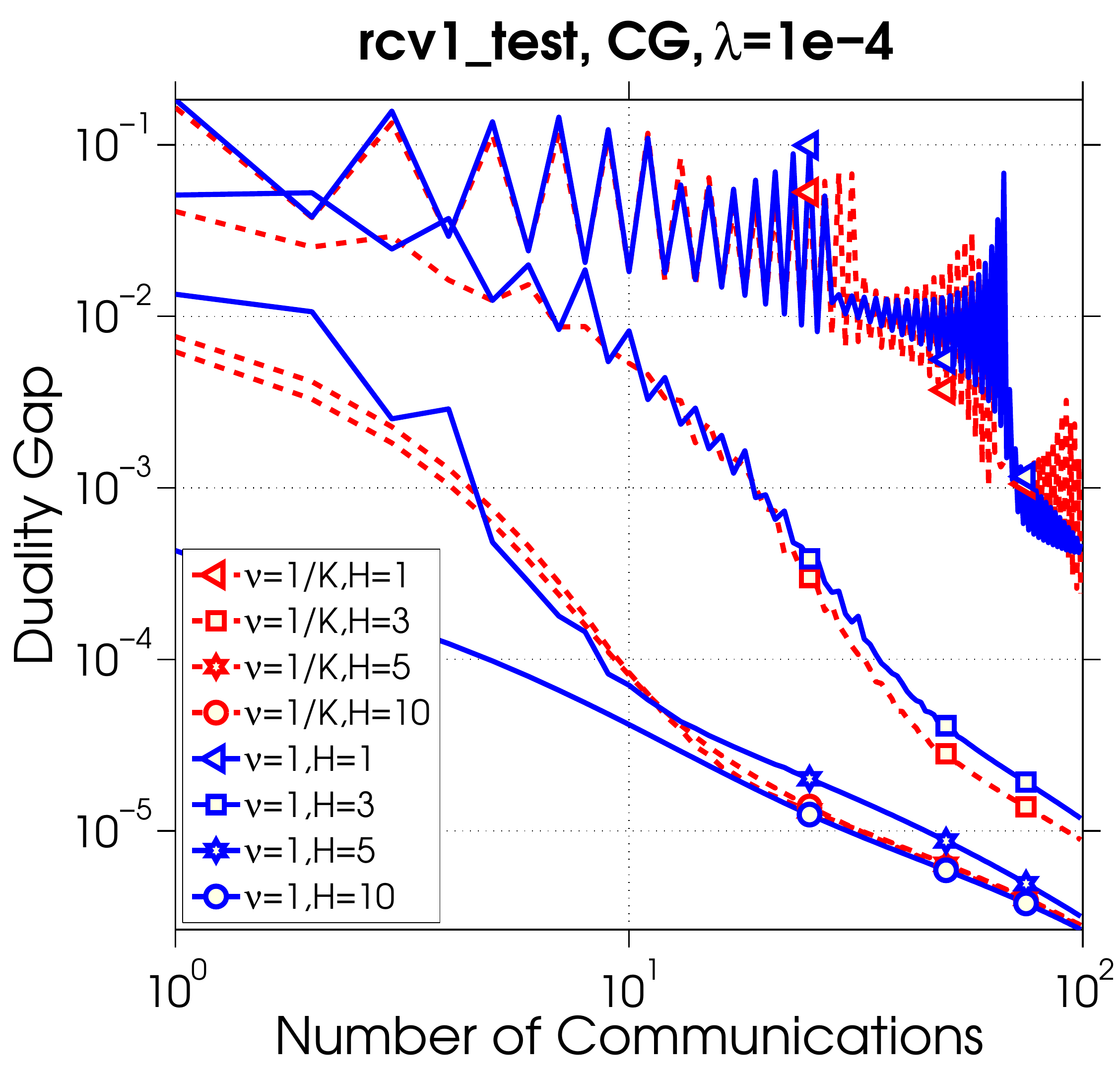}
\includegraphics[scale=.19]{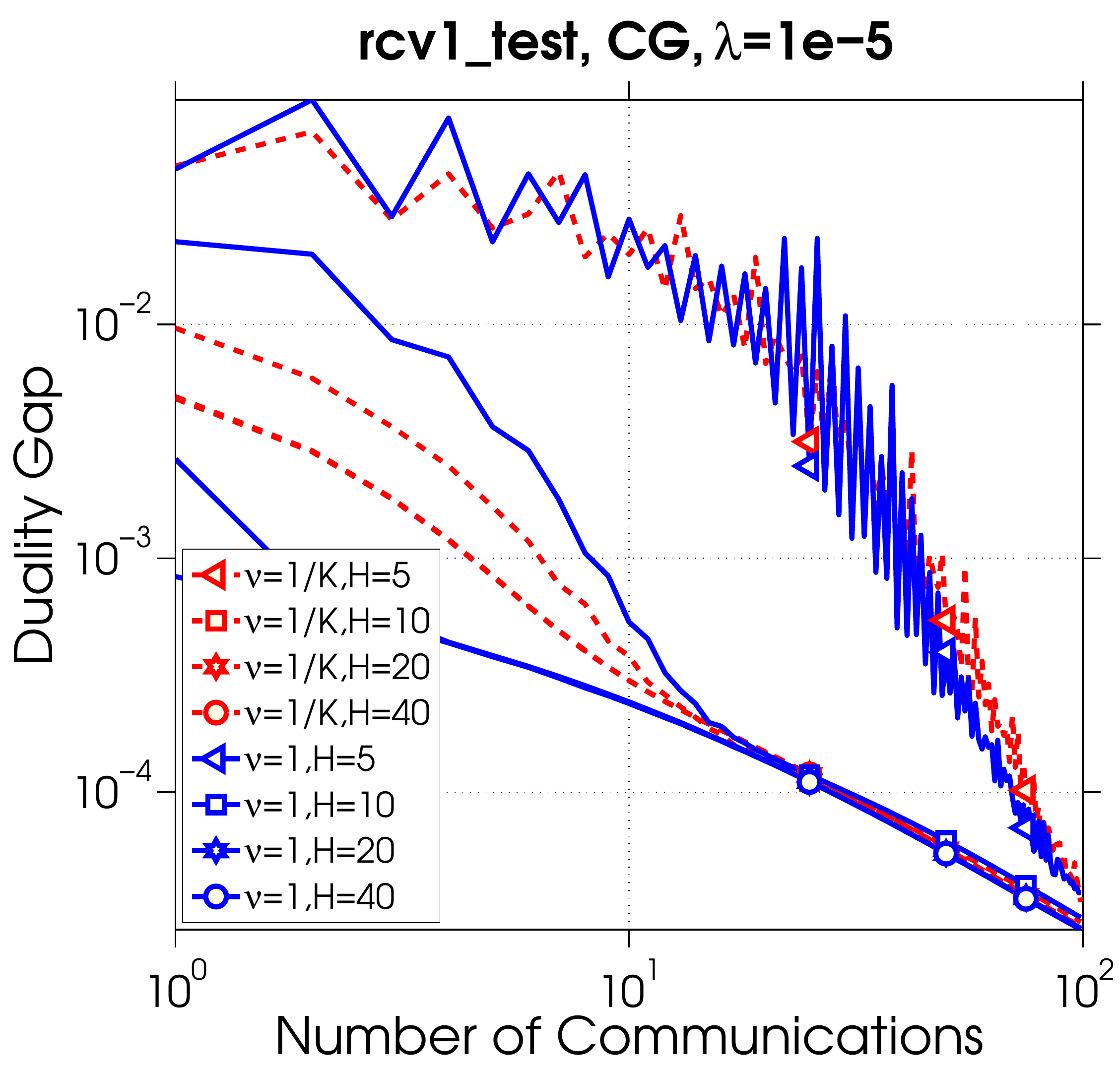}

\includegraphics[scale=.19]{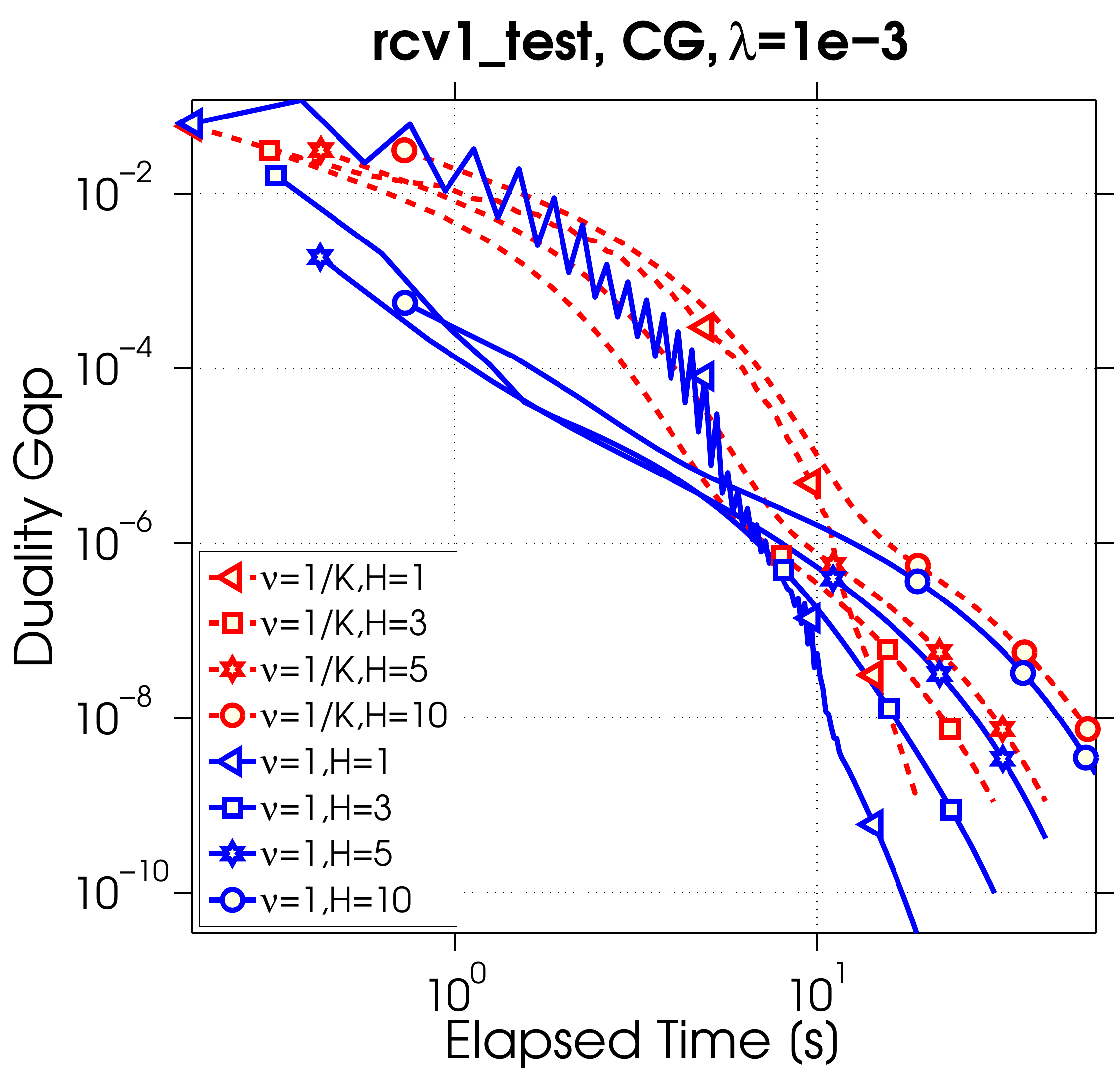}
\includegraphics[scale=.19]{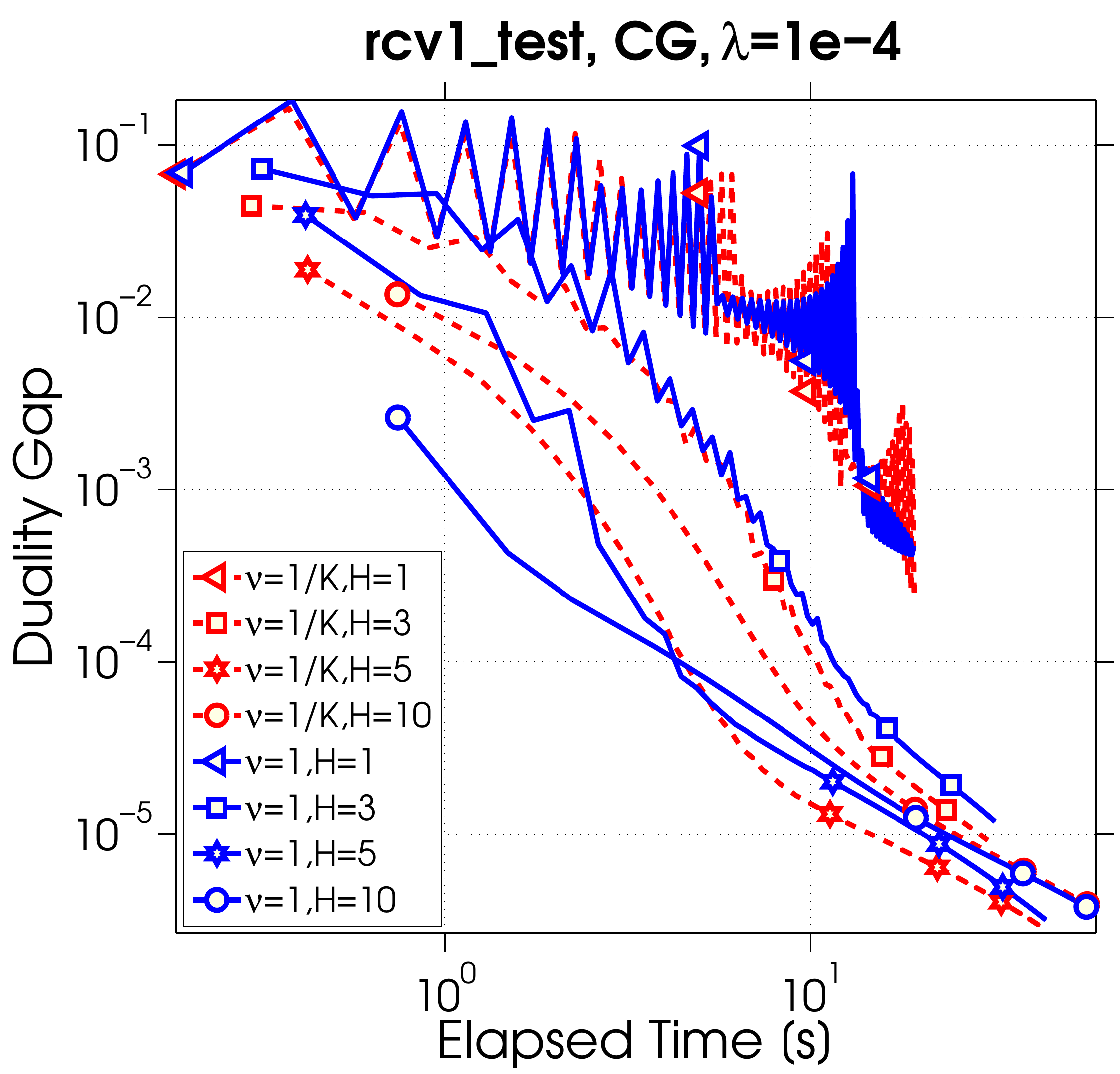}
\includegraphics[scale=.19]{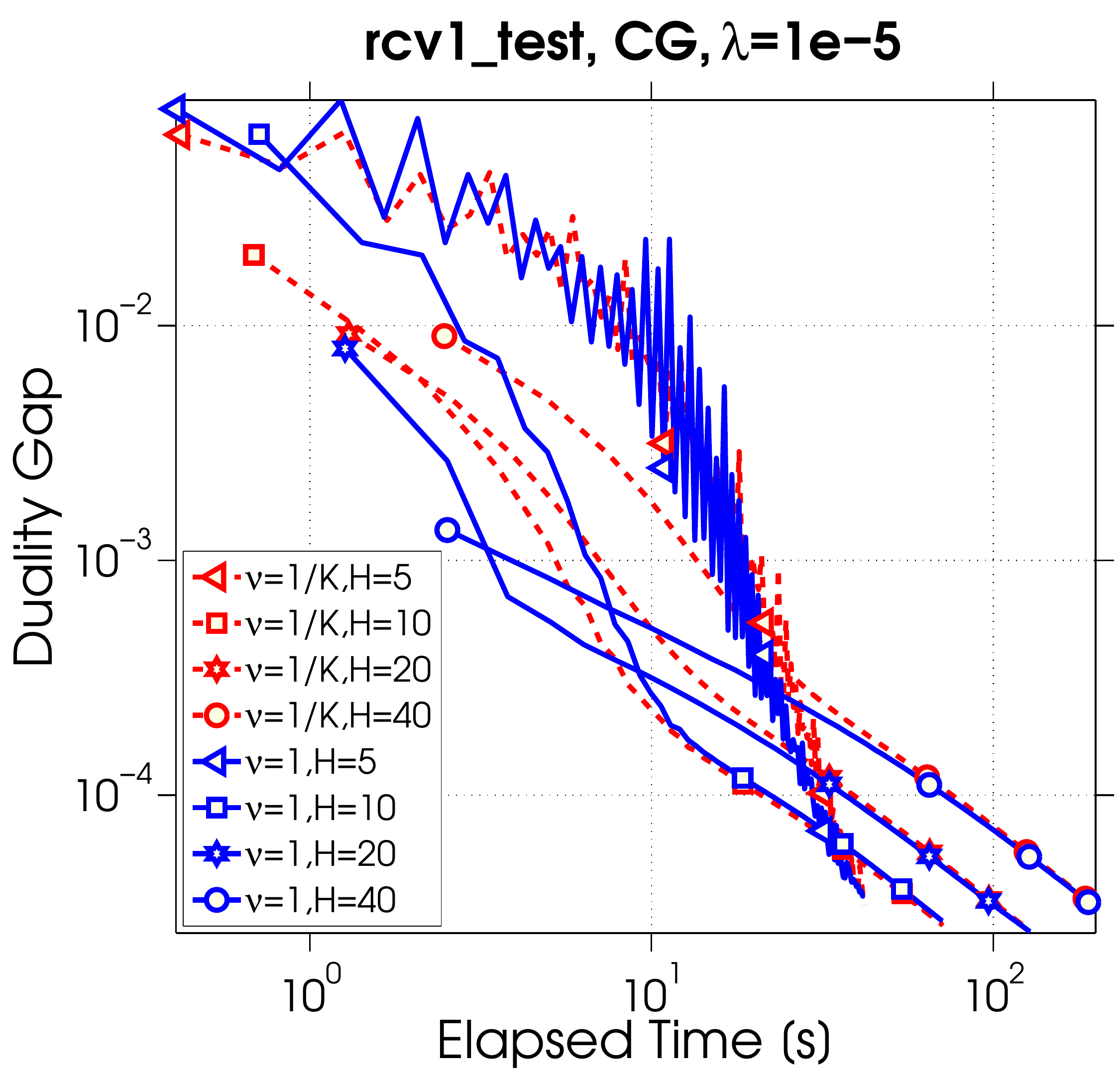}
\caption{Adding (blue solid line) vs Averaging (red dashed line) for Conjugate Gradient Method as the local solver.} 
\label{fig:soler5}
\end{figure}

\begin{figure}[H]
\centering
\includegraphics[scale=.19]{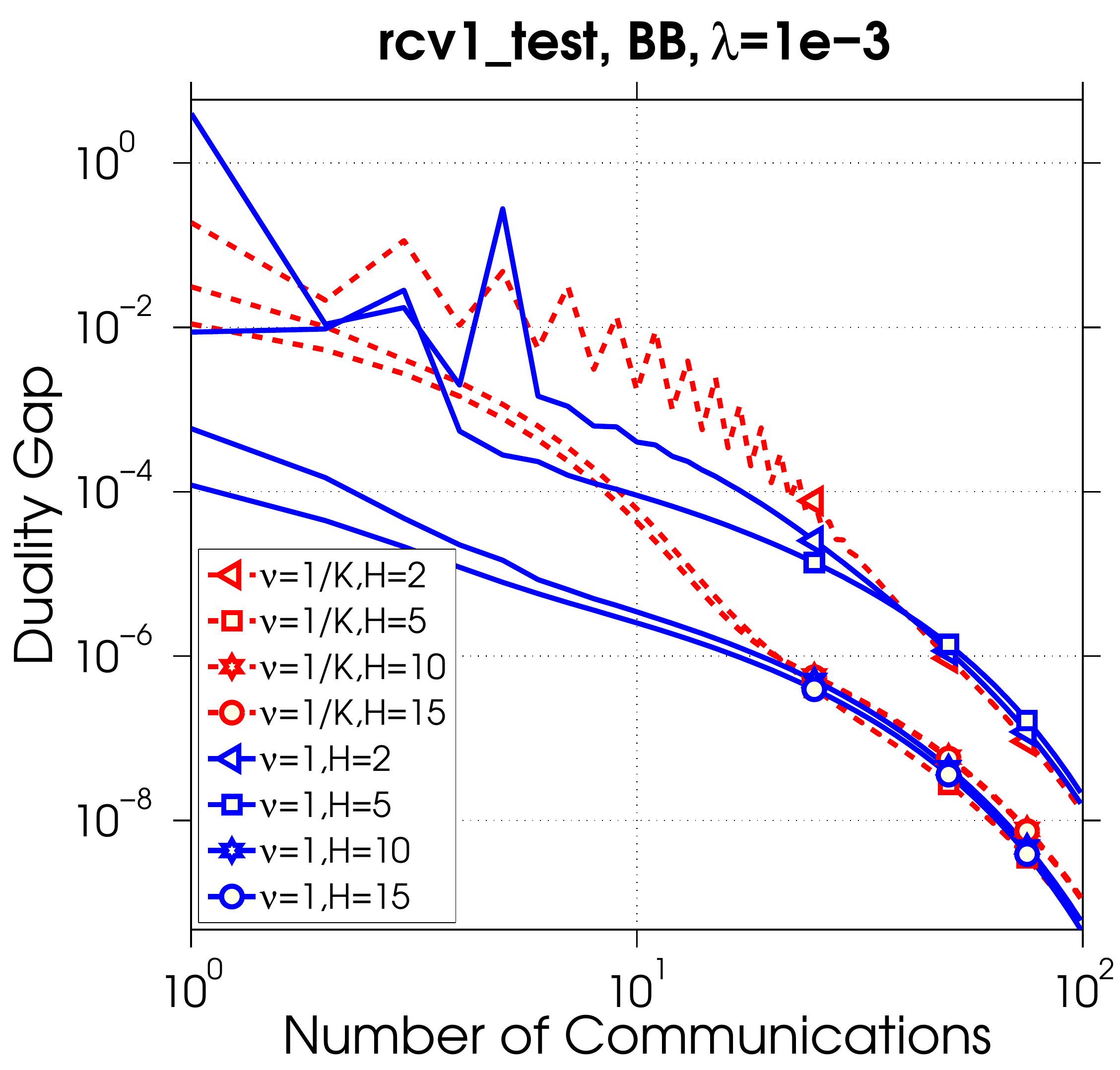}
\includegraphics[scale=.19]{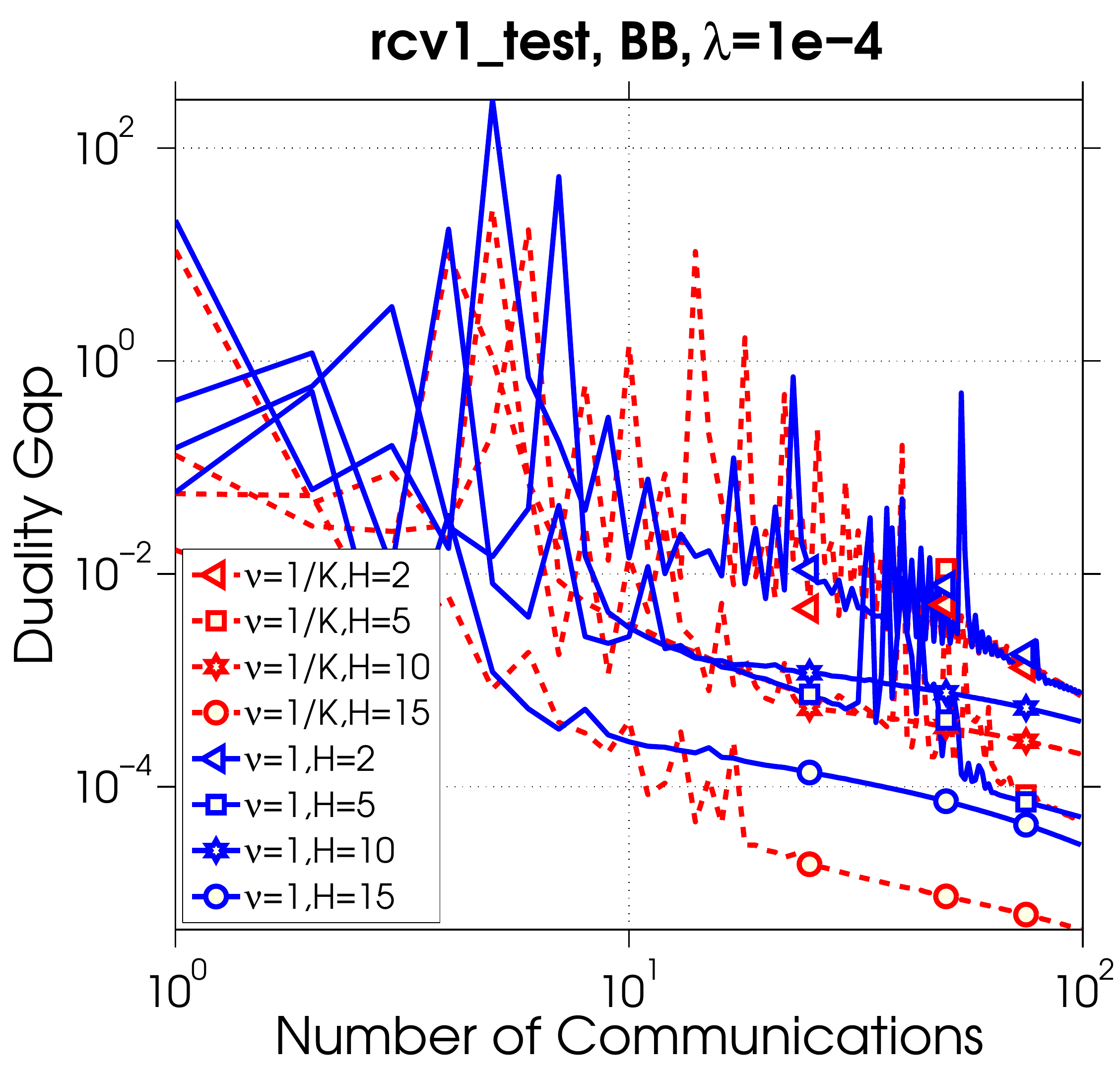}
\includegraphics[scale=.19]{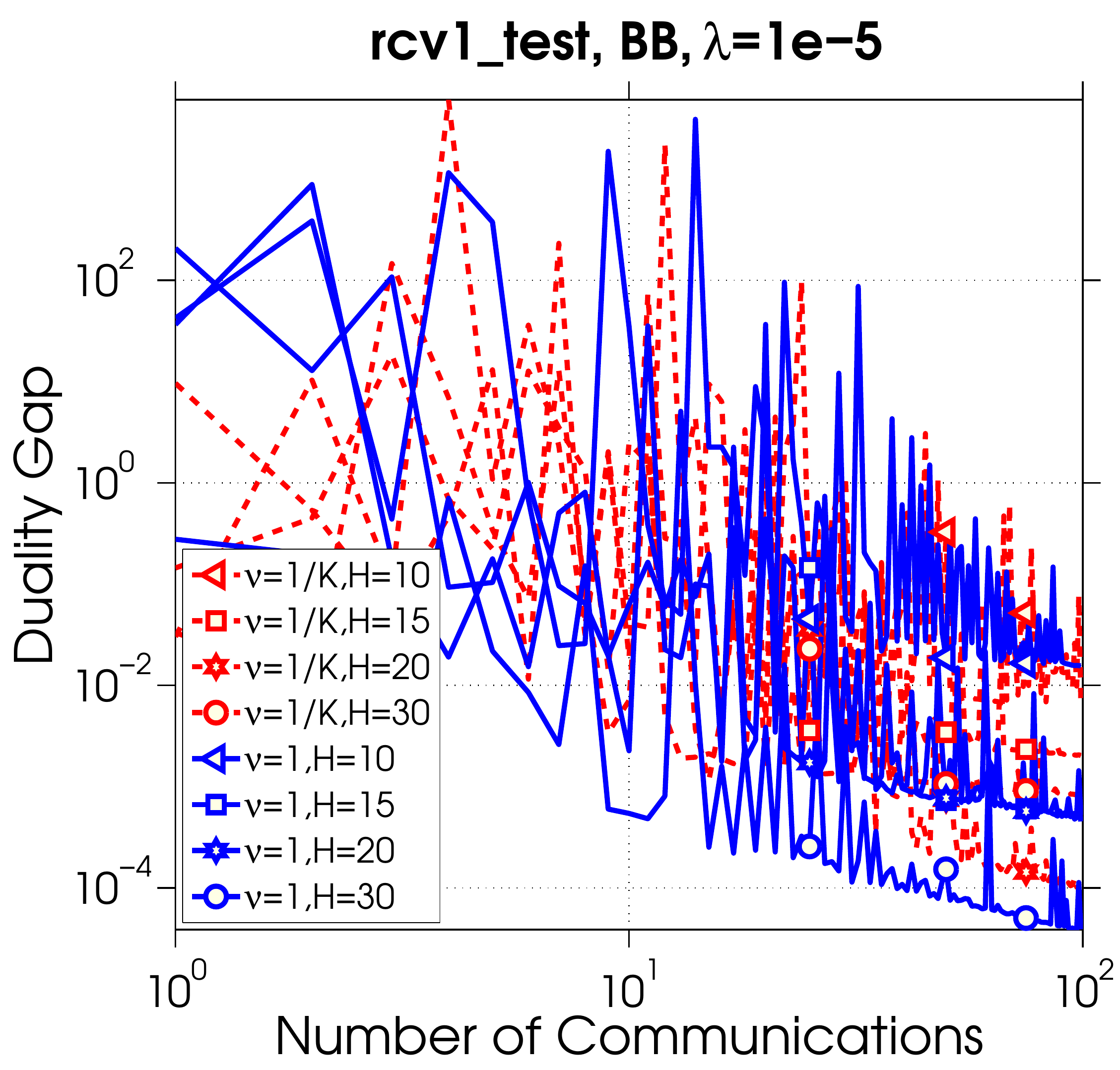}

\includegraphics[scale=.19]{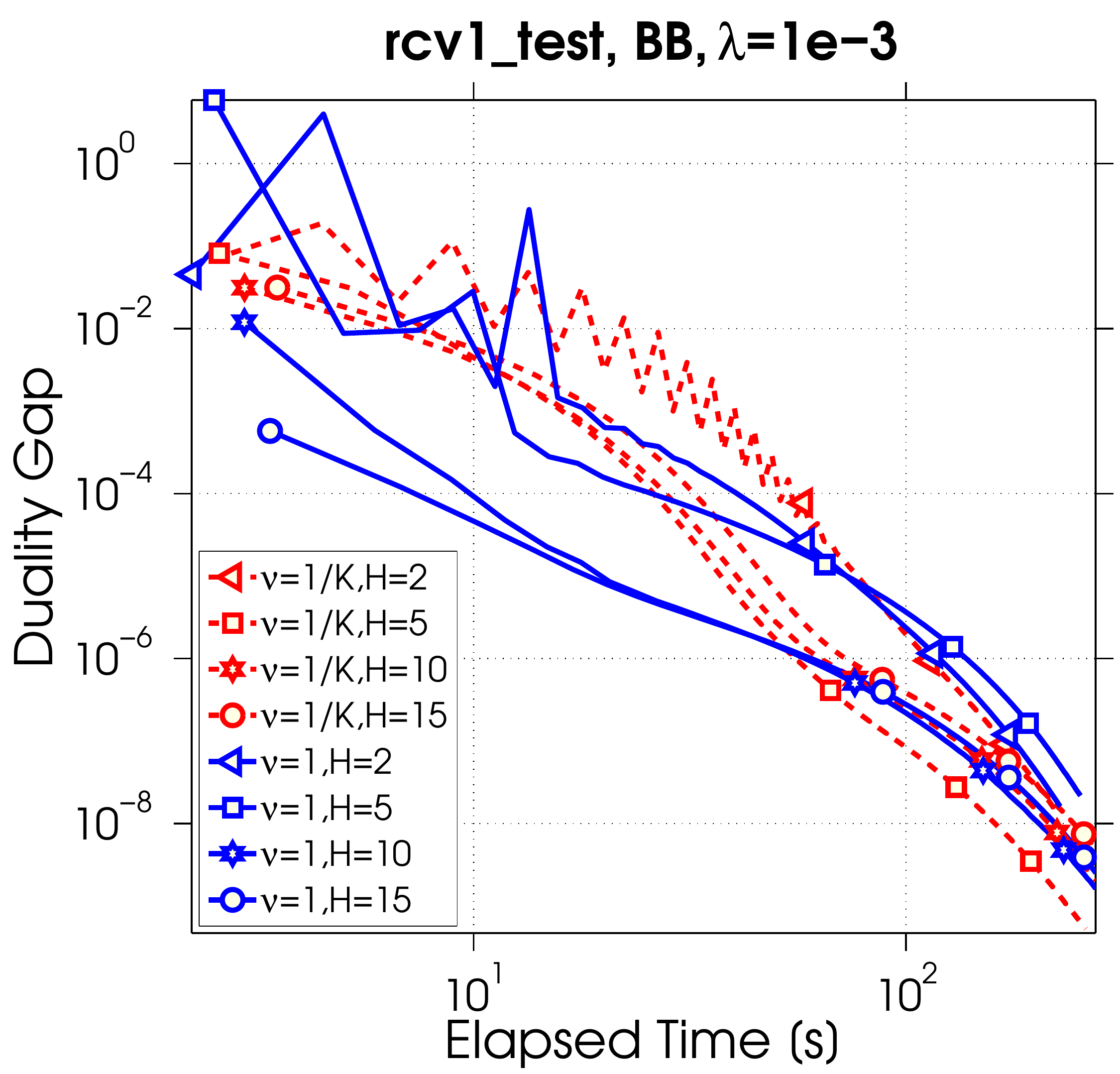}
\includegraphics[scale=.19]{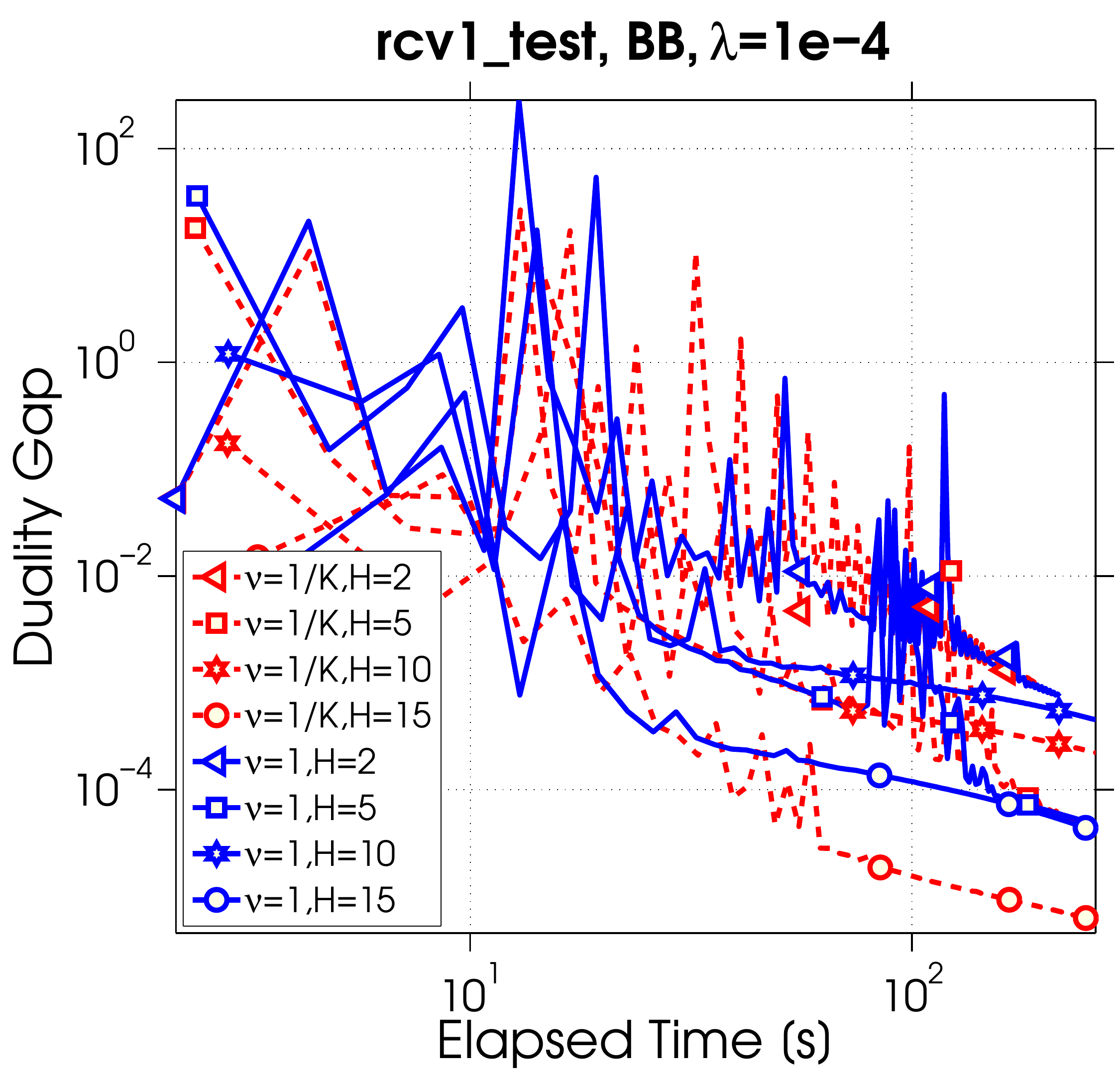}
\includegraphics[scale=.19]{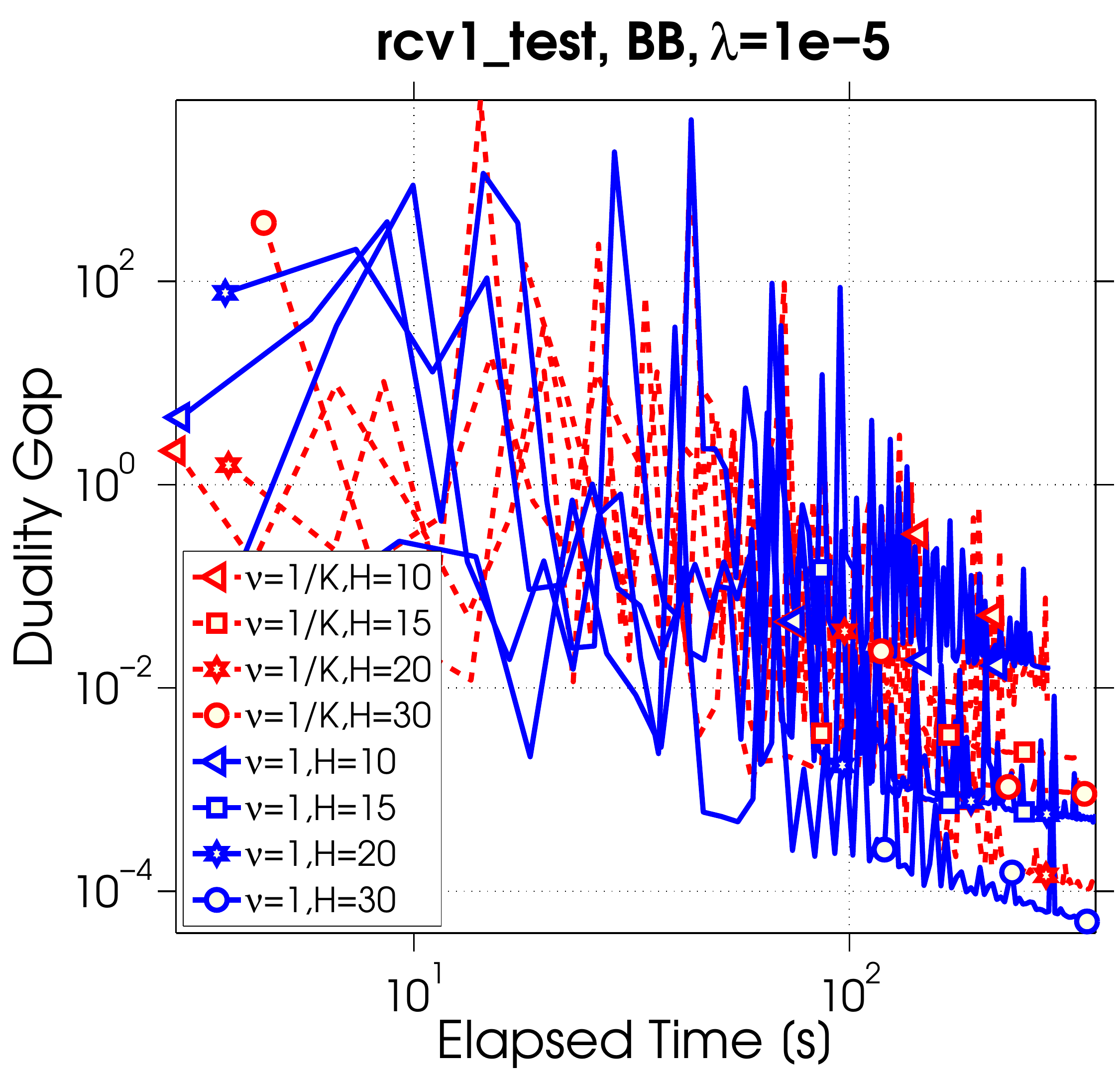}
\caption{Adding (blue solid line) vs Averaging (red dashed line) for BB as the local solver.} 
\label{fig:soler6}
\end{figure}

\begin{figure}[H]
\centering
\includegraphics[scale=.19]{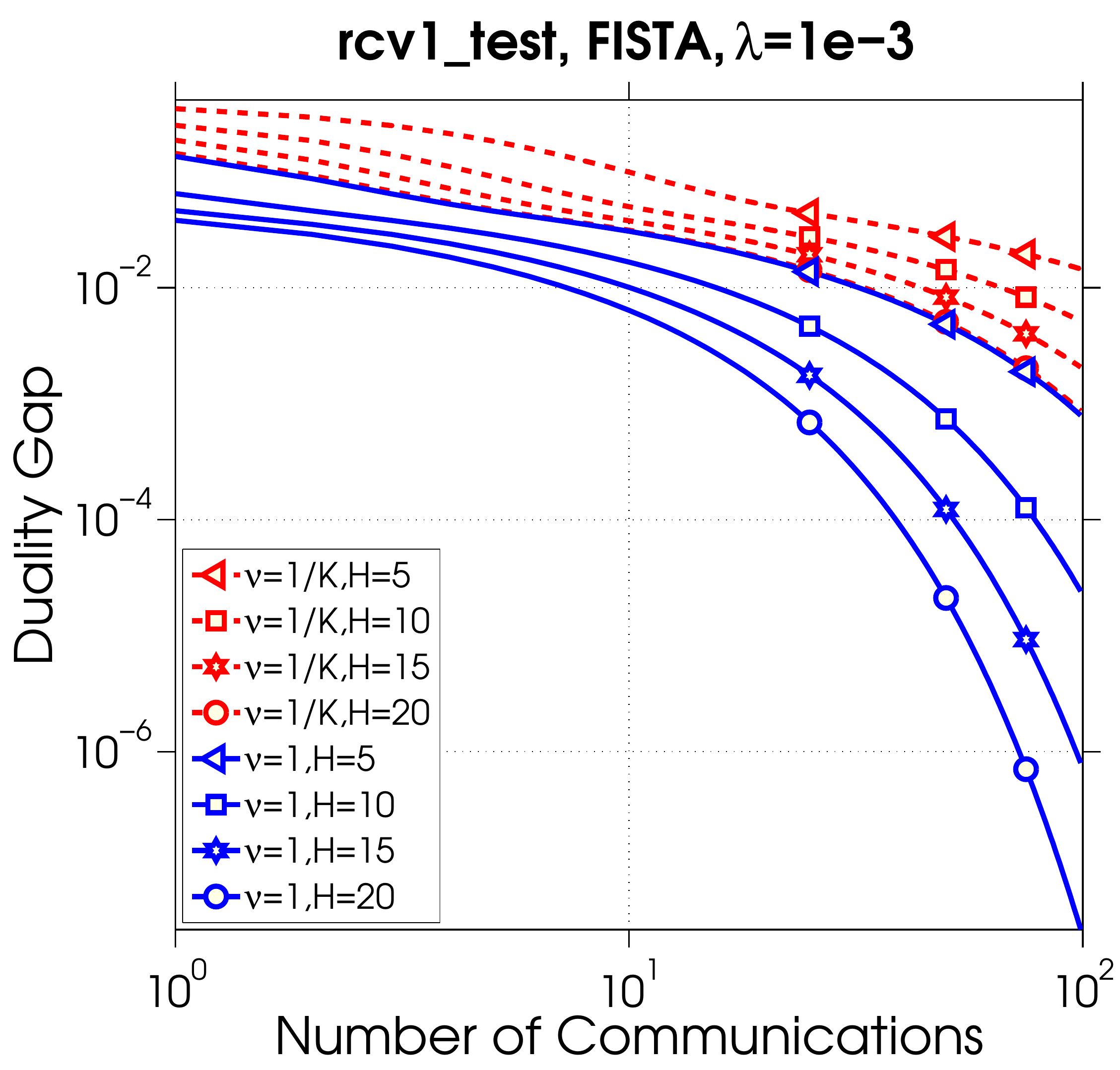}
\includegraphics[scale=.19]{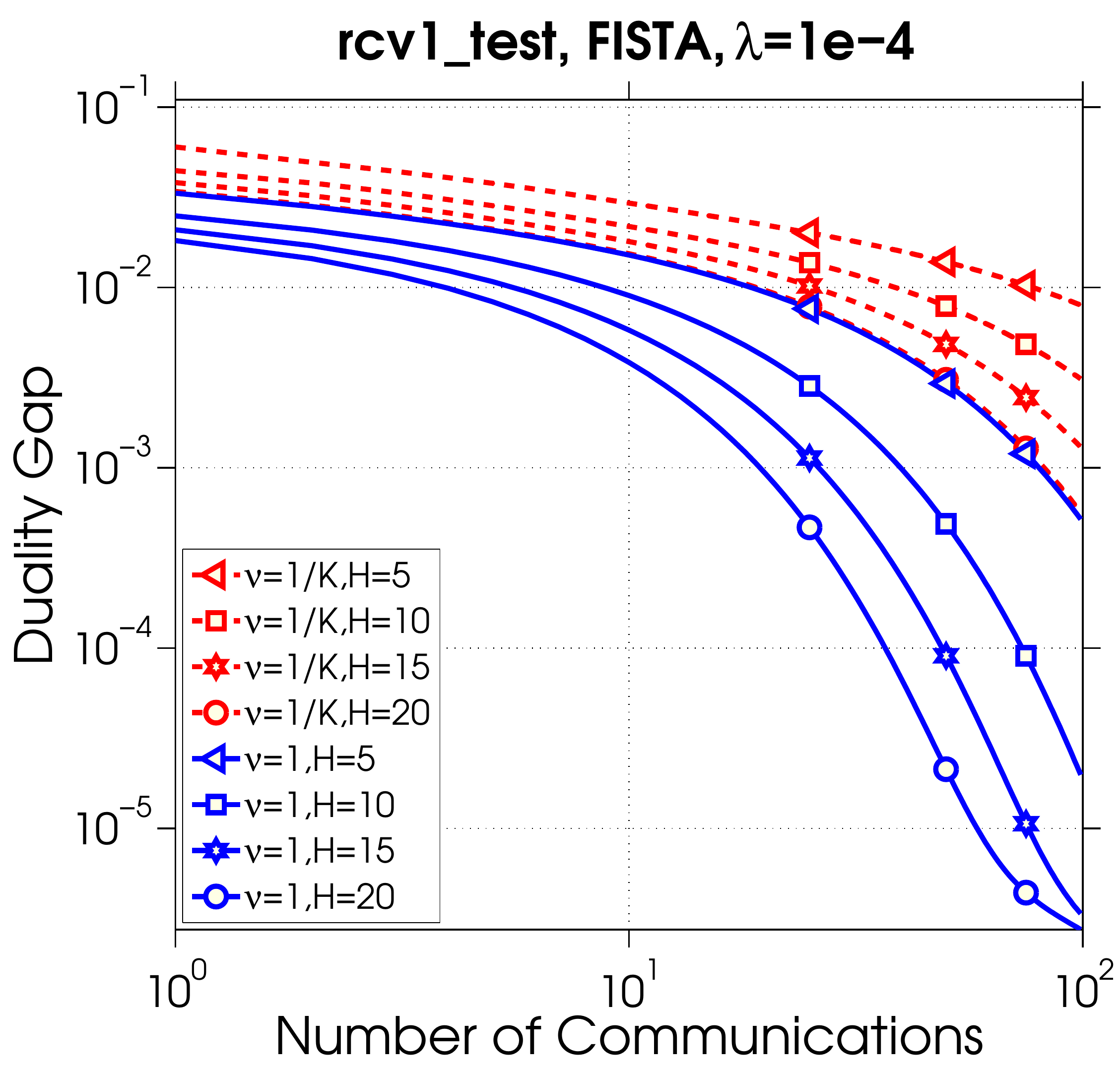}
\includegraphics[scale=.19]{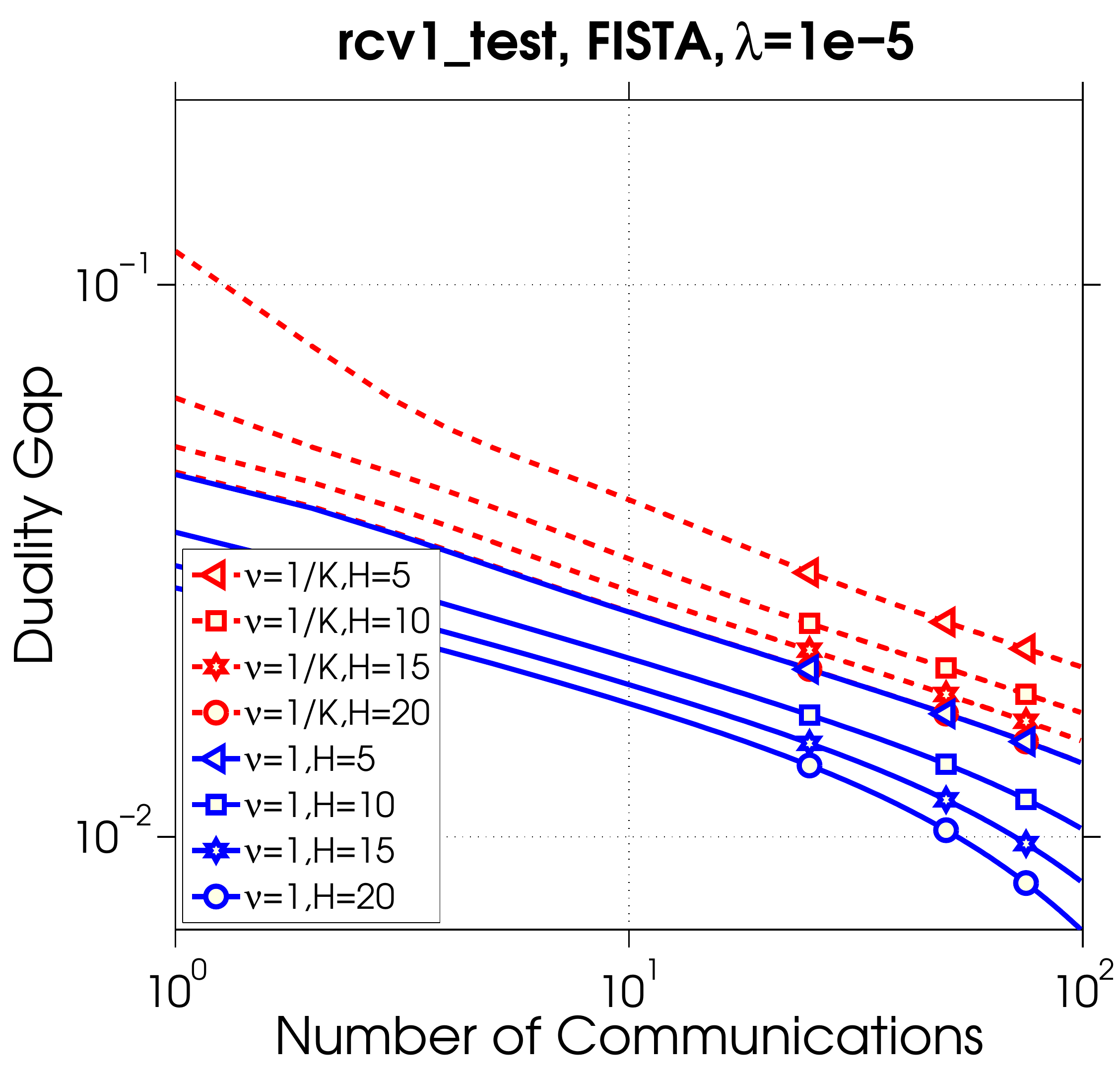}

\includegraphics[scale=.19]{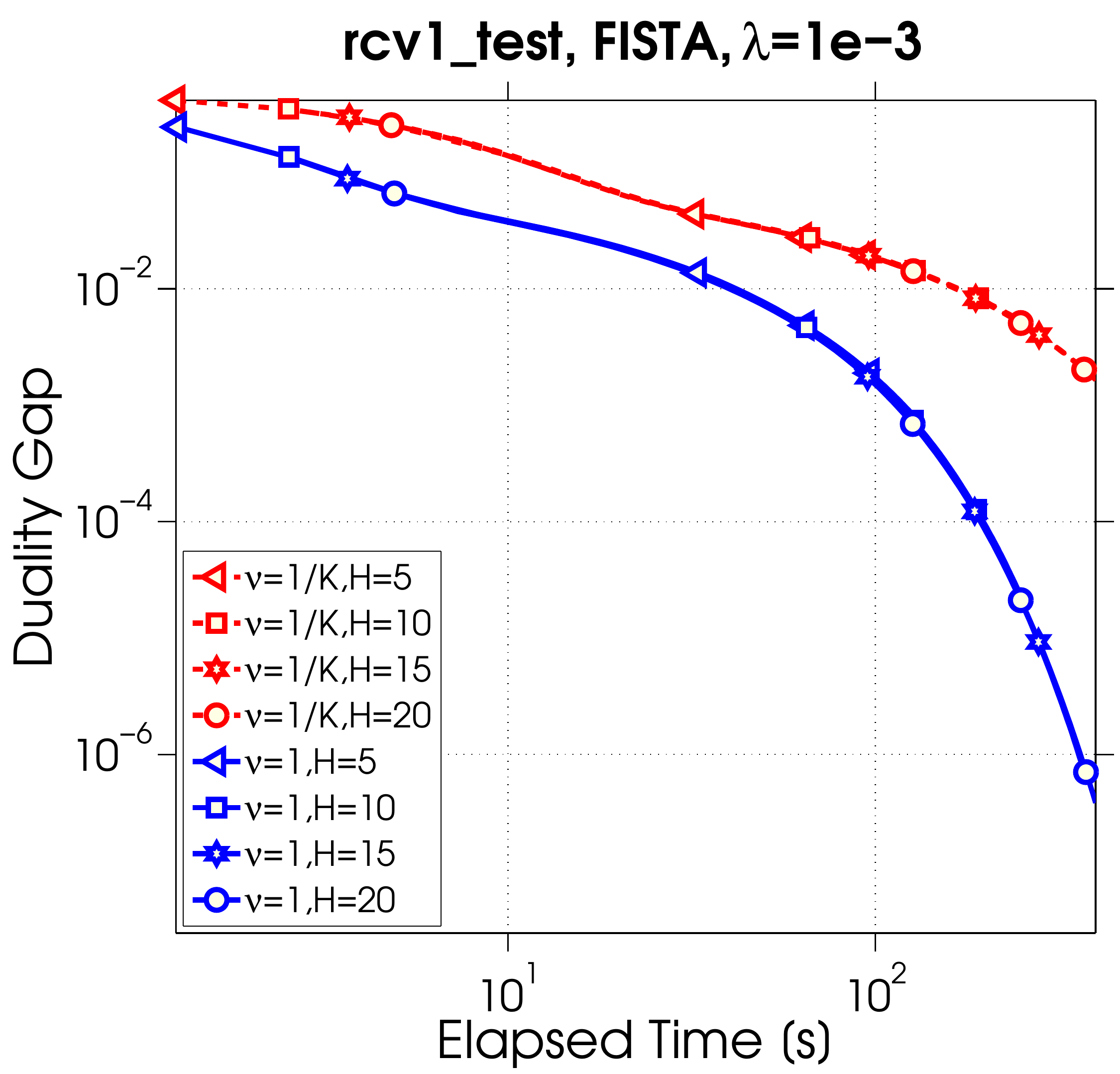}
\includegraphics[scale=.19]{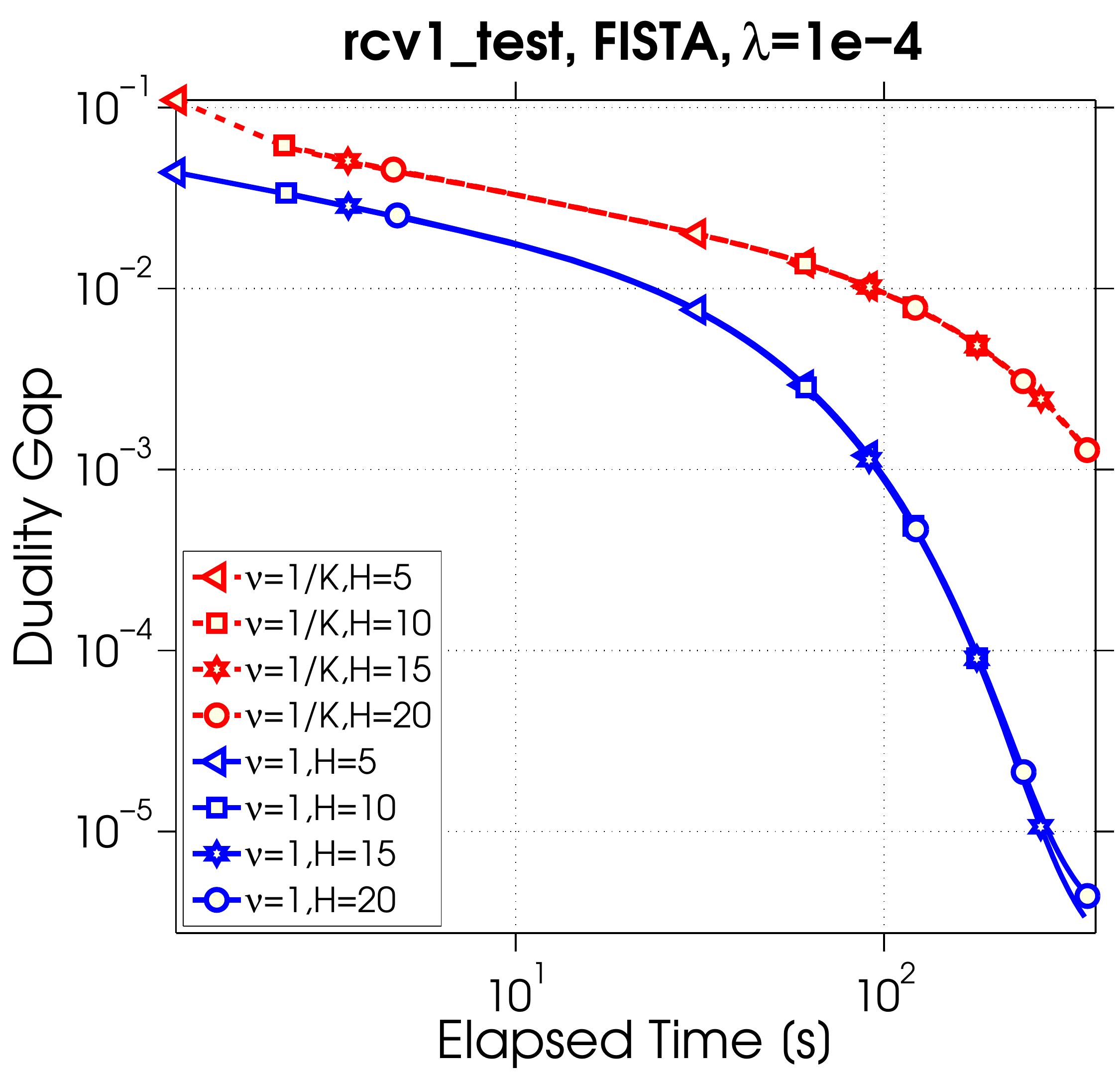}
\includegraphics[scale=.19]{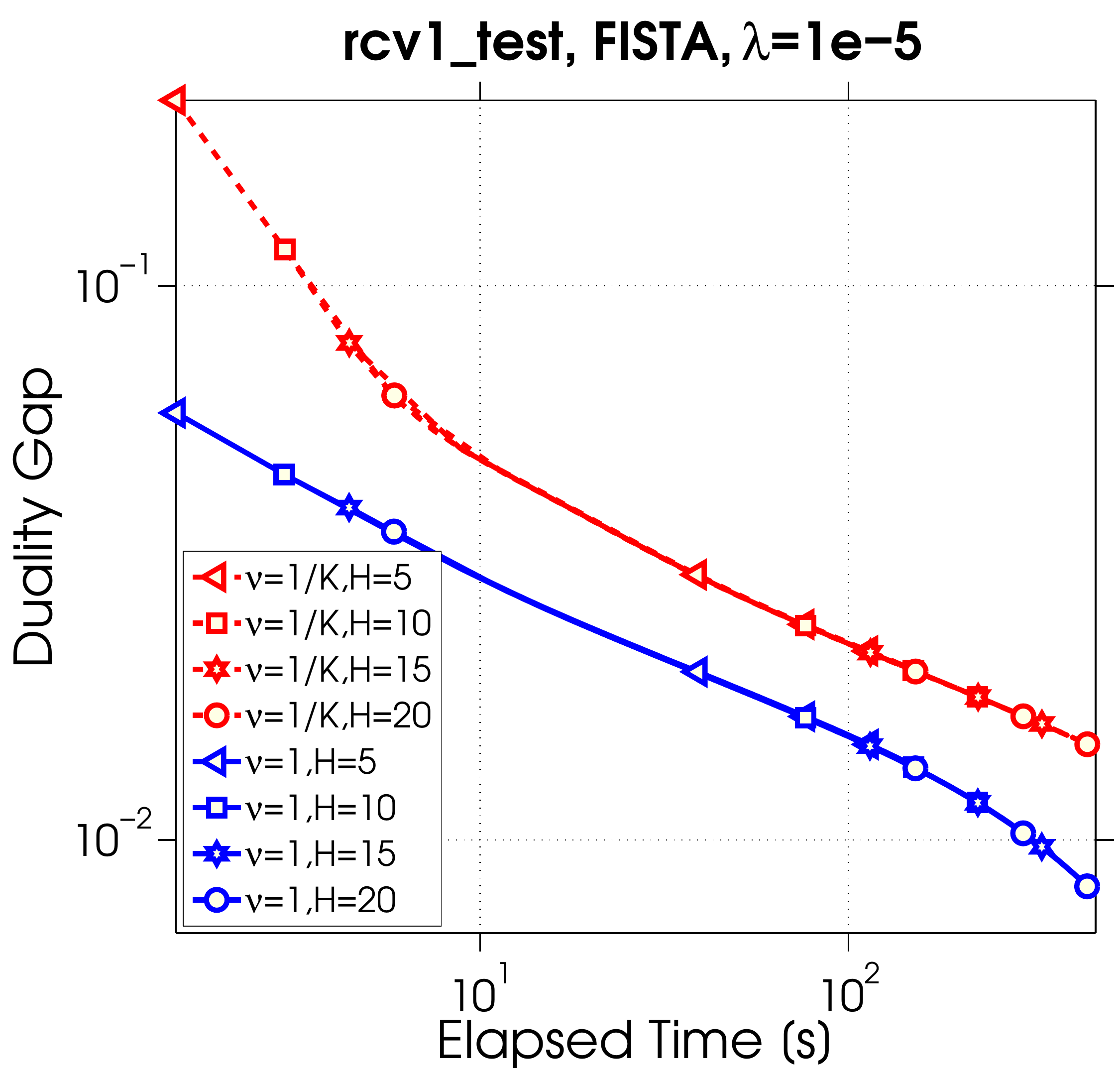}
\caption{Adding (blue solid line) vs Averaging (red dashed line) for FISTA as the local solver.} 
\label{fig:soler7}
\end{figure}

\subsection{The Effect of the Subproblem Parameter $\sigma'$}
\label{sec:cocoa:subproblemParamExps}

In this section we consider the effect of the choice of the subproblem parameter on convergence (Figure \ref{fig:sigmaTest}). We plot {the} duality gap over the number of communications for {the rcv1\_test} and epsilon datasets with quadratic loss{,} and set $K = 8$, $\lambda = 10^{-5}$. For $\aggpar=1$ (adding the local updates), we consider several different values of~$\sigma'$, ranging from $1$ to $8$. The value $\sigma'=8$ represents the safe upper bound of $\aggpar K$, as given in Lemma \ref{lem:sigmaPrimeNotBad}. 

Decreasing $\sigma'$ improves performance in terms of communication until a certain point, after which the algorithm diverges.  For the {rcv1\_test} dataset, the optimal convergence occurs around $\sigma'=5$, and diverges fast for $\sigma' \le 3$. For {the} epsilon dataset, $\sigma'$ around $6$ is the best choice and the algorithm will not converge to {the} optimal solution if $\sigma'\leq 5.$ However, more importantly, the ``safe'' upper bound of $\sigma':=\aggpar K=8$ has only slightly worse performance than the practically best (but ``un-safe'') value of~$\sigma'$. 

\begin{figure}[H]
\centering
\includegraphics[scale=.22]{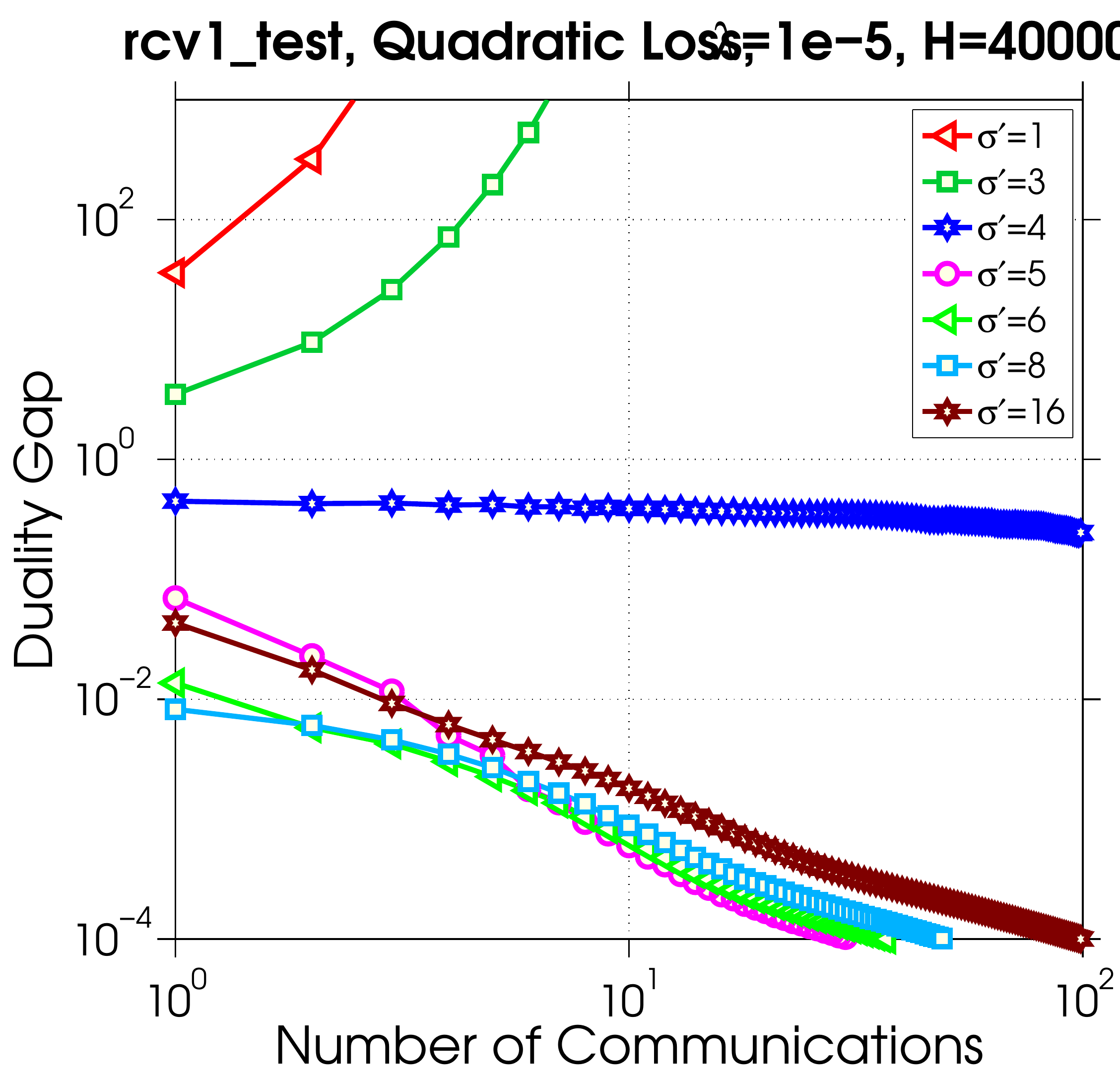}
\includegraphics[scale=.22]{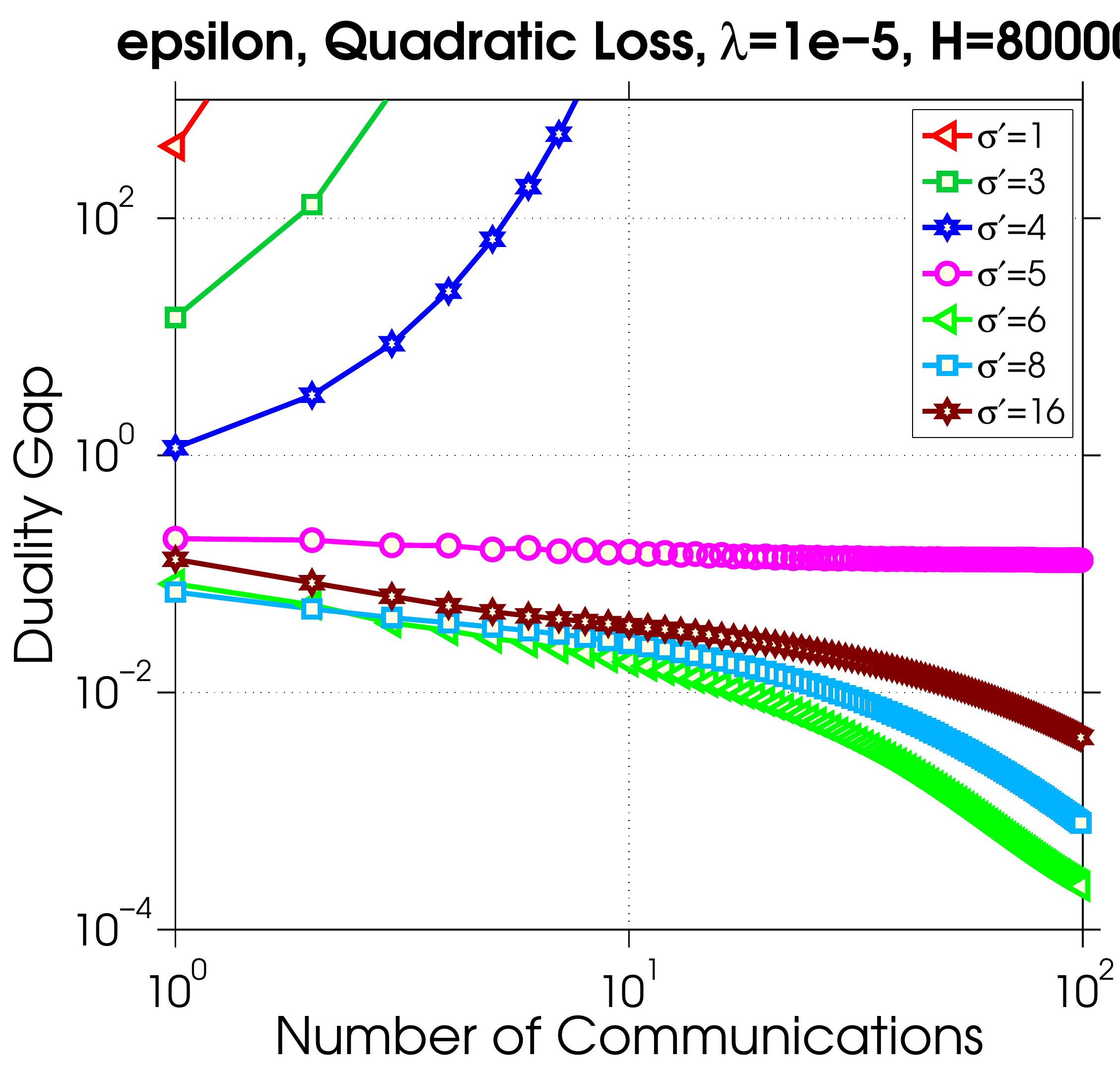}
\caption{The effect of $\sigma'$ on convergence for the rcv1\_test and epsilon datasets distributed across 8 machines.} 
\label{fig:sigmaTest}
\end{figure}

\subsection{Scaling Property}

\begin{figure}[H]
\centering
\includegraphics[scale=.19]{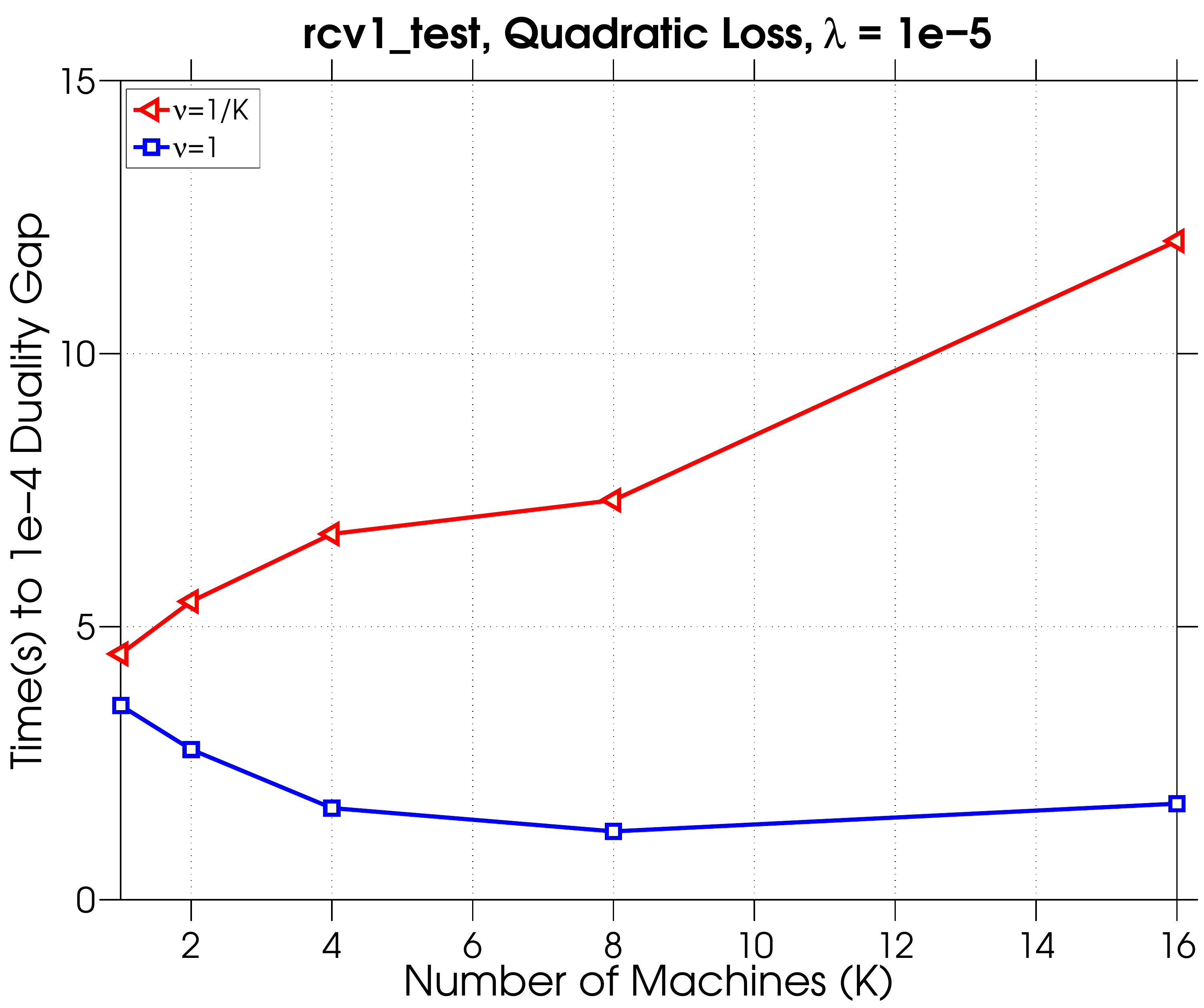}
\includegraphics[scale=.19]{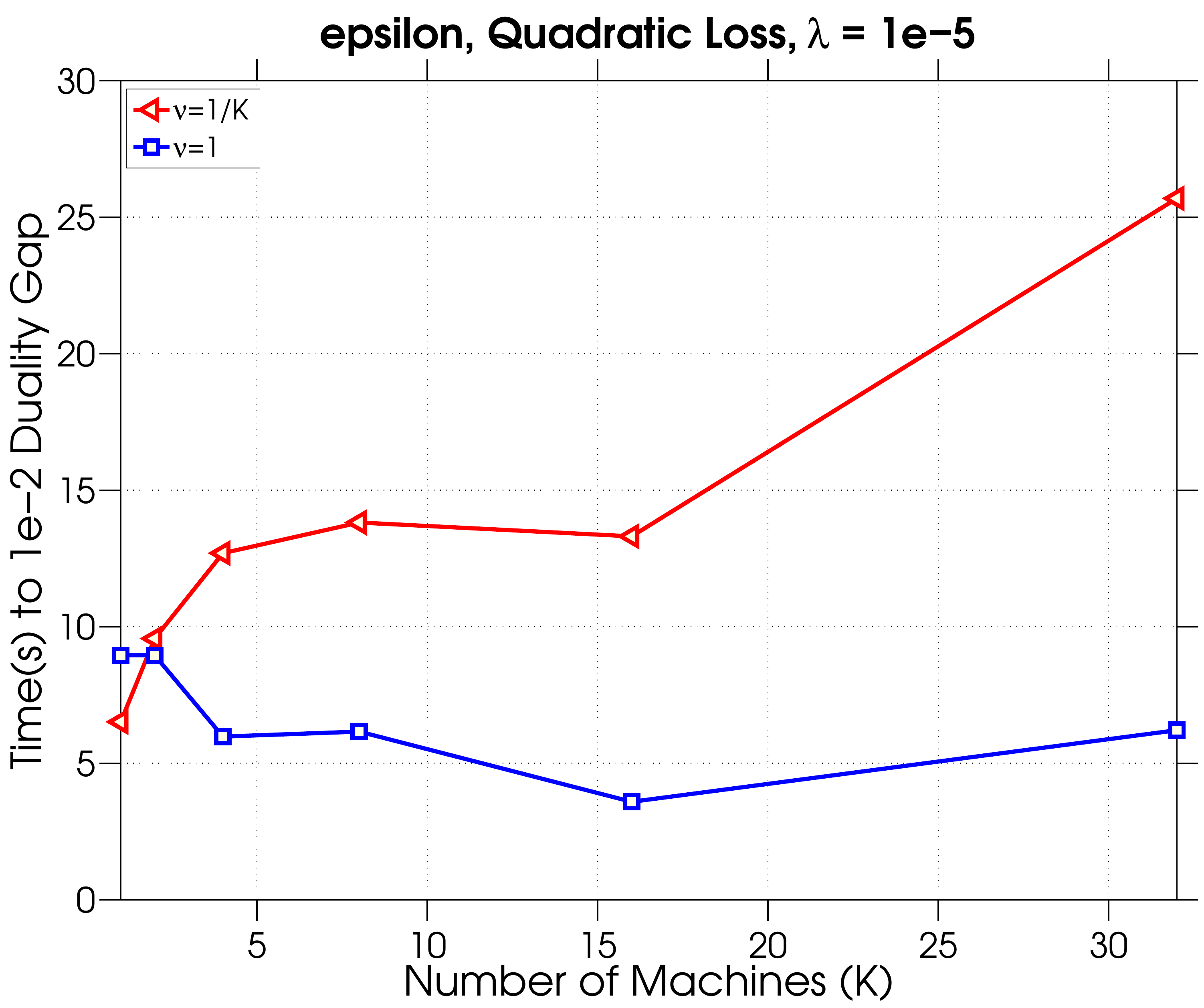}
\caption{The effect of increasing the number of machines $K$ on the time (s) to reach a solution with expected duality gap.} 
\label{fig:scaletest1}
\end{figure}

 {Here we demonstrate the ability of our framework to scale with  $K$ (number of machines). We compare the {runtime} to reach a specific tolerance on duality gap ($10^{-4}$ and $10^{-2}$) for two choices of $\aggpar$. Looking at Figure~\ref{fig:scaletest1}, we see that when choosing $\aggpar=1$, the performance improves as the number of machines increases. However, when $\aggpar = \frac{1}{K}$, the algorithm slows down as $K$ increases. {These} observations support our analysis in Section 4.  }

\subsection{Performance on a Big Dataset}
\label{sec:cocoa:hugeDatasetExp}
As shown in Figure~\ref{fig:hugedata}, we test the algorithm on the \emph{splice-site.t} dataset, whose size is about 280 GB. 
We show experiments for three different loss functions $\ell$, namely logistic loss, hinge loss and least squares loss (see Table \ref{tbl:differentLossFunctions}).
We set $\lambda = 10^{-6}$ for the squared norm regularizer. 
The dataset is distributed  across $K=4$ machines and we use {CD} as the local solver with $H = 50,000$. In all the cases, an optimal solution can be reached in about 20 minutes and again, we observe that setting the aggregation parameter $\nu:=1$ leads to faster convergence than $\nu:=\frac 1 K$ (averaging). 

Also, the number of communication rounds for the three different loss functions are almost the same if we set all the other parameters to be same. However, the {duality gap decreases in a different manner for the three loss functions}.

\begin{figure}[H]
\centering
\includegraphics[scale=.19]{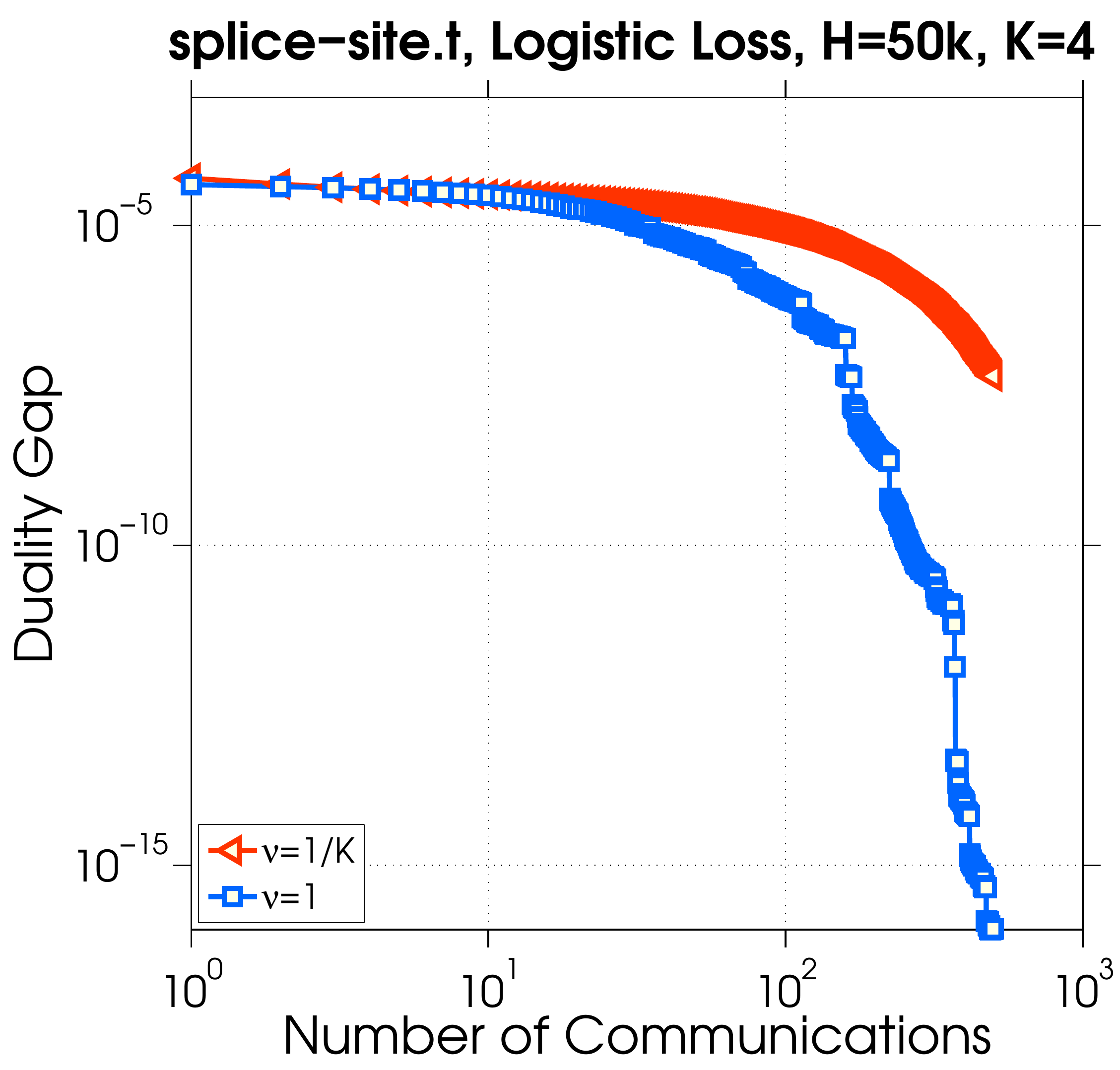}
\includegraphics[scale=.19]{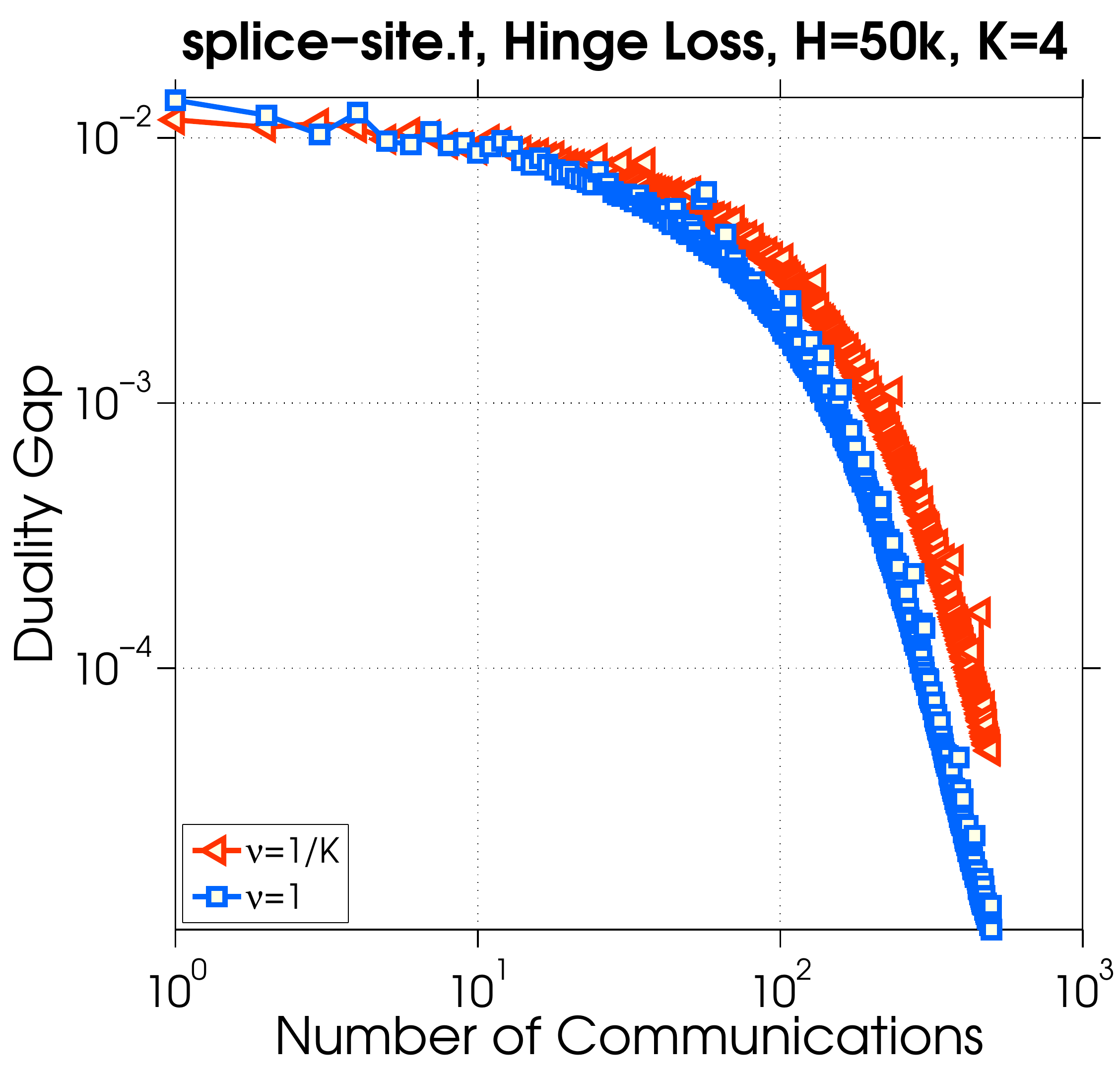}
\includegraphics[scale=.19]{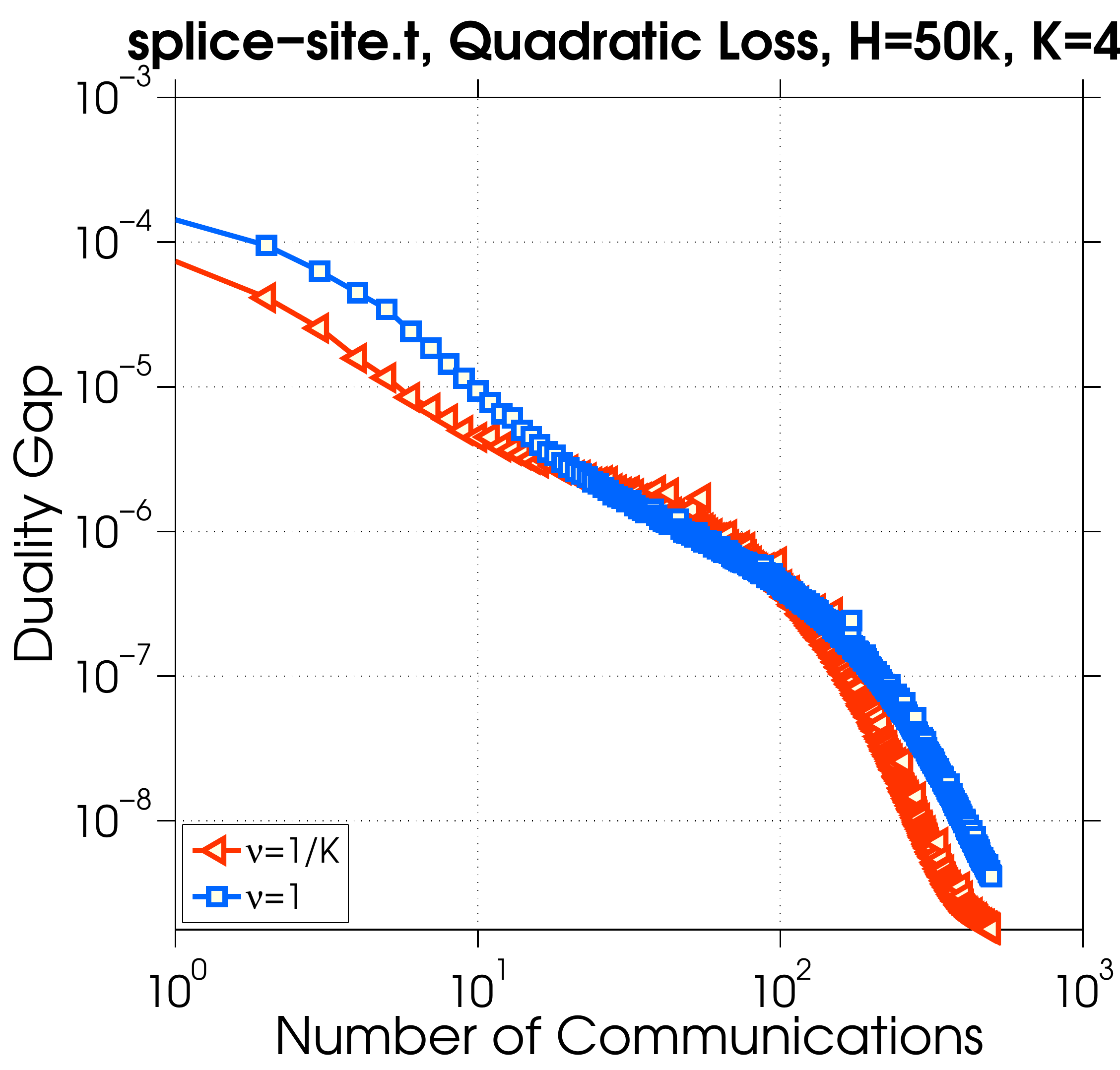}

\includegraphics[scale=.19]{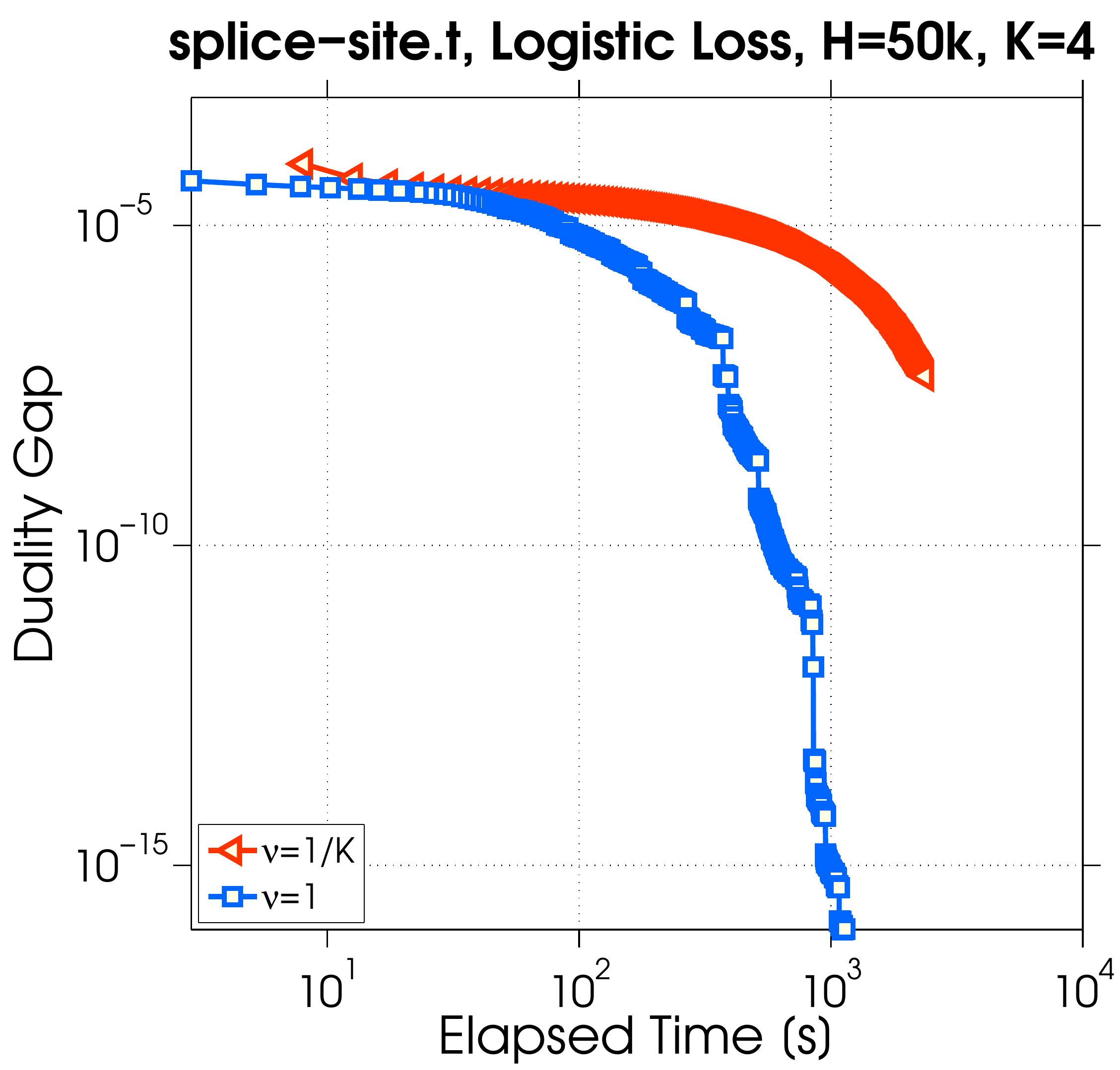}
\includegraphics[scale=.19]{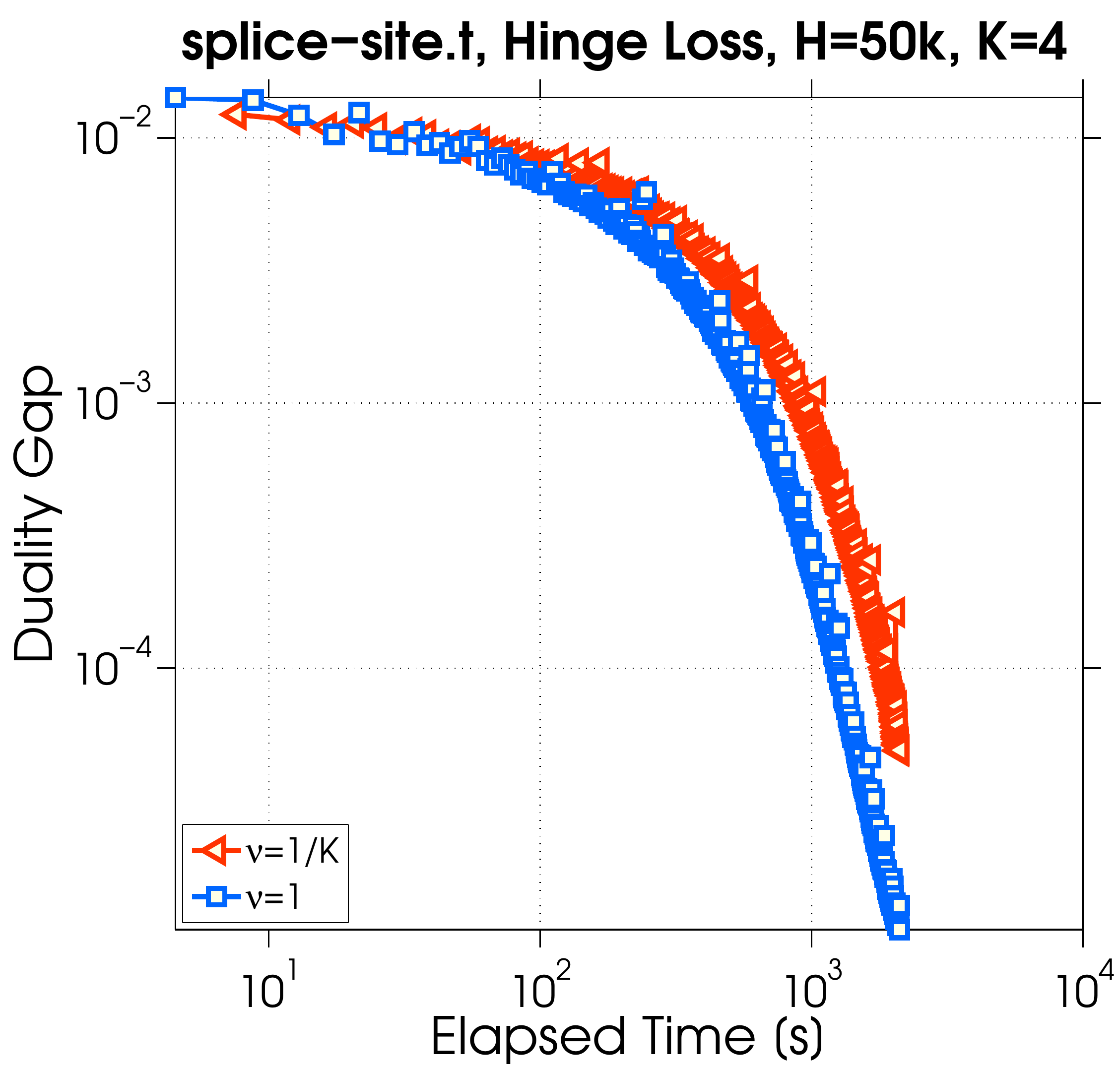}
\includegraphics[scale=.19]{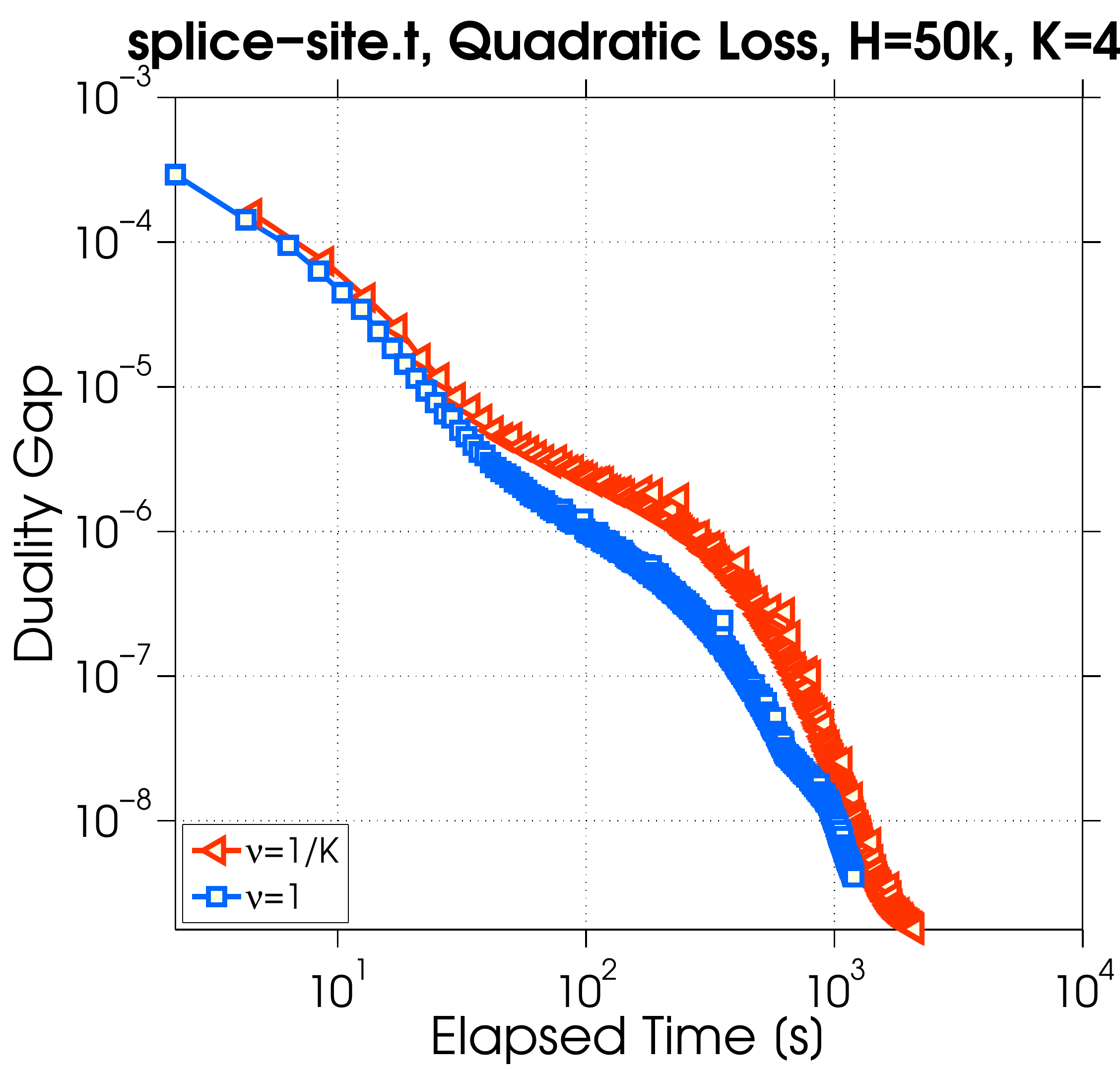}
\caption{Performance of Algorithm \ref{alg:cocoaPractical} on \emph{splice-site.t} dataset, with three different loss functions. } 
\label{fig:hugedata}
\end{figure}

\subsection{Comparison with other distributed methods}

 {Finally, we compare the \cocoap framework with several competing distributed optimization algorithms. The DiSCO algorithm~\cite{zhang2015communication} is a Newton-type method, where in each iteration the updates on iterates are computed inexactly using a Preconditioned Conjugate Gradients (PCG) method. As suggested in \cite{zhang2015communication}, in our implementation of DISCO we apply the Stochastic Average Gradient (SAG) method~\cite{SAGjournal2013}  as the solver to get the initial solutions for each local machine and solve the linear system during PCG.  DiSCO-F~\cite{ma2016distributedDISCO}, improves on the computational efficiency of original DiSCO, {by partitioning the data across} features rather than examples. The DANE algorithm~\cite{DANE} is another distributed Newton-type method, where in each iteration there are two rounds of {communication}. Also, a subproblem is to be solved in each iteration to obtain updates. For {each of} these algorithms, {we tune the hyperparameters manually to optimize} performance.

The experiments are conducted on a compute cluster with $K=4$ machines. We run all algorithms using two datasets. Since not all  methods  are primal-based in nature,  it is difficult to continue using duality gap as a measure of optimality. Instead, the norm of the gradient of the primal objective function \eqref{eq:primal} is used to compare the relative quality of the solutions obtained. 

As shown in Figure~\ref{fig:diffmeth}, in {terms} of {the} number of communications, \cocoap usually converges more rapidly than competing methods during the early  iterations, but tends to get slower later on in the iterative process. This  illustrates that the Newton-type methods can accelerate in the vicinity of the optimal solution, as expected. However, \cocoap can still beat other methods in running time. The main reason for this is the fact that the subproblems in our framework can be solved more efficiently, compared with DiSCO and DANE.}

\begin{figure}[H]
\centering
\includegraphics[scale=.13]{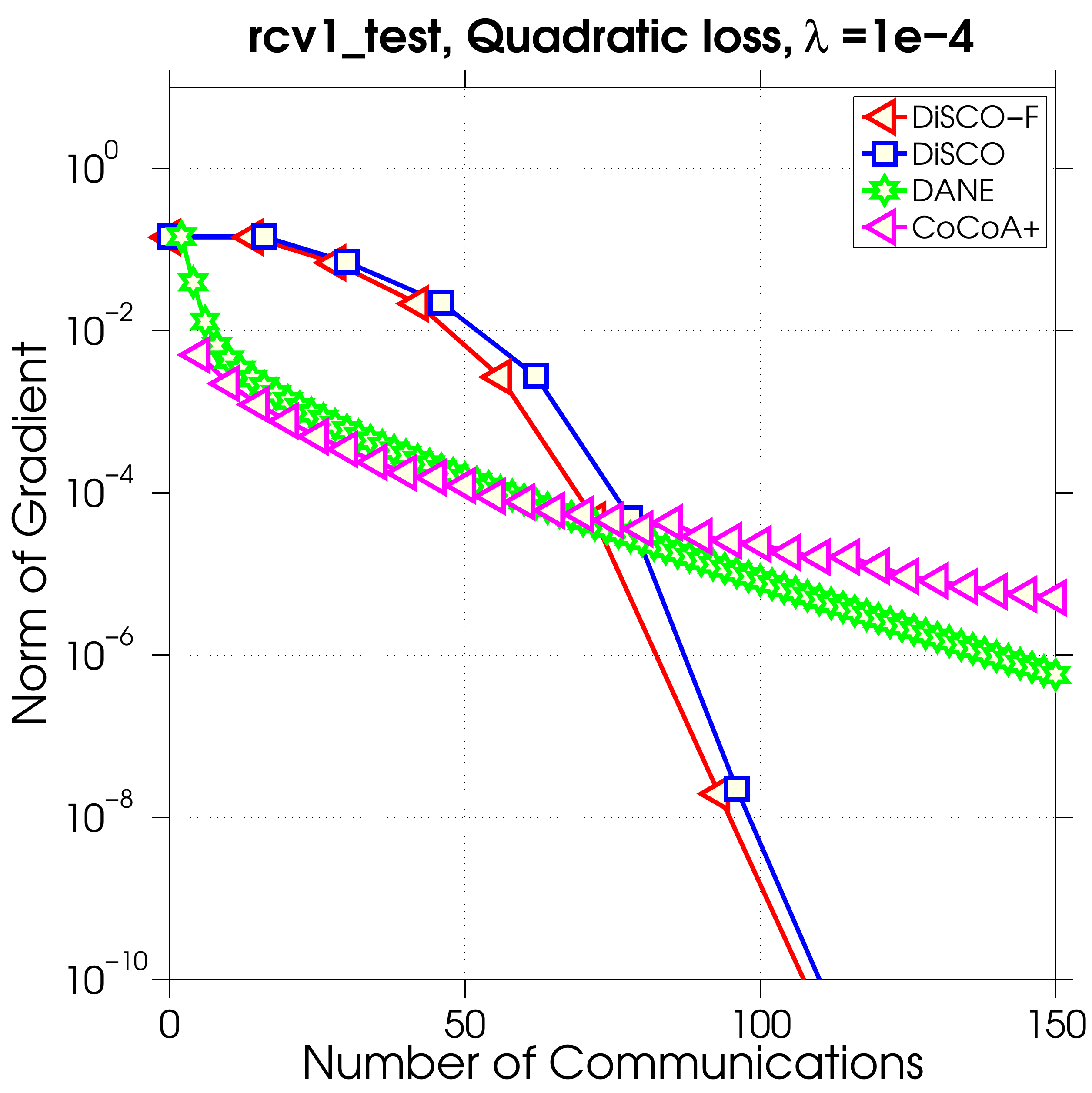}
\includegraphics[scale=.13]{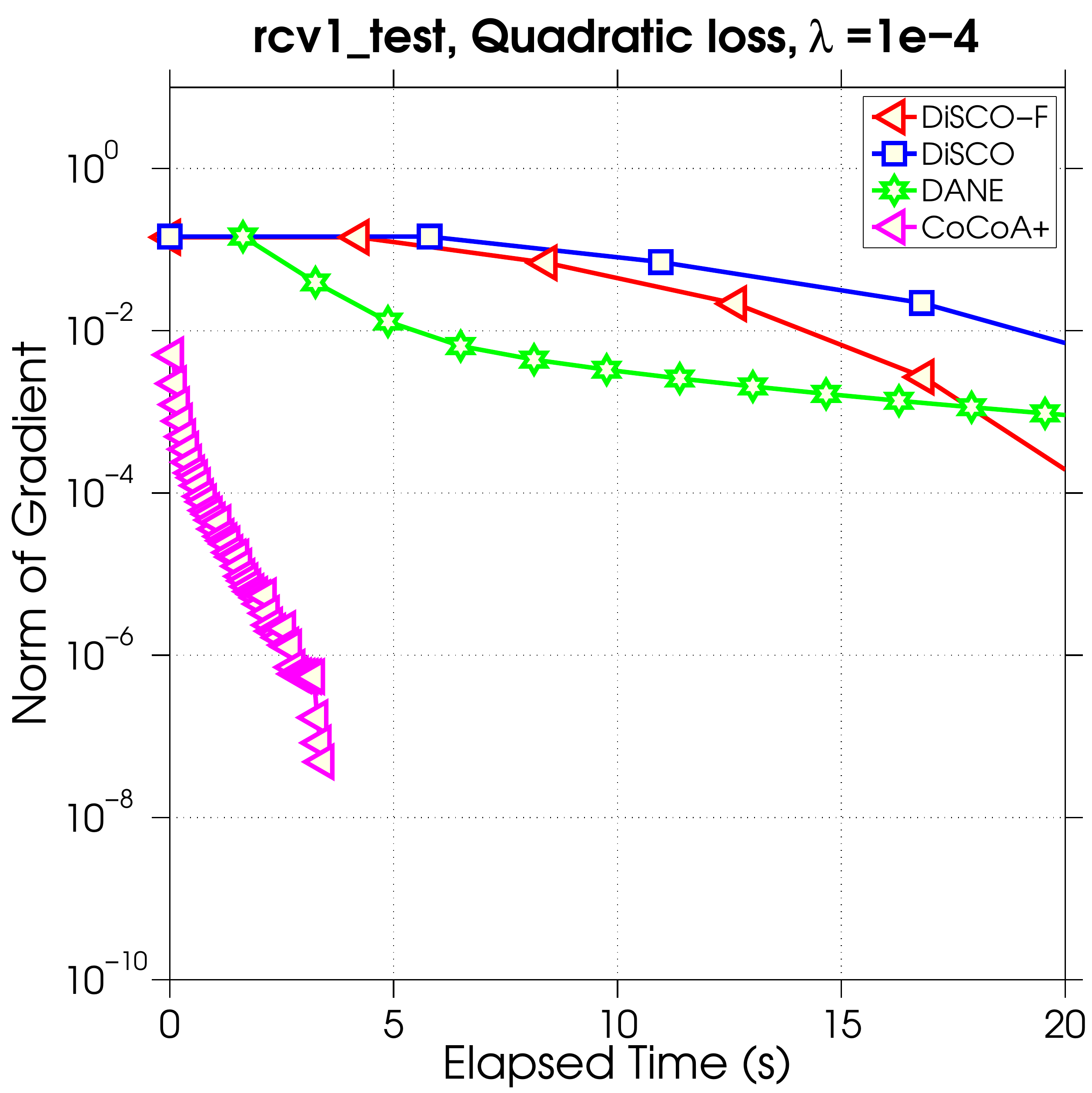}
\includegraphics[scale=.13]{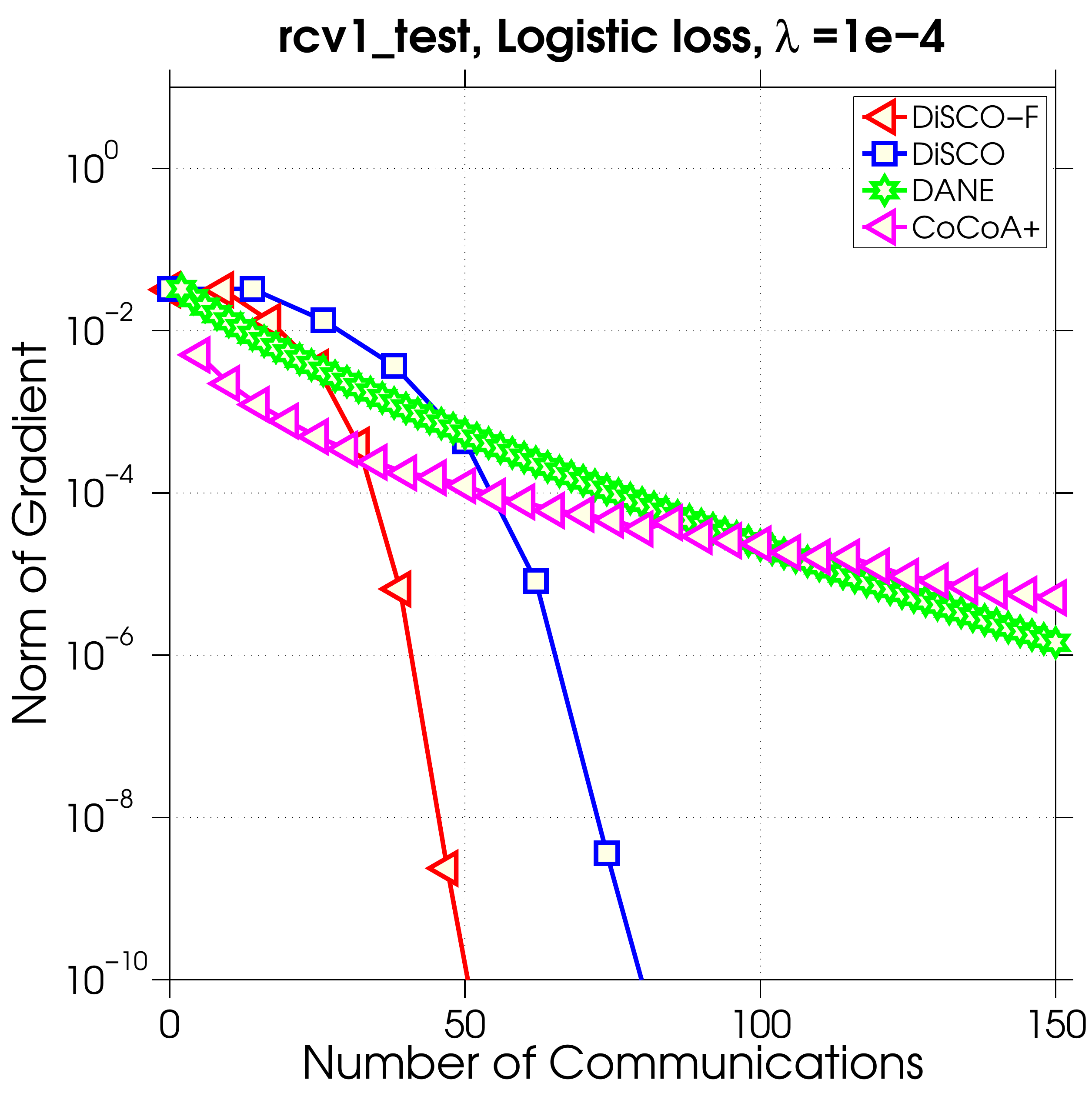}
\includegraphics[scale=.13]{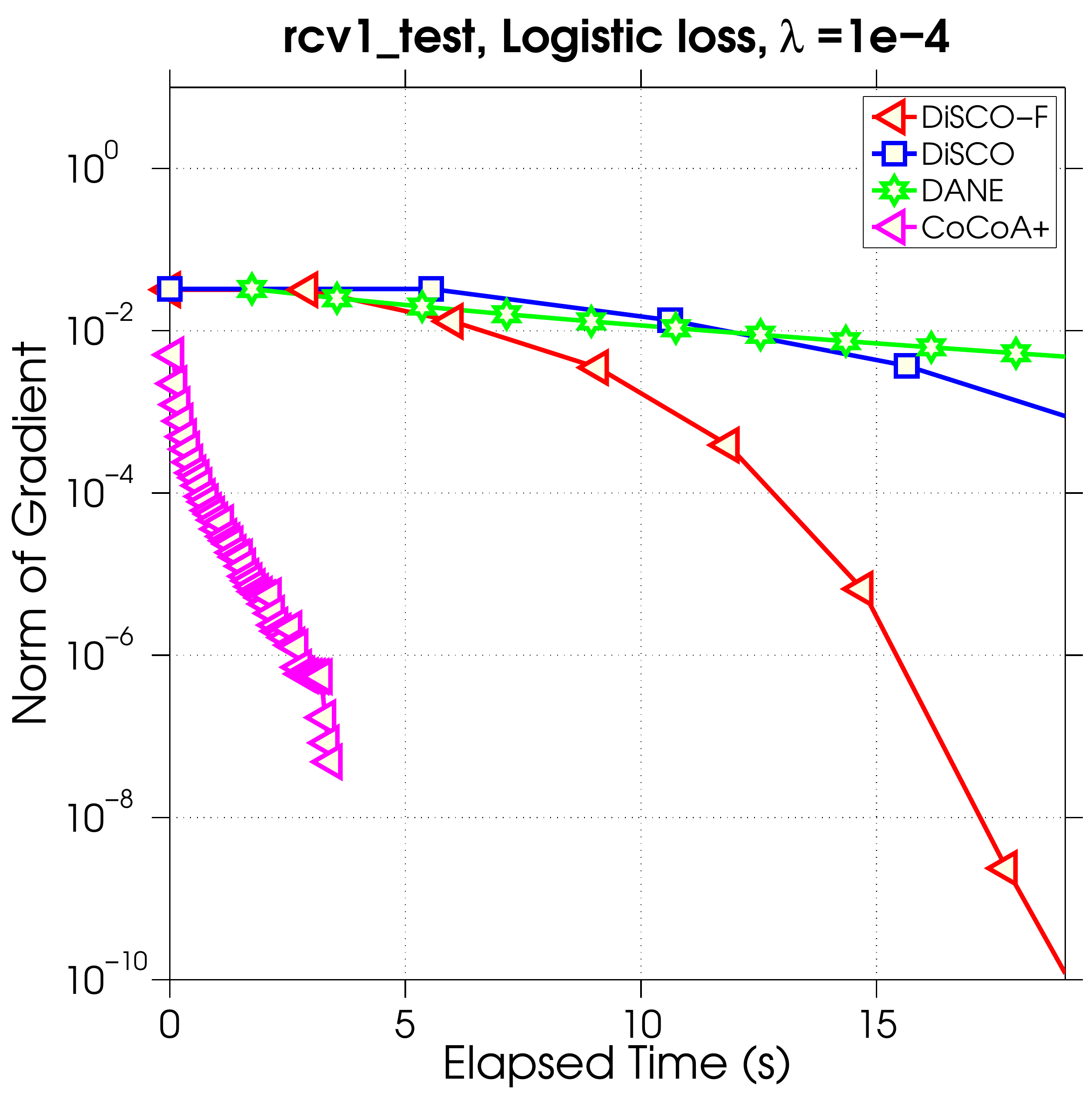}

\includegraphics[scale=.13]{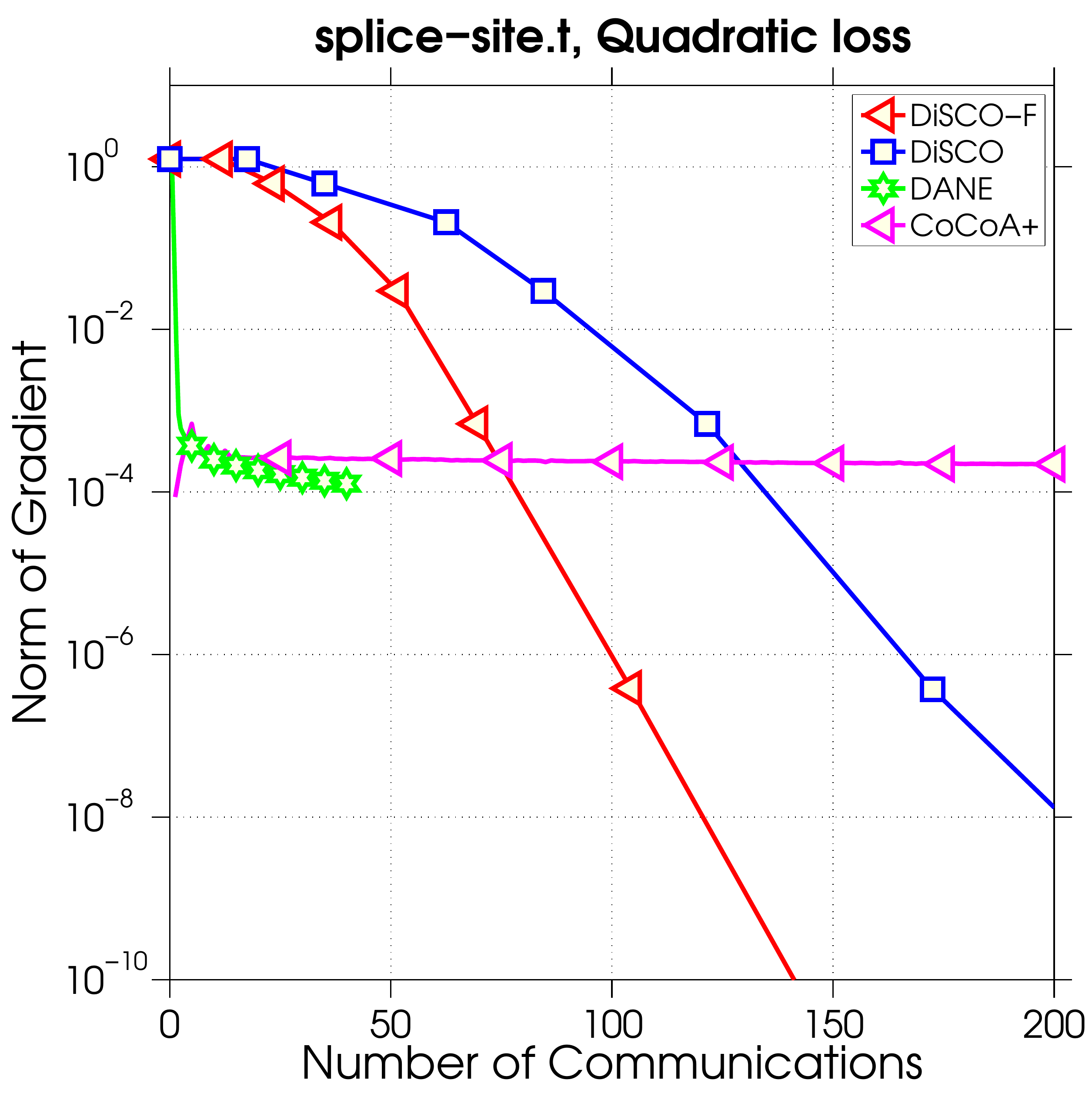}
\includegraphics[scale=.13]{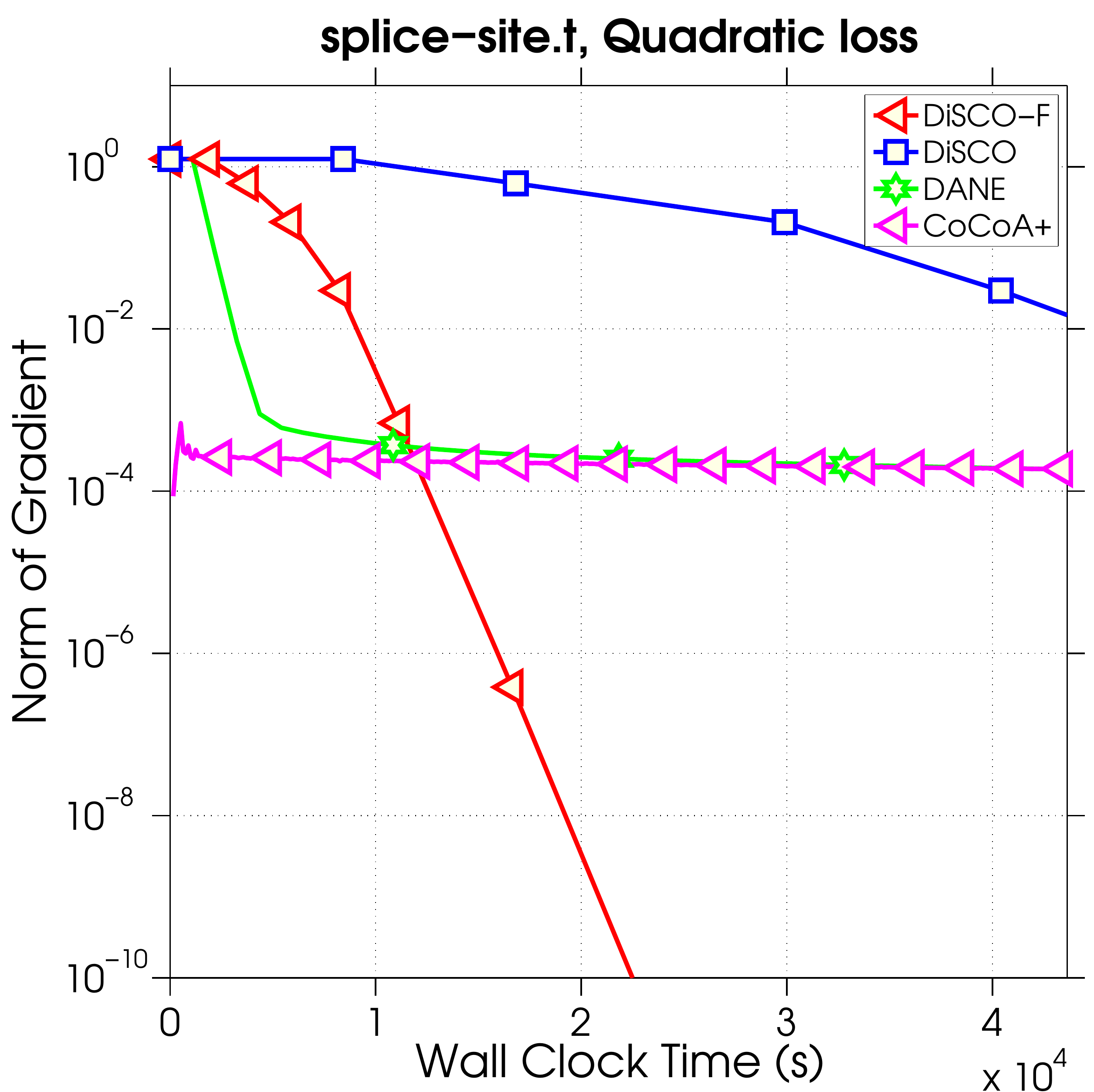}
\includegraphics[scale=.13]{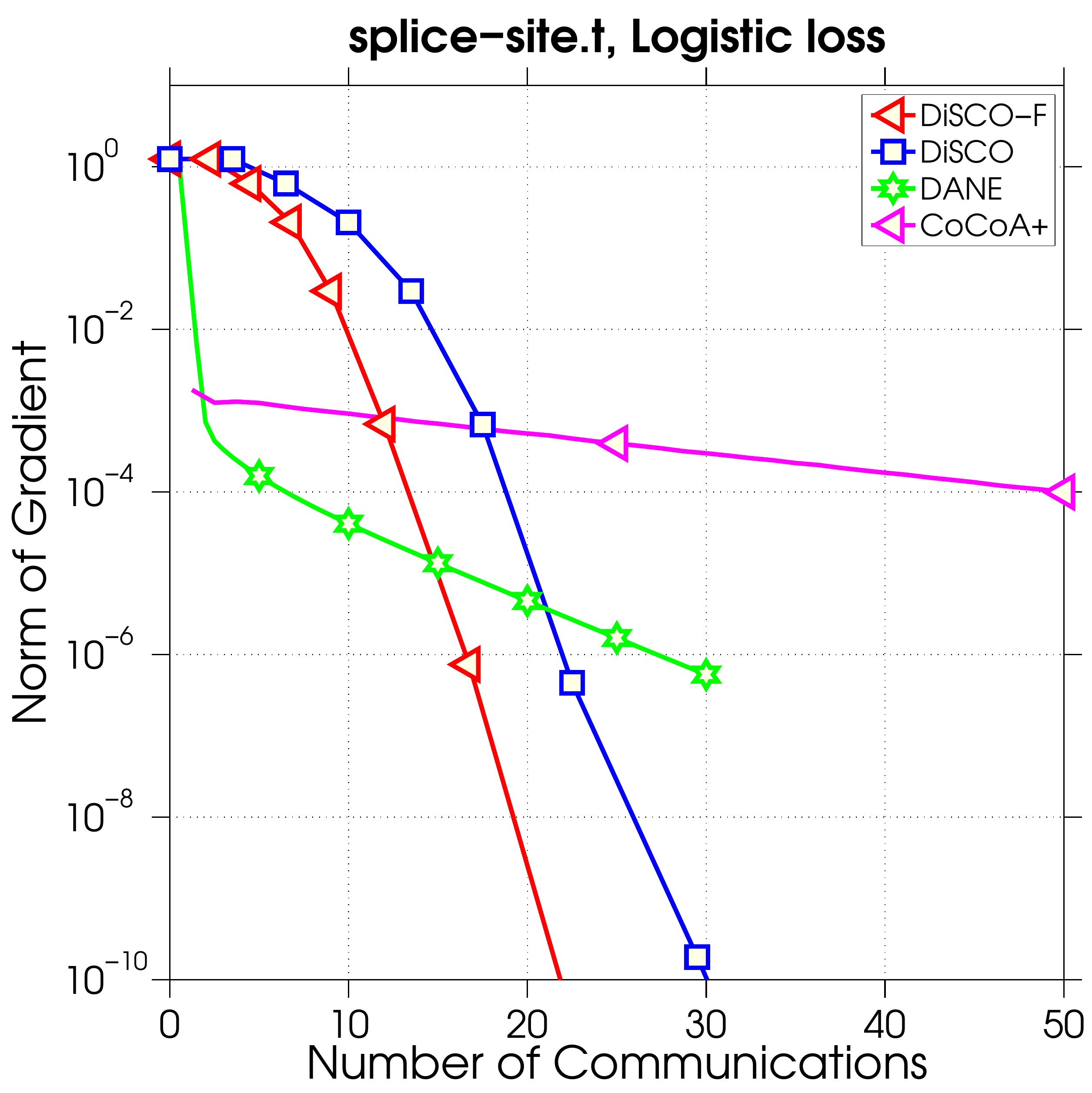}
\includegraphics[scale=.13]{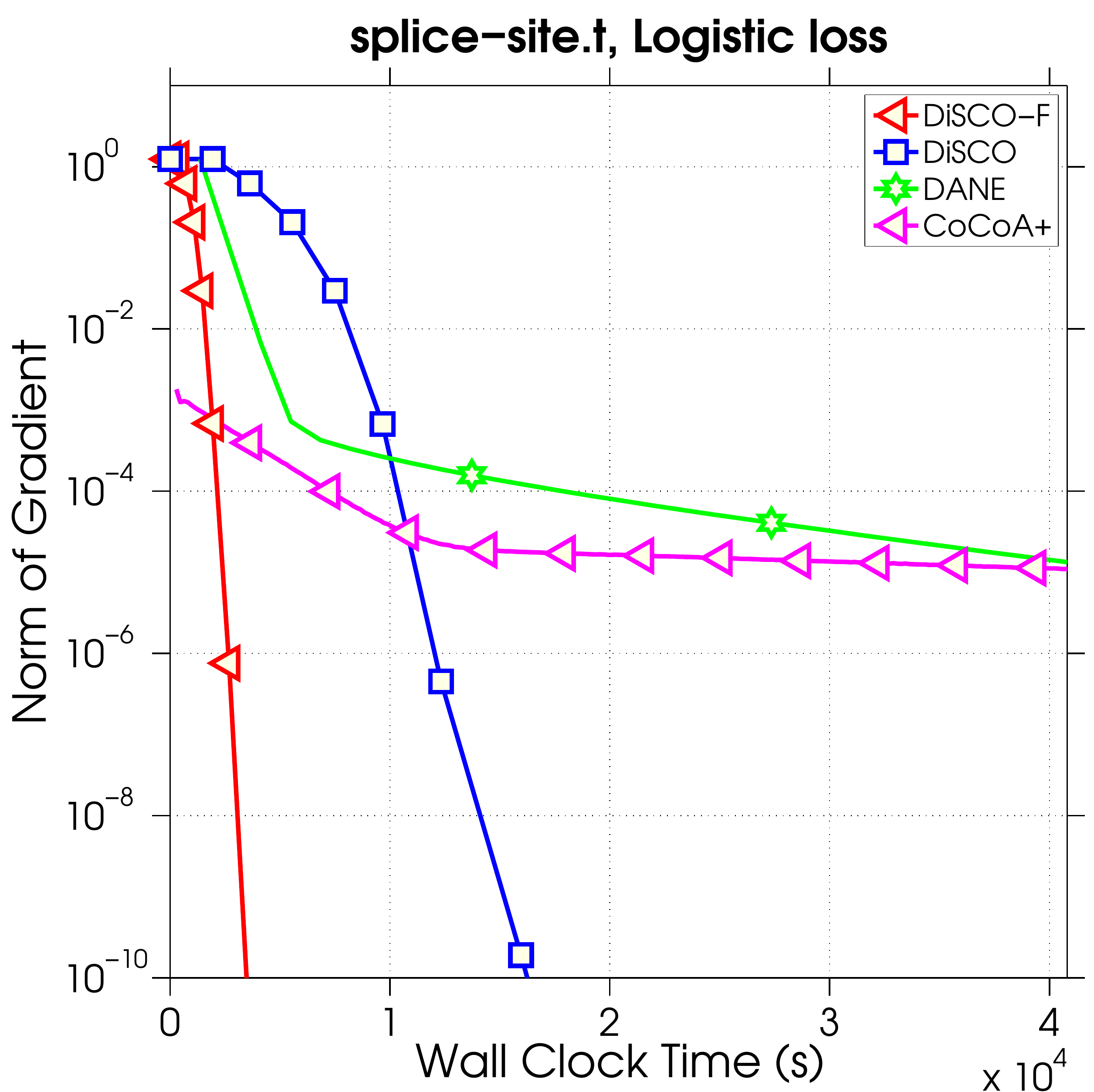}
\caption{Performance of several distributed frameworks on solving \eqref{eq:primal} with different losses on two datasets.} 
\label{fig:diffmeth}
\end{figure}

\section{Conclusion}
We present {\cocoap,} a novel framework {that} enables fast and communication-efficient \textit{additive aggregation} in distributed primal-dual optimization.  We analyze the theoretical complexity  of  \cocoap, giving strong 
primal-dual convergence rates with outer iterations scaling independently of  the number of machines. 
We extended the basic theory to allow for non-smooth loss functions, arbitrary strongly convex regularizers, and primal-dual convergence results. Our 
experimental results show significant speedups in terms of runtime over previous methods, including the 
original \cocoa framework as well as other state-of-the-art methods.

\section{Proofs}

\subsection{Proof of Lemma \ref{lem:quartz}}
 
Since $g$ is $1$-strongly convex, $g^*$ is $1$-smooth, and thus we can use \eqref{def:Lsmoothness:gstar} as follows
\begin{align*}
f(\alphav + \hv) &= \lambda g^*\left( \frac1{\lambda n} X \alphav + \frac1{\lambda n} X \hv \right) \\ 
&\overset{\eqref{def:Lsmoothness:gstar}}{\leq} \lambda \left( g^* \left( \frac1{\lambda n} X \alphav \right) + \< \nabla g^* \left( \frac1{\lambda n} X \alphav \right) , \frac1{\lambda n} X \hv > + \tfrac12 \left\| \frac1{\lambda n} X \hv  \right\|^2 \right) \\ 
&= f(\alphav) + \< \nabla f(\alphav), \hv > + \frac{1}{2 \lambda n^2} \hv^T X^T X \hv.\qedhere 
\end{align*}

\subsection{Proof of Lemma \ref{lem:step:sigma1}}

Indeed, 
\begin{align*}
D(\alphav+\nu \sum_{k=1}^K \hk)
&
=D(\alphav+\nu \hv)
\\
&\overset{\eqref{eq:dual}}{=}
 \frac1n \sum_{i=1}^n -\ell_i^*(- \alpha_i-\nu h_i)
 - \lambda g^* \left( \frac1{\lambda n} X (\alphav+\nu \hv) \right) 
\\
&\overset{\eqref{eq:fRdefinition}}{=}
 \frac1n \sum_{i=1}^n -\ell_i^*(- \alpha_i-\nu h_i)
 - f(\alphav+\nu \hv)  
\\
&\overset{\eqref{eq:quartz}}{\geq} 
 \frac{1-\nu}n \sum_{i=1}^n -\ell_i^*(- \alpha_i)
+ \nu \frac1n \sum_{i=1}^n -\ell_i^*(- \alpha_i-h_i)
\\& \qquad 
 - f(\alphav) - \nu \< \nabla f(\alphav), \hv> -\nu^2 \frac{1}{2 \lambda n^2} \hv^T X^T X \hv
\\
&\overset{\eqref{eq:dual},\eqref{eq:sigmaPrimeSafeDefinition}}{\geq} 
(1-\nu) D(\alphav)  
- \nu \sum_{k=1}^K  R_k \left( \alphak + \hk \right)
\\& \qquad 
 -\nu \frac1K \sum_{k=1}^K f(\alphav) - \nu \sum_{k=1}^K \< \nabla f(\alphav), \hk> -\nu  \sigma' \frac{1}{2 \lambda n^2} \hv^T G \hv
\\
&\overset{\eqref{eq:subproblem:sigma1}}{=}
(1-\nu) D(\alphav)  
+ \nu \frac1K \Gks(\hk; \alphav),
\end{align*}
where the first inequality follows from {Jensen's inequality} and the last equality follows from the block diagonal definition of $\nBG$ given in \eqref{eq:nBGDefinition}, i.e.
\begin{equation}
\label{eq:nBG:XTX:relation}
\hv^T \nBG \hv = \sum_{k = 1}^ K  \hk \Xk^T \Xk \hk.
\end{equation}

\subsection{Proof of Lemma \ref{lem:sigmaPrimeNotBad}}
Considering $\hv\in\R^n$ with zeros in all coordinates except those that belong to the $k$-th block $\mathcal P_k$, we have $\hv^T X^TX \hv = \hv^T \nBG \hv$, and thus $\sigma' \geq \nu$.
Let $\vsub{\hv}{k, l}$ denote $\hk - \vsub{\hv}{l}$. Since $X^TX$
is a positive semi-definite matrix, for $k, l \in \{ 1, \dots, K \}, k \neq l$ we have 
\begin{equation}
\label{eq:hklhelper}
0 \leq \vsub{\hv}{k, l}^T X^TX \vsub{\hv}{k, l} = \hk^T X^TX \hk + \vsub{\hv}{l}^T X^TX \vsub{\hv}{l} - 2 \hk^T X^TX \vsub{\hv}{l}.
\end{equation}
By taking any $\hv \in \R^n$ for which $ \hv^T \nBG \hv \leq 1$, in view of \eqref{eq:sigmaPrimeSafeDefinition}, we get
\begin{align*}
\hv^T X^TX \hv &= \sum_{k=1}^K \sum_{l=1}^K \hv^T_{[k]} X^TX \hv^T_{[l]}
= \sum_{k=1}^K \hk^T X^TX \hk^T + \sum_{k\neq l} \hv^T_{[k]} X^TX \hv^T_{[l]}
\\
&\overset{\eqref{eq:hklhelper}}{\leq} \sum_{k=1}^K \hk^T X^TX \hk^T + \sum_{k \neq l} \frac12 \left[ \hk^T X^TX \hk + \vsub{\hv}{l}^T X^TX \vsub{\hv}{l} \right]
\\
&= K \sum_{k=1}^K \hk^T X^TX \hk = K \hv^T \nBG \hv \leq K.
\end{align*}
Therefore we can conclude that $\nu \hv^TX^TX\hv \leq \nu K$ for all $\hv$ included in the definition~\eqref{eq:sigmaPrimeSafeDefinition} of $\sigma'_{\min}$, proving the claim.

\subsection{Proofs of Theorems~\ref{thm:mainResult}~and~\ref{thm:mainResult:gcc}}
\label{sec:cocoa:analysis}

Before we state the proofs of the main theorems, we will write and prove {a} few crucial lemmas.

\begin{lemma}
\label{lem:basicLemma}
Let $\ell_i^*$ be strongly\footnote{%
Note that the case of weakly convex $\ell_i^*(.)$ is explicitly allowed here as well, as the Lemma holds for the case $\gamma = 0$.
} %
convex with convexity parameter
$\gamma \geq 0$ with respect to the norm $\|\cdot\|$, $\forall i\in[n]$.
Then for all iterations~$t$ of Algorithm~\ref{alg:cocoa} under Assumption~\ref{ass:localImprovement}, and any $s\in [0,1]$, it holds that
\begin{align}
\label{eq:lemma:dualDecrease_VS_dualityGap}
&\E{
\bD(\vc{\alphav}{t+1})
-
\bD(\vc{\alphav}{t})
 }
\geq
\aggpar
(1-\Theta)
 \Big(
 s \gap(\vc{\alphav}{t})
-
\frac{\sigma' s^2}{2\lambda n^2}
\vc{R}{t}
 \Big), \vspace{-2mm}
\end{align}
where\vspace{-2mm}
\begin{align*}
\numberthis \label{eq:defOfR}
\vc{R}{t}&:=
-
\frac{ \lambda\gamma n (1-s)}{\sigma' s }
 \|\vc{\uv}{t}-\vc{\alphav}{t}\|^2 
+ 
\sum_{k=1}^K   
  \left\| X \vsubset{  (\vc{\uv}{t}  - \vc{\alphav}{t} )}{k} \right\|^2,
\end{align*}
for $\vc{\uv}{t} \in\R^n$
with 
\begin{equation}
\label{eq:defintionOfUi}
-\vc{u_i}{t}
 \in \partial \ell_i(\wv(\vc{\alphav}{t})^T \xv_i).
\end{equation}
\end{lemma}
\begin{proof}
For sake of notation, 
we will write 
$\alphav$ instead of $\vc{\alphav}{t}$,
$\wv$ instead of $\wv(\vc{\alphav}{t})$
and
$\uv$ instead of $\vc{\uv}{t}$.

Now, let us estimate the expected change of the dual objective. 
Using the definition of the dual update $\vc{\alphav}{t+1} := \vc{\alphav}{t} + \aggpar \, \sum_k \hk$ resulting in Algorithm~\ref{alg:cocoaPractical}, we have
\begin{align*}
&\E{ \bD(\vc{\alphav}{t})
 - \bD(\vc{\alphav}{t+1}) }
 =
\E{ \bD(\alphav)
 - \bD(\alphav +
  \aggpar \sum_{k=1}^K
  \hk) }
\\ 
&\overset{\eqref{eq:asfdjalkfjlsaflasdfa}}{\leq}
\E{ \bD(\alphav)
-(1-\aggpar)\bD(\alphav)
-\aggpar 
 \sum_{k=1}^K 
 \Ggk (\hk^t; \alphav)
} \\
&=
\aggpar
\E{
 \bD(\alphav)
- 
 \sum_{k=1}^K 
 \Ggk (\hk^t; \alphav)
}
\\
&
=
\aggpar
\E{
 \bD(\alphav)
 -
 \sum_{k=1}^K 
 \Ggk(\hk^\star; \alphav)
 +
 \sum_{k=1}^K 
 \Ggk(\hk^\star; \alphav)
- 
 \sum_{k=1}^K 
 \Ggk (\hk^t; \alphav)
}
\\ 
&\overset{\eqref{eq:localQualityOfImprovement}}{\leq}
\aggpar
\bigg(
 \bD(\alphav)
 -
 \sum_{k=1}^K 
 \Ggk(\hk^\star; \alphav)
 +
 \Theta
 \Big(
 \sum_{k=1}^K  
 \Ggk(\hk^\star; \alphav)
 -
\underbrace{  \sum_{k=1}^K  
 \Ggk({\bf 0}; \alphav)
 }_{\bD(\alphav)}
 \Big)
\bigg)
\\
&=
\aggpar
(1-\Theta)
\Big(
\underbrace{
 \bD(\alphav)
 -
 \sum_{k=1}^K 
 \Ggk(\hk^\star; \alphav)
 }_{C}
\Big).
\numberthis
\label{eq:Afasfwafewaef}
\end{align*} 
Now, let us upper bound 
the $C$ term 
(we will denote by
$\hv^\star 
 = \sum_{k=1}^K \hk^\star$):
\begin{align*}
C&
\overset{\eqref{eq:dual},
\eqref{eq:subproblem:sigma1}}{=}
   \frac1n 
 \sum_{i =1}^n 
 \left(
\ell_i^*(-\alpha_i - h^*_i)
-\ell_i^*(- \alpha_i)
\right)
 +\< \nabla f(\alphav), \hv >
 + \sum_{k=1}^K 
\frac\lambda2
 \sigma'   \Big\|\frac1{\lambda n} X \hk^\star\Big\|^2
\\
&\leq  
   \frac1n 
 \sum_{i =1}^n 
 \left(
\ell_i^*(-\alpha_i - s (u_i - \alpha_i))
-\ell_i^*(- \alpha_i)
\right)
 +
\< \nabla f(\alphav), s (\uv  - \alphav ) > 
 \\ &\qquad + \sum_{k=1}^K 
\frac\lambda2
 \sigma'   \Big\|\frac1{\lambda n}X\vsubset{s (\uv  - \alphav )}{k}\Big\|^2
\\
&\overset{\mbox{Strong conv.}}{\leq} 
   \frac1n 
 \sum_{i =1}^n 
 \left(
s \ell_i^*(-u_i )
+
(1-s)
\ell_i^*(-\alpha_i )
-
\frac{\gamma}{2}
(1-s)s (u_i -\alpha_i)^2
-\ell_i^*(- \alpha_i)
\right)
\\&\quad\quad\quad\quad\quad   +
\< \nabla f(\alphav), s (\uv  - \alphav ) > 
 + \sum_{k=1}^K 
\frac\lambda2
 \sigma'   \Big\|\frac1{\lambda n} X \vsubset{s (\uv  - \alphav )}{k}\Big\|^2 
\\
&=
   \frac1n 
 \sum_{i =1}^n 
 \left(
s \ell_i^*(-u_i )
  -s 
\ell_i^*(-\alpha_i )
-
\frac{\gamma}{2}
(1-s)s (u_i -\alpha_i)^2
\right)
\\&\qquad 
+\< \nabla f(\alphav), s (\uv  - \alphav ) > 
  + \sum_{k=1}^K 
\frac\lambda2
 \sigma'   \Big\|\frac1{\lambda n} X \vsubset{s (\uv  - \alphav )}{k}\Big\|^2.  
\end{align*}
The convex conjugate maximal property implies that
\begin{equation}
\label{eq:adjwofcewa}
\ell_i^*(-u_i)
\overset{\eqref{eq:defintionOfUi}}{=} -u_i \wv(\alphav)^T \xv_i
  -\ell_i(\wv(\alphav)^T \xv_i).
\end{equation}
Moreover, from the definition of the primal and dual optimization problems \eqref{eq:primal},
\eqref{eq:dual}, we can write the duality gap as
\begin{align}
\notag
\gap(\alphav) &:= P(\wv(\alphav))-\bD(\alphav)
\\
\notag
&\overset{
\eqref{eq:primal},
\eqref{eq:dual}
}{=}
 \frac1{n} 
 \sum_{i=1}^n
 \left(
  \ell_i( \xv_i^T \wv(\alphav)) 
 +  \ell_i^*(- \alpha_i)
  \right)
    +\lambda g(\wv(\alphav)) + \lambda g^*(\vv(\alphav))   
\\
\notag
&= \frac1{n} 
 \sum_{i=1}^n
 \left(
  \ell_i( \xv_i^T \wv(\alphav)) 
 +  \ell_i^*(- \alpha_i)
  \right)
    +\lambda g(\nabla g^*(\vv(\alphav))) + \lambda g^*(\vv(\alphav))  
\\
\notag
&= \frac1{n} 
 \sum_{i=1}^n
 \left(
  \ell_i( \xv_i^T \wv(\alphav)) 
 +  \ell_i^*(- \alpha_i)
  \right)
    +\lambda \vv(\alphav)^T \wv(\alphav)
\\
\label{eq:asdfjiwjfeojawfa}
&= \frac1{n} 
 \sum_{i=1}^n
 \left(
  \ell_i( \xv_i^T \wv(\alphav)) 
 +  \ell_i^*(- \alpha_i)
 +  \wv(\alphav)^T\xv_i \alpha_i
  \right).    
\end{align}
Hence,
\begin{align*}
C
&\overset{
\eqref{eq:adjwofcewa}}
{\leq}
  \frac1n 
 \sum_{i =1}^n 
 \left( 
-s u_i \wv(\alphav)^T \xv_i
  -s\ell_i(\wv(\alphav)^T \xv_i)
  -s 
\ell_i^*(-\alpha_i )
\underbrace{-s \wv(\alphav)^T \xv_i \alpha_i
+s \wv(\alphav)^T \xv_i \alpha_i
}_{0}
\right)
\\&\qquad 
+ \frac1n \sum_{i=1}^n \left(-
\frac{\gamma}{2}
(1-s)s (u_i -\alpha_i)^2
\right)
 +\< \nabla f(\alphav), s (\uv  - \alphav ) > 
 + \sum_{k=1}^K 
\frac\lambda2
 \sigma'   \Big\|\frac1{\lambda n} X \vsubset{s (\uv  - \alphav )}{k}\Big\|^2 
\\
&=
  \frac1n 
 \sum_{i =1}^n 
 \left( 
  -s\ell_i(\wv(\alphav)^T \xv_i)
  -s\ell_i^*(-\alpha_i )
  -s \wv(\alphav)^T \xv_i \alpha_i
\right)
\\ &\qquad +
  \frac1n 
 \sum_{i =1}^n 
 \left(  s \wv(\alphav)^T \xv_i
( \alpha_i-u_i )
 -
\frac{\gamma}{2}
(1-s)s (u_i -\alpha_i)^2
\right)
\\&\qquad  +\frac1n  
\wv(\alphav)^T X  s (\uv  - \alphav )
 + \sum_{k=1}^K 
\frac\lambda2
 \sigma'   \Big\|\frac1{\lambda n} X \vsubset{s (\uv  - \alphav )}{k}\Big\|^2  
\\
&\overset{\eqref{eq:asdfjiwjfeojawfa}}{=}
 -s \gap(\alphav)
-
\frac{\gamma}{2}
(1-s)s 
  \frac1n 
 \sum_{i =1}^n 
 \|\uv-\alphav\|^2 
 + 
\frac{\sigma' s^2}{2\lambda n^2}
\sum_{k=1}^K   
  \| X \vsubset{  (\uv  - \alphav )}{k}\|^2.
  \numberthis 
  \label{eq:asdfafdas}
 \end{align*}
Now, the claimed improvement bound
\eqref{eq:lemma:dualDecrease_VS_dualityGap}
follows
by plugging 
\eqref{eq:asdfafdas}
into \eqref{eq:Afasfwafewaef}.
\end{proof}

\begin{lemma}
\label{lemma:BoundOnR}
If $\ell_i$ are $L$-Lipschitz 
continuous for all $i\in [n]$, then\vspace{-3mm}
\begin{equation}
\label{eq:asfjoewjofa}
\forall t: 
\vc{R}{t}
\leq 4L^2 
\underbrace{\sum _{k=1}^K 
\sigma_k  |\mathcal{P}_k|}_{=: \sigma}, \vspace{-2mm} 
\end{equation}
where\vspace{-1mm}
\begin{equation}
\label{eq:definitionOfSigmaK}
\sigma_k \eqdef
 \max_{\vsubset{\alphav}{k} \in \R^n}
 \frac{\|\Xk \vsubset{\alphav}{k}\|^2}{
 \|\vsubset{\alphav}{k}\|^2}.
\end{equation}
\end{lemma}
\begin{proof}
For general convex functions, the strong convexity parameter is 
$\gamma=0$, and hence the definition of $\vc{R}{t}$ becomes
\begin{align*} 
\vc{R}{t}
\overset{\eqref{eq:defOfR}}{=}
  \sum _{k=1}^K   
  \| X \vsubset{  (\vc{\uv} {t} - \vc{\alphav}{t} )}{k}\|^2
\overset{\eqref{eq:definitionOfSigmaK}}{\leq}   
\sum _{k=1}^K 
\sigma_k  
  \|   \vsubset{  (\vc{\uv} {t} - \vc{\alphav}{t} )}{k}\|^2
\overset{\mbox{\cite[Lemma 3]{SDCA}}}{\leq}   
\sum _{k=1}^K 
\sigma_k  |\mathcal{P}_k| 4L^2. 
\end{align*}
\end{proof}

\subsubsection{Proof of Theorem \ref{thm:convergenceNonsmooth}}

At first let us estimate expected change of dual feasibility. By using the main Lemma \ref{lem:basicLemma}, we have
\begin{align*} 
 &\E{ \bD(\alphav^\star)-\bD(\vc{\alphav}{t+1}) }
 =
\E{ \bD(\alphav^\star)-\bD(\vc{\alphav}{t+1})+\bD(\vc{\alphav}{t})-\bD(\vc{\alphav}{t}) }
\\
&
\overset{\eqref{eq:lemma:dualDecrease_VS_dualityGap}
}{=}
\bD(\alphav^\star)-\bD(\vc{\alphav}{t})
-\aggpar
(1-\Theta)  
 s \gap(\vc{\alphav}{t})
+
\aggpar
(1-\Theta)
\frac{\sigma'}{2\lambda }
\left( \frac s{  n} \right)^2
\vc{R}{t}
\\
&
\overset{\eqref{eq:gap}
}{=}
\bD(\alphav^\star)-\bD(\vc{\alphav}{t})
-\aggpar
(1-\Theta)
   s  (P(\wv(\vc{\alphav}{t}))-\bD(\vc{\alphav}{t}))
+
\aggpar
(1-\Theta)  \frac{\sigma'}{2\lambda }
\left( \frac s{  n} \right)^2
\vc{R}{t} 
\\
&\leq
\bD(\alphav^\star)-\bD(\vc{\alphav}{t})
-\aggpar
(1-\Theta)
 s  (\bD(\alphav^\star )-\bD(\vc{\alphav}{t}) )
+
\aggpar
(1-\Theta) 
\frac{\sigma'}{2\lambda }
\left( \frac s{  n} \right)^2
\vc{R}{t} \\
&
\overset{\eqref{eq:asfjoewjofa}}{\leq} 
\left( 
 1-\aggpar
(1-\Theta)
   s
\right) 
   (\bD(\alphav^\star )-\bD(\vc{\alphav}{t}))
+
\aggpar
(1-\Theta) 
\frac{\sigma'}{2\lambda }
\left( \frac s{  n} \right)^2
4L^2  \sigma.
\numberthis 
\label{eq:asoifejwofa}
\end{align*}
Using
\eqref{eq:asoifejwofa}
recursively we have 
 \begin{align*} 
 &\E{ \bD(\alphav^\star)-\bD(\vc{\alphav}{t}) }
 =
\left( 
 1-\aggpar
(1-\Theta)
   s
\right)^t 
   (\bD(\alphav^\star )-\bD(\vc{\alphav}{0}))
\\&\qquad \qquad\qquad \qquad +
\aggpar
(1-\Theta) 
\frac{\sigma'}{2\lambda }
\left(\frac s{  n}\right)^2
4L^2  \sigma 
\sum_{j=0}^{t-1}
\left( 
 1-\aggpar
(1-\Theta)
   s
\right)^j
\\
&=
\left( 
 1-\aggpar
(1-\Theta)
   s
\right)^t 
   (\bD(\alphav^\star )-\bD(\vc{\alphav}{0}))
+
\aggpar
(1-\Theta) 
\frac{\sigma'}{2\lambda }
\left( \frac s{  n} \right)^2
4L^2  \sigma 
\frac{1-\left( 
 1-\aggpar
(1-\Theta)
   s
\right)^t}
     { 
  \aggpar
(1-\Theta)
   s }
\\
&\leq
\left( 
 1-\aggpar
(1-\Theta)
   s
\right)^t 
   (\bD(\alphav^\star )-\bD(\vc{\alphav}{0}))
+
 s
\frac{4L^2  \sigma   \sigma'}{2\lambda n^2}. 
\numberthis
\label{eq:asfwefcaw}  
 \end{align*}
The choice of 
$s:=1$ and $t= t_0:= \max\{0,\lceil  
\frac1{\aggpar (1-\Theta)}
\log(
 2\lambda n^2 (\bD(\alphav^\star )-\bD(\vc{\alphav}{0}))
  / (4 L^2 \sigma \sigma')
  )
 \rceil\}$
will lead to 
\begin{align}\label{eq:induction_step1}
\E{ \bD(\alphav^\star)-\bD(\vc{\alphav}{t}) }
 &\leq  
\left( 
 1-\aggpar
(1-\Theta)  
\right)^{t_0}
  (\bD(\alphav^\star )-\bD(\vc{\alphav}{0}))
+ 
\frac{4L^2  \sigma   \sigma'}{2\lambda n^2}
\notag 
\\&\leq 
\frac{4L^2  \sigma   \sigma'}{2\lambda n^2}
+
\frac{4L^2  \sigma   \sigma'}{2\lambda n^2}
=
\frac{4L^2  \sigma   \sigma'}{\lambda n^2}.
\end{align} 
Now, we are going to show that 
\begin{align}
\label{eq:expectationOfDualFeasibility}
\forall t\geq t_0 :  \E{ \bD(\alphav^\star )-\bD(\vc{\alphav}{t}) }
&\leq 
\frac{4L^2  \sigma   \sigma'}{\lambda n^2( 1+ \frac12  \aggpar (1-\Theta)  (t-t_0))}.
\end{align}
Clearly, \eqref{eq:induction_step1} implies that \eqref{eq:expectationOfDualFeasibility} holds for $t=t_0$.
Now imagine that it holds for any $t\geq t_0$ then we show that it also has to hold for $t+1$. 
Indeed, using 
\begin{equation}
\label{eq:asdfjoawjdfas}
s=
\frac{1}
 {1+ \frac12 \aggpar (1-\Theta) (t-t_0)} \in [0,1]
\end{equation} 
  we obtain
\begin{align*}
&\E{
\bD(\alphav^\star )-\bD(\vc{\alphav}{t+1}) }
\overset{\eqref{eq:asoifejwofa}
}{\leq}
\left( 
 1-\aggpar
(1-\Theta)
   s
\right) 
   (\bD(\alphav^\star )-\bD(\vc{\alphav}{t}))
+
\aggpar
(1-\Theta) 
\frac{\sigma'}{2\lambda }
\left(\frac s{  n}\right)^2
4L^2  \sigma
\\
&\overset{\eqref{eq:expectationOfDualFeasibility}
}{\leq}
\left( 
 1-\aggpar
(1-\Theta)
   s
\right) 
   \frac{4L^2  \sigma   \sigma'}{\lambda n^2( 1+ \frac12  \aggpar (1-\Theta)  (t-t_0))}
+
\aggpar
(1-\Theta) 
\frac{\sigma'}{2\lambda }
\left( \frac s{  n} \right)^2
4L^2  \sigma
\\
&
\overset{\eqref{eq:asdfjoawjdfas}}{=}
\frac{4L^2  \sigma   \sigma'}
     {\lambda n^2}
\left( 
\frac{
1+ \frac12 \aggpar (1-\Theta) (t-t_0)
-\aggpar
(1-\Theta)
+
\aggpar
(1-\Theta) 
\frac{1}{2}
}
 {(1+ \frac12 \aggpar (1-\Theta) (t-t_0))^2}
\right)
\\
&=
\frac{4L^2  \sigma   \sigma'}
     {\lambda n^2}
\underbrace{\left( 
\frac{
1+ \frac12 \aggpar (1-\Theta) (t-t_0)
-\frac12 \aggpar
(1-\Theta)
}
 {(1+ \frac12 \aggpar (1-\Theta) (t-t_0))^2}
\right)}_{E}.
\end{align*}
Now, we will upperbound $E$ as follows
\begin{align*}
E&=
\frac1
{1+ \frac12 \aggpar (1-\Theta) (t+1-t_0)}
\underbrace{
\frac{
(1+ \frac12 \aggpar (1-\Theta) (t+1-t_0))
(1+ \frac12 \aggpar (1-\Theta) (t-1-t_0))
}
 {(1+ \frac12 \aggpar (1-\Theta) (t-t_0))^2}}_{\leq 1}
 \\
&\leq  
\frac1
{1+ \frac12 \aggpar (1-\Theta) (t+1-t_0)},
\end{align*}
where in the last inequality we have used the fact that geometric mean
 is less or equal to arithmetic mean. 
 
If $\overline \alphav$ is defined as in \eqref{eq:averageOfAlphaDefinition}
then we obtain that
\begin{align*}
\E{ \gap(\overline\alphav) } &=  
 \E{ \gap\left(\sum_{t=T_0}^{T-1} \frac1{T-T_0} \vc{\alphav}{t}\right) }
 \leq
  \frac1{T-T_0} \E{ \sum_{t=T_0}^{T-1} \gap\left( \vc{\alphav}{t}\right) }
\\
&
\overset{
\eqref{eq:lemma:dualDecrease_VS_dualityGap}
,\eqref{eq:asfjoewjofa}
}{\leq}
  \frac1{T-T_0} \E{ \sum_{t=T_0}^{T-1} 
\left(
\frac1{\aggpar
(1-\Theta)
 s}
(
\bD(\vc{\alphav}{t+1})
-
\bD(\vc{\alphav}{t})
 )
 +
\frac{4L^2 \sigma \sigma' s}{2\lambda n^2 }
\right)  
  }
\\  
 &=
\frac1{\aggpar
(1-\Theta)
 s}
   \frac1{T-T_0} 
   \E{
\bD(\vc{\alphav}{T})
-
\bD(\vc{\alphav}{T_0})
  }
+\frac{4L^2 \sigma \sigma' s}{2\lambda n^2 }  
\\  
 &\leq
\frac1{\aggpar
(1-\Theta)
 s}
   \frac1{T-T_0} 
   \E{
\bD(\alphav^\star)
-
\bD(\vc{\alphav}{T_0})
  } 
+\frac{4L^2 \sigma \sigma' s}{2\lambda n^2 }.  
\numberthis \label{eq:askjfdsanlfas}
  \end{align*}
Now, if $T\geq \left\lceil \frac1{\aggpar (1-\Theta)} \right\rceil+T_0$ such that $T_0\geq t_0$
we obtain
\begin{align*}
\E{ \gap(\overline\alphav) }
&\overset{\eqref{eq:askjfdsanlfas}
,\eqref{eq:expectationOfDualFeasibility}
}{\leq}
\frac1{\aggpar
(1-\Theta)
 s}
   \frac1{T-T_0} 
\left(
\frac{4L^2  \sigma   \sigma'}{\lambda n^2( 1+ \frac12  \aggpar (1-\Theta)  (T_0-t_0))}
\right)
+\frac{4L^2 \sigma \sigma' s}{2\lambda n^2 }
\\
&=
\frac{
4L^2  \sigma   \sigma'}{\lambda n^2}
\left(
\frac1{\aggpar
(1-\Theta)
 s}
   \frac1{T-T_0} 
\frac{1}{ 1+ \frac12  \aggpar (1-\Theta)  (T_0-t_0)}
+\frac{  s}{2 }
\right). 
\numberthis
\label{eq:fawefwafewa}
\end{align*}
Choosing 
\begin{equation}
\label{eq:afskoijewofaw}
s=\frac{1}{(T-T_0) \aggpar (1-\Theta)} \in [0,1]
\end{equation}
gives us
\begin{align*}
\E{ \gap(\overline\alphav) }
&
\overset{\eqref{eq:fawefwafewa},
\eqref{eq:afskoijewofaw}}{\leq}
\frac{
4L^2  \sigma   \sigma'}{\lambda n^2}
\left(
\frac{1}{ 1+ \frac12  \aggpar (1-\Theta)  (T_0-t_0)}
+\frac{1}{(T-T_0) \aggpar (1-\Theta)} \frac{  1}{2 }
\right). 
\numberthis
\label{eq:afsjweofjwafea}
\end{align*}
To have right hand side of
\eqref{eq:afsjweofjwafea}
smaller then 
$\epsilon_\gap$
it is sufficient to choose
$T_0$ and $T$ such that
\begin{eqnarray}
\label{eq:sfadwafeewafa}
\frac{4L^2  \sigma   \sigma'}{\lambda n^2}
\left(
\frac{1}{ 1+ \frac12  \aggpar (1-\Theta)  (T_0-t_0)}
\right)
&\leq & \frac12 \epsilon_\gap,
\\
\label{eq:sfadwafeewafa2}
\frac{4L^2  \sigma   \sigma'}{\lambda n^2}
\left(
\frac{1}{(T-T_0) \aggpar (1-\Theta)} \frac{  1}{2 }
\right)
&\leq & \frac12 \epsilon_\gap.
\end{eqnarray}
Hence{, if}
\begin{eqnarray*}
t_0+
\frac{2}{ \aggpar (1-\Theta) }
\left(
\frac
{8L^2  \sigma   \sigma'}
{\lambda n^2 \epsilon_\gap}
-1
\right)
&\leq & 
 T_0 
,
\\
T_0
+
\frac
{4L^2  \sigma   \sigma'}
{\lambda n^2 \epsilon_\gap
\aggpar (1-\Theta)}
&\leq &  T,  
\end{eqnarray*}
then 
\eqref{eq:sfadwafeewafa}
and
\eqref{eq:sfadwafeewafa2}
are satisfied.

\subsubsection{Proof of Theorem \ref{thm:convergenceSmoothCase}
}
If the function $\ell_i(.)$ is $(1/\gamma)$-smooth then $\ell_i^*(.)$ is $\gamma$-strongly convex with respect to the
$\|\cdot\|$ norm.
From \eqref{eq:defOfR}
we have
\begin{align*}
\vc{R}{t}&
\overset{\eqref{eq:defOfR}}{=}
-
\frac{ \lambda\gamma n (1-s)}{\sigma' s }
   \|\vc{\uv}{t}-\vc{\alphav}{t}\|^2 
+ 
 \sum_{k=1}^K   
  \| X \vsubset{  (\vc{\uv}{t}  - \vc{\alphav}{t} )}{k}\|^2
\\
&
\overset{\eqref{eq:definitionOfSigmaK}}{\leq}  
-
\frac{ \lambda\gamma n (1-s)}{\sigma' s }
   \|\vc{\uv}{t}-\vc{\alphav}{t}\|^2 
+ 
 \sum_{k=1}^K   
 \sigma_k
  \|  \vsubset{   (\vc{\uv}{t}  - \vc{\alphav}{t}  )}{k}\|^2
\\
&\leq
-
\frac{ \lambda\gamma n (1-s)}{\sigma' s }
   \|\vc{\uv}{t}-\vc{\alphav}{t}\|^2 
+
\sigma_{\max} 
 \sum_{k=1}^K   
  \|  \vsubset{ (  \vc{\uv}{t}  - \vc{\alphav}{t}  )}{k}\|^2
\\
&=
\left(
-
\frac{ \lambda\gamma n (1-s)}{\sigma' s }
+\sigma_{\max}
\right)
   \|\vc{\uv}{t}-\vc{\alphav}{t}\|^2.
\numberthis
   \label{eq:afjfocjwfcea} 
\end{align*}
 If we plug 
 \begin{equation}
 s=
  \frac{ \lambda\gamma n }
      {\lambda\gamma n+
\sigma_{\max} \sigma'}\in [0,1]
\label{eq:fajoejfojew}
\end{equation} 
into
\eqref{eq:afjfocjwfcea}
we obtain that
$\forall t: \vc{R}{t}\leq 0$.
Putting the  same $s$
into
\eqref{eq:lemma:dualDecrease_VS_dualityGap}
will give us
\begin{align*}
\E{
\bD(\vc{\alphav}{t+1})
-
\bD(\vc{\alphav}{t})
 }
&\overset{\eqref{eq:lemma:dualDecrease_VS_dualityGap}
,\eqref{eq:fajoejfojew}}{\geq}
\aggpar
(1-\Theta)
 \frac{ \lambda\gamma n }
      {\lambda\gamma n+
\sigma_{\max} \sigma'} \gap(\vc{\alphav}{t})
\notag
\\&\geq
\aggpar
(1-\Theta)
 \frac{ \lambda\gamma n }
      {\lambda\gamma n+
\sigma_{\max} \sigma'} \bD(\alphav^\star)-\bD(\vc{\alphav}{t}).
\numberthis
\label{eq:fasfawfwaf}
\end{align*}
Using the fact that
$\E{ \bD(\vc{\alphav}{t+1})-\bD(\vc{\alphav}{t}) }
= \E{ \bD(\vc{\alphav}{t+1})-\bD(\alphav^\star) }
+\bD(\alphav^\star)-\bD(\vc{\alphav}{t})
$
we have 
\begin{align*}
\E{ \bD(\vc{\alphav}{t+1})-\bD(\alphav^\star) }
+\bD(\alphav^\star)-\bD(\vc{\alphav}{t})
\overset{
\eqref{eq:fasfawfwaf}}
{
\geq
}
\aggpar
(1-\Theta)
 \frac{ \lambda\gamma n }
      {\lambda\gamma n+
\sigma_{\max} \sigma'} \bD(\alphav^\star)-\bD(\vc{\alphav}{t})
\end{align*}
which is equivalent with
\begin{align*}
\E{ \bD(\alphav^\star)-\bD(\vc{\alphav}{t+1}) }
\leq 
\left(
1-\aggpar
(1-\Theta)
 \frac{ \lambda\gamma n }
      {\lambda\gamma n+
\sigma_{\max} \sigma'}\right)
\bD(\alphav^\star)-\bD(\vc{\alphav}{t}).
\numberthis 
\label{eq:affpja}
\end{align*}
Therefore if we denote by $\vc{\epsilon_\bD}{t} = \bD(\alphav^\star)-\bD(\vc{\alphav}{t})$
we have that
\begin{align*}
 \E{ \vc{\epsilon_\bD}{t} }
 &\overset{\eqref{eq:affpja}}{\leq}   \left(
 1-\aggpar
(1-\Theta)
 \frac{ \lambda\gamma n }
      {\lambda\gamma n+
\sigma_{\max} \sigma'}
   \right)^t \vc{\epsilon_\bD}{0}
\leq 
\left(
 1-\aggpar
(1-\Theta)
 \frac{ \lambda\gamma n }
      {\lambda\gamma n+
\sigma_{\max} \sigma'}
   \right)^t
\\&\leq \exp\left(-t \aggpar
(1-\Theta)
 \frac{ \lambda\gamma n }
      {\lambda\gamma n+
\sigma_{\max} \sigma'}
     \right).
\end{align*}
The right hand side will be smaller than some $\epsilon_\bD$ if 
$$
 t   
    \geq 
\frac{1}
   {\aggpar
(1-\Theta)}
\frac
{\lambda\gamma n+
\sigma_{\max} \sigma'}
{ \lambda\gamma n }
    \log \frac1{\epsilon_\bD}.
$$
Moreover, to bound the duality gap, we have
\begin{align*}
\aggpar
(1-\Theta)
 \frac{ \lambda\gamma n }
      {\lambda\gamma n+
\sigma_{\max} \sigma'} \gap(\vc{\alphav}{t})
&
\overset{
\eqref{eq:fasfawfwaf}
}{\leq}
\E{
\bD(\vc{\alphav}{t+1})
-
\bD(\vc{\alphav}{t})
 }
\leq 
\E{
\bD(\alphav^\star)
-
\bD(\vc{\alphav}{t})
 }. 
\end{align*}
Therefore  $\gap(\vc{\alphav}{t})\leq 
\frac1{
\aggpar
(1-\Theta)}
 \frac      {\lambda\gamma n+
\sigma_{\max} \sigma'} 
{ \lambda\gamma n }    \vc{\epsilon_\bD}{t}$.  
Hence if $\epsilon_\bD \leq 
\aggpar
(1-\Theta)
 \frac{ \lambda\gamma n }
      {\lambda\gamma n+
\sigma_{\max} \sigma'} 
 \epsilon_\gap $
then $\gap(\vc{\alphav}{t})\leq \epsilon_\gap$.
Therefore
after 
$$
 t   
    \geq 
\frac{1}
   {\aggpar
(1-\Theta)}
\frac
{\lambda\gamma n+
\sigma_{\max} \sigma'}
{ \lambda\gamma n }
    \log 
\left(
\frac{1}
   {\aggpar
(1-\Theta)}
\frac
{\lambda\gamma n+
\sigma_{\max} \sigma'}
{ \lambda\gamma n }
    \frac1{\epsilon_\gap}
    \right) 
$$
iterations we have obtained a duality gap less than $\epsilon_\gap$.

\part{Federated Optimization and Learning}

\chapter{Federated Optimization: Distributed Machine Learning for On-device Intelligence}
\label{ch:feopt}

\section{Introduction}
\label{sec:intro}
Mobile phones and tablets are now the primary computing devices for
many people. In many cases, these devices are rarely separated from 
their owners \cite{CNNSmartphoneUsage},
and the combination of rich user interactions and powerful sensors
means they have access to an unprecedented amount of data, much of it
private in nature. Models learned on such data hold the promise of
greatly improving usability by powering more intelligent applications,
but the sensitive nature of the data means there are risks and
responsibilities to storing it in a centralized location.

We advocate an alternative --- {\em federated learning} --- that leaves the training data distributed on the mobile devices, and learns a shared model by aggregating locally computed updates via a central coordinating server. This is a direct application of the principle of focused collection or data minimization proposed by the 2012 White House report on the privacy of consumer data \cite{whitehouse13privacy}. Since these updates are specific to improving the current model, they can be purely ephemeral --- there is no reason to store them on the server once they have been applied. Further, they will never contain more information than the raw training data (by the data processing inequality), and will generally contain much less. A principal advantage of this approach is the decoupling of model training from the need for direct access to the raw training data. Clearly, some trust of the server coordinating the training is still required, and depending on the details of the model and algorithm, the updates may still contain private information. However, for applications where the training objective can be specified on the basis of data available on each client, federated learning can significantly reduce privacy and security risks by limiting the attack surface to only the device, rather than the device and the cloud.

If additional privacy is needed, randomization techniques from differential privacy can be used.  The centralized algorithm could be modified to produce a differentially private model~\cite{chaudhuri11dperm, dwork14book, abadi2016deep}, which allows the model to be released while protecting the privacy of the individuals contributing updates to the training process.  If protection from even a malicious (or compromised) coordinating server is needed, techniques from local differential privacy can be applied to privatize the individual updates \cite{duchi14privacy}.  Details of this are beyond the scope of the current work, but it is a promising direction for future research.

A more complete discussion of applications of federated learning as
well as privacy ramifications can be found
in~\cite{mcmahan2016federated}. Our focus in this work will be on \fedopt, the
optimization problem that must be solved in order to make federated
learning a practical alternative to current approaches.

\subsection{Problem Formulation}
The optimization community has seen an explosion of interest in solving problems with finite-sum structure in recent years. In general, the objective is formulated as
\begin{equation}
\label{eq:problem}
\min_{w \in \R^d} P(w) \qquad \text{where} \qquad P(w) \eqdef \frac{1}{n} \sum_{i=1}^n f_i(w).
\end{equation}
The main source of motivation are problems arising in machine learning. The problem structure~\eqref{eq:problem} covers linear or logistic regressions, support vector machines, but also more complicated models such as conditional random fields or neural networks.

We suppose we have a set of input-output pairs $\{ x_i, y_i \}_{i = 1}^n$, and a loss function, giving rise to the functions $f_i$. Typically, $x_i \in \R^d$ and $y_i \in \R$ or $y_i \in \{ -1, 1 \}$. Simple examples include 
\begin{itemize}
\item linear regression:  $f_i(w) = \frac12 (x_i^Tw - y_i)^2$, $y_i \in \R$
\item logistic regression: $f_i(w) = -\log(1 + \exp(-y_i x_i^T w))$, $y_i \in \{-1, 1\}$ 
\item support vector machines: $f_i(w) = \max\{0, 1 - y_i x_i^T w \}$, $y_i \in \{-1, 1\}$
\end{itemize}

More complicated non-convex problems arise in the context of neural networks, where rather than via the linear-in-the-features mapping $x_i^T w$, the network makes prediction through a non-convex function of the feature vector $x_i$. However, the resulting loss can still be written as $f_i(w)$, and gradients can be computed efficiently using backpropagation.

The amount of data that businesses, governments and academic projects collect is rapidly increasing. Consequently, solving problem~\eqref{eq:problem} arising in practice is often impossible on a single \node, as merely storing the whole dataset on a single \node becomes infeasible. This necessitates the use of a distributed computational framework, in which the training data describing the problem is stored in a distributed fashion across a number of interconnected \nodes and the optimization problem is solved collectively by the cluster of nodes.

Loosely speaking, one can use any network of \nodes to simulate a single powerful \node, on which one can run any algorithm. The
practical issue is that the time it takes to communicate between a
processor and memory on the same \node is normally many
orders of magnitude smaller than the time needed for two \nodes to
communicate; similar conclusions hold for the energy required
\cite{shalf2011exascale}. Further, in order to take advantage of parallel
computing power on each node, it is necessary to subdivide the problem
into subproblems suitable for independent/parallel computation.

State-of-the-art optimization algorithms are typically inherently sequential. Moreover, they usually rely on performing a large number of very fast iterations. The problem stems from the fact that if one needs to perform a round of communication after each iteration, practical performance drops down dramatically, as the round of communication is much more time-consuming than a single iteration of the algorithm.

These considerations have lead to the development of novel algorithms specialized for distributed optimization (we defer thorough review until Section~\ref{sec:relatedWork}). For now, we note that most of the results in literature work in the setting where the data is evenly distributed, and further suppose that $K \ll n / K$ where $K$ is the number of \nodes. This is indeed often close to reality when data is stored in a large data center. Additionally, an important subfield of the field of distributed learning relies on the assumption  that each machine has a representative sample of the data available locally. That is, it is assumed that each machine has an \iid sample from the underlying distribution. However, this assumption is often too strong; in fact,  even in the data center paradigm this is often not the case since the data on a single \node can be close to each other on a temporal scale, or clustered by its geographical origin. Since the patterns in the data can change over time, a feature might be present frequently on one \node, while not appear on another at all.

The \fedopt setting describes a novel optimization scenario where none of the above assumptions  hold. We outline this setting in more detail in the following section.

\subsection{The Setting of \FedOpt}
\label{sec:intro:challenge}

The main purpose of this work is to bring to the attention of the machine learning and optimization communities a new and increasingly practically relevant setting for distributed optimization, where none of the typical assumptions are satisfied, and communication efficiency is of utmost importance. In particular, algorithms for \fedopt must handle training data with the following characteristics:
\begin{itemize}
\item \textbf{Massively Distributed}: Data points are stored across a large number of \nodes $K$. In particular, the number of \nodes can be much bigger than the average number of training examples stored on a given \node ($n/K$).
\item \textbf{Non-\iid}: Data on each \node may be drawn from a different distribution; that is, the data points available locally are far from being a representative sample of the overall distribution.
\item \textbf{Unbalanced}: Different \nodes may vary by orders of magnitude in the number of training examples they hold.
\end{itemize}

In this work, we are particularly concerned with \textbf{sparse} data, where some features occur on  a small subset of nodes or data points only. Although this is not necessary characteristic of the setting of \fedopt, we will show that the sparsity structure can be used to develop an effective algorithm for \fedopt. Note that data arising in the largest machine learning problems being solved nowadays, ad click-through rate predictions, are extremely sparse.

We are particularly interested in the setting where training data lives on users' mobile devices (phones and tablets), and the data may be privacy sensitive. The data $\{x_i, y_i\}$ is generated through device usage, e.g., via interaction with apps. Examples include predicting the next word a user will type (language modeling for smarter keyboard apps), predicting which photos a user is most likely to share, or predicting which notifications are most important. 

To train such models using traditional distributed learning algorithms, one would collect the training examples in a centralized location (data center) where it could be shuffled and distributed evenly over proprietary compute nodes. In this chapter we propose and study an alternative model: the training examples are not sent to a centralized location, potentially saving  significant network bandwidth and providing additional privacy protection. In exchange, users allow some use of their devices' computing power, which shall be used to train the model.

In the communication model of this chapter, in each round we send an update
$\delta \in \R^d$ to a centralized server, where $d$ is the dimension of the model
being computed/improved. The update $\delta$ could be a gradient vector, for
example.  While it is certainly possible that in some applications the
$\delta$ may encode some private information of the user, it is likely
much less sensitive (and orders of magnitude smaller) than the
original data itself. For example, consider the case where the raw
training data is a large collection of video files on a mobile
device.  The size of the update $\delta$ will be \emph{independent} of
the size of this local training data corpus. 
We show that a global model can be trained using a small number of
communication rounds, and so this also reduces the network
bandwidth needed for training by orders of magnitude compared to
copying the data to the datacenter.

Further, informally, we choose $\delta$ to be the minimum piece of
information necessary to improve the global model; its utility for
other uses is significantly reduced compared to the original
data. Thus, it is natural to design a system that does not store these
$\delta$'s longer than necessary to update the model, again increasing
privacy and reducing liability on the part of the centralized model
trainer. This setting, in which a single vector $\delta \in \R^d$ is communicated in each round, covers most existing first-order methods, including dual methods such as CoCoA+ \cite{ma2015distributed}.

Communication constraints arise naturally in the massively distributed setting, as network connectivity may be limited (e.g., we may wish to deffer all communication until the mobile device is charging and connected to a wi-fi network).  Thus, in realistic scenarios we may be limited to only a single round of communication per day. This implies that, within reasonable bounds, we have access to essentially unlimited local computational power. Consequently, the practical objective is solely to minimize the number of  communication rounds.

The main purpose of this work is initiate research into, and design a first practical implementation of \fedopt. Our results suggest that with suitable optimization algorithms, very little is lost by not having an \iid sample of the data available, and that even in the presence of a large number of \nodes, we can still achieve convergence in relatively few rounds of communication.

\section{Related Work}
\label{sec:relatedWork}

In this section we provide a detailed overview of the relevant literature. We particularly focus on algorithms that can be used to solve  problem~\eqref{eq:problem} in various contexts. First, in Sections~\ref{sec:relatedWork:general} and \ref{sec:oh0s9hs} we look at algorithms designed to be run on a single computer. In Section~\ref{sec:relatedWork:distributedSetting} we follow with a discussion of the distributed setting, where no single \node has direct access to all data describing $f$. We describe a paradigm for measuring the efficiency of distributed methods, followed by overview of existing methods and commentary  on whether they were designed with communication efficiency in mind or not.

\subsection{Baseline Algorithms}
\label{sec:relatedWork:general}

In this section we shall describe several fundamental baseline algorithms which can be used to solve problems of the form \eqref{eq:problem}.

\paragraph{Gradient Descent.} A trivial benchmark for solving  problems of structure~\eqref{eq:problem} is {\em Gradient Descent} (GD) in the case when functions $f_i$ are smooth (or Subgradient Descent for non-smooth functions) \cite{nesterov2004convex}. The GD algorithm performs the iteration \[w^{t+1} = w^t - h_t \nabla P(w^t),\] where $h_t>0$ is a  stepsize parameter. As we mentioned earlier, the number of functions, or equivalently, the number of training data pairs, $n$,  is typically very large. This makes GD impractical, as it needs to process the whole dataset in order to evaluate a single gradient and update the model. 

Gradient descent can be substantially accelerated, in theory and practice, via the addition of a momentum term. Acceleration ideas for gradient methods in convex optimization can be traced back to the work of Polyak \cite{Polyak-heavy-ball} and Nesterov \cite{Nesterov-1983, nesterov2004convex}. While accelerated GD methods have a substantially better convergence rate, in each iteration they still need to do at least one pass over all data. As a result, they are not practical for problems where $n$ very large.

\paragraph{Stochastic Gradient Descent.} At present a basic, albeit in practice extremely popular, alternative to GD is { \em Stochastic Gradient Descent} (SGD), dating back to the seminal work of Robbins and Monro \cite{RM1951}. In the context of~\eqref{eq:problem}, SGD samples a random function (i.e., a random data-label pair) $i_t \in \{1, 2, \dots, n\}$ in iteration $t$, and performs the update \[w^{t+1} = w^t - h_t \nabla f_{i_t}(w^t),\] where $h_t>0$ is a stepsize parameter. Intuitively speaking, this method works because if $i_t$ is sampled uniformly at random from indices $1$ to $n$, the update direction is an unbiased estimate of the gradient --- $\E{ \nabla f_{i_t}(w) } = \nabla P(w)$. However, noise introduced by sampling slows down the convergence, and a diminsihing sequence of stepsizes $h_k$ is necessary for convergence. For a theoretical analysis for convex functions we refer the reader to \cite{nemirovski2009robust, moulines2011non, NeedellWard2015} and \cite{pegasos, takac-minibatch} for SVM problems. In a recent review \cite{bottou2016optimization}, the authors outline further research directions. For a more practically-focused discussion, see \cite{bottou2012stochastic}. In the context of neural networks, computation of stochastic gradients is referred to as \emph{backpropagation} \cite{lecun2012efficientbackprop}. Instead of specifying the functions $f_i$ and its gradients explicitly, backpropagation is a general way of computing the gradient. Performance of several competitive algorithms for training deep neural networks has been compared in \cite{ngiam2011optimization}.

One common trick that has been practically observed to provide superior performance, is to replace random sampling in each iteration by going through all the functions in a random order. This ordering is replaced by another random order after each such cycle \cite{bottou2009curiously}. Theoretical understanding of this phenomenon had been a long standing open problem, understood recently in \cite{gurbuzbalaban2015randomreshuffling}.

The core differences between GD and SGD can be summarized as follows. GD has a fast convergence rate, but each iteration in the context of \eqref{eq:problem} is potentially very slow, as it needs to process the entire dataset in each iteration. On the other hand, SGD has slower convergence rate, but each iteration is fast, as the work needed is independent of number of data points $n$. For the problem structure of \eqref{eq:problem}, SGD is usually better, as for practical purposes relatively low accuracy is required, which SGD can in extreme cases achieve after single pass through data, while GD would make just a single update. However, if a high accuracy was needed, GD or its faster variants would prevail.

\subsection{A Novel Breed of Randomized Algorithms}\label{sec:oh0s9hs}

Recent years have seen an explosion of new randomized methods which, in a first approximation, combine the benefits of cheap iterations of SGD with fast convergence of GD. Most of these methods can be said to belong to one of two classes --- dual methods of the randomized coordinate descent variety, and primal methods of the stochastic gradient descent with variance reduction variety.

\paragraph{Randomized Coordinate Descent.}  Although the idea of coordinate descent has been around for several decades in various contexts (and for quadratic functions dates back even much further, to works on the Gauss-Seidel methods), it came to prominence in machine learning and optimization with the work of  Nesterov \cite{nesterovRCDM} which equipped the method with a randomization strategy. Nesterov's work on {\em Randomized Coordinate Descent} (RCD) popularized the method and demonstrated that randomization can be very useful for problems of structure \eqref{eq:problem}.

The RCD algorithm in each iteration chooses a random coordinate $j_t \in \{ 1, \dots, d\}$ and performs the update \[w^{t+1} = w^t - h_{j_t} \nabla_{j_t} P(w^t) e_{j_t},\] where $h_{j_t}>0$ is a stepsize parameter, $\nabla_j P(w)$ denotes the $j^{th}$ partial derivative of function $f$, and $e_j$ is the $j^{th}$ unit standard basis vector in $\R^d$. For the case of generalized linear models, when the data exhibits certain sparsity structure, it is possible to evaluate the partial derivative $\nabla_j P(w)$ efficiently, i.e., without need to process the entire dataset, leading to a practically efficient algorithm, see for instance  \cite[Section 6]{RichtarikTakacIteration}.

Numerous follow-up works extended the concept to proximal setting \cite{RichtarikTakacIteration}, single processor parallelism \cite{bradley2011PCDM, RT:PCDM} and develop efficiently implementable acceleration \cite{leeSidfordCD}. All of these three properties were connected in a single algorithm in \cite{APPROX}, to which we refer the reader for a review of the early developments in the area of RCD, particularly to overview in Table~1 therein.

\paragraph{Stochastic Dual Coordinate Ascent.} When an explicit strongly convex, but not necessarily smooth, regularizer is added to the average loss \eqref{eq:problem}, it is possible to write down its (Fenchel) dual and the dual variables live in $n$-dimensional space. Applying RCD leads to an algorithm for solving \eqref{eq:problem} known under the name {\em Stochastic Dual Coordinate Ascent} \cite{SDCA}. This method has gained broad popularity with practicioners, likely due to the fact that for a number of loss functions, the method comes without the need to tune any hyper-parameters. The work \cite{SDCA} was first to show that by applying RCD  \cite{RichtarikTakacIteration} to the dual problem, one also solves the primal problem \eqref{eq:problem}. For a theoretical and computational comparison of applying RCD to the primal versus the dual problems, see \cite{faceoff}.

A directly primal-dual randomized coordinate descent method called Quartz, was developed in \cite{QUARTZ}. It has been recently shown in SDNA \cite{SDNA} that incorporating curvature information contained in random low dimensional subspaces spanned by a few coordinates can sometimes lead to dramatic speedups. Recent works \cite{dfSDCAupdate, csiba2015primal} interpret the SDCA method in primal-only setting, shedding light onto why this method works as a SGD method with a version of variance reduction property. 

\bigskip
We now move the the second class of novel randomized algorithms which can be generally interpreted as variants of SGD, with an attempt to reduce variance inherent in the process of gradient estimation.

\paragraph{Stochastic Average Gradient.} The first notable algorithm from this class is the {\em Stochastic Average Gradient} (SAG) \cite{SAG, SAGjournal2013}. The SAG algorithm stores an average of $n$ gradients of functions $f_i$ evaluated at different points in the history of the algorithm. In each iteration, the algorithm, updates randomly chosen gradient out of this average, and makes a step in the direction of the average. This way, complexity of each iteration is independent of $n$, and the algorithm enjoys a fast convergence. The drawback of this algorithm is that it needs to store $n$ gradients in memory because of the update operation. In the case of generalized linear models, this memory requirement can be reduced to the need of $n$ scalars, as the gradient is a scalar multiple of the data point. This methods has been recently extended for use in Conditional Random Fields \cite{SAGCRF}. Nevertheless, the memory requirement makes the algorithm infeasible for application even in relatively small neural networks.

A followup algorithm SAGA \cite{saga} and its simplification \cite{defazio2016simple}, modifies the SAG algorithm to achieve unbiased estimate of the gradients. The memory requirement is still present, but the method significantly simplifies theoretical analysis, and yields a slightly stronger convergence guarantee.

\paragraph{Stochastic Variance Reduced Gradient.} Another algorithm from the SGD class of methods is {\em Stochastic Variance Reduced Gradient}\footnote{The same algorithm was simultaneously introduced as Semi-Stochastic Gradient Descent (S2GD) \cite{S2GD}. Since the former work gained more attention, we will for clarity use the name SVRG throughout this chapter.} (SVRG) \cite{SVRG} and \cite{S2GD, proxSVRG, konecny2015mini}. The SVRG algorithm runs in two nested loops. In the outer loop, it computes full gradient of the whole function, $\nabla P(w^t)$, the expensive operation one tries to avoid in general. In the inner loop, the update step is iteratively computed as \[w = w - h [\nabla f_i(w) - \nabla f_i(w^t) + \nabla P(w^t)].\] The core idea is that the stochastic gradients are used to estimate the change of the gradient between point $w^t$ and $w$, as opposed to estimating the gradient directly. We return to more detailed description of this algorithm in Section~\ref{sec:SVRG}.

The SVRG has the advantage that it does not have the additional memory requirements of SAG/SAGA, but it needs to process the whole dataset every now and then. Indeed, comparing to SGD, which typically makes significant progress in the first pass through data, SVRG does not make any update whatsoever, as it needs to compute the full gradient. This and several other practical issues have been recently addressed in \cite{practicalSVRG}, making the algorithm competitive with SGD early on, and superior in later iterations. Although there is nothing that prevents one from applying SVRG and its variants in deep learning, we are not aware of any systematic assessment of its performance in this setting. Vanilla experiments in \cite{SVRG, reddi2016stochastic} suggest that SVRG matches basic SGD, and even outperforms in the sense that variance of the iterates seems to be significantly smaller for SVRG. However, in order to draw any meaningful conclusions, one would need to perform extensive experiments and compare with state-of-the-art methods usually equipped with numerous heuristics.

There already exist attempts at combining SVRG type algorithms with randomized coordinate descent \cite{S2CD, wang2014randomized}. Although these works highlight some interesting theoretical properties, the algorithms do not seem to be practical at the moment; more work is needed in this area. The first attempt to unify algorithms such as SVRG and SAG/SAGA already appeared in the SAGA paper \cite{saga}, where the authors interpret SAGA as a midpoint between SAG and SVRG. Recent work \cite{reddi2015variance} presents a general algorithm, which recovers SVRG, SAGA, SAG and GD as special cases, and obtains an asynchronous variant of these algorithms as a byproduct of the formulation. SVRG can be equipped with momentum (and negative momentum), leading to a new accelerated SVRG method known as Katyusha \cite{Katyusha}. SVRG can be further accelerated via a raw clustering mechanism \cite{SVRG-rawclusters}.

\paragraph{Stochastic Quasi-Newton Methods.} A third class of new algorithms are the {\em Stochastic quasi-Newton} methods \cite{byrd2014stochastic, bordes2009sgd}. These algorithms in general try to mimic the limited memory BFGS method (L-BFGS) \cite{LBFGS}, but model the local curvature information using inexact gradients --- coming from the SGD procedure. A recent attempt at combining these methods with SVRG can be found in \cite{moritz2015linearly}. In \cite{SBFGS}, the authors utilize recent progress in the area of stochastic matrix inversion \cite{inverse} revealing new connections with quasi-Newton methods, and devise a new  stochastic limited memory BFGS method working in tandem  with SVRG.  The fact that the theoretical understanding of this branch of research is the least understood and having several details making the implementation more difficult compared to the methods above may limit its wider use. However, this approach could be most promising for deep learning once understood better.

One important aspect of machine learning is that the Empirical Risk Minimization problem \eqref{eq:problem} we are solving is just a proxy for the Expected Risk we are ultimately interested in. When one can find exact minimum of the empirical risk, everything reduces to balancing approximation--estimation tradeoff that is the object of abundant literature --- see for instance \cite{vapnik1999overview}. An assessment of asymptotic performance of some optimization algorithms as \emph{learning} algorithms in large-scale learning problems\footnote{See \cite[Section 2.3]{BottouBousquet} for their definition of large scale learning problem.} has been introduced in \cite{BottouBousquet}. Recent extension in \cite{practicalSVRG} has shown that the variance reduced algorithms (SAG, SVRG, \dots) can in certain setting be better \emph{learning} algorithms than SGD, not just better optimization algorithms.

\paragraph{Further Remarks.} A general method, referred to as Universal Catalyst \cite{lin2015universal, frostig2015regularizing}, effectively enables conversion of a number of the algorithms mentioned in the previous sections to their `accelerated' variants. The resulting convergence guarantees nearly match lower bounds in a number of cases. However, the need to tune additional parameter makes the method rather impractical.

Recently, lower and upper bounds for complexity of stochastic methods on problems of the form \eqref{eq:problem} were recently obtained in \cite{Srebro-lower_and_upper2016}.

\subsection{Distributed Setting}
\label{sec:relatedWork:distributedSetting}

In this section we review the literature concerning algorithms for solving \eqref{eq:problem} in the distributed setting. When we speak about distributed setting, we refer to the case when the data describing the functions $f_i$ are not stored on any single storage device. This can include setting where one's data just don't fit into a single RAM/computer/node, but two is enough. This also covers the case where data are distributed across several datacenters around the world, and across many \nodes in those datacenters. The point is that in the system, there is no single processing unit that would have direct access to all the data. Thus, the distributed setting does not include single processor parallelism\footnote{It should be noted that some of the works presented in this section were originally presented as parallel algorithms. We include them anyway as many of the general ideas in carry over to the distributed setting.}. Compared with local computation on any single \node, the cost of communication between \nodes is much higher both in terms of speed and energy consumption \cite{bekkerman2011scaling, shalf2011exascale}, introducing new computational challenges, not only for optimization procedures.

We first review a theoretical decision rule for determining the practically best algorithm for a given problem in Section~\ref{sec:relatedWork:paradigm}, followed by overview of distributed algorithms in Section~\ref{sec:relatedWork:distributedAlgorithms}, and communication efficient algorithms in Section~\ref{sec:relatedWork:commEffAlgs}. The following paradigm highlights why the class of communication efficient algorithms are not only preferable choice in the trivial sense. The communication efficient algorithms provide us with much more flexible tools for designing overall optimization procedure, which can make the algorithms inherently adaptive to differences in computing resources and architectures.

\subsubsection{A Paradigm for Measuring Distributed Optimization Efficiency}
\label{sec:relatedWork:paradigm}

This section reviews a paradigm for comparing efficiency of distributed algorithms. Let us suppose we have many algorithms $\A$ readily available to solve the problem~\eqref{eq:problem}. The question is: ``How do we decide which algorithm is the best for our purpose?'' Initial version of this reasoning already appeared in \cite{ma2015distributed}, and applies also to \cite{reddi2016aide}.

First, consider the basic setting on a single machine. Let us define $\mathcal{I}_\A(\epsilon)$ as the number of iterations algorithm $\A$ needs to converge to some fixed $\epsilon$ accuracy. Let $\mathcal{T}_\A$ be the time needed for a single iteration. Then, in practice, the best algorithm is one that minimizes the following quantity.\footnote{Considering only algorithms that can be run on a given machine.}
\begin{equation}
\label{eq:paradigmBasic}
\text{TIME} = \mathcal{I}_\A(\epsilon) \times \mathcal{T}_\A.
\end{equation}

The number of iterations $\mathcal{I}_\A(\epsilon)$ is usually given by theoretical guarantees or observed from experience. The $\mathcal{T}_\A$ can be empirically observed, or one can have idea of how the time needed per iteration varies between different algorithms in question. The main point of this simplified setting is to highlight key issue with extending algorithms to the distributed setting.

The natural extension to distributed setting is the following. Let $c$ be time needed for communication during a single iteration of the algorithm $\A$. For sake of clarity, we suppose we consider only algorithms that need to communicate a single vector in $\R^d$ per round of communication. Note that essentially all first-order algorithms fall into this category, so this is not a restrictive assumption, which effectively sets $c$ to be a constant, given any particular distributed architecture one has at disposal.

\begin{equation}
\label{eq:paradigm}
\text{TIME} = \mathcal{I}_\A(\epsilon) \times (c + \mathcal{T}_\A)
\end{equation}

The communication cost $c$ does not only consist of actual exchange of the data, but also many other things like setting up and closing a connection between \nodes. Consequently, even if we need to communicate very small amount of information, $c$ always remains above a nontrivial threshold.

Most, if not all, of the current state-of-the-art algorithms that are the best in setting of~\eqref{eq:paradigmBasic}, are stochastic and rely on doing very large number (big $\mathcal{I}_\A(\epsilon)$) of very fast (small $\mathcal{T}_\A$) iterations. As a result, even relatively small $c$ can cause the practical performance of those algorithms drop down dramatically, because $c \gg \mathcal{T}_\A$.

This has been indeed observed in practice, and motivated development of new methods, designed with this fact in mind from scratch, which we review in Section~\ref{sec:relatedWork:distributedAlgorithms}. Although this is a good development for academia --- motivation to explore new setting, it is not necessarily a good news for the industry.

Many companies have spent significant resources to build excellent algorithms to tackle their problems of form~\eqref{eq:problem}, fine tuned to the specific patterns arising in their data and side applications required. When the data companies collect grows too large to be processed on a single machine, it is understandable that they would be reluctant to throw away their fine tuned algorithms. This issue was first time explicitly addressed in CoCoA~\cite{ma2015distributed}, which is rather framework than a algorithm, which works as follows (more detailed description follows in Section~\ref{sec:relatedWork:commEffAlgs}).

The CoCoA framework formulates a general way to form a specific subproblem on each \node, based on data available locally and a single shared vector that needs to be distributed to all \nodes. Within a iteration of the framework, each \node uses \emph{any} optimization algorithm $\A$, to reach a relative $\Theta$ accuracy on the local subproblem. Updates from all \nodes are then aggregated to form an update to the global model.

The efficiency paradigm changes as follows:

\begin{equation}
\label{eq:paradigmNew}
\text{TIME} = \mathcal{I}(\epsilon, \Theta) \times (c + \mathcal{T}_\A(\Theta))
\end{equation}

The number of iterations $\mathcal{I}(\epsilon, \Theta)$ is independent of choice of the algorithm $\A$ used as a local solver, because there is theory predicting how many iterations of the CoCoA framework are needed to achieve $\epsilon$ accuracy, if we solve the local subproblems to relative $\Theta$ accuracy. Here, $\Theta = 0$ would mean we require the subproblem to be solved to optimality, and $\Theta = 1$ that we don't need any progress whatsoever. The general upper bound on number of iterations of the CoCoA framework is $\mathcal{I}(\epsilon, \Theta) = \frac{\mathcal{O}(\log(1/\epsilon))}{1 - \Theta}$ \cite{Jaggi:cocoa, Ma:2015ti, ma2015distributed} for strongly convex objectives. From the inverse dependence on $1 - \Theta$, we can see that there is a fundamental limit to the number of communication rounds needed. Hence, it will probably not be efficient to spend excessive resources to attain very high local accuracy (small $\Theta$). Time per iteration $\mathcal{T}_\A(\Theta)$ denotes the time algorithm $\A$ needs to reach the relative $\Theta$ accuracy on the local subproblem.

This efficiency paradigm is more powerful for a number of reasons.
\begin{enumerate}
\item It allows practicioners to continue using their fine-tuned solvers, that can run only on single machine, instead of having to implement completely new algorithms from scratch.

\item The actual performance in terms of number of rounds of communication is independent from the choice of optimization algorithm, making it much easier to optimize the overall performance.

\item Since the constant $c$ is architecture dependent, running optimal algorithm on one \node network does not have to be optimal on another. In the setting~\eqref{eq:paradigm}, this could mean moving from one cluster to another, a completely different algorithm is optimal, which is a major change. In the setting~\eqref{eq:paradigmNew}, this can be improved by simply changing $\Theta$, which is typically implicitly determined by number of iterations algorithm $\A$ runs for.
\end{enumerate}

In this work we propose a different way to formulate the local subproblems, which does not rely on duality as in the case of CoCoA. We also highlight that some algorithms seem to be particularly suitable to solve those local subproblems, effectively leading to novel algorithms for distributed optimization.

\subsubsection{Distributed Algorithms}
\label{sec:relatedWork:distributedAlgorithms}

As discussed below in Section \ref{sec:relatedWork:paradigm}, this setting creates unique challenges. Distributed optimization algorithms typically require a small number (1--4) of communication rounds per iteration. By communication round we typically understand a single MapReduce operation \cite{dean2008mapreduce}, implemented efficiently for iterative procedures \cite{MPI}, such as optimization algorithms. Spark \cite{zaharia2010spark} has been established as a popular open source framework for implementing distributed iterative algorithms, and includes several of the algorithms mentioned in this section.

Optimization in distributed setting has been studied for decades, tracing back to at least works of Bertsekas and Tsitsiklis \cite{bertsekas1989parallel, bertsekas1983distributed, tsitsiklis1984problems}. Recent decade has seen an explosion of interest in this area, greatly motivated by rapid increase of data availability in machine learning applications.

Much of the recent effort was focused on creating new optimization algorithms, by building variants of popular algorithms suitable for running on a single processor (See Section~\ref{sec:relatedWork:general}). A relatively common feature of many of these efforts is a) The computation overhead in the case of synchronous algorithms, and b) The difficulty of analysing asynchronous algorithms without restrictive assumptions. By computation overhead we mean that if optimization program runs in a compute-communicate-update cycle, the update part cannot start until all \nodes finish their computation. This causes some of the \nodes be idle, while remaining \nodes finish their part of computation, clearly an inefficient use of computational resources. This pattern often diminishes or completely reverts potential speed-ups from distributed computation. In the asynchronous setting in general, an update can be applied to a parameter vector, followed by computation done based on a now-outdated version of that parameter vector. Formally grasping this pattern, while keeping the setting realistic is often quite challenging. Consequently, this is very open area, and optimal choice of algorithm in any particular case is often heavily dependent on the problem size, details in its structure, computing architecture available, and above all, expertise of the practitioner.

This general issue is best exhibited with numerous attempts at parallelizing the Stochastic Gradient Descent and its variants. As an example, \cite{dekel2012optimal, duchi2012dual} provide theoretically linear speedup with number of \nodes, but are difficult to implement efficiently, as the \nodes need to synchronize frequently in order to compute reasonable gradient averages. As an alternative, no synchronization between workers is assumed in \cite{Niu:2011wo, agarwal2011distributed, Duchi:2013te}. Consequently, each worker reads $w^t$ from memory, parameter vector $w$ at time point $t$, computes a stochastic gradient $\nabla f_i(w^t)$ and applies it to already changed state of the parameter vector $w^{t+\tau}$. The above mentioned methods assume that the delay $\tau$ is bounded by a constant, which is not necessarily realistic assumption\footnote{A bound on the delay $\tau$ can be deterministic or probabilistic. However, in practice, the delays are mostly about the number of \nodes in the network, and there rare very long delays, when a variety of operating system-related events can temporarily postpone computation of a single \node. To the best of our knowledge, no formal assumptions reflect this setting well. In fact, two recent works \cite{mania2015perturbed, leblond2016asaga} highlight subtle but important issue with labeling of iterates in the presence of asynchrony, rendering most of the existing analyses of asynchronous optimization algorithms incorrect.}. Some of the works also introduce assumptions on the sparsity structures or conditioning of the Hessian of $f$. Asymptotically optimal convergent rates were proven in \cite{duchi2015asynchronous} with considerably milder assumptions. Improved analysis of asynchronous SGD was also presented in \cite{de2015taming}, simultaneously with a version that uses lower-precision arithmetic was introduced without sacrificing performance, which is a trend that might find use in other parts of machine learning in the following years.

The negative effect of asynchronous distributed implementations of SGD seem to be negligible, when applied to the task of training very large deep networks --- which is the ultimate industrial application of today. The practical usefulness has been demonstrated for instance by Google's Downpour SGD \cite{largeNN} and Microsoft's Project Adam \cite{chilimbi2014project}.

The first distributed versions of Coordinate Descent algorithms were the Hydra and its accelerated variant, Hydra$^2$, \cite{richtarik2013distributed, fercoq2014fast}, which has been demonstrated to be very efficient on large sparse problems implemented on a computing cluster. An extended version with description of implementation details is presented in \cite{marecek2014distributed}. Effect of asynchrony has been explored and partially theoretically understood in the works of \cite{liu2013asynchronous, liu2015asynchronous}. Another asynchronous, rather framework than an algorithm, for coordinate updates, applicable to wider class of objectives is presented in \cite{peng2015arock}.

The data are assumed to be partitioned to \nodes by features/coordinates in the above algorithms. This setting can be restrictive if one is not able to distribute the data beforehand, but instead the data are distributed ``as is'' --- in which case the data are most commonly distributed by data points. This does not need to be an issue, if a dual version of coordinate descent is used --- in which the distribution is done by data points \cite{takac2015distributed} followed by works on Communication Efficient Dual Coordinate Ascent, described in next section. The use of duality however requires usage of additional explicit strongly convex regularization term, hence can be used to solve smaller class of problems. Despite the apparent practical disadvantages, variants of distributed coordinate descent algorithms are among the most widely used methods in practice.

Moving to variance reduced methods, distributed versions of SAG/SAGA algorithms have not been proposed yet. However, several distributed versions of the SVRG algorithm already exist. A scheme for replicating data to simulate iid sampling in distributed environment was proposed in \cite{lee2015distributed}. Although the performance is well analyzed, the setting requires significantly stronger control of data distribution which is rarely practically feasible. A relatively similar method to Algorithm~\ref{alg:DS2GDnaive} presented here has been proposed in \cite{reddi2016aide}, which was analyzed, and in \cite{mahajan2015efficient}, a largely experimental work that can be also cast as communication efficient --- described in detail in Section~\ref{sec:relatedWork:commEffAlgs}.

Another class of algorithms relevant for this work is Alternating Direction Method of Multipliers (ADMM) \cite{Boyd:2010bw, deng2016global}. These algorithms are in general applicable to much broader class of problems, and hasn't been observed to perform better than other algorithms presented in this section, in the machine learning setting of \eqref{eq:problem}.

\subsubsection{Communication-Efficient Algorithms}
\label{sec:relatedWork:commEffAlgs}

In this Section, we describe algorithms that can be cast as ``communication efficient''. The common theme of the algorithms presented here, is that in order to perform better in the sense of \eqref{eq:paradigm}, one should design algorithms with high $\mathcal{T}_\A$, in order to make the cost of communication $c$ negligible.

Before moving onto specific methods, it is worth the noting some of the core limits concerning the problem we are trying to solve in distributed setting. Fundamental limitations of stochastic versions of the problem \eqref{eq:problem} in terms of runtime, communication costs and number of samples used are studied in \cite{shamir2014distributed}. Efficient algorithms and lower bounds for distributed statistical estimation are established in \cite{Zhang:2013wq, zhang2013information}. 

However, these works do not fit into our framework, because they assume that each \node has access to data generated \iid from a single distribution. In the case of \cite{Zhang:2013wq, zhang2013information} also $K \ll n / K$, that the number of \nodes $K$ is much smaller than the number of data point on each \node is also assumed. As we stress in the Introduction, these assumptions are far from being satisfied in our setting. Intuitively, relaxing these assumptions should make the problem harder. However, it is not as straightforward to conclude this, as there are certainly particular non-iid data distributions that simplify the problem --- for instance if data are distributed according to separability structure of the objective. Lower bounds on communication complexity of distributed convex optimization of \eqref{eq:problem} are presented in \cite{Arjevani:2015vka}, concluding that for \iid data distributions, existing algorithms already achieve optimal complexity in specific settings. 

Probably first, rather extreme, work \cite{zinkevich2010parallelized} proposed to parallelize SGD in a single round of communication. Each node simply runs SGD on the data available locally, and their outputs are averaged to form a final result. This approach is however not very robust to differences in data distributions available locally, and it has been shown \cite[Appendix A]{DANE} that in general it cannot perform better than using output of a single machine, ignoring all the other data.

Shamir et al. proposed the DANE algorithm, Distributed Approximate Newton \cite{DANE}, to exactly solve a general subproblem available locally, before averaging their solutions. The method relies on similarity of Hessians of local objectives, representing their iterations as an average of inexact Newton steps. We describe the algorithm in greater detail in Section~\ref{sec:algorithms:DANE} as our proposed work builds on it. A quite similar approach was proposed in \cite{mahajan2015efficient}, with richer class class of subproblems that can be formulated locally, and solved approximately. An analysis of inexact version of DANE  and its accelerated variant, AIDE, appeared recently in \cite{reddi2016aide}. Inexact DANE is closely related to the algorithms presented in this chapter. We, however, continue in different direction shaped by the setting of \fedopt.

The DiSCO algorithm \cite{zhang2015communication} of Zhang and Xiao is based on inexact damped Newton method. The core idea is that the inexact Newton steps are computed by distributed preconditioned conjugate gradient, which can be very fast, if the data are distributed in an \iid fashion, enabling a good preconditioner to be computed locally. The theoretical upper bound on number of rounds of communication improves upon DANE and other methods, and in certain settings matches the lower bound presented in \cite{Arjevani:2015vka}. The DiSCO algorithm is related to \cite{lin2014large, zhuang2015distributed}, a distributed truncated Newton method. Although it was reported to perform well in practice, the total number of conjugate gradient iterations may still be high to be considered a communication efficient algorithm.

Common to the above algorithms is the assumption that each \node has access to data points sampled \iid from the same distribution. This assumption is not required only in theory, but can cause the algorithms to converge significantly slower or even diverge (as reported for instance in \cite[Table 3]{DANE}). Thus, these algorithms, at least in their default form, are not suitable for the setting of Federated Optimization presented here.

An algorithm that bypasses the need for \iid data assumption is CoCoA, which provably converges under any distribution of the data, while the convergence rate does depend on properties of the data distribution. The first version of the algorithm was proposed as DisDCA in \cite{Yang:2013vl}, without convergence guarantees. First analysis was introduced in \cite{Jaggi:cocoa}, with further improvements in \cite{Ma:2015ti}, and a more general version in \cite{ma2015distributed}. Recently, its variant for L1-regularized objectives was introduced in \cite{Smith:2015ua}.

The CoCoA framework formulates general local subproblems based on the dual form of \eqref{eq:problem} (See for instance \cite[Eq.\ (2)]{ma2015distributed}). Data points are distributed to \nodes, along with corresponding dual variables. Arbitrary optimization algorithm is used to attain a relative $\Theta$ accuracy on the local subproblem --- by changing only local dual variables. These updates have their corresponding updates to primal variable $w$, which are synchronously aggregated (could be averaging, adding up, or anything in between; depending on the local subproblem formulation).

From the description in this section it appears that the CoCoA framework is the only usable tool for the setting of Federated Optimization. However, the theoretical bound on number of rounds of communications for ill-conditioned problems scales with the number of \nodes $K$. Indeed, as we will show in Section~\ref{sec:experiments} on real data, CoCoA framework does converge very slowly.

\section{Algorithms for \FedOpt}
\label{sec:algorithms}

In this section we introduce the first algorithm that was designed with the unique challenges of \fedopt in mind. Before proceeding with the explanation, we first revisit two important and at first sight unrelated algorithms. The connection between these algorithms helped to motivate our research. Namely, the algorithms are the Stochastic Variance Reduced Gradient (\SVRG) \cite{SVRG, S2GD}, a stochastic method with explicit variance reduction, and the Distributed Approximate Newton (DANE) \cite{DANE} for distributed optimization.

The descriptions are followed by their connection, giving rise to a new distributed optimization algorithm, at first sight almost identical to the \SVRG algorithm, which we call Federated \SVRG (\algname).

Although this algorithm seems to work well in practice in simple circumstances, its performance is still unsatisfactory in the general setting we specify in Section~\ref{sec:problem}. We proceed by making the FSVRG algorithm adaptive to different local data sizes, general sparsity patterns and significant differences in patterns in data available locally, and those present in the entire data set.

\subsection{Desirable Algorithmic Properties}
It is a useful thought experiment to consider the properties one would
hope to find in an algorithm for the non-\iid, unbalanced, and
massively-distributed setting we consider.  In particular:
\begin{enumerate}[(A)]
\item \label{propStartOpt} If the algorithm is initialized to the optimal solution, it stays there.
\item \label{propOneNode} If all the data is on a single \node, the algorithm should converge in $\BO(1)$ rounds of communication.
\item \label{propDisjointFeatures} If each feature occurs on a single \node, so the problems are fully decomposable (each machine is essentially learning a disjoint block of parameters), then the algorithm should converge in $\BO(1)$ rounds of communication\footnote{This is valid only for generalized linear models.}.
\item \label{propIdentical} If each \node contains an identical dataset, then the algorithm should converge in $\BO(1)$ rounds of communication. 
\end{enumerate}
For convex problems, ``converges'' has the usual technical meaning of finding a solution sufficiently close to the global minimum, but these properties also make sense for non-convex problems where ``converge'' can be read as ``finds a solution of sufficient quality''. In these statements, $\BO(1)$ round is ideally exactly one round of communication.

Property \ref{propStartOpt} is valuable in any optimization setting.
Properties \ref{propOneNode} and \ref{propDisjointFeatures} are
extreme cases of the \fedopt setting (non-\iid, unbalanced, and sparse),
whereas \ref{propIdentical} is an extreme case of the classic
distributed optimization setting (large amounts of \iid data per
machine). Thus, \ref{propIdentical} is the least important property
for algorithms in the \fedopt setting.

\subsection{\SVRG}
\label{sec:SVRG}

The \SVRG algorithm \cite{SVRG, S2GD} is a stochastic method designed to solve problem~\eqref{eq:problem} on a single \node. We present it as Algorithm~\ref{alg:S2GD} in a slightly simplified form.

\begin{algorithm}[!h]
\begin{algorithmic}[1]
\State \textbf{parameters:} $m$ = number of stochastic steps per epoch, $h$ = stepsize
\For {$s = 0, 1, 2, \dots$}
	\State Compute and store $\nabla P(w^t) = \frac{1}{n} \sum_{i=1}^n \nabla f_i(w^t)$ \label{line:fullgrad}
	\Comment Full pass through data
	\State Set $w = w^t$
	\For {$t = 1$ to $m$}
		\State Pick $i \in \{ 1, 2, \dots, n \}$, uniformly at random
		\State $w = w - h \left( \nabla f_i(w) - \nabla f_i(w^t) + \nabla P(w^t) \right) $ 
		\Comment Stochastic update \label{line:stoch_update}
	\EndFor
	\State $w^{t+1} = w$
\EndFor
\end{algorithmic}

\caption{\SVRG}
\label{alg:S2GD}
\end{algorithm}

The algorithm runs in two nested loops. In the outer loop, it computes gradient of the entire function $P$ (Line~\ref{line:fullgrad}). This constitutes for a full pass through data --- in general expensive operation one tries to avoid unless necessary. This is followed by an inner loop, where $m$ fast stochastic updates are performed. In practice, $m$ is typically set to be a small multiple (1--5) of $n$. Although the theoretically optimal choice for $m$ is a small multiple of a condition number \cite[Theorem 6]{S2GD}, this is often of the same order as $n$ in practice.

The central idea of the algorithm is to avoid using the stochastic gradients to estimate the entire gradient $\nabla P(w)$ directly. Instead, in the stochastic update in Line \ref{line:stoch_update}, the algorithm evaluates two stochastic gradients, $\nabla f_i(w)$ and $\nabla f_i(w^t)$. These gradients are used to estimate the change of the gradient of the entire function between points $w^t$ and $w$, namely $\nabla P(w) - \nabla P(w^t)$. Using this estimate together with $\nabla P(w^t)$ pre-computed in the outer loop, yields an unbiased estimate of $\nabla P(w)$.

Apart from being an unbiased estimate, it could be intuitively clear that if $w$ and $w^t$ are close to each other, the variance of the estimate $\nabla f_i(w) - \nabla f_i(w^t)$ should be small, resulting in estimate of $\nabla P(w)$ with small variance. As the inner iterate $w$ goes further, variance grows, and the algorithm starts a new outer loop to compute new full gradient $\nabla P(w^{t+1})$ and reset the variance.

The performance is well understood in theory. For $\lambda$-strongly convex $P$ and $L$-smooth functions $f_i$, convergence results are in the form
\begin{equation}
\label{eq:S2GD:convergence}
\E{ P(w^t) - P(w^*)} \leq c^t [P(w^0) - P(w^*)],
\end{equation}
where $w^*$ is the optimal solution, and $c = \Theta \left( \frac{1}{m h} \right) + \Theta(h)$.\footnote{See \cite[Theorem 4]{S2GD} and \cite[Theorem 1]{SVRG} for details.}

It is possible to show \cite[Theorem 6]{S2GD} that for appropriate choice of parameters $m$ and $h$, the convergence rate~\eqref{eq:S2GD:convergence} translates to the need of $$ \left( n + \BO ( L/\lambda) \right) \log(1/\epsilon) $$
evaluations of $\nabla f_i$ for some $i$ to achieve $\E{ P(w) - P(w^*)} < \epsilon$.

\subsection{Distributed Problem Formulation}
\label{sec:problem}

In this section, we introduce notation and specify the structure of the distributed version of the problem we consider \eqref{eq:problem}, focusing on the case where the $f_i$ are convex.
We assume the data $\{x_i, y_i\}_{i=1}^n$, describing functions $f_i$ are stored
across a large number of \nodes.

Let $K$ be the number of \nodes.  Let $\pp_k$ for $k \in \{1, \dots, K\}$ denote a partition of data point indices $\{1, \dots, n\}$, so $\pp_k$ is the set stored on \node $k$, and define $n_k = |\pp_k|$. That is, we assume that $\pp_k \cap \pp_l = \emptyset$ whenever $k \neq l$, and $\sum_{k=1}^K n_k = n$.
We then define local empirical loss as
\begin{equation}\label{eq:98hs98hs8}
F_k(w) \eqdef \frac{1}{n_k} \sum_{i \in \mathcal{P}_k} f_i(w),
\end{equation}
which is the local objective based on the data stored on machine $k$. We can then rephrase the objective~\eqref{eq:problem} as
\begin{equation}
\label{eq:problem:distributed}
P(w) = \sum_{k = 1}^K \frac{n_k}{n} F_k(w) 
= \sum_{k=1}^K \frac{n_k}{n} \cdot \frac{1}{n_k} \sum_{i \in \mathcal{P}_k} f_i(w).
\end{equation}

The way to interpret this structure is to see the empirical loss $P(w) = \frac1n \sum_{i=1}^n f_i(w)$ as a convex combination of the local empirical losses $F_k(w)$, available locally to \node $k$. Problem \eqref{eq:problem} then takes the simplified form
\begin{equation}
\label{eq:problem:distributed:simple}
\min_{w\in \R^d} P(w) \equiv \sum_{k=1}^K \frac{n_k}{n} F_k(w).
\end{equation}

\subsection{DANE}
\label{sec:algorithms:DANE}

In this section, we introduce a general reasoning providing stronger intuitive support for the DANE algorithm \cite{DANE}, which we describe in detail below. We will follow up on this reasoning in Appendix~\ref{sec:appendix} and draw a connection between two existing methods that was not known in the literature.

If we wanted to design a distributed algorithm for solving the above problem \eqref{eq:problem:distributed:simple}, where \node $k$ contains the data describing function $F_k$. The first, and as we shall see, a rather naive idea is to ask each node to minimize their local functions, and average the results (a variant of this idea appeared in \cite{zinkevich2010parallelized}):
$$ w_k^{t+1} = \arg \min_{w\in \R^d} F_k(w), 
\qquad w^{t+1} = \sum_{k=1}^K \frac{n_k}{n} w_k^{t+1}. $$

Clearly, it does not make sense to run this algorithm for more than one iteration as the output $w$ will always be the same. This is simply because $w_k^{t+1}$ does not depend on $t$. In other words, this method effectively performs just a single round of communication. While the simplicity is appealing, the drawback of this method is that it can't work. Indeed, there is no reason to expect that in general the solution of \eqref{eq:problem:distributed:simple} will be a weighted average of the local solutions, unless the local functions are all the same --- in which case we do not need a distributed algorithm in the first place and can instead solve the much simpler problem $\min_{w\in \R^d} F_1(w)$. This intuitive reasoning can be also formally supported, see for instance \cite[Appendix A]{DANE}.

One remedy to the above issue is to modify the local problems before each aggregation step. One of the simplest strategies would be to perturb the local function $F_k$ in iteration $t$ by a quadratic term of the form: $-( a_k^t)^T w + \tfrac{\mu}{2}\|w-w^t\|^2$ and to ask each node to solve the perturbed problem instead. With this change, the improved method then takes the form
\begin{equation}
\label{eq:alg:perturbation}
w_k^{t+1} = \arg \min_{w\in \R^d} F_k(w) - (a_k^t)^T w  + \frac{\mu}{2}\|w-w^t\|^2, \qquad w^{t+1} = \frac{1}{K}\sum_{k=1}^K w_k^{t+1}.
\end{equation}

The idea behind iterations of this form is the following. We would like each node $k\in [K]$ to use as much curvature information stored in $F_k$ as possible. By keeping the function $F_k$ in the subproblem in its entirety, we are keeping the curvature information nearly intact --- the Hessian of the subproblem is $\nabla^2 F_k + \mu I$, and we can even choose $\mu = 0$.

As described, the method is not yet well defined, since we have not described how the vectors $a_k^t$ would change from iteration to iteration, and how one should choose $\mu$.  In order to get some insight into how such a method might work, let us examine the optimality conditions. Asymptotically as $t\to \infty$, we would like $a_k^t$ to be such that the minimum of each subproblem is equal to $w^*$; the minimizer of \eqref{eq:problem:distributed:simple}. Hence, we would wish for $w^*$ to be the solution of 
$$ \nabla F_k(w) - a_k^t +\mu(w - w^t) = 0. $$

Hence, in the limit, we would ideally like to choose $a_k^t = \nabla F_k(w^*) + \mu(w^* - w^t) \approx \nabla F_k(w^*)$, since $w^* \approx w^t$.  Not knowing $w^*$ however, we cannot hope to be able to simply set $a_k^t$ to this value. Hence, the second option is to come up with an update rule which would guarantee that $a_k^t $ converges to $\nabla F_k(w^*)$ as $t \to \infty$. Notice at this point that it has been long known in the optimization community that the gradient of the objective at the optimal point is intimately related to the optimal solution of a dual problem. Here the situation is further complicated by the fact that we need to learn $K$ such gradients. In the following, we show that DANE is in fact a particular instantiation of the scheme above.

\paragraph{DANE.} We present the Distributed Approximate Newton algorithm (DANE) \cite{DANE}, as Algorithm~\ref{alg:DANE}. The algorithm was originally analyzed for solving the problem of structure \eqref{eq:problem:distributed}, with $n_k$ being identical for each $k$ --- i.e., each computer has the same number of data points. Nothing prevents us from running it in our more general setting though.

\begin{algorithm}[!h]
\caption{Distributed Approximate Newton (DANE)}\label{alg:DANE}
\begin{algorithmic}[1]
\State {\bf Input:} regularizer $\mu \geq 0$, parameter $\eta$ (default: $\mu = 0, \eta = 1$)
\For {$s = 0, 1, 2, \dots $}
  \State Compute $\nabla P(w^t) = \frac{1}{n} \sum_{i=1}^n \nabla f_i(w^t)$ and distribute to all machines \label{line:gradient}
  \State For each \node $k \in \{1, \dots, K\}$, solve \label{line:subproblem}
  \begin{equation}
  \label{eq:DANEsubproblem}
  w_k = \argmin_{w \in \R^d} \left\{ F_k(w) - \left( \nabla F_k(w^t) - \eta \nabla P(w^t)  \right)^T w  + \frac{\mu}{2} \| w - w^t \|^2 \right\}
  \end{equation}
  \State Compute $w^{t+1} = \frac{1}{K}\sum_{k=1}^K w_k$ \label{line:aggregate}
\EndFor
\end{algorithmic}
\end{algorithm}

As alluded to earlier, the main idea of DANE is to form a local subproblem, dependent only on local data, and gradient of the entire function --- which can be computed in a single round of communication (Line \ref{line:gradient}). The subproblem is then solved exactly (Line \ref{line:subproblem}), and updates from individual \nodes are averaged to form a new iterate (Line \ref{line:aggregate}). This approach allows any algorithm to be used to solve the local subproblem~\eqref{eq:DANEsubproblem}. As a result, it often achieves communication efficiency in the sense of requiring expensive local computation between rounds of communication, hopefully rendering the time needed for communication insignificant (see Section~\ref{sec:relatedWork:paradigm}). Further, note that DANE belongs to the family of distributed method that operate via the quadratic perturbation trick \eqref{eq:alg:perturbation} with 
$$ a_k^t = \nabla F_k(w^t) - \eta \nabla P(w^t). $$
If we assumed that the method works, i.e., that $w^t \to w^*$ and hence $\nabla P(w^t) \to \nabla P(w^*) = 0$, then $a_k^t \to \nabla F_k(w^*)$, which agrees with the earlier discussion.

In the default setting when $\mu = 0$ and $\eta = 1$, DANE achieves desirable property
\ref{propIdentical} (immediate convergence when all local datasets are
identical), since in this case $\nabla F_k(w^t) - \eta \nabla P(w^t)
= 0$, and so we exactly minimize $F_k(w) = P(w)$ on each machine.
For any choice of $\mu$ and $\eta$, DANE also achieves property
\ref{propStartOpt}, since in this case $\nabla P(w^t) = 0$, and
$w^t$ is a minimizer of $F_k(w) - \nabla F_k(w^t)\cdot w$ as well
as of the regularization term.
Unfortunately, DANE does not achieve the more \fedopt-specific desirable properties \ref{propOneNode} and \ref{propDisjointFeatures}.

The convergence analysis for DANE assumes that the functions are twice differentiable, and relies on the assumption that each \node has access to \iid samples from the same underlying distribution. This implies that that the Hessians of $\nabla^2 F_k(w)$ are similar to each other \cite[Lemma 1]{DANE}. In case of linear regression, with $\lambda = \BO(1 / \sqrt{n})$-strongly convex functions, the number of DANE iterations needed to achieve $\epsilon$-accuracy is $\BO(K \log(1/\epsilon))$. However, for general $L$-smooth loss, the theory is significantly worse, and does not match its practical performance.

The practical performance also depends on the additional local regularization parameter $\mu$. For small number of \nodes $K$, the algorithm converges quickly with $\mu = 0$. However, as reported \cite[Figure 3]{DANE}, it can diverge quickly with growing $K$. Bigger $\mu$ makes the algorithm more stable at the cost of slower convergence. Practical choice of $\mu$ remains an open question.

\subsection{\SVRG meets DANE}

As we mentioned above, the DANE algorithm can perform poorly in certain settings, even without the challenging aspects of \fedopt. Another point that is seen as drawback of DANE is the need to find the \emph{exact} minimum of~\eqref{eq:DANEsubproblem} --- this can be feasible for quadratics with relatively small dimension, but infeasible or extremely expensive to achieve for other problems. We adapt the idea from the CoCoA algorithm \cite{ma2015distributed}, in which an arbitrary optimization algorithm is used to obtain relative $\Theta$ accuracy on a locally defined subproblem. We replace the exact optimization with an approximate solution obtained by using any optimization algorithm.

Considering all the algorithms one could use to solve~\eqref{eq:DANEsubproblem}, the \SVRG algorithm seems to be a particularly good candidate. Starting the local optimization of \eqref{eq:DANEsubproblem} from point $w^t$, the algorithm automatically has access to the derivative at $w^t$, which is identical for each \node\xspace --- $\nabla P(w^t)$. Hence, the \SVRG algorithm can skip the initial expensive operation, evaluation of the entire gradient (Line~3, Algorithm~\ref{alg:S2GD}), and proceed only with the stochastic updates in the inner loop.

It turns out that this modified version of the DANE algorithm is equivalent to a distributed version of \SVRG.

\begin{proposition}
\label{prop:equivalence}
Consider the following two algorithms.
\begin{enumerate}
\item Run the DANE algorithm (Algorithm~\ref{alg:DANE}) with $\eta = 1$ and $\mu = 0$, and use \SVRG (Algorithm~\ref{alg:S2GD}) as a local solver for \eqref{eq:DANEsubproblem}, running it for a single iteration, initialized at point $w^t$.
\item Run a distributed variant of the \SVRG algorithm, described in Algorithm~\ref{alg:DS2GDnaive}.
\end{enumerate}

The algorithms are equivalent in the following sense. If both start from the same point $w^t$, they generate identical sequence of iterates $\{ w^t \}$.
\end{proposition}

\begin{proof}
We construct the proof by showing that single step of the \SVRG algorithm applied to the problem \eqref{eq:DANEsubproblem} on computer $k$ is identical to the update on Line~\ref{line:DSVRGupdate} in Algorithm~\ref{alg:DS2GDnaive}.

The way to obtain a stochastic gradient of \eqref{eq:DANEsubproblem} is to sample one of the functions composing $F_k(w) = \frac{1}{n_k} \sum_{i \in \mathcal{P}_k} f_i(w)$, and add the linear term $\nabla F_k(w^t) - \eta P(w^t)$, which is known and does not need to be estimated. Upon sampling an index $i \in \mathcal{P}_k$, the update direction follows as 
\begin{align*}
\left[ \nabla f_i(w) - \nabla F_k(w^t) - \nabla P(w^t) \right] - \left[ \nabla f_i(w^t) - \nabla F_k(w^t) - \nabla P(w^t) \right] + \nabla P(w^t) = \\
\nabla f_i(w) - \nabla f_i(w^t) + \nabla P(w^t)
\end{align*}
which is identical to the direction in Line~\ref{line:DSVRGupdate} in Algorithm~\ref{alg:DS2GDnaive}. The claim follows by chaining the identical updates to form identical iterate $w^{t+1}$.
\end{proof}

\begin{algorithm}[!h]
\begin{algorithmic}[1]
\State \textbf{parameters:} $m$ = \# of stochastic steps per epoch, $h$ = stepsize, data partition $\{\pp_k\}_{k=1}^K$
\For {$s = 0, 1, 2, \dots$}
	\Comment Overall iterations
	\State Compute $\nabla P(w^t) = \frac{1}{n} \sum_{i=1}^n \nabla f_i(w^t)$
	\For {$k = 1$ to $K$} \textbf{in parallel} over \nodes $k$
	\Comment Distributed loop
	\State Initialize: $w_k = w^t$
	\For {$t = 1$ to $m$}
		\Comment Actual update loop
		\State Sample $i \in \mathcal{P}_k$ uniformly at random
		\State $ w_k = w_k - h \left( \nabla f_i(w_k) - \nabla f_i(w^t) + \nabla P(w^t) \right) $
		\label{line:DSVRGupdate}
		\EndFor
	\EndFor
	\State $w^{t+1} = w^t + \frac{1}{K} \sum_{k=1}^K (w_k - w^t)$
	\Comment Aggregate
\EndFor
\end{algorithmic}

\caption{naive Federated \SVRG (\algname)}
\label{alg:DS2GDnaive}
\end{algorithm}

\begin{remark}
The algorithms considered in Proposition~\ref{prop:equivalence} are inherently stochastic. The statement of the proposition is valid under the assumption that in both cases, identical sequence of samples $i \in \mathcal{P}_k$ would be generated by all \nodes $k \in \{1, 2, \dots, K\}$.
\end{remark}

\begin{remark}
In the Proposition~\ref{prop:equivalence} we consider the DANE algorithm with particular values of $\eta$ and $\mu$. The Algorithm~\ref{alg:DS2GDnaive} and the Proposition can be easily generalized, but we present only the default version for the sake of clarity.
\end{remark}

Since the first version of this work, this connection has been mentioned in \cite{reddi2016aide}, which analyzes an inexact version of the DANE algorithm. We proceed by adapting the above algorithm to other challenges arising in the context of \fedopt.

\subsection{Federated \SVRG}

Empirically, the Algorithm~\ref{alg:DS2GDnaive} fits in the model of distributed optimization efficiency described in Section~\ref{sec:relatedWork:paradigm}, since we can balance how many stochastic iterations should be performed locally against communication costs. However, several modifications are necessary to achieve good performance in the full \fedopt setting (Section~\ref{sec:problem}). 
Very important aspect that needs to be addressed is that the number of data points available to a given node can differ greatly from the average number of data points available to any single \node. Furthermore, this setting always comes with the data available locally being clustered around a specific pattern, and thus not being a representative sample of the overall distribution we are trying to learn.
In the Experiments section we focus on the case of L2 regularized logistic regression, but the ideas carry over to other generalized linear prediction problems.

\subsubsection{Notation}

Note that in large scale generalized linear prediction problems, the data arising are almost always sparse, for example due to bag-of-words style feature representations. This means that only a small subset of $d$ elements of vector $x_i$ have nonzero values. In this class of problems, the gradient $\nabla f_i(w)$ is a multiple of the data vector $x_i$. This creates additional complications, but also potential for exploitation of the problem structure and thus faster algorithms. Before continuing, let us summarize and denote a number of quantities needed to describe the algorithm. 
\begin{itemize}[noitemsep]
\item $n$ --- number of data points / training examples / functions.
\item $\mathcal{P}_k$ --- set of indices, corresponding to data points stored on device $k$.
\item $n_k = |\mathcal{P}_k|$ --- number of data points stored on device $k$.
\item $n^j = \left| \{ i \in \{ 1, \dots, n\} : x_i^T e_j \neq 0 \} \right|$ --- the number of data points with nonzero $j^{th}$ coordinate
\item $n_k^j = \left| \{  i \in \mathcal{P}_k : x_i^T e_j \neq 0 \} \right| $ --- the number of data points stored on \node $k$ with nonzero $j^{th}$ coordinate
\item $\phi^j = n^j / n$ --- frequency of appearance of nonzero elements in $j^{th}$ coordinate
\item $\phi_k^j = n_k^j / n_k$ --- frequency of appearance of nonzero elements in $j^{th}$ coordinate on \node $k$
\item $s_k^j = \phi^j / \phi_k^j$ --- ratio of global and local appearance frequencies on \node $k$ in $j^{th}$ coordinate
\item $S_k = \text{Diag}(s_k^j)$ --- diagonal matrix, composed of $s_k^j$ as $j^{th}$ diagonal element
\item $\omega^j = \left|\{ \mathcal{P}_k : n_k^j \neq 0 \}\right|$ --- Number of \nodes that contain data point with nonzero $j^{th}$ coordinate
\item $a^j = K / \omega^j$ --- aggregation parameter for coordinate $j$
\item $A = \text{Diag}(a_j)$ --- diagonal matrix composed of $a_j$ as $j^{th}$ diagonal element
\end{itemize}

With these quantities defined, we can state our proposed algorithm as Algorithm~\ref{alg:DS2GDv7}. Our experiments show that this algorithm works very well in practice, but the motivation for the particular scaling of the updates may not be immediately clear.  In the following section we provide the intuition that lead to the development of this algorithm.

\begin{algorithm}[!h]
\begin{algorithmic}[1]
\State \textbf{parameters:} $h$ = stepsize, data partition $\{\pp_k\}_{k=1}^K$, \newline {\color{white}} \hspace{64pt} diagonal matrices $A, S_k \in \R^{d \times d}$ for $k \in \{1, \dots, K\}$ 
\For {$s = 0, 1, 2, \dots$}
	\Comment Overall iterations
	\State Compute $\nabla P(w^t) = \frac{1}{n} \sum_{i=1}^n \nabla f_i(w^t)$
	\For {$k = 1$ to $K$} \textbf{in parallel} over \nodes $k$
	\Comment Distributed loop
	\State Initialize: $w_k = w^t$ and $h_k = h / n_k$
	\State Let $\{ i_t \}_{t=1}^{n_k}$ be random permutation of $\mathcal{P}_k$
	\For {$t = 1, \dots, n_k$}
		\Comment Actual update loop
		\State $ w_k = w_k - h_k \left( S_k \left[ \nabla f_{i_t}(w_k) - \nabla f_{i_t}(w^t) \right] + \nabla P(w^t) \right) $
		\EndFor
	\EndFor
	\State $w^t = w^t + A \sum_{k=1}^K \frac{n_k}{n} (w_k - w^t)$
	\Comment Aggregate
\EndFor
\end{algorithmic}
\caption{Federated \SVRG (\algname)}
\label{alg:DS2GDv7}
\end{algorithm}

\subsubsection{Intuition Behind \algname Updates}
\label{sec:algorithms:intuition}

The difference between the Algorithm~\ref{alg:DS2GDv7} and Algorithm~\ref{alg:DS2GDnaive} is in the introduction of the following properties.

\begin{enumerate}
\item Local stepsize --- $h_k = h / n_k$.
\item Aggregation of updates proportional to partition sizes --- $\frac{n_k}{n} (w_k - w^t)$
\item Scaling stochastic gradients by diagonal matrix --- $S_k$
\item Per-coordinate scaling of aggregated updates --- $A (w_k - w^t)$
\end{enumerate}

Let us now explain what motivated us to get this particular implementation. 

As a simplification, assume that at some point in time, we have for some $w$, $w_k = w$ for all $k \in [K]$. In other words, all the \nodes have the same local iterate. Although this is not exactly the case in practice, thinking about the issue in this simplified setting will give us insight into what would be meaningful to do if it was true. Further, we can hope that the reality is not too far from the simplification and it will still work in practice. Indeed, all \nodes do start from the same point, and adding the linear term $\nabla F_k(w^t) - \nabla P(w^t)$ to the local objective forces all \nodes to move in the same direction, at least initially.

Suppose the \nodes are about to make a single step synchronously. Denote the update direction on \node $k$ as $G_k = \nabla f_i(w) - \nabla f_i(w^t) + \nabla P(w^t)$, where $i$ is sampled uniformly at random from $\mathcal{P}_k$.

If we had only one \node, i.e., $K = 1$, it is clear that we would have $\E{G_1} = \nabla P(w^t)$. If $K$ is more than $1$, the values of $G_k$ are in general biased estimates of $\nabla P(w^t)$. We would like to achieve the following: $\E{ \sum_{k=1}^K \alpha_k G_k } = \nabla P(w^t)$, for some choice of $\alpha_k$. This is motivated by the general desire to make stochastic first-order methods to make a gradient step in expectation.

We have 
\begin{equation*}
\E{ \sum_{k=1}^K \alpha_k G_k } = \sum_{k=1}^K \alpha_k \frac{1}{n_k} \sum_{i \in \mathcal{P}_k} \left[ \nabla f_i(w) - \nabla f_i(w^t) + \nabla P(w^t) \right].
\end{equation*}
By setting $\alpha_k = \frac{n_k}{n}$, we get
\begin{equation*}
\E{ \sum_{k=1}^K \alpha_k G_k } = \frac{1}{n} \sum_{k=1}^K \sum_{i \in \mathcal{P}_k} \left[ \nabla f_i(w) - \nabla f_i(w^t) + \nabla P(w^t) \right] = \nabla P(w).
\end{equation*}

This motivates the aggregation of updates from \nodes proportional to $n_k$, the number of data points available locally (Point 2).

Next, we realize that if the local data sizes, $n_k$, are not identical, we likely don't want to do the same number of local iterations on each \node $k$. Intuitively, doing one pass through data (or a fixed number of passes) makes sense. As a result, the aggregation motivated above does not make perfect sense anymore. Nevertheless, we can even it out, by setting the stepsize $h_k$ inversely proportional to $n_k$, making sure each \node makes progress of roughly the same magnitude overall. Hence, $h_k = h / n_k$ (Point 1).

To motivate the Point 3, scaling of stochastic gradients by diagonal matrix $S_k$, consider the following example. We have $1,000,000$ data points, distributed across $K = 1,000$ \nodes. When we look at a particular feature of the data points, we observe it is non-zero only in $1,000$ of them. Moreover, all of them happen to be stored on a single \node, that stores only these $1,000$ data points. Sampling a data point from this \node and evaluating the corresponding gradient, will clearly yield an estimate of the gradient $\nabla P(w)$ with $1000$-times larger magnitude. This would not necessarily be a problem if done only once. However, repeatedly sampling and overshooting the magnitude of the gradient will likely cause the iterative process to diverge quickly.

Hence, we scale the stochastic gradients by a diagonal matrix. This can be seen as an attempt to enforce the estimates of the gradient to be of the correct magnitude, conditioned on us, algorithm designers, being aware of the structure of distribution of the sparsity pattern.

Let us now highlight some properties of the modification in Point 4. Without any extra information, or in the case of fully dense data, averaging the local updates is the only way that actually makes sense --- because each \node outputs approximate solution of a proxy to the overall objective, and there is no induced separability structure in the outputs such as in CoCoA \cite{ma2015distributed}. However, we could do much more in the other extreme. If the sparsity structure is such that each data point only depends on one of disjoint groups of variables, and the data were distributed according to this structure, we would efficiently have several disjoint problems. Solving each of them locally, and adding up the results would solve the problem in single iteration --- desired algorithm property (C).

What we propose is an interpolation between these two settings, on a per-variable basis. If a variable appears in data on each \node, we are going to take average. However, the less \nodes a particular variable appear on, the more we want to trust those few \nodes in informing us about the meaningful update to this variable --- or alternatively, take a longer step. Hence the per-variable scaling of aggregated updates.

\subsection{Further Notes}
\label{sec:algorithms:furthernotes}

Looking at the Proposition~\ref{prop:equivalence}, we identify equivalence of two algorithms, take the second one and try modify it to make it suitable for the setting of \fedopt. A question naturally arise: Is it possible to achieve the same by modifying the first algorithm suitable for \fedopt\ --- by only altering the local optimization objective? 

We indeed tried to experiment with idea, but we don't report the details for two reasons. First, the requirement of exact solution of the local subproblem is often impractical. Relaxing it gradually moves us to the setting we presented in the previous sections. But more importantly, using this approach we have only managed to get results significantly inferior to those reported later in the Experiments section.

\section{Experiments}
\label{sec:experiments}

In this section we present the first experimental results in the setting of \fedopt. In particular, we provide results on a dataset based on public Google+ posts\footnote{The posts were public at the time the experiment was performed, but since a user may decide to delete the post or make it non-public, we cannot release (or even permanently store) any copies of the data.}, clustered by user --- simulating each user as a independent \node. This preliminary experiment demonstrates why none of the existing algorithms are suitable for \fedopt, and the robustness of our proposed method to challenges arising there. 

\subsection{Predicting Comments on Public Google+ Posts}
\label{sec:experiments:gplus}
The dataset presented here was generated based on public Google+ posts. We randomly picked $10,000$ authors that have at least $100$ public posts in English, and try to predict whether a post will receive at least one comment (that is, a binary classification task). 

We split the data chronologically on a per-author basis, taking the earlier $75\%$ for training and the following $25\%$ for testing. The total number of training examples is $n = 2,166,693$. We created a simple bag-of-words language model, based on the $20,000$ most frequent words in dictionary based on all Google+ data. This results in a problem with dimension $d = 20,002$. The extra two features represent a bias term and variable for unknown word. We then use a logistic regression model to make a prediction based on these features.

We shape the distributed optimization problem as follows. Suppose that each user corresponds to one \node, resulting in $K = 10,000$. The average $n_k$, number of data points on \node $k$ is thus roughly $216$. However, the actual numbers $n_k$ range from $75$ to $9,000$, showing the data is in fact substantially unbalanced.

It is natural to expect that different users can exhibit very different patterns in the data generated. This is indeed the case, and hence the distribution to \nodes cannot be considered an \iid sample from the overall distribution. Since we have a bag-of-words model, our data are very sparse --- most posts contain only small fraction of all the words in the dictionary. This, together with the fact that the data are naturally clustered on a per-user basis, creates additional challenge that is not present in the traditional distributed setting. 

\begin{figure}[!h]
\centering
\includegraphics[width=0.5\textwidth]{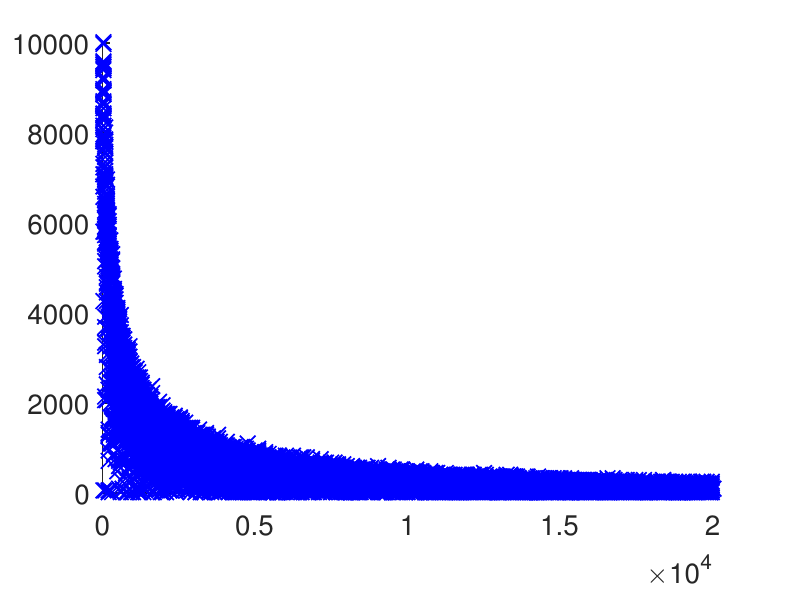}
\caption{Features vs. appearance on \nodes.  The $x$-axis is a feature index, and the $y$-axis represents the number of \nodes where a given feature is present.}
\label{fig:omegaprime}
\end{figure}

Figure~\ref{fig:omegaprime} shows the frequency of different features across \nodes. Some features are present everywhere, such as the bias term, while most features are relatively rare. In particular, over $88\%$ of features are present on fewer than $1,000$ \nodes. However, this distribution does not necessarily resemble the overall appearance of the features in data examples. For instance, while an unknown word is present in data of almost every user, it is far from being contained in every data point.

\paragraph{Naive prediction properties.} Before presenting the results, it is useful to look at some of the important basic prediction properties of the data. We use L2-regularized logistic regression, with regularization parameter $\lambda = 1/n$. We chose $\lambda$ to be the best in terms of test error in the optimal solution.
\begin{itemize} \itemsep -2pt
\item If one chooses to predict $-1$ (no comment), classification error is $\textbf{33.16}\%$. 
\item The optimal solution of the global logistic regression problem yields $\textbf{26.27}\%$ test set error. 
\item Predicting the per-author majority from the training data yields $\textbf{17.14}\%$ test error. That is, predict $+1$ or $-1$ for all the posts of an author, based on which label was more common in that author's training data. This indicates that knowing the author is actually more useful than knowing what they said, which is perhaps not surprising.
\end{itemize}

In summary, this data is representative for our motivating application in \fedopt. It is possible to improve upon naive baseline using a fixed global model. Further, the per-author majority result suggests it is possible to improve further by adapting the global model to each user individually. Model personalization is common practice in industrial applications, and the techniques used to do this are orthogonal to the challenges of \fedopt. Exploring its performance is a natural next step, but beyond the scope of this work.

While we do not provide experiments for per user personalized models, we remark that this could be a good descriptor of how far from \iid the data is distributed. Indeed, if each \node has access to an \iid sample, any adaptation to local data is merely over-fitting. However, if we can significantly improve upon the global model by per user/\node adaptation, this means that the data available locally exhibit patterns specific to the particular \node.

The performance of the Algorithm~\ref{alg:DS2GDv7} is presented below. The only parameter that remains to be chosen by user is the stepsize $h$. We tried a set of stepsizes, and retrospectively choose one that works best --- a typical practice in machine learning.

\begin{figure}[!h]
\centering
\includegraphics[width=0.48\textwidth]{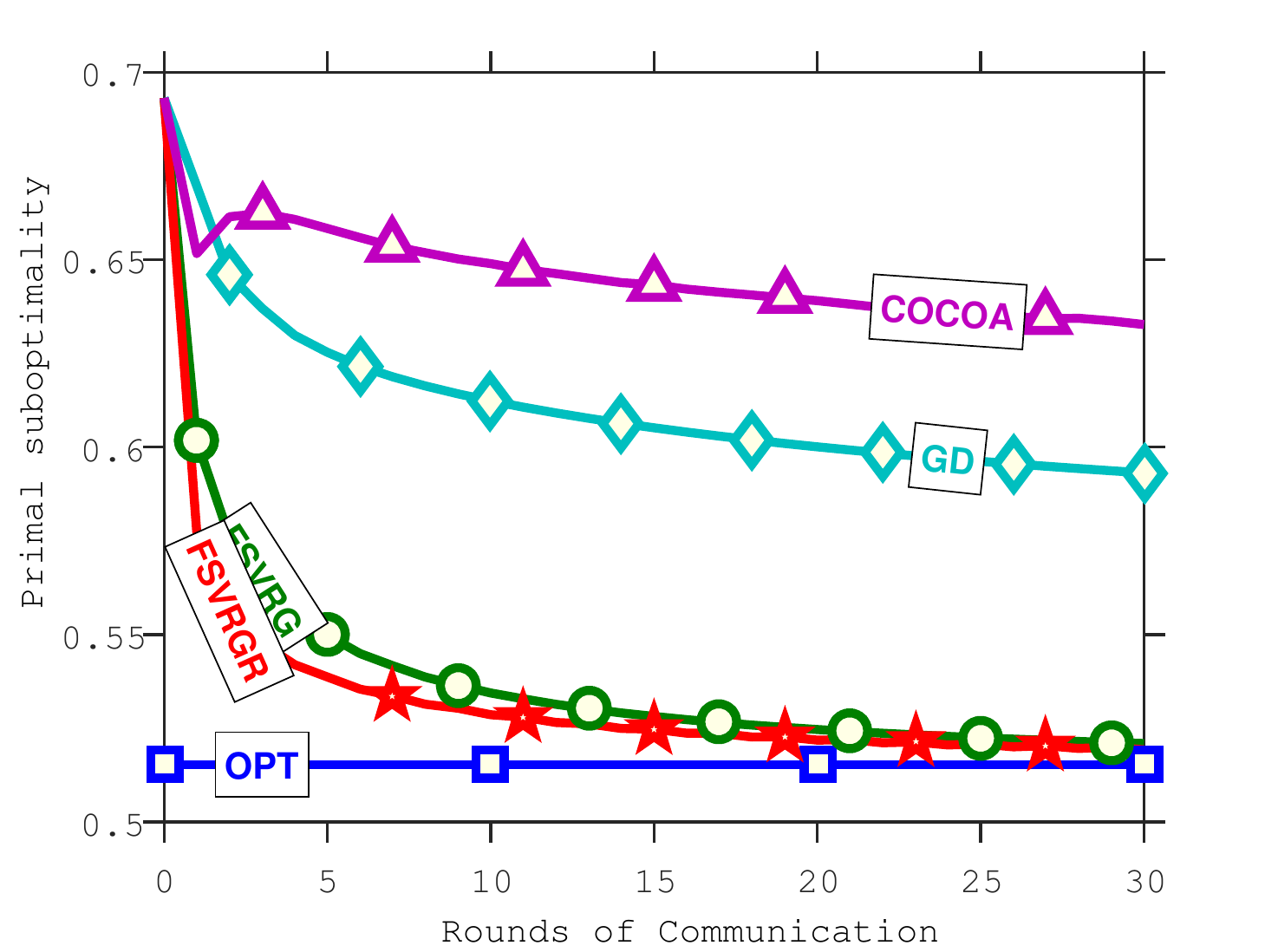}
\includegraphics[width=0.48\textwidth]{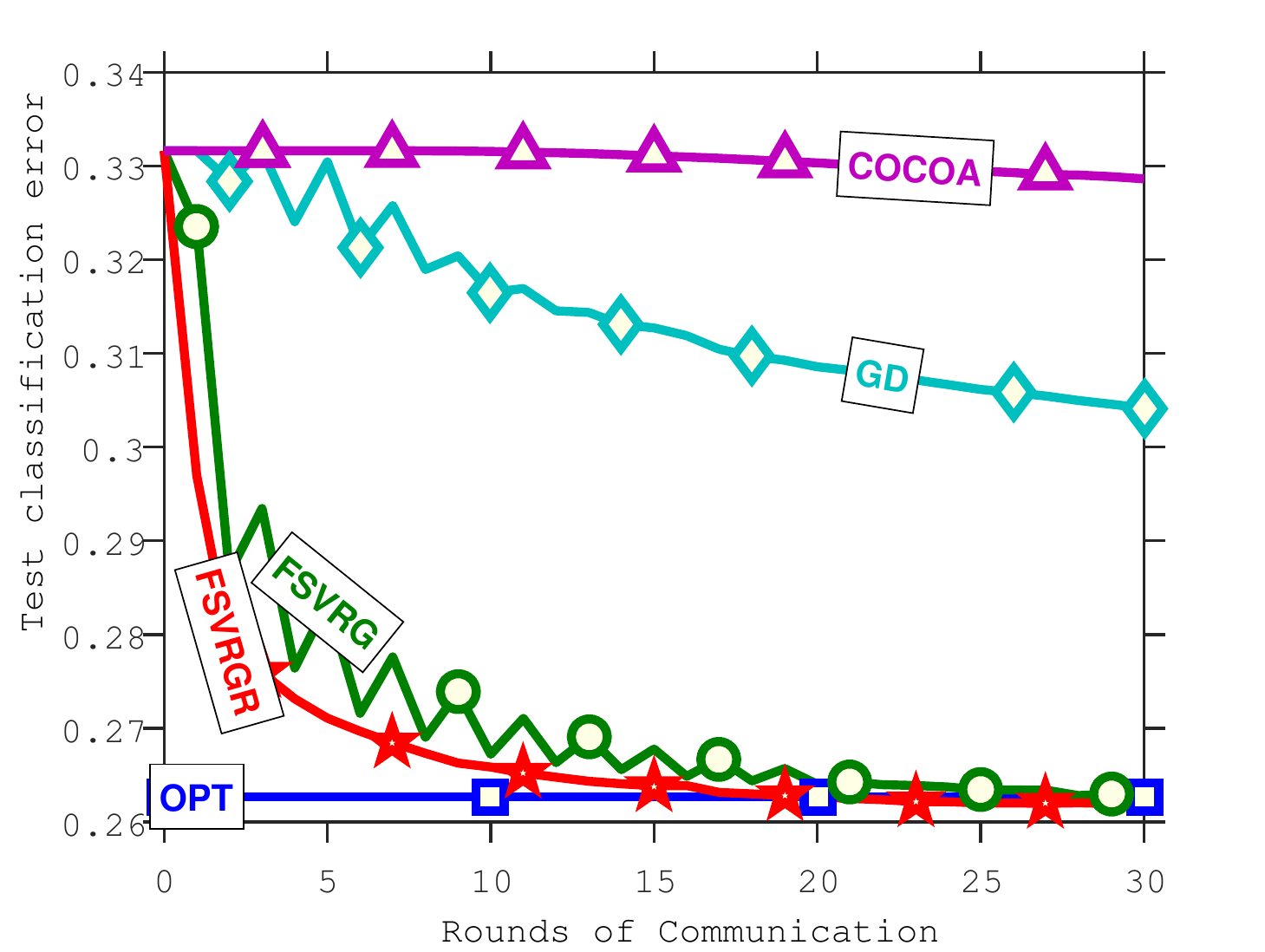}
\caption{Rounds of communication vs. objective function (left) and test prediction error (right).}
\label{fig:ex_final_001}
\end{figure}

In Figure~\ref{fig:ex_final_001}, we compare the following optimization algorithms\footnote{We thank Mark Schmidt for his \texttt{prettyPlot} function, available on his website.}:
\begin{itemize} \itemsep -2pt
 \item The blue squares (OPT) represent the best possible offline value (the optimal value of the optimization task in the first plot, and the test error corresponding to the optimum in the second plot).
 \item  The teal diamonds (GD) correspond to a simple distributed gradient descent. 
 \item The purple triangles (COCOA) are for the CoCoA+ algorithm \cite{ma2015distributed}.
\item The green circles (FSVRG) give values for our proposed algorithm.
\item The red stars (FSVRGR) correspond to the same algorithm applied to the same problem with randomly reshuffled data.  That is, we keep the unbalanced number of examples per node, but populate each node with randomly selected examples.
\end{itemize}

The first thing to notice is that CoCoA+ seems to be worse than trivial benchmark --- distributed gradient descent. This behavior can be predicted from theory, as the overall convergence rate directly depends on the best choice of aggregation parameter $\sigma'$. For sparse problems, it is upperbounded by the maximum of the values reported in Figure~\ref{fig:omegaprime}, which is $K$, and it is close to it also in practice. Although it is expected that the algorithm could be modified to depend on average of these quantities (which could be orders of magnitude smaller), akin to coordinate descent algorithms \cite{RichtarikTakacIteration}, it has not been done yet. Note that other communication efficient algorithms fail to converge altogether.

The algorithm we propose, FSVRG, converges to optimal test classification accuracy in just $30$ iterations. Recall that in the setting of \fedopt we introduced in Section~\ref{sec:intro:challenge}, minimization of rounds of communication is the principal goal. However, concluding that the approach is stunningly superior to existing methods would not be completely fair nor correct. The conclusion is that the \emph{FSVRG is the first algorithm to tackle \fedopt}, a problem that existing methods fail to generalize to. It is important to stress that none of the existing methods were designed with these particular challenges in mind, and we formulate the first benchmark.

Since the core reason other methods fail to converge is the non-IID data distribution, we test our method on the same problem, with data randomly reshuffled among the same number of \nodes (FSVRGR; red stars). Since the difference in convergence is subtle, we can conclude that the techniques described in Section~\ref{sec:algorithms:intuition} serve its purpose and make the algorithm robust to challenges present in \fedopt.

This experiment demonstrates that learning from massively decentralized data, clustered on a per-user basis is indeed problem we can tackle in practice. Since the earlier version of this work \cite{konecny2015federated}, additional experimental results were presented in \cite{mcmahan2016federated}. We refer the reader to this paper for experiments in more challenging setting of deep learning, and a further discussion on how such system would be implemented in practice.

\section{Conclusions and Future Challenges}
\label{sec:conclusions}

We have introduced a new setting for distributed optimization, which we call \emph{\fedopt}. This setting is motivated by the outlined vision, in which users do not send the data they generate to companies at all, but rather provide part of their computational power to be used to solve optimization problems. This comes with a unique set of challenges for distributed optimization. In particular, we argue that the massively distributed, non-\iid, unbalanced, and sparse properties of \fedopt problems need to be addressed by the optimization community.

We explain why existing methods are not applicable or effective in this setting. Even the distributed algorithms that can be applied converge very slowly in the presence of large number of \nodes on which the data are stored. We demonstrate that in practice, it is possible to design algorithms that work surprisingly efficiently in the challenging setting of \fedopt, which makes the vision conceptually feasible.

We realize that it is important to scale stochastic gradients on a per-coordinate basis, differently on each \node to improve performance. To the best of our knowledge, this is the first time such per-\node scaling has been used in distributed optimization. Additionally, we use per-coordinate aggregation of updates from each \node, based on distribution of the sparsity patterns in the data.

Even though our results are encouraging, there is a lot of room for future work.  One natural direction is to consider fully asynchronous versions of our algorithms, where the updates are applied as soon as they arrive. Another is developing a better theoretical understanding of our algorithm, as we believe that development of a strong understanding of the convergence properties will drive further research in this area.

Study of the \fedopt problem for non-convex objectives is another important avenue of research. In particular, neural networks are the most important example of a machine learning tool that yields non-convex functions $f_i$, without any convenient general structure. Consequently, there are no useful results describing convergence guarantees of optimization algorithms. Despite the lack of theoretical understanding, neural networks are now state-of-the-art in many application areas, ranging from natural language understanding to visual object detection.  Such applications arise naturally in \fedopt settings, and so extending our work to such problems is an important direction.

The non-\iid data distribution assumed in \fedopt, and mobile applications in particular, suggest that one should consider the problem of training a \emph{personalized} model together with that of learning a global model.  That is, if there is enough data available on a given node, and we assume that data is drawn from the same distribution as future test examples for that node, it may be preferable to make predictions based on a personalized model that is biased toward good performance on the local data, rather than simply using the global model.


\section{Appendix: Distributed Optimization via Quadratic Perturbations}
\label{sec:appendix}

This appendix follows from the discussion motivating DANE algorithm by a general algorithmic perturbation template~\eqref{eq:alg:perturbation} for $\lambda$-strongly convex objectives. We use this to propose a similar but new method, which unlike DANE converges under arbitrary data partitioning $\{\mathcal{P}_k\}_{k=1}^K$, and we highlight its relation to the dual CoCoA algorithm for distributed optimization.

For simplicity and ease of drawing the above connections we assume that $n_k$ is identical for all $k \in \{ 1, 2, \dots, K \}$ throughout the appendix. All the arguments can be simply extended, but would unnecessarily complicate the notation for current purpose.

\subsection{New Method}

We now present a new method (Algorithm~\ref{alg:PRIMAL}), which also belongs to the family of quadratic perturbation methods \eqref{eq:alg:perturbation}. However, the perturbation vectors $a_k^t$ are different from those of DANE. In particular, we set
\[a_k^t  \eqdef \nabla F_k(w^t) - (\eta \nabla F_k(w^t) +  g_k^t),\]
where $\eta>0$ is a parameter, and the vectors $g_k^t$ are maintained by the method. As we show in Lemma~\ref{lem:sum_zero}, Algorithm~\ref{alg:PRIMAL} satisfies \[\sum_{k=1}^K g_k^t = 0\] for all iterations $t$. This implies that  $\tfrac{1}{K}\sum_{k=1}^K a_k^t = (1-\eta) \nabla P(w^t)$. That is, both DANE and the new method use a linear perturbation which, when averaged over the nodes, involves the gradient of the objective function $f$ at the latest iterate $w^t$. Therefore, the methods have one more property in common beyond both being of the form  \eqref{eq:alg:perturbation}. However, as we shall see in the rest of this section, Algorithm~\ref{alg:PRIMAL} allows an insightful dual interpretation. Moreover, while DANE may not converge for arbitrary problems (even when restricted to ridge regression)---and is only known to converge under the assumption that the data stored on each node are in some precise way similar, Algorithm~\ref{alg:PRIMAL} converges for any ridge regression problem and any data partitioning.

Let us denote by $X_k$ the matrix obtained by stacking the data points $x_i$ as column vectors for all $i \in \mathcal{P}_k$. We have the following Lemma.
\begin{lemma}
\label{lem:sum_zero} 
For all $t \geq 0$ we have $\sum_{k=1}^K g_k^t = 0$.
\end{lemma}
\begin{proof}
The statement holds for $t=0$. Indeed,
\[\sum_{k=1}^K g_k^t = \eta\sum_{k=1}^K \left(\frac{K}{n}X_k \alpha_k^0 - \lambda w^0\right) =  0,\]
where the last step follows from the definition of $w^0$. Assume now that the statement hold for $t$. Then
\[\sum_{k=1}^K g_k^{t+1} = \sum_{k=1}^K \left(g_k^t +\eta \lambda(w_k^{t+1} - w^{t+1}) \right)\\
= \eta \lambda \sum_{k=1}^K (w_k^{t+1} - w^{t+1}).\]
The first equation follows from the way $g_k$ is updated in the algorithm. The second equation follows from the inductive assumption, and the last equation follows from the definition of $w^{t+1}$ in the algorithm.
\end{proof}

\begin{algorithm}[t]
\begin{algorithmic}[1]
\State \textbf{Input:} $\sigma \in [1,K]$
\State \textbf{Choose:} $\alpha_k^0\in \R^{|{\cal P}_k|}$ for $k=1,2,\dots,K$ 
\State \textbf{Set:} $\eta = \frac{K}{\sigma}, \; \mu = \lambda (\eta-1)$
\State \textbf{Set:} $w^0 = \frac{1}{\lambda n} \sum_{k=1}^K X_k \alpha_k^0$
\State \textbf{Set:} $g_k^0 = \eta (\frac{K}{n} X_k \alpha_k^0 - \lambda w^0)$ for $k=1,2,\dots,K$
\For {$t = 0,1,2,\dots$}
	 \For {$k = 1$ \textbf{to} $K$} 
	      \State	   $w^{t+1}_k =\arg \min_{w\in \R^d} F_k(w) -  \left( \nabla F_k(w^t)  - (\eta \nabla F_k(w^t)  +   g_k^t)\right)^T w + \frac{\mu}{2}\|w-w^t\|^2 $ 
	\EndFor	
	\State 	 $w^{t+1} = \frac{1}{K} \sum_{k=1}^K w^{t+1}_k$	 	 
	 \For {$k = 1$ \textbf{to} $K$} 
		  \State  $g_k^{t+1} = g_k^t + \lambda \eta (w_k^{t+1} - w^{t+1})$	 
	 \EndFor  	  
\State \textbf{return} $w^t$
\EndFor
\end{algorithmic}
\caption{Primal Method}
\label{alg:PRIMAL}

\end{algorithm}

\subsection{L2-Regularized Linear Predictors}

In the rest of this section we consider the case of L2-regularized {\em linear predictors}.  That is, we focus on problem \eqref{eq:problem}  with $f_i$ of the form
\[f_i(w) = \phi_i(x_i^T w) + \frac{\lambda}{2}\|w\|^2,\]
where $\lambda>0$ is a regularization parameter. 
This leads to L2 regularized empirical risk minimization (ERM) problem
\begin{equation} 
\label{eq:ERM-linear-pred} 
\min_{w\in \R^d} \left\{ P(w)\eqdef \frac{1}{n}\sum_{i=1}^n\phi_i(x_i^T w) + \frac{\lambda}{2}\|w\|^2 \right\}.
\end{equation}
We assume that the loss functions $\phi_i:\R  \to \R$ are convex and  $1/\gamma$-smooth for some $\gamma>0$; these are standard assumptions.   As usual, we allow the loss function $\phi_i$ to depend on the label $y_i$. For instance, we may choose the quadratic loss: $\phi_i(t) = \tfrac{1}{2}(t-y_i)^2$ (for which $\gamma=1$).

Let $X = [x_1,\dots,x_n]\in \R^{d\times n}$. As described in Section~\ref{sec:problem}, we assume that the data $(x_i,y_i)_{i=1}^n$ is distributed among $K$ nodes of a computer cluster as follows: node $k=1,2,\dots,K$ contains pairs $(x_i,y_i)$ for $i\in {\cal P}_k$, where ${\cal P}_1,\dots,{\cal P}_K$ forms a partition of the set $[n]=\{1,2,\dots,n\}$. Letting $X = [X_1,\dots,X_K]$, where $X_k\in \R^{d\times |{\cal P}_k|}$ is a submatrix of $A$ corresponding to columns $i\in {\cal P}_k$, and $y_k\in \R^{|{\cal P}_k|}$ is the subvector of $y$ corresponding to entries $i\in {\cal P}_k$. Hence, node $k$ contains the pair $(X_k,y_k)$. With this notation, we can write the problem in the form \eqref{eq:problem:distributed:simple}, where 
\begin{equation}
\label{eq:iuhd9898g9} 
F_k(w) = \frac{K}{n}\sum_{i\in {\cal P}_k} \phi_i(x_i^T w) + \frac{\lambda}{2}\|w\|^2.
\end{equation}

\subsection{A Dual Method: Dual Block Proximal Gradient Ascent}

The dual of \eqref{eq:ERM-linear-pred} is the problem

\begin{equation}
\label{eq:dual_9878979s8}
\max_{\alpha\in \R^n} \left\{ D(\alpha) \eqdef - \frac{1}{2\lambda n^2}\left\| X \alpha\right\|^2 - \frac{1}{n}\sum_{i=1}^n\phi_i^*(-\alpha_i ) \right\},
\end{equation}
where $\phi_i^*$ is the convex conjugate of $\phi_i$. Since we assume that $\phi_i$ is $1/\gamma$ smooth, it follows that $\phi_i^*$ is $\gamma$ strongly convex. Therefore, $D$ is a strongly concave function.

\paragraph{From dual solution to a primal solution.} It is well known that if $\alpha^*$ is the optimal solution of the dual problem \eqref{eq:ERM-linear-pred}, then
$w^* \eqdef \frac{1}{\lambda n} X \alpha^*$
is the optimal solution of the primal problem. Therefore, for any dual algorithm producing a sequence of iterates $\alpha^t$, we can define a corresponding primal algorithm  via the linear mapping  \begin{equation}\label{eq:primal_from_dual}w^t \eqdef \frac{1}{\lambda n}X \alpha^t.\end{equation} Clearly, if $\alpha^t\to \alpha^*$, then $w^t \to w^*$.
We shall now design a method for maximizing the dual function $D$ and then in Theorem~\ref{thm:equiv} we claim that for quadratic loss functions, Algorithm~\ref{alg:PRIMAL} arises as an image, defined via   \eqref{eq:primal_from_dual},  of dual iterations of this  dual ascent method.

\paragraph{Design of the dual gradient ascent method.} Let $\xi(\alpha) \eqdef \tfrac{1}{2}\|X\alpha\|^2$. Since $\xi$ is a convex quadratic, we have
\[\xi(\alpha + h) = \xi(\alpha) + \langle \nabla \xi(\alpha), h\rangle + \frac{1}{2}h^T \nabla^2 \xi(\alpha)h,\\
\leq  \xi(\alpha) + \langle \nabla \xi(\alpha), h\rangle + \frac{\sigma}{2}\|h\|_B^2,
\]
where $\nabla \xi (\alpha)  = X^T X \alpha $ and $\nabla^2 \xi (\alpha) = X^T X $. Further, we define the block-diagonal matrix $B\eqdef Diag(X_1^T X_1, \dots,X_{K}^T X_K)$, and a norm associate with this matrix: \[\|h\|_B^2 \eqdef \sum_{k=1}^K \|X_k h_k\|^2.\] By $\sigma$ we refer to a large enough constant for which  $X^TX\preceq \sigma B$. In order to avoid unnecessary  technicalities, we shall assume that the matrices $X_k^T X_k$ are positive definite, which implies that $\|\cdot\|_B$ is a norm. It can be shown that $1\leq \sigma \leq K$. Clearly, $\xi$ is $\sigma$-smooth with respect to the norm $\|\cdot \|_B$. In view of the above, for all $h\in \R^n$ we can estimate $D$ from below as follows:
\begin{eqnarray*}
D(\alpha^t+h) &\geq &  -\frac{1}{\lambda n^2}\left(\xi (\alpha^t) + \langle \nabla \xi (\alpha^t), h \rangle + \frac{\sigma}{2}\sum_{k=1}^K \|X_k h_k\|^2\right) - \frac{1}{n}\sum_{i=1}^n \phi_i^*(-\alpha_i^t-h_i)\\
&=& -\frac{1}{\lambda n^2} \xi(\alpha^t) - \sum_{k=1}^K \left[ \frac{1}{\lambda n^2} \langle \nabla_k \xi (\alpha^t),h_k\rangle +  \frac{\sigma}{2\lambda n^2}\|X_k h_k\|^2 + \frac{1}{n}\sum_{i\in {\cal P}_k} \phi_i^*(-\alpha_i^t - h_i) \right] ,
\end{eqnarray*}
 where  $\nabla_k \xi (\alpha^t)$ corresponds to the subvector of $\nabla \xi(\alpha^t)$ formed by entries $i\in {\cal P}_k$. 

We  now let $h^t = (h^t_1,\dots,h^t_K)$  be the maximizer of this lower bound. Since the lower bound is separable in the blocks $\{h^t_k\}_k$, we can simply set
\begin{equation}\label{eq:DUAL-eq} h^t_k := \arg\min_{u\in \R^{|{\cal P}_k|}} \left\{D_k^t(u) \eqdef \frac{1}{\lambda n^2} \langle \nabla_k \xi (\alpha^{t}),u\rangle +  \frac{\sigma}{2\lambda n^2}\|X_k u\|^2 + \frac{1}{n}\sum_{i\in {\cal P}_k} \phi_i^*(-\alpha_i^t - u_i) \right\}.\end{equation}
Having computed $h_k^t$ for all $k$, we can set $\alpha_k^{t+1} = \alpha_k^t + h_k^t$ for all $k$, or equivalently, $\alpha^{t+1} = \alpha^t + h^t$.  This is formalized as Algorithm~\ref{alg:DUAL}.  Algorithm~\ref{alg:DUAL}  is a proximal gradient ascent method applied to the dual problem, with smoothness being measured using the block norm $\|h\|_B$. It is known that gradient ascent converges at a linear rate for smooth and strongly convex (for minimization problems) objectives.

\begin{algorithm}[t]
\begin{algorithmic}[1]
\State \textbf{Input:} $\sigma \in [1,K]$
\State \textbf{Choose:} $\alpha_k^0\in \R^{|{\cal P}_k|}$ for $k=1,2,\dots,K$
\For {$t = 0,1,2,\dots$}
	 \For {$k = 1$ \textbf{to} $K$} 
	 		\State	   $h^{t+1}_k =\arg\min_{u\in \R^{|{\cal P}_k|}}   D_k^t(u)$
	 		\Comment See \eqref{eq:DUAL-eq}
	 \EndFor
	 \State	 $\alpha^{t+1} = \alpha^t + h^t$
\EndFor	 	
\State \textbf{return} $w^t$
\end{algorithmic}
\caption{Dual Method}
\label{alg:DUAL}

\end{algorithm}

One of the main insights of this section is the following equivalence result.

\begin{theorem}[Equivalence of Algorithms \ref{alg:PRIMAL} and \ref{alg:DUAL} for Quadratic Loss] \label{thm:equiv} Consider the ridge regression problem. That is, set  $\phi_i(t) = \tfrac{1}{2}(t-y_i)^2$ for all $i$. Assume $\alpha_1^0, \dots, \alpha_K^0$ is chosen in the same way in Algorithms~\ref{alg:PRIMAL} and \ref{alg:DUAL}. Then the dual iterates $\alpha^t$ and the primal iterates $w^t$ produced by the two algorithms are related via \eqref{eq:primal_from_dual} for all $t\geq 0$. 
\end{theorem}

Since the dual method converges linearly, in view of the above theorem, so does the primal method. Here we only remark that the popular algorithm CoCoA+ \cite{ma2015distributed} arises if Step 5 in Algorithm~\ref{alg:DUAL} is done inexactly. Hence, we show that duality provides a deep relationship between the CoCoA+ and DANE algorithms, which were previously considered completely different.

\subsection{Proof of Theorem~\ref{thm:equiv} }

In this part we prove the theorem.

\paragraph{Primal and Dual Problems.} Since $\phi_i(t) = \tfrac{1}{2}(t-y_i)^2$, the primal problem \eqref{eq:ERM-linear-pred} is a ridge regression problem of the  form 
\begin{equation}
\label{eq:P}
\min_{w\in \R^d} P(w) = \frac{1}{2n}\|X^T w-y\|^2 + \frac{\lambda}{2}\|w\|^2,
\end{equation}
where $X\in \R^{d\times n}$ and $y\in \R^n$. In view of \eqref{eq:dual_9878979s8}, the dual of \eqref{eq:P} is 
\begin{equation}
\label{eq:D} 
\min_{\alpha\in \R^n} D(\alpha) =  \frac{1}{2\lambda n^2}\|X\alpha\|^2 + \frac{1}{2n}\|\alpha\|^2 -\frac{1}{n}y^T \alpha.
\end{equation}

\paragraph{Primal Problem: Distributed Setup.} The primal objective function is of the form \eqref{eq:problem:distributed:simple}, where  in view of \eqref{eq:iuhd9898g9}, we have
$F_k(w) = \frac{K}{2n}\|X_k^T w - y_k\|^2 + \frac{\lambda}{2}\|w\|^2$. Therefore, 
\begin{equation} \label{eq:DF_k} \nabla F_k(w) = \frac{K}{n}X_k(X_k^T w - y_k) + \lambda w\end{equation}
and
$\nabla P(w) = \frac{1}{K}\sum_k \nabla F_k(w) = \frac{1}{K}\sum_k  \left(\frac{K}{n}X_k(X_k^T w - y_k) + \lambda w \right).$

\paragraph{Dual Method.}  Since $D$ is a quadratic, we have
\begin{eqnarray*}D(\alpha^t + h) &=& D(\alpha^t) + \nabla D(\alpha^t)^T h + \frac{1}{2}h^T \nabla^2 D(\alpha^t)h,\end{eqnarray*}
with \[\nabla D(\alpha^t)  = \frac{1}{\lambda n^2}X^T X \alpha^t + \frac{1}{n}(\alpha^t - y), \qquad \nabla^2 D(\alpha^t) = \frac{1}{\lambda n^2}X^T X + \frac{1}{n}I.\] 
We know that  $X^T X\preceq \sigma Diag(X_1^T X_1, \dots,X_{K}^T X_K)$. With this approximation, for all $h\in \R^n$ we can estimate $D$ from above by a node-separable quadratic function as follows:
\begin{eqnarray*}
D(\alpha^t+h) &\leq&  D(\alpha^t) + \left(\frac{1}{\lambda n^2}X^T X \alpha^t + \frac{1}{n}(\alpha^t - y)\right)^T h + \frac{1}{2n}\|h\|^2 + \frac{\sigma}{2\lambda n^2}\sum_{k=1}^K \|X_k h_k\|^2\\
&=& D(\alpha^t) + \frac{1}{n} \left[  \frac{1}{\lambda n}(X\alpha^t)^T X h + (\alpha^t - y)^T h + \frac{1}{2}\|h\|^2 + \frac{\sigma}{2\lambda n}\sum_{k=1}^K \|X_k h_k\|^2  \right]\\
&=&D(\alpha^t) + \frac{1}{n} \sum_{k=1}^K \left( (w^t)^T X_k h_k + (\alpha^t_k - y_k)^T h_k+ \frac{1}{2}\|h_k\|^2 + \frac{\sigma}{2\lambda n}\|X_k h_k\|^2  \right).
\end{eqnarray*}

Next, we shall define 
\begin{equation}\label{eq:xD}h_k^t \eqdef \arg \min_{h_k \in \R^{|{\cal P }_k|}}  \frac{\sigma}{2\lambda n}\|X_k h_k\|^2     + \frac{1}{2}\|h_k\|^2 - ( y_k - X_k^T w^t - \alpha^t_k )^T h_k \end{equation}
for $k=1,2,\dots,K$ and then set \begin{equation}
\label{eq:xxx}\alpha^{t+1} = \alpha^t + h^t.\end{equation}

\paragraph{Primal Version of the Dual Method.} Note that \eqref{eq:xD} has the same form as \eqref{eq:D}, with $X$ replaced by $X_k$, $\lambda$ replaced by $\lambda/\sigma$ and $y$ replaced by $c_k:=y_k - X_k^T w^t - \alpha^t_k$. Hence, we know that 
\begin{equation}\label{eq:s}s_k^t \eqdef \frac{1}{(\lambda/\sigma ) n}X_k h_k^t\end{equation}
is the optimal solution of the primal problem of \eqref{eq:xP}:
\begin{equation}\label{eq:xP}s_k^t =\arg \min_{s\in \R^d}  \frac{1}{2n}\|X_k^T s-c_k\|^2 + \frac{\lambda/\sigma}{2}\|s\|^2.\end{equation}

Hence, the primal version of  method \eqref{eq:xxx} is given by
\begin{eqnarray*}w^{t+1}&\overset{\eqref{eq:primal_from_dual}}{=}& \frac{1}{\lambda n}X \alpha^{t+1} \overset{\eqref{eq:xxx}}{=}\frac{1}{\lambda n} X (\alpha^t + h^t) \overset{\eqref{eq:primal_from_dual}}{=} w^t + \frac{1}{\lambda n} \sum_{k=1}^K X_k h^t_k\\
& = &\frac{1}{K} \sum_{k=1}^K \left( w^t + \frac{K}{\sigma}\frac{\sigma}{\lambda n}  X_k h^t_k \right) \overset{\eqref{eq:s}}{=} \frac{1}{K} \sum_{k=1}^K \left( w^t + \frac{K}{\sigma}s_k^t \right).
\end{eqnarray*}

With the change of variables $w := w^t + \frac{K}{\sigma}s$ (i.e., $s = \frac{\sigma}{K}(w-w^t)$), from \eqref{eq:xP} we know that $w_k^{t+1}:=w^t + \frac{K}{\sigma}s_k^t$ solves
\begin{equation}\label{eq:xPnew}w_k^{t+1} =\arg \min_{w\in \R^d} \left\{L_k(w) \eqdef \frac{1}{2n}\left\|X_k^T \frac{\sigma}{K}(w-w^t)-c_k \right\|^2 + \frac{\lambda/\sigma}{2} \left\|\frac{\sigma}{K}(w-w^t) \right\|^2 \right\}\end{equation}
and $w^{t+1} = \frac{1}{K}\sum_{k=1}^K w_k^{t+1}$.

Let us now rewrite the function in \eqref{eq:xPnew} so as to connect it to Algorithm~\ref{alg:PRIMAL}:
\begin{eqnarray*} L_{k}(w) &=& \frac{1}{2n}\left\|X_k^T \frac{\sigma}{K}(w-w^t)-c_k \right\|^2 + \frac{\lambda/\sigma}{2} \left\|\frac{\sigma}{K}(w-w^t) \right\|^2\\
&=& \frac{1}{2n}\frac{\sigma^2}{K^2}\left\|  (X_k^T w-y_k)  - \underbrace{\left(X_k^T w^t -y_k + \frac{K}{\sigma} c_k\right)}_{d_k} \right\|^2 \\
&&\qquad + \frac{\lambda \sigma^2}{2K^3}\|w\|^2 - \frac{\lambda \sigma^2}{2K^3}\|w\|^2 + \frac{\lambda/\sigma}{2} \left\|\frac{\sigma}{K}(w-w^t) \right\|^2\\
&=& \frac{1}{2n}\frac{\sigma^2}{K^2} \left( \left\| X_k^T w - y_k\right\|^2 +  \|d_k\|^2 - 2 (X_k^T w - y_k)^T d_k\right) \\
&& \qquad + \frac{\lambda \sigma^2}{2K^3}\|w\|^2 - \frac{\lambda \sigma^2}{2K^3}\|w\|^2 + \frac{\lambda\sigma}{2 K^2} \left\| w-w^t \right\|^2\\
&=&\frac{\sigma^2}{K^3} \left(\frac{K}{2n}  \left\| X_k^T w - y_k\right\|^2 +  \frac{K}{2n}\|d_k\|^2 - \frac{K}{n} (X_k^T w - y_k)^T d_k\right) + \frac{\lambda \sigma^2}{2K^3}\|w\|^2 \\
&& \qquad - \frac{\lambda \sigma^2}{2K^3}\|w\|^2 + \frac{\lambda\sigma}{2 K^2} \left\| w-w^t \right\|^2\\ 
&=&\frac{\sigma^2}{K^3} \underbrace{\left(\frac{K}{2n}  \left\| X_k^T w - y_k\right\|^2 + \frac{\lambda }{2}\|w\|^2 \right)}_{F_k(w)} + \frac{\sigma^2}{K^3} \left(  \frac{K}{2n}\|d_k\|^2 - \frac{K}{n} (X_k^T w - y_k)^T d_k\right)  \\
&& \qquad - \frac{\lambda \sigma^2}{2K^3}\|w\|^2 + \frac{\lambda\sigma}{2 K^2} \left\| w-w^t \right\|^2\\
&=&\frac{\sigma^2}{K^3} F_k(w) - \frac{\sigma^2}{K^2 n}  (X_k^T w - y_k)^T d_k + \frac{\sigma^2}{2n K^2} \|d_k\|^2  - \frac{\lambda \sigma^2}{2K^3}\|w\|^2 + \frac{\lambda\sigma}{2 K^2} \left\| w-w^t \right\|^2\\
&=&\frac{\sigma^2}{K^3} F_k(w) - \frac{\sigma^2}{K^2 n}  (X_k d_k)^T w     - \frac{\lambda \sigma^2}{2K^3}\|w\|^2 + \frac{\lambda\sigma}{2 K^2} \left\| w-w^t \right\|^2 \\
&&\qquad+ \underbrace{\left( \frac{\sigma^2}{2n K^2} \|d_k\|^2+  \frac{\sigma^2}{K^2 n}  y_k^T d_k\right)}_{\beta_1}.
\end{eqnarray*}

Next, since $\|w\|^2 = \|w-w^t\|^2 - \|w^t\|^2 + 2(w^t)^T w$, we can further write 

\begin{align*} 
L_{k}(w) &= \frac{\sigma^2}{K^3} F_k(w) - \frac{\sigma^2}{K^2 n}  (X_k d_k)^T w     - \frac{\lambda \sigma^2}{2K^3} ( \|w-w^t\|^2 - \|w^t\|^2 + 2(w^t)^T w) \\
& \qquad + \frac{\lambda\sigma}{2 K^2} \left\| w-w^t \right\|^2 + \beta_1\\
&=\frac{\sigma^2}{K^3} F_k(w) - \frac{\sigma^2}{K^2 n}  (X_k d_k)^T w  - \frac{\lambda \sigma^2}{K^3}   (w^t)^T w + \left(\frac{\lambda\sigma}{2 K^2} - \frac{\lambda \sigma^2}{2K^3} \right) \left\| w-w^t \right\|^2 \\
& \qquad + \underbrace{ \frac{\lambda \sigma^2}{2K^3} \|w^t\|^2 + \beta_1}_{\beta_2}\\
&= \frac{\sigma^2}{K^3} \left(F_k(w) - \left(\frac{K}{n} X_k d_k + \lambda w^t \right)^T w  + \frac{\lambda}{2}\left(\frac{K}{\sigma}-1\right)\|w-w^t\|^2 \right) + \beta_2\\
&= \frac{\sigma^2}{K^3} \left(F_k(w) - \left(\nabla F_k(w^t) - \frac{K^2}{\sigma n} X_k ( X_k^T w^t -y_k + \alpha_k^t)\right)^T w  + \frac{\mu}{2}\|w-w^t\|^2 \right) + \beta_2\\
&= \frac{\sigma^2}{K^3} \left(F_k(w) - \left(\nabla F_k(w^t) - \frac{ K}{\sigma}\underbrace{\frac{K}{ n} X_k ( X_k^T w^t -y_k + \alpha_k^t)}_{z^t_k}\right)^T w  + \frac{\mu}{2}\|w-w^t\|^2 \right) + \beta_2\\
&= \frac{\sigma^2}{K^3} \left(F_k(w) - \left(\nabla F_k(w^t) - (\eta \nabla F_k(w^t) + g_k^t) \right)^T w  + \frac{\mu}{2}\|w-w^t\|^2 \right) + \beta_2,
\end{align*}
where the last step follows from the claim that  $\eta z_k^t = \eta \nabla F_k(w^t) + g_k^t$. We now prove the claim. First, we have
\begin{eqnarray*}
\eta z_k^t &=& \eta \frac{K}{ n} X_k ( X_k^T w^t -y_k + \alpha_k^t)\\
&=& \eta \frac{K}{ n} X_k ( X_k^T w^t -y_k) + \eta \frac{K}{ n} X_k  \alpha_k^t\\
&=& \eta \left(\frac{K}{ n} X_k ( X_k^T w^t -y_k) + \lambda w^t \right) + \eta \left(\frac{K}{ n} X_k  \alpha_k^t - \lambda w^t \right) \\
&\overset{\eqref{eq:DF_k}}{=} & \eta \nabla F_k(w^t) +  \eta \left(\frac{K}{ n} X_k  \alpha_k^t - \lambda w^t \right).
\end{eqnarray*}
Due to the definition of $g_k^0$ in Step 5 of Algorithm~\ref{alg:PRIMAL} as $g_k^0 = \eta (\frac{K}{n} X_k \alpha_k^0 - \lambda w^0)$, we observe that the claim holds for $t=0$. If we show that 
\[g_k^{t} =  \eta \left(\frac{K}{ n} X_k  \alpha_k^{t} - \lambda w^{t} \right)\]
for all $t\geq 0$, then we are done.  This can be shown by induction.  This finishes the proof of  Theorem~\ref{thm:equiv}.

\chapter{Randomized Distributed Mean Estimation: Accuracy vs Communication}
\label{ch:mean}

\section{Introduction} 
\label{sec:introduction}

In this chapter, we address the problem of approximately computing the arithmetic mean of $n$ vectors,  $X_1, \dots, X_n \in \R^d$, stored in a distributed fashion across $n$ compute nodes, subject to a constraint on the communication cost. 

In particular, we consider a star network topology  with a single server at the center and $n$ nodes connected to it. All nodes send an encoded (possibly via a lossy randomized transformation) version of their vector  to the server, after which the server  performs a decoding operation to estimate the true mean \[X \eqdef \frac{1}{n}\sum_{i=1}^n X_i.\] The purpose of the encoding operation is to compress the vector so as to save on communication cost, which is typically the bottleneck in practical applications. 

To better illustrate the setup, consider the naive approach in which all nodes send the vectors without performing any encoding operation, followed by the application of a simple averaging decoder by the server. This results in zero estimation error at the expense of maximum communication cost of $ndr$ bits, where $r$ is the number  of bits needed to communicate a single floating point entry/coordinate  of $X_i$.

\subsection{Background and  Contributions}

The distributed mean estimation problem was recently  studied in a statistical framework where it is assumed that  the vectors $X_i$ are independent and identically distributed samples from some specific underlying distribution. In such a setup, the goal  is to estimate the true mean of the underlying distribution \cite{zhang2012communication, zhang2013information, comm_distr2014, braverman2015communication}. These works formulate lower and upper bounds on the communication cost  needed to achieve  the minimax optimal estimation error.

In contrast, we do not make any statistical assumptions on the source of the vectors, and study the trade-off between expected communication costs and mean square error of the estimate. Arguably, this setup is a more robust and accurate model of  the distributed mean estimation problems arising as subproblems in  applications such as reduce-all operations within algorithms for distributed and federated optimization \cite{richtarik2013distributed, Ma:2015ti, ma2015distributed, reddi2016aide, konecny2016federated}. In these applications, the averaging operations need to be done repeatedly throughout the iterations of a master learning/optimization algorithm, and  the vectors $\{X_i\}$  correspond to updates to a global model/variable. In these applications, the vectors evolve throughout the iterative process in a complicated  pattern, typically approaching zero as the master algorithm converges to optimality. Hence, their statistical properties change, which renders  fixed statistical assumptions not satisfied in practice.

For instance, when training a deep neural network model in a distributed environment, the vector $X_i$ corresponds to a stochastic gradient based on a minibatch of data stored on node $i$. In this setup we do not  have any useful prior statistical knowledge about the high-dimensional vectors to be aggregated. It has recently been observed that when communication cost is high, which is typically the case for commodity clusters, and even more so in a federated optimization framework, it is can be very useful to sacrifice on estimation accuracy in favor of reduced communication \cite{mcmahan2016federated, Federated_learning2016}.

In this chapter we propose a {\em parametric family of randomized methods for estimating the mean $X$}, with parameters being a set of {\em probabilities} $p_{ij}$ for $i=1,\dots,n$ and $j=1,2,\dots,d$ and {\em node centers} $\mu_i \in \R^d$ for $i=1,2,\dots,n$. The exact meaning of these parameters is explained in Section~\ref{sec:encode}.   By varying the probabilities, at one extreme, we recover the exact  method described, enjoying zero estimation error at the expense of full communication cost. At the opposite extreme are methods with arbitrarily small expected communication cost, which is achieved at the expense of suffering an exploding estimation error. Practical methods appear somewhere on the continuum between these two extremes, depending on the specific requirements of the application at hand.  Suresh et al.\ \cite{Distributed_mean} propose a method combining a pre-processing step via a random structured rotation, followed by randomized binary quantization. Their quantization protocol arises as a suboptimal special case of our parametric family of methods.

To illustrate our results, consider the  special case in which we choose to communicate a single bit per element of $X_i$ only. We then obtain an $\mathcal{O}\left( \frac{r}{n}R \right)$ bound on the mean square error, where $r$ is number of bits used to represent a floating point value, and $R = \frac1n \sum_{i=1}^n \| X_i - \mu_i 1 \|^2$ with $\mu_i \in \R$ being the average of elements of $X_i$, and $1$ the all-ones vector in $\R^d$ (see Example 7 in Section~\ref{sec:examples}). Note that this bound improves upon the performance of the method of \cite{Distributed_mean} in two aspects. First, the bound is independent of $d$, improving from logarithmic dependence. Further, due to a preprocessing rotation step, their method requires $\mathcal{O}(d \log d)$ time to be implemented on each node, while our method is linear in  $d$. This and other special cases are summarized in Table~\ref{tbl:main1} in Section~\ref{sec:examples}. 

While the above already improves upon the state of the art, the improved results are in fact obtained for a suboptimal choice of the parameters of our method (constant probabilities $p_{ij}$, and node centers fixed to the mean $\mu_i$). One can decrease the MSE further by optimizing over the probabilities and/or node centers (see Section~\ref{sec:opt}). However, apart from a very low communication cost regime in which we have a closed form expression for the optimal probabilities, the problem needs to be solved numerically, and hence we do not have expressions for how much improvement is possible. We illustrate the effect of fixed and optimal probabilities on the trade-off between communication cost and MSE experimentally on a few selected datasets in Section~\ref{sec:opt} (see Figure~\ref{fig:uniform_vs_optimal}).

\subsection{Outline}

In Section~\ref{sec:3protocols} we formalize the concepts of encoding and decoding protocols. 
In Section~\ref{sec:encode} we describe a parametric family of randomized (and unbiased) encoding protocols and give a simple formula for the mean squared error. Subsequently, in Section~\ref{sec:comm} we formalize the notion of communication cost, and describe several communication protocols, which are optimal under different circumstances. We give simple instantiations of our protocol in Section~\ref{sec:examples}, illustrating the trade-off between communication costs and accuracy. In Section~\ref{sec:opt} we address the question of the optimal choice of parameters of our protocol. Brief remarks on possible extensions that are summarized in Section~\ref{sec:further}. Finally, we present experimental application in Federated Learning in Section~\ref{sec:meanFL}.

\section{Three Protocols}
\label{sec:3protocols}

In this work we consider (randomized) {\em encoding protocols} $\alpha$, {\em communication protocols} $\beta$ and {\em decoding protocols} $\gamma$ using which the averaging is performed inexactly as follows. Node $i$ computes a (possibly stochastic) estimate of $X_i$ using the encoding protocol, which we denote $Y_i = \alpha(X_i) \in \R^d$, and sends it to the server using communication protocol $\beta$. By $\beta(Y_i)$ we denote the number of bits that need to be transferred under $\beta$. The server then estimates $X$ using the decoding protocol $\gamma$ of the estimates:
$$ Y \eqdef \gamma(Y_1, \dots, Y_n). $$

The objective of this work is to study the trade-off between the (expected) number of bits that need to be communicated, and the accuracy of $Y$ as an estimate of $X$. 

In this work we focus on encoders which are unbiased, in the following sense.

\begin{definition}[Unbiased and Independent Encoder] 
We say that encoder $\alpha$ is unbiased if  $\EE{\alpha}{\alpha(X_i)} = X_i$ for all $i=1,2,\dots,n$. We say that it is independent, if $\alpha(X_i)$ is independent from $\alpha(X_j)$ for all $i\neq j$.
\end{definition}

\begin{example}[Identity Encoder]
A trivial  example of an encoding protocol is the identity function: $\alpha(X_i) = X_i$.  It is both unbiased and independent. This encoder does not lead to any savings in communication that would be otherwise infeasible though.
\end{example}

We now formalize the notion of accuracy of estimating $X$ via $Y$. Since $Y$ can be random, the notion of accuracy will naturally be  probabilistic.

\begin{definition}[Estimation Error / Mean Squared Error] 
The {\em mean squared error} of protocol $(\alpha,\gamma)$ is the quantity 
\begin{eqnarray*} 
MSE_{\alpha,\gamma}(X_1, \dots, X_n) &=& \EE{\alpha,\gamma}{\| Y - X \|^2}\\
& =& \EE{\alpha, \gamma}{\left\| \gamma(\alpha(X_1),\dots,\alpha(X_n)) - X\right\|^2}.\end{eqnarray*}
\end{definition}

To illustrate the above concept, we now give a few examples:

\begin{example}[Averaging Decoder]
\label{ex:avg_decoder}
If $\gamma$ is the averaging function, i.e., $\gamma(Y_1, \dots, Y_n) = \frac1n \sum_{i=1}^n Y_i,$ then 
\[MSE_{\alpha,\gamma}(X_1, \dots, X_n) = \frac{1}{n^2}\EE{\alpha}{\left\| \sum_{i=1}^n \alpha(X_i) - X_i \right\|^2}.\]
\end{example}

The next example generalizes the identity encoder and averaging decoder.

\begin{example}[Linear Encoder and Inverse Linear Decoder]
\label{ex:linear_encoder}
Let $A:\R^d\to \R^d$ be linear and invertible.   Then we can set $Y_i = \alpha(X_i) \eqdef A X_i$ and $\gamma(Y_1,\dots,Y_n) \eqdef A^{-1} \left(\frac{1}{n}\sum_{i=1}^n Y_i\right)$. If $A$ is random, then $\alpha$ and $\gamma$ are random (e.g., a structured random rotation, see \cite{yu2016orthogonal}). Note that \[\gamma(Y_1,\dots,Y_n) =  \frac{1}{n}\sum_{i=1}^n A^{-1} Y_i = \frac{1}{n}\sum_{i=1}^n X_i = X,\]
and hence the MSE of $(\alpha,\gamma)$ is zero.
\end{example}

We shall now prove a simple result for unbiased and independent encoders used in subsequent sections.
\begin{lemma}[Unbiased and Independent Encoder + Averaging Decoder] 
\label{lem:general_MSE} 
If the encoder $\alpha$ is unbiased and independent, and $\gamma$ is the averaging decoder, then 
\[MSE_{\alpha,\gamma}(X_1, \dots, X_n) = \frac{1}{n^2}\sum_{i=1}^n \EE{\alpha}{\|Y_i-X_i\|^2} = \frac{1}{n^2}\sum_{i=1}^n \VV{\alpha}{\alpha(X_i)} .\]
\end{lemma}

\begin{proof} 
Note that $\EE{\alpha}{Y_i} = X_i$ for all $i$. We have
\begin{eqnarray}
MSE_{\alpha}(X_1,\dots,X_n) &=& \EE{\alpha}{\|Y-X\|^2} \notag \\
&\overset{(*)}{=}& \frac{1}{n^2} \EE{\alpha}{\left\|\sum_{i=1}^n Y_i - X_i \right\|^2} \notag \\
&\overset{(**)}{=}& \frac{1}{n^2} \sum_{i=1}^n \EE{\alpha}{\left\|Y_i - \EE{\alpha}{Y_i} \right\|^2} \notag \\
&=& \frac{1}{n^2} \sum_{i=1}^n \VV{\alpha}{\alpha(X_i)} \notag,
\end{eqnarray}
where (*) follows from unbiasedness and (**) from independence.
\end{proof}

One may wish to define the encoder as a combination of two or more separate encoders: $\alpha(X_i) = \alpha_2(\alpha_1(X_i))$. See \cite{Distributed_mean} for an example where $\alpha_1$ is a random rotation and $\alpha_2$ is binary quantization.

\section{A Family of Randomized Encoding Protocols}
\label{sec:encode}

Let $X_1,\dots, X_n\in \R^d$ be given. We shall write $X_{i} = (X_{i}(1),\dots, X_{i}(d))$ to denote the entries  of vector $X_i$. In addition, with each $i$ we also associate a parameter $\mu_i\in \R$. We refer to $\mu_i$ as the center of data at node $i$, or simply as {\em node center}. For now, we assume these parameters are fixed and we shall later comment on how to choose them optimally.

We shall define \emph{support} of $\alpha$ on node $i$ to be the set $S_i \eqdef \{ j \;:\; Y_i(j) \neq \mu_i \}$. We now define two parametric families of randomized encoding protocols. The first results in $S_i$ of random size, the second has  $S_i$ of a fixed size.

\subsection{Encoding Protocol with Variable-size Support}

With each pair $(i,j)$ we associate a parameter $0< p_{ij}\leq 1$,  representing a probability. The collection of parameters $\{p_{ij}, \mu_i\}$ defines an encoding protocol $\alpha$ as follows:
\begin{equation}
\label{eq:randomized_protocol}
Y_{i}(j) = 
\begin{cases}
\frac{X_{i}(j)}{p_{ij}} - \frac{1-p_{ij}}{p_{ij}} \mu_i & \quad \text{with probability} \quad p_{ij}, \\
\mu_i & \quad \text{with probability} \quad 1-p_{ij}.
\end{cases}
\end{equation}

\begin{remark}
Enforcing the probabilities to be positive, as opposed to nonnegative,  leads to vastly simplified notation in what follows. However, it is more natural to allow $p_{ij}$ to be  zero, in which case we have $Y_i(j)=\mu_i$ with probability 1. This raises issues such as potential lack of unbiasedness, which can be resolved, but only at the expense of a larger-than-reasonable notational overload.
\end{remark}


In the rest of this section, let $\gamma$ be the averaging decoder (Example~\ref{ex:avg_decoder}). Since $\gamma$ is fixed and deterministic, we shall for simplicity write $\EE{\alpha}{\cdot}$ instead of $\EE{\alpha,\gamma}{\cdot}$. Similarly, we shall write $MSE_\alpha(\cdot)$ instead of $MSE_{\alpha, \gamma}(\cdot)$.

We now prove two lemmas describing properties of the encoding protocol $\alpha$. Lemma~\ref{lem:unbiasedness} states that the protocol yields an unbiased estimate of the average $X$ and Lemma~\ref{lem:MSE} provides the expected mean square error of the estimate.

\begin{lemma}[Unbiasedness] 
\label{lem:unbiasedness}
The encoder $\alpha$ defined in \eqref{eq:randomized_protocol} is unbiased. That is,  $\EE{\alpha}{\alpha(X_i)} = X_i$ for all $i$. As a result, $Y$ is an unbiased estimate of the true average: $\EE{\alpha}{Y} = X$.
\end{lemma}

\begin{proof}
Due to linearity of expectation, it is enough to show that $\EE{\alpha}{Y(j)}=X(j)$ for all $j$. Since $Y(j) = \frac{1}{n}\sum_{i=1}^n Y_{i}(j)$ and $X(j) = \frac{1}{n}\sum_{i=1}^n X_{i}(j)$, it suffices to show that $\EE{\alpha}{Y_i(j)}=X_i(j)$:
$$ \EE{\alpha}{Y_i(j)} = p_{ij} \left( \frac{X_{i}(j)}{p_{ij}} - \frac{1-p_{ij}}{p_{ij}} \mu_i(j) \right) + (1-p_{ij}) \mu_i(j) = X_{i}(j), $$
and the claim is proved.
\end{proof}

\begin{lemma}[Mean Squared Error]
\label{lem:MSE}
Let $\alpha = \alpha(p_{ij},\mu_i)$ be the encoder defined in \eqref{eq:randomized_protocol}. Then
\begin{equation}
\label{eq:MSE_general}
MSE_{\alpha}(X_1, \dots, X_n) = \frac{1}{n^2} \sum_{i,j} \left( \frac{1}{p_{ij}} - 1 \right) \left( X_i(j) - \mu_i \right)^2.
\end{equation}
\end{lemma}

\begin{proof} 
Using Lemma~\ref{lem:general_MSE}, we have
\begin{eqnarray}
MSE_{\alpha}(X_1,\dots,X_n) &=&  \frac{1}{n^2} \sum_{i=1}^n \EE{\alpha}{\left\|Y_i - X_i \right\|^2} \notag \\
&=&\frac{1}{n^2} \sum_{i=1}^n \EE{\alpha}{\sum_{j=1}^d (Y_i(j) - X_i(j) )^2} \notag \\
&=&\frac{1}{n^2} \sum_{i=1}^n \sum_{j=1}^d \EE{\alpha}{ (Y_i(j) - X_i(j) )^2}.
\label{eq:6867sgs7}
\end{eqnarray}
For any $i,j$ we further have
\begin{align*}
\EE{\alpha}{ (Y_i(j) - X_i(j) )^2} &= 
p_{ij} \left( \frac{X_{i}(j)}{p_{ij}} - \frac{1-p_{ij}}{p_{ij}} \mu_i - X_i(j) \right)^2 + (1 - p_{ij}) \left( \mu_i - X_i(j) \right)^2 \\
&= \frac{(1 - p_{ij})^2}{p_{ij}} \left( X_i(j) - \mu_i \right)^2 + (1 - p_{ij}) \left( \mu_i - X_i(j) \right)^2 \\
&= \left( \frac{1-p_{ij}}{p_{ij}} \right) \left( X_i(j) - \mu_i \right)^2.
\end{align*}
It suffices to substitute the above into \eqref{eq:6867sgs7}.
\end{proof}

\subsection{Encoding Protocol with Fixed-size Support}

Here we propose an alternative encoding protocol, one with deterministic support size. As we shall see later, this results in deterministic communication cost.

Let $\sigma_k(d)$ denote the set of all subsets of $\{1, 2, \dots, d\}$ containing $k$ elements. The protocol $\alpha$ with a single integer parameter $k$ is then working as follows: First, each node $i$ samples $\mathcal{D}_i \in \sigma_k(d)$ uniformly at random, and then sets
\begin{equation}
\label{eq:randomized_protocol_2}
Y_{i}(j) = 
\begin{cases}
\frac{d X_{i}(j)}{k} - \frac{d-k}{k} \mu_i & \quad \text{if} \quad j \in \mathcal{D}_i, \\
\mu_i & \quad \text{otherwise}.
\end{cases}
\end{equation}

Note that due to the design, the size of the support of $Y_i$ is always $k$, i.e., $|S_i| = k$. Naturally, we can expect this protocol to perform practically the same as the protocol \eqref{eq:randomized_protocol} with $p_{ij} = k/d$, for all $i, j$. Lemma~\ref{lem:MSE_2} indeed suggests this is the case. While this protocol admits  a more efficient communication protocol (as we shall see in Section~),  protocol \eqref{eq:randomized_protocol} enjoys a larger parameters space, ultimately leading to better MSE.  We comment on this tradeoff in subsequent sections.

As for the data-dependent protocol, we prove basic properties. The proofs are similar to those of Lemmas~\ref{lem:unbiasedness} and \ref{lem:MSE} and we defer them to Appendix~\ref{sec:app:alternative_protocol}.

\begin{lemma}[Unbiasedness] 
\label{lem:unbiasedness_2}
The encoder $\alpha$ defined in \eqref{eq:randomized_protocol} is unbiased. That is,  $\EE{\alpha}{\alpha(X_i)} = X_i$ for all $i$. As a result, $Y$ is an unbiased estimate of the true average: $\EE{\alpha}{Y} = X$.
\end{lemma}

\begin{lemma}[Mean Squared Error]
\label{lem:MSE_2} 
Let $\alpha = \alpha(k)$ be encoder defined as in \eqref{eq:randomized_protocol_2}. Then
\begin{equation}
\label{eq:MSE_general_2}
MSE_{\alpha}(X_1, \dots, X_n) = \frac{1}{n^2} \sum_{i=1}^n \sum_{j=1}^d \left( \frac{d-k}{k} \right) \left( X_i(j) - \mu_i \right)^2.
\end{equation}
\end{lemma}

\section{Communication Protocols}
\label{sec:comm}

Having defined the encoding protocols $\alpha$, we need to specify the way the encoded vectors $Y_i = \alpha(X_i)$, for $i=1,2,\dots,n$, are communicated to the server. Given a specific  {\em communication protocol} $\beta$, we write $\beta(Y_i)$ to denote the (expected) number of bits that are communicated by node $i$ to the server. Since $Y_i = \alpha(X_i)$ is in general not deterministic, $\beta(Y_i)$ can be a random variable.

\begin{definition}[Communication Cost] 
The {\em communication cost} of communication protocol $\beta$ under randomized encoding $\alpha$ is the total expected number of bits transmitted to the server: 
\begin{equation}
\label{eq:989d8dd} 
C_{\alpha,\beta}(X_1,\dots,X_n) = \EE{\alpha}{\sum_{i=1}^n  \beta(\alpha(X_i))}. 
\end{equation}
\end{definition}

Given $Y_i$, a good communication protocol is able to encode $Y_i=\alpha(X_i)$ using a few bits only.  Let $r$ denote the number of bits used to represent a floating point number.  Let $\bar{r}$ be the  the number of bits representing $\mu_i$.

In the rest of this section we describe several communication protocols $\beta$ and calculate their communication cost.

\subsection{Naive} 
Represent $Y_i=\alpha(X_i)$ as $d$ floating point numbers. Then for all encoding protocols $\alpha$ and all $i$ we have $\beta(\alpha(X_i)) = dr$, whence \[C_{\alpha,\beta} = \EE{\alpha}{\sum_{i=1}^n \beta(\alpha(X_i))} = ndr.\]

\subsection{Varying-length}
We will use a single variable for every element of the vector $Y_i$, which does not have constant size. The first bit decides whether the value represents $\mu_i$ or not. If yes, end of variable, if not, next $r$ bits represent the value of $Y_i(j)$. In addition, we need to communicate $\mu_i$, which takes $\bar r$ bits\footnote{The distinction here is because $\mu_i$ can be chosen to be data independent, such as $0$, so we don't have to communicate anything (i.e., $\bar r = 0$)}. We thus have
\begin{equation}
\label{eq:s09y09hjfff}
\beta(\alpha(X_i)) = \bar r + \sum_{j=1}^d \left(1_{(Y_i(j) = \mu_i)} + (r+1) \times 1_{(Y_i(j) \neq \mu_i)}\right),
\end{equation}
where $1_{e}$ is the indicator function of event $e$. The expected number of bits communicated is given by

\begin{align*} 
\label{eq:9y80s9y0sd}
C_{\alpha,\beta} = \EE{\alpha}{\sum_{i=1}^n\beta(\alpha(X_i)))} &\overset{\eqref{eq:s09y09hjfff}}{=} n \bar r +\sum_{i=1}^n \sum_{j=1}^d \left(1-p_{ij} + (r+1) p_{ij} \right) \\
&= n \bar r +\sum_{i=1}^n \sum_{j=1}^d \left(1 + r p_{ij} \right)
\end{align*}
In the special case when $p_{ij} = p>0 $ for all $i,j$, we get
$$ C_{\alpha,\beta} = n(\bar{r} + d + pdr). $$

\subsection{Sparse Communication Protocol for Encoder~\eqref{eq:randomized_protocol}} 
We can represent $Y_i$ as a sparse vector; that is, a list of pairs $(j, Y_i(j))$ for which $Y_i(j) \neq \mu_i$. The number of bits to represent each pair is $\lceil \log(d)\rceil + r$. Any index not found in the list, will be interpreted by server as having value $\mu_i$. Additionally, we have to communicate the value of $\mu_i$ to the server, which takes $\bar r$ bits. We assume that the value $d$, size of the vectors, is known to the server. Hence,
\[\beta(\alpha(X_i)) = \bar r + \sum_{j=1}^d 1_{(Y_i(j) \neq \mu_i)} \times \left( \lceil \log d \rceil + r \right) . \]
Summing up through $i$ and  taking expectations, the the communication cost is given by
\begin{equation}
\label{eq:s09y09y9ff}
C_{\alpha,\beta} = \EE{\alpha}{\sum_{i=1}^n\beta(\alpha(X_i))} = n \bar{r} + (\lceil \log d \rceil + r)  \sum_{i=1}^n \sum_{j=1}^d p_{ij}.
\end{equation}
In the special case when $p_{ij} = p>0 $ for all $i,j$, we get
$$ C_{\alpha,\beta} = n\bar{r} + (\lceil \log d \rceil + r) ndp. $$

\begin{remark}
A practical improvement upon this could be to (without loss of generality) assume that the pairs $(j, Y_i(j))$ are ordered by $j$, i.e., we have $\{ (j_s, Y_i(j_s)) \}_{s=1}^k$ for some $k$ and $j_1 < j_2 < \dots < j_k$. Further, let us denote $j_0 = 0$. We can then use a variant of variable-length quantity \cite{wikiVLQ} to represent the set $\{ (j_s - j_{s-1}, Y_i(j_s)) \}_{s=1}^k$. With careful design one can hope to reduce the $\log(d)$ factor in the average case. Nevertheless, this does not improve the worst case analysis we focus on in this chapter, and hence we do not delve deeper in this.
\end{remark}

\subsection{Sparse Communication Protocol for Encoder \eqref{eq:randomized_protocol_2}} 

We now describe a sparse  communication protocol compatible only with fixed length encoder defined in \eqref{eq:randomized_protocol_2}. Note that subset selection can  be compressed in the form of a random seed, letting us avoid the $\log(d)$ factor in \eqref{eq:s09y09y9ff}. This includes the protocol defined in \eqref{eq:randomized_protocol_2} but also \eqref{eq:randomized_protocol} with uniform probabilities $p_{ij}$.

In particular, we can represent $Y_i$ as a sparse vector containing the list of the values for which $Y_i(j) \neq \mu_i$, ordered by $j$. Additionally, we need to communicate  the value  $\mu_i$ (using $\bar r$ bits) and a random seed (using $\bar r_s$ bits), which can be used to reconstruct the indices $j$, corresponding to the communicated values. Note that for any fixed $k$ defining protocol \eqref{eq:randomized_protocol_2}, we have $|S_i|=k$. Hence, communication cost is deterministic:
\begin{equation}
\label{eq:beta_deterministic_sparse}
C_{\alpha, \beta} = \sum_{i=1}^n \beta(\alpha(X_i)) = n (\bar r + \bar r_s) + nkr.
\end{equation}

In the case of the variable-size-support encoding protocol \eqref{eq:randomized_protocol} with $p_{ij} = p > 0$ for all $i, j$, the  sparse communication protocol described here yields expected communication cost
\begin{equation}
\label{eq:exp_beta_deterministic_sparse}
C_{\alpha,\beta} = \EE{\alpha}{\sum_{i=1}^n\beta(\alpha(X_i))} = n (\bar r + \bar r_s) + ndpr.
\end{equation}

\subsection{Binary}
If the elements of $Y_i$ take only two different values, $Y_i^{min}$ or $Y_i^{max}$, we can use a {\em binary communication protocol}. That is, for each node $i$, we communicate the values of $Y_i^{min}$ and $Y_i^{max}$ (using $2r$ bits), followed by a single bit per element of the array indicating whether $Y_i^{max}$ or $Y_i^{min}$ should be used. The resulting (deterministic) communication cost is 
\begin{equation}
\label{eq:binary_comm_cost}
C_{\alpha, \beta} = \sum_{i=1}^n \beta(\alpha(X_i)) = n (2 r) + nd.
\end{equation}

\subsection{Discussion}
In the above, we have presented several communication protocols of different complexity. However, it is not possible to claim any of them is the most efficient one. Which communication protocol is the best, depends on the specifics of the used encoding protocol. Consider the extreme case of encoding protocol \eqref{eq:randomized_protocol} with $p_{ij} = 1$ for all $i, j$. The naive communication protocol is clearly the most efficient, as all other protocols need to send some additional information.

However, in the interesting case when we consider small communication budget, the sparse communication protocols are the most efficient. Therefore, in the following sections, we focus primarily on optimizing the performance using these protocols.

\section{Examples}
\label{sec:examples}

In this section, we highlight on  several instantiations of our protocols, recovering existing techniques and formulating novel ones. We comment on the resulting trade-offs between  communication cost and estimation error.

\subsection{Binary Quantization}

We start by recovering an existing method, which turns every element of the vectors $X_i$ into a particular binary representation.

\begin{example}
\label{ex:suresh}
If we set the parameters of protocol \eqref{eq:randomized_protocol} as $\mu_i = X_i^{min}$ and $p_{ij} = \frac{X_i(j) - X_i^{min}}{\Delta_i}$, where $\Delta_i\eqdef X_i^{max} - X_i^{min}$ (assume, for simplicity, that $\Delta_i \neq 0$), we exactly recover the  quantization algorithm proposed in \cite{Distributed_mean}: 
\begin{equation}
\label{eq:protocol_special_case_felix}
Y_{i}(j) = 
\begin{cases}
X_i^{max} & \quad \text{with probability} \quad \frac{X_i(j) - X_i^{min}}{\Delta_i}, \\
X_i^{min} & \quad \text{with probability} \quad \frac{X_i^{max} - X_i(j)}{\Delta_i}.
\end{cases}
\end{equation}

Using the formula \eqref{eq:MSE_general} for the encoding protocol $\alpha$, we get
\begin{align*}
MSE_\alpha &= \frac{1}{n^2} \sum_{i=1}^n \sum_{j=1}^d \frac{X_i^{max} - X_i(j)}{X_i(j) - X_i^{min}} \left( X_i(j) - X_i^{min} \right)^2 \leq \frac{d}{2 n} \cdot \frac1n \sum_{i=1}^n \|X_i\|^2.
\end{align*}
This exactly recovers the MSE bound established in \cite[Theorem 1]{Distributed_mean}. Using the binary communication protocol yields the communication cost of $1$ bit per element if $X_i$, plus a two real-valued scalars \eqref{eq:binary_comm_cost}.
\end{example}

\begin{remark}
\label{rem:rotation_quantization_protocol}
If we use the above protocol jointly with randomized linear encoder and decoder (see Example~\ref{ex:linear_encoder}), where the linear transform is the randomized Hadamard transform, we recover the method described in \cite[Section 3]{Distributed_mean} which yields improved $MSE_\alpha = \frac{2\log d + 2}{n} \cdot \frac1n \sum_{i=1}^n \|X_i\|^2$ and can be implemented in $\mathcal{O} (d \log d)$ time.
\end{remark}

\subsection{Sparse Communication Protocols}

Now we move to comparing the communication costs and estimation error of various instantiations of the encoding protocols, utilizing the deterministic sparse communication protocol and uniform probabilities.

For the remainder of this section, let us only consider instantiations of our protocol where $p_{ij} = p > 0$ for all $i, j$, and assume that the node centers are set to the vector averages, i.e., $\mu_i = \frac1d \sum_{j=1}^d X_i(j)$. Denote $R = \frac1n \sum_{i=1}^n \sum_{j=1}^d (X_i(j) - \mu_i)^2$. For simplicity, we also assume that $|S| = nd$, which is what we can in general expect without any prior knowledge about the vectors $X_i$. 

The properties of the following examples follow from Equations~\eqref{eq:MSE_general} and \eqref{eq:exp_beta_deterministic_sparse}. When considering the communication costs of the protocols, keep in mind that the trivial benchmark is $C_{\alpha, \beta} = ndr$, which is achieved by simply sending the vectors unmodified. Communication cost of $C_{\alpha, \beta} = nd$ corresponds to the interesting special case when we use (on average) one bit per element of each $X_i$.

\begin{example}[Full communication]
\label{ex:full}
If we choose $p = 1$, we get
$$ C_{\alpha, \beta} = n(\bar r_s + \bar r) + ndr, \qquad MSE_{\alpha, \gamma} = 0. $$
In this case, the encoding protocol is lossless, which ensures $MSE = 0$. Note that in this case, we could get rid of the $n(\bar r_s + \bar r)$ factor by using naive communication protocol.
\end{example}

\begin{example}[Log MSE]
\label{ex:log_MSE}
If we choose $p = 1 / \log d$, we get
$$ C_{\alpha, \beta} = n(\bar r_s + \bar r) + \frac{ndr}{\log d}, \qquad MSE_{\alpha, \gamma} = \frac{\log(d) - 1}{n} R. $$
This protocol order-wise matches the $MSE$ of the method in Remark~\ref{rem:rotation_quantization_protocol}. However, as long as $d > 2^r$, this protocol attains this error with \emph{smaller} communication cost. In particular, this is on expectation \emph{less} than a single bit per element of $X_i$. Finally, note that the factor $R$ is always smaller or equal to the factor $\frac1n \sum_{i=1}^n \| X_i \|^2$ appearing in Remark~\ref{rem:rotation_quantization_protocol}.
\end{example}

\begin{example}[1-bit per element communication]
\label{ex:1_bit}
If we choose $p = 1 / r$, we get
$$ C_{\alpha, \beta} = n(\bar r_s + \bar r) + nd, \qquad MSE_{\alpha, \gamma} = \frac{r - 1}{n} R. $$
This protocol communicates on expectation single bit per element of $X_i$ (plus additional $\bar r_s + \bar r$ bits per client), while attaining bound on $MSE$ of $\mathcal{O}(r/n)$. To the best of out knowledge, this is the first method to attain this bound without additional assumptions. \end{example}

\begin{example}[Alternative 1-bit per element communication]
\label{ex:1_bit_alternative}
If we choose $p = \frac{d - \bar r_s - \bar r}{dr}$, we get
$$ C_{\alpha, \beta} =  nd, \qquad MSE_{\alpha, \gamma} = \frac{\frac{dr}{d - \bar r_s - \bar r} - 1}{n} R. $$
This alternative protocol attains on expectation exactly single bit per element of $X_i$, with (a slightly more complicated) $\mathcal{O} (r/n)$ bound on $MSE$.
\end{example}

\begin{example}[Below 1-bit communication]
\label{ex:below_1_bit}
If we choose $p = 1 / d$, we get
$$ C_{\alpha, \beta} = n(\bar r_s + \bar r) + nr, \qquad MSE_{\alpha, \gamma} = \frac{d - 1}{n} R. $$
This protocol attains the MSE of protocol in Example~\ref{ex:suresh} while at the same time communicating on average significantly less than a single bit per element of $X_i$.
\end{example}

\begin{table}
\begin{center}
\begin{tabular}{|c|c|c|c|}
\hline
Example & $p$ & $C_{\alpha, \beta}$ & $MSE_{\alpha, \gamma}$ \\
\hline
Example \ref{ex:full} (Full) & $1$ & $ndr$ & $0$ \\
Example \ref{ex:log_MSE} (Log $MSE$) & $1 / \log d$ & $n(\bar r_s + \bar r) + \frac{ndr}{\log d}$ & $(\log(d) - 1) \tfrac{R}{n}$ \\
Example \ref{ex:1_bit} ($1$-bit) & $1 / r$ & $n(\bar r_s + \bar r) + nd$ & $(r - 1) \tfrac{R}{n}$ \\
Example \ref{ex:below_1_bit} (below $1$-bit) & $1 / d$ & $n(\bar r_s + \bar r) + nr$ & $(d - 1) \tfrac{R}{n}$ \\
\hline
\end{tabular}
\end{center}
\caption{Summary of achievable communication cost and estimation error, for various choices of probability $p$.}
\label{tbl:main1}
\end{table}

We summarize these examples in Table~\ref{tbl:main1}.

Using the deterministic sparse protocol, there is an obvious lower bound on the communication cost --- $n(\bar r_s + \bar r)$. We can bypass this threshold by using the sparse protocol, with a data-independent choice of $\mu_i$, such as $0$, setting $\bar r = 0$. By setting $p = \epsilon / d( \lceil \log d \rceil + r)$, we get arbitrarily small expected communication cost of $C_{\alpha, \beta} = \epsilon$, and the cost of exploding estimation error $MSE_{\alpha, \gamma} = \mathcal{O}(1 / \epsilon n)$.

Note that all of the above examples have random communication costs. What we present is the \emph{expected} communication cost of the protocols. All the above examples can be modified to use the encoding protocol with fixed-size support defined in $\eqref{eq:randomized_protocol_2}$ with the parameter $k$ set to the value of $pd$ for corresponding $p$ used above, to get the same results. The only practical difference is that the communication cost will be deterministic for each node, which can be useful for certain applications.

\section{Optimal Encoders} \label{sec:opt}

Here we consider $(\alpha, \beta, \gamma)$, where $\alpha=\alpha(p_{ij},\mu_i)$ is the encoder defined in \eqref{eq:randomized_protocol}, $\beta$ is the associated the sparse communication protocol, and $\gamma$ is the averaging decoder.  Recall from Lemma~\ref{eq:MSE_general} and \eqref{eq:s09y09y9ff}
that the mean square error and communication cost are given by:
\begin{equation}
\label{eq:MSE+C}
MSE_{\alpha,\gamma} = \frac{1}{n^2} \sum_{i,j} \left( \frac{1}{p_{ij}} - 1 \right) \left( X_i(j) - \mu_i \right)^2 , \quad C_{\alpha,\beta} =  n \bar{r} + (\lceil \log d \rceil + r)  \sum_{i=1}^n \sum_{j=1}^d p_{ij}.
\end{equation}
 
Having these closed-form formulae as functions of the parameters $\{p_{ij}, \mu_i\}$, we can now ask questions such as:
\begin{enumerate}
\item Given a communication budget, which encoding protocol has the smallest mean squared error?
\item Given a bound on the mean squared error, which encoder  suffers the minimal communication cost? 
\end{enumerate}
Let us now address the first question; the second question can  be handled in a similar fashion.  In particular, consider the optimization problem
\begin{eqnarray}
\text{minimize}& & \sum_{i,j} \left(\frac{1}{p_{ij}} -1\right) (X_i(j)-\mu_i)^2 \notag \\
\text{subject to} && \mu_i \in \R, \quad i=1,2,\dots,n \notag\\
&& \sum_{i,j} p_{ij} \leq B \label{eq:optimalCCyyy}\\
&& 0< p_{ij} \leq 1, \quad i=1,2,\dots,n; \quad j=1,2,\dots,d,
\end{eqnarray}
where $B>0$ represents a  bound on the part of the total communication cost in \eqref{eq:MSE+C} which depends on the choice of the probabilities $p_{ij}$.

Note that while the constraints in \eqref{eq:optimalCCyyy} are convex (they are linear), the objective is not jointly convex in $\{p_{ij},\mu_i\}$. However, the objective is convex in $\{p_{ij}\}$ and convex in $\{\mu_i\}$. This suggests a simple {\em alternating minimization} heuristic for solving the above problem: 
\begin{enumerate}
\item  Fix the probabilities and optimize over the node centers,
 \item Fix the node centers and optimize over probabilities.
\end{enumerate}

These two steps are  repeated until a suitable convergence criterion is reached. Note that the first step has a closed form solution. Indeed, the problem decomposes across the node centers to $n$ univariate unconstrained convex quadratic minimization problems, and the solution is given by
\begin{equation}\label{eq:opt_node_centers_given_prob}\mu_i = \frac{\sum_j w_{ij} X_i(j)}{\sum_{j} w_{ij}},\qquad w_{ij}\eqdef \frac{1}{p_{ij}}-1.\end{equation}
The second step does not  have a closed form solution in general; we provide an analysis of  this step in Section~\ref{sec:0y09fyhff}.

\begin{remark}Note that the upper bound $\sum_{i,j}(X_i(j)-\mu_i)^2 / p_{ij}$ on the objective is jointly convex in $\{p_{ij},\mu_i\}$. We may therefore instead optimize this upper bound by a suitable convex optimization algorithm.
\end{remark}

\begin{remark} An alternative and a more practical model to \eqref{eq:optimalCCyyy} is to choose per-node budgets $B_1,\dots,B_n$ and require $\sum_{j} p_{ij} \leq B_i$ for all $i$. The problem becomes separable across the nodes, and can therefore be solved by each node independently. If we set $B=\sum_i B_i$, the optimal solution obtained this way will lead to MSE which is lower bounded by the MSE obtained through \eqref{eq:optimalCCyyy}.
\end{remark}

\subsection{Optimal Probabilities for Fixed Node Centers}\label{sec:0y09fyhff}

Let the node centers $\mu_i$ be fixed. Problem \eqref{eq:optimalCCyyy} (or, equivalently, step 2 of the alternating minimization method described above)  then takes the form 
\begin{eqnarray}
\text{minimize}& & \sum_{i,j} \frac{(X_i(j)-\mu_i)^2}{p_{ij}} \notag \\
\text{subject to} && \sum_{i,j} p_{ij} \leq B \label{eq:optimalCCxxx}\\
&& 0< p_{ij} \leq 1, \quad  i=1,2,\dots n, \quad j=1,2,\dots,d \notag.
\end{eqnarray}

Let $S = \{(i,j)\;:\; X_i(j) \neq \mu_i\}$. Notice that as long as $B\geq |S|$, the optimal solution is to set $p_{ij}=1$ for all $(i,j)\in S$ and $p_{ij}= 0$ for all $(i,j)\notin S$.\footnote{We interpret $0/0$ as $0$ and do not worry about infeasibility. These issues can be properly formalized by allowing $p_{ij}$ to be zero in the encoding protocol and in \eqref{eq:optimalCCxxx}. However, handling this singular situation requires a notational overload which we are not willing to pay.} In such a case, we have $MSE_{\alpha,\gamma} = 0.$ Hence, we can without loss of generality assume that $B \leq |S|$.

While we are not able to derive a closed-form solution to this problem, we can formulate upper and lower bounds on the optimal estimation error, given a bound on the communication cost formulated via $B$.

\begin{theorem}[MSE-Optimal Protocols subject to a Communication Budget]
\label{thm:main}
Consider problem \eqref{eq:optimalCCxxx} and fix any $B\leq |S|$. Using the sparse communication protocol $\beta$, the optimal encoding protocol $\alpha$ has communication complexity \begin{equation}\label{eq:9sy09yhiffs}C_{\alpha,\beta} = n\bar r + (\lceil \log d\rceil + r )B,\end{equation} and the mean squared error satisfies the bounds
\begin{equation}\label{eq:s0y09yhff} \left(\frac{1}{B}-1\right)\frac{R}{n} \leq MSE_{\alpha,\gamma} \leq \left(\frac{|S|}{B}-1\right)\frac{R}{n},\end{equation}
where $R = \frac{1}{n}\sum_{i=1}^n\sum_{j=1}^d (X_i(j)-\mu_i)^2 = \frac{1}{n}\sum_{i=1}^n \|X_i - \mu_i 1\|^2$.
Let $a_{ij}=|X_{i}(j)-\mu_i|$ and $W=\sum_{i,j} a_{ij}$. If, moreover, $B \leq \sum_{(i,j)\in S} a_{ij}/ \max_{(i,j)\in S} a_{ij}$ (which is true, for instance, in the ultra-low communication regime with $B\leq 1$), then 
\begin{equation}\label{eq:s08y09fh9ff}MSE_{\alpha,\gamma}  = \frac{W^2}{n^2 B}- \frac{R}{n}.\end{equation}
\end{theorem}
\begin{proof}
Setting $p_{ij} = B/|S|$ for all $(i,j)\in S$ leads to a feasible solution of \eqref{eq:optimalCCxxx}. In view of \eqref{eq:MSE+C}, one then has 
\[MSE_{\alpha,\gamma} = \frac{1}{n^2} \left(\frac{|S|}{B}-1\right)\sum_{(i,j)\in S}  \left( X_i(j) - \mu_i \right)^2 =  \left(\frac{|S|}{B}-1\right) \frac{R}{n},\]
where 
$R = \frac{1}{n}\sum_{i=1}^n\sum_{j=1}^d (X_i(j)-\mu_i)^2 = \frac{1}{n}\sum_{i=1}^n \|X_i - \mu_i 1\|^2$.

If we relax the problem by removing the constraints $p_{ij}\leq 1$, the optimal solution satisfies  $a_{ij}/p_{ij} = \theta>0$  for all $(i,j)\in S$.  At optimality the bound involving $B$ must be tight, which leads to
$\sum_{(i,j)\in S} a_{ij}/\theta = B$, whence $\theta = \tfrac{1}{B}\sum_{(i,j)\in S} a_{ij}$. So, $p_{ij} = a_{ij}B/\sum_{(i,j)\in S}a_{ij}$. The optimal MSE therefore satisfies the lower bound
\[MSE_{\alpha,\gamma} \geq \frac{1}{n^2} \sum_{(i,j)\in S} \left(\frac{1}{p_{ij}}-1\right) \left( X_i(j) - \mu_i \right)^2 =  \frac{1}{n^2 B}W^2 - \frac{R}{n},\]
where $W \eqdef \sum_{(i,j)\in S} a_{ij} \geq \left(\sum_{(i,j)\in S} a_{ij}^2\right)^{1/2} = (nR)^{1/2}$. Therefore,
$MSE_{\alpha,\gamma} \geq \left(\frac{1}{B}-1\right) \frac{R}{n}$. If $B \leq \sum_{(i,j)\in S}a_{ij} / \max_{(i,j)\in S} a_{ij}$, then $p_{ij}\leq 1$ for all $(i,j)\in S$, and hence we have optimality. (Also note that, by Cauchy-Schwarz inequality,  $W^2\leq n R |S|$.)
\end{proof}


\subsection{Trade-off Curves}

To illustrate the trade-offs between communication cost and estimation error (MSE) achievable by the protocols discussed in this section, we present simple numerical examples  in Figure~\ref{fig:uniform_vs_optimal}, on three synthetic data sets with  $n=16$ and $d=512$. We choose an array of values  for $B$, directly bounding the communication cost via \eqref{eq:9sy09yhiffs}, and evaluate the $MSE$ \eqref{eq:MSE_general} for three encoding protocols (we use the sparse communication protocol and averaging decoder). All these protocols have the same communication cost, and only differ in the selection of the parameters $p_{ij}$ and $\mu_i$. In particular, we consider 
\begin{itemize}
\item[(i)] uniform probabilities $p_{ij}=p>0$ with average node centers $\mu_i = \frac{1}{d}\sum_{j=1}^d X_i(j)$ (blue dashed line), 
\item[(ii)]  optimal probabilities $p_{ij}$ with average node centers $\mu_i = \frac{1}{d}\sum_{j=1}^d X_i(j)$ (green dotted line), and 
\item[(iii)] optimal probabilities with optimal node centers,  obtained via the alternating minimization approach described above (red solid line).
\end{itemize}

In order to put a scale on the horizontal axis, we assumed that $r = 16$. Note that, in practice, one would choose $r$ to be as small as possible without adversely affecting  the application utilizing our  distributed mean estimation method. The three plots represent $X_i$ with entries drawn in an i.i.d.\ fashion from Gaussian ($\mathcal{N}(0, 1)$), Laplace ($\mathcal{L}(0, 1)$) and chi-squared ($\chi^2(2)$) distributions, respectively. As we can see, in the case of non-symmetric distributions, it is not necessarily optimal to set the node centers to  averages. 

As expected, for fixed node centers, optimizing over probabilities results in improved performance, across the entire trade-off curve. That is, the curve shifts downwards. In the first two plots based on data from symmetric distributions (Gaussian and Laplace), the average node centers are nearly optimal,  which explains why the red solid and green dotted lines coalesce. This can be also established formally. In the third plot, based on the non-symmetric chi-squared data, optimizing over node centers leads to further improvement, which gets more pronounced with increased communication budget. It is possible to generate data where the difference between any pair of the three trade-off curves becomes arbitrarily large. 

Finally, the black cross represents performance of the quantization protocol from Example~\ref{ex:suresh}. This approach appears as a single point in the trade-off space due to lack of any parameters to be fine-tuned.

\begin{figure}
\centering
\includegraphics[width=0.32\linewidth]{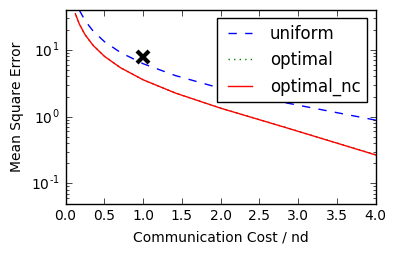}
\includegraphics[width=0.32\linewidth]{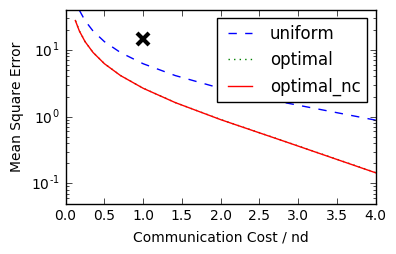}
\includegraphics[width=0.32\linewidth]{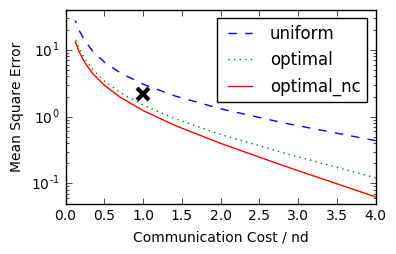}
\caption{{\em Trade-off curves} between communication cost and estimation error (MSE) for four protocols. The plots correspond to vectors $X_i$ drawn in an i.i.d.\ fashion from Gaussian, Laplace and $\chi^2$ distributions, from left to right. The black cross marks the performance of binary quantization (Example~\ref{ex:suresh}).}
\label{fig:uniform_vs_optimal}
\end{figure}

%
%
%

\section{Further Considerations} \label{sec:further}

In this section we outline further ideas worth consideration. However, we leave a detailed analysis to future work.

\subsection{Beyond Binary Encoders}

We can generalize the binary encoding protocol~\eqref{eq:randomized_protocol} to a $k$-ary protocol. To illustrate the concept without unnecessary notation overload, we present only the ternary (i.e., $k=3$) case.

Let the collection of parameters $\{p'_{ij}, p''_{ij}, \bar{X}'_i, \bar{X}''_i\}$ define an encoding protocol $\alpha$ as follows:

\begin{equation}
\label{eq:randomized_protocol_3}
Y_{i}(j) = 
\begin{cases}
\bar X'_i & \quad \text{with probability} \quad p'_{ij}, \\
\bar X''_i & \quad \text{with probability} \quad p''_{ij}, \\
\frac{1}{1 - p'_{ij} - p''_{ij}} \left( X_i(j) - p'_{ij} \bar X'_i - p''_{ij} \bar X''_i \right) & \quad \text{with probability} \quad 1 - p'_{ij} - p''_{ij}.
\end{cases}
\end{equation}

It is straightforward to generalize Lemmas~\ref{lem:unbiasedness} and \ref{lem:MSE} to this case. We omit the proofs for brevity.

\begin{lemma}[Unbiasedness] 
\label{lem:unbiasedness_3} 
The encoder $\alpha$ defined in \eqref{eq:randomized_protocol_3} is unbiased. That is,  $\EE{\alpha}{\alpha(X_i)} = X_i$ for all $i$. As a result, $Y$ is an unbiased estimate of the true average: $\EE{\alpha}{Y} = X$.
\end{lemma}

\begin{lemma}[Mean Squared Error]
\label{lem:MSE_3}
Let $\alpha = \alpha \left( p'_{ij}, p''_{ij}, \bar{X}'_i, \bar{X}''_i\right) $ be the protocol defined in \eqref{eq:randomized_protocol_3}. Then
\begin{equation*}
MSE_{\alpha}(X_1, \dots, X_n) = \frac{1}{n^2} \sum_{i=1}^n \sum_{j=1}^d \left( p'_{ij} \left( X_i(j) - \bar X'_i \right)^2 + p''_{ij} \left( X_i(j) - \bar X''_i \right)^2 + \left( p'_{ij} \bar X'_i + p''_{ij} \bar X''_i \right)^2 \right).
\end{equation*}
\end{lemma}

We expect the $k$-ary protocol to lead to better (lower) MSE bounds, but at the expense of an increase in communication cost. Whether or not the trade-off offered by $k>2$ is better than that for the $k=2$ case investigated in this work is an interesting question to consider.

\subsection{Preprocessing via Random Rotations}

Following the idea  proposed in \cite{Distributed_mean}, one can explore an encoding protocol $\alpha_Q$ which arises as the composition of a random rotation, $Q$, applied to $X_i$ for all $i$, followed by the protocol $\alpha$ described in Section~\ref{sec:encode}.  Letting $Z_i = Q X_i$ and $Z=\frac{1}{n}\sum_i Z_i$, we thus have
\[Y_i = \alpha(Z_i), \qquad i=1,2,\dots,n.\]
With this protocol we associate the decoder
$\gamma(Y_1,\dots,Y_n) = \frac{1}{n}\sum_{i=1}^n Q^{-1} Y_i .$

Note that
\begin{eqnarray*}
MSE_{\alpha, \gamma}&=& \E{\left\|\gamma(Y_1,\dots,Y_n) - X\right\|^2} \\
&=& \E{\left\|Q^{-1}\gamma(Y_1,\dots,Y_n) - Q^{-1}Z\right\|^2} \\
&=& \E{\left\|\gamma(\alpha(Z_1),\dots,\alpha(Z_n)) - Z\right\|^2} \\
&=& \E{\E{\left\|\gamma(\alpha(Z_1),\dots,\alpha(Z_n)) - Z\right\|^2\;|\; Q}}.
\end{eqnarray*}

This approach is motivated by the following observation: a random rotation can be identified by a single random seed, which is easy to  communicate to the server without the need to communicate all floating point entries defining $Q$. So, a random rotation pre-processing step implies only a minor  communication overhead. However, {\em if} the preprocessing step helps to dramatically reduce the MSE, we get an improvement. Note that the inner expectation above is the formula for MSE of our basic encoding-decoding protocol, given that the data is $Z_i = QX_i$ instead of $\{X_i\}$. The outer expectation is over $Q$. Hence,  we would like the to find a mapping $Q$ which tends to transform the data $\{X_i\}$ into new data $\{Z_i\}$ with better MSE, in expectation.

From now on, for simplicity assume the node centers are set to the average, i.e.,  $\bar{Z}_i = \frac{1}{d}\sum_{j=1}^d Z_i(j)$. For any vector $x\in \R^d$, define
$$ \sigma(x) \eqdef \sum_{j=1}^d (x(j) - \bar{x})^2 = \|x-\bar{x}1\|^2, $$
where $\bar{x} = \tfrac{1}{d}\sum_j x(j)$ and $1$ is the vector of all ones. Further, for simplicity assume that $p_{ij}=p$ for all $i,j$. Then using Lemma~\ref{lem:MSE}, we get
$$ MSE = \frac{1-p}{pn^2}  \sum_{i=1}^n \EE{Q}{ \|Z_i - \bar{Z}_i 1\|^2} =  \frac{1-p}{pn^2}  \sum_{i=1}^n \EE{Q}{\sigma(Q X_i)}. $$

It is interesting to investigate whether   choosing $Q$ as a random rotation, rather than identity (which is the implicit choice done in previous sections), leads to improvement in MSE, i.e., whether we can in some well-defined sense obtain an inequality of the type
\[ \sum_i \EE{Q}{\sigma(Q  X_i)} \ll \sum_i \sigma(X_i).\]

This is the case for the quantization protocol proposed in \cite{Distributed_mean}, which arises as a special case of our more general protocol. This is because the quantization protocol is suboptimal within our family of encoders. Indeed, as we have shown, with a different choice of the parameter we  can obtained results which improve, in theory, on the rotation + quantization approach. This suggests that perhaps combining an appropriately chosen rotation pre-processing step with our optimal encoder, it may be possible to achieve further improvements in MSE for any fixed communication budget.  Finding suitable random rotations $Q$ requires  a careful study which we leave to future research.

\section{Application to Federated Learning}
\label{sec:meanFL}

In this Section, we experiment with applying some of these techniques in the context of Federated learning \cite{konecny2016federated, mcmahan2016federated,
  konecny2015federated}. As discussed in previous chapter, by Federated Optimization or Learning we refer to a setting where we
train a shared global model under the coordination of a central
server, from a federation of participating devices. The participating
devices (clients) are typically large in number and have slow or
unstable internet connections. A motivating example for federated
optimization arises when the training data is kept locally on users'
mobile devices, and the devices are used as nodes performing
computation on their local data in order to update a global model. The
framework differs from conventional distributed machine learning
\cite{reddi2016aide, ma2015distributed, DANE, zhang2015communication, largeNN, chilimbi2014project}
due to the the large number of clients, highly unbalanced and
non-i.i.d.\ data and unreliable network connections.

Federated learning offers distinct practical advantages compared to performing
learning in the data center. The model update is generally less
privacy-sensitive than the data itself, and the server never needs to
store these updates. Thus, when applicable, federated learning can
significantly reduce privacy and security risks by limiting the attack
surface to only the device, rather than the device and the cloud. This
approach also leverages the data-locality and computational power of
the large number of mobile devices.

For simplicity, we consider synchronized algorithms for federated learning \cite{mcmahan2016federated,chen16revisiting}, where a typical round consists of the following steps:
\begin{enumerate}[noitemsep]
  \item A subset of clients is selected, each of which downloads the current model.
  \item Each client in the subset computes an updated model based on their local data.
  \item The updated models are sent from the selected clients to the sever.
  \item The server aggregates these models (typically by averaging) to
    construct an improved global model.
\end{enumerate}

A naive implementation of the above framework requires that each
client sends a full model (or a full model gradient) back to the
server in each round. For large models, this step is likely to be the
bottleneck of federated learning due to the asymmetric property of
internet connections: the uplink is typically much slower than
downlink. The US average broadband speed was 55.0Mbps download vs. 
18.9Mbps upload, with some ISPs being significantly more asymmetric,
e.g., Xfinity at 125Mbps down vs.\ 15Mbps up \cite{speedtest}. 
Cryptographic protocols used to protect individual update \cite{bonawitz2016practical} further increase the size of the data needed to be communicated back to server.
It is therefore important to investigate methods which can reduce the
uplink communication cost. In this section, we study two general
approaches:
\begin{itemize}
\item Structured updates, where we learn an update from a restricted
  lower-dimensional space.
\item Sketched updates, where we learn a full model update, but then
    compress it before sending to the server.
\end{itemize}
These approaches can be combined, e.g., first learning a structured
update and then sketching it; however, we do not experiment with this
combination in the current work.

In the following, we formally describe the problem. 
The goal of federated learning is to learn a model with parameters embodied in a real matrix $\bW\in \R^{d_1\times d_2}$ from data stored across a large number of clients. We first provide a communication-naive version of the federated learning. In round $t\geq 0$, the server distributes the current model $\bW_{t}$ to a subset $S_t$ of $n_t$ clients (for example, to a selected subset of clients whose devices are plugged into power, have access to broadband, and are idle). These clients independently update the model based on their local data. Let the updated local models be $\bW^1_{t}, \bW^2_{t}, \dots, \bW^{n_t}_{t}$, so  the updates can be written as $\bH^i_{t} := \bW^i_{t} - \bW_{t}$, for $ i \in S_t$. Each selected client then sends the update back to the sever, where the global update is computed by aggregating\footnote{A weighted sum might be used to replace the average based on specific implementations.} all the client-side updates: 
\[
  \bW_{t+1} = \bW_{t} + \eta_t \bH_{t}, \qquad \bH_t := \frac{1}{n_t}\sum_{i \in S_t}  \bH^i_{t}.
\]
 The sever  chooses the learning rate $\eta_t$ (for simplicity, we choose $\eta_t = 1$). Recent works show that a careful choice of the server-side learning rate can lead to faster convergence \cite{Ma:2015ti, ma2015distributed, konecny2016federated}. 

In this chapter, we describe federated learning for neural networks, where we use a separate 2D matrix $\bW$ to represent the parameters of each layer. We suppose that $\bW$ gets right-multiplied, i.e., $d_1$ and $d_2$ represent the output and input dimensions respectively. Note that the parameters of a fully connected layer are naturally represented as 2D matrices. However, the kernel of a convolutional layer is a 4D tensor of the shape $\# \text{input} \times \text{width} \times \text{height} \times \# \text{output}$. In such a case, $\bW$ is reshaped from the kernel to the shape $(\# \text{input} \times \text{width} \times \text{height}) \times \# \text{output}$. 

The goal of increasing communication efficiency of federated learning is to reduce the cost of sending $\bH^i_{t}$ to the server. We propose two general strategies of achieving this, discussed next.

\subsection{Structured Update}

The first type of communication efficient update restricts the updates $\bH^i_{t}$ to have a pre-specified {\em structure}. Two types of structures are considered:

{\bf Low rank.} We enforce $\bH^i_{t} \in \R^{d_1 \times
  d_2}$ to be low-rank matrices of rank at most $k$, where $k$ is a fixed number. We express $\bH^i_{t}$ as the product of two matrices:
$\bH^i_{t} = \bA^i_{t} \bB^i_{t}$, where $\bA^i_{t} \in \R^{d_1 \times
  k}$, $\bB^i_{t} \in \R^{k \times d_2}$, and $\bA^i_{t}$ is
generated randomly and fixed, and only $\bB^i_{t}$ is optimized.
Note that $\bA^i_{t}$ can then be compressed in the form of a random seed
and the clients only need to send $\bB^i_{t}$ to the server. We also
tried fixing $\bB_t^i$ and training $\bA^i_t$, as well as training
both $\bA_t^i$ and $\bB^i_t$; neither performed as well. Our approach
seems to perform as well as the best techniques considered in
\cite{denil2013predicting}, 

{\bf Random mask.} We restrict the update $\bH^i_{t}$ to be  sparse matrices, following a pre-defined random sparsity pattern (i.e., a random mask). The pattern is generated afresh in each round and for each client. Similar to the low-rank approach, the sparse pattern can be fully specified by a random seed, and therefore it is only required to send the values of the non-zeros entries of $\bH^i_{t}$. This strategy can be seen as the combination of the master training method and a randomized block coordinate minimization approach \cite{QUARTZ, ALPHA}.

\subsection{Sketched Update}
The second type of communication-efficient update, which we call {\em sketched}, first computes the full unconstrained $\bH^i_{t}$, and then encodes the update in a (lossy) compressed form before sending to the server. The server decodes the updates before doing the aggregation. Such sketching methods have application in many domains \cite{woodruff2014sketching}. This technique corresponds to the general focus of previous sections of this chapter. 

We propose two ways of performing the sketching:

{\bf Subsampling.} 
Instead of sending $\bH^i_t$, each client only communicates matrix $\hat{\bH}^i_t$ which is formed from a random subset of the (scaled) values of $\bH^i_t$. The server then averages the sampled updates, producing the global update $\hat{\bH}_t$. This can be done so that the average of the sampled updates is an unbiased estimator of the true average: $\E{\hat{\bH}_t} = \bH_t$. Similar to the random mask structured update, the mask is randomized independently for each client in each round, and the mask itself is stored as a synchronized seed. It was recently shown that, in a certain setting,  the {\em expected iterates} of SGD converge to the optimal point \cite{ASDA}. Perturbing the iterates by a random matrix of zero mean, which is what our subsampling strategy would do, does not affect this type of convergence. 

{\bf Probabilistic quantization.}
Another way of compressing the updates is by {\em quantizing} the weights. 
We first describe the algorithm of quantizing each scalar to one bit. 
Consider the update $\bH^i_t$, let $h = (h_1,\dots,h_{d_1\times d_2}) = \text{vec}(\bH^i_t)$, and let 
$h_{\max} = \max_j(h_j)$, $h_{\min} = \min_j(h_j)$. 
The compressed update of $h$, denoted by $\tilde{h}$, is generated as follows:
\[
    \tilde{h}_j= 
\begin{cases}
    h_{\max},&    \text{with probability}\quad \frac{h_j - h_{\min}}{h_{\max} - h_{\min}} \\
    h_{\min},   & \text{with probability} \quad \frac{h_{\max} - h_j}{h_{\max} - h_{\min}} 
\end{cases}\enspace .
\]
It is easy to show that $\tilde{h}$ is an unbiased estimator of $h$.
This method provides $32\times$ of compression compared to a 4 byte float. One can also generalize the above to more than 1 bit for each scalar. For $b$-bit quantization, we first equally divide $[h_{\min}, h_{\max}]$ into $2^b$ intervals. Suppose $h_i$ falls in the interval bounded by $h'$ and $h''$. The quantization operates by replacing $h_{\min}$ and $h_{\max}$ of the above equation by $h'$ and $h''$, respectively. Incremental, randomized and distributed optimization algorithms can be similarly analyzed in a quantized updates setting \cite{quantized2005,golovin13randomization,quantized2016}.

{\bf Improving the quantization by structured random rotations.} 
The above 1-bit and multi-bit quantization approaches work best when the scales are approximately equal across different dimensions. 
For example, when $\max = 100$ and $\min= -100$ and most of values are $0$, the 1-bit quantization will lead to large quantization error. We note that performing a random rotation of $h$ before the quantization (multiplying $h$ by an orthogonal matrix) will resolve this issue. 
In the decoding phase, the server needs to perform the inverse rotation before aggregating all the updates. Note that in practice, the dimension of $h$ can be as high as $d = 1e6$, and it is computationally prohibitive to generate ($\mathcal{O}(d^3)$) and apply  ($\mathcal{O}(d^2)$) a rotation matrix. In this work, we use a type of structured rotation matrix which is the product of a Walsh-Hadamard matrix and a binary diagonal matrix, motivated by the recent advance in this topic \cite{yu2016orthogonal}. This reduces the computational complexity of generating and applying the matrix to $\mathcal{O}(d)$ and $\mathcal{O}(d \log d)$.

\subsection{Experiments}

We conducted the experiments using federated learning to train deep
neural networks for the CIFAR-10 image classification task
\cite{krizhevsky09learningmultiple}.
There are 50,000 training examples, which we partitioned into 100 
clients each containing 500 training examples. The 
model architecture was taken from the TensorFlow tutorial
\cite{tensorflowcifar}, which consists of two convolutional layers
followed by two fully connected layers and then a linear
transformation layer to produce logits, for a total of over 1,000,000
parameters. While this model is not the state-of-the-art, it is
sufficient for our needs, as our goal is to evaluate our compression
methods, not achieve the best possible accuracy on this task.

We employ the Federated Averaging algorithm
\cite{mcmahan2016federated}, which significantly decreases the number
of rounds of communication required to train a good model. However, we
expect our techniques will show a similar reduction in communication
costs when applied to synchronized SGD. For Federated Averaging, on
each round we select 10 clients at random, each of which
performs 10 epochs of SGD with a learning rate of $\eta$
on their local dataset using minibatches of 50 images, for a total of 100
local updates. From this updated model we compute the deltas for each
layer $\bH^i_{t}$.

\begin{table}[h]
\vspace{0.05in}
\begin{small}
\begin{tabular}{lllll}
                      & (Low) Rank       & Sampling Probabilities           & model size & reduction \\ 
\hline
Full Model (baseline) & 64, 64, 384, 192 & 1, 1, 1, 1                       & 4.075 MB   & --- \\
Medium subsampling    & 64, 64, 12, 6    & 1, 1, 0.03125, 0.03125           & 0.533 MB   & 7.6$\times$ \\
High subsampling      & 8, 8, 12, 6      & 0.125, 0.125, 0.03125, 0.03125   & 0.175 MB   & 23.3$\times$  \\
\hline
\end{tabular}
  \caption{%
  \footnotesize
    Low rank and sampling parameters for the CIFAR experiments. The
    Sampling Probabilities column gives the fraction of elements uploaded
    for the two convolutional layers and the two fully-connected layers,
    respectively; these parameters are used by \StructMask, \SketchMask, and
    \SketchRotMask. The Low Rank column gives the rank restriction $k$ for these 
    four layers. The final softmax layer is small, so we do not compress updates 
    to it.
  }
\label{table:params}
\end{small}
\end{table}

We define medium and high low-rank/sampling parameterizations that
result in the same compression rates for both approaches, as given in
Table~\ref{table:params}.  The left and center columns of
Figure~\ref{fig:results} present non-quantized results for test-set
accuracy, both as a function of the number of rounds of the algorithm,
and the total number of megabytes uploaded. For all experiments,
learning rates were tuned using a multiplicative grid of resolution
$\sqrt{2}$ centered at 0.15; we plot results for the learning rate
with the best median accuracy over rounds 400 -- 800. We used a
multiplicative learning-rate decay of 0.988, which we selected by tuning only for
the baseline algorithm.

\newlength{\plotwidth}
\setlength{\plotwidth}{0.321\linewidth}
\begin{figure}[h]
  \begin{center}
  \includegraphics[width=\plotwidth]{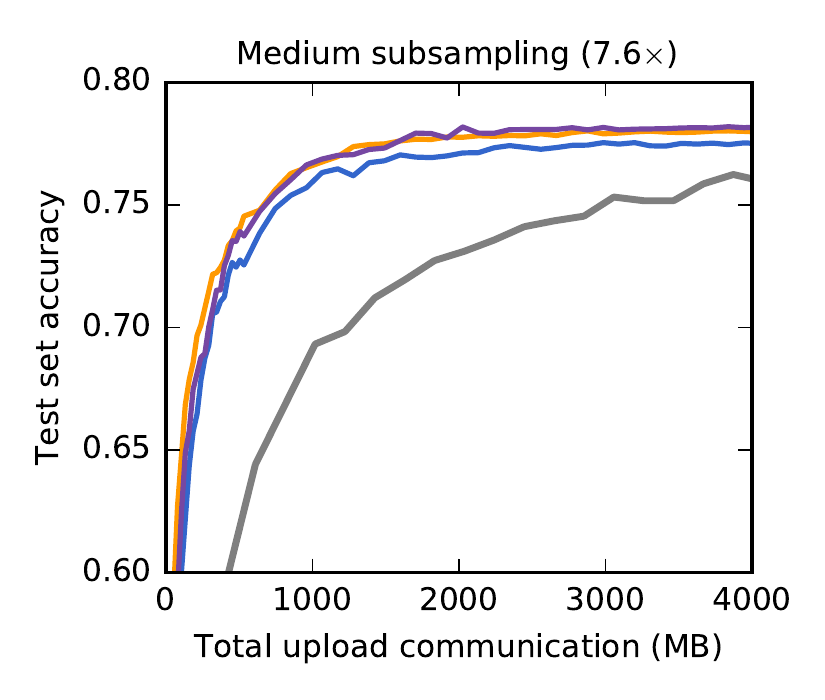}
  \includegraphics[width=\plotwidth]{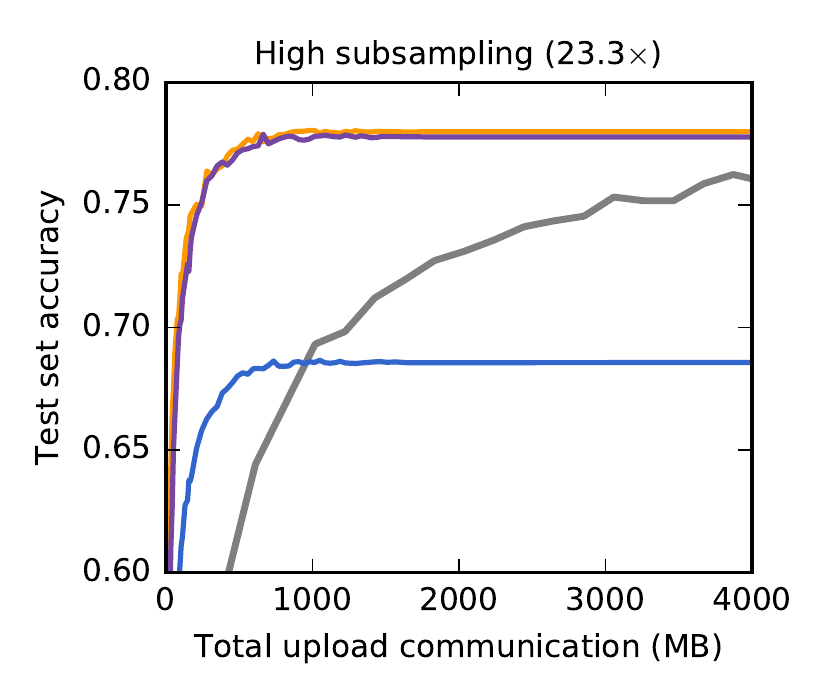}
  \includegraphics[width=\plotwidth]{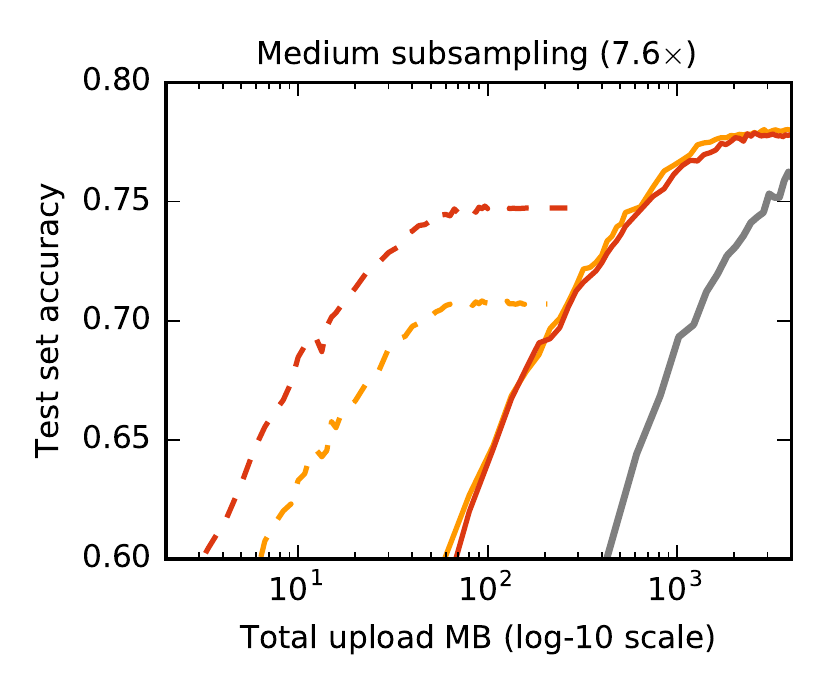}

  \includegraphics[width=\plotwidth]{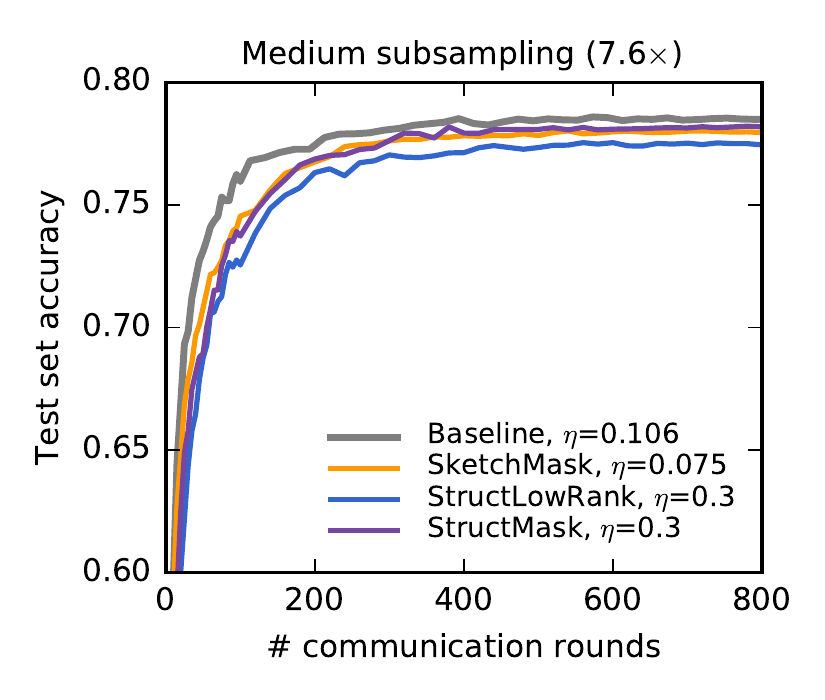}
  \includegraphics[width=\plotwidth]{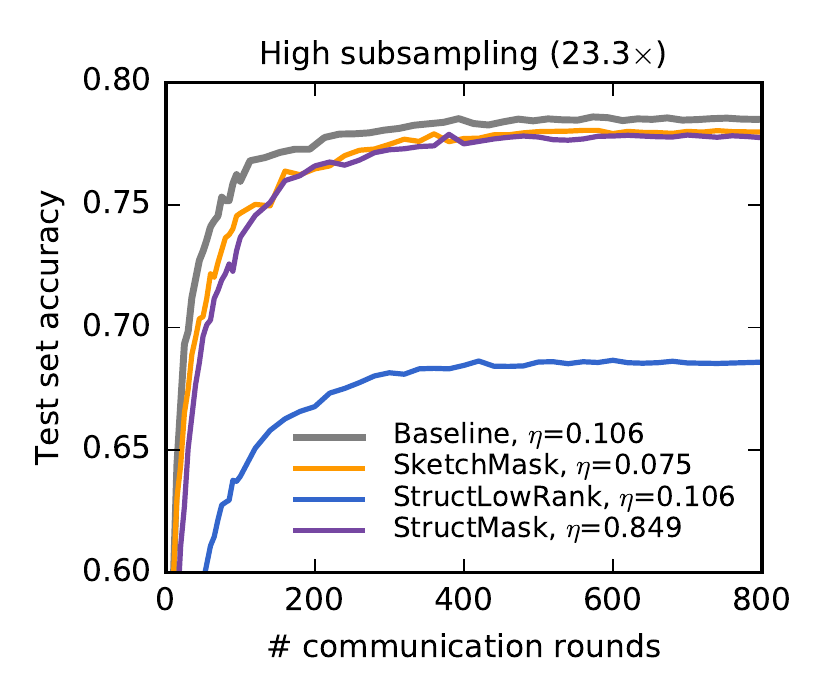}
  \includegraphics[width=\plotwidth]{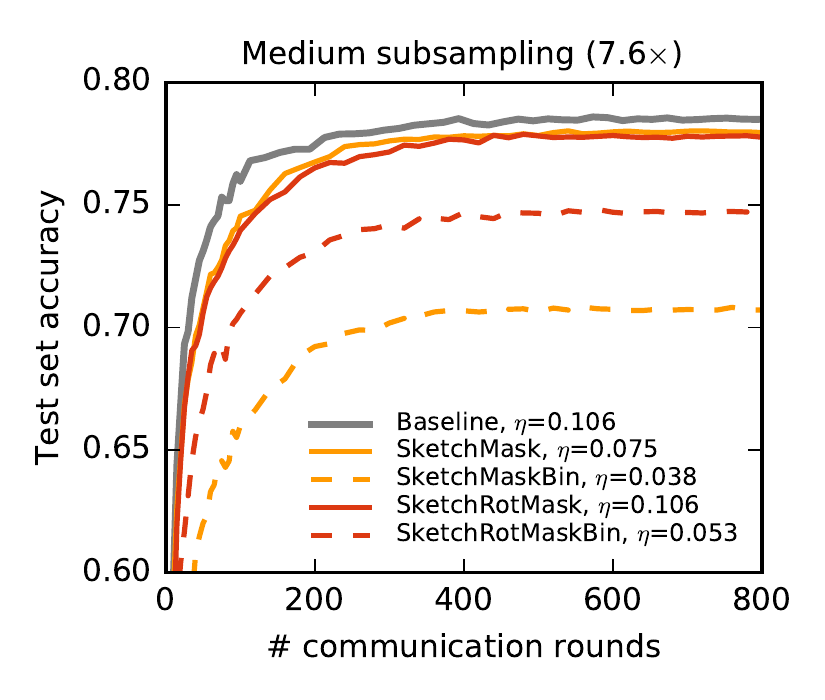}

  \end{center}
  \captionof{figure}{\footnotesize Non-quantized results (left and
    middle columns), and results including binary quantization
    (dashed lines \SketchRotMaskBin and \SketchMaskBin, right column). Note the
    $x$-axis is on a log scale for top-right plot. We achieve over
    70\% accuracy with fewer than 100MB of
    communication.
  }
\label{fig:results}
\end{figure}

For medium subsampling, all three approaches provide a dramatic
improvement in test set accuracy after a fixed amount of bandwidth
usage; the lower row of plots shows little loss in accuracy as a
function of the number of update rounds.  The exception is that the
\StructLowRank approach performs poorly for the high subsampling
parameters. This may suggest that requiring a low-rank update
structure for the convolution layers works poorly. Also, perhaps
surprisingly, we see no advantage for \StructMask, which optimizes 
for a random sparse set of coefficients, as compared to
\SketchMask, which chooses a sparse set of parameters to update 
\emph{after} a full update is learned.

The right two plots in Figure~\ref{fig:results} give results for 
\SketchMask and \SketchRotMask, with and without binary quantization;
we consider only the medium subsampling regime which is
representative. We observe that (as expected) introducing the random
rotation without quantization has essentially no effect (solid red and
orange lines). However, binary quantization dramatically decreases the
total communication cost, and further introducing the random rotation
significantly speeds convergence, and also allows us to converge
to a higher level of accuracy. We are able to learn a reasonable model
(70\% accuracy) in only $\sim$100MB of communication, two orders of
magnitude less than the baseline.

\section{Additional Proofs}
\label{sec:app:alternative_protocol}

In this section we provide proofs of Lemmas~\ref{lem:unbiasedness_2} and \ref{lem:MSE_2}, describing properties of the encoding protocol $\alpha$ defined in \eqref{eq:randomized_protocol_2}. For completeness, we also repeat the statements.

\begin{lemma}[Unbiasedness]
The encoder $\alpha$ defined in \eqref{eq:randomized_protocol} is unbiased. That is,  $\EE{\alpha}{\alpha(X_i)} = X_i$ for all $i$. As a result, $Y$ is an unbiased estimate of the true average: $\EE{\alpha}{Y} = X$.
\end{lemma}

\begin{proof}
Since $Y(j) = \frac{1}{n}\sum_{i=1}^n Y_{i}(j)$ and $X(j) = \frac{1}{n}\sum_{i=1}^n X_{i}(j)$, it suffices to show that $\EE{\alpha}{Y_i(j)}=X_i(j)$:
\begin{align*}
\EE{\alpha}{Y_i(j)} &= \frac{1}{|\sigma_k(d)|} \sum_{\sigma \in \sigma_k(d)} \left[ 1_{(j \in \sigma)} \left( \frac{d X_{i}(j)}{k} - \frac{d-k}{k} \mu_i \right) + 1_{(j \not\in \sigma)} \mu_i \right] \\
&= \binom{d}{k}^{-1} \left[ \binom{d-1}{k-1} \left( \frac{d X_{i}(j)}{k} - \frac{d-k}{k} \mu_i \right) + \binom{d-1}{k} \mu_i \right] \\
&= \binom{d}{k}^{-1} \left[ \binom{d-1}{k-1} \frac{d}{k} X_{i}(j) + \left( \binom{d-1}{k} - \binom{d-1}{k-1} \frac{d-k}{k} \right) \mu_i \right] \\
&= X_i(j)
\end{align*}
and the claim is proved.
\end{proof}

\begin{lemma}[Mean Squared Error]
Let $\alpha = \alpha(k)$ be encoder defined as in \eqref{eq:randomized_protocol_2}. Then
\begin{equation*}
MSE_{\alpha}(X_1, \dots, X_n) = \frac{1}{n^2} \sum_{i=1}^n \sum_{j=1}^d \frac{d-k}{k} \left( X_i(j) - \mu_i \right)^2.
\end{equation*}
\end{lemma}

\begin{proof} 
Using Lemma~\ref{lem:general_MSE}, we have
\begin{eqnarray}
MSE_{\alpha}(X_1,\dots,X_n) &=& \frac{1}{n^2} \sum_{i=1}^n \EE{\alpha}{\left\|Y_i - X_i \right\|^2} \notag \\
&=&\frac{1}{n^2} \sum_{i=1}^n \EE{\alpha}{\sum_{j=1}^d (Y_i(j) - X_i(j) )^2} \notag \\
&=&\frac{1}{n^2} \sum_{i=1}^n \sum_{j=1}^d \EE{\alpha}{ (Y_i(j) - X_i(j) )^2}.
\label{eq:6867sgs7_v2}
\end{eqnarray}
Further, 
\begin{align*}
\EE{\alpha}{ (Y_i(j) - X_i(j) )^2} &= 
\binom{d}{k}^{-1} \sum_{\sigma \in \sigma_k(d)} \left[ 1_{(j \in \sigma)} \left( \frac{d X_{i}(j)}{k} - \frac{d-k}{k} \mu_i - X_i(j) \right)^2 + 1_{(j \not\in \sigma)} \left( \mu_i - X_i(j) \right)^2 \right] \\
&= \binom{d}{k}^{-1} \left[ \binom{d-1}{k-1} \frac{(d-k)^2}{k^2} \left( X_i(j) - \mu_i \right)^2 + \binom{d-1}{k} \left( \mu_i - X_i(j) \right)^2 \right] \\
&= \frac{d-k}{k} \left( X_i(j) - \mu_i \right)^2.
\end{align*}
It suffices to substitute the above into \eqref{eq:6867sgs7_v2}.
\end{proof}

\part{Conclusion}

\chapter{Conclusion and Future Challenges}

In this thesis, we addressed the problem of minimizing a finite average of functions:
\begin{equation*}
\min_{w \in \R^d} \frac1n \sum_{i=1}^n f_i(w).
\end{equation*}

In Chapters \ref{ch:s2gd} and \ref{ch:s2cd}, we formulated two stochastic algorithms with a variance reduction property for solving the above optimization problem. By using stochastic gradients in a particular way, relying on computation of the entire gradient occasionally, we are able to achieve progressive reduction of variance of the stochastic estimates of the gradients. We have demonstrated that these approaches vastly outperform traditional methods at the optimization task.

In Chapters~\ref{ch:ms2gd} and \ref{ch:cocoa}, we described methods that can utilize parallel and distributed computing architectures. The \cocoa framework for distributed optimization addresses primarily communication efficiency --- a common problem of many distributed optimization algorithms. In particular, Section~\ref{sec:cocoa:problem} introduces a conceptual framework in which one can think of the overall efficiency of distributed optimization algorithms, and explains the usefulness of decomposing the choice of local optimization procedure from design of general local subproblems closely related to the overall optimization objective. The \cocoa framework is the first systematic realization of this approach.

In Chapters~\ref{ch:feopt} and \ref{ch:mean}, we introduced the concept of Federated Optimization, which goes beyond what is usually addressed in the area of distributed optimization, by changing the often implicit assumptions on distribution of data among different computing nodes. A particular motivating example arises in the case of user-generated data. Instead of collecting the data and storing them in datacenters, one could keep the data in users' possession, and run optimization algorithms on this massively distributed collection of user devices. We have shown that this idea is conceptually feasible, and has since been experimentally deployed by Google in Android Keyboard \cite{FederatedBlogpost}.

In the following, we try to highlight some of the important challenges that are remaining open in the field. 

The area of single-machine optimization algorithms for finite average of functions recently saw numerous contributions, such as \cite{SDCA, SAGjournal2013, SVRG, S2GD, saga, Katyusha, nguyen2017sarah}. Two major questions spanning all of these works concern the adaptability to strong convexity which is not explicitly known, and the use of stochastic higher-order information.

In general, methods relying on duality, such as SDCA \cite{SDCA}, are easier to use in practice, compared to their primal-only alternatives. This is because for a number of problems structures, SDCA comes without the need to tune any hyperparameters. However, the construction of the dual relies on explicit knowledge of the strong convexity parameter of the objective, or a lower bound of it, which then drives the convergence speed of the algorithm, both in theory and in practice. The explicit knowledge usually comes from the use of regularizer. Nevertheless, the true strong convexity parameter can be larger, caused by for instance structure in the data or can be larger locally near optimum. In contrast, the primal-only algorithms in general adapt to whatever is the true strong convexity, without the having to explicitly know what it is.

These benefits cause disagreements in the community over which algorithms are more useful in practice. It is likely that primal methods can be enhanced with adaptive techniques, removing the need of tuning hyperparameters, and the dual methods can be made adaptive to strong convexity without its explicit knowledge. Both of these ideas were to some extent addressed, but none of them provide a satisfactory answer. In particular, the authors of \cite{SAGjournal2013} provide a heuristic stepsize which `usually works' for their method, but to the best of our knowledge, there is no such method supported by theory. Recent contribution in \cite{wang2017exploiting} shows that an existing primal-dual method can be made adaptive to the true strong convexity from data, if explicitly known, or can be estimated during the run of the algorithm. While this presents a step forward, it falls short of the simplicity of primal methods in handling this issue.

On another front, it is expected that the use of higher-order information will improve the performance of existing methods on this task. The challenge is to do so utilizing stochastic information arising from the structure of the problem, while maintaining computational stability and efficiency. There have been several attempts to do this recently, see \cite[Section 3.4]{bottou2016optimization} for a detailed overview. However, all of them either fail to be computationally efficient and thus inferior to the existing methods, or work well only under restrictive assumptions. It is widely expected that progress in this regard will have significant influence.

An expected theoretical advance in communication-efficient distributed algorithms is the acceleration of the \cocoap framework in the sense of \cite{Nesterov-1983}. It is commonly agreed that this is possible, but it has not been successfully done yet. In terms of practice, scalability and robustness are the major issues going forward. All of the existing methods to some extent either don't scale to large number of nodes, or require assumptions about the way data is distributed across individual nodes. In particular, many method assume that each node has access to data drawn iid from the same distribution.

Assumptions on the distribution of data are problematic even in the datacenter setting, where the data are commonly partitioned `as is', and reshuffling data to match the assumptions is either infeasible of very impractical. The data can be naturally clustered based on its geographical origin and/or time at which it was collected. Methods that can converge for any distribution of the data, such as \cocoap, are naturally very useful in many practical applications. Nevertheless, the performance of \cocoap was observed to degrade when scaled to large number of nodes such as in the context of \FedOpt, see for instance Figure~\ref{fig:ex_final_001}.

An interesting question bridging traditional distributed optimization with \FedOpt, is whether the latter could serve as useful computational model for the former. Scaling distributed training of deep neural networks has been particularly problematic, generating large number of works in recent years. A recent work \cite{chen2016revisiting} revisits the conventional belief that synchronous methods are inferior to their asynchronous variants, showing that a synchronous alternative with backups can be superior to both. In a sense this can be thought of as a negative result for the optimization and systems community, showing the lack of sufficient robustness of methods used in practice. Nevertheless, the task of training large modes very fast remain of significant interest, see for instance \cite{goyal2017accurate}, in which the authors use up to $256$ GPUs. This is considered to be a huge number, yet nowhere near the scale of convex problems in click-through rate prediction systems \cite{mcmahan2013ad}.

The concept of \FedOpt offers a systematic way to decompose the necessary computation into more independent blocks, alleviating most of the issues related to communication necessary between individual nodes. As the concept seems stable enough to scale to support training of recurrent neural networks on phones \cite{FederatedBlogpost}, it could support scaling parallelism of training deep neural networks to many more computing nodes in a stable way. To the best of our knowledge, none experiments in this direction has been done yet.

Deployment of \FedOpt in general opens up many new research questions. For instance, if the goal of the overall system is to make sure servers do not get any data about individual users, a natural question to ask is whether server can reconstruct what data could have caused the observed communication patterns. In this regard, recent work proposed a cryptographic protocol for secure aggregation, ensuring that the server can only see an aggregate function of individual updates, such as an average, across many participating users \cite{bonawitz2017practical}. 

While the secure aggregation protocol provides an intuitive and practical obstacle for recovering the data of a participating user, it does not provide any formal guarantees. Another option is to provide quantitative guarantees for differential privacy \cite{dwork14book}. Adding a carefully designed noise to each user's update would likely ensure the final trained model is differentially private. While this have not been done, it is an important problem to address.

As we remarked in Section~\ref{sec:meanFL}, communication is a likely candidate to be a bottleneck in practice. We proposed few techniques to reduce the amount of bits each user has to upload. However, it is not as straightforward to make all of the ideas compatible with the above secure aggregation protocol, or noise insertion commonly used to obtain differential privacy. As an example, the secure aggregation protocol expects a vector of a particular structure where on which addition can be performed directly. This structure can be obtained by naive update compression mechanisms, but not using data adaptive ones, which naturally give better performance.

Finally, if and when \FedOpt becomes commonly used tool in machine learning practice, it will require innovation in tools for machine learning engineers. Many of the common blocks of machine learning workflow will become unavailable. For instance, it will be harder to rapidly experiment with novel network architectures if they are not trained only in a datacenter, increasing the need to get more theoretical understanding into what makes training of neural networks hard or easy, and what makes them generalize well. Further, any concerns about system stability are mostly irrelevant for machine learning engineers, as it is not an issue if a system does not train or crashes in datacenter. However, crashing every phone an experiment runs on would be a major issue. These examples and many more will drive thorough rethinking of how machine learning tools are used to ultimately design products.

\bibliographystyle{plain}
\bibliography{references}

\end{document}